\documentclass[12pt]{report}

\usepackage[T1]{fontenc}
\usepackage{lmodern}
\usepackage[utf8]{inputenc}  % 如果你用 UTF-8 编码
\usepackage[numbers]{natbib}
\usepackage{thesis}
\usepackage{enumitem}
\setlist[enumerate,1]{label={[\arabic*]}, leftmargin=*, align=left}
\usepackage{amsthm}
\usepackage{graphicx}
\usepackage{subfigure}
\usepackage[ruled]{algorithm2e}
\theoremstyle{plain}
\newtheorem{theorem}{Theorem}[section]
\newtheorem{lemma}[theorem]{Lemma}
\usepackage{array}
\def \pnorm {B}
\usepackage{amsmath}
\usepackage{amsthm}
\usepackage{amssymb,bm, bbm}

\usepackage{booktabs} % for professional tables
\newcommand{\bbb}{\mathbf{b}}
\usepackage[colorlinks=true, linkcolor=blue, citecolor=blue, urlcolor=blue]{hyperref}
\usepackage{bm}
\usepackage{array}   
\usepackage{booktabs}

\DeclareMathOperator{\Var}{{\rm Var}}

\def \algbandit {\text{WeightedOFUL}^+}% \toprule, \midrule, \bottomrule
\usepackage{multirow}              % \multirow
\usepackage{diagbox}               % \diagbox
\usepackage[table]{xcolor}         % 支持表格着色
\usepackage[table]{xcolor}  % 确保带 table 选项
\definecolor{LightCyan}{rgb}{0.88,1,1}  % 自定义颜色，淡青色

\usepackage{algorithm}
\usepackage{wrapfig}
\usepackage{algorithmic}
\usepackage{graphicx}
\usepackage{subcaption}  % 提供 subfigure 环境
\usepackage{wrapfig}     % For wraptable and wrapfigure
\usepackage{multirow}    % For \multirow in tables
\usepackage{diagbox}     % For \diagbox (slashed header box)
\usepackage{graphicx}    % For \resizebox

\usepackage{booktabs, tabularx}
\usepackage[labelfont=bf]{caption}
\usepackage{comment}
\usepackage{enumitem}
\usepackage{fancyhdr}
\usepackage{graphicx}
\usepackage{hyperref,cleveref}
\usepackage[right]{lineno}
\usepackage{listings}
\usepackage{multirow}
\usepackage{tikz}
\usepackage{thmtools, thm-restate}
\usepackage{url}
\usepackage{wasysym}  % For permil
\usepackage{xcolor}
\usepackage[binary-units=true]{siunitx}

\definecolor{red}{HTML}{E51400}  %red
\definecolor{blue}{HTML}{0050EF} %cobalt
\definecolor{green}{HTML}{008A00} %emerald
\definecolor{purple}{HTML}{AA00FF} %violet
\definecolor{dark-red}{rgb}{0.4, 0.15, 0.15}
\definecolor{dark-blue}{rgb}{0.15, 0.15, 0.4}
\definecolor{medium-red}{rgb}{0.5, 0, 0}
\definecolor{medium-blue}{rgb}{0, 0, 0.5}
\definecolor{medium-green}{rgb}{0,0.5,0}
\definecolor{medium-yellow}{rgb}{0.5,0.5,0}
\definecolor{light-red}{rgb}{0.7, 0, 0}
\definecolor{light-blue}{rgb}{0, 0, 0.7}
\definecolor{gray}{HTML}{848482}
\definecolor{Gray}{gray}{0.85}
\definecolor{LightGray}{gray}{0.96}

\hypersetup{
  colorlinks,
  linkcolor={black},
  citecolor={dark-blue},
  urlcolor={medium-blue}
}

\usepackage{mathtools}
\usepackage{booktabs}
\usepackage{tikz}
\usepackage[skip=0pt]{caption}
\usepackage[skip=0pt]{subcaption}
\usepackage{threeparttable, array, float, xcolor}

\newcommand{\compilefullversion}{true}%SHOW full version
\ifthenelse{\equal{\compilefullversion}{false}}{%
	\newcommand{\OnlyInFull}[1]{}
	\newcommand{\OnlyInShort}[1]{#1}
}{%
	\newcommand{\OnlyInFull}[1]{#1}%
	\newcommand{\OnlyInShort}[1]{}%
}%
\newcommand{\ts}[1]{}
\def\biggiven{\,\big{|}\,}
\newcommand{\bepsilon}{\boldsymbol{\epsilon}}
\DeclareMathOperator*{\argmin}{argmin}
\DeclareMathOperator*{\argmax}{argmax}

\newcommand{\norm}[1]{\left\lVert#1\right\rVert}

\newcommand{\cS}{\mathcal{S}}
\newcommand{\abs}[1]{\left| #1 \right|}

\newcommand{\bM}{\boldsymbol{M}}

\newcommand{\hB}{\hat{\mathbb{B}}}
\newcommand{\ud}{{\,\mathrm{d}}}
\def\given{{\,|\,}}
\newcommand{\hatpistar}{\hat{\pi}^*}

\newcommand{\E}{\mathbb{E}}

\newcommand{\actions}{\mathcal{A}}
\newcommand{\bSigma}{\bm{\Sigma}}
\newcommand{\BB}{\mathbb{B}}
\newcommand{\la}{\langle}
\newcommand{\ra}{\rangle}
\newcommand{\bA}{\boldsymbol{A}}
\newcommand{\ba}{\boldsymbol{a}}
\newcommand{\bb}{\boldsymbol{b}}

\newcommand{\bI}{\boldsymbol{I}}

\newcommand{\bG}{\boldsymbol{G}}

\newcommand{\gtclusters}{\textit{ground-truth clusters}}
\newcommand{\gtcluster}{\textit{ground-truth cluster}}

\newcommand{\bV}{\boldsymbol{V}}

\newcommand{\bX}{\boldsymbol{X}}

\newcommand{\bx}{\boldsymbol{x}}
\newcommand{\by}{\boldsymbol{y}}

\newcommand{\bmu}{\boldsymbol{\mu}}

\newcommand{\btheta}{\boldsymbol{\theta}}

\newcommand{\bzero}{\boldsymbol{0}}

\newcommand{\cA}{\mathcal{A}}
\newcommand{\cB}{\mathcal{B}}

\newcommand{\cD}{\mathcal{D}}
\newcommand{\cE}{\mathcal{E}}
\newcommand{\cF}{\mathcal{F}}
\newcommand{\cG}{\mathcal{G}}

\newcommand{\cI}{\mathcal{I}}
\newcommand{\cL}{\mathcal{L}}

\newcommand{\cM}{\mathcal{M}}
\newcommand{\cN}{\mathcal{N}}

\newcommand{\cP}{{\mathcal{P}}}

\allowdisplaybreaks[1]

\newtheorem{corollary}{{\bf Corollary}}[chapter]
\newtheorem{definition}{{\bf Definition}}[chapter]
\newtheorem{assumption}{\bf Assumption}[chapter]

\newtheorem{fact}{{\bf Fact}}[chapter]

\newtheorem{proposition}{{\bf Proposition}}[chapter]
\ifx\techreport\undefined
  \newcommand{\onlytech}[1]{\ignorespaces}
  \newcommand{\onlypaper}[1]{#1}
\else
  \newcommand{\onlytech}[1]{#1}
  \newcommand{\onlypaper}[1]{\ignorespaces}
\fi

\crefname{ineq}{inequality}{inequalities}
\creflabelformat{ineq}{#2{\upshape(#1)}#3} 

\RequirePackage{amsmath}
\RequirePackage{amssymb}
\RequirePackage{mathtools}
\ifx\proof\undefined 
\RequirePackage{amsthm}
\fi
\RequirePackage{bm} 
\RequirePackage{url}

\newcommand{\braces}[1]{\left\{#1\right\}}

\newcommand{\ab}{\mathbf{a}}

\newcommand{\xb}{\mathbf{x}}

\newcommand{\Ib}{\mathbf{I}}

\newcommand{\Zb}{\mathbf{Z}}

\newcommand{\Acal}{\mathcal{A}}

\newcommand{\Dcal}{\mathcal{D}}
\newcommand{\Ecal}{\mathcal{E}}
\newcommand{\Fcal}{\mathcal{F}}
\newcommand{\Gcal}{\mathcal{G}}

\newcommand{\Kcal}{\mathcal{K}}

\newcommand{\Mcal}{\mathcal{M}}
\newcommand{\Ncal}{\mathcal{N}}

\newcommand{\Pcal}{\mathcal{P}}

\newcommand{\Scal}{{\mathcal{S}}}

\newcommand{\Xcal}{\mathcal{X}}
\newcommand{\Ycal}{\mathcal{Y}}

\newcommand{\EE}{\mathbb{E}} % Expectation
\newcommand{\VV}{\mathbb{V}} % Variance
\newcommand{\NN}{\mathbb{N}} % Natural numbers
\newcommand{\PP}{\mathbb{P}} % Probability
\newcommand{\RR}{\mathbb{R}} % Real numbers
\newcommand{\var}{\mathrm{VaR}}
\usepackage[table]{xcolor}               % 允许表格单元格上色
\usepackage{array}                       % 支持 \newcolumntype
\definecolor{LightCyan}{rgb}{0.88,1,1}   % 自定义颜色
\newcolumntype{g}{>{\columncolor{LightCyan}}c}  % 定义 g 列类型
\newcommand{\cX}{\mathcal{X}}                          % 或 \usepackage{algcompatible}                        \newcommand{\theHalgorithm}{\arabic{algorithm}}     
\newcommand*{\zero}{{\textbf 0}}

\newcommand{\event}{\mathcal{E}}

\newcommand{\conf}{\mathrm{conf}}

\newcommand{\alglinelabel}{%
  \addtocounter{ALC@line}{-1}% Reduce line counter by 1
  \refstepcounter{ALC@line}% Increment line counter with reference capability
  \label% Regular \label
}

\ifx\BlackBox\undefined
\newcommand{\BlackBox}{\rule{1.5ex}{1.5ex}}  % end of proof
\fi

\ifx\QED\undefined
\def\QED{~\rule[-1pt]{5pt}{5pt}\par\medskip}
\fi

\ifx\proof\undefined
\newenvironment{proof}{\par\noindent{\bf Proof\ }}{\hfill\BlackBox\\[2mm]}
\fi

\ifx\theorem\undefined
\newtheorem{theorem}{Theorem}
\fi
\ifx\example\undefined
\newtheorem{example}{Example}
\fi
\ifx\lemma\undefined
\newtheorem{lemma}{Lemma}
\fi
\ifx\proposition\undefined
\newtheorem{proposition}{Proposition}
\fi
\ifx\remark\undefined
\newtheorem{remark}{Remark}
\fi
\ifx\corollary\undefined
\newtheorem{corollary}{Corollary}
\fi
\ifx\definition\undefined
\newtheorem{definition}{Definition}
\fi
\ifx\conjecture\undefined
\newtheorem{conjecture}{Conjecture}
\fi
\ifx\fact\undefined
\newtheorem{fact}{Fact}
\fi
\ifx\claim\undefined
\newtheorem{claim}{Claim}
\fi
\ifx\assum\undefined
\newtheorem{assum}{Assumption}
\fi
\ifx\exercise\undefined
\newtheorem{exercise}{Exercise}
\fi

\usepackage{prettyref}
\newcommand{\pref}[1]{\prettyref{#1}}

\newcommand{\savehyperref}[2]{\texorpdfstring{\hyperref[#1]{#2}}{#2}}
\newrefformat{eq}{\savehyperref{#1}{\textup{(\ref*{#1})}}}
\newrefformat{eqn}{\savehyperref{#1}{Equation~\ref*{#1}}}
\newrefformat{lem}{\savehyperref{#1}{Lemma~\ref*{#1}}}
\newrefformat{def}{\savehyperref{#1}{Definition~\ref*{#1}}}
\newrefformat{line}{\savehyperref{#1}{line~\ref*{#1}}}
\newrefformat{thm}{\savehyperref{#1}{Theorem~\ref*{#1}}}
\newrefformat{corr}{\savehyperref{#1}{Corollary~\ref*{#1}}}
\newrefformat{sec}{\savehyperref{#1}{Section~\ref*{#1}}}
\newrefformat{ssec}{\savehyperref{#1}{Subsection~\ref*{#1}}}
\newrefformat{app}{\savehyperref{#1}{Appendix~\ref*{#1}}}
\newrefformat{ass}{\savehyperref{#1}{Assumption~\ref*{#1}}}
\newrefformat{ex}{\savehyperref{#1}{Example~\ref*{#1}}}
\newrefformat{fig}{\savehyperref{#1}{Figure~\ref*{#1}}}
\newrefformat{alg}{\savehyperref{#1}{Algorithm~\ref*{#1}}}
\newrefformat{rem}{\savehyperref{#1}{Remark~\ref*{#1}}}
\newrefformat{conj}{\savehyperref{#1}{Conjecture~\ref*{#1}}}
\newrefformat{prop}{\savehyperref{#1}{Proposition~\ref*{#1}}}
\newrefformat{proto}{\savehyperref{#1}{Protocol~\ref*{#1}}}
\newrefformat{prob}{\savehyperref{#1}{Problem~\ref*{#1}}}
\newrefformat{claim}{\savehyperref{#1}{Claim~\ref*{#1}}}
\newrefformat{fact}{\savehyperref{#1}{Fact~\ref*{#1}}}

\newcommand{\cJ}{\mathcal{J}}
\newcommand{\cK}{\mathcal{K}}

\usepackage[numbers]{natbib}
\usepackage{lipsum}


\begin{document}

% set double line spacing
\linespread{1.2}

\thesistitle{Towards More Efficient, Robust, Instance-adaptive, and Generalizable Sequential Decision Making}
\authorname{WANG, Zhiyong}
\degree{Doctor of Philosophy}
\programme{Computer Science and Engineering}
\supervisor{Professor LUI Chi Shing John}
\submitdate{April 2025}
\coverpage

\newtheorem{thm}{\hspace{- 0.18 in} {\bf Theorem}}
\newtheorem{cor}[thm]{Corollary}
\newtheorem{prop}[thm]{Proposition}
\newtheorem{lem}[thm]{\hspace{- 0.18 in} {\bf Lemma}}
\newtheorem{conj}[thm]{Conjecture}
\newtheorem{problem}[thm]{\hspace{- 0.18 in} {\bf Problem}}

\theoremstyle{definition}
\newtheorem{property}{{\bf \em Property}}
\newtheorem{defn}[thm]{Definition}
\newtheorem{defns}[thm]{Definitions}
\newtheorem{con}[thm]{Construction}
\newtheorem{exmp}{Example}
\newtheorem{notn}[thm]{Notation}
\newtheorem{notns}[thm]{Notations}
\newtheorem{addm}[thm]{Addendum}
\newtheorem{exer}[thm]{Exercise}

\theoremstyle{remark}
\newtheorem{rem}[thm]{Remark}
\newtheorem{rems}[thm]{Remarks}
\newtheorem{warn}[thm]{Warning}
\newtheorem{sch}[thm]{Cilium}
%\def\done{\hspace*{\fill} \rule{1.8mm}{2.5mm} \\}

% \SetKwInput{KwInit}{Init}
% \SetKwProg{myproc}{Procedure}{:}{}%
\makeatletter
\newcommand{\algrule}{\par\vskip.2\baselineskip{\color{black!30}\hrule}\par\vskip.2\baselineskip}
\makeatother
\newcommand{\CH}{\mathcal{H}}
\newcommand{\header}[1]{\smallskip\noindent\textbf{#1}}

%committee
\newpage
\thispagestyle{empty}
\vspace*{4cm}
\begin{center}
  \Large {\underline{Thesis Assessment Committee}}

  \vskip 1cm
  \large {
    Professor FARNIA Farzan (Chair)\\ \vspace{0.08 in}
    Professor LUI Chi Shing John (Thesis Supervisor)\\ \vspace{0.08 in}
    Professor YU Bei (Committee Member)\\ \vspace{0.08 in}
    Professor LUO Xiapu Daniel (External Examiner)
  }
\end{center}
\vfill
\newpage
\newcommand{\brank}{R_B} \newcommand{\mrank}{R_M}
\newcommand{\olive}{\textsc{Olive}}
\newcommand{\MLE}{\mathrm{MLE}}
\newcommand{\Bayes}{\text{Bayes}}
\newcommand{\masa}[1]{\noindent{\textcolor{purple}{\{{\bf Masa:} \em #1\}}}}
\newcommand{\ming}[1]{\noindent{\textcolor{ProcessBlue}{\{{\bf Ming:} \em #1\}}}}
\newcommand{\newedit}{\color{blue}}
\newcommand{\alg}{\textsc{Refuel}}
\newcommand{\Mid}{\mathrel{\Vert}}
\newcommand{\SEC}{\textsc{SEC}}
\newcommand{\op}{\operatorname}
\newcommand{\DE}{\op{DE}}
\newcommand{\dimRL}{d_{\normalfont\textsf{RL}}}
\newcommand{\distEluDim}{\op{dim}_{\ell_1\normalfont\textsf{DE}}}

% =======================================================

\pagenumbering{roman}

% abstract

        \addcontentsline{toc}{chapter}{Abstract}

	\vspace*{2cm}
	\large \noindent
	Abstract of thesis entitled: \\
	\indent \thesistitle \\
	Submitted by \authorname \\
	for the degree of \degree \\
	at \institution~in \submitdate

	\vskip 1cm \noindent
	The primary goal of my research is to develop provably efficient and practical algorithms for data-driven sequential decision-making under uncertainty. My work focuses on reinforcement learning (RL), multi-armed bandits, and their applications, including recommendation systems, computer networks, video analytics, and large language models (LLMs). Sequential decision-making methods, such as bandits and RL, have demonstrated remarkable success—ranging from outperforming human players in complex games like Atari and Go to advancing robotics, recommendation systems, and fine-tuning LLMs.

Despite these successes, many established algorithms rely on idealized models that can fail under model misspecifications or adversarial perturbations, particularly in settings where accurate prior knowledge of the underlying model class is unavailable or where malicious users operate within dynamic systems. These challenges are pervasive in real-world applications, where robust and adaptive solutions are critical. Furthermore, while worst-case guarantees provide theoretical reliability, they often fail to capture instance-dependent performance, which can lead to more efficient and practical solutions. Another key challenge lies in generalizing to new, unseen environments, a crucial requirement for deploying these methods in dynamic and unpredictable settings. To address these important issues, my research aims to address these limitations by driving the field toward 

\begin{center}
\textbf{\textit{more efficient, robust, instance-adaptive, and generalizable sequential decision-making}}.
\end{center}

Towards this end, I focus on developing more efficient, robust, instance-adaptive, and generalizable for both general reinforcement learning (RL) and bandits.

\noindent\textbf{1. Efficient, Instance-adaptive, and Generalizable Reinforcement Learning:} 
Reinforcement Learning (RL) has achieved significant breakthroughs in various applications, from game playing to autonomous systems. However, two major challenges persist in the field: developing algorithms that provide efficient, instance-adaptive performance guarantees and ensuring that these algorithms can generalize effectively to new, unseen environments. Current state-of-the-art RL methods often rely on worst-case performance analyses, which can be overly conservative and fail to leverage the specific structure of individual problems. Additionally, many RL algorithms struggle with generalization, particularly in offline settings where the agent must perform well in environments that differ from the training data. Addressing these challenges is crucial for advancing RL theory and enabling its application in more complex and dynamic real-world scenarios.

I proved that surprisingly standard model-based RL approaches can achieve horizon-free and variance-dependent regret bounds \cite{wang2024model}, contributing significantly to the RL theory community by identifying the simplest approach for achieving such tight bounds in large-scale online and offline RL problems with general function approximation. Furthermore, I made substantial contributions to the area of zero-shot generalization in offline RL \cite{wang2024towards}. Previous empirical works \cite{mediratta2023generalization} demonstrated that standard offline RL algorithms struggle to generalize to new, unseen environments. I initiated the first theoretical analysis in this area, identifying the causes of these failures and proposing provably efficient offline RL algorithms that address these challenges, which are supported by significant improvements in large-scale experiments over prior methods. Details are as follows.

% \lily{These research themes are a helpful way to organize your research. Could you add a sentence or two about what the biggest (overarching) challenges are / why other work didn't solve these issues / the general aspects of your contributions?}
\begin{itemize}

 \item \textbf{Minimalist Approach to Horizon-Free and Second-Order Bounds} \cite{wang2024model}: We are the first to identify the minimalist algorithms and analyses to achieve horizon-free and instance-dependent (second-order) bounds for both online and offline RL with general function approximations. ``Horizon-free'' implies that our bounds do not depend polynomially on the Markov Decision Process horizon. Our second-order bounds scale with the variances of the policy returns, which can be small when the system is nearly deterministic or the optimal policy has small values. These bounds offer significant tight and instance-dependent theoretical guarantees for efficient RL. \textit{This work was recently selected as a reference in Cornell's CS 6789: Foundations of Reinforcement Learning course.}
 
\item \textbf{Zero-Shot Generalization in Offline RL} \cite{wang2024towards,wang2025provable}: We studied offline RL with zero-shot generalization (ZSG), where the agent accesses an offline dataset from various environments and aims to perform well on unseen test environments without further interaction. We proposed novel frameworks to find near-optimal policies with ZSG, providing both nearly optimal theoretical guarantees (tight upper bounds of suboptimality gaps) and empirical validations (significant outperformance over previous offline RL methods on the real-world Procgen dataset). Our frameworks represent a significant advancement in understanding and enhancing generalization in offline reinforcement learning.
\end{itemize}

\noindent\textbf{2. Efficient, Robust, and Instance-adaptive Multi-armed Bandits:} 
Multi-armed bandits (MAB) are fundamental tools for sequential decision-making under uncertainty, with widespread applications in areas such as recommendation systems, online advertising, and user engagement. However, existing algorithms often struggle in real-world scenarios involving model misspecifications, adversarial corruptions, or dynamic environments, making the development of efficient and robust methods a critical research direction.

I have made significant contributions towards more efficient, robust, and instance-adaptive multi-armed bandit algorithms, particularly for real-world applications like recommendation systems. I designed provably efficient large-scale bandit algorithms that are robust to model misspecifications \cite{wang2024onlinea}, adversarial corruptions \cite{wang2024onlineb}, and preference feedback \cite{wang2025online}, scenarios where classic algorithms often fail. Motivated by real-world applications of bot detection proposed by my collaborators at Adobe Research, I also pioneered the study of online detection of malicious users in dynamic systems \cite{wang2024onlineb}, an important but open research area. Additionally, I have contributed to the field of conversational contextual bandits, which are widely applied to conversational recommendation systems \cite{wang2023efficient, dai2024conversational, dai2024online}. Details are as follows.

\begin{itemize}
\item \textbf{Robust Clustering in Bandits} \cite{wang2024onlinea}: Clustering of bandits (CB) utilizes similarities over user preferences and has shown significant success in large-scale recommender systems. We addressed the limitations of existing CB algorithms that require well-specified linear user models. We developed robust CB algorithms accommodating inaccurate user preference estimations and erroneous clustering due to model misspecifications. Our algorithms achieve tight regret upper bounds matching the lower bounds up to logarithmic factors and have significant empirical improvements in real-world datasets for recommendation systems.
\item \textbf{Online Malicious User Detection} \cite{wang2024onlineb}: Recognizing challenges such as click fraud, fake reviews, and bot detection, we introduced a novel online learning problem named LOCUD. Our work is the first to address the dual objectives of performing online detection of malicious users and minimizing regret by learning and leveraging unknown user relations inferred from disrupted behaviors. We proposed general frameworks that demonstrate both strong theoretical guarantees—nearly optimal regret bounds and tight detection accuracy—and robust experimental performance, achieving high rewards in recommender systems and online detection accuracy comparable to state-of-the-art offline deep learning-based methods.

\item \textbf{Conversational Contextual Bandits} \cite{wang2023efficient, dai2024conversational, dai2024online}: We explored conversational contextual bandits, which accelerate learning in recommendation systems by eliciting user preferences through occasional queries for explicit feedback on key terms. We proposed the ConLinUCB framework, achieving better incorporation of arm-level and key-term-level feedback. Our algorithms, ConLinUCB-BS and ConLinUCB-MCR, achieve state-of-the-art performance both theoretically and empirically (up to 54\% improvement in learning accuracy and up to 72\% improvement in computational efficiency over previous SOTA methods) \cite{wang2023efficient}. We also studied various different settings and proposed corresponding robust algorithms for conversational bandits \cite{dai2024conversational,dai2024online}. 

\item \textbf{Variance-Adaptive Regret in Non-Stationary Bandits} \cite{wang2024variance}: We investigated non-stationary stochastic linear bandits with evolving reward distributions. We proposed algorithms that utilize the variance of the reward distribution, introducing Restarted Weighted$\text{OFUL}^+$ and Restarted $\text{SAVE}^+$, which achieve variance-dependent bounds. When the total variance $V_K$ is smaller than the total round $K$, our algorithms outperform previous state-of-the-art results.

\item\textbf{Clustering of Bandits with Preference Feedback \cite{wang2025online}}: We introduce the first "clustering of dueling bandit algorithms" to enable collaborative decision-making based on preference feedback. We propose two novel algorithms: (1) Clustering of Linear Dueling Bandits (COLDB) which models the user reward functions as linear functions of the context vectors, and (2) Clustering of Neural Dueling Bandits (CONDB) which uses a neural network to model complex, non-linear user reward functions. Both algorithms are supported by rigorous theoretical analyses, demonstrating that user collaboration leads to improved regret bounds. Extensive empirical evaluations on synthetic and real-world datasets further validate the effectiveness of our methods.

\item \textbf{Other Collaborative Works} \cite{dai2024quantifying, chu2024online, liu2024combinatorial, yang2024federated, zuo2023adversarial,xie2025cascading,huang2025federated,sun2025large,kong2025meta}: I also contributed to projects on applications of bandits in computer networks for adaptive congestion control \cite{dai2024quantifying}, theory of combinatorial bandits \cite{liu2024combinatorial}, federated bandits for recommendation systems \cite{yang2024federated}, and the safety study of adversarial attacks on bandits \cite{zuo2023adversarial}, broadening the impact of online learning methods in these domains.
\end{itemize}

	\newpage

% \input{abstract_in_chinese_2}

% =======================================================

% acknowledgement
%
        \chapter*{Acknowledgement}

        \addcontentsline{toc}{chapter}{Acknowledgement}
        First and foremost, I would like to express my deepest gratitude to my advisor, \textbf{Professor John C.S. Lui}, for his invaluable guidance and unwavering support throughout my Ph.D. journey. John has always respected and encouraged my decisions, especially during the most challenging moments over the past four years—when I struggled to find research ideas, encountered repeated setbacks, or dealt with paper rejections. His patience and trust gave me the strength and confidence to persevere. To me, he is not only a mentor, but also a true friend. He taught me how to think critically, how to conduct meaningful research, and, more importantly, how to face adversity with resilience. I was especially touched during the final year of my Ph.D., when I was applying for postdoctoral positions: John kindly offered to be CC’d on every inquiry email I sent and followed up with personalized recommendation letters—over a hundred in total. His dedication and generosity moved me deeply. Words cannot fully express my appreciation—thank you, John, for everything.

I am also deeply grateful to \textbf{Professor Shuai Li} from Shanghai Jiao Tong University, who guided me during the early stages of my Ph.D. journey. She generously mentored me through my first three research projects, from which I learned a great deal. I have been deeply impressed by her dedication and strong motivation. She instilled in me the discipline and rigor essential for beginning a research career and helped me build a solid foundation in the field of bandits. Her encouragement and support during challenging times in my research meant a great deal to me. Thank you, Prof. Li, for your mentorship, patience, and generosity.

My sincere thanks go to \textbf{Professor Wen Sun} from Cornell University, who graciously hosted me during my research visit. I had the pleasure of meeting him at NeurIPS 2023, and he generously welcomed me to join his group. During my time at Cornell, I benefited tremendously from his expertise in reinforcement learning. His weekly discussions and insightful feedback were instrumental to my development as a researcher. Thank you, Wen, for the enriching experience and your kind mentorship.

I would also like to extend my heartfelt thanks to \textbf{Professor Dongruo Zhou} from Indiana University, a close collaborator and dear mentor. We worked together on several projects, and I was continually inspired by his brilliance, dedication, and insightful thinking. More than a collaborator, Dongruo has been like an older brother to me—offering not only technical guidance, but also thoughtful career advice and warm encouragement. Thank you, Dongruo, for your friendship, support, and the many lessons you’ve taught me.

I am thankful to \textbf{Dr. Wei Chen} and \textbf{Dr. Siwei Wang} from Microsoft Research Asia, who mentored me during my internship. Their guidance broadened my research vision and taught me how to identify and tackle impactful problems. Thank you for your mentorship and for sharing your deep insights and experience.

I am also grateful to my other close collaborators: \textbf{Prof. Zhongxiang Dai} (CUHK Shenzhen), \textbf{Dr. Tong Yu} (Adobe Research), and \textbf{Prof. Jiancheng Ye} (Macau University of Science and Technology). Working with them has been both intellectually rewarding and personally enjoyable. Through our collaborations, I gained not only valuable research experience but also fresh perspectives and renewed motivation. I sincerely appreciate their support, generosity, and enthusiasm for research.

I would also like to thank my thesis committee members—\textbf{Prof. Farzan Farnia}, \textbf{Prof. Bei Yu}, and \textbf{Prof. Xiapu Daniel Luo}—for their time, insightful feedback, and constructive suggestions. I am truly grateful for their contributions to my academic development.

This journey would not have been the same without the camaraderie and companionship of my fellow lab members at CUHK and Cornell: \textbf{Xutong Liu}, \textbf{Zhuohua Li}, \textbf{Xuchuang Wang}, \textbf{Jincheng Wang}, \textbf{Shiyuan Zheng}, \textbf{Maoli Liu}, \textbf{Chengchang Liu}, \textbf{Xudong Liu}, \textbf{Xiangxiang Dai}, \textbf{Bin Luo}, \textbf{Zeyu Zhang}, \textbf{Ziyi Han}, \textbf{Dian Shen}, \textbf{Fang Kong}, \textbf{Bo Sun}, \textbf{Yuwen Huang}, \textbf{Qiang Zhao}, \textbf{Jianhao He}, \textbf{Runzhe Wu}, \textbf{Yiyi Zhang}, \textbf{Nico Espinosa Dice}, \textbf{Zhaolin Gao}, \textbf{Jinyan Su}, \textbf{Yiding Chen}, \textbf{Rebecca Liu}, \textbf{Yann Hicke}, and \textbf{Owen Oertell}. Thank you all for the great memories and support throughout these years. Special thanks to \textbf{Xutong} for his generous help during the early stages of my research, to \textbf{Xiangxiang} for our enjoyable and productive collaborations, and to \textbf{Runzhe} for accompanying me to dinners and helping me move into my new apartment during my stay at Cornell.

To my beloved family—thank you for your unconditional love and unwavering support. I am especially grateful to my parents, who have always been my strongest pillars. Your sacrifices, belief in me, and endless encouragement have carried me through every step of this journey. I owe everything to you.

	\newpage

% =======================================================

% dedication
%
	\vspace*{1cm}
	\vfill
	\begin{center}
	\vspace{-6cm}
\emph{
  This thesis is dedicated to my beloved parents and my beloved girl.
}

	\end{center}
	\vfill
	\newpage

% =======================================================

% preamble sections

\tableofcontents
\listoffigures
\listoftables
\chapter*{List of Publications}
\markboth{\MakeUppercase{List of Publications}}{\MakeUppercase{List of Publications}}
\noindent\textbf{Papers in Submission (* denotes equal contribution)}
\begin{enumerate}

  \item \textbf{Large Language Model-Enhanced Multi-Armed Bandits,}\\
  Jiahang Sun*, \underline{Zhiyong Wang}*, Runhan Yang*, Chenjun Xiao, John C.S. Lui, Zhongxiang Dai,\\
  Accepted in ICLR 2025 Workshop on Reasoning and Planning for Large Language Models\\
  In submission.

  \item \textbf{Meta-Prompt Optimization for LLM-Based Sequential Decision Making,}\\
  Mingze Kong, \underline{Zhiyong Wang}, Yao Shu, Zhongxiang Dai,\\
  Accepted in ICLR 2025 Workshop on Reasoning and Planning for Large Language Models\\
  In submission.

  \item \textbf{Federated Linear Dueling Bandits,}\\
  Xuhan Huang, Yan Hu, Zhiyan Li, \underline{Zhiyong Wang}, Benyou Wang, Zhongxiang Dai,\\
  In submission.

    \item \textbf{Federated Linear Dueling Bandits,}\\
  Xuhan Huang, Yan Hu, Zhiyan Li, \underline{Zhiyong Wang}, Benyou Wang, Zhongxiang Dai,\\
  In submission.

        \item \textbf{Cascading Bandits Robust to Adversarial Corruptions,}\\
Jize Xie, Cheng Chen, \underline{Zhiyong Wang}, Shuai Li,\\
In submission.

\end{enumerate}

\noindent\textbf{PUBLICATIONS (* denotes equal contribution, \# denotes corresponding author)}
\begin{enumerate}
 \item \textbf{Towards Zero-Shot Generalization in Offline Reinforcement Learning,}\\
  \underline{Zhiyong Wang}, Chen Yang, John C.S. Lui, Dongruo Zhou,\\
  Adaptive Learning in Complex Environments TTIC Workshop, 2024.\\
  ICML 2024 Workshop: Aligning Reinforcement Learning Experimentalists and Theorists.\\
  TTIC Summer Workshop 2024: Data-Driven Decision Processes: From Theory to Practice.\\
  Accepted in the Forty-Second International Conference on Machine Learning (ICML), 2025.

  \item \textbf{Online Clustering of Dueling Bandits,}\\
  \underline{Zhiyong Wang}, Jiahang Sun, Mingze Kong, Jize Xie, Qinghua Hu, John C.S. Lui, Zhongxiang Dai,\\
  Accepted in the Forty-Second International Conference on Machine Learning (ICML), 2025.

  \item \textbf{In-Context Federated Learning: A Collaborative Approach for Iterative Answer Refinement,}\\
  Ruhan Wang*, \underline{Zhiyong Wang}*, Chengkai Huang*, Rui Wang, Tong Yu, Lina Yao, John C.S. Lui, Dongruo Zhou,\\
  Accepted in the Forty-Second International Conference on Machine Learning (ICML), 2025.
  \item \textbf{Model-based RL as a Minimalist Approach to Horizon-Free and Second-Order Bounds,}\\
  \underline{Zhiyong Wang}, Dongruo Zhou, John C.S. Lui, Wen Sun.\\
  \textbf{Selected as a course reference paper for CS 6789: Foundations of Reinforcement Learning at Cornell University.}\\
  Accepted in the Thirteenth International Conference on Learning Representations (ICLR), 2025.

  \item \textbf{Variance-Dependent Regret Bounds for Non-stationary Linear Bandits,}\\
  \underline{Zhiyong Wang}, Jize Xie, Yi Chen, John C.S. Lui, Dongruo Zhou,\\
  Adaptive Learning in Complex Environments TTIC Workshop, 2024.\\
  ICML 2024 Workshop: Foundations of Reinforcement Learning and Control -- Connections and Perspectives.\\
  Presented at the 25th International Symposium on Mathematical Programming (ISMP), 2024.\\
  Accepted in the 28th International Conference on Artificial Intelligence and Statistics (AISTATS), 2025.

  \item \textbf{Online Learning and Detecting Corrupted Users for Conversational Recommendation Systems,}\\
  Xiangxiang Dai*, \underline{Zhiyong Wang*\#}, Jize Xie, Tong Yu, John C.S. Lui,\\
  Accepted in the IEEE Transactions on Knowledge and Data Engineering (TKDE), 2024.

  \item \textbf{Conversational Recommendation with Online Learning and Clustering on Misspecified Users,}\\
  Xiangxiang Dai*, \underline{Zhiyong Wang*\#}, Jize Xie, Xutong Liu, John C.S. Lui,\\
  Accepted in the IEEE Transactions on Knowledge and Data Engineering (TKDE), 2024.

  \item \textbf{Combinatorial Multivariant Multi-Armed Bandits with Applications to Episodic Reinforcement Learning and Beyond,}\\
  Xutong Liu, Siwei Wang, Jinhang Zuo, Han Zhong, Xuchuang Wang, \underline{Zhiyong Wang}, Shuai Li, Mohammad Hajiesmaili, John C.S. Lui, Wei Chen,\\
  Accepted in the Forty-first International Conference on Machine Learning (ICML), 2024.

  \item \textbf{Quantifying the Merits of Network-Assist Online Learning in Optimizing Network Protocols,}\\
  Xiangxiang Dai*, \underline{Zhiyong Wang*}, Jiancheng Ye, John C.S. Lui,\\
  Accepted in the IEEE/ACM International Symposium on Quality of Service (IWQoS), 2024.

  \item \textbf{Online Optimal Service Caching for Multi-Access Edge Computing: A Constrained Multi-Armed Bandit Optimization Approach,}\\
  Weibo Chu, Xiaoyan Zhang, Xinming Jia, John C.S. Lui, \underline{Zhiyong Wang},\\
  Accepted in the Computer Networks, 2024.

  \item \textbf{Federated Contextual Cascading Bandits with Asynchronous Communication and Heterogeneous Users,}\\
  Hantao Yang, Xutong Liu, \underline{Zhiyong Wang}, Hong Xie, John C.S. Lui, Defu Lian, Enhong Chen,\\
  Accepted in the AAAI Conference on Artificial Intelligence (AAAI), 2024.

  \item \textbf{Learning Context-Aware Probabilistic Maximum Coverage Bandits: A Variance-Adaptive Approach,}\\
  Xutong Liu, Jinhang Zuo, Junkai Wang, \underline{Zhiyong Wang}, Yuedong Xu, John C.S. Lui,\\
  IEEE International Conference on Computer Communications (INFOCOM), 2024.

  \item \textbf{Online Clustering of Bandits with Misspecified User Models,}\\
  \underline{Zhiyong Wang}, Jize Xie, Xutong Liu, Shuai Li, John C.S. Lui,\\
  Thirty-seventh Conference on Neural Information Processing Systems (NeurIPS), 2023.

  \item \textbf{Online Corrupted User Detection and Regret Minimization,}\\
  \underline{Zhiyong Wang}, Jize Xie, Xutong Liu, Shuai Li, John C.S. Lui,\\
  Thirty-seventh Conference on Neural Information Processing Systems (NeurIPS), 2023.

  \item \textbf{Adversarial Attacks on Online Learning to Rank with Click Feedback,}\\
  Jinhang Zuo, Zhiyao Zhang, \underline{Zhiyong Wang}, Shuai Li, Mohammad Hajiesmaili, Adam Wierman,\\
  Thirty-seventh Conference on Neural Information Processing Systems (NeurIPS), 2023.

  \item \textbf{Efficient Explorative Key-term Selection Strategies for Conversational Contextual Bandits,}\\
  \underline{Zhiyong Wang}, Jize Xie, Xutong Liu, Shuai Li, John C.S. Lui,\\
  Thirty-seventh AAAI Conference on Artificial Intelligence (AAAI), 2023.
\end{enumerate}

% =======================================================

% initialization

\newpage
\setcounter{page}{0}
\pagenumbering{arabic}
\pagestyle{headings}

\chapter{Introduction}\label{chap:intro}

Online learning methods—such as reinforcement learning (RL) and multi-armed bandits (MAB)—have achieved remarkable progress across a broad range of applications. These include surpassing human-level performance in complex games like Atari and Go, driving breakthroughs in robotics, powering large-scale recommendation systems, and facilitating the fine-tuning of large language models (LLMs).

Nevertheless, despite these advances, many state-of-the-art algorithms are designed under idealized assumptions, which can be fragile in the face of model misspecifications or adversarial perturbations. Such assumptions—often involving prior knowledge of the model class or environment stationarity—rarely hold in practice, particularly in dynamic systems influenced by unknown or even malicious factors. This disconnect between theoretical assumptions and practical realities calls for the development of algorithms that are both robust and adaptive to unforeseen conditions.

Moreover, worst-case theoretical guarantees, while valuable for providing general reliability, often overlook the performance advantages that can be gained from exploiting instance-specific structure. Such conservatism may result in inefficient learning in practice. A further key challenge is generalization—the ability to extend what is learned in the training phase to unseen environments or tasks. This ability is essential for deploying learning systems in real-world scenarios that are dynamic, partially observable, or fundamentally different from the training conditions.

Motivated by these critical issues, this thesis aims to address them by advancing the field toward:

\begin{center}
\textbf{\textit{more efficient, robust, instance-adaptive, and generalizable online learning.}}
\end{center}

To this end, we focus on developing theoretical foundations and algorithms for both reinforcement learning and multi-armed bandits that embody these four properties. Below, we outline the major research problems tackled in this thesis.

\section{Model-based RL as a Minimalist Approach to Horizon-Free and Second-Order Bounds}

Model-based reinforcement learning (MBRL) typically involves two steps: first, learning a model of the environment's transition dynamics using collected data; second, performing planning or policy optimization within the learned model. This paradigm is attractive due to its simplicity and has been successfully applied to a wide array of real-world domains such as control and robotics (e.g., \cite{aboaf1989task,deisenroth2011learning,venkatraman2017improved,williams2017information,chua2018deep,kaiser2019model,yang2023learning}).

The simplicity of MBRL also lends itself to theoretical analysis. Prior work has studied its performance in both online RL \cite{sun2019model} and offline RL \cite{uehara2021pessimistic} settings. For instance, \cite{mania2019certainty} showed that in the classic linear quadratic regulator (LQR) setting, the basic model-fitting and planning scheme enjoys strong performance guarantees. Similarly, \cite{liu2023optimistic} demonstrated that when optimism is incorporated, MBRL achieves solid sample complexity bounds for online RL with rich function classes. In the offline case, \cite{uehara2021pessimistic} showed that pessimism-augmented MBRL can provide robust guarantees, while \cite{ross2012agnostic} further established its effectiveness in hybrid settings involving both online and offline data, even without explicit optimism or pessimism mechanisms.

Rather than proposing new MBRL algorithms, this thesis shows that the standard MBRL approach—using Maximum Likelihood Estimation (MLE) for model fitting, combined with optimistic or pessimistic planning (depending on whether the setting is online or offline)—already achieves strong theoretical results. Specifically, under the conditions where the trajectory-level reward is normalized and transitions are time-homogeneous, these algorithms can yield nearly horizon-free and instance-dependent regret and sample complexity bounds even in the presence of general, non-linear function approximation.

Nearly horizon-free bounds imply that the regret or sample complexity does not scale polynomially with the time horizon $H$, suggesting that long-term planning is not necessarily the bottleneck for statistical efficiency. For instance-dependent analysis, we focus on second-order bounds, where regret scales with the variance of policy returns and directly implies first-order bounds as a special case. This leads to significantly smaller regret when the environment is nearly deterministic or when the optimal policy has low return variance. In the case of deterministic transitions (which the algorithm does not need to know in advance), we demonstrate that these algorithms can converge faster than what worst-case analysis would suggest.

\begin{center}
\emph{Simple and standard MLE-based MBRL algorithms are sufficient for achieving nearly horizon-free and second-order bounds in online and offline RL with function approximation.}
\end{center}

\section{Provable Zero-Shot Generalization in Offline Reinforcement Learning}

Offline RL has become an increasingly vital framework as it enables learning policies from fixed datasets without requiring direct interaction with the environment. However, in practical deployments, the training dataset often comes from environments that differ from the ones where the policy will ultimately be applied. This leads to the need for zero-shot generalization (ZSG), where an agent is trained on a finite number of environments sampled from a distribution and then evaluated on previously unseen environments—without access to additional interactions. This problem has been explored in the online RL literature \cite{Rajeswaran2017TowardsGA, Machado2018RevisitingTA, Justesen2018IlluminatingGI, Packer2018AssessingGI, Zhang2018ADO, Zhang2018ASO}, but remains under-theorized in the offline setting.

Although recent empirical studies \cite{mediratta2023generalization, yang2023essential, mazoure2022improving} have tackled this problem by proposing various ZSG-capable methods, most suffer from strong limitations. Some methods work only when environment differences are restricted to observations \cite{mazoure2022improving}, while others reduce to imitation learning setups \cite{yang2023essential}, limiting their generality. On the theoretical front, multi-task offline RL approaches \cite{bose2024offline, ishfaq2024offline} leverage shared representations but rely on access to downstream interactions, thus deviating from the offline ZSG formulation.

This motivates the central question:

\begin{center}
\emph{Can we design provable offline RL algorithms that support zero-shot generalization?}
\end{center}

To this end, we develop novel algorithmic frameworks for offline RL that deliver provable ZSG guarantees and significantly outperform prior methods in large-scale experiments.

\section{Online Clustering of Bandits with Misspecified User Models}

Stochastic multi-armed bandits (MAB) are foundational models for sequential decision-making under uncertainty. In each round, a learning agent selects an action and observes a corresponding reward, with the goal of maximizing cumulative rewards. They are widely adopted in applications such as recommendation systems and network optimization \cite{kohli2013fast,liu2023variance,wang2023efficient,cai2018online}.

To handle complex settings, contextual linear bandits incorporate side information, modeling expected rewards as linear functions of observed features. These methods enable personalization in large-scale systems \cite{li2010contextual,chu2011contextual,abbasi2011improved,liu2023contextual,kong2023best}, but do not exploit similarities among users. Clustering of bandits (CB) addresses this by adaptively grouping users and sharing information across clusters \cite{gentile2014online}.

However, prior CB methods assume perfectly linear reward models and identical preferences within clusters. This fails to reflect real-world variation caused by noise or user heterogeneity \cite{hainmueller2014kernel,ghosh2017misspecified}. To overcome this, we introduce clustering of bandits with misspecified user models (CBMUM), where users in the same cluster share a linear reward component but have individual deviations that better capture diverse behaviors.

\section{Online Corrupted User Detection and Regret Minimization}

In online recommendation platforms, user data arrives sequentially, and some users may behave adversarially—through click fraud, fake reviews, or coordinated disruptions. These corrupted signals can degrade the system’s performance by misleading preference estimations \cite{lykouris2018stochastic,ma2018data,he2022nearly,hajiesmaili2020adversarial,gupta2019better}.

Prior bandit algorithms with corruption tolerance are limited to single-user settings \cite{lykouris2018stochastic,gupta2019better,li2019stochastic}, and offline user detection approaches \cite{wang2019semi,dou2020enhancing,zhang2021fraudre} cannot operate dynamically in streaming scenarios.

We propose a novel learning framework called Learning and Online Corrupted Users Detection (LOCUD), which simultaneously performs preference learning, cluster inference, and online anomaly detection under potential adversarial corruption. This setting models latent user clusters and assumes only a minority of users are corrupted. The algorithm aims to maximize reward and detect corrupted users on the fly—despite dynamic and partially adversarial behavior.

\section{Online Clustering of Dueling Bandits}

In many applications like recommendation and prompt tuning for LLMs, it is more realistic to obtain relative feedback (i.e., preferences) instead of absolute scores. Dueling bandits formalize this feedback mode by asking users to compare two options. Classical dueling bandit algorithms, however, do not incorporate user collaboration.

We introduce the first clustering of dueling bandits framework, which enables adaptive user grouping in preference-feedback environments. This approach combines pairwise comparison modeling with collaborative structure, expanding the applicability of contextual dueling bandits \cite{JCSS12_yue2012k,lin2024prompt,verma2024neural}.

\section{Efficient Explorative Key-term Selection Strategies for Conversational Contextual Bandits}

Conversational recommender systems (CRS) improve learning efficiency by eliciting user preferences through occasional interactions involving explicit feedback on key terms \cite{christakopoulou2016towards,zhang2020conversational}. Existing conversational bandit methods treat feedback from different levels independently and lack effective strategies to select informative key terms.

We propose ConLinUCB, a unified framework that jointly integrates key-term and arm-level feedback into a single estimation process. Within this framework, we design two explorative strategies: ConLinUCB-BS, which samples from a barycentric spanner of key terms, and ConLinUCB-MCR, which selects key terms based on confidence radius to maximize exploration. These methods significantly improve the speed and quality of recommendation.

\section{Variance-Dependent Regret Bounds for Non-stationary Linear Bandits}

In non-stationary linear bandits, the expected reward functions change over time. Most existing approaches focus on bounding regret in terms of total variation in reward means (e.g., $B_K$), but ignore the influence of reward variance.

We propose new algorithms that exploit both mean drift and variance information to produce regret bounds that scale more favorably in heteroscedastic environments. This advancement is motivated by real-world applications like hyperparameter tuning in physical systems, where the noise profile depends on the evaluation point. Our methods demonstrate that variance-awareness can yield sharper bounds and better adaptivity in non-stationary settings.
\chapter{Literature Review}
\label{chapter:background}

In this chapter, we summarize previous researches that are related to this thesis and differentiate our results from theirs.

\section{Model-based RL}
Learning transition models with function approximation and planning with the learned model is a standard approach in RL and control. In the control literature, certainty-equivalence control learns a model from some data and plans using the learned model, which is simple but effective for controlling systems such as Linear Quadratic Regulators (LQRs) \cite{mania2019certainty}.  In RL, such a simple model-based framework has been widely used in theory with rich function approximation, for online RL \cite{sun2019model,foster2021statistical,song2021pc,zhan2022pac,liu2022partially, liu2023optimistic,zhong2022gec}, offline RL \cite{uehara2021pessimistic}, RL with representation learning \cite{agarwal2020flambe, uehara2021representation}, and hybrid RL using both online and offline data for model fitting \cite{ross2012agnostic}. Our work builds on the maximum-likelihood estimation (MLE) approach, a standard method for estimating transition models in model-based RL.

\section{Horizon-free and Instance-dependent bounds}
Most existing works on horizon-free RL typically focus on tabular settings or linear settings. For instance, \cite{wang2020long} firstly studied horizon-free RL for tabular MDPs and proposed an algorithm that depends on horizon logarithmically. Several follow-up work studied horizon-free RL for tabular MDP with better sample complexity \cite{zhang2021reinforcement}, offline RL \cite{ren2021nearly}, stochastic shortest path \cite{tarbouriech2021stochastic} and RL with linear function approximation \cite{kim2022improved, zhang2021improved, zhou2022computationally, di2023nearly, zhang2024horizon, zhang2023optimal, zhao2023variance}. Note that all these works have logarithmic dependence on the horizon $H$. For the tabular setting, recent work further improved the regret or sample complexity to be completely independent of the horizon (i.e., removing the logarithmic dependence on the horizon) \cite{li2022settling, zhang2022horizon} with a worse dependence on the cardinality of state and action spaces $|\mathcal{S}|$ and $|\mathcal{A}|$. To compare with, we show that simple MBRL algorithms are already enough to achieve completely horizon-free (i.e., no log dependence) sample complexity for offline RL when the transition model class is finite, and we provide a simpler approach to achieve the nearly horizon-free results for tabular MDPs, compared with \cite{zhang2021reinforcement}. A recent work \cite{huang2024horizon} also studied the horizon-free and instance-dependent online RL in the function approximation setting with small Eluder dimensions. They estimated the variances to conduct variance-weighted regression. To compare, in our online RL part, we use the simple and standard MLE-based MBRL approach and analysis to get similar guarantees. A more recent work also studied horizon-free behavior cloning \cite{foster2024behavior}, which is different from our settings.

\section{Offline RL}
Offline reinforcement learning (RL) \cite{ernst2005tree,riedmiller2005neural,lange2012batch,levine2020offline} addresses the challenge of learning a policy from a pre-collected dataset without direct online interactions with the environment. A central issue in offline RL is the inadequate dataset coverage, stemming from a lack of exploration \cite{levine2020offline,liu2020provably}. A common strategy to address this issue is the application of the pessimism principle, which penalizes the estimated value of under-covered state-action pairs. Numerous studies have integrated pessimism into various single-environment offline RL methodologies. This includes model-based approaches \cite{rashidinejad2021bridging,uehara2021pessimistic,jin2021pessimism,yu2020mopo,xie2021policy,uehara2021representation,yin2022near}, model-free techniques \cite{kumar2020conservative,wu2021uncertainty,bai2022pessimistic,ghasemipour2022so,yan2023efficacy}, and policy-based strategies \cite{rezaeifar2022offline,xie2021bellman,zanette2021provable,nguyen2024sample}. \cite{yarats2022don} has observed that with sufficient offline data diversity and coverage, the need for pessimism to mitigate extrapolation errors and distribution shift might be reduced. To the best of our knowledge, we are the first to theoretically study the generalization ability of offline RL in the contextual MDP setting. 

\section{Generalization in online RL}
There are extensive empirical studies on training online RL agents that can generalize to new transition and reward functions~\cite{Rajeswaran2017TowardsGA, Machado2018RevisitingTA, Justesen2018IlluminatingGI, Packer2018AssessingGI, Zhang2018ADO, Zhang2018ASO, Nichol2018GottaLF, cobbe2018quantifying,  kuettler2020nethack,  Bengio2020InterferenceAG, Bertrn2020InstanceBG, ghosh2021generalization,  kirk2021generalisation, obstacletower, ajay2021understanding, samvelyan2021minihack, frans2022powderworld, albrecht2022avalon, ehrenberg2022study, Song2020ObservationalOI,lyle2022learning,ye2020rotation, lee2020network,jiang2022uncertainty}. They  use techniques including implicit regularization \cite{Song2020ObservationalOI}, data augmentation \cite{ye2020rotation, lee2020network}, uncertainty-driven exploration~\cite{jiang2022uncertainty}, successor feature \cite{touati2023does}, etc. These works focus mostly on the online RL setting and do not provide theoretical guarantees, thus differing a lot from ours. Moreover, \cite{touati2023does} has studied zero-shot generalization in offline RL, but to unseen reward functions rather than unseen environments. 

% Though there are also some recent works on the generalization of offline RL \cite{mazoure2022improving,mediratta2023generalization,yang2023essential}, they are all pure empirical works, and have no theoretical guarantees.

There are also some recent works aimed at understanding online RL generalization from a theoretical perspective. \cite{wang2019generalization} examined a specific class of reparameterizable RL problems and derived generalization bounds using Rademacher complexity and the PAC-Bayes bound. \cite{malik2021generalizable} established lower bounds and introduced efficient algorithms that ensure a near-optimal policy for deterministic MDPs. A recent work \cite{ye2023power} studied how much pre-training can improve online RL test performance under different generalization settings. To the best of our knowledge, no previous work exists on theoretical understanding of the zero-shot generalization of offline RL.

Our paper is also related to recent works studying multi-task learning in reinforcement learning (RL)~\cite{brunskill2013sample,tirinzoni2020sequential,hu2021near,zhang2021provably,lu2021power,bose2024offline,ishfaq2024offline,zhang2023provably,lu2025towards}, which focus on transferring the knowledge learned from upstream tasks to downstream ones.  Additionally, these works typically assume that all tasks share similar transition dynamics or common representations while we do not. Meanwhile, they typically require the agent to interact with the downstream tasks, which does not fall into the ZSG regime. 

\section{Online Clustering of Bandits (CB)}
The paper \cite{gentile2014online} first formulates the  CB problem and proposes a graph-based algorithm.
The work \cite{li2016collaborative} further considers leveraging the collaborative effects on items to guide the clustering of users. The work \cite{li2018online} considers the CB problem in the cascading bandits setting with random prefix feedback. The paper \cite{10.5555/3367243.3367445} also considers users with different arrival frequencies. A recent work \cite{liu2022federated} proposes the setting of clustering of federated bandits, considering both privacy protection and communication requirements. However, all these works assume that the reward model for each user follows a perfectly linear model, which is unrealistic in many real-world applications. To the best of our knowledge, this paper is the first work to consider user model misspecifications in the CB problem.

\section{Misspecified Linear Bandits (MLB)}The work \cite{ghosh2017misspecified} first proposes the misspecified linear bandits (MLB) problem, shows the vulnerability of linear bandit algorithms under deviations, and designs an algorithm RLB that is only robust to non-sparse deviations. The work \cite{lattimore2020learning} proposes two algorithms to handle general deviations, which are modifications of the phased elimination algorithm \cite{lattimore2020bandit} and LinUCB \cite{abbasi2011improved}. Some recent works \cite{pacchiano2020model,foster2020adapting} use model selection methods to deal with unknown exact maximum model misspecification level. Note that the work \cite{foster2020adapting} has an additional assumption on the access to an online regression oracle, and the paper \cite{pacchiano2020model} still needs to know an upper bound of the unknown exact maximum model deviation level. None of them consider the CB setting with multiple users, thus differing from ours. 

\section{Bandits with Adversarial Corruption}The work \cite{lykouris2018stochastic} first studies stochastic bandits with adversarial corruption, where the rewards are corrupted with the sum of corruption magnitudes in all rounds constrained by the \textit{corruption level} $C$. 
They propose a robust elimination-based algorithm.
The paper \cite{gupta2019better} proposes an improved algorithm with a tighter regret bound. 

The paper \cite{li2019stochastic} first studies stochastic linear bandits with adversarial corruptions. To tackle the contextual linear bandit setting where the arm set changes over time, the work \cite{ding2022robust} proposes a variant of the OFUL \cite{abbasi2011improved} that achieves a sub-linear regret. A recent work \cite{he2022nearly} proposes the CW-OFUL algorithm that achieves a nearly optimal regret bound. 
All these works focus on designing robust bandit algorithms for a single user; none consider how to robustly learn and leverage the implicit relations among potentially corrupted users for more efficient learning. Moreover, none of them consider how to online detect corrupted users in the multiple-user case.

\section{Dueling Bandits and Neural Bandits}
Dueling bandits has been receiving growing attention over the years since its introduction \cite{ICML09_yue2009interactively,ICML11_yue2011beat,JCSS12_yue2012k} due to the prevelance of preference or relative feedback in real-world applications.
Many earlier works on dueling bandits have focused on MAB problems with a finte number of arms \cite{WSDM14_zoghi2014relative,ICML14_ailon2014reducing,ICML14_zoghi2014relative,COLT15_komiyama2015regret,ICML15_gajane2015relative,UAI18_saha2018battle,AISTATS19_saha2019active,ALT19_saha2019pac,AISTATS22_saha2022exploiting,ICML23_zhu2023principled}.
More recently, contextual dueing bandits, which model the reward function using a parametric function of the features of the arms, have attracted considerable attention \cite{NeurIPS21_saha2021optimal,ALT22_saha2022efficient,ICML22_bengs2022stochastic,arXiv23_di2023variance,arXiv24_li2024feelgood,verma2024neural}.

To apply MABs to complicated real-world applications with non-linear reward functions, neural bandits have been proposed which use a neural network to model the reward function \cite{zhou2020neural,zhang2020neural}.
Recently, we have witnessed a significant growing interest in further improving the theoretical and empirical performance of neural bandits and applying it to solve real-world problems \cite{xu2020neural,kassraie2021neural,gu2021batched,nabati2021online,lisicki2021empirical,ban2021ee,ban2021convolutional,jia2021learning,nguyen2021offline,zhu2021pure,kassraie2022graph,salgia2022provably,dai2022sample,hwang2023combinatorial,qi2023graph,qi2024meta}.
In particular, the work of \cite{ban2024meta} has adopted a neural network as a meta-learner for adapting to users in different clusters within the framework of clustering of bandits, and the work of \cite{verma2024neural} has combined neural bandits with dueling bandits.

\section{Conversational Contextual Bandits}Contextual linear bandit is an online sequential decision-making problem where at each time step, the agent has to choose an action and receives a corresponding reward whose expected value is an unknown linear function of the action \cite{li2010contextual,chu2011contextual,
abbasi2011improved,wu2016contextual}. The objective is to collect as much reward as possible in $T$ rounds. 

Traditional linear bandits need extensive exploration to capture the user preferences in recommender systems. To speed up online recommendations, the idea of conversational contextual bandits was first proposed in \cite{zhang2020conversational}, where 
conversational feedback on key-terms is leveraged to assist the user preference elicitation. In that work, they propose the ConUCB algorithm with a theoretical regret bound of $O(d\sqrt{T}\log T)$. Some follow-up works try to improve the performance of ConUCB with the help of additional information, such as self-generated key-terms~\cite{wu2021clustering}, relative feedback~\cite{xie2021comparison}, and knowledge graph~\cite{zhao2022knowledge}. 
Unlike these works, we adopt the same problem settings as ConUCB and improve the underlying mechanisms without relying on additional information.
Yet one can use the principles of efficient information incorporation and explorative conversations proposed in this work to enhance these works when additional information is available, which is left as an interesting future work.

\section{Non-stationary (Linear) Bandits}
There have been a series of works about non-stationary bandits \cite{auer2002nonstochastic,10.1007/978-3-642-24412-4_16,besbes2014stochastic,wei2016tracking,cheung2019learning,russac2019weighted,pmlr-v99-auer19a,pmlr-v99-chen19b,russac2020algorithms,zhao2020simple,kim2020randomized,wei2021non,russac2021self,chen2021combinatorial,deng2022weighted,suk2022tracking,liu2023definition,abbasi2023new,clerici2023linear}.

In non-stationary linear bandits, the unknown feature vector $\btheta_k$ can be dynamically and adversarially adjusted, with the total change upper bounded by the \emph{total variation budget} $B_K$ over $K$ rounds, \emph{i.e.}, $\sum_{k=1}^{K-1} \|\btheta_{k+1}-\btheta_k\|_2\leq B_K$. To tackle this problem, some works proposed forgetting strategies such as sliding window, restart, and weighted regression \cite{cheung2019learning,russac2019weighted,zhao2020simple}.  \cite{kim2020randomized} also introduced the randomized exploration with weighting strategy. The regret upper bounds in these works are all of $\Tilde{O}(B_K^{\frac{1}{4}}K^{\frac{3}{4}})$. A recent work by  \cite{wei2021non} proposed the MASTER-OFUL algorithm based on a black-box approach, which can achieve a regret bound of $\Tilde{O}(B_K^{\frac{1}{3}}K^{\frac{2}{3}})$ in the case where the arm set is fixed over $K$ rounds. To the best of our knowledge, none of the existing works consider how to utilize the variance information to improve the regret bound in the case with time-dependent variances. The only exception of utilizing the variance information in the non-stationary bandit setting is \cite{wei2016tracking}, which proposed the Rerun-UCB-V algorithm for the non-stationary MAB setting with a regret dependent on the action set size $|\cA|$. To compare with, the regret upper bounds of our algorithms are independent of the action set size, thus our algorithms are more efficient for the case where the number of actions is large.  

\section{Linear Bandits with Heteroscedastic Noises}
Some recent works study the heteroscedastic linear bandit problem, where the noise distribution is assumed to vary over time. \cite{kirschner2018information} first proposed the linear bandit model with heteroscedastic noise. In this model, the noise at round $k \in [K]$ is assumed to be $\sigma_k$-sub-Gaussian. Some follow-up works relaxed the $\sigma_k$-sub-Gaussian assumption by assuming the noise at the $k$-th round to be of variance $\sigma_k^2$ \cite{zhou2021nearly, zhang2021improved, kim2022improved, zhou2022computationally, dai2022variance,zhao2023variance}. Specifically, \cite{zhou2021nearly} and \cite{zhou2022computationally} considered the case where $\sigma_k$ is observed by the learner after the $k$-th round. \cite{zhang2021improved} and \cite{kim2022improved} proposed statistically efficient but computationally inefficient algorithms for the unknown-variance case. A recent work by \cite{zhao2023variance} proposed an algorithm that achieves both statistical and computational efficiency in the unknown-variance setting. \cite{dai2022variance} also considered a specific heteroscedastic linear bandit problem where the linear model is sparse.
\chapter{Model-based RL as a Minimalist Approach to Horizon-Free and Second-Order Bounds}
\label{chapter:iclr2025}
Learning a transition model via Maximum Likelihood Estimation (MLE) followed by planning inside the learned model is perhaps the most standard and simplest Model-based Reinforcement Learning (RL) framework. In this work, we show that such a simple Model-based RL scheme, when equipped with optimistic and pessimistic planning procedures, achieves strong regret and sample complexity bounds in online and offline RL settings. Particularly, we demonstrate that under the conditions where the trajectory-wise reward is normalized between zero and one and the transition is time-homogenous, it achieves nearly horizon-free and second-order bounds. 
This chapter is based on our publication~\cite{wang2024model}.

\vspace{-0.16cm}
\section{Introduction} 
\label{sec:intro}

The framework of model-based Reinforcement Learning (RL) often consists of two steps: fitting a transition model using data and then performing planning inside the learned model. Such a simple framework turns out to be powerful and has been used extensively in practice on applications such as robotics and control (e.g., \cite{aboaf1989task,deisenroth2011learning,venkatraman2017improved,williams2017information,chua2018deep,kaiser2019model,yang2023learning}).  \looseness=-1

The simplicity of model-based RL also attracts researchers to analyze its performance in settings such as online RL \cite{sun2019model} and offline RL \cite{uehara2021pessimistic}. \cite{mania2019certainty} showed that this simple scheme --- fitting model via data followed by optimal planning inside the model, has a strong performance guarantee under the classic linear quadratic regulator (LQR) control problems.  \cite{liu2023optimistic} showed that this simple MBRL framework when equipped with optimism in the face of the uncertainty principle, can achieve strong sample complexity bounds for a wide range of online RL problems with rich function approximation for the models. For offline settings where the model can only be learned from a static offline dataset, \cite{uehara2021pessimistic} showed that MBRL equipped with the pessimism principle can again achieve robust performance guarantees for a large family of MDPs. \cite{ross2012agnostic} showed that in the hybrid RL setting where one has access to both online and offline data, this simple MBRL framework again achieves favorable performance guarantees without any optimism/pessimism algorithm design. \looseness=-1

In this work, we do not create new MBRL algorithms, instead, we show that the extremely simple and standard MBRL algorithm -- fitting models using Maximum Likelihood Estimation (MLE), followed by optimistic/pessimistic planning (depending on whether operating in online RL or offline RL mode), can already achieve surprising theoretical guarantees. Particularly, we show that under the conditions that trajectory-wise reward is normalized between zero and one, and the transition is time-homogenous, they can achieve nearly horizon-free and instance-dependent regret and sample complexity bounds, in both online and offline RL with non-linear function approximation. Nearly horizon-free bounds mean that the regret or sample complexity bounds have no explicit polynomial dependence on the horizon $H$. The motivation for studying horizon-free RL is to see if RL problems are harder than bandits due to the longer horizon planning in RL. \emph{Our result here indicates that, even under non-linear function approximation, long-horizon planning is not the bottleneck of achieving statistical efficiency in RL.}
For instance-dependent bounds, we focus on second-order bounds.  A second-order regret bound scales with respect to the variances of the returns of policies and also directly implies a first-order regret bound which scales with the expected reward of the optimal policy. Thus our instance-dependent bounds can be small under situations such as nearly-deterministic systems or the optimal policy having a small value. 
When specializing to the case of deterministic ground truth transitions (but the algorithm does not need to know this a priori), we show that these simple MBRL algorithms demonstrate a faster convergence rate than the worst-case rates.  
%can be small under situations such as nearly-deterministic systems and policies. 
The key message of our work is \looseness=-1
\vspace{-0.1cm}
\begin{center}
    \emph{Simple and standard MLE-based MBRL algorithms are sufficient for achieving nearly horizon-free and second-order bounds in online and offline RL with function approximation. }
\end{center}
\vspace{-0.1cm}
% \begin{displayquote}
% \emph{Simple and standard MLE-based MBRL algorithms are sufficient for achieving nearly horizon-free and instance-dependent\zhiyong{or second-order is better?} bounds in online and offline RL with general function approximation. }
% \end{displayquote}

We provide a fairly standard analysis to support the above claim. Our analysis follows the standard frameworks of optimism/pessimism in the face of uncertainty.
%We show that standard and simple MBRL algorithms indeed achieve the two properties at the same time.  %We emphasize that we use the principle of optimism or pessimism when performing planning inside the learned models for the purpose of performing strategic exploration or being conservative for online learning and offline learning, respectively.  For settings such as hybrid RL \cite{song2022hybrid} where we have access to both good-quality offline data and online data, we can replace the optimistic/pessimism planning procedure with a more standard optimal control/planning procedure.
%Through fairly standard analysis frameworks based on optimism/pessimism in the face of uncertainty, we show that the horizon-free and instance-dependent bounds hold for MDPs with general nonlinear function approximation. Particularly, 
For online RL. we use $\ell_1$ Eluder dimension \cite{liu2022partially,wang2024more}, a condition that uses both the MDP structure and the function class, to capture the structural complexity of exploration. For offline RL, we use the similar concentrability coefficient in \cite{ye2024corruption} to capture the coverage condition of the offline data. The key technique we leverage is the \emph{triangular discrimination} -- a divergence that is equivalent to the squared Hellinger distance up to some universal constants. Triangular discrimination was used in contextual bandit and model-free RL for achieving first-order and second-order instance-dependent bounds \cite{foster2021efficient,wang2023benefits,wang2024more}. Here we show that it also plays an important role in achieving horizon-free bounds. Our contributions can be summarized as follows.
\vspace{-0.26cm}
\begin{enumerate}[leftmargin = *]
\item Our results extend the scope of the prior work on horizon-free RL which only applies to tabular MDPs or MDPs with linear functions. %\zhiyong{Now we can not say we extend the scope of works on functions with small Eluder dimensions? May need to tune the statement here?}
Given a finite model class $\Pcal$ (which could be exponentially large), we show that in online RL, the agent achieves an {\small$O\left(  \sqrt{  ( \sum_k \var_{\pi^k} ) \cdot  d_{\text{RL}}  \log ( KH | \Pcal | / \delta )   } + d_{\text{RL}} \log ( KH |\Pcal|/\delta)  \right)$} regret, where $K$ is the number of episodes, $d_{\text{RL}}$ is the $\ell_1$ Eluder dimension, $\var_{\pi^k}$ is the variance of the total reward of policy $\pi^k$ learned in episode $k$ and $\delta \in (0,1)$ denotes the failure probability. Similarly, for offline RL, the agent achieves an {\small$O\left(\sqrt{{C^{\pi^*} \var_{\pi^*}\log(|\Pcal|/\delta)}/{K} } + {C^{\pi^*}\log(|\Pcal|/\delta)}/{K}\right)$} performance gap in finding a comparator policy $\pi^*$, where $C^{\pi^*}$ is the single policy concentrability coefficient over $\pi^*$, $K$ denotes the number of offline trajectories, $\var_{\pi^*}$ is the variance of the total reward of $\pi^*$. For offline RL with finite $\Pcal$, our result is \emph{completely horizon-free}, not even with $\log H$ dependence.
%The regret bound looks shockingly similar to the bounds in Contextual Bandits which often scale like $O( \sqrt{K d_{cb} \ln( | \Fcal  |  ) }  )$ where $d_{cb}$ is the structural complexity for the exploration in CB (e.g., decoupling coefficient \cite{zhang2022feel}), and $\Fcal$ is the ``model class" for capturing the ground truth reward function in CB. 

\item When specializing to MDPs with deterministic ground truth transition (but rewards, and models in the model class could still be stochastic), we show that the same simple MBRL algorithms can adapt to the deterministic environment and achieve a better statistical complexity. For online RL, the regret becomes $O(d_{\text{RL}} \log( K H|\Pcal|/\delta) )$, which only depends on the number of episodes $K$ poly-logarithmically. For offline RL, the performance gap to a comparator policy $\pi^*$ becomes {\small$O\left( { C^{\pi^*} \log( |\Pcal | /\delta ) }/{K }\right)$}, which is tighter than the worst-case $O(1/\sqrt{K})$ rate. All our results can be extended to continuous model class $\Pcal$ using bracket number as the complexity measure. \looseness=-1

%\item We also extend our results from a finite model class to an arbitrary model class with infinite cardinality. For both online and offline settings, we show that MBRL algorithms are able to achieve an instance-dependent performance regret/gap which only depends on the horizon logarithmically. %Specifically, our results yield new nearly horizon-free statistical complexity results for tabular MDPs with a simple model-based approach, which may be of independent interest. 

%by achieving first and second order bounds on a more general class of MDPs, while at the same time, significantly sharpen the horizon dependence.    
\end{enumerate}
\vspace{-0.16cm}
Overall, our work identifies the \emph{minimalist} algorithms and analysis for nearly horizon-free and instance-dependent (first \& second-order) online \& offline RL.

\vspace{-0.26cm}
\section{Preliminaries}
\vspace{-0.16cm}
\paragraph{Markov Decision Processes.} We consider finite horizon time homogenous MDP $\Mcal = \{\Scal, \Acal, H, P^\star, r, s_0\}$ where $\Scal, \Acal$ are the state and action space (could be large or even continuous), $H\in \NN^+$ is the horizon, $P^\star: \Scal\times\Acal\mapsto \Delta(\Scal)$ is the ground truth transition, $r:\Scal\times\Acal\mapsto \RR$ is the reward signal which we assume is known to the learner, and $s_0$ is the fixed initial state.\footnote{For simplicity, we assume initial state $s_0$ is fixed and known. Our analysis can be easily extended to a setting where the initial state is sampled from an unknown fixed distribution.} Note that the transition $P^\star$ here is time-homogenous. For notational easiness, we denote $[K-1]=\{0,1,\ldots, K-1\}$.

 We denote $\pi$ as a deterministic non-stationary policy $\pi = \{ \pi_0, \dots, \pi_{H-1} \}$  where $\pi_h: \Scal\mapsto \Acal$ maps from a state to an action. Let $\Pi$ denote the set of all such policies. $V^\pi_h(s)$ represents the expected total reward of policy $\pi$ starting at $s_h = s$, and $Q^\pi_h(s,a)$ is the expected total reward of the process of executing $a$ at $s$ at time step $h$ followed by executing $\pi$ to the end. The optimal policy $\pi^\star$ is defined as $\pi^\star = \argmax_{\pi} V^{\pi}_0(s_0)$.  For notation simplicity, we denote $V^\pi := V^{\pi}_0(s_0)$. We will denote $d_h^{\pi}(s,a)$ as the state-action distribution induced by policy $\pi$ at time step $h$. We sometimes will overload notation and denote $d^\pi_h(s)$ as the corresponding state distribution at $h$. Sampling $s\sim d^\pi_h$ means executing $\pi$ starting from $s_0$ to $h$ and returning the state at time step $h$. 

Since we use the model-based approach for learning, we define a general model class $\Pcal \subset \Scal\times\Acal\mapsto \Delta(\Scal)$. 
% For our most general result, we assume $\Pcal$ is discrete, and we aim to pay sample complexity that only scales polynomially with respect to the log of the size of $\Pcal$. 
% Extending $\Pcal$ to continuous class is standard (e.g., via a covering argument) and we will give several concrete examples. 
Given a transition $P$, we denote $V_{h; P}^\pi$ and $Q_{h;P}^\pi$ as the value and Q functions of policy $\pi$ under the model $P$.  %$V_0^\pi(s_0)$ under the MDP with transition $P$, and 
Given a function $f:\Scal\times\Acal\mapsto \RR$, we denote the $(P f)(s,a) := \EE_{s'\sim P(s,a)} f(s')$.  We then denote the \emph{variance induced by one-step transition $P$ and  function $f$} as $(\VV_P f)(s,a):=  \left( P f^2 \right)(s,a)  - \left( P f(s,a) \right)^2$ which is equal to $\EE_{s'\sim P(s,a)} f^2(s') - \left( \EE_{s'\sim P(s,a)} f(s')  \right)^2$. 
%(Pf^2)(s,a)-(Pf)^2(s,a)$, where the square is point-wise. 

\vspace{-0.16cm}
\paragraph{Assumptions.} We make the realizability assumption that $P^\star \in \Pcal$. We assume that the rewards are normalized such that $r(\tau) \in [0,1]$ for any trajectory $\tau:= \{s_0,a_0,\dots, s_{H-1},a_{H-1}\}$ where $r(\tau)$ is short for $\sum_{h=0}^{H-1} r(s_h,a_h)$. Note that this setting is more general than assuming each one-step reward is bounded, i.e., $r(s_h,a_h) \in [0,1/H]$, and allows to represent the sparse reward setting. Without loss of generalizability, we assume $V^{\pi}_{h;P}(s)\in[0,1]$, for all $\pi\in \Pi, h\in[0,H], P\in\Pcal, s\in\Scal$\footnote{$r(\tau) \in [0,1]$ implies $V^{\pi}_{h;P^\star}(s)\in[0,1]$. If we do not assume $V^{\pi}_{h;P}(s)\in[0,1]$ for all $P\in\Pcal$, we can simply add a filtering step in the algorithm to only choose $\pi$,$P$ with $V^{\pi}_{h;P}(s_0)\in[0,1]$ to get the same guarantees.}. \looseness=-1
\vspace{-0.16cm}
\paragraph{Online RL.}
For the online RL setting, we focus on the episodic setting where the learner can interact with the environment for $K$ episodes. At episode $k$, the learner proposes a policy $\pi^k$ (based on the past interaction history),  executes $\pi^k$ starting from $s_0$ to time step $H-1$. We measure the performance of the online learning via \emph{regret}: $\sum_{k=0}^{K-1} \left( V^{\pi^\star} - V^{\pi^k} \right)$. To achieve meaningful regret bounds, we often need additional structural assumptions on the MDP and the model class $\Pcal$. We use a $\ell_1$ Eluder dimension \cite{liu2022partially} as the structural condition due to its ability to capture non-linear function approximators (formal definition will be given in \pref{sec:online}). 
\vspace{-0.16cm}
\paragraph{Offline RL.}
For the offline RL setting, we assume that we have a pre-collected offline dataset $\Dcal = \{\tau^{i}\}_{i=1}^K$ which contains $K$ trajectories. For each trajectory, we allow it to potentially be generated by an adversary, i.e., at step $h$ in trajectory $k$, (i.e., $s_h^k$), the adversary can select $a^k_h$ based on all history (the past $k-1$ trajectories and the steps before $h$ within trajectory $k$) with a fixed strategy, with the only condition that the state transitions follow the underlying transition dynamics, i.e., $s^i_{h+1} \sim P^\star(s_h^i,a_h^i)$. We emphasize that $\Dcal$ is not necessarily generated by some offline trajectory distribution. %Thus in our analysis, we treat $\Dcal$ as a fixed input instead of some random data generated by some behavior policy. 
Given $\Dcal$, we can split the data into $H K$ many state-action-next state $(s,a,s')$ tuples which we can use to learn the transition. To succeed in offline learning, we typically require the offline dataset to have good coverage over some high-quality comparator policy $\pi^*$ (formal definition of coverage will be given in \pref{sec:offline}). Our goal here is to learn a policy $\widehat\pi$ that is as good as $\pi^*$, and we are interested in the \emph{performance gap} between $\hat \pi$ and $\pi^*$, i.e., $V^{\pi^*} - V^{\hat \pi}$.

%We introduce the formal definiton of SEC in \pref{sec:online}. \wen{need to verify if SEC works, and see if SEC hides some log H dependence (i assume not..)}

\vspace{-0.16cm}
\paragraph{Horizon-free and Second-order Bounds.} Our goal is to achieve regret bounds (online RL) or performance gaps (offline RL) that are (nearly) horizon-free, i.e., logarithmical dependence on $H$. In addition to the horizon-free guarantee, we also want our bounds to scale with respect to the variance of the policies. Denote $\var_\pi$ as the variance of trajectory reward, i.e., $\var_\pi := \EE_{\tau \sim \pi} (  r(\tau) - \EE_{\tau\sim\pi} r(\tau) )^2$.  Second-order bounds in offline RL scales with $\var_{\pi^*}$ -- the variance of the comparator policy. Second-order regret bound in online setting scales with respect to $\sqrt{ \sum_k \var_{\pi^k}  }$ instead of $\sqrt{K}$.
Note that in the worst case, $\sqrt{ \sum_k \var_{\pi^k}  }$ scales in the order of $\sqrt{K}$, but can be much smaller in benign cases such as nearly deterministic MDPs.  We also note that second-order regret bound immediately implies first-order regret bound in the reward maximization setting, which scales in the order $\sqrt{ K V^{\pi^\star}  }$ instead of just $\sqrt{K}$. The first order regret bound $\sqrt{ K V^{\pi^\star} }$ is never worse than $\sqrt{K}$ since $V^{\pi^\star} \leq 1$. Thus, by achieving a second-order regret bound, our algorithm immediately achieves a first-order regret bound. \looseness=-1
\vspace{-0.16cm}\\
\noindent\textbf{Additional notations.} Given two distributions $p \in \Delta(\Xcal)$ and $q \in \Delta(\Xcal)$, we denote the triangle discrimination  $D_\triangle( p \Mid q ) =  \sum_{x\in\Xcal} \frac{ (p(x) - q(x))^2 }{ p(x) + q(x) } $, and squared Hellinger distance $\mathbb H^2( p \Mid q) = \frac{1}{2} \sum_{x\in \Xcal} \left( \sqrt{ q(x)} - \sqrt{ p(x)} \right)^2 $ (we replace sum via integral when $\Xcal$ is continuous and $p$ and $q$ are pdfs). Note that $D_\triangle$ and $\mathbb H^2$ are equivalent up to universal constants. We will frequently use the following key lemma in \cite{wang2024more} to control the difference between means of two distributions. 
\begin{lemma}[Lemma 4.3 in \cite{wang2024more}]\label{lem: mean to variance}
For two distributions $f \in \Delta([0,1])$ and $g \in \Delta([0,1])$:\looseness=-1
\begin{equation}
    \abs{\EE_{x\sim f}[x] - \EE_{x\sim g} [x] }\leq 4\sqrt{ \var_f \cdot D_\triangle(f\Mid g) } + 5D_\triangle(f\Mid g). \label{eq:var-key-ineq2}
\end{equation} where $\var_f := \EE_{x\sim f} ( x - \EE_{x\sim f}[x])^2$ denotes the variance of the distribution $f$.
\end{lemma}
The lemma plays a key role in achieving second-order bounds \cite{wang2024more}. The intuition is the means of the two distributions can be closer if one of the distributions has a small variance. A more naive way of bounding the difference in means is $|\EE_{x\sim f}[x] - \EE_{x\sim g} [x]  | \leq ( \max_{x\in \Xcal} |x| )  \| f - g \|_1 \lesssim ( \max_{x\in \Xcal} |x| )  \mathbb H(f \Mid g )\lesssim  (\max_{x\in \Xcal} |x| ) \sqrt{ D_\triangle(f \Mid g)    }$. Such an approach would have to pay the maximum range $ \max_{x\in \Xcal} |x|$ and thus can not leverage the variance $\var_{f}$. In the next sections, we show this lemma plays an important role in achieving horizon-free and second-order bounds.

\vspace{-0.26cm}
\section{Online Setting}
\label{sec:online}

\begin{algorithm}[t]%
\caption{Optimistic Model-based RL (O-MBRL)}
\label{alg:mleonline}

\begin{algorithmic}[1]
    \STATE\textbf{Input:} model class $\Pcal$, confidence parameter $\delta\in(0,1)$, threshold $\beta$.
    \STATE Initialize $\pi^0$, initialize dataset $\Dcal = \emptyset$.
    \FOR{$k = 0 \to K-1$}
    	\STATE Collect a trajectory $\tau = \{ s_0,a_0, \cdots, s_{H-1},a_{H-1} \}$ from $\pi^k$, split it into tuples of $\{s,a,s'\}$ and add to $\Dcal$.
	\STATE Construct a version space $\widehat \Pcal^k$:  
		\begin{align*}
	        		\widehat\Pcal^k = \braces{ P\in\Pcal: \sum_{s,a,s' \in\Dcal} \log P(s_i'|s_i,a_i)\geq \max_{\Tilde{P}\in\Pcal}\sum_{s,a,s' \in\Dcal}\log \Tilde P (s_i'|s_i,a_i)-\beta}.
    		\end{align*}\label{online conf}
	\STATE Set $(\pi^k,\widehat{P}^k)\leftarrow \argmax_{\pi\in\Pi, P\in\widehat\Pcal^k} V_{0; P}^\pi(s_0)$.\label{pessimistic policy}
    \ENDFOR
    \end{algorithmic}
\end{algorithm}
\vspace{-0.16cm}
In this section, we study the online setting. We present the optimistic model-based RL algorithm (O-MBRL) in \pref{alg:mleonline}. The algorithm starts from scratch, and iteratively maintains a version space $\widehat \Pcal^k$ of the model class using the historical data collected so far.  Again the version space is designed such that for all $k\in [0,K-1]$, we have $P^\star \in \widehat \Pcal_k$ with high probability. The policy $\pi^k$ in this case is computed via the optimism principle, i.e., it selects $\pi^k$ and $\widehat P^k$ such that $V^{\pi^k}_{\widehat P^k} \geq V^{\pi^\star}$.

Note that the algorithm design in \pref{alg:mleonline} is not new and in fact is quite standard in the model-based RL literature. For instance, \cite{sun2019model} presented a similar style of algorithm except that they use a min-max GAN style objective for learning models. \cite{zhan2022pac} used MLE oracle with optimism planning for Partially observable systems such as Predictive State Representations (PSRs), and \cite{liu2023optimistic} used them for both partially and fully observable systems.  However, their analyses do not give horizon-free and instance-dependent bounds. We show that under the structural condition that captures nonlinear function class with small eluder dimensions, \pref{alg:mleonline} achieves horizon-free and second-order bounds. Besides, since second-order regret bound implies first-order bound \cite{wang2024more}, our result immediately implies a first-order bound as well.

% \subsection{Analysis}p
We first introduce the $\ell_p$ Eluder dimension as follows. 
    \begin{definition}[$\ell_p$ Eluder Dimension]
        $DE_p(\Psi, \Xcal, \epsilon)$  is the eluder dimension for $\Xcal$ with function class $\Psi$, when the longest $\epsilon$-independent sequence $x^1,\dots, x^L\subseteq \Xcal$ enjoys the length less than $DE_p(\Psi, \Xcal, \epsilon)$, i.e.,  there exists $g \in \Psi$ such that for all $t\in[L]$, $\sum_{l=1}^{t-1} |g(x^l)|^p \leq \epsilon^p$ and $|g(x^t)|>\epsilon$. 
    \end{definition}

% We first introduce the $\ell_1$ Eluder dimension. 
%     \begin{definition}[$\ell_1$ Eluder Dimension]
%         $DE_1(\Psi, \Scal \times \Acal, \epsilon)$ is the eluder dimension for $\Scal \times \Acal$ with function class $\Psi$, when the longest $\epsilon$-independent sequence $x^1,\dots, x^L\subseteq \Scal \times \Acal$ enjoys the length less than $DE_1(\Psi, \Scal \times \Acal, \epsilon)$, i.e.,  there exists $g \in \Psi$ such that $\sum_{l=1}^{L-1} |g(x^l)| \leq \epsilon$ and $|g(x^L)|>\epsilon$. 
%     \end{definition}

    We work with the $\ell_1$ Eluder dimension $DE_1(\Psi, \Scal \times \Acal, \epsilon)$ with the function class $\Psi$ specified as:
\begin{align*}
     \Psi=\{(s,a)\mapsto \mathbb H^2(P^\star(s,a)\Mid P(s,a)):P\in\Pcal\}\,.
\end{align*}
\begin{remark}
        The $\ell_1$ Eluder dimension has been used in previous works such as \cite{liu2022partially}. We have the following corollary to demonstrate that the $\ell_1$ dimension generalizes the original $\ell_2$ dimension of \cite{russo2013eluder}, it can capture tabular, linear, and generalized linear models.

\begin{lemma}[Proposition 19 in \cite{liu2022partially}]
    \label{lem:eluder-one-generalizes-two}
For any $\Psi,\Xcal$, $\epsilon>0$, $\DE_1(\Psi,\Xcal,\epsilon)\leq\DE_2(\Psi,\Xcal,\epsilon)$.
\end{lemma}
\end{remark}

We are ready to present our main theorem for the online RL setting.
 \begin{theorem}[Main theorem for online setting] \label{thm:online_theorem} 
For any $\delta\in (0,1)$, let $\beta=4\log\left(\frac{K\left|\Pcal\right|}{\delta}\right)$, with probability at least $1-\delta$, \pref{alg:mleonline} achieves the following regret bound:
\begin{small}
    \begin{align}
\sum_{k=0}^{K-1} (V^{\pi^\star} -  V^{\pi^k}) &\leq O\Big(\sqrt{\sum_{k=0}^{K-1}\var_{\pi^k}\cdot\text{DE}_1(\Psi,\Scal \times \Acal,1/KH)\cdot\log(KH\left|\Pcal\right|/\delta)\log(KH)}\notag\\
&\quad+\text{DE}_1(\Psi,\Scal \times \Acal,1/KH)\cdot\log(KH\left|\Pcal\right|/\delta)\log(KH)\Big)\,.
\end{align}
\end{small}
\end{theorem}

The above theorem indicates the standard and simple O-MBRL algorithm is already enough to achieve horizon-free and second-order regret bounds: our bound does not have explicit polynomial dependences on horizon $H$, the leading term scales with $\sqrt{ \sum_{k} \var_{\pi^k} }$ instead of the typical $\sqrt{K}$.

We have the following result about the first-order regret bound.

%Using the result from \cite{wang2024more} which shows that a second-order regret bound implies a first-order regret bound, we can immediately arrive at the following corollary. \zhiyong{This is obvious for our reward maximization setting? We do not use the theorem in Kaiwen's paper. Because $\var_{\pi^k}\leq V^{\pi^k}\leq V^\star$ (the total return is in [0,1])? In minimization setting, $\var_{\pi^k}\leq V^{\pi^k}$, but $V^{\pi^k}\geq V^\star$ so we need to do some other tricks.}
\begin{corollary}[Horizon-free and First-order regret bound] Let $\beta=4\log\left(\frac{K\left|\Pcal\right|}{\delta}\right)$, with probability at least $1-\delta$, \pref{alg:mleonline} achieves the following regret bound:
\begin{small}
    \begin{align*}
\sum_{k=0}^{K-1} V^{\pi^\star} - V^{\pi^k} &\leq O\Big(\sqrt{KV^{\pi^\star}\cdot\text{DE}_1(\Psi,\Scal \times \Acal,1/KH)\cdot\log(KH\left|\Pcal\right|/\delta)\log(KH)}\notag\\
&\quad+\text{DE}_1(\Psi,\Scal \times \Acal,1/KH)\cdot\log(KH\left|\Pcal\right|/\delta)\log(KH)\Big)\,.
\end{align*}
\end{small}
\end{corollary}
\begin{proof}
    Note that $\var_{\pi} \leq V^{\pi} \leq V^{\pi^\star}$ where the first inequality is because the trajectory-wise reward is bounded in $[0,1]$. Therefore, combining with \pref{thm:online_theorem}, we directly obtain the first-order result. 
\end{proof}
Note that the above bound scales with respect to $\sqrt{K V^{\pi^\star}}$ instead of just $\sqrt{K}$. Since $V^{\pi^\star} \leq 1$, this bound improves the worst-case regret bound when the optimal policy has total reward less than one.\footnote{Typically a first-order regret bound makes more sense in the cost minimization setting instead of reward maximization setting. We believe that our results are transferable to the cost-minimization setting.}
%\wen{verify that this corollary is true -- note that in my paper with kaiwen, we were doing everything in the cost minimization setting, and all his lemmas are build on the cost-minimization setting.}

\paragraph{Faster rates for deterministic transitions.} When the underlying MDP has deterministic transitions, we can achieve a smaller regret bound that only depends on the number of episodes logarithmically. 
\begin{corollary}[$\log K$ regret bound with deterministic transitions] \label{corr:online_coro_faster}
When the transition dynamics of the MDP are deterministic, setting $\beta=4\log\left(\frac{K\left|\Pcal\right|}{\delta}\right)$, w.p. at least $1-\delta$, \pref{alg:mleonline} achieves:
\begin{align*}
\sum_{k=0}^{K-1} V^{\pi^\star} - V^{\pi^k} \leq O\left( \text{DE}_1(\Psi,\Scal \times \Acal,1/KH)\cdot\log(KH\left|\Pcal\right|/\delta)\log(KH)\right).
\end{align*}
\end{corollary}
% \begin{proof}
%     See \pref{app:online_coro_faster}. \wen{i guess we can skip these to save space. if you want to link to the proof, saying it in the text is enough i think.}
% \end{proof}

% \wen{remove the stuff below?}

%\wen{move this after the main theorem as a corollary for tabular MDP}

\paragraph{Extension to infinite class $\Pcal$.} For infinite model class $\Pcal$, we have a similar result. First, we define the bracketing number of an infinite model class as follows.

\begin{definition}[Bracketing Number \cite{geer2000empirical}]\label{def: bracketing number}
Let $\Gcal$ be a set of functions mapping $\Xcal\to\RR$.
Given two functions $l,u$ such that $l(x)\leq u(x)$ for all $x\in\Xcal$, the bracket $[l,u]$ is the set of functions $g\in\Gcal$ such that $l(x)\leq g(x)\leq u(x)$ for all $x\in\Xcal$.
We call $[l,u]$ an $\epsilon$-bracket if $\norm{u-l}\leq\epsilon$.
Then, the $\epsilon$-bracketing number of $\Gcal$ with respect to $\norm{\cdot}$, denoted by $\Ncal_{[]}(\epsilon,\Gcal,\norm{\cdot})$ is the minimum number of $\epsilon$-brackets needed to cover $\Gcal$.
\end{definition}

We use the bracketing number of $\mathcal{P}$ to denote the complexity of the model class, similar to $|\mathcal{P}|$ in the finite class case. Next, we propose a corollary to characterize the regret with an infinite model class. 
\begin{corollary}[Regret bound for \pref{alg:mleonline} with infinite model class $\Pcal$]\label{corr:online_coro_infinite} 
When $\Pcal$ is infinite, let $\beta=7\log(K\Ncal_{[]}((KH|\Scal|)^{-1},\Pcal,\|\cdot\|_\infty)/\delta)$, with probability at least $1-\delta$, \pref{alg:mleonline} achieves the following regret bound:
\begin{small}
  \begin{align*}
&\sum_{k=0}^{K-1} V^{\pi^\star} - V^{\pi^k}  
\leq O\Bigg(  \text{DE}_1(\Psi,\Scal \times \Acal,\frac{1}{KH}) \log(\frac{KH\Ncal_{[]}((KH|\Scal|)^{-1},\Pcal,\|\cdot\|_\infty)}{\delta})\log(KH)\notag\\
&\quad+\sqrt{  \sum_{k=0}^{K-1}\var_{\pi^k} \cdot\text{DE}_1(\Psi,\Scal \times \Acal,\frac{1}{KH}) \log(\frac{KH\Ncal_{[]}((KH|\Scal|)^{-1},\Pcal,\|\cdot\|_\infty)}{\delta})\log(KH)}\Bigg)\,,
\end{align*}  
\end{small}

where $\Ncal_{[]}((KH|\Scal|)^{-1},\Pcal,\|\cdot\|_\infty)$ is the bracketing number defined in \pref{def: bracketing number}.
\end{corollary}
% \begin{proof}
%     See \pref{app:online_coro_infinite}.
% \end{proof}

% In this case, the regret bound would have $\log H$ dependence due to the braketing number analysis of MLE guarantee with infinite model class. We leave getting rid of this $\log$ dependence on the horizon $H$ as an open problem.
A specific example of the infinite model class is the tabular MDP, where $\mathcal{P}$ is the collection of all the conditional distributions over $\mathcal{S} \times \mathcal{A} \rightarrow \Delta(\mathcal{S})$. By \pref{corr:online_coro_infinite}, we also have a new regret bound for MBRL under the tabular MDP setting, which is nearly horizon-free and second-order. 

\begin{example}[Tabular MDPs] 
When specializing to tabular MDPs, use the fact that tabular MDP has $\ell_2$ Eluder dimension being at most $|\Scal| |\Acal|$ (Section D.1 in \cite{russo2013eluder}), $\ell_1$ dimension is upper bounded by $\ell_2$ dimension (\pref{lem:eluder-one-generalizes-two}), and use the standard $\epsilon$-net argument to show that $\Ncal_{[]}(\epsilon,\Pcal,\|\cdot\|_\infty)$ is upper-bounded by $(c/\epsilon)^{|\Scal|^2|\Acal|}$ (e.g., see  \cite{uehara2021pessimistic}), we can show that \pref{alg:mleonline} achieves the following regret bound for tabular MDP: with probability at least $1-\delta$, 

        \begin{align*}
\sum_k V^{\pi^\star} - V^{\pi^k} &\leq O\Big(|\Scal|^{1.5}|\Acal|\sqrt{\sum_k\var_{\pi^k}\cdot\log(\frac{KH|\Scal|}{\delta})\log(KH)}\notag\\
&\quad+|\Scal|^3|\Acal|^2\log(\frac{KH|\Scal|}{\delta})\log(KH) \Big)\,.
\end{align*}

\end{example}

% \zhiyong{Maybe this part should be revised.} We close this section by noting that the analysis of the above theorem follows the standard optimism in the face of uncertainty framework, and uses the same set of tools that we used for dealing with variances in the offline learning setting. The only noticeable difference is that we need to deal with the learned models' extrapolation errors, i.e., transferring the training error of $\widehat P^k$ under the training distribution $\sum_{i=0}^{k-1} d^{\pi^i}_h$ to the test error of $\widehat P^k$ under the test distribution $d^{\pi^k}_h$, which is done by the $\ell_1$ Eluder dimension, again in a very standard way. % \zhiyong{We may need to rewrite this SEC to be Eluder dimension.} 
%Overall, we strive for a minimalist approach for horizon-free and instance-dependent (first and second-order) regret bounds for online RL with non-linear function approximation.  

In summary, we have shown that a simple MLE-based MBRL algorithm is enough to achieve nearly horizon-free and second-order regret bounds under non-linear function approximation. 

\subsection{Proof Sketch of \pref{thm:online_theorem}}

Now we are ready to provide a proof sketch of \pref{thm:online_theorem} with the full proof deferred to \pref{app:online}. For ease of presentation, we use $\dimRL$ to denote $\text{DE}_1(\Psi,\Scal \times \Acal,1/KH)$, and ignore some $\log$ terms.

Overall, our analysis follows the general framework of optimism in the face of uncertainty, but with (1) careful analysis in leveraging the MLE generalization bound and (2) more refined proof in the training-to-testing distribution transfer via Eluder dimension. 

By standard MLE analysis, we can show w.p. $1-\delta$, for all $k\in[K-1]$, we have $P^\star \in \widehat\Pcal^k$, and
\begin{align}
     \sum_{i=0}^{k-1}\sum_{h=0}^{H-1}\mathbb H^2(P^\star(s_h^i,a_h^i)||\widehat P^k(s_h^i,a_h^i))\leq O(\log(K\left|\Pcal\right|/\delta))\,.\label{eqn: proof sketch online mle generalization}
\end{align}

From here, trivially applying training-to-testing distribution transfer via the Eluder dimension as previous works (e.g., \cite{wang2024more}) would cause poly-dependence on $H$. With new techniques detailed in \pref{app: eluder new}, which is one of our technical contributions and may be of independent interest, we can get:
there exists a set $\Kcal \subseteq [K-1]$ such that $|\Kcal| \leq O(\dimRL\log(K|\Pcal|/\delta))$, and
\begin{align}
   & \sum_{k \in [K-1]\setminus \Kcal}\sum_h \mathbb H^2\Big(P^\star(s_h^k,a_h^k)\Mid \widehat P^k\big(s_h^k,a_h^k)\Big)\notag\\
   &\leq O (\dimRL\cdot\log(K\left|\Pcal\right|/\delta)\log(KH))\,.\label{eqn: proof sketch online mle eluder}
\end{align}

Recall that $(\pi^k,\widehat{P}^k)\leftarrow \argmax_{\pi\in\Pi, P\in\widehat\Pcal^k} V_{0; P}^\pi(s_0)$, with the above realization guarantee $P^\star \in \widehat\Pcal^k$, we can get the following optimism guarantee: $V^\star_{0;P^\star}\leq \max_{\pi\in\Pi,P\in\widehat \Pcal^k} V^\pi_{0;P}=V^{\pi^k}_{0;\widehat P^k}$.

At this stage, one straight-forward way to proceed is to use the standard simulation lemma (\pref{lem:simulation}): 
\begin{align}
    &\sum_{k=0}^{K-1}V^{\pi^k}_{0;\widehat P^k}-V^{\pi^k}_{0;P^\star}\notag\\
    &\leq \sum_{k=0}^{K-1}\sum_{h=0}^{H-1} \EE_{s,a\sim d^{\pi^k}_h} \left[ \left|\EE_{s'\sim P^\star(s,a)} V^{\pi^k}_{h+1;\widehat P^k}(s') -\EE_{s'\sim \widehat P^k(s,a)} V^{\pi^k}_{h+1;\widehat P^k}(s')    \right|\right]. 
\end{align}
% Now we need to compare $\widehat P$ and $P^\star$ under the state-action distributions induced by $\pi^*$, which can be done by applying the standard simulation lemma (\pref{lem:simulation}): 
% \begin{align}
% V^{\pi^*} - V_{\widehat P}^{\pi^*} \leq \sum_{h=0}^{H-1} \EE_{s,a\sim d^{\pi^*}_h} \left[ \left|\EE_{s'\sim P^\star(s,a)} V^{\pi^*}_{h+1;\widehat P}(s') -\EE_{s'\sim \widehat P(s,a)} V^{\pi^*}_{h+1;\widehat P}(s')    \right|\right].\label{eqn:simulation_for_proof_sketch_offline}
% \end{align} 

However, from here, if we naively bound each term on the RHS via $\EE_{s,a\sim d^{\pi^k}_h} \|  P^\star(s,a) - \widehat P(s,a)     \|_1$, which is what previous works such as \cite{uehara2021pessimistic} did exactly, we would end up paying a linear horizon dependence $H$ due to the summation over $H$ on the RHS the above expression. Given the mean-to-variance lemma (\pref{lem: mean to variance}), we may consider using it to bound the difference between two means $\EE_{s'\sim P^\star(s,a)} V^{\pi^k}_{h+1;\widehat P^k}(s') -\EE_{s'\sim \widehat P^k(s,a)} V^{\pi^k}_{h+1;\widehat P^k}(s')$. This still can not work if we start from here, because we would eventually get $\sum_k\sum_h\EE_{s,a\sim d^{\pi^k}_h}[\mathbb H^2(P^\star(s,a)||\widehat P^k(s,a))]$ terms, which can not be further upper bounded easily with the MLE generalization guarantee.

To achieve horizon-free and second-order bounds, we need a novel and more careful analysis. 

First, we carefully decompose and upper bound the regret in $\Tilde{\Kcal}:=[K-1]\setminus \Kcal$ w.h.p. as follows using Bernstain's inequality (for regret in $\Kcal$ we simply upper bound it by $|\Kcal|$)
\begin{small}
    \begin{align}
    % &\sum_{k\in\Tilde{\Kcal}}V^{\pi^k}_{0;\widehat P^k}(s_h^k)-V^{\pi^k}_{0;P^\star}=
    &\sum_{k\in\Tilde{\Kcal}}\left(V^{\pi^k}_{0;\widehat P^k}(s_h^k)-\sum_{h=0}^{H-1} r(s_h^k, a_h^k)\right)+\sum_{k\in\Tilde{\Kcal}}\left(\sum_{h=0}^{H-1} r(s_h^k, a_h^k)-V^{\pi^k}_{0;P^\star} \right)\lesssim \sqrt{\sum_{k\in\Tilde{\Kcal}}\sum_h \big(\VV_{P^\star} V_{h+1; \widehat P^k}^{\pi^k}\big)(s_h^k,a_h^k)}\notag\\
    &+ \sum_{k\in\Tilde{\Kcal}}\sum_h  \left|\EE_{s'\sim \widehat P^k(s_h^k, a_h^k)}V^{\pi^k}_{h+1;\widehat P^k}(s') - \EE_{s'\sim P^*(s_h^k, a_h^k)}V^{\pi^k}_{{h+1};\widehat P^k}(s')\right|+\sqrt{\sum_{k}\var_{\pi^k}\log(1/\delta)}\,.\label{eqn:proof sketch online decompose}
\end{align}

\end{small}

% \looseness=-1

% \wen{@zhiyong, defer the full proof to appendix, and here try to write a proof sketch that contains some of the kep lemmas and steps, i.e., basically write a short version which you would hope readers to read first and that would help them when they dive into the details of the full proof...}

Then, we bound the difference of two means $\EE_{s'\sim \widehat P^k(s_h^k, a_h^k)}V^{\pi^k}_{h+1;\widehat P^k}(s') - \EE_{s'\sim P^*(s_h^k, a_h^k)}V^{\pi^k}_{{h+1};\widehat P^k}(s')$ using variances and the triangle discrimination (see \pref{lem: mean to variance} for more details), together with the fact that $D_\triangle \leq 4 \mathbb H^2$, and information processing inequality on the squared Hellinger distance, we have
\begin{small}
    \begin{align*}
&\lvert\EE_{s'\sim \widehat P^k(s_h^k, a_h^k)}V^{\pi^k}_{h+1;\widehat P^k}(s') - \EE_{s'\sim P^*(s_h^k, a_h^k)}V^{\pi^k}_{{h+1};\widehat P^k}(s')\rvert\\
&\quad \leq O\Big(\sqrt{\big(\VV_{P^\star} V_{h+1; \widehat P^k}^{\pi^k}\big)(s_h^k, a_h^k)D_\triangle\Big(V_{h+1; \widehat P^k}^{\pi^k}\big(s'\sim P^\star(s_h^k, a_h^k)\big)\Mid  V_{h+1; \widehat P^k}^{\pi^k}(s'\sim \widehat P^k\big(s_h^k, a_h^k)\big)\Big)}  \\
& \qquad \qquad + D_\triangle\Big(V_{h+1; \widehat P^k}^{\pi^k}\big(s'\sim P^\star(s_h^k, a_h^k)\big)\Mid  V_{h+1; \widehat P^k}^{\pi^k}(s'\sim \widehat P^k\big(s_h^k, a_h^k)\big)\Big)\Big)\\
&\leq O\Big(\sqrt{\big(\VV_{P^\star} V_{h+1; \widehat P^k}^{\pi^k}\big)(s_h^k, a_h^k)\mathbb H^2\Big(P^\star(s_h^k, a_h^k)\Mid \widehat P^k\big(s_h^k, a_h^k)\Big)}+\mathbb H^2\Big(P^\star(s_h^k, a_h^k)\Mid \widehat P^k\big(s_h^k, a_h^k)\Big)
\Big)
\end{align*} 
\end{small}
where we denote $V^{\pi^*}_{h+1;\widehat P}(s' \sim P^\star(s,a))$ as the distribution of the random variable $ V^{\pi^*}_{h+1;\widehat P}(s')$ with $s'\sim P^\star(s,a)$. This is the key lemma used by \cite{wang2024more} to show distributional RL can achieve second-order bounds. We show that this is also crucial for achieving a horizon-free bound.  

Then, summing up over $k, h$, with Cauchy-Schwartz and the MLE generalization bound via Eluder dimension in \pref{eqn: proof sketch online mle eluder}, we have
\begin{small}
    \begin{align}
    &\sum_{k\in\Tilde{\Kcal}}\sum_h  \left|\EE_{s'\sim \widehat P^k(s_h^k, a_h^k)}V^{\pi^k}_{h+1;\widehat P^k}(s') - \EE_{s'\sim P^*(s_h^k, a_h^k)}V^{\pi^k}_{{h+1};\widehat P^k}(s')\right|\leq O\Big(\sum_{k\in\Tilde{\Kcal}}\sum_h\mathbb H^2\Big(P^\star(s_h^k, a_h^k)\Mid \widehat P^k\big(s_h^k, a_h^k)\Big)\notag\\
&+\sqrt{\sum_{k\in\Tilde{\Kcal}}\sum_h\big(\VV_{P^\star} V_{h+1; \widehat P^k}^{\pi^k}\big)(s_h^k, a_h^k)\sum_{k\in\Tilde{\Kcal}}\sum_h\mathbb H^2\Big(P^\star(s_h^k, a_h^k)\Mid \widehat P^k\big(s_h^k, a_h^k)\Big)}
\Big)\notag\\
&\leq O\Big(\sqrt{\sum_{k\in\Tilde{\Kcal}}\sum_h\big(\VV_{P^\star} V_{h+1; \widehat P^k}^{\pi^k}\big)(s_h^k, a_h^k)\dimRL \log(K\left|\Pcal\right|/\delta)\log(KH)}+\dimRL \log(K\left|\Pcal\right|/\delta)\log(KH)\Big)\,.\label{eqn: proof sketch online sum mean-variance final}
\end{align}
\end{small}

Note that we have $\big(\VV_{P^\star} V_{h+1; \widehat P^k}^{\pi^k}\big)(s_h^k, a_h^k)$ depending on $\widehat P^k$. To get a second-order bound, we need to convert it to the variance under ground truth transition $P^\star$, and we want to do it without incurring any $H$ dependence. \emph{This is another key difference from \cite{wang2024more}.}

We aim to replace $\big(\VV_{P^\star} V_{h+1; \widehat P^k}^{\pi^k}\big)(s_h^k, a_h^k)$ by $\big(\VV_{P^\star} V_{h+1}^{\pi^k}\big)(s_h^k, a_h^k)$ which is the variance under $P^\star$ (recall that $V^{\pi}$ is the value function of $\pi$ under $P^\star$), and we want to control the difference\\ $\big(\VV_{P^\star} \left( V_{h+1; \widehat P^k}^{\pi^k} - V_{h+1}^{\pi^k} \right) \big)(s_h^k, a_h^k)$.
To do so, we need to bound the $2^m$ moment of the difference $V_{h+1; \widehat P^k}^{\pi^k} - V_{h+1}^{\pi^k}$ following the strategy in \cite{zhang2021reinforcement,zhou2022computationally,zhao2023variance}. Let us define the following terms:
\begin{small}
    \begin{align}
    &A:=\sum_{k\in\tilde\Kcal}\sum_h \left[\big(\VV_{P^\star} V_{h+1; \widehat P^k}^{\pi^k}\big)(s_h^k,a_h^k)\right], C_m:=\sum_{k\in\tilde\Kcal}\sum_h\left[\big(\VV_{P^\star}( V_{h+1; \widehat P^k}^{\pi^k}-V_{h+1}^{\pi^k})^{2^m}\big)(s_h^k,a_h^k)\right], \notag\\
&B:=\sum_{k\in\tilde\Kcal}\sum_h\left[\big(\VV_{P^\star} V_{h+1}^{\pi^k}\big)(s_h^k,a_h^k)\right],
    G := \sqrt{A\cdot\dimRL \log(\frac{K\left|\Pcal\right|}{\delta})\log(KH)}+\dimRL \log(\frac{K\left|\Pcal\right|}{\delta})\log(KH)\notag\,.
\end{align}
\end{small}

With the fact $\VV_{P^\star}(a+b)\leq 2\VV_{P^\star}(a)+2\VV_{P^\star}(b)$ we have $A\leq 2B+2C_0$. For $C_m$, we prove that w.h.p. it has the recursive form $C_m\lesssim 2^mG+\sqrt{\log(1/\delta)C_{m+1}}+\log(1/\delta)$, during which process we also leverage the above \pref{eqn: proof sketch online sum mean-variance final} and some careful analysis (detailed in \pref{app:online}). Then, with the recursion lemma (\pref{lem:recursion bound C_m}), we can get $C_0\lesssim G$, which further gives us
    \begin{align}
    A&\lesssim B+\dimRL \log(\frac{K\left|\Pcal\right|}{\delta})\log(KH) +\sqrt{A\cdot\dimRL \log(\frac{K\left|\Pcal\right|}{\delta})\log(KH)} \notag\\
    &\leq O\big(B+\dimRL \log(\frac{K\left|\Pcal\right|}{\delta})\log(KH)\big)\,,\notag
\end{align}

where in the last step we use the fact $x\leq 2a+b^2$ if $x\leq a+b\sqrt{x}$.
Finally, we note that $B \leq O(\sum_k \var_{\pi^k}+\log(1/\delta))$ w.h.p.. Plugging the upper bound of $A$ back into \pref{eqn: proof sketch online sum mean-variance final} and then to \pref{eqn:proof sketch online decompose}, we conclude the proof.

\section{Offline Setting}
\label{sec:offline}

For the offline setting, we directly analyze the Constrained Pessimism Policy Optimization (CPPO-LR) algorithm (\pref{alg:mleoffline}) proposed by \cite{uehara2021pessimistic}. We first explain the algorithm and then present its performance gap guarantee in finding the comparator policy $\pi^*$. 

\pref{alg:mleoffline} splits the offline trajectory data that contains $K$ trajectories into a dataset of $(s,a,s')$ tuples (note that in total we have $n:=KH$ many tuples)  which is used to perform maximum likelihood estimation $\max_{ \Tilde P\in \Pcal } \sum_{i=1}^n \log \Tilde P(s'_i | s_i, a_i)$. It then builds a version space $\widehat \Pcal$ which contains models $P\in \Pcal$ whose log data likelihood is not  below by too much than that of the MLE estimator. The threshold for the version space is constructed so that with high probability, $P^\star \in \widehat \Pcal$. Once we build a version space, we perform pessimistic planning to compute $\widehat \pi$.

% We denote $V_0^\pi(P)$ as the value function $V_0^\pi(s_0)$ under the MDP with transition $P$. We denote $(\VV_P f)(s,a):=(Pf^2)(s,a)-(Pf)^2(s,a)$, where the square is point-wise. \wen{move these notations to the preliminary section }
\begin{algorithm}[t]%
\caption{(\cite{uehara2021pessimistic}) Constrained Pessimistic Policy Optimization with Likelihood-Ratio based constraints (CPPO-LR)}
\label{alg:mleoffline}
\resizebox{0.93\columnwidth}{!}{
\begin{minipage}{\columnwidth}
\begin{algorithmic}[1]
    \STATE\textbf{Input:} dataset $\Dcal = \{s,a,s'\}$, model class $\Pcal$, policy class $\Pi$, confidence parameter $\delta\in(0,1)$, threshold $\beta$.
    \STATE Calculate the confidence set based on the offline dataset:
    \begin{align*}
        \widehat\Pcal = \braces{ P\in\Pcal: \sum_{i=1}^n \log P(s_i'|s_i,a_i)\geq \max_{\Tilde{P}\in\Pcal}\sum_{i=1}^n \log \Tilde P (s_i'|s_i,a_i)-\beta}.
    \end{align*} \label{offline conf}
    % \STATE Set $(\widehat\pi,\widehat{P})\leftarrow \argmax_{\pi\in\Pi}\argmin_{P\in\widehat\Pcal} V_{0; P}^\pi(s_0)$.
    \STATE \textbf{Output:} $\hat\pi\leftarrow \argmax_{\pi\in\Pi}\min_{P\in\widehat\Pcal} V_{0; P}^\pi(s_0)$.
    \end{algorithmic}
    \end{minipage}}
\end{algorithm}

We first define the single policy coverage condition as follows. 
% \wen{let's use the following concentrability coeffocient}
\begin{definition}[Single policy coverage] \label{def: coverage offline} Given any comparator policy $\pi^*$, denote the data-dependent single policy concentrability coefficient $C^{\pi^*}_{\Dcal}$ as follows:
% \begin{align*}
% C^{\pi^*} := \max_{h, P \in \Pcal}  \frac{ \EE_{s,a\sim d^{\pi^*}_h} \mathbb{H}^2\left( P(s,a) \Mid P^\star(s,a) \right)   }{ \EE_{s,a\sim d^{\pi^b}_h} \mathbb H^2\left(   P(s,a)  \Mid   P^\star(s,a)  \right)    } . 
% \end{align*} 
\begin{small}
    \begin{align*}
C^{\pi^*}_{\Dcal} := \max_{h, P \in \Pcal}  \frac{ \EE_{s,a\sim d^{\pi^*}_h} \mathbb{H}^2\left( P(s,a) \Mid P^\star(s,a) \right)   }{ 1/K \sum_{k=1}^K \mathbb H^2\left(   P(s_h^k,a_h^k)  \Mid   P^\star(s_h^k,a_h^k)  \right)    }.
\end{align*} 
\end{small}
We assume w.p. at least $1-\delta$ over the randomness of the generation of $\Dcal$, we have $C^{\pi^*}_{\Dcal}\leq C^{\pi^*}$.
\end{definition} 
The existence of $C^{\pi^*}$ is certainly an assumption. We now give an example in the tabular MDP where we show that if the data is generated from some fixed behavior policy $\pi^b$ which has non-trivial probability of visiting every state-action pair, then we can show the existence of $C^{\pi^*}$.
\begin{example}[Tabular MDP with good behavior policy coverage]\label{ex: offline coverage}
    If the $K$ trajectories are collected $i.i.d.$ with a fixed behavior policy $\pi^b$, and $d^{\pi^b}_h(s,a) \geq \rho_{\min}, \forall s,a, h$ (similar to \cite{ren2021nearly}), then we have: if $K$ is large enough, i.e., $K\geq 2\log(|\Scal||\Acal|H)/\rho_{\min}^2$, w.p. at least $1-\delta$, $C^{\pi^*}_{\Dcal}\leq 2/\rho_{\min}$. 
\end{example}

Our coverage definition (\pref{def: coverage offline}) shares similar spirits as the one in \cite{ye2024corruption}. It reflects how well the state-action samples in the offline dataset $\Dcal$ cover the state-action pairs induced by the comparator policy $\pi^\star$. It is different from the coverage definition in \cite{uehara2021pessimistic} in which the denominator is {\small$\EE_{s,a\sim d^{\pi^b}_h} \mathbb{H}^2\left( P(s,a) \Mid P^\star(s,a) \right)$} where $\pi^b$ is the fixed behavior policy used to collect $\Dcal$. This definition does not apply in our setting since $\Dcal$ is not necessarily generated by some underlying fixed behavior policy. On the other hand, our horizon-free result does not hold in the setting of \cite{uehara2021pessimistic} where $\Dcal$ is collected with a fixed behavior policy $\pi^b$ with the concentrability coefficient defined in their way. We leave the derivation of horizon-free results in the setting from \cite{uehara2021pessimistic}  as a future work.

Now we are ready to present the main theorem of \pref{alg:mleoffline}, which provides a tighter performance gap than that by \cite{uehara2021pessimistic}.
%Now we present the sample complexity bound 
%prove the following theorem for the suboptimality gap of the output policy $\hat\pi$ of Algo.\ref{alg:mleoffline} against any high quality comparator policy $\pi^*$. 

\begin{theorem}[Performance gap of \pref{alg:mleoffline}]\label{thm:mleoffline}
For any $\delta\in(0,1)$, let $\beta=4\log(|\Pcal|/\delta)$, w.p. at least $1-\delta$, \pref{alg:mleoffline} learns a policy $\widehat\pi$ that enjoys the following performance gap with respect to any comparator policy $\pi^*$:
\begin{equation*}
    V^{\pi^*}-V^{\widehat\pi} \leq O\left(\sqrt{{C^{\pi^*} \var_{\pi^*}\log(|\Pcal|/\delta)}/{K} } + {C^{\pi^*}\log(|\Pcal|/\delta)}/{K}\right)\,.
\end{equation*}
%where $
%    C^{\pi^*} = \max_h\norm{\frac{P^{\pi^*}_h}{P^{\pi^b}_h}}_\infty
%$ denotes the single-policy concentrability.
\end{theorem}
% \begin{proof}
%     See \pref{app:offline}.
% \end{proof}

Comparing to the theorem (Theorem 2) of CPPO-LR from \cite{uehara2021pessimistic}, our bound has two improvements. First, our bound is horizon-free (not even any $\log (H)$ dependence), while the bound in  \cite{uehara2021pessimistic} has $\text{poly}(H)$ dependence. Second, our bound scales with $\var_{\pi^*} \in [0,1]$, which can be small when $\var_{\pi^*} \ll 1$. For deterministic system and policy $\pi^*$, we have $\var_{\pi^*} = 0$ which means the sample complexity now scales at a faster rate $C^{\pi^*} / K$. The proof is in \pref{app:offline}.

We show that the same algorithm can achieve $1/K$ rate when $P^\star$ is deterministic (but rewards could be random, and the algorithm does not need to know the condition that $P^\star$ is deterministic).

\begin{corollary}[${C^{\pi^*}}/{K}$ performance gap of \pref{alg:mleoffline} with deterministic transitions] \label{corr:coro_faster}
When the ground truth transition $P^\star$ of the MDP is deterministic, for any $\delta\in(0,1)$, let $\beta=4\log(|\Pcal|/\delta)$, w.p. at least $1-\delta$, \pref{alg:mleoffline} learns a policy $\widehat\pi$ that enjoys the following performance gap with respect to any comparator policy $\pi^*$:
\begin{equation*}
    V^{\pi^*}-V^{\widehat\pi} \leq O\left({C^{\pi^*}\log(|\Pcal|/\delta)}/{K}\right)\,.
\end{equation*} 
% where $K$ is the total number of trajectories in the offline dataset. 
\end{corollary}

% \begin{proof}
%     See \pref{app:offline_coro_faster}. 
% \end{proof}

 %Note that the condition for the above corollary only relies on deterministic transition dynamics, but the rewards and policy can be random.

For infinite model class $\Pcal$, we have a similar result in the following corollary. 

% which directly follows by changing the $4\log\left(\frac{\left|\Pcal\right|}{\delta}\right)$ in Line \ref{offline conf} of \pref{alg:mleoffline} to $7\log(\Ncal_{[]}((NH|\Scal|)^{-1},\Pcal,\|\cdot\|_\infty)/\delta)$ (the bracketing number is defined in \pref{def: bracketing number}), 

\begin{corollary}[Performance gap of \pref{alg:mleoffline} with infinite model class $\Pcal$] \label{corr:offline_coro_infinite}
When the model class $\Pcal$ is infinite, for any $\delta\in(0,1)$, let $\beta=7\log(\Ncal_{[]}((KH|\Scal|)^{-1},\Pcal,\|\cdot\|_\infty)/\delta)$, w.p. at least $1-\delta$, \pref{alg:mleoffline} learns a policy $\widehat\pi$ that enjoys the following PAC bound w.r.t. any comparator policy $\pi^*$:
\begin{footnotesize}
           \begin{align*}
    V^{\pi^*}-V^{\widehat\pi}\leq O\left(\sqrt{\frac{C^{\pi^*} \var_{\pi^*}\log(\Ncal_{[]}((KH|\Scal|)^{-1},\Pcal,\|\cdot\|_\infty)/\delta)}{K} } + \frac{C^{\pi^*}\log(\Ncal_{[]}((KH|\Scal|)^{-1},\Pcal,\|\cdot\|_\infty)/\delta)}{K}\right),
\end{align*} 
\end{footnotesize}
 where $\Ncal_{.[]}((KH|\Scal|)^{-1},\Pcal,\|\cdot\|_\infty)$ is the bracketing number defined in \pref{def: bracketing number}.
\end{corollary}
% \begin{proof}
%     See \pref{app:offline_coro_infinite}.
% \end{proof}

Our next example gives the explicit performance gap bound for tabular MDPs. 

%\wen{@zhiyong, add a paragraph for discussion about this theorem and compare it to the theorem in Uehara and Sun 2021.}
\begin{example}[Tabular MDPs] 
For tabular MDPs, we have $\Ncal_{[]}(\epsilon,\Pcal,\|\cdot\|_\infty)$ upper-bounded by $(c/\epsilon)^{|\Scal|^2|\Acal|}$ (e.g., see \cite{uehara2021pessimistic}). Then with probability at least $1-\delta$, let $\beta=7\log(\Ncal_{[]}((KH|\Scal|)^{-1},\Pcal,\|\cdot\|_\infty)/\delta)$, \pref{alg:mleoffline} learns a policy $\widehat\pi$ satisfying the following performance gap with respect to any comparator policy $\pi^*$:
\begin{align}
    V^{\pi^*}-V^{\widehat\pi} &\leq O\Bigg(|\Scal|\sqrt{{|\Acal| C^{\pi^*} \var_{\pi^*}\log(KH|\Scal|/\delta)}/{K} } \notag\\
    &+ {|\Scal|^2|\Acal|C^{\pi^*}\log(KH|\Scal|/\delta)}/{K}\Bigg),
\end{align}
\end{example}

The closest result to us is from \cite{ren2021nearly}, which analyzes the MBRL for tabular MDPs and obtains a performance gap $\tilde O(\sqrt{\frac{1}{Kd_m}} + \frac{|\mathcal{S}|}{Kd_m})$, where $d_m$ is the minimum visiting probability for the behavior policy to visit each state and action. Note that their result is not instance-dependent, which makes their gap only $\tilde O(1/\sqrt{K})$ even when the environment is deterministic and $\pi^*$ is deterministic. In a sharp contrast, our analysis shows a better $\tilde O(1/K)$ gap under the deterministic environment. Our result would still have the $\log H$ dependence, and we leave getting rid of this logarithmic dependence on the horizon $H$ as an open problem.

\chapter{Provable Zero-Shot Generalization in Offline Reinforcement Learning}\label{chapter ZSG}
In this chapter, we study offline reinforcement learning (RL) with zero-shot generalization property (ZSG), where the agent has access to an offline dataset including experiences from different environments, and the goal of the agent is to train a policy over the training environments which performs well on test environments without further interaction. Existing work showed that classical offline RL fails to generalize to new, unseen environments. We propose pessimistic empirical risk minimization (PERM) and pessimistic proximal policy optimization (PPPO), which leverage pessimistic policy evaluation to guide policy learning and enhance generalization. We show that both PERM and PPPO are capable of finding a near-optimal policy with ZSG. Our result serves as a first step in understanding the foundation of the generalization phenomenon in offline reinforcement learning. This chapter is based on our publication \cite{wang2024towards}.

\section{Introduction}

Offline reinforcement learning (RL) has become increasingly significant in modern RL because it eliminates the need for direct interaction between the agent and the environment; instead, it relies solely on learning from an offline training dataset. However, in practical applications, the offline training dataset often originates from a different environment than the one of interest. This discrepancy necessitates evaluating RL agents in a generalization setting, where the training involves a finite number of environments drawn from a specific distribution, and the testing is conducted on a distinct set of environments from the same or different distribution. This scenario is commonly referred to as the zero-shot generalization (ZSG) challenge which has been studied in online RL\cite{Rajeswaran2017TowardsGA, Machado2018RevisitingTA, Justesen2018IlluminatingGI, Packer2018AssessingGI, Zhang2018ADO, Zhang2018ASO}, as the agent receives no training data from the environments it is tested on.

A number of recent empirical studies \cite{mediratta2023generalization, yang2023essential,mazoure2022improving} have recognized this challenge and introduced various offline RL methodologies that are capable of ZSG. Notwithstanding the lack of theoretical backing, these methods are somewhat restrictive; for instance, some are only effective for environments that vary solely in observations\cite{mazoure2022improving}, while others are confined to the realm of imitation learning\cite{yang2023essential}, thus limiting their applicability to a comprehensive framework of offline RL with ZSG capabilities. Concurrently, theoretical advancements \cite{bose2024offline,ishfaq2024offline} in this domain have explored multi-task offline RL by focusing on representation learning. These approaches endeavor to derive a low-rank representation of states and actions, which inherently requires additional interactions with the downstream tasks to effectively formulate policies based on these representations. Therefore, we raise a natural question:
\begin{center}
    \emph{Can we design provable offline RL with zero-shot generalization ability?}
\end{center}

% \begin{center}
%     \emph{What is essential for offline RL to achieve zero-shot generalization ability?}
% \end{center}

We propose novel offline RL frameworks that achieve ZSG to address this question affirmatively. Our contributions are listed as follows.
\begin{itemize}[leftmargin = *]
   \item We first analyze when existing offline RL approaches fail to generalize without further algorithm modifications. Specifically, we prove that if the offline dataset does not contain context information, then it is impossible for vanilla RL that equips a Markovian policy to achieve a ZSG property. We show that the offline dataset from a contextual Markov Decision Process (MDP) is not distinguishable from a vanilla MDP which is the average of contextual Markov Decision Process over all contexts. Such an analysis verifies the necessity of new RL methods with ZSG property. 
    \item We propose two meta-algorithms called pessimistic empirical risk minimization (PERM) and pessimistic proximal policy optimization (PPPO) that enable ZSG for offline RL \cite{jin2021pessimism}. In detail, both of our algorithms take a pessimistic policy evaluation (PPE) oracle as its component and output policies based on offline datasets from multiple environments. Our result shows that the sub-optimalities of the output policies are bounded by both the supervised learning error, which is controlled by the number of different environments, and the reinforcement learning error, which is controlled by the coverage of the offline dataset to the optimal policy. Please refer to Table \ref{tab:example}
for a summary of our results. To the best of our knowledge, our proposed algorithms are the first offline RL methods that provably enjoy the ZSG property.

    %we further specify our analysis to a more concrete setting called linear MDPs \cite{yang2019sample, jin2019provably}. We show that under the proper coverage assumptions made on the offline dataset distribution, both of our algorithms enjoy the suboptimality gap $O(n^{-1/2} + K^{-1/2}\cdot C_n^*)$, where $n$ is the number of environments in offline dataset, $K$ is the number of trajectories for each environment and $C_n^*$ is a mixed coverage parameter over $n$ environments. \footnote{Here we only include $n,K, C_n^*$ in our bound for simplicity.}. Our result generalizes the convergence result for offline RL in single-environment \cite{jin2021pessimism}. 
\end{itemize}

\begin{table*}[t!]
\centering
\caption{Summary of our algorithms and their suboptimality gaps, where $\cA$ is the action space, $H$ is the length of episode, $n$ is the number of environments in the offline dataset. Note that in the multi-environment setting, $\pi^*$ is the near-optimal policy w.r.t. expectation (defined in Section \ref{sec:setting}). $\mathcal{N}$ is the covering number of the policy space $\Pi$ w.r.t. distance $\mathrm d(\pi^1,\pi^2) = \max_{s\in \mathcal{S}, h \in [H]} \|\pi^1_h(\cdot|s) - \pi^2_h (\cdot|s)\|_{1}$. The uncertainty quantifier $\Gamma_{i,h}$ are tailored with the oracle return in the corresponding algorithms (details are in Section \ref{sec:withcontext}). }
\label{tab:example}

\begin{tabular}{|c|c|}
\hline
\textbf{Algorithm} & \textbf{Suboptimality Gap}\\
\hline
{PERM (our Algo.\ref{alg:erm})} & 
{
  $\sqrt{\log(\mathcal{N})/n} + n^{-1}\sum_{i=1}^n \sum_{h=1}^H$
  $\EE_{i,\pi^*}\big[\Gamma_{i,h}(s_h,a_h)\,\big\vert\, s_1=x_1\big]$
}
\\
\hline
{PPPO (our Algo.\ref{alg:modelfree})} &
{
  $\sqrt{\log|\actions|\,H^2/n} + n^{-1}\sum_{i=1}^n \sum_{h=1}^H$
  $\EE_{i,\pi^*}\big[\Gamma_{i,h}(s_h,a_h)\,\big\vert\, s_1=x_1\big]$
}
\\
\hline
\end{tabular}
\end{table*}

% \vspace{-0.3cm}
\noindent\textbf{Notation} 
We use lower case letters to denote scalars, and use lower and upper case bold face letters to denote vectors and matrices respectively. We denote by $[n]$ the set $\{1,\dots, n\}$. For a vector $\xb\in \RR^d$ and a positive semi-definite matrix $\bSigma\in \RR^{d\times d}$, we denote by $\|\xb\|_2$ the vector's Euclidean norm and define $\|\xb\|_{\bSigma}=\sqrt{\xb^\top\bSigma\xb}$. For two positive sequences $\{a_n\}$ and $\{b_n\}$ with $n=1,2,\dots$, 
    we write $a_n=O(b_n)$ if there exists an absolute constant $C>0$ such that $a_n\leq Cb_n$ holds for all $n\ge 1$ and write $a_n=\Omega(b_n)$ if there exists an absolute constant $C>0$ such that $a_n\geq Cb_n$ holds for all $n\ge 1$. We use $\tilde O(\cdot)$ to further hide the polylogarithmic factors. 
We use $(x_i)_{i=1}^n$ to denote sequence $(x_1, ..., x_n)$, and we use $\{x_i\}_{i=1}^n$ to denote the set $\{x_1, ...,x_n\}$. We use $\text{KL}(p\|q)$ to denote the KL distance between distributions $p$ and $q$, defined as $\int p\log(p/q)$. We use $\EE[x],\mathbb{V}[x] $ to denote expectation and variance of a random variable $x$.

\section{Preliminaries}\label{sec:setting}

\noindent\textbf{Contextual MDP} We study \emph{contextual episodic MDPs}, where each MDP $\cM_c$ is associated with a context $c \in C$ belongs to the context space $C$. Furthermore, $\cM_c = \{M_{c,h}\}_{h=1}^H$ consists of $H$ different individual MDPs, where each individual MDP $M_{c,h}:=(\cS, \cA, P_{c,h}(s'|s,a), r_{c,h}(s,a))$. Here $\cS$ denotes the state space, $\cA$ denotes the action space, $P_{c,h}$ denotes the transition function and $r_{c,h}$ denotes the reward function at stage $h$. We assume the starting state for each $\cM_c$ is the same state $x_1$. In this work, we interchangeablely use ``environment" or MDP to denote the MDP $\cM_c$ with different contexts.

\noindent\textbf{Policy and value function}
We denote the policy $\pi_h$ at stage $h$ as a mapping $\cS \rightarrow \Delta(\cA)$, which maps the current state to a distribution over the action space. We use $\pi = \{\pi_h\}_{h=1}^H$ to denote their collection. Then for any episodic MDP $\cM$, we define the value function for some policy $\pi$ as
\begin{small}
       \begin{align}
    &V_{\cM,h}^{\pi}(x):=\EE[ r_h+...+r_H|s_h = x, a_{h'}\sim \pi_{h'}, r_{h'}\sim r_{h'}(s_{h'}, a_{h'}), s_{h'+1}\sim P_{h'}(\cdot|s_{h'}, a_{h'}),~h'\geq h]\,,\notag\\
     &Q_{M,h}^{\pi}(x,a):=\EE[r_h+...+r_H|s_h = x,a_h = a, r_h\sim r_h(s_h,a_h), s_{h'}\sim P_{h'-1}(\cdot|s_{h'-1}, a_{h'-1}),  a_{h'}\sim \pi_{h'}, \notag\\
     &\quad r_{h'}\sim r_{h'}(s_{h'}, a_{h'}),~h'\geq h+1].\notag
\end{align} 
\end{small}
For any individual MDP $M$ with reward $r$ and transition dynamic $P$, we denote its Bellman operator $[\BB_{M}f](x,a)$ as $[\BB_{M}f](s,a):=\EE[r_h(s,a) + f(s')|s'\sim P(\cdot|s,a)]$. Then we have the well-known Bellman equation
\begin{small}
    \begin{align}
    &V_{\cM,h}^{\pi}(x)\notag = \la Q_{\cM,h}^\pi(x, \cdot), \pi_h(\cdot|x)\ra_{\cA},\ Q_{\cM,h}^{\pi}(x,a) = [\BB_{M_h} V_{\cM,h+1}^{\pi}](x,a).\notag
\end{align}
\end{small}

For simplicity, we use $V_{c,h}^\pi, Q_{c,h}^\pi, \BB_{c,h}$ to denote $V_{\cM_c,h}^\pi, Q_{\cM_c,h}^\pi, \BB_{M_{c,h}}$. We also use $\PP_c$ to denote $\PP_{\cM_c}$, the joint distribution of any potential objects under the $\cM_c$ episodic MDP. We would like to find the near-optimal policy $\pi^{*}$ w.r.t. expectation, i.e., $\pi^*:=\argmax_{\pi \in \Pi}\EE_{c\sim C}V_{c,1}^\pi(x_c)$, where $\Pi$ is the set of collection of Markovian policies, and with a little abuse of notation, we use $\EE_{c\sim C}$ to denote the expectation taken w.r.t. the i.i.d. sampling of context $c$ from the context space. Then our goal is to develop the \emph{generalizable RL} with small \emph{zero-shot generalization gap (ZSG gap)}, defined as follows:
\begin{small}
    \begin{align}
    \text{SubOpt}(\pi):=\EE_{c\sim C}\big[V_{c,1}^{\pi^*}(x_1)\big] - \EE_{c\sim C}\big[V_{c,1}^\pi(x_1)\big].\notag
\end{align}
\end{small}

\begin{remark}
We briefly compare generalizable RL with several related settings. Robust RL \cite{pinto2017robust} aims to find the best policy for the worst-case environment, whereas generalizable RL seeks a policy that performs well in the average-case environment. Meta-RL \cite{beck2023survey} enables few-shot adaptation to new environments, either through policy updates \cite{finn2017model} or via history-dependent policies \cite{duan2016rl}. In contrast, generalizable RL primarily focuses on the zero-shot setting. In the general POMDP framework \cite{cassandra1994acting}, agents need to maintain history-dependent policies to implicitly infer environment information, while generalizable RL aims to discover a single state-dependent policy that generalizes well across all environments.
\end{remark}

\begin{remark}
\cite{ye2023power} showed that in online RL, for a certain family of contextual MDPs, it is inherently impossible to determine an optimal policy for each individual MDP. Given that offline RL poses greater challenges than its online counterpart, this impossibility extends to finding optimal policies for each MDP in a zero-shot offline RL setting as well, which justifies our optimization objective on the ZSG gap. Moreover, \cite{ye2023power} showed that the few-shot RL is able to find the optimal policy for individual MDPs. Clearly, such a setting is stronger than ours, and the additional interactions are often hard to be satisfied in real-world practice. We leave the study of such a setting for future work.   
\end{remark}

\noindent\textbf{Offline RL data collection process}
The data collection process is as follows. An experimenter i.i.d. samples number $n$ of contextual episodic MDP $M_i$ from the context set (\emph{e.g.}, $i\sim C)$. For each episodic MDP $M_i$, the experimenter collects dataset $\cD_i:=\{(x_{i,h}^\tau, a_{i,h}^\tau, r_{i,h}^\tau)_{h=1}^H\}_{\tau=1}^{K}$ which includes $K$ trajectories. Note that the action $a_{i,h}^\tau$ selected by the experimenter can be arbitrary, and it does not need to follow a specific behavior policy \cite{jin2021pessimism}. 
% We assume that we have an access to an offline dataset $\cD := \{\cD_{i}\}_{i=1}^{n}$, where each 
We assume that $\cD_{i}$ is compliant with the episodic MDP $\cM_{i}$, which is defined as follows. 

\begin{definition}[\cite{jin2021pessimism}]\label{def:com}
    % For an offline dataset $\hat \cD = \{(s,a,s',r)\}\subseteq \cS \times \cA \times \cS \times [0,1]$, we say it is compliant with an individual MDP $M$ if 
    % \begin{align}
    %     \PP_{\hat \cD}(r = r', s' = x'|s = x,a = a) = \PP(r'\sim r(x,a), s'\sim\cP(\cdot|x,a)).\notag
    % \end{align}
    For $\cD_i:=\{(x_{i,h}^\tau, a_{i,h}^\tau, r_{i,h}^\tau)_{h=1}^H\}_{\tau=1}^{K}$, let $\PP_{\cD_i}$ be the joint distribution of the data collecting process. We say $\cD_{i}$ is compliant with episodic MDP $\cM_i$ if for any $x'\in \cS, r', \tau\in[K],h\in[H]$, we have
    \begin{align}
        &\PP_{\cD_i}(r_{i,h}^\tau = r', x_{i,h+1}^\tau = x'|\{(x_{i,h}^j, a_{i,h}^j)\}_{j=1}^\tau,\{(r_{i,h}^j, x_{i,h+1}^j)\}_{j=1}^{\tau-1}) \notag \\
        &\quad = \PP_{i}(r_{i,h}(s_h,a_h)=r',s_{h+1} = x'|s_h = x_h^\tau, a_h = a_h^\tau).\notag
    \end{align}
\end{definition}
% \begin{remark}
%     Existing work about offline RL \cite{ bose2024offline} assumed that the offline dataset is generated by adapting trajectory-type data $(x_{i,1}^\tau, a_{i,1}^\tau, r_{i,1}^\tau, x_{i,2}^\tau, a_{i,2}^\tau, r_{i,2}^\tau,...)_{\tau}$ and setting dataset $\cD_{i,h}$ as the collection of stage $h$ data, i.e., $\cD_{i,h}:=(x_{i,h}^\tau, a_{i,h}^\tau, r_{i,h}^\tau, x_{i,h+1}^\tau)_\tau$. 
% \end{remark}

% Note that there are different types of assumptions made on $\cD_{i,h}$, e.g., $\cD_i$ follows a joint distribution of trajectories \cite{jin2021pessimism}. We adapt the independent offline dataset assumption for the simplicity of presentation, and we will consider the other settings in the future. 

In general, we claim $\cD_{i}$ is compliant with $\cM_{i}$ when the conditional distribution of any tuple of reward and next state in $\cD_i$ follows the conditional distribution determined by MDP $\cM_i$.

\section{Offline RL without context indicator information}\label{sec:nocontext}

%In Section \ref{sec:withcontext} and \ref{sec:linear}, we have proposed provable offline RL with small generalization gaps. Both PERM and PPPO require the agent to maintain policies/models for different MDPs separately. One may ask whether directly applying existing offline RL algorithms over datasets from multiple environments \emph{without} maintaining their identity information can yield the same ZSG property. 

In this section, we show that directly applying existing offline RL algorithms over datasets from multiple environments \emph{without} maintaining their identity information cannot yield a sufficient ZSG property, which is aligned with the existing observation of the poor generalization performance of offline RL \cite{mediratta2023generalization}. 

%\noindent\textbf{Numerical experiments} To show that the context information is important to ZSG property, we conduct numerical experiments on Combination Lock environment \cite{bose2024offline} to compare the performance of Pessimistic Value Iteration (PEVI) (\cite{jin2021pessimism}), a representative algorithm for offline RL, and PPPO. Due to the space limit, the detailed description of experiments and results are deferred to Appendix \ref{app:exp}. Our results show that PPPO performs better, which backs up our theoretical claim.

In detail, given contextual MDPs $\cM_1,...,\cM_n$ and their corresponding offline datasets $\cD_1, ..., \cD_n$, we assume the agent only has the access to the offline dataset $\bar\cD = \cup_{i=1}^n \cD_i$, where $\bar\cD = \{(x_{c_\tau, h}^\tau, a_{c_\tau, h}^\tau, r_{c_\tau, h}^\tau)_{h=1}^H\}_{\tau = 1}^{K}.$ Here $c_\tau \in C$ is the context information of trajectory $\tau$, which is \emph{unknown} to the agent. To explain why offline RL without knowing context information performs worse, we have the following proposition suggesting the offline dataset from multiple MDPs is not distinguishable from an ``average MDP" if the offline dataset does not contain context information.
\begin{proposition}\label{thm: nodistinguish}
    $\bar\cD$ is compliant with \emph{average MDP} $\bar\cM:=\{\bar M_h\}_{h=1}^H$, $\bar{M}_h:=\big(\cS,\cA,H,\bar P_h,\bar r_h\big)$,
\begin{align}    
&\bar P_h(x'|x,a)
        :=\EE_{c \sim C} \frac{P_{c,h}(x'|x,a) \mu_{c,h}(x, a)}{\EE_{c \sim C} \mu_{c,h}(x, a)},\notag\\
        &\PP(\bar r_h = r|x,a) := \EE_{c \sim C} \frac{\PP(\bar r_{c,h} = r|x,a) \mu_{c,h}(x, a)}{\EE_{c \sim C} \mu_{c,h}(x, a)},\notag
\end{align}
where $\mu_{c,h}(\cdot, \cdot)$ is the data collection distribution of $(s,a)$ at stage $h$ in dataset $\cD_c$. 
\end{proposition}

\begin{proof}
    See Appendix \ref{app:nodistin}.
\end{proof}

% For any policy $\pi=\{\pi_h\}_{h=1}^H$, we define the (state-)value function (V-function) $\overline{V}_h^\pi:\cS\to \RR$ at each step $h\in[H]$ on this ``average MDP" $\mathcal{M}(\overline{c})$ as
% \begin{equation}
% \overline{V}_h^\pi(x) = \EE_{\pi,\overline{\cP}}\Big[ \sum_{i=h}^H \overline{r}_i(s_i, a_i)\biggiven s_h=x  \Big]
% \label{eq:def_value_fct}
% \end{equation}
% and the action-value function (Q-function) $Q_h^\pi:\cS\times \cA\to \RR$ at each step $h\in[H]$ as
% \begin{equation}
% \overline{Q}_h^\pi(x,a) = \EE_{\pi,\overline{\cP}}\Big[\sum_{i=h}^H \overline{r}_i(s_i, a_i)\biggiven s_h=x, a_h=a  \Big]\,,
% \label{eq:def_q_fct}
% \end{equation}
% where the expectation  $\EE_{\pi,\overline{\cP}}$ is taken with respect to the randomness of the trajectory induced by $\pi$ with transition $\overline{\cP}$. 

Proposition \ref{thm: nodistinguish} suggests that if no context information is revealed, then the merged offline dataset $\bar\cD$ is equivalent to a dataset collected from the average MDP $\bar \cM$. Therefore, for any offline RL which outputs a Markovian policy, it converges to the optimal policy $\bar\pi^*$ of the average MDP $\bar \cM$. 

% We leave the convergence analysis of PEVI to Appendix \ref{app:pevi}. 

In general, $\bar\pi^*$ can be very different from $\pi^*$  when the transition probability functions of each environment are different. For example, consider the 2-context cMDP problem shown in Figure \ref{fig:eg2}, each context consists of one state and three possible actions. The offline dataset distributions $\mu$ are marked on the arrows that both of the distributions are following near-optimal policy. By Proposition \ref{thm: nodistinguish}, in average MDP $\bar{\mathcal{M}}$ the reward of the middle action is deterministically 0, while both upper and lower actions are deterministically 1. As a result, the optimal policy $\bar{\pi}^\ast$ will only have positive probabilities toward upper and lower actions. This leads to $\mathbb{E}_{c\sim C}[V^{\overline{\pi}^\ast}_{c,1}(x_1)]=0$, though we can see that $\pi^\ast$ is deterministically choosing the middle action and $\mathbb{E}_{c\sim C}[V^{\pi^\ast}_{c,1}(x_1)]=0.5$. This theoretically illustrates that the generalization ability of offline RL algorithms without leveraging context information is weak. In sharp contrast, imitation learning such as behavior cloning (BC) converges to the teacher policy that is independent of the specific MDP. Therefore, offline RL methods such as CQL \cite{kumar2020conservative} might enjoy worse generalization performance compared with BC, which aligns with the observation made by \cite{mediratta2023generalization}.

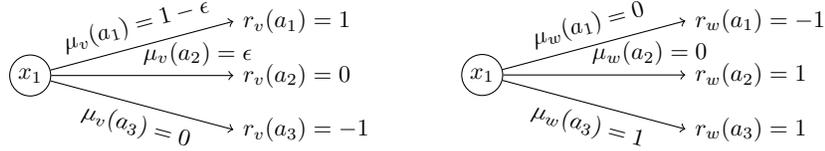
\begin{figure}[H]
    \centering
    
    \begin{tikzpicture}[
        every node/.style={font=\footnotesize},
        arrowstyle/.style={draw=none, midway, sloped, above},
        belowarrowstyle/.style={draw=none, midway, sloped, below},
        scale=0.9, transform shape]
        
        \node[draw,circle,inner sep=2pt] (x1_left) at (0,0) {$x_1$};
        \draw[->] (x1_left) -- ++(3,0.8) node[arrowstyle] {$\mu_v(a_1) = 1 - \epsilon$} node[anchor=west] {$r_v(a_1) = 1$};
        \draw[->] (x1_left) -- ++(3,0)   node[arrowstyle] {$~~~~~~~~~~~~~~\mu_v(a_2) = \epsilon$} node[anchor=west] {$r_v(a_2) = 0$};
        \draw[->] (x1_left) -- ++(3,-0.8) node[belowarrowstyle] {$\mu_v(a_3) = 0$} node[anchor=west] {$r_v(a_3) = -1$};
    \end{tikzpicture}
    \hspace{2em} % 添加水平间距
    % 右侧tikz图
    \begin{tikzpicture}[
        every node/.style={font=\footnotesize},
        arrowstyle/.style={draw=none, midway, sloped, above},
        belowarrowstyle/.style={draw=none, midway, sloped, below},
        scale=0.9, transform shape]
        
        \node[draw,circle,inner sep=2pt] (x1_right) at (0,0) {$x_1$};
        \draw[->] (x1_right) -- ++(3,0.8) node[arrowstyle] {$\mu_w(a_1) = 0$} node[anchor=west] {$r_w(a_1) = -1$};
        \draw[->] (x1_right) -- ++(3,0)   node[arrowstyle] {$~~~~~~~~~~~~~~\mu_w(a_2) = 0$} node[anchor=west] {$r_w(a_2) = 1$};
        \draw[->] (x1_right) -- ++(3,-0.8) node[belowarrowstyle] {$\mu_w(a_3) = 1$} node[anchor=west] {$r_w(a_3) = 1$};
    \end{tikzpicture}
    \caption{Two Contextual MDPs with the same compliant average MDPs. The discrete contextual space is defined as $C=\{v,w\}$ and both MDPs satisfies $\cS=\{x_1\},\cA=\{a_1,a_2,a_3\},H=1$. The data collection distributions $\mu$ and rewards $r$ for each action of each context are specified in the graph.}
    \label{fig:eg2}
\end{figure}
%\begin{figure}[H]
%    \centering
%    \includegraphics[width=0.75\linewidth]{Neurips 2024/eg3.jpg}
%    \caption{A cMDP problem with 2 equally distributed contexts.}
%    \label{fig:eg2}
%\end{figure}

\section{Provable offline RL with zero-shot generalization}\label{sec:withcontext}

In this section, we propose offline RL with small ZSG gaps. We show that two popular offline RL approaches, \emph{model-based RL} and \emph{policy optimization-based RL}, can output RL agent with ZSG ability, with a pessimism-style modification that encourages the agent to follow the offline dataset pattern.

% We assume that the agent is given a multi-task offline dataset $\cD$, which includes trajectories collected from different tasks. We assume that the agent is able to distinguish whether two trajectories are from the same task or not. Such an assumption naturally holds in the online setting since the agent needs to interact with different task to collect the data. However, such a condition does not always hold for the offline setting. We will show the case where the agent is not able to distinguish different tasks in Section \ref{sec:nocontext}. 

\subsection{Pessimistic policy evaluation}

We consider a meta-algorithm to evaluate any policy $\pi$ given an offline dataset, which serves as a key component in our proposed offline RL with ZSG. To begin with, we consider a general individual MDP and an oracle $\mathbb{O}$, which returns us an empirical Bellman operator and an uncertainty quantifier, defined as follows.

\begin{definition}[\cite{jin2021pessimism}]\label{def:oracle}
    For any individual MDP $M $, a dataset $\cD\subseteq \cS \times \cA \times \cS \times [0,1]$ that is compliant with $M$, a test function $V_{\cD}\subseteq  [0,H]^{\cS}$ and a confidence level $\xi$, we have an oracle $\mathbb{O}(\cD, V_{\cD},\xi)$ that returns $(\hat \BB V_\cD(\cdot, \cdot), \Gamma(\cdot, \cdot))$, a tuple of Empirical Bellman operator and uncertainty quantifier, satisfying
    \begin{small}
            \begin{align}
        &\PP_{\cD}\Big(\big|(\hat\BB V_{\cD})(x,a) - (\BB_M V_\cD)(x,a)\big|\notag \leq \Gamma(x,a)~ \text{for all}~(x,a)\in \cS\times \cA   \Big) \geq 1-\xi.\notag
    \end{align}
    \end{small}
\end{definition}

\begin{remark}
    Here we adapt a test function $V_{\cD}$ that can depend on the dataset $\cD$ itself. Therefore, $\Gamma$ is a function that depends on both the dataset and the test function class. We do not specify the test function class in this definition, and we will discuss its specific realization in Section \ref{sec:linear}.
\end{remark}
\begin{remark}
    For general non-linear MDPs, one may employ the bootstrapping technique to estimate uncertainty, in line with the bootstrapped DQN approach developed by \cite{osband2016deep}. We note that when the bootstrapping method is straightforward to implement, the assumption of having access to an uncertainty quantifier is reasonable.
\end{remark}
\begin{algorithm}[t!]
\begin{small}
  \caption{\underline{P}essimistic \underline{P}olicy \underline{E}valuation (PPE)}\label{alg:model based general}
  \begin{algorithmic}[1]
    \REQUIRE Offline dataset $\{\cD_{i,h}\}_{h=1}^H$, 
    policy $\pi = (\pi_h)_{h=1}^H$, confidence probability $\delta\in(0,1)$.
    \STATE Initialize $\hat{V}^\pi_{i,H+1}(\cdot) \leftarrow 0, \ \forall i\in[n]$.
    \FOR{step $h=H, H-1, \hdots, 1$}      
        \STATE  Let $(\hB_{i,h}\hat{V}_{i,h+1}^\pi)(\cdot,\cdot),\Gamma_{i,h} (\cdot,\cdot)\leftarrow \mathbb{O}(\cD_{i,h},\hat{V}_{i,h+1}^\pi, \delta)$
        \STATE Set $\hat{Q}_{i,h}^\pi(\cdot,\cdot)\leftarrow \min\{H-h+1, (\hB_{i,h}\hat{V}_{i,h+1}^\pi)(\cdot,\cdot)-\Gamma_{i,h}(\cdot,\cdot)\}^+$
        \STATE Set $\hat{V}_{i,h}^\pi(\cdot)\leftarrow \langle \hat{Q}_{i,h}^\pi(\cdot,\cdot),\pi_h(\cdot|\cdot)\rangle_\actions$
    \ENDFOR
    \RETURN $\hat{V}_{i,1}^\pi(\cdot),\dots,\hat{V}_{i,H}^\pi(\cdot), \hat{Q}_{i,1}^\pi(\cdot,\cdot),\dots,\hat{Q}_{i,H}^\pi(\cdot,\cdot)$.
  \end{algorithmic}
\end{small}
\end{algorithm}
Based on the oracle $\mathbb{O}$, we propose our pessimistic policy evaluation (PPE) algorithm as Algorithm \ref{alg:model based general}. In general, PPE takes a given policy $\pi$ as its input, and its goal is to evaluate the V value and Q value $\{(V_{i,h}^{\pi}, Q_{i,h}^{\pi}) \}_{h=1}^H$ of $\pi$ on MDP $\cM_i$. Since the agent is not allowed to interact with $\cM_i$, PPE evaluates the value based on the offline dataset $\{\cD_{i,h}\}_{h=1}^H$. At each stage $h$, PPE utilizes the oracle $\mathbb{O}$ and obtains the empirical Bellman operator based on $\cD_{i,h}$ as well as its uncertainty quantifier, with high probability. Then PPE applies the \emph{pessimism principle} to build the estimation of the Q function based on the empirical Bellman operator and the uncertainty quantifier. Such a principle has been widely studied and used in offline policy optimization, such as pessimistic value iteration (PEVI) \cite{jin2021pessimism}. To compare with, we use the pessimism principle in the policy evaluation problem. 

\begin{remark}
    In our framework, pessimism can indeed facilitate generalization, rather than hinder it. Specifically, we employ pessimism to construct reliable Q functions for each environment individually. This approach supports broader generalization by maintaining multiple Q-networks separately. By doing so, we ensure that each Q function is robust within its specific environment, while the collective set of Q functions enables the system to generalize across different environments.
\end{remark}

\subsection{Model-based approach: pessimistic empirical risk minimization}
Given PPE, we propose algorithms that have the ZSG ability. We first propose a pessimistic empirical risk minimization (PERM) method which is model-based and conceptually simple. The algorithm details are in Algorithm \ref{alg:erm}. In detail, for each dataset $\cD_i$ drawn from $i$-th environments, PERM builds a model using PPE to evaluate the policy $\pi$ under the environment $\cM_i$. Then PERM outputs a policy $\pi^{\text{PERM}} \in \Pi$ that maximizes the average pessimistic value, i.e., $1/n\sum_{i=1}^n \hat{V}_{i,1}^\pi (x_1)$. Our approach is inspired by the classical empirical risk minimization approach adopted in supervised learning, and the Optimistic Model-based ERM proposed in \cite{ye2023power} for online RL. Our setting is more challenging than the previous ones due to the RL setting and the offline setting, where the interaction between the agent and the environment is completely disallowed. Therefore, unlike \cite{ye2023power}, which adopted an optimism-style estimation to the policy value, we adopt a pessimism-style estimation to fight the distribution shift issue in the offline setting.

Next we propose a theoretical analysis of PERM. Denote $\mathcal{N}_\epsilon^\Pi$ as the $\epsilon$-covering number of the policy space $\Pi$ w.r.t. distance $\mathrm d(\pi^1,\pi^2) = \max_{s\in \mathcal{S}, h \in [H]} \|\pi^1_h(\cdot|s) - \pi^2_h (\cdot|s)\|_{1}$. Then we have the following theorem to provide an upper bound of the suboptimality gap of the output policy $\pi^{\text{PERM}}$. 

\begin{algorithm}[t!]
\begin{small}
  \caption{\underline{P}essimistic \underline{E}mpirical \underline{R}isk \underline{M}inimization (PERM)}\label{alg:erm}
  \begin{algorithmic}[1]
    \REQUIRE Offline dataset $\cD = \{\cD_i\}_{i=1}^n, \cD_i:=\{(x_{i,h}^\tau, a_{i,h}^\tau, r_{i,h}^\tau)_{h=1}^H\}_{\tau=1}^{K}$, policy class $\Pi$, confidence probability $\delta\in(0,1)$, a pessimistic offline policy evaluation algorithm $\textbf{Evaluation}$ as a subroutine.
    \STATE Set $\cD_{i,h} =\{(x_{i,h}^\tau, a_{i,h}^\tau, r_{i,h}^\tau, x_{i,h+1}^\tau)\}_{\tau=1}^{K} $
    \STATE$\pi^{\text{PERM}}=\argmax_{\pi\in\Pi} \frac{1}{n}\sum_{i=1}^n \hat{V}_{i,1}^\pi (x_1)$, \\where $[\hat V^\pi_{i,1}(\cdot),\cdot,\dots, \cdot] = \textbf{Evaluation}\Big(\{\cD_{i,h}\}_{h=1}^H,\pi,\delta/(3nH\mathcal{N}_{(Hn)^{-1}}^\Pi))\Big)$
%     \STATE Initialize $\hat{V}^\pi_{i,H+1}(\cdot) \leftarrow 0, \ \forall i\in[n]$.
%     \FOR{$i=1, 2, \hdots, n$}
%     \FOR{step $h=H, H-1, \hdots, 1$}       
%         \STATE  Let $(\hB_{i,h}\hat{V}_{i,h+1}^\pi)(\cdot,\cdot),\Gamma_{i,h} (\cdot,\cdot)\leftarrow \mathbb{O}(\cD_{i,h},\hat{V}_{i,h+1}^\pi, \xi)$
%         \STATE Set $\hat{Q}_{i,h}^\pi(\cdot,\cdot)\leftarrow \min\{H-h+1, (\hB_{i,h}\hat{V}_{i,h+1}^\pi)(\cdot,\cdot)-\Gamma_{i,h}(\cdot,\cdot)\}^+$
%         \STATE Set $\hat{V}_{i,h}^\pi(\cdot)\leftarrow \langle \hat{Q}_{i,h}^\pi(\cdot,\cdot),\pi_h(\cdot|\cdot)\rangle_\actions$
%         \ENDFOR
%     \ENDFOR
% \hfill
%     \STATE 
%     \STATE$\hat{\pi}^*=\argmax_\pi \frac{1}{n}\sum_{i=1}^n \hat{V}_{i,1}^\pi (x_1)$
\RETURN $\pi^{\text{PERM}}$.

  \end{algorithmic}
\end{small}
\end{algorithm}

\begin{theorem}\label{thm:model based regret_upper_bound_general}
Set the Evaluation subroutine in Algorithm \ref{alg:erm} as PPE (Algo.\ref{alg:model based general}). Let $\Gamma_{i,h}$ be the uncertainty quantifier returned by $\mathbb{O}$ through the PERM. Then w.p. at least $1-\delta$, the output $\pi^{\text{PERM}}$ of  Algorithm \ref{alg:erm} satisfies 
\begin{small}
    \begin{align}&\normalfont
\text{SubOpt}(\pi^{\text{PERM}})\leq \underbrace{7\sqrt{\frac{2\log(6\mathcal{N}_{(Hn)^{-1}}^\Pi/\delta)}{n}}}_{I_1: \text{Supervised learning (SL) error}} +\underbrace{\frac{2}{n}\sum_{i=1}^n\sum_{h=1}^H\E{i,\pi^*}{\Gamma_{i,h}(s_h,a_h)|s_1=x_1}}_{I_2: \text{Reinforcement learning (RL) error}}\,,
\label{eq:model based gap general}
\end{align}
\end{small}
where $\EE_{i,\pi^*}$ is w.r.t. the trajectory induced by $\pi^*$ with the transition $\cP_i$ in the underlying MDP $\cM_i$.
% $\mathcal{N}_\epsilon^\Pi$ is the $\epsilon$-covering number of the policy space $\Pi$ w.r.t. distance $\mathrm d(\pi^1,\pi^2) = \max_{s\in \mathcal{S}, h \in [H]} \|\pi^1_h(\cdot|s) - \pi^2_h (\cdot|s)\|_{1}$. 
\end{theorem}

% \begin{theorem}\label{thm:model based regret_upper_bound_general}
% Set the Evaluation subroutine in Algorithm \ref{alg:erm} as PPE. Let $\Gamma_{i,h}$ be the uncertainty quantifier returned by $\mathbb{O}$ through the PERM. Then w.p. at least $1-\delta$, the output $\pi^{\text{PERM}}$ of  Algorithm \ref{alg:erm} satisfies 
% \begin{align}\normalfont
% \text{SubOpt}(\pi^{\text{PERM}})\leq 4H\epsilon+2\sqrt{\frac{2\log(2H\mathcal{N}_{(Hn)^{-1}}^\Pi)}{n}}+\frac{2}{n}\sum_{i=1}^n\sum_{h=1}^H\E{\pi^*,\cM_i}{\Gamma_{i,h}(s_h,a_h)|s_1=x_1}\,,
% \label{eq:model based gap general}
% \end{align}
% where $\EE_{\pi^*,\cM_i}$ is with respect to the trajectory induced by $\pi^*$ with the transition $\cP_i$ in the underlying MDP $\cM_i$.
% $\mathcal{N}_\epsilon^\Pi$ is the $\epsilon$-covering number of the policy space $\Pi$ w.r.t. distance $\mathrm d(\pi^1,\pi^2) = \max_{s\in \mathcal{S}, h \in [H]} \|\pi^1_h(\cdot|s) - \pi^2_h (\cdot|s)\|_{1}$. 
% \end{theorem}
\begin{proof}
    See Appendix \ref{appendix model based}.
\end{proof}
\begin{remark}
The covering number $\mathcal{N}_{(Hn)^{-1}}^\Pi$ depends on the policy class $\Pi$. Without any specific assumptions, the policy class $\Pi$ that consists of all the policies  $\pi = \{\pi_h\}_{h=1}^H, \pi_h:\mathcal S \mapsto \Delta (\cA)$ and the log $\epsilon$-covering number $\log \mathcal{N}^{\Pi}_{\epsilon}=O(|\cA| |\cS| H \log(1+ |\cA|/\epsilon))$. 
\end{remark}
\begin{remark}
    The SL error can be easily improved to a distribution-dependent bound $\log \cN\cdot \text{Var}/\sqrt{n}$, where $\cN$ is the covering number term denoted in $I_1$, $\text{Var} = \max_\pi\mathbb{V}_{c \sim C}V^\pi_{c,1}(x_1)$ is the variance of the context distribution, by using a Bernstein-type concentration inequality in our proof. Therefore, for the singleton environment case where $|C|=1$, our suboptimality gap reduces to the one of PEVI in \cite{jin2021pessimism}.
\end{remark}
\begin{remark}\label{rmk:merge}
    In real-world settings, as the number of sampled contexts $n$ may be very large, it is unrealistic to manage $n$ models simultaneously in the implementation of PERM algorithm, thus we provide the suboptimality bound in line with Theorem \ref{thm:model based regret_upper_bound_general} when the offline dataset is merged into $m$ contexts such that $m<n$. See Theorem \ref{thm:permv} in Appendix \ref{appendix:merge}.
\end{remark}
% \begin{remark}
%     For the policy class $\Pi$ that consists of all the policies  $\pi = \{\pi_h\}_{h=1}^H, \pi_h:\mathcal S \mapsto \Delta (\cA)$ and $\pi$, the $\epsilon$-covering number $\mathcal{N}^{\Pi}_{\epsilon}=(1+\frac{\actions}{\epsilon})^{(|\actions|-1)|\cS|H}$. Therefore, the second term in Eq.(\ref{eq:model based gap general}) is less than $2\sqrt{\frac{2|\actions||\cS|H\log(2H)}{n}}$.
% \end{remark}

Theorem \ref{thm:model based regret_upper_bound_general} shows that the ZSG gap of PERM is bounded by two terms $I_1$ and $I_2$. $I_1$, which we call \emph{supervised learning error}, depends on the number of environments $n$ in the offline dataset $\cD$ and the covering number of the function (policy) class, which is similar to the generalization error in supervised learning. $I_2$, which we call it \emph{reinforcement learning error}, is decided by the optimal policy $\pi^*$ that achieves the best zero-shot generalization performance and the uncertainty quantifier $\Gamma_{i,h}$. In general, $I_2$ is the ``intrinsic uncertainty" denoted by \cite{jin2021pessimism} over $n$ MDPs, which characterizes how well each dataset $\cD_i$ covers the optimal policy $\pi^*$.

\subsection{Model-free approach: pessimistic proximal policy optimization}

\begin{algorithm}[t!]
\begin{small}
  \caption{\underline{P}essimistic \underline{P}roximal \underline{P}olicy \underline{O}ptimzation (PPPO)}\label{alg:modelfree}
  \begin{algorithmic}[1]
    \REQUIRE Offline dataset $\cD = \{\cD_i\}_{i=1}^n, \cD_i:=\{(x_{i,h}^\tau, a_{i,h}^\tau, r_{i,h}^\tau)_{h=1}^H\}_{\tau=1}^{K}$, confidence probability $\delta\in(0,1)$, a pessimistic offline policy evaluation algorithm $\textbf{Evaluation}$ as a subroutine.
    \STATE Set $\cD_{i,h} =\{(x_{i,h}^{\tau\cdot H+h}, a_{i,h}^{\tau\cdot H+h}, r_{i,h}^{\tau\cdot H+h}, x_{i,h+1}^{\tau\cdot H+h})\}_{\tau=0}^{\lfloor K/H\rfloor-1 }$
    \STATE Set $\pi_{0,h}(\cdot|\cdot)$ as uniform distribution over $\actions$ and $\hat Q^{\pi_0}_{0,h}(\cdot,\cdot)$ as zero functions.
    \FOR{$i=1, 2, \cdots, n$}
    \STATE Set $\pi_{i,h}(\cdot|\cdot)\propto \pi_{i-1,h}(\cdot|\cdot)\cdot \exp(\alpha\cdot \hat{Q}^{\pi_{i-1}}_{i-1,h}(\cdot, \cdot))$
    % \STATE Set $\hat{V}_{i,H+1}(\cdot) = 0$
    %     \FOR{$h=H,\cdots,1$}
    %     \STATE Let $(\hB_{i,h}\hat{V}_{i,h+1})(\cdot,\cdot),\Gamma_{i,h} (\cdot,\cdot)\leftarrow \mathbb{O}(\cD_{i,h},\hat{V}_{i,h+1}, \xi)$
    %     \STATE Set $\hat{Q}_{i,h}(\cdot, \cdot):=\min\{H-h+1, \hB_{i,h} \hat{V}_{h+1}^i(\cdot, \cdot) - \Gamma_{i,h}(\cdot, \cdot)\}^+$.
    %     \item Set $\hat{V}_{i,h}(\cdot)\leftarrow \la \pi_{i,h}(\cdot|\cdot),  \hat{Q}_{i,h}(\cdot, \cdot)\ra_A$.
    %     \ENDFOR
    \STATE Set $[\cdot,\dots, \cdot, \hat Q^{\pi_i}_{i,1}(\cdot,\cdot),\dots, \hat Q^{\pi_i}_{i,H}(\cdot,\cdot)] = \textbf{Evaluation}(\{\cD_{i,h}\}_{h=1}^H,\pi_i,\delta/(nH))$
    \ENDFOR
    \RETURN $\pi^{\text{PPPO}}=\text{random}(\pi_1, ..., \pi_n)$
  \end{algorithmic}
\end{small}
\end{algorithm}
% We follow the following method to collect offline dataset. 
% \begin{itemize}
%     \item First we sample $n$ i.i.d. number of different tasks $M_1, ...M_n \sim \cC$. 
%     \item For each $M_i$, we obtain $\cD_i$. 
% \end{itemize}
% We consider the episodic MDP setting, where for each MDP $M_c$, it has $P^c(s'|s,a), r^c(s,a)$. For simplicity we assume $r_c$ is a deterministic function. We also assume $P^c:=\{P^c_1, ...,P^c_H\}$ and $r^c:=\{r^c_1,...,r^c_H\}$. 

% We follow Jin's paper to assume that the dataset is collected in compliance. We have the following assumption.
% \begin{assumption}
% The offline dataset $\cD:=\cD_1\cup...\cup \cD_n$, where $M_1, ..., M_n$ are i.i.d. drawn from the context space. For each $\cD_i = \{(s,a,s',r)\}$, we have $s'\sim P_i(\cdot|s,a)$ and $(s,a)$ are drawn i.i.d. from a joint distribution $\mu_i(s,a)$. 
% \end{assumption}

% \begin{definition}\label{def:offlineevent}
%     Then given one dataset $\cD$, we denote $\hat{\mathbb{B}}_h^c$ to be a operator, such that
% \begin{align}
%     \mathbb{P}\bigg(\forall s,a,h,c,|\hB^c_h V^c_{h+1}(\cdot, \cdot) - \mathbb{B}^c_h V^c_{h+1}(\cdot, \cdot)| \leq \Gamma^c_h(\cdot, \cdot)\bigg)\geq 1-\xi
% \end{align}
% \end{definition}

PERM in Algorithm \ref{alg:erm} works as a general model-based algorithm framework to enable ZSG for any pessimistic policy evaluation oracle. However, note that in order to implement PERM, one needs to maintain $n$ different models or critic functions simultaneously in order to evaluate $\sum_{i=1}^n \hat{V}_{i,1}^\pi (x_1)$ for any candidate policy $\pi$. Note that existing online RL \cite{ghosh2021generalization} achieves ZSG by a model-free approach, which only maintains $n$ policies rather than models/critic functions. Therefore, one natural question is whether we can design a \emph{model-free} offline RL algorithm also with access only to policies.

We propose the pessimistic proximal policy optimization (PPPO) in Algorithm \ref{alg:modelfree} to address this issue. Our algorithm is inspired by the optimistic PPO \cite{cai2020provably} originally proposed for online RL. PPPO also adapts PPE as its subroutine to evaluate any given policy pessimistically. Unlike PERM, PPPO only maintains $n$ policies $\pi_1, ...,\pi_n$, each of them is associated with an MDP $\cM_n$ from the offline dataset. In detail, PPPO assigns an order for MDPs in the offline dataset and names them $\cM_1, ..., \cM_n$. For $i$-th MDP $\cM_i$, PPPO selects the $i$-th policy $\pi_i$ as the solution of the proximal policy optimization starting from $\pi_{i-1}$, which is
\begin{align}
    \pi_{i}&\leftarrow \argmax_{\pi}V_{i-1,1}^\pi(x_1) - \alpha^{-1}\EE_{i-1, \pi_{i-1}}[\text{KL}(\pi\|\pi_{i-1})|s_1 = x_1],\label{ooo}
\end{align}
where $\alpha$ is the step size parameter. Since $V_{i-1,1}^\pi(x_1)$ is not achievable, we use a linear approximation $L_{i-1}(\pi)$ to replace $V_{i-1,1}^\pi(x_1)$, where
\begin{small}
    \begin{align}
   & L_{i-1}(\pi) = V_{i-1,1}^{\pi_{i-1}}(x_1) + \EE_{i-1, \pi_{i-1}}\bigg[\sum_{h=1}^H \la\hat Q_{i-1,h}^{\pi_{i-1}}(x_h,\cdot), \pi_h(\cdot|x_h) - \pi_{i-1,h}(\cdot|x_h) \ra\bigg|s_1 = x_1\bigg],\label{ggg}
\end{align}
\end{small}
where $\hat Q_{i-1,h}^{\pi_{i-1}}\approx Q_{i-1,h}^{\pi_{i-1}}$ are the Q values evaluated on the offline dataset for $\cM_{i-1}$. \eqref{ooo} and \eqref{ggg} give us a close-form solution of $\pi$ in Line 4 in Algorithm \ref{alg:modelfree}. Such a routine corresponds to one iteration of PPO \cite{schulman2017proximal}. Finally, PPPO outputs $\pi^{\text{PPPO}}$ as a random selection from $\pi_1,...,\pi_n$. 

\begin{remark}
    In Algorithm \ref{alg:modelfree}, we adopt a data-splitting trick \cite{jin2021pessimism} to build $\cD_{i,h}$, where we only utilize each trajectory once for one data tuple at some stage $h$. It is only used to avoid the statistical dependency of $\hat V_{i,h+1}^{\pi_i}(\cdot)$ and $x_{i,h+1}^\tau$ for the purpose of theoretical analysis. 

\end{remark}
% There are several algorithm design tricks we have also adapted in Algorithm \ref{alg:modelfree}. First, we adapt a data-splitting trick \cite{jin2021pessimism} to build $\cD_{i,h}$, where we only utilize each trajectory once for one data tuple at some stage $h$. It is only used to avoid the statistical dependency of $\hat V_{i,h+1}^{\pi_i}(\cdot)$ and $x_{i,h+1}^\tau$ for the purpose of theoretical proof. Second, our $\pi^{\text{PPPO}}$ is selected as a random policy from $\pi_1,...,\pi_n$ to guarantee its generalization property. An alternative way to construct $\pi^{\text{PPPO}}$ is to make $\pi^{\text{PPPO}}$ Markovian is to adapt the link function technique \cite{ghosh2021generalization}

The following theorem bounds the suboptimality of PPPO.
\begin{theorem}\label{thm:model-free}
    Set the Evaluation subroutine in Algorithm \ref{alg:modelfree} as Algorithm \ref{alg:model based general}. Let $\Gamma_{i,h}$ be the uncertainty quantifier returned by $\mathbb{O}$ through the PPPO. Selecting $\alpha = 1/\sqrt{H^2n}$. Then selecting $\delta = 1/8$, w.p. at least $2/3$, we have
    \begin{small}
        \begin{align}
    &\text{SubOpt}(\pi^{\text{PPPO}}) \leq 10\bigg(\underbrace{\sqrt{\frac{\log|\actions|H^2}{n}}}_{I_1: \text{SL error}} + \underbrace{\frac{1}{n}\sum_{i=1}^n \sum_{h=1}^H\E{i,\pi^*}{\Gamma_{i,h}(s_h,a_h)|s_1=x_1}}_{I_2: \text{RL error}}\bigg).\notag
\end{align}
    \end{small}
where $\EE_{i,\pi^*}$ is w.r.t. the trajectory induced by $\pi^*$ with the transition $\cP_i$ in the underlying MDP $\cM_i$.
\end{theorem}
\begin{proof}
    See Appendix \ref{proof:modelfree}.
\end{proof}
\begin{remark} \label{rmk:mergefree}
    As in Remark \ref{rmk:merge}, we also provide the suboptimality bound in line with Theorem \ref{thm:model-free} when the offline dataset is merged into $m$ contexts such that $m<n$. See Theorem \ref{thm:pppov} in Appendix \ref{appendix:merge}.
\end{remark}
Theorem \ref{thm:model-free} shows that the suboptimality gap of PPPO can also be bounded by the SL error $I_1$ and RL error $I_2$. Interestingly, $I_1$ in Theorem \ref{thm:model-free} for PPPO only depends on the cardinality of the action space $|\cA|$, which is different from the covering number term in $I_1$ for PERM. Such a difference is due to the fact that PPPO outputs the final policy $\pi^{\text{PPPO}}$ as a random selection from $n$ existing policies, while PERM outputs one policy $\pi^{\text{PERM}}$. Whether these two guarantees can be unified into one remains an open question. 

% \noindent\textbf{Provable generalization for offline linear MDPs} In Appendix \ref{sec:linear}, we provide a detailed instantiation of our proposed algorithms for linear MDPs, which leverage known feature mappings to model both the transition dynamics and reward functions. Specifically, we adapt our meta-algorithms (Algorithm \ref{alg:erm} and Algorithm \ref{alg:modelfree}) by incorporating a policy evaluation subroutine (Algorithm \ref{alg:linear mdp model based}) tailored for linear MDPs. We establish theoretical guarantees on the suboptimality of the output policies, demonstrating that the algorithms achieve provable generalization for offline linear MDPs.

\section{Provable generalization for offline linear MDPs}\label{sec:linear}

In this section, we instantiate our Algo.\ref{alg:erm} and Algo.\ref{alg:modelfree} for general MDPs on specific MDP classes. We consider the linear MDPs defined as follows. 

\begin{assumption}[\cite{yang2019sample, jin2019provably}]
We assume $\forall i \in C, \cM_i$ is a linear MDP with a known feature map $\phi:\cS\times \cA\to \RR^d$ if there exist $d$ unknown measures ${\mu}_{i,h}=(\mu_{i,h}^{(1)},\ldots,\mu_{i,h}^{(d)})$ over $\cS$ and an unknown vector $\theta_{i,h}\in \RR^d$ such that
\begin{align}
    &P_{i,h}(x'\given x,a) = \langle \phi(x,a),\mu_{i,h}(x')\rangle, \notag\\
    &\EE\bigl[r_{i,h}(s_h, a_h) \biggiven s_h=x,a_h=a\bigr] = \langle \phi(x,a),\theta_{i,h}\rangle\label{eq:w07}
\end{align}
for all $(x,a,x')\in \cS\times \cA\times \cS$ at every step $h\in[H]$. We assume $\|\phi(x,a)\|\leq 1$ for all $(x,a)\in \cS\times \cA$ and $\max\{\|\mu_{i,h}(\cS) \| ,\|\theta_{i,h}\|\}\leq \sqrt{d}$ at each step $h\in[H]$, and we define  $\| \mu_{i,h} (\cS) \| = \int_{\cS } \| \mu_{i,h} (x) \| \,\ud x$.
\label{assump:linear_mdp}
\end{assumption}
% \begin{remark}
%     We assume that each environment $\cM_{i}$ shares the same feature mapping $\phi(x,a)$. Such an assumption is for the ease of presentation, and our results can be easily extended to the setting where different environments enjoy different feature mappings.
% \end{remark}

% The estimation $\hat{w}_{i,h}$ has the following closed-form solution
% \begin{align}
%     \label{eq:w18}
% &\hat{w}_{i,h} =  \Lambda_{i,h} ^{-1} \Big( \sum_{\tau=1}^{K} \phi(x_{i,h}^\tau,a_{i,h}^\tau) \cdot \bigl(r_{i,h}^\tau + \hat{V}_{i,h+1}(x_{i,h}^{-,\tau})\bigr) \Bigr ) , \notag\\
% &\text{where~~} \Lambda_{i,h} = \sum_{\tau=1}^K \phi(x_{i,h}^\tau,a_{i,h}^\tau)  \phi(x_{i,h}^\tau,a_{i,h}^\tau) ^\top + \lambda\cdot I. 
% \end{align}

We first specialize the general PPE algorithm (Algo.\ref{alg:model based general}) to obtain the PPE algorithm tailored for linear MDPs (Algo.\ref{alg:linear mdp model based}). This specialization is achieved by constructing $\hat\BB_{i,h}\hat{V}^\pi_{i,h+1}$, $\Gamma_{i,h}$, and $\hat{V}^\pi_{i,h}$ based on the dataset $\cD_i$. We denote the set of trajectory indexes in $\cD_{i,h}$ as $\cB_{i,h}$. Algo.\ref{alg:linear mdp model based} subsequently functions as the policy evaluation subroutine in Algo.\ref{alg:erm} and Algo.\ref{alg:modelfree} for linear MDPs. In detail, we construct $\hat\BB_{i,h}\hat{V}_{i,h+1}$ (which is the estimation of $\BB_{i,h}\hat{V}_{i,h+1}$) as $(\hat\BB_{i,h}\hat{V}_{i,h+1})(x, a) = \phi(x, a)^\top \hat{w}_{i,h}$,
% \begin{equation}
%     \label{eq:wlin}
% (\hat\BB_{i,h}\hat{V}_{i,h+1})(x, a) = \phi(x, a)^\top \hat{w}_{i,h},
% \end{equation}
where
\begin{align}
    &\textstyle{\hat{w}_{i,h} =  \argmin_{w\in \RR^d} \sum_{\tau \in \cB_{i,h}}} \bigl(r_{i,h}^\tau + \hat{V}_{i,h+1}(x_{i,h}^{-,\tau})  - \phi (x_{i,h}^\tau,a_{i,h}^\tau)^\top w\bigr)^2 + \lambda \cdot \|w\|_2^2\,\label{eq:w188}
\end{align}
with $\lambda>0$ being the regularization parameter. The closed-form solution to \eqref{eq:w188} is in Line 4 in Algorithm \ref{alg:linear mdp model based}. Besides, we construct the uncertainty quantifier $\Gamma_{i,h}$ based on $\cD_i$ as 
\begin{align}
    \textstyle{\Gamma_{i,h}(x, a)}& = \beta(\delta)\cdot\|\phi(x, a)\|_{\Lambda_{i,h} ^{-1}}\,,\Lambda_{i,h} = \sum_{\tau \in \cB_{i,h}}\phi(x_{i,h}^\tau,a_{i,h}^\tau)  \phi(x_{i,h}^\tau,a_{i,h}^\tau) ^\top + \lambda\cdot I,\notag
\end{align}  
with $\beta(\delta)>0$ being the scaling parameter.

% \footnote{Spefically, for Algo.\ref{alg:erm}, $\tilde\Lambda_{i,h} = \sum_{\tau =1}^K\phi(x_{i,h}^\tau,a_{i,h}^\tau)  \phi(x_{i,h}^\tau,a_{i,h}^\tau) ^\top + \lambda\cdot I$, for Algo.\ref{alg:modelfree}, $\bar\Lambda_{i,h} = \sum_{\tau=1}^{\lfloor K/H\rfloor-1}\phi(x_{i,h}^{\tau\cdot H+h},a_{i,h}^{\tau\cdot H+h})  \phi(x_{i,h}^{\tau\cdot H+h},a_{i,h}^{\tau\cdot H+h}) ^\top + \lambda\cdot I$ due to the data-splitting techniques.}
% We can then replace the general oracle operations in the general PPE algorithm (Algo.\ref{alg:model based general}) with $(\hat\BB_{i,h}\hat{V}_{i,h+1})(x, a)$ and $\Gamma_{i,h}$ specified above to get the PPE algorithm for Linear MDPs (Algo.\ref{alg:linear mdp model based}). Using the specified Algo.\ref{alg:linear mdp model based} as a subroutine, Algo.\ref{alg:erm} and Algo.\ref{alg:modelfree} will have provable generalization ability for offline Linear MDPs.

% Finally, we construct $\hat{Q}_{i,h}$ as 
% \begin{equation}
%     \hat{Q}_{i,h}(\cdot,\cdot) \leftarrow \min\{\phi(\cdot,\cdot)^\top \hat{w}_{i,h} - \Gamma_{i,h}(\cdot,\cdot),H-h+1\}^+\,.
% \end{equation}
\begin{algorithm}[H]
\begin{small}
\caption{\underline{P}essimistic \underline{P}olicy \underline{E}valuation (PPE): Linear MDP}\label{alg:linear mdp model based}
\begin{algorithmic}[1]
\REQUIRE Offline dataset $\{\cD_{i,h}\}_{h=1}^H, \cD_{i,h}=\{(x_{i,h}^\tau,a_{i,h}^\tau,r_{i,h}^\tau, x_{i,h}^{-,\tau})\}_{\tau \in \cB_{i,h}}$, policy $\pi$, confidence probability $\delta\in(0,1)$.
    \STATE Initialize $\hat{V}^\pi_{i,H+1}(\cdot) \leftarrow 0, \ \forall i\in[n]$.
\FOR{step $h=H,H-1,\ldots,1$}
\STATE Set $\Lambda_{i,h} \leftarrow \sum_{\tau \in \cB_{i,h}} \phi(x_{i,h}^\tau,a_{i,h}^\tau)  \phi(x_{i,h}^\tau,a_{i,h}^\tau) ^\top + \lambda\cdot I$. %\hfill  {//Estimation}
\STATE Set $\hat{w}_{i,h}\leftarrow  \Lambda_{i,h} ^{-1}( \sum_{\tau \in \cB_{i,h}} \phi(x_{i,h}^\tau,a_{i,h}^\tau) \cdot (r_{i,h}^\tau + \hat{V}_{i,h+1}^\pi(x_{i,h}^{-,\tau})) ) $. 
\STATE Set $\Gamma_{i,h}(\cdot,\cdot) \leftarrow \beta(\delta)\cdot ( \phi(\cdot,\cdot)^\top  \Lambda_{i,h} ^{-1} \phi(\cdot,\cdot) )^{1/2}$. 
\STATE Set $\hat{Q}_{i,h}^\pi(\cdot,\cdot) \leftarrow \min\{\phi(\cdot,\cdot)^\top \hat{w}_{i,h} - \Gamma_{i,h}(\cdot,\cdot),H-h+1\}^+$. 
        \STATE Set $\hat{V}_{i,h}^\pi(\cdot)\leftarrow \langle \hat{Q}_{i,h}^\pi(\cdot,\cdot),\pi_h(\cdot|\cdot)\rangle_\actions$
\ENDFOR 
    \RETURN $\hat{V}_{i,1}^\pi(\cdot),\dots,\hat{V}_{i,H}^\pi(\cdot), \hat{Q}_{i,1}^\pi(\cdot,\cdot),\dots,\hat{Q}_{i,H}^\pi(\cdot,\cdot)$.
\end{algorithmic}
    
\end{small}
\end{algorithm}
The following theorem shows the suboptimality gaps for Algo.\ref{alg:erm} (utilizing subroutine Algo.\ref{alg:linear mdp model based}) and Algo.\ref{alg:modelfree} (also with subroutine Algo.\ref{alg:linear mdp model based}).

\begin{theorem}\label{thm:regret_upper_linear}
Under Assumption \ref{assump:linear_mdp}, in Algorithm \ref{alg:linear mdp model based}, we set $\lambda=1,\quad \beta(\delta) = c\cdot dH\sqrt{\log(2dHK/\delta)}$, where $c>0$ is a positive constant. Then, we have: \\
(i) for the output policy $\pi^{\text{PERM}}$ of Algo.\ref{alg:erm} with subroutine Algo.\ref{alg:linear mdp model based}, w.p. at least $1-\delta$, the suboptimality gap satisfies

    \begin{align}
&\text{SubOpt}(\pi^{\text{PERM}})\leq 7\sqrt{\frac{7\log(6\mathcal{N}_{(Hn)^{-1}}^\Pi/\delta)}{n}}\notag\\
&\quad+\frac{2\beta\big(\frac{\delta}{3nH\mathcal{N}_{(Hn)^{-1}}^\Pi}\big)}{n}\cdot\sum_{i=1}^n\sum_{h=1}^H \EE_{i,\pi^*}\Bigl[ \|\phi(s_h,a_h)\|_{\tilde\Lambda_{i,h}^{-1}} \biggiven s_1=x_1\Bigr]\,,
\label{eq:model based gap linear}
\end{align}

(ii) for the output policy $\pi^{\text{PPPO}}$ of Algo.\ref{alg:modelfree} with subroutine Algo.\ref{alg:linear mdp model based}, setting $\delta = 1/8$, then with probability at least $2/3$, the suboptimality gap satisfies
\begin{small}
    \begin{align}&\normalfont
\text{SubOpt}(\pi^{\text{PPPO}})\leq 10\bigg(\sqrt{\frac{\log|\actions|H^2}{n}}+\frac{\beta\big(\frac{1}{4nH}\big)}{n}\cdot\sum_{i=1}^n\sum_{h=1}^H \EE_{i,\pi^*}\Bigl[ \|\phi(s_h,a_h)\|_{\bar\Lambda_{i,h}^{-1}} \biggiven s_1=x_1\Bigr]\bigg),
\label{eq:modelfree gap linear}
\end{align}
\end{small}
where $\EE_{i,\pi^*}$ is with respect to the trajectory induced by $\pi^*$ with the transition $\cP_i$ in the underlying MDP $\cM_i$ given the fixed matrix $\tilde\Lambda_{i,h}$ or $\bar\Lambda_{i,h}$.
% $\mathcal{N}_\epsilon^\Pi$ is the $\epsilon$-covering number of the policy space $\Pi$ w.r.t. distance $\mathrm d(\pi^1,\pi^2) = \max_{s\in \mathcal{S}, h \in [H]} \|\pi^1_h(\cdot|s) - \pi^2_h (\cdot|s)\|_{1}$.
\end{theorem}

% $\Lambda_{i,h} = \sum_{\tau=1}^K \phi(x_{i,h}^\tau,a_{i,h}^\tau)  \phi(x_{i,h}^\tau,a_{i,h}^\tau) ^\top + \lambda\cdot I$ is the covariance matrix on dataset $\cD_i$, and 
$\|\phi(s_h,a_h)\|_{\Lambda_{i,h}^{-1}}$ indicates how well the state-action pair $(s_h,a_h)$ is covered by the dataset $\cD_i$. $\sum_{i=1}^n\sum_{h=1}^H \EE_{i,\pi^*}\Bigl[ \|\phi(s_h,a_h)\|_{\Lambda_{i,h}^{-1}} \biggiven s_1=x_1\Bigr]$ in the suboptimality gap in Theorem \ref{thm:regret_upper_linear} is small if for each context $i\in[n]$, the dataset $\cD_i$ well covers the trajectory induced by the optimal policy $\pi^*$ on the corresponding MDP $\cM_i$.

% \begin{algorithm}[H]
% \begin{small}
%   \caption{Conservative Online-based Offline RL Algorithm: Linear MDP}\label{alg:modelfree linear mdp}
%   \begin{algorithmic}[1]
%     \REQUIRE Offline dataset $\cD=\{\cD_{i,h}\}_{i=1,h=1}^{n,H}$
%     \STATE Set $\pi_{0,h}(\cdot|\cdot)$ as uniform distribution over $A$ and $Q_{0,h}(\cdot|\cdot)$ as zero functions.
%     \FOR{$i=1, 2, \cdots, n$}
%     \STATE Set $\pi_{i,h}(\cdot|\cdot)\propto \pi_{i-1,h}(\cdot|\cdot)\cdot \exp(\alpha\cdot \hat{Q}_{i-1,h}(\cdot, \cdot))$
%     \STATE Set $\hat{V}_{i,H+1}(\cdot) = 0$
%         \FOR{$h=H,\cdots,1$}
% \STATE Set $\Lambda_{i,h} \leftarrow \sum_{\tau=1}^K \phi(x_{i,h}^\tau,a_{i,h}^\tau)  \phi(x_{i,h}^\tau,a_{i,h}^\tau) ^\top + \lambda\cdot I$. %\hfill  {//Estimation}
% \STATE Set $\hat{w}_{i,h}\leftarrow  \Lambda_{i,h} ^{-1}( \sum_{\tau=1}^{K} \phi(x_{i,h}^\tau,a_{i,h}^\tau) \cdot (r_{i,h}^\tau + \hat{V}_{i,h+1}(x_{i,h+1}^{\tau})) ) $. 
% \STATE Set $\Gamma_{i,h}(\cdot,\cdot) \leftarrow \beta(\delta/n)\cdot ( \phi(\cdot,\cdot)^\top  \Lambda_{i,h} ^{-1} \phi(\cdot,\cdot) )^{1/2}$. 
% \STATE Set $\hat{Q}_{i,h}(\cdot,\cdot) \leftarrow \min\{\phi(\cdot,\cdot)^\top \hat{w}_{i,h} - \Gamma_{i,h}(\cdot,\cdot),H-h+1\}^+$. 
%         \item Set $\hat{V}_{i,h}(\cdot)\leftarrow \la \pi^i_h(\cdot|\cdot),  \hat{Q}_{i,h}(\cdot, \cdot)\ra_A$.
%         \ENDFOR
%     \ENDFOR
%     \STATE Return $\hat{\pi}^*=\text{random}(\pi_1, ..., \pi_n)$
%   \end{algorithmic}
% \end{small}
% \end{algorithm}

\noindent\textbf{Well-explored behavior policy} Next we consider a case where the dataset $\cD$ consists of i.i.d. trajectories collecting from different environments. Suppose $\cD$ consists of $n$ independent datasets $\cD_1,\ldots,\cD_n$, and for each environment $i$, $\cD_i$ consists of $K$ trajectories $\cD_i = \{(x_{i,h}^\tau, a_{i,h}^\tau, r_{i,h}^\tau)_{h=1}^H\}_{\tau=1}^{K}$ independently and identically induced by a fixed behavior policy $\bar\pi_i$ in the linear MDP $\cM_i$. We have the following assumption on well-explored policy:
\begin{definition}[\cite{duan2020minimax, jin2021pessimism}]\label{ass:wellexp}
     For an behavior policy $\bar\pi$ and an episodic linear MDP $\cM$ with feature map $\phi$, we say $\bar\pi$ well-explores $\cM$ with constant $c$ if there exists an absolute positive constant $c > 0$ such that 
         \begin{align*}
   &\forall h \in [H], \lambda_{\min}(\Sigma_{h})\geq c/d, \text{where~~} \Sigma_{h} = \EE_{\bar\pi, \cM}\bigl[\phi(s_h,a_h)\phi(s_h,a_h)^\top\bigr].\notag
\end{align*}
\end{definition}
A well-explored policy guarantees that the obtained trajectories is ``uniform" enough to represent any policy and value function. The following corollary shows that with the above assumption, the suboptimality gaps of Algo.\ref{alg:erm} (with subroutine Algo.\ref{alg:linear mdp model based}) and Algo.\ref{alg:modelfree} (with subroutine Algo.\ref{alg:linear mdp model based}) decay to 0 when $n$ and $K$ are large enough.

\begin{corollary}\label{cor:well_explore}
Suppose that for each $i\in[n]$, $\cD_i$ is generated by behavior policy $\bar\pi_i$ which well-explores MDP $\cM_i$ with constant $c_i\geq c_{\text{min}}$. In Algo.\ref{alg:linear mdp model based}, we set $\lambda=1,\beta(\delta) = c'\cdot dH\sqrt{\log(4dHK/\delta)}$ where $c' >0$ is a positive constant. 
Suppose we have 
$K \geq 40d/c_{\text{min}}\log (4 dnH/ \delta)$ and set $C_n^*:=1/n\cdot \sum_{i=1}^n c_i^{-1/2}$. Then we have: \\
(i) for the output $\pi^{\text{PERM}}$ of Algo.\ref{alg:erm} with subroutine Algo.\ref{alg:linear mdp model based}, w.p. at least $1-\delta$, the suboptimality gap satisfies 

    \begin{align}
&\text{SubOpt}(\pi^{\text{PERM}})\leq 7\sqrt{\frac{2\log(6\mathcal{N}_{(Hn)^{-1}}^\Pi/\delta)}{n}}\notag\\
&+2\sqrt{2} c'\cdot d^{3/2} H^2 K^{-1/2} \sqrt{\log(12dHnK\mathcal{N}_{(Hn)^{-1}}^\Pi/\delta)}\cdot C_n^*\,,
\label{eq:event_opt_explore_d model based}
\end{align}

(ii) for the output policy $\pi^{\text{PPPO}}$ of Algo.\ref{alg:modelfree} with subroutine Algo.\ref{alg:linear mdp model based}, setting $\delta = 1/8$, then with probability at least $2/3$, the suboptimality gap satisfies
\begin{small}
    \begin{align}&\normalfont
\text{SubOpt}(\pi^{\text{PPPO}})\leq 10\bigg(\sqrt{\frac{\log|\actions|H^2}{n}}+2\sqrt{2} c'\cdot d^{3/2} H^{2.5} K^{-1/2} \sqrt{\log(16dHnK)}\cdot C_n^*\bigg).
\label{eq:event_opt_explore_d model free}
\end{align}
\end{small}
\end{corollary}

\begin{remark}
    The mixed coverage parameter $C_n^*=\frac{1}{n}\sum_{i=1}^n\frac{1}{\sqrt{c_i}}$ is small if for any $i\in[n]$, $c_i$ is large, \emph{i.e.}, the minimum eigenvalue of $\Sigma_{i,h}=\EE_{\bar\pi_i, \cM_i}\bigl[\phi(s_h,a_h)\phi(s_h,a_h)^\top\bigr]$ is large. Note that $\lambda_{\text{min}}(\Sigma_{i,h})$ indicates how well the behavior policy $\bar\pi_i$ explores the state-action pairs on MDP $\cM_i$; this shows that if for each environment $i\in[n]$, the behavior policy explores $\cM_i$ well, the suboptimality gap will be small.
\end{remark}
\begin{remark}
    Under the same conditions of Corollary \ref{cor:well_explore}:\\
    (i) If $n\geq\frac{392\log(6\mathcal{N}_{(Hn)^{-1}}^\Pi/\delta)}{\epsilon^2}$ \\and $K\geq\max\{\frac{40d}{c_{\text{min}}}\log ( \frac{4dnH}{\delta}),\frac{32c'^2d^3H^4\log(12dHnK\mathcal{N}_{(Hn)^{-1}}^\Pi/\delta)C_n^{*2}}{\epsilon^2}\}$, then w.p. at least $1-\delta$, $\text{SubOpt}(\pi^{\text{PERM}})\leq \epsilon$.
    \\(ii)   If $n\geq\frac{400H^2\log(|\actions|)}{\epsilon^2}$ \\and $K\geq\max\{\frac{40d}{c_{\text{min}}}\log ( 16dnH),\frac{32c'^2d^3H^5\log(16dHnK)C_n^{*2}}{\epsilon^2}\}$, then w.p. at least $2/3$, $\text{SubOpt}(\pi^{\text{PPPO}})\leq \epsilon$.
\end{remark}
Corollary \ref{cor:well_explore} suggests that both of our proposed algorithms enjoy the $O(n^{-1/2} + K^{-1/2}\cdot C_n^*)$ convergence rate to the optimal policy $\pi^*$ given a well-exploration data collection assumption, where $C_n^*$ is a mixed coverage parameter over $n$ environments defined in Corollary \ref{cor:well_explore}.

\chapter{Online Clustering of Bandits with Misspecified User Models}\label{chapter:mis}
The contextual linear bandit is an important online learning problem where given arm features, a learning agent selects an arm at each round to maximize the cumulative rewards in the long run. A line of works, called the clustering of bandits (CB), utilize the collaborative effect over user preferences and have shown significant improvements over classic linear bandit algorithms. However, existing CB algorithms require well-specified linear user models and can fail when this critical assumption does not hold. Whether robust CB algorithms can be designed for more practical scenarios with misspecified user models remains an open problem. In this paper, we are the first to present the important problem of clustering of bandits with misspecified user models (CBMUM), where the expected rewards in user models can be perturbed away from perfect linear models. We devise two robust CB algorithms, RCLUMB and RSCLUMB (representing the learned clustering structure with dynamic graph and sets, respectively), that can accommodate the inaccurate user preference estimations and erroneous clustering caused by model misspecifications. We prove regret upper bounds of $O(\epsilon_*T\sqrt{md\log T}  + d\sqrt{mT}\log T)$ for our algorithms under milder assumptions than previous CB works (notably, we move past a restrictive technical assumption on the distribution of the arms), which match the lower bound asymptotically in $T$ up to logarithmic factors, and also match the state-of-the-art results in several degenerate cases. The techniques in proving the regret caused by misclustering users are quite general and may be of independent interest. Experiments on both synthetic and real-world data show our outperformance over previous algorithms. This chapter is based on our publication \cite{wang2024onlinea}.

\section{Introduction}

Stochastic multi-armed bandit (MAB) \cite{auer2002finite,bubeck2012regret,lattimore2020bandit} is an online sequential decision-making problem,  where the learning agent
selects an action and receives a corresponding reward at each round, so as to maximize the cumulative reward in the long run. MAB algorithms have been widely applied in recommendation systems and computer networks to handle the exploration and exploitation trade-off \cite{kohli2013fast,liu2023variance,wang2023efficient,cai2018online}.

To deal with large-scale applications, the contextual linear bandits \cite{li2010contextual,chu2011contextual,abbasi2011improved,liu2023contextual,kong2023best} have been studied, where the expected reward of each arm is assumed to be perfectly linear in their features. Leveraging the contextual side information
about the user and arms, linear bandits can provide more personalized recommendations \cite{hariri2014context}. Classical linear bandit approaches, however, ignore the often useful tool of collaborative filtering. To utilize the relationships among users, the problem of clustering of bandits (CB) has been proposed \cite{gentile2014online}. Specifically, CB algorithms adaptively partition users into clusters and utilize the collaborative effect of users to enhance learning performance.

Although existing CB algorithms have shown great success in improving recommendation qualities, there exist two major limitations. First, all previous works on CB \cite{gentile2014online,li2018online,10.5555/3367243.3367445,wang2023online} assume that for each user, the expected rewards follow a \textit{perfectly linear} model with respect to the user preference vector and arms' feature vectors. In many real-world scenarios, due to feature noises or uncertainty \cite{hainmueller2014kernel}, the reward may not necessarily conform to a perfectly linear function, or even deviates a lot from linearity \cite{ghosh2017misspecified}. Second, previous CB works assume that for users within the same cluster, their preferences are exactly the same. Due to the heterogeneity in users' personalities and interests, similar users may not have identical preferences, invalidating this strong assumption. 

To address these issues, we propose a novel problem of clustering of bandits with misspecified user models (CBMUM). In CBMUM, the expected reward model of each user does not follow a perfectly linear function but with possible additive deviations. We assume users in the same underlying cluster share a common preference vector, meaning they have the same linear part in reward models, but the deviation parts are allowed to be different, better reflecting the varieties of user personalities. 

The relaxation of perfect linearity and the reward homogeneity within the same cluster bring many challenges to the CBMUM problem. In CBMUM, we not only need to handle the uncertainty from the \textit{unknown} user preference vectors, but also have to tackle the additional uncertainty from model misspecifications. Due to such uncertainties, it becomes highly challenging to design a robust algorithm that can cluster the users appropriately and utilize the clustered information judiciously. On the one hand, the algorithm needs to be more tolerant in the face of misspecifications so that more similar users can be clustered together to utilize the collaborative effect. On the other hand, it has to be more selective to rule out the possibility of \textit{misclustering} users with large preference gaps.
% On the other hand, it has to be more selective and perhaps conduct a double-check procedure to rule out the possibility of \textit{misclustering} dissimilar users together.

\subsection{Our Contributions}

This paper makes the following four contributions. 

\noindent{\textbf{New Model Formulation.}}
We are the first to formulate the clustering of bandits with misspecified user models (CBMUM) problem, which is more practical by removing the perfect linearity assumption in previous CB works.

\noindent{\textbf{Novel Algorithm Designs.}} 
 % We design a novel algorithm RCLUMB, which robustly learns the clustering structure and utilizes this collaborative information for faster user preference elicitation. Specifically, RCLUMB keeps updating a dynamic graph over all users, where users connected directly by edges are supposed to be in the same cluster. RCLUMB will adaptively remove edges  and recommend items based on historical interactive information.
We design two novel algorithms, RCLUMB and RSCLUMB, which robustly learn the clustering structure and utilize this collaborative information for faster user preference elicitation. Specifically, RCLUMB keeps updating a dynamic graph over all users, where users connected directly by edges are supposed to be in the same cluster. RCLUMB adaptively removes edges and recommends items based on historical interactions. RSCLUMB represents the clustering structure with sets, which are dynamicly merged and split during the learning process. Due to the page limit, we only illustrate the RCLUMB algorithm in the main paper. We leave the exposition, illustration, and regret analysis of the RSCLUMB algorithm in Appendix \ref{RSCLUMB section}.

To overcome the challenges brought by model misspecifications, we do the following key steps in the RCLUMB algorithm.
(i) To ensure that with high probability, similar users will not be partitioned apart, we design a more tolerant edge deletion rule by taking model misspecifications into consideration. (ii) Due to inaccurate user preference estimations caused by model misspecifications, trivially following previous CB works \cite{gentile2014online,li2018online,liu2022federated} to directly use connected components in the maintained graph as clusters would \textit{miscluster} users with big preference gaps, causing a large regret. To be discriminative in cluster assignments, we filter users directly linked with the current user in the graph to form the cluster used in this round. With these careful designs of (i) and (ii), we can guarantee that with high probability, information of all similar users can be leveraged, and only users with close enough preferences might be \textit{misclustered}, which will only mildly impair the learning accuracy. Additionally: (iii) we design an enlarged confidence radius to incorporate both the exploration bonus and the additional uncertainty from misspecifications when recommending arms. 
% With all these careful designs, we can guarantee that with high probability, information of all similar users can be leveraged, and only users with close enough preferences might be \textit{misclustered}, which will only mildly impair the learning accuracy. 
The design of RSCLUMB follows similar ideas, which we leave in the Appendix \ref{RSCLUMB section} due to page limit.

\noindent{\textbf{Theoretical Analysis with Milder Assumptions}.}
We prove regret upper bounds for our algorithms of $O(\epsilon_*T\sqrt{md\log T}  + d\sqrt{mT}\log T)$ in CBMUM under much milder and practical assumptions (in arm generation distribution) than previous CB works, which match the state-of-the-art results in degenerate cases. Our proof is quite different from the typical proof flow of previous CB works (details in Appendix \ref{appendix: theory}). One key challenge is to bound the regret caused by \textit{misclustering} users with close but not the same preference vectors and use the inaccurate cluster-based information to recommend arms. 
% Some straightforward methods would fall victim to making
% mistakes. 
To handle the challenge, we prove a key lemma (Lemma \ref{bound for mis}) to bound this part of regret. We defer its details in Section \ref{section: theoretical} and Appendix \ref{key lemma appendix}. The techniques and results for bounding this part are quite general and may be of independent interest.
We also give a regret lower bound of $\Omega(\epsilon_*T\sqrt{d})$ for CBMUM, showing that our upper bounds are asymptotically tight with respect to $T$ up to logarithmic factors. We leave proving a tighter lower bound for CBMUM as an open problem.

\noindent{\textbf{Good Experimental Performance.}}
Extensive experiments on both synthetic and real-world data show the advantages of our proposed algorithms over the existing algorithms.

\vspace{-0.36cm}
\section{Problem Setup}
\label{sec3}
\vspace{-0.26cm}
This section formulates the problem of ``clustering of bandits with misspecified user models" (CBMUM). We use boldface \textbf{lowercase} and boldface
\textbf{CAPITALIZED} letters for vectors and matrices. We use $\left|\mathcal{A}\right|$ to denote the number of elements in $\mathcal{A}$, $[m]$ to denote $\{1,\ldots,m\}$, and $\norm{\boldsymbol{x}}_{\boldsymbol{M}}=\sqrt{\boldsymbol{x}^{\top}\boldsymbol{M}\boldsymbol{x}}$ to denote the
matrix norm of vector $\bx$ regarding the positive semi-definite (PSD) matrix $\boldsymbol{M}$. 
% We use $\norm{\bx}$ to denote the $l_2$ norm of $\bx$, unless otherwise stated. 

In CBMUM, there are $u$ users denoted by $\mathcal{U}=\{1,2,\ldots,u\}$. Each user $i\in \mathcal{U}$ is associated with an \textit{unknown} preference vector $\btheta_i\in\mathbb{R}^d$, with $\norm{\btheta_i}_2\leq 1$. We assume there is an \textit{unknown} underlying clustering structure over users representing the similarity of their behaviors. Specifically, $\mathcal{U}$ can be partitioned into a small number $m$ (i.e., $m\ll u$) clusters, $V_1, V_2,\ldots V_m$, where $\cup_{j \in [m]}V_j=\mathcal{U},$ and $V_j \cap V_{j'}=\emptyset,$ for $j\neq j'$. We call these clusters \gtclusters{} and use $\mathcal{V}=\{V_1, V_2,\ldots, V_m\}$ to denote the set of these clusters. Users in the same \gtcluster{} share the same preference vector, while users from different \gtclusters{} have different preference vectors. Let $\btheta^j$ denote the common preference vector for $V_j$ and $j(i) \in [m]$ denote the index of the \gtcluster{} that user $i$ belongs to. For any $\ell\in\mathcal{U}$, if $\ell\in V_{j(i)}$, then $\btheta_\ell=\btheta_i=\btheta^{j(i)}$.
% \xutong{$V_{j(i)}$ is a set, and should not be used as an index! Use $\theta^{j(i)}$ instead, for the rest of the paper.}
% \begin{equation*}
%     \btheta_i=\btheta^{V_{j(i)}},\forall{i\in\mathcal{U}},j(i)\in[m]\,, 
% \end{equation*}
% where we use superscript $\btheta^{V_j}$ to denote the common preference vector of the users in the \gtcluster{} $V_j$, and we denote
% $[m]=\{1,2,\ldots,m\}$.

 At each round $t\in[T]$, a user $i_t\in\mathcal{U}$ comes to be served. The learning agent receives a finite arm set $\mathcal{A}_t\subseteq\mathcal{A}$ to choose from (with $\left|\mathcal{A}_t\right|\leq C, \forall{t}$), where each arm $a\in \mathcal{A}$ is associated with a feature vector $\bx_a\in\RR^d$, and $\norm{\bx_a}_2  \le 1$. The agent assigns an appropriate cluster $\overline{V}_t$ for user $i_t$ and recommends an item $a_t\in \mathcal{A}_t$ based on the aggregated historical information gathered from cluster $\overline{V}_t$. After receiving the recommended item $a_t$, user $i_t$ gives a random reward $r_t\in [0,1]$ to the agent. To better model real-world
scenarios,
 we assume that the reward $r_t$ follows a misspecified linear function of the item feature vector $\bx_{a_t}$ and the $\textit{unknown}$ user preference vector $\btheta_{i_t}$. Formally, 
 \begin{small}
     \begin{equation}    r_t=\bx_{a_t}^\top\btheta_{i_t}+\bepsilon^{i_t,t}_{a_t}+\eta_t\,,
    \label{reward def}
\end{equation}
 \end{small}
where $\bepsilon^{i_t,t}=[\bepsilon^{i_t,t}_1,\bepsilon^{i_t,t}_2,\ldots,\bepsilon^{i_t,t}_{\left|\mathcal{A}_t\right|}]^\top\in\RR^{\left|\mathcal{A}_t\right|}$ denotes the \textit{unknown} deviation in the expected rewards of arms in $\mathcal{A}_t$ from linearity for user $i_t$ at $t$, and $\eta_t$ is the 1-sub-Gaussian noise.
We allow the deviation vectors for users in the same \gtcluster{} to be different.
% We assume the clusters, users, items, and model misspecifications satisfy the following assumptions. All these assumptions follow the settings from classical works on CB \cite{gentile2014online,li2018online,liu2022federated} and misspecified linear bandits \cite{lattimore2020learning}.

We assume the clusters, users, items, and model misspecifications satisfy the following assumptions.
% The first three assumptions follow the settings from classical works on CB \cite{gentile2014online,li2018online}.

\begin{assumption}[Gap between different clusters]
\label{assumption1}
The gap between any two preference vectors for different \gtclusters{} is at least an \textit{unknown} positive constant $\gamma$
\begin{small}
\begin{equation*}
    \norm{\btheta^{j}-\btheta^{j^{\prime}}}_2\geq \gamma>0\,, \forall{j,j^{\prime}\in [m]\,, j\neq j^{\prime}}\,.
\end{equation*}  
\end{small}
\end{assumption}

\begin{assumption}[Uniform arrival of users]
\label{assumption2}
At each round $t$, a user $i_t$ comes uniformly at random from $\mathcal{U}$ with probability $1/u$, independent of the past rounds.
\end{assumption}

\begin{assumption}[Item regularity]
\label{assumption3}
At each time step $t$, the feature vector $\bx_a$ of each arm $a\in \mathcal{A}_t$ is drawn independently from a fixed but unknown distribution $\rho$ over $\{\bx\in\RR^d:\norm{\bx}_2\leq1\}$, where $\EE_{\bx\sim \rho}[\bx \bx^{\top}]$ is full rank with minimal eigenvalue $\lambda_x > 0$. Additionally, at any time $t$, for any fixed unit vector $\btheta \in \RR^d$, $(\btheta^{\top}\bx)^2$ has sub-Gaussian tail with variance upper bounded by $\sigma^2$.
\end{assumption}

\begin{assumption}[Bounded misspecification level]
\label{assumption4}
% We assume $\norm{\bepsilon^{i,t}}_{\infty}\leq \epsilon_*$, $\forall{i\in\mathcal{U}, t\in [T]}$, and the maximum misspecification level parameter $\epsilon_*$ is pre-specified.
We assume that there is a pre-specified maximum misspecification level parameter $\epsilon_*$ such that $\norm{\bepsilon^{i,t}}_{\infty}\leq \epsilon_*$, $\forall{i\in\mathcal{U}, t\in [T]}$.
\end{assumption}
\vspace{-0.2cm}
\noindent\textbf{Remark 1.} All these assumptions basically follow previous works on CB \cite{gentile2014online,gentile2017context,
li2018online,
ban2021local,
liu2022federated} and MLB \cite{lattimore2020learning}. Note that Assumption \ref{assumption3} is less stringent and more practical than previous CB works which also put restrictions on the variance upper bound $\sigma^2$. For Assumption \ref{assumption2}, our results can easily generalize to the case where the user arrival follows any distributions with minimum arrival probability greater than $p_{min}$. For Assumption \ref{assumption4}, note that $\epsilon_*$ can be an upper bound on the maximum misspecification level, not the exact maximum itself. In real-world applications, the deviations are usually small \cite{ghosh2017misspecified}, and we can set a relatively big $\epsilon_{*}$ as an upper bound. For more discussions please refer to Appendix \ref{appendix: assumptions}

Let $a_t^*\in{\arg\max}_{a\in{\mathcal{A}_t}}\bx_a^{\top}\btheta_{i_t}+\bepsilon^{i_t,t}_a$ denote an optimal arm which gives the highest expected reward at $t$. The goal of the agent is to minimize the expected cumulative regret
\begin{small}
\begin{equation}\textstyle{
R(T)=\EE[\sum_{t=1}^T(\bx_{a_t^*}^{\top}\btheta_{i_t}+\bepsilon^{i_t,t}_{a_t^*}-\bx_{a_t}^{\top}\btheta_{i_t}-\bepsilon^{i_t,t}_{a_t})]\,.}
\label{regret def}
\end{equation} 
\end{small}

% where the expectation is taken over the randomness of the algorithm and the environment
% regarding the users $i_1,\ldots, i_T$, the deviations and the arm sets $\mathcal{A}_1,\ldots,\mathcal{A}_T$.
% \xutong{where the expectation is taken over the randomness of what(?)}

\section{Algorithm}
\label{section: algorithm}
\begin{algorithm}[tbh!] %RCLUMB
\caption{Robust Clustering of Misspecified Bandits Algorithm (RCLUMB)}
\resizebox{1\columnwidth}{!}{
\begin{minipage}{\columnwidth}
\label{alg:RCLUMB}
\begin{algorithmic}[1]
\STATE {{\bf Input:}  Deletion parameter $\alpha_1,\alpha_2>0$, $f(T)=\sqrt{\frac{1 + \ln(1+T)}{1 + T}}$, $\lambda, \beta, \epsilon_*>0$.}
\STATE {{\bf Initialization:} 
% \begin{itemize}
% \item $\bM_{i,0} = 0_{d\times d}, \bb_{i,0} = 0_{d \times 1}, T_{i,0}=0$ , $\forall{i \in \mathcal{U}}$;
% \item A complete Graph $G_0 = (\mathcal{U},E_0)$ over $\mathcal{U}$.
% \end{itemize}
$\bM_{i,0} = 0_{d\times d}, \bb_{i,0} = 0_{d \times 1}, T_{i,0}=0$ , $\forall{i \in \mathcal{U}}$; a complete Graph $G_0 = (\mathcal{U},E_0)$ over $\mathcal{U}$.
\label{alg:RCLUMB:initialize} }
\FORALL {$t=1,2,\ldots, T$}
\STATE {Receive the index of the current user $i_t\in \mathcal{U}$, and the current feasible arm set $\mathcal{A}_t$; \label{alg:RCLUMB:receive user index} }
\STATE {\label{alg:RCLUMB:find V} Filter user $i_t$ and users $i\in \mathcal{U}$ that are \textit{directly} connected with user $i_t$ via edge $(i,i_t)\in E_{t-1}$, to form the cluster $\overline{V}_t$;} 
\STATE {\label{alg:RCLUMB:compute common vector} Compute the estimated statistics for cluster $\overline{V}_t$ 
% \begin{align*}
% \overline{\bM}_{\overline{V}_t,t-1} &= \lambda \bI + \sum_{i \in \overline{V}_t} \bM_{i,t-1}\,,\\ \overline{\bb}_{\overline{V}_t,t-1} = \sum_{i \in \overline{V}_t} &\bb_{i,t-1}\,,
% \hat{\btheta}_{\overline{V}_t,t-1} = \overline{\bM}_{\overline{V}_t,t-1}^{-1}\overline{\bb}_{\overline{V}_t,t-1}; 
% \end{align*}\vspace{-0.5cm}
% }
\begin{align*}
\textstyle{\overline{\bM}_{\overline{V}_t,t-1} = \lambda \bI + \sum_{i \in \overline{V}_t} \bM_{i,t-1}\,,\overline{\bb}_{\overline{V}_t,t-1} = \sum_{i \in \overline{V}_t} \bb_{i,t-1}\,,
\hat{\btheta}_{\overline{V}_t,t-1} = \overline{\bM}_{\overline{V}_t,t-1}^{-1}\overline{\bb}_{\overline{V}_t,t-1};}
\end{align*}
}
\STATE {\label{alg:RCLUMB:recommend and receive feedback}
Recommend an arm $a_t$ with the largest UCB index (Eq.(\ref{UCB})),
and receive the reward $r_t\in[0,1]$;}
\STATE {\label{alg:RCLUMB:update1} Update the statistics for user $i_t$
% \begin{align*}
% \bM_{i_t,t} &= \bM_{i_t,t-1} + \bx_{a_t}\bx_{a_t}^{\top}\,, \\
% \bb_{i_t,t} = \bb_{i_t,t-1} &+ r_t\bx_{a_t}\,,
% T_{i_t,t} = T_{i_t,t-1} + 1\,,\\
% \hat{\btheta}_{i_t,t} &= (\lambda \bI + \bM_{i_t,t})^{-1} \bb_{i_t,t};
% \end{align*}\vspace{-0.5cm}
% }
% \begin{align*}
$\bM_{i_t,t} = \bM_{i_t,t-1} + \bx_{a_t}\bx_{a_t}^{\top}\,, 
\bb_{i_t,t} = \bb_{i_t,t-1} + r_t\bx_{a_t}\,,
T_{i_t,t} = T_{i_t,t-1} + 1\,,
\hat{\btheta}_{i_t,t} = (\lambda \bI + \bM_{i_t,t})^{-1} \bb_{i_t,t}$;
% \end{align*}\vspace{-0.5cm}
}
\STATE {Keep the statistics of other users unchanged\label{alg:RCLUMB:update2}\\
% \begin{equation*}
% \bM_{\ell,t} = \bM_{\ell,t-1}, \bb_{\ell,t} = \bb_{\ell,t-1}, T_{\ell,t} = T_{\ell,t-1},     
% \end{equation*}
% \quad\,\,\,\,$\hat{\btheta}_{\ell,t} = \hat{\btheta}_{\ell,t-1}$, for all $\ell\in\mathcal{U}, \ell \ne i_t$;
$\bM_{\ell,t} = \bM_{\ell,t-1}, \bb_{\ell,t} = \bb_{\ell,t-1}, T_{\ell,t} = T_{\ell,t-1},     
\hat{\btheta}_{\ell,t} = \hat{\btheta}_{\ell,t-1}$, for all $\ell\in\mathcal{U}, \ell \ne i_t$;
}
\STATE {\label{alg:RCLUMB:delete} Delete the edge $(i_t, \ell) \in E_{t-1}$, if
\begin{equation*}
    \norm{\hat{\btheta}_{i_t,t} - \hat{\btheta}_{\ell,t}}_2 \ge \alpha_1\bigg(f(T_{i_t,t})+f(T_{\ell,t})\bigg)+\alpha_2\epsilon_*\,,
\end{equation*}
and get an updated graph $G_t = (\mathcal{U}, E_t)$;}
\ENDFOR 
% {~~$t$}
\end{algorithmic}
\end{minipage}
}
\end{algorithm}
This section introduces our algorithm called ``Robust CLUstering of Misspecified Bandits" (RCLUMB) (Algo.\ref{alg:RCLUMB}). RCLUMB is a graph-based algorithm. The ideas and techniques of RCLUMB can be easily generalized to set-based algorithms. To illustrate this generalizability, we also design a set-based algorithm RSCLUMB. We leave the exposition and analysis of RSCLUMB in Appendix \ref{RSCLUMB section}.

% We define the coefficient $\zeta=4\epsilon_*\sqrt{\frac{2}{\lambda_x}}$, which is theoretically the minimum gap between two users' preference vectors that an algorithm can distinguish their dissimilarity with high probability, as supported by the proof process of the following Lemma \ref{sufficient time} in the Appendix.
For ease of interpretation, we define the coefficient
\begin{small}
\begin{equation}
\zeta\triangleq2\epsilon_*\sqrt{\frac{2}{\tilde{\lambda}_x}}\,,\label{zeta}   
% \zeta\triangleq2\sqrt{2}\epsilon_*\sqrt{{1}/{\tilde{\lambda}_x}}\,,\label{zeta}   
\end{equation}
\end{small}
where $\tilde{\lambda}_x\triangleq\int_{0}^{\lambda_x} (1-e^{-\frac{(\lambda_x-x)^2}{2\sigma^2}})^{C} dx$. $\zeta$ is theoretically the minimum gap between two users' preference vectors that an algorithm can distinguish with high probability, as supported by Eq.(\ref{condition gamma/4}) in the proof of Lemma \ref{sufficient time} in Appendix \ref{section T0}. Note that the algorithm does not require knowledge of $\zeta$. We also make the following definition for illustration.

 % For ease of illustration, we make the following definition.
\begin{definition}[$\zeta$-close users and $\zeta$-good clusters]\label{def:close}
Two users $i, i^{\prime} \in \mathcal{U}$ are $\zeta$-close if $\norm{\btheta_i-\btheta_{i^{\prime}}}_2\le \zeta$. Cluster $\overline{V}$ is a $\zeta$-good cluster at time $t$, if $\forall \, i \in \overline{V}$, user $i$ and the coming user $i_t$ are $\zeta$-close.
\end{definition}
We also say that two \gtclusters{} are ``$\zeta$-close" if their preference vectors' gap is less than $\zeta$.

Now we introduce the process and intuitions of RCLUMB (Algo.\ref{alg:RCLUMB}).
The algorithm maintains an undirected user graph $G_t=(\mathcal{U}, E_t)$, where users are connected with edges if they are inferred to be in the same cluster. 
We denote the connected component in $G_{t-1}$ containing user $i_t$ at round $t$ as $\tilde{V}_t$.

\noindent\textbf{Cluster Detection.} $G_0$ is initialized to be a complete graph, and will be updated adaptively based on the interactive information. At round $t$, user $i_t\in\mathcal{U}$ comes to be served with a feasible arm set $\mathcal{A}_t$ (Line \ref{alg:RCLUMB:receive user index}). 

Due to model misspecifications, it is impossible to cluster users with exactly the same preference vector $\btheta$, but similar users whose preference vectors are within the distance of $\zeta$. According to the proof of Lemma \ref{sufficient time}, after a sufficient time, with high probability, any pair of users directly connected by an edge in $E_{t-1}$ are $\zeta$-close. However, if we trivially follow previous CB works \cite{gentile2014online,li2018online,liu2022federated} to directly use the connected component $\tilde{V}_t$ as the inferred cluster for user $i_t$ at round $t$, it will cause a large regret. The reason is that in the worst case, the preference vector $\btheta$ of the user in $\tilde{V}_t$ who is $h$-hop away from user $i_t$ could deviate by $h\zeta$ from $\btheta_{i_t}$, where $h$ can be as large as $|\tilde{V}_t|$. Based on this reasoning, our key point is to select the cluster $\overline{V}_t$ as the users at most 1-hop away from $i_t$ in the graph. In other words, after some interactions, $\overline{V}_t$ forms a $\zeta$-good cluster with high probability; thus, RCLUMB can avoid using misleading information from dissimilar users for recommendations.

\noindent\textbf{Cluster-based Recommendation.} 
After finding the appropriate cluster $\overline{V}_t$ for $i_t$, the agent estimates the common user preference vector based on the historical information associated with cluster $\overline{V}_t$ by 

    \begin{equation}
\textstyle{\hat{\btheta}_{\overline{V}_t,t-1}=\mathop{\arg\min}\limits_{\btheta\in\RR^d}\sum_{s\in[t-1]\atop i_s\in \overline{V}_t}(r_s-\bx_{a_s}^{\top}\btheta)^2+\lambda\norm{\btheta}_2^2\,,}
\end{equation}

where $\lambda>0$ is a regularization coefficient. Its closed-form solution is $\hat{\btheta}_{\overline{V}_t,t-1}=\overline{\bM}_{\overline{V}_t,t-1}^{-1}\overline{\bb}_{\overline{V}_t,t-1}$,\\
where $\overline{\bM}_{\overline{V}_t,t-1}=\lambda\bI+\sum_{s\in[t-1]\atop i_s\in \overline{V}_t}\bx_{a_s}\bx_{a_s}^{\top}\,,\,
    \overline{\bb}_{\overline{V}_t,t-1}=\sum_{s\in[t-1]\atop i_s\in \overline{V}_t}r_{a_s}\bx_{a_s}$\,.

Based on this estimation, in Line \ref{alg:RCLUMB:recommend and receive feedback}, the agent recommends an arm using the UCB strategy
\begin{small}
  \begin{equation}
\label{UCB}
    \begin{aligned}
    a_t&= \argmax_{a\in \mathcal{A}_t}\min \{1, \underbrace{\bx_a^{\top}\hat{\btheta}_{\overline{V}_t,t-1}}_{\hat{R}_{a,t}} + \underbrace{\beta \norm{\bx_a}_{\overline{\bM}_{\overline{V}_t,t-1}^{-1}}
    +\epsilon_*\textstyle{\sum_{s\in[t-1]\atop i_s\in \overline{V}_t}}\left|\bx_a^{\top}\overline{\bM}_{\overline{V}_t,t-1}^{-1}\bx_{a_s}\right|}_{C_{a,t}}\}\,,
    \end{aligned}
\end{equation}   
\end{small}

where \begin{small}
$\beta=\sqrt{\lambda}+\sqrt{2\log(\frac{1}{\delta})+d\log(1+\frac{T}{\lambda d})}$,  $\hat{R}_{a,t}$     
\end{small}denotes the estimated reward of arm $a$ at $t$, $C_{a,t}$ denotes the confidence radius of arm $a$ at round $t$.

% Due to the deviations from linearity, the estimated reward $\hat{R}_{a,t}$ computed by a linear function is no longer accurate. To handle the estimation uncertainty caused by model misspecifications, we design an enlarged confidence radius $C_{a,t}$. The first term of $C_{a,t}$ in Eq.(\ref{UCB}) captures the uncertainty of online learning for the linear part, and the second term related to $\epsilon_*$ reflects the additional uncertainty caused by deviations from linearity. The design of $C_{a,t}$ theoretically relies on Lemma \ref{concentration bound} which will be given in Section \ref{section: theoretical}.
Due to deviations from linearity, the estimation $\hat{R}_{a,t}$ computed by a linear function is no longer accurate. To handle the estimation uncertainty of model misspecifications, we design an enlarged confidence radius $C_{a,t}$. The first term of $C_{a,t}$ in Eq.(\ref{UCB}) captures the uncertainty of online learning for the linear part, and the second term related to $\epsilon_*$ reflects the additional uncertainty from deviations from linearity. The design of $C_{a,t}$ theoretically relies on Lemma \ref{concentration bound} which will be given in Section \ref{section: theoretical}.
% \xutong{I don't know why ``Section 5" does not show correctly here, please help fix here.}

\noindent\textbf{Update User Statistics.} Based the feedback $r_t$, in Line \ref{alg:RCLUMB:update1} and \ref{alg:RCLUMB:update2}, the agent updates the statistics for user $i_t$. Specifically, the agent estimates the preference vector $\btheta_{i_t}$ by
\begin{equation}    \textstyle{\hat{\btheta}_{i_t,t}=\mathop{\arg\min}\limits_{\btheta\in\RR^d}\sum_{s\in[t]\atop i_s=i_t}(r_s-\bx_{a_s}^{\top}\btheta)^2+\lambda\norm{\btheta}_2^2\,,}
\end{equation}
with solution
% \begin{equation}
%         \hat{\btheta}_{i_t,t}={(\lambda\bI+\bM_{i_t,t})}^{-1}\bb_{i_t,t}\,,
% \end{equation}
% where $\bM_{i_t,t}=\sum_{s\in[t]\atop i_s=i_t}\bx_{a_s}\bx_{a_s}^{\top}\,, \bb_{i_t,t}=\sum_{s\in[t]\atop i_s=i_t}r_{a_s}\bx_{a_s}\,.$
$        \hat{\btheta}_{i_t,t}={(\lambda\bI+\bM_{i_t,t})}^{-1}\bb_{i_t,t}\,,$\\
where $\bM_{i_t,t}=\sum_{s\in[t]\atop i_s=i_t}\bx_{a_s}\bx_{a_s}^{\top}\,, \bb_{i_t,t}=\sum_{s\in[t]\atop i_s=i_t}r_{a_s}\bx_{a_s}\,.$  
% \begin{equation}
%     \begin{aligned}
%     \bM_{i_t,t}&=\sum_{s\in[t]\atop i_s=i_t}\bx_{a_s}\bx_{a_s}^{\top}\,,\\
%     \bb_{i_t,t}&=\sum_{s\in[t]\atop i_s=i_t}r_{a_s}\bx_{a_s}\,.
%     \end{aligned}
% \end{equation}
% \begin{equation}
%     \bM_{i_t,t}=\sum_{s\in[t]\atop i_s=i_t}\bx_{a_s}\bx_{a_s}^{\top}\,, \bb_{i_t,t}=\sum_{s\in[t]\atop i_s=i_t}r_{a_s}\bx_{a_s}\,.
% \end{equation}
% are the Gramian matrix and the moment matrix of regressand related with user $i_t$.

\noindent\textbf{Update the Graph $G_t$.} Finally, in Line \ref{alg:RCLUMB:delete}, the agent verifies whether the similarities between user $i_t$ and other users are still true based on the updated estimation $\hat{\btheta}_{i_t,t}$. For every user $\ell\in\mathcal{U}$ connected with user $i_t$ via edge $(i_t,\ell)\in E_{t-1}$, if the gap between her estimated preference vector $\hat{\btheta}_{\ell,t}$ and $\hat{\btheta}_{i_t,t}$ is larger than a threshold supported by Lemma \ref{sufficient time}, the agent will delete the edge $(i_t,\ell)$ to split them apart. 
The threshold in Line \ref{alg:RCLUMB:delete} is carefully designed, taking both estimation uncertainty in a linear model and deviations from linearity into consideration. As shown in the proof of Lemma \ref{sufficient time} (in Appendix \ref{section T0}), using this threshold, with high probability, edges between users in the same \gtclusters{} will not be deleted, and edges between users that are not $\zeta$-close will always be deleted. Together with the filtering step in Line \ref{alg:RCLUMB:find V}, with high probability, the algorithm will leverage all the collaborative information of similar users and avoid misusing the information of dissimilar users. The updated graph $G_t$ will be used in the next round.

\section{Theoretical Analysis}\label{section: theoretical}

In this section, we theoretically analyze the performance of the RCLUMB algorithm by giving an upper bound of the expected regret defined in Eq.(\ref{regret def}). Due to the space limitation, we only show the main result (Theorem \ref{thm:main}), key lemmas, and a sketched proof for Theorem \ref{thm:main}. Detailed proofs, other technical lemmas, and the regret analysis of the RSLUMB algorithm can be found in the Appendix.

% We first define the following constants given the problem instance, which will be used in this section.
To state our main result, we first give two definitions as follows. The first definition is about the minimum separable gap constant $\gamma_1$ of a CBMUM problem instance.
\begin{definition}[Minimum separable gap $\gamma_1$]
\label{def:gap}
The minimum separable gap constant $\gamma_1$ of a CBMUM problem instance is the minimum gap over the gaps among users that are greater than $\zeta$ (Eq. (\ref{zeta}))
\begin{equation}
    \gamma_1=\min\{\norm{\btheta_i-\btheta_\ell}_2:\norm{\btheta_i-\btheta_\ell}_2>\zeta,\forall{i,\ell\in\mathcal{U}}\}\,, \text{with} \min \emptyset=\infty.\nonumber
\end{equation}

\end{definition}

\noindent\textbf{Remark 2.}
% \st{Due to the estimation uncertainty caused by model misspecifications, the algorithm can only distinguish the dissimilarities between users with gaps greater than $\zeta$ with high probability, which gives the intuition of the above definition of the minimum separable gap $\gamma_1$}.
% \xutong{I do not see the intuition given by the first sentence.} 
In CBMUM, the role of $\gamma_1-\zeta$ is similar to that of $\gamma$ (given in Assumption \ref{assumption1}) in the previous CB problem with perfectly linear models, quantifying the hardness of performing clustering on the problem instance. Intuitively, users are easier to cluster if $\gamma_1$ is larger, and the deduction of $\zeta$ shows the additional difficulty due to model diviations. If there are no misspecifications, i.e., $\zeta=2\epsilon_*\sqrt{\frac{2}{\lambda_x}}=0$, then $\gamma_1=\gamma$, recovering the minimum separable gap between clusters in the classic CB problem \cite{gentile2014online,li2018online} without model misspecifications.
% recovering the results of previous works.

% Due to the estimation uncertainty caused by model misspecifications, the algorithm can only distinguish the dissimilarities between users with gaps greater than $\zeta$ with high probability, which gives the intuition of the above definition of the minimum separable gap $\gamma_1$. If there are no misspecifications, i.e., $\zeta=4\epsilon_*\sqrt{\frac{2}{\lambda_x}}=0$, then $\gamma_1=\gamma$, where $\gamma$ is the minimum gap between \gtclusters{} given in Assumption \ref{assumption1}. Note that in the CBMUM problem, the role of $\gamma_1-\zeta$ is similar to that of $\gamma$ in the previous CB problem with perfectly linear user models, quantifying the hardness of doing clustering on the problem instance. 
% , after a sufficient time $T_0$, edges between users with gap no less than $\gamma_1$ will be deleted from the graph with high probability.

The second definition is about the number of ``hard-to-cluster users" $\tilde{u}$.
\begin{definition}[Number of ``hard-to-cluster users" $\tilde{u}$]
\label{def:user number}
The number of ``hard-to-cluster users" $\tilde{u}$ is the number of users in the \gtclusters{} which are $\zeta$-close to some other \gtcluster{}s
\begin{equation}
    \tilde{u}=\sum_{j\in[m]}\left|V_j\right|\times\mathbb{I}\{\exists{j^{\prime}\in[m],j^{\prime}\neq j}: \norm{\btheta^{j^{\prime}}-\btheta^j}_2\leq\zeta\}\,,\notag
\end{equation}
where $\mathbb{I}\{\cdot\}$ denotes the indicator function of the argument, $\left|V_j\right|$ denotes the number of users in $V_j$.
\end{definition}

\noindent\textbf{Remark 3.}
% The algorithm can only distinguish the dissimilarities between users with gaps greater than $\zeta$.
$\tilde{u}$ captures the number of users who belong to different \gtclusters{} but their gaps are less than $\zeta$. These users may be merged into one cluster by mistake and cause certain regret.

The following theorem gives an upper bound on the expected regret achieved by RCLUMB.

\begin{theorem}[Main result on regret bound]
\label{thm:main}
Suppose that the assumptions in Section \ref{sec3} are satisfied. Then the expected regret of the RCLUMB algorithm for $T$ rounds satisfies
\begin{small}
\begin{align}
    R(T) &\le O\bigg(u\left( \frac{d}{\tilde{\lambda}_x (\gamma_1-\zeta)^2}+\frac{1}{\tilde{\lambda}_x^2}\right)\log T+\frac{\tilde{u}}{u}\frac{\epsilon_*\sqrt{d}T}{\tilde{\lambda}_x^{1.5}}
    +\epsilon_*T \sqrt{md\log T}  + d\sqrt{mT}\log T\bigg)\label{intial}\\
    &\le O(\epsilon_*T\sqrt{md\log T}  + d\sqrt{mT}\log T)\label{bigO}\,,
    % &\le O(\epsilon_*\sqrt{d}T (\sqrt{m\log T}+{\tilde{u}}/{u\lambda_x^{1.5}})  + d\sqrt{mT}\log T)
\end{align}
\end{small}
% where \revise{$\tilde{\lambda}_x\triangleq\int_{0}^{\lambda_x} (1-e^{-\frac{(\lambda_x-x)^2}{2\sigma^2}})^{\left|\mathcal{A}_t\right|} dx$},
% where $\gamma_1$ is the minimum separable gap (Definition \ref{def:gap}), $\tilde{u}$ is number of ``hard-to-cluster users" (Definition \ref{def:user number}).
where $\gamma_1$ is defined in Definition \ref{def:gap}, and $\tilde{u}$ is defined in Definition \ref{def:user number}).
\end{theorem}

\noindent\textbf{Discussion and Comparison.}
The bound in Eq.(\ref{intial}) has four terms. The first term is the time needed to gather enough information to assign appropriate clusters for users. The second term is the regret caused by \textit{misclustering} $\zeta$-close but not precisely similar users together, which is unavoidable with model misspecifications. The third term is from the preference estimation errors caused by model deviations. The last term is 
the usual term in CB with perfectly linear models \cite{gentile2014online,li2018online,10.5555/3367243.3367445}. 

% Note that the second and the third terms in Eq.(\ref{intial}) are linear in $T$. Lower bounds given in existing works \cite{ghosh2017misspecified,lattimore2020learning} indicate that a linear term is inevitable in the misspecified linear bandits problem.

Let us discuss how the parameters affect this regret bound.\\ 
% \begin{itemize}
%     \item First, if $\gamma_1-\zeta$ is large, meaning the gaps between clusters that are not ``$\zeta$-close" are much greater than the minimum gap $\zeta$ for the algorithm to distinguish, the first term of the regret in Eq.(\ref{intial}) will be small since it is easy to identify their dissimilarities. The role of $\gamma_1-\zeta$ in CBMUM is similar to that of $\gamma$ in the previous CB.\vspace{-0.2cm}
%     \item Second, if the number of ``hard-to-cluster users" $\tilde{u}$ is small, indicating that few \gtclusters{} are  ``$\zeta$-close", the algorithm will hardly \textit{miscluster} different \gtclusters{} together thus the second term of the regret in Eq.(\ref{intial}) will be small.\vspace{-0.2cm}
%     \item Finally, if the deviation level $\epsilon_*$ is small, meaning the user models are close to linear models and the misspecifications will not affect the estimations much, then both the second and third term in Eq.(\ref{intial}) will be small.
% \end{itemize}\vspace{-0.3cm}
    $\bullet$ If $\gamma_1-\zeta$ is large, the gaps between clusters that are not ``$\zeta$-close" are much greater than the minimum gap $\zeta$ for the algorithm to distinguish, the first term in Eq.(\ref{intial}) will be small as it is easy to identify their dissimilarities. The role of $\gamma_1-\zeta$ in CBMUM is similar to that of $\gamma$ in the previous CB.\\
    $\bullet$ If $\tilde{u}$ is small, indicating that few \gtclusters{} are  ``$\zeta$-close", RCLUMB will hardly \textit{miscluster} different \gtclusters{} together thus the second term in Eq.(\ref{intial}) will be small.\\
    $\bullet$ If the deviation level $\epsilon_*$ is small, the user models are close to linearity and the misspecifications will not affect the estimations much, then both the second and third term in Eq.(\ref{intial}) will be small.\\
The following theorem gives a regret lower bound of the CBMUM problem.
\begin{theorem}[Regret lower bound for CBMUM]
\label{thm: lower bound}
There exists a problem instance for the CBMUM problem such that for any algorithm
$R(T)\geq \Omega(\epsilon_*T\sqrt{d})\,.$
\end{theorem}

The proof can be found in Appendix \ref{appendix: lower bound}. The upper bounds in Theorem \ref{thm:main}
asymptotically match this lower bound with respect to $T$ up to logarithmic factors (and a constant factor of $\sqrt{m}$ where 
 $m$ is typically small in real-applications), showing the tightness of our theoretical results. Additionally, we conjecture the gap for the $m$ factor is due to the strong assumption that cluster structures are known to prove this lower bound, and whether there exists a tighter lower bound is left for future work.

We then
compare our results with two degenerate cases. First, when $m=1$ (indicating $\tilde{u}=0$), our setting degenerates to
the MLB problem where all users share the same preference vector. In this case, our regret bound is $O(\epsilon_*T\sqrt{d\log T}  + d\sqrt{T}\log T)$, exactly matching the current best bound of MLB \cite{lattimore2020learning}. Second, when $\epsilon_*=0$, our setting reduces to the CB problem with perfectly linear user models and our bounds become $O(d\sqrt{mT}\log T)$, also perfectly match the existing best bound of the CB problem \cite{li2018online,10.5555/3367243.3367445}. The above discussions and comparisons show the tightness of our regret bounds. Additionally, we also provide detailed discussions on why trivially combining existing works on CB and MLB would not get any non-vacuous regret upper bound in Appendix \ref{appendix: why trivial not apply}.

% \xutong{How the regret changes when the parameter changes ? Explain how the regret changes regarding $\gamma_1$, $\tilde{u}$, $\lambda_x$ so that the readers can better understand your result.}

% \xutong{How the current setting compares with the degenerate setting? Consider $m=1$, and $\epsilon=0$ so that the readers can feel the tightness of your result.}

% \xutong{How the current result compares with the lower bound? (to emphasize linear regret $T$ is inevitable.)}

% If we do not cluster the similar users and simply use an independent bandit algorithm of misspecified linear bandits for each user, the regret bound will have a term of $O(\sqrt{uT}d\log T)$. As the number of users $u$ is usually vast (e.g., millions of users) in real-world applications, it is worse than our bound in practice. If we directly use existing CLUB algorithms without considering the model misspecifications, to the best of our knowledge, we can not get any non-vacuous regret bounds.
We define the following ``good partition" for ease of interpretation.

% \xutong{If we do not have space, we can further cut down the proof. We only need to mention the challenge and the novel ideas to solve the challenge.}

\begin{definition} [Good partition]\label{def:good partition}
RCLUMB does a ``good partition" at $t$, if the cluster $\overline{V}_t$ assigned to $i_t$ is a $\zeta$-good cluster, and it contains all the users in the same \gtcluster{} as $i_t$, i.e.,
\begin{equation}
    \norm{\btheta_{i_t}-\btheta_\ell}_2\leq \zeta, \forall{\ell\in \overline{V}_t}\,,\text{and}\,V_{j(i_t)}\subseteq \overline{V}_t\,. 
\end{equation}
\end{definition}

Note that when the algorithm does a ``good partition" at $t$, $\overline{V}_t$ will contain all the users in the same \gtcluster{} as $i_t$ and may only contain some other $\zeta$-close users with respect to $i_t$, which means the gathered information 
 associated with $\overline{V}_t$ can be used to infer user $i_t$'s preference with high accuracy. Also, it is obvious that under a ``good partition", if $\overline{V}_t\in \mathcal{V}$, then $\overline{V}_t=V_{j(i_t)}$ by definition.

Next, we give a sketched proof for Theorem \ref{thm:main}. 

\begin{proof}\noindent\textbf{ [Sketch for Theorem \ref{thm:main}]} The proof mainly contains two parts. First, we prove there is a sufficient time $T_0$ for RCLUMB to get a ``good partition" with high probability. Second, we prove the regret upper bound for RCLUMB after maintaining a ``good partition". The most challenging part is to bound the regret caused by \textit{misclustering} $\zeta$-close users after getting a ``good partition".

\noindent\textbf{1. Sufficient time to maintain a ``good partition".} With the item regularity (Assumption \ref{assumption3}), we can prove after some $T_0$ (defined in Lemma \ref{sufficient time} in Appendix \ref{section T0}), RCLUMB will always have a ``good partition". After $t\geq O\big(u\left( \frac{d}{\tilde{\lambda}_x (\gamma_1-\zeta)^2}+\frac{1}{\tilde{\lambda}_x^2}\right)\log T\big)$,
for any user $i\in\mathcal{U}$, the gap between
the estimated $\hat{\btheta}_{i,t}$ and the ground-truth $\btheta^{{j(i)}}$ is less than $\frac{\gamma_1}{4}$ with high probability. With this, we can get: for any two users $i$ and $\ell$, if their gap is greater than $\zeta$, it will trigger the deletion of the edge $(i,\ell)$ (Line \ref{alg:RCLUMB:delete} of Algo.\ref{alg:RCLUMB}) with high probability; on the other hand, when the deletion condition of the edge $(i,\ell)$ is satisfied, then \begin{small}
 $\norm{\btheta^{j(i)}-\btheta^{j(\ell)}}_2>0$   
\end{small}, which means user $i$ and $\ell$ belong to different \gtclusters{} by Assumption \ref{assumption1} with high probability. Therefore, we can get that with high probability, all those users in the same \gtcluster{} as $i_t$ will be directly connected with $i_t$, and users directly connected with $i_t$ must be $\zeta$-close to $i_t$. By filtering users directly linked with $i_t$ as the cluster $\overline{V}_t$ (Algo.\ref{alg:RCLUMB} Line \ref{alg:RCLUMB:find V}) and the definition of ``good partition", we can ensure that RCLUMB will keep a ``good partition" afterward with high probability.

\noindent\textbf{2. Bounding the regret after getting a ``good partition".} After $T_0$, with the ``good partition", we can prove the following lemma that gives a bound of the difference between $\hat{\btheta}_{\overline{V}_t,t-1}$ and ground-truth $\btheta_{i_t}$
in direction of action vector $\bx_a$, and supports the design of the confidence radius $C_{a,t}$ in Eq.(\ref{UCB}).

% The following lemma bounds $\left|\bx_a^{\top}(\btheta_{i_t}-\hat{\btheta}_{\overline{V}_t,t-1})\right|$
% at $t\geq T_0$ for RCLUMB, which supports the analysis of (i) (ii) and the design of the confidence radius $C_{a,t}$ in Eq. (\ref{UCB}).

\begin{lemma}
\label{concentration bound}
With probability at least $1-5\delta$ for some $\delta\in(0,\frac{1}{5})$, $\forall{t\geq T_0}$
\begin{small}
\begin{equation*}
        \left|\bx_a^{\top}(\btheta_{i_t}-\hat{\btheta}_{\overline{V}_t,t-1})\right|\leq \frac{\epsilon_*\sqrt{2d}}{\tilde{\lambda}_x^{\frac{3}{2}}}\mathbb{I}\{\overline{V}_t\notin \mathcal{V}\}+\epsilon_*\textstyle{\sum_{s\in[t-1]\atop i_s\in \overline{V}_t}}\left|\bx_a^{\top}\overline{\bM}_{\overline{V}_t,t-1}^{-1}\bx_{a_s}\right|+\beta\norm{\bx_a}_{\overline{\bM}_{\overline{V}_t,t-1}^{-1}}\,.
        % &\triangleq \frac{8\epsilon_*\sqrt{2d}}{\lambda_x^{\frac{3}{2}}}\mathbb{I}\{\overline{V}_t\notin V\}+C_{a,t}\,,
\end{equation*}
\end{small}
% where $\beta=\sqrt{\lambda}+\sqrt{2\log(\frac{1}{\delta})+d\log(1+\frac{T}{\lambda d})}$, $\delta\in(0,\frac{1}{5})$.
\end{lemma}

To prove this lemma, we consider the following two situations.
% at round $t$, we bound the instantaneous regret $R_t$ in the following two situations:\\

\noindent\textbf{(i) Assigning a perfect cluster for $i_t$.}  In this case, $\overline{V}_t\in \mathcal{V}$, meaning the cluster assigned for user $i_t$ is the same as her \gtcluster{}, i.e., $\overline{V}_t=V_{j(i_t)}$. Therefore, we have that $\forall{\ell\in\overline{V}_t}, \btheta_{\ell}=\btheta_{i_t}$. With careful analysis, we can bound \begin{footnotesize}
$\left|\bx_a^{\top}(\btheta_{i_t}-\hat{\btheta}_{\overline{V}_t,t-1})\right|$    
\end{footnotesize} by $C_{a,t}$ (defined in Eq.(\ref{UCB})).
% $R_t$ can be bounded by $2C_{a_t,t}+2\epsilon_*$, where $C_{a_t,t}$ is the confidence radius defined in Eq. (\ref{UCB}). 

\noindent\textbf{(ii) Bounding the term of \textit{misclustering} $i_t$'s $\zeta$-close users.} In this case, $\overline{V}_t\notin \mathcal{V}$, meaning the algorithm \textit{misclusters} user $i_t$, i.e., $\overline{V}_t\neq V_{j(i_t)}$. Thus, we do not have $\forall{\ell\in\overline{V}_t}, \btheta_{\ell}=\btheta_{i_t}$ anymore, but we have all the users in $\overline{V}_t$ are $\zeta$-close to $i_t$ (by ``good partition"), i.e., $\norm{\btheta_{i_s}-\btheta_{i_t}}_2\leq\zeta\,,\forall{\ell\in \overline{V}_t}$. Then an additional term can be caused by using the information of $i_t$'s $\zeta$-close users in $\overline{V}_t$ lying in different \gtclusters{} from $i_t$ to estimate $\btheta_{i_t}$. It is highly challenging to bound this part. 

We will get an extra term \begin{footnotesize}
    $\left|\bx_a^{\top}\overline{\bM}_{\overline{V}_t,t-1}^{-1}\sum_{s\in[t-1]\atop i_s\in \overline{V}_t}\bx_{a_s}\bx_{a_s}^{\top}(\btheta_{i_s}-\btheta_{i_t})\right|$
\end{footnotesize}
when bounding the regret in this case,
where $\norm{\btheta_{\ell}-\btheta_{i_t}}_2\leq\zeta\,,\forall{\ell\in \overline{V}_t}$. It is an easy-to-be-made mistake to directly drag $\norm{\btheta_{i_s}-\btheta_{i_t}}_2$ out to bound it by \begin{footnotesize}
$\norm{\bx_a^{\top}\overline{\bM}_{\overline{V}_t,t-1}^{-1}\sum_{s\in[t-1]\atop i_s\in \overline{V}_t}\bx_{a_s}\bx_{a_s}^{\top}}_2\times\zeta$    
\end{footnotesize}. With subtle analysis, we propose the following lemma to bound the above term.

\begin{lemma}[Bound of error caused by \textit{misclustering}]
\label{bound for mis}
$\forall{t\geq T_0}$, if the current partition by RCLUMB is a ``good partition", and $\overline{V}_t\notin \mathcal{V}$, then for all $\bx_a\in \RR^d, \norm{\bx_a}_2\leq 1$, with probability at least $1-\delta$:
\begin{equation}
    \textstyle{\left|\bx_a^{\top}\overline{\bM}_{\overline{V}_t,t-1}^{-1}\sum_{s\in[t-1]\atop i_s\in \overline{V}_t}\bx_{a_s}\bx_{a_s}^{\top}(\btheta_{i_s}-\btheta_{i_t})\right|\leq \frac{\epsilon_*\sqrt{2d}}{\tilde{\lambda}_x^{\frac{3}{2}}}\nonumber\,.}
\end{equation}
\end{lemma}

This lemma is quite general. Please see Appendix \ref{key lemma appendix} for details about its proof.

The expected occurrences of $\{\overline{V}_t\notin\mathcal{V}\}$ is bounded by $\frac{\tilde{u}}{u}T$ with Assumption \ref{assumption2}, Definition \ref{def:user number} and \ref{def:good partition}. The result follows by bounding the expected sum of the bounds for the instantaneous regret using Lemma \ref{concentration bound} with delicate analysis due to the time-varying clustering structure kept by RCLUMB.\end{proof}

\section{Experiments}
\label{section:experiments}

This section compares RCLUMB and RSCLUMB with CLUB~\cite{gentile2014online}, SCLUB~\cite{10.5555/3367243.3367445}, LinUCB with a single estimated vector for all users, LinUCB-Ind with separate estimated vectors for each user, and two modifications of LinUCB in ~\cite{lattimore2020learning} which we name as RLinUCB and RLinUCB-Ind. We use averaged reward as the evaluation metric, where the average is taken over ten independent trials. 
\begin{figure}[htbp]

\centering
\resizebox{0.916\columnwidth}{!}{
\begin{minipage}{\columnwidth}
\centering
    \subfigure[Synthetic]
    {\includegraphics[scale=0.15]{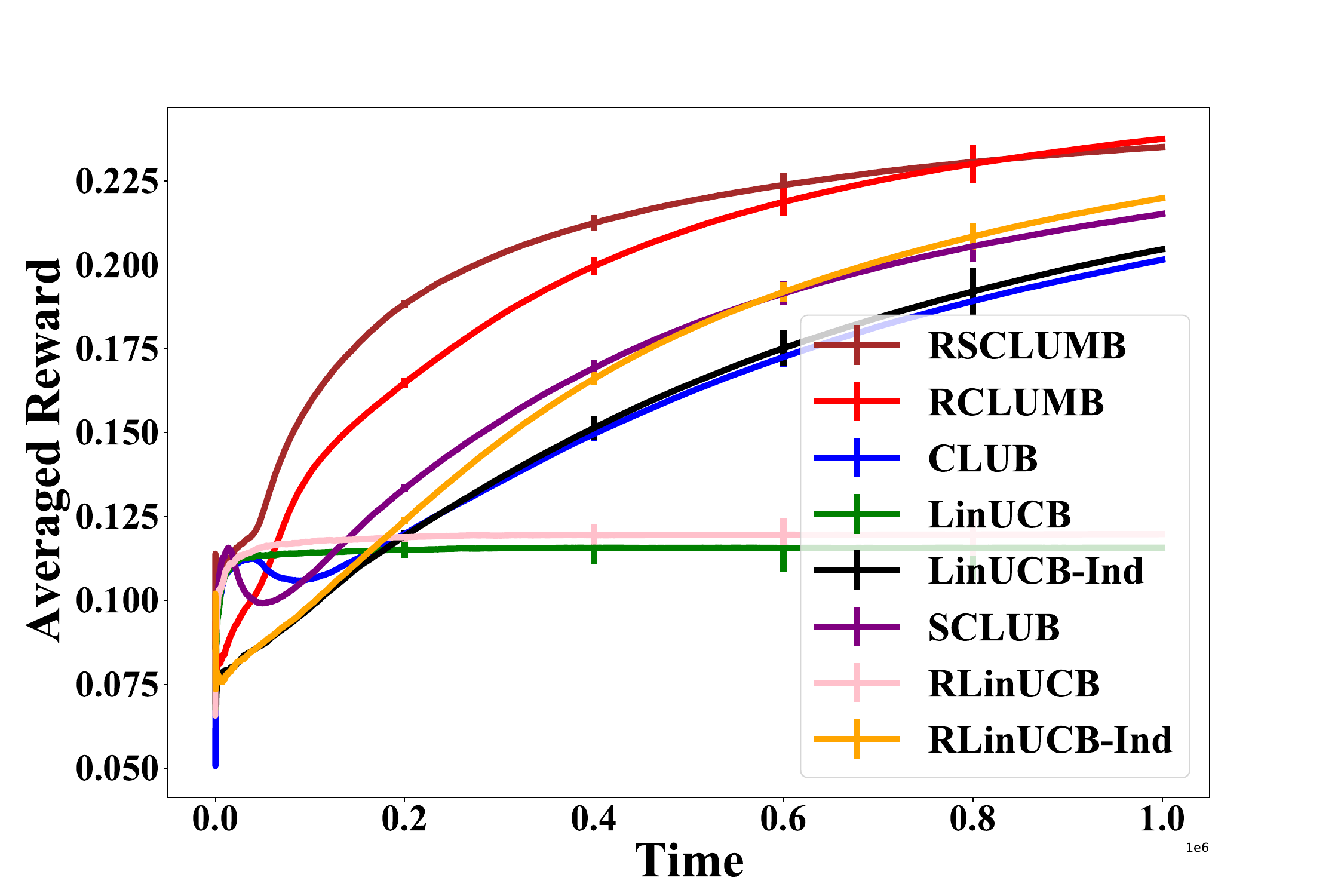}
    }
    \subfigure[Yelp Case 1]{
\includegraphics[scale=0.15]{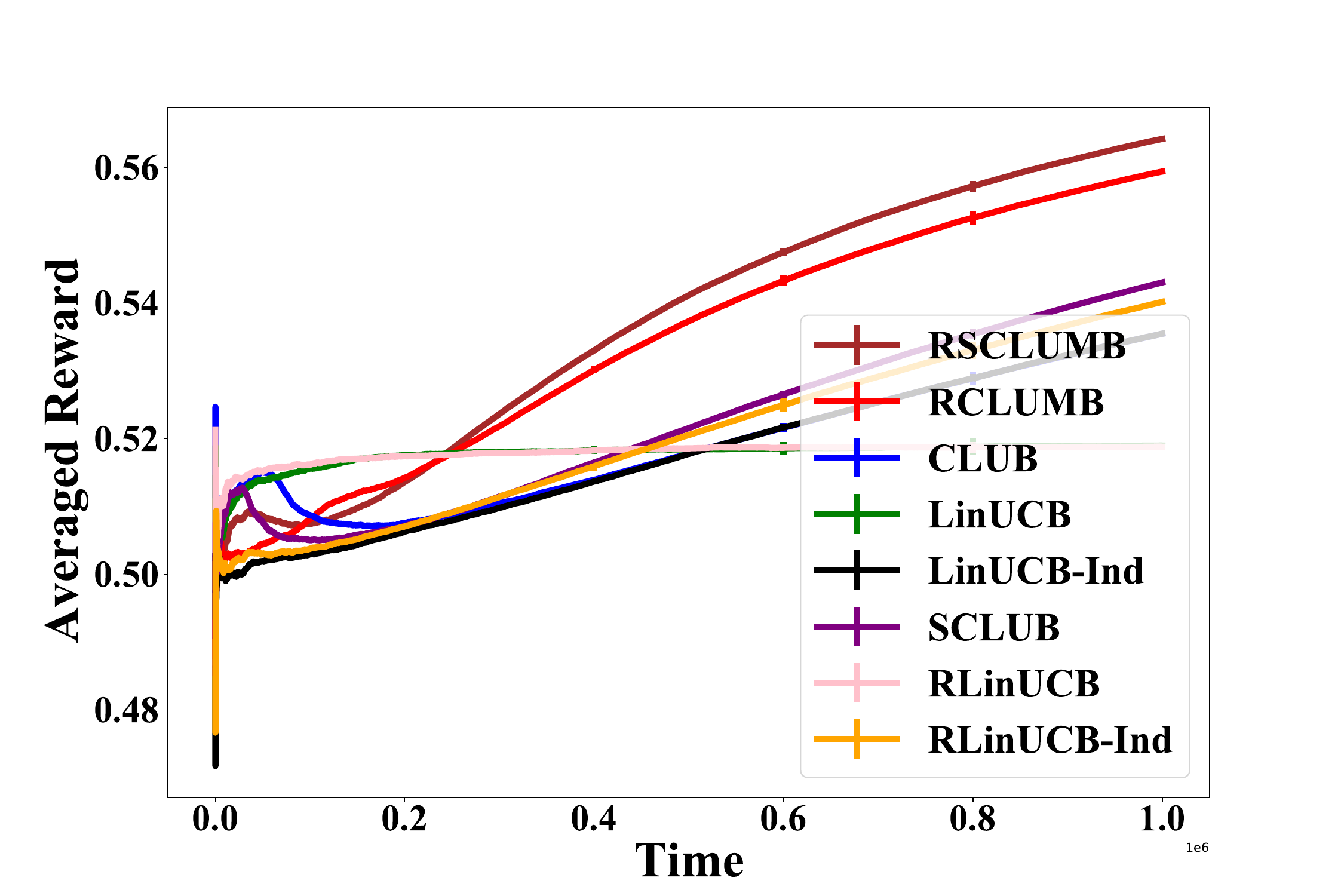}
    }
    \subfigure[Yelp Case 2]{
\includegraphics[scale=0.2]{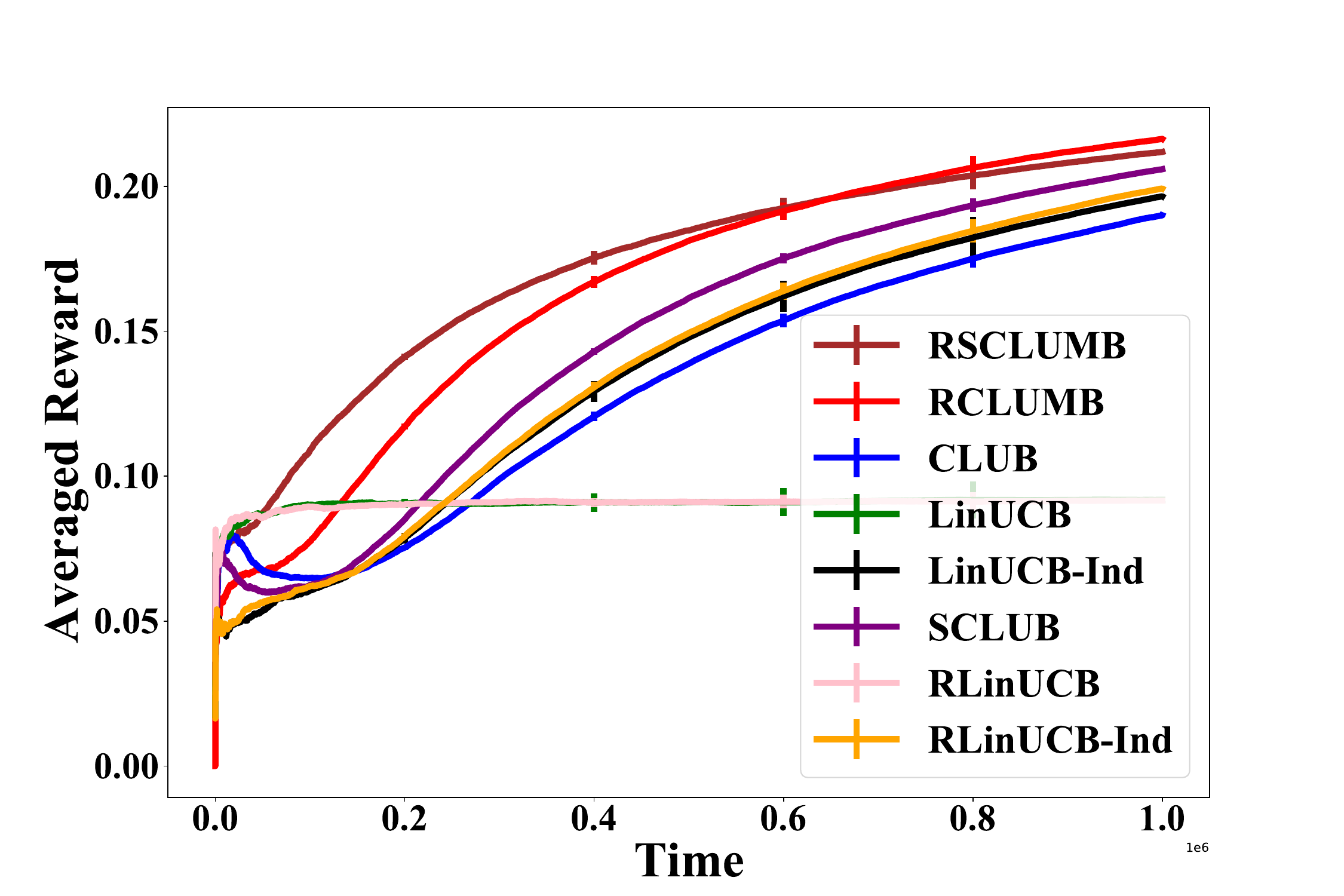}
    }
    \subfigure[Movielens Case 1]{
\includegraphics[scale=0.15]{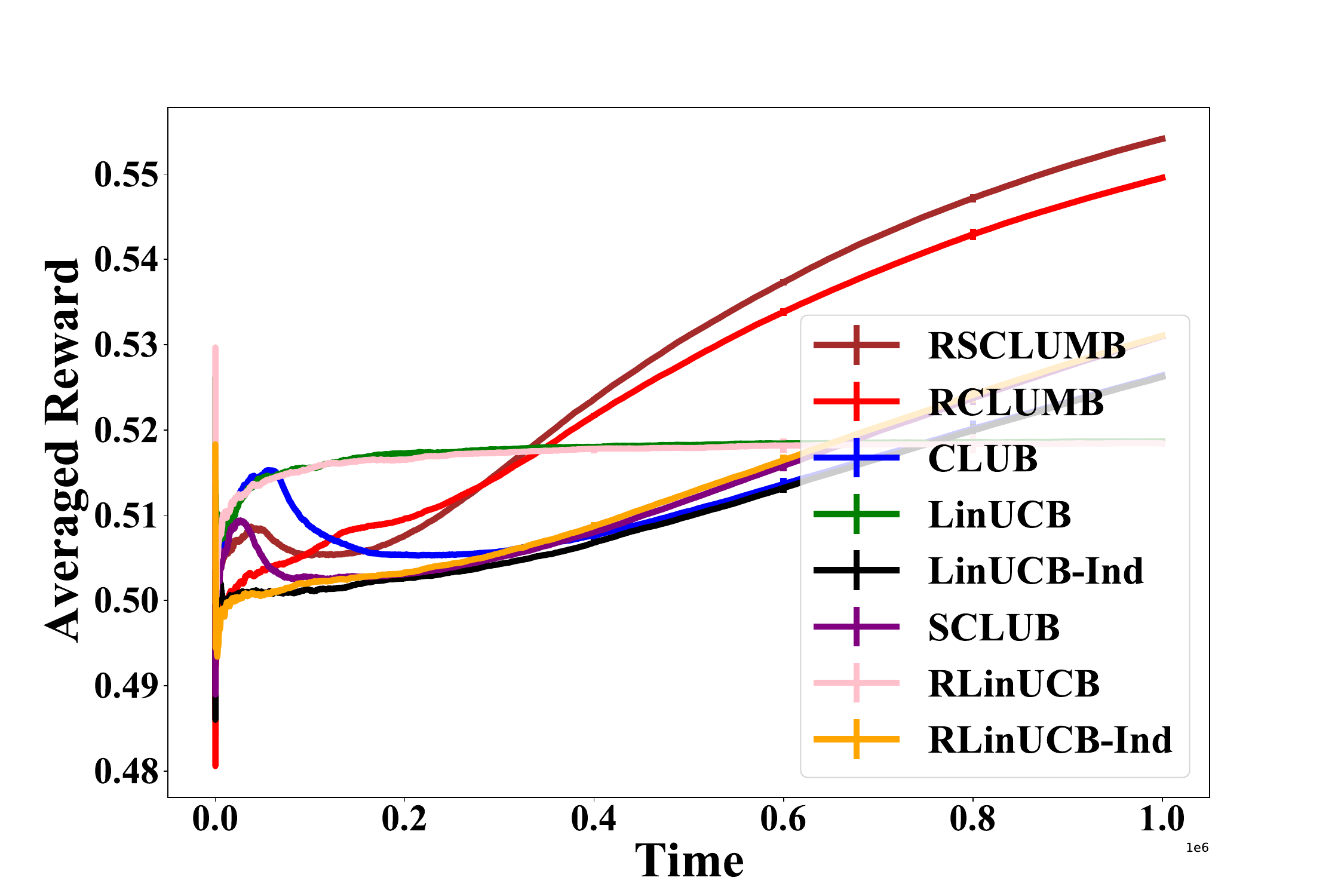}
    }
    \subfigure[Movielens Case 2]{
\includegraphics[scale=0.15]{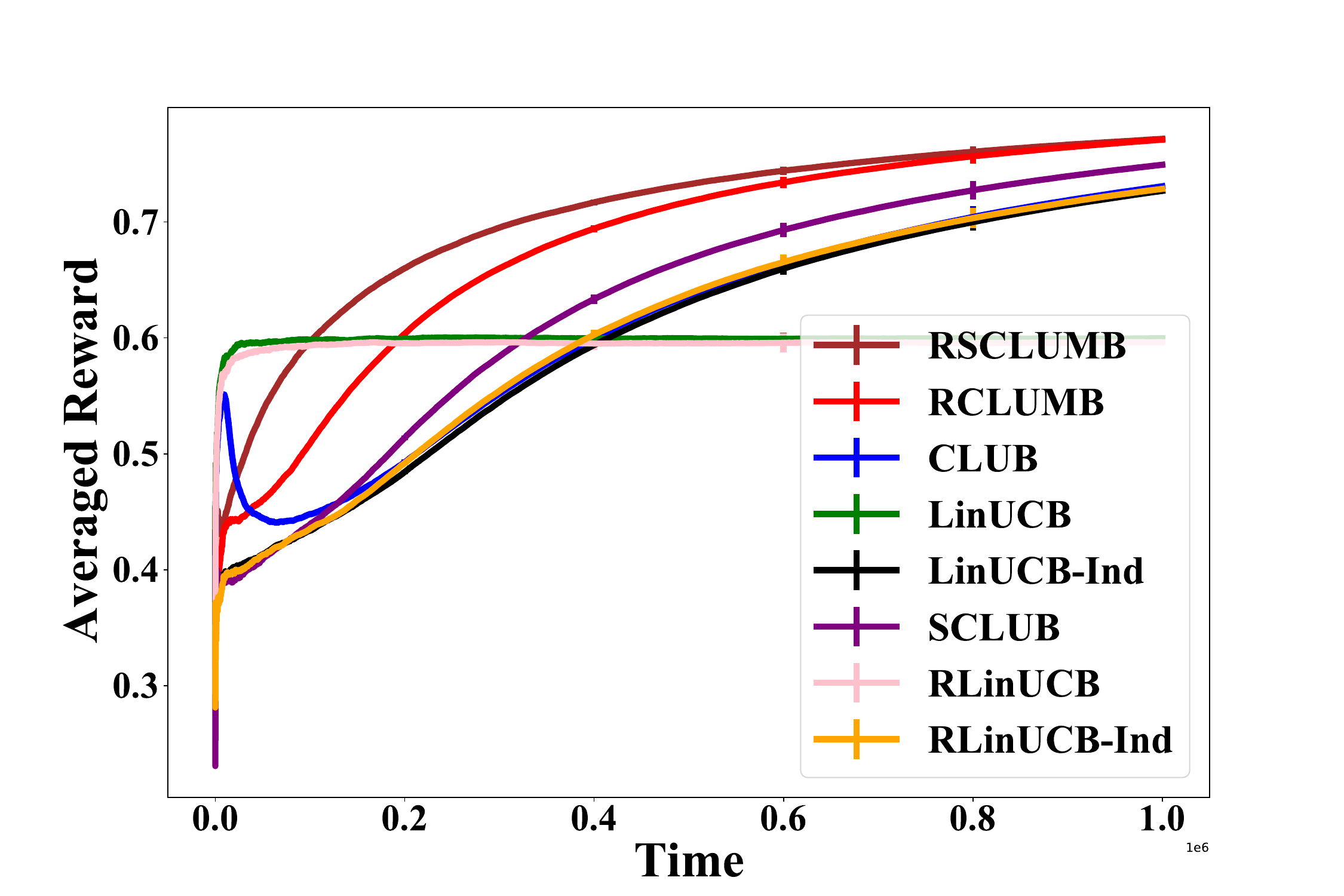}
    }
    \caption{The figures compare RCLUMB and RSCLUMB with the baselines. (a) shows the result on synthetic data, (b) and (c) show the results on Yelp dataset, (d) and (e) show the results on Movielens dataset. All experiments are under the setting of $u = 1,000$ users, $m =10$ clusters, and $d=50$. All results are averaged under $10$ random trials. The error bars are standard deviations divided by $\sqrt{10}$. }
\label{fig:my_label}\end{minipage}}
\end{figure}

\subsection{Synthetic Experiments}

We consider a setting with $u = 1,000$ users, $m = 10$ clusters and $T=10^6$ rounds. The preference and feature vectors are in $d=50$ dimension with each entry drawn from a standard Gaussian distribution, and are normalized to vectors with $\norm{.}_2=1$~\cite{10.5555/3367243.3367445}. We fix an arm set with $
\left|\mathcal{A}\right|= 1000$ items, at each round $t$, 20 items are randomly selected to form a set $\mathcal{A}_{t}$ for the user to choose from. We construct a matrix $\boldsymbol{\epsilon} \in \mathbb{R}^{1,000\times1,000}$ in which each element $\boldsymbol{\epsilon}(i,j)$ is drawn uniformly from the range $(-0.2, 0.2)$ to represent the deviation. At $t$, for user $i_{t}$ and the item $a_{t}$, $\boldsymbol{\epsilon}(i_{t},a_{t})$ will be added to the feedback as the deviation, which corresponds to the $\bepsilon^{i_t, t}_{a_{t}}$ defined in Eq.(\ref{reward def}).

The result is provided in Figure \ref{fig:my_label}(a), showing that our algorithms have clear advantages: RCLUMB improves over CLUB by 21.9\%, LinUCB by 194.8\%,  LinUCB-Ind by 20.1\%, SCLUB by 12.0\%, RLinUCB by 185.2\% and RLinUCB-Ind by 10.6\%. The performance difference between RCLUMB and RSCLUMB is very small as expected. 
RLinUCB performs better than LinUCB; RLinUCB-Ind performs better than LinUCB-Ind and CLUB, showing that the modification of the recommendation policy is effective. The set-based RSCLUMB and SCLUB can separate clusters quicker and have advantages in the early period, but eventually RCLUMB catches up with RSCLUMB, and SCLUB is surpassed by RLinUCB-Ind because it does not consider misspecifications. RCLUMB and RSCLUMB perform better than RLinUCB-Ind, which shows the advantage of the clustering. So it can be concluded that both the modification for misspecification and the clustering structure are critical to improving the algorithm's performance.
% \jize{According to the results, we have following observations. (1) CLUB and SCLUB perform poorly because they do not consider how to do robust clustering under model misspecifications, thus would do erroneous clustering and be misled by the estimation inaccuracy caused by misspecifications to make sub-optimal recommendations. (2) the other baseline algorithms do not consider how to infer the underlying user clustering structure to speed up learning. Our two algorithms (RCLUMB and RSCLUMB) all achieve good performance good because: (i) they can robustly cluster users with similar preferences together to leverage the social connections among users for faster learning convergence; (ii) they can avoid being misled by the estimation inaccuracy caused by model misspecifications. These points are supported by the fact that: both with clustering structure, our algorithms perform better than SCLUB and CLUB; with similar recommendation rules, our algorithms perform better than RLinUCB-Ind. 
We also have done some ablation experiments on different scales of $\epsilon^{*}$ in Appendix \ref{appendix: more experiments}
, and we can notice that under different $\epsilon^{*}$
, our algorithms always outperform the baselines, and some baselines will perform worse as $\epsilon^{*}$
 increases.%Here, each user has a different deviation for each item; even the users in the same cluster (with the same vector)  will get different feedback when selecting the same items, so RCLUMB, SCLUB and CLUB will separate the 1000 users into 1000 clusters eventually, and the speed of this process will influence the performance, so SCLUB performs better than CLUB. However, the clustering structure can help to improve the speed of information accumulation, so at the beginning, the averaged reward of SCLUB and CLUB increases most fast; this also explains why LinUCB-Ind performs worse than CLUB and RLinUCB-Ind performs worse than RCLUMB. During the process of deleting edges, users will lose some information borrowed, so the averaged reward of SCLUB and CLUB will decrease until all edges are deleted, while RCLUMB borrows less information from its cluster can avoid such a decrease. Faced with such known misspecification level, RLinUCB and RLinUCB-Ind perform better than LinUCB and LinUCB-Ind as expected.

\subsection{Experiments on Real-world Datasets}

We conduct experiments on the Yelp data and the $20m$ MovieLens data \cite{harper2015movielens}. For both data, we have two cases due to the different methods for generating feedback. For case 1, we extract 1,000 items with most ratings and 1,000 users who rate most; then we construct a binary
matrix $\boldsymbol{H}^{1,000\times1,000}$ based on the user rating \cite{wu2021clustering, zong2016cascading}: if the user rating is greater than 3, the feedback is 1; otherwise, the feedback is 0. Then we use this binary matrix to generate the preference and feature vectors by singular-value decomposition (SVD) \cite{10.5555/3367243.3367445,li2018online, wu2021clustering}. Similar to the synthetic experiment, we construct a matrix $\boldsymbol{\epsilon} \in \mathbb{R}^{1,000\times1,000}$ in which each element is drawn uniformly from the range $(-0.2, 0.2)$. For case 2, we extract 1,100 users who rate most and 1000 items with most ratings. We construct a binary feedback matrix $\boldsymbol{H}^{1,100\times1,000}$ based on the same rule as case 1. Then we select the first 100 rows $\boldsymbol{H}^{100\times1,000}_{1}$ to generate the feature vectors by SVD. The remaining 1,000 rows $\boldsymbol{F}^{1,000\times1,000}$ is used as the feedback matrix, meaning user $i$ receives $\boldsymbol{F}(i,j)$ as feedback while choosing item $j$. In both cases, at time $t$, we randomly select $20$ items for the algorithms to choose from. In case 1, the feedback is computed by the preference and feature vector with misspecification, in case 2, the feedback is from the feedback matrix. %In addition, to keep the parameters the same as in other experiments, we also set the $\epsilon_{*}$ as $0.2$.

% In both cases, at time $t$, we randomly select $20$ items for the user to choose from. In case 1, the feedback is computed by the preference and feature vector with misspecification, in case 2, the feedback is from the feedback matrix. 
The results on Yelp are shown in Fig \ref{fig:my_label}(b) and Fig \ref{fig:my_label}(c). In case 1, RCLUMB improves CLUB by 45.1\%, SCLUB by 53.4\%, LinUCB-One by 170.1\% , LinUCB-Ind by 46.2\%, RLinUCB by 171.0\% and RLinUCB-Ind by 21.5\%. In case 2, RCLUMB improves over CLUB by 13.9\%, SCLUB by 5.1\%, LinUCB-One by 135.6\% , LinUCB-Ind by 10.1\%, RLinUCB by 138.6\% and RLinUCB by 8.5\%. 
% Besides, RCLUMB's advantage is much larger in case 1 than in case 2, owing to the larger and unknown misspecification in case 2. 
It is notable that our modeling assumption \ref{assumption4} is violated in case 2 since the misspecification range is unknown. We set $\epsilon_*=0.2$ following our synthetic dataset and it can still perform better than other algorithms. When the misspecification level is known as in case 1, our algorithms' improvement is significantly enlarged, e.g., RCLUMB improves over SCLUB from 5.1\% to 53.4\%. %But RLinUCB and RLinUCB-Ind (with the same assumption as ours) don't have advantages compared to LinUCB and LinUCB-Ind in case 2; maybe RCLUMB still benefits from the clustering structure as mentioned in the synthetic experiment. 
% It is notable that, though our algorithm requires a known range for the misspecification, it still performs better than existing algorithms when faced with the unknown misspecification.
% \xutong{In this case, how do we choose $\epsilon_*$? Maybe we need to explain why our algorithm is better even without knowing the explicit misspecification.}\jize{not sure how to explain the second question: why we perform better without knowing the explicit misspecification}

%In both cases, at time $t$, we randomly select $20$ items for the user to choose from. In case 1, the feedback is computed by the preference vector and feature vector with misspecification, while in case 2, the feedback is a binary value from the feedback matrix.
The results on Movielens are shown in Fig \ref{fig:my_label}(d) and \ref{fig:my_label}(e). In case 1, RCLUMB improves CLUB by 58.8\%, SCLUB by 92.1\%, LinUCB-One by 107.7\%, LinUCB-Ind by 61.5 \%, RLinUCB by 109.5\%, and RLinUCB-Ind by 21.3\%. In case 2, RCLUMB improves over CLUB by 5.5\%, SCLUB by 2.9\%, LinUCB-One by 28.5\%, LinUCB-Ind by 6.1\%, RLinUCB by 29.3\% and RLinUCB-Ind by 5.8\%. The results are consistent with the Yelp data, confirming our superior performance.

\chapter{Online Corrupted User Detection and Regret Minimization}\label{chapter: detect}

In real-world online web systems, multiple users usually arrive sequentially into the system. For applications like click fraud and fake reviews, some users can maliciously perform corrupted (disrupted) behaviors to trick the system. Therefore, it is crucial to design efficient online learning algorithms to robustly learn from potentially corrupted user behaviors and accurately identify the corrupted users in an online manner. 
Existing works propose bandit algorithms robust to adversarial corruption. However, these algorithms are designed for a single user, and cannot leverage the implicit social relations among multiple users for more efficient learning. Moreover, none of them consider how to detect corrupted users online in the multiple-user scenario. In this paper, we present an important online learning problem named LOCUD to learn and utilize unknown user relations from disrupted behaviors to speed up learning, and identify the corrupted users in an online setting. To robustly learn and utilize the unknown relations among potentially corrupted users, we propose a novel bandit algorithm RCLUB-WCU. To detect the corrupted users, we devise a novel online detection algorithm OCCUD based on RCLUB-WCU's inferred user relations. We prove a regret upper bound for RCLUB-WCU, which asymptotically matches the lower bound with respect to $T$ up to logarithmic factors, and matches the state-of-the-art results in degenerate cases. We also give a theoretical guarantee for the detection accuracy of OCCUD. With extensive experiments, our methods achieve superior performance over previous bandit algorithms and high corrupted user detection accuracy. This chapter is based on our publication \cite{wang2024onlineb}.

\section{Introduction}

In real-world online recommender systems, data from many users arrive in a streaming fashion \cite{chu2011contextual,kohli2013fast,aggarwal2016recommender,garcelon2020adversarial,wang2023efficient,liu2023contextual,liu2022batch}. There may exist some corrupted (malicious) users, whose behaviors (\emph{e.g.}, click, rating) can be adversarially corrupted (disrupted) over time to fool the system \cite{lykouris2018stochastic,ma2018data,he2022nearly,hajiesmaili2020adversarial,gupta2019better}. 
These corrupted behaviors could disrupt the user preference estimations of the algorithm. As a result, the system would easily be misled and make sub-optimal recommendations \cite{jun2018adversarial,liu2019data,garcelon2020adversarial,zuo2023adversarial}, which would hurt the user experience.
Therefore, it is essential to design efficient online learning algorithms to robustly learn from potentially disrupted behaviors and detect corrupted users in an online manner.

There exist some works on bandits with adversarial corruption \cite{lykouris2018stochastic,gupta2019better,li2019stochastic,ding2022robust,he2022nearly,kong2023best}. However, they have the following limitations. First, existing algorithms are initially designed 
for 
robust online preference learning of a single user. In real-world 
scenarios with multiple users, they cannot robustly infer and utilize the implicit user relations for more efficient learning. Second, none of them consider how to identify corrupted users online in the multiple-user scenario. Though there also exist some works on corrupted user detection \cite{wang2019semi,dou2020enhancing,zhang2021fraudre,liu2021pick,huang2022auc}, they all focus on detection with \textit{known} user information in an offline setting, thus 
can not be applied to do online detection from bandit feedback.

To address these limitations, we propose a novel bandit problem ``\textit{Learning and Online Corrupted Users Detection from bandit feedback}" (LOCUD). To model and utilize the relations among users, we assume there is an \textit{unknown} clustering structure over users, where users with similar preferences lie in the same cluster \cite{gentile2014online,li2018online,li2019improved}. The agent can infer the clustering structure to leverage the information of similar users for better recommendations. Among these users, there exists a small fraction of corrupted users. They can occasionally perform corrupted behaviors to fool the agent \cite{he2022nearly,lykouris2018stochastic,ma2018data,gupta2019better} while mimicking the behaviors of normal users most of the time to make themselves hard to discover. 
% In this new problem, there is an \textit{unknown} clustering structure over users with \textit{unknown} preferences, and some of the users can occasionally perform corrupted behaviors to fool the agent. 
The agent not only needs to learn the \textit{unknown} user preferences and relations robustly from potentially disrupted feedback, balance the exploration-exploitation trade-off to maximize the cumulative reward, but also needs to detect the corrupted users online from bandit feedback. 
%  \begin{figure}[htp]
%    \centering
%    \includegraphics[width=0.6888\linewidth]{use case new.pdf}
%     \caption{Illustration of LOCUD-BFAC. The \textit{unknown} user relations are represented by dotted circles, \emph{e.g.}, user 3, 7 have similar preferences and thus can be in the same user segment (\emph{i.e.}, cluster). Users 6 and 8 are corrupted users, and their behaviors are dynamic over time \big(\emph{e.g.}, for user 8, the behaviors at $t_1$ and $t_3$ are normal (blue), but at $t_2$ and $t_4$ are corrupted (red)\cite{lykouris2018stochastic,he2022nearly}\big), making them hard to be detected online. The agent needs to learn user relations to utilize information among similar users to improve recommendation qualities, and detect corrupted users 6, 8 online with bandit feedback.}
%     \label{fig: use case}
% \end{figure}

The LOCUD problem is very challenging. 
First, the corrupted behaviors would cause inaccurate user preference estimations, which could lead to erroneous user relation inference and sub-optimal recommendations. 
Second, it is nontrivial to detect corrupted users online
since their behaviors are dynamic over time (sometimes regular while sometimes corrupted), whereas, in the offline setting, corrupted users' information can be fully represented by static embeddings and the existing approaches \cite{li2022dual,qin2022explainable} can typically do binary classifications offline, which are not adaptive over time.

We propose a novel learning framework composed of two algorithms to address these challenges. 

\noindent{\textbf{RCLUB-WCU.}} 
% \tong{Repeat 'preference estimations,' if the following is to address the first challenge mentioned above.}
To robustly estimate user preferences, learn the unknown relations from potentially corrupted behaviors, and perform high-quality recommendations,
we propose a novel bandit algorithm ``\textit{Robust CLUstering of Bandits With Corrupted Users}" (RCLUB-WCU), which maintains a dynamic graph over users to represent the learned clustering structure, where users linked by edges are inferred to be in the same cluster. RCLUB-WCU adaptively deletes edges and recommends arms based on aggregated interactive information in clusters. 
 We do the following to ensure robust clustering structure learning. (i) To relieve the estimation inaccuracy caused by disrupted behaviors, we use weighted ridge regressions for robust user preference estimations.
 % robustly estimate the \textit{unknown} user preference vectors. 
 Specifically, we use the inverse of the confidence radius to weigh each sample. If the confidence radius associated with user $i_t$ and arm $a_t$ is large at $t$, the learner is quite uncertain about the estimation of $i_t$'s preference on $a_t$, indicating the sample at $t$ is likely to be corrupted. 
 % Therefore, to avoid being misled by corrupted users, we use the inverse of confidence radius to assign minor importance to the possibly corrupted samples when doing estimations. 
Therefore, we use the inverse of the confidence radius to assign minor importance to the possibly disrupted samples when doing estimations. 
% For the common preference estimations of a cluster, we also weight each sample associated with the cluster by the inverse of the confidence radius for the corresponding served user. 
(ii) We design a robust edge deletion rule to divide the clusters by considering the potential effect of corruptions, which, together with (i), can ensure that after some interactions, users in the same connected component of the graph are in the same underlying cluster with high probability.

\noindent{\textbf{OCCUD.}} To detect corrupted users online, based on the learned clustering structure of RCLUB-WCU, we devise a novel algorithm named ``\textit{Online Cluster-based Corrupted User Detection}" (OCCUD). At each round,
% \tong{is round the same as time? Please make it consistent.} 
we compare each user's non-robustly estimated preference vector (by ridge regression) and the robust estimation (by weighted regression) of the user's inferred cluster. If the gap exceeds a carefully-designed threshold, we detect this user as corrupted. The intuitions are as follows. With misleading behaviors, the non-robust preference estimations of corrupted users would be far from ground truths. On the other hand, with the accurate clustering of RCLUB-WCU, the robust estimations of users' inferred clusters should be close to ground truths. Therefore, for corrupted users, their non-robust estimates should be far from the robust estimates of their inferred clusters.

% We prove a sub-linear regret upper bound for RCLUB-WCU, which matches the lower bound asymptotically in $T$ up to logarithmic factors, showing the tightness of our theoretical result. We also give a theoretical performance guarantee for the detection algorithm OCCUD. Experiments on synthetic and real-world data show that our proposed algorithms achieve 
% superior performance over previous bandit algorithms and high corrupted user detection accuracy.

% To robustly infer and leverage the \textit{unknown} clustering structure among users with corrupted behaviors, we propose a novel bandit algorithm called ``Robust CLUstering of Bandits With Corrupted Users" (RCLUB-WCU), which can robustly cluster users with similar preferences adaptively during the learning process, and use the cluster-based interactive feedback to improve the recommendation quality. To detect the corrupted users in the online bandit setting, based on the learned user clustering structure of RCLUB-WCU, we propose a novel algorithm named ``Online Cluster-based Corrupted User Detection" (OCCUD), which detects corrupted users by adaptively comparing the non-robustly estimated preferences of the users and the robustly estimated preferences of their inferred clusters.
% We prove a sub-linear regret upper bound for RCLUB-WCU and theoretically 
% analyze the corrupted user detection accuracy of OCCUD. Extensive experiments
% validate our theoretical
% analysis.
% \subsection{Our Contributions}
We summarize our contributions as follows.\\
    % We are the first to study how to simultaneously do online learning from corrupted user behaviors and detect corrupted users in an online setting. 
% $\bullet$ Given multiple users where some unknown malicious users' feedback are adversarially corrupted over time in an online setting, we propose the first work to explore how to online detect corrupted users and accurately learn models from potentially disrupted user behaviors. 
    $\bullet$ We present a novel online learning problem LOCUD, where the agent needs to (i) robustly learn and leverage the unknown user relations to improve online recommendation qualities under the disruption of corrupted user behaviors; 
    (ii) detect the corrupted users online 
 from bandit feedback. \\
    $\bullet$ We propose a novel online learning framework composed of two algorithms, RCLUB-WCU and OCCUD, to tackle the challenging LOCUD problem. RCLUB-WCU robustly learns and utilizes the unknown social relations among potentially corrupted users to efficiently minimize regret. Based on RCLUB-WCU’s inferred user relations, OCCUD accurately detects corrupted users online. \\
    $\bullet$ We prove a regret upper bound for RCLUB-WCU, which matches the lower bound asymptotically in $T$ up to logarithmic factors and matches the state-of-the-art results in several degenerate
cases. We also give a theoretical performance guarantee for the online detection algorithm OCCUD.\\
    $\bullet$ Experiments on both synthetic and real-world data clearly
show the advantages of our methods.

\section{Problem Setup}
\label{setup}

% \begin{wrapfigure}{R}[0.5cm]{0.5\textwidth}
% \vspace{-0.56cm}
%    \centering
%    \includegraphics[width=0.96\linewidth]{use case new.pdf}
%     \caption{Illustration of LOCUD. The \textit{unknown} user relations are represented by dotted circles, \emph{e.g.}, user 3, 7 have similar preferences and thus can be in the same user segment (\emph{i.e.}, cluster). Users 6 and 8 are corrupted users with dynamic behaviors over time \big(\emph{e.g.}, for user 8, the behaviors are normal at $t_1$ and $t_3$ (blue), but are adversarially corrupted at $t_2$ and $t_4$ (red)\cite{lykouris2018stochastic,he2022nearly}\big), making them hard to be detected online. The agent needs to learn user relations to utilize information among similar users to speed up learning, and detect corrupted users 6, 8 online from bandit feedback.}
%     \label{fig: use case}
% \end{wrapfigure}

\begin{wrapfigure}{R}{7.6cm}

   \centering
\includegraphics[width=7.6cm]{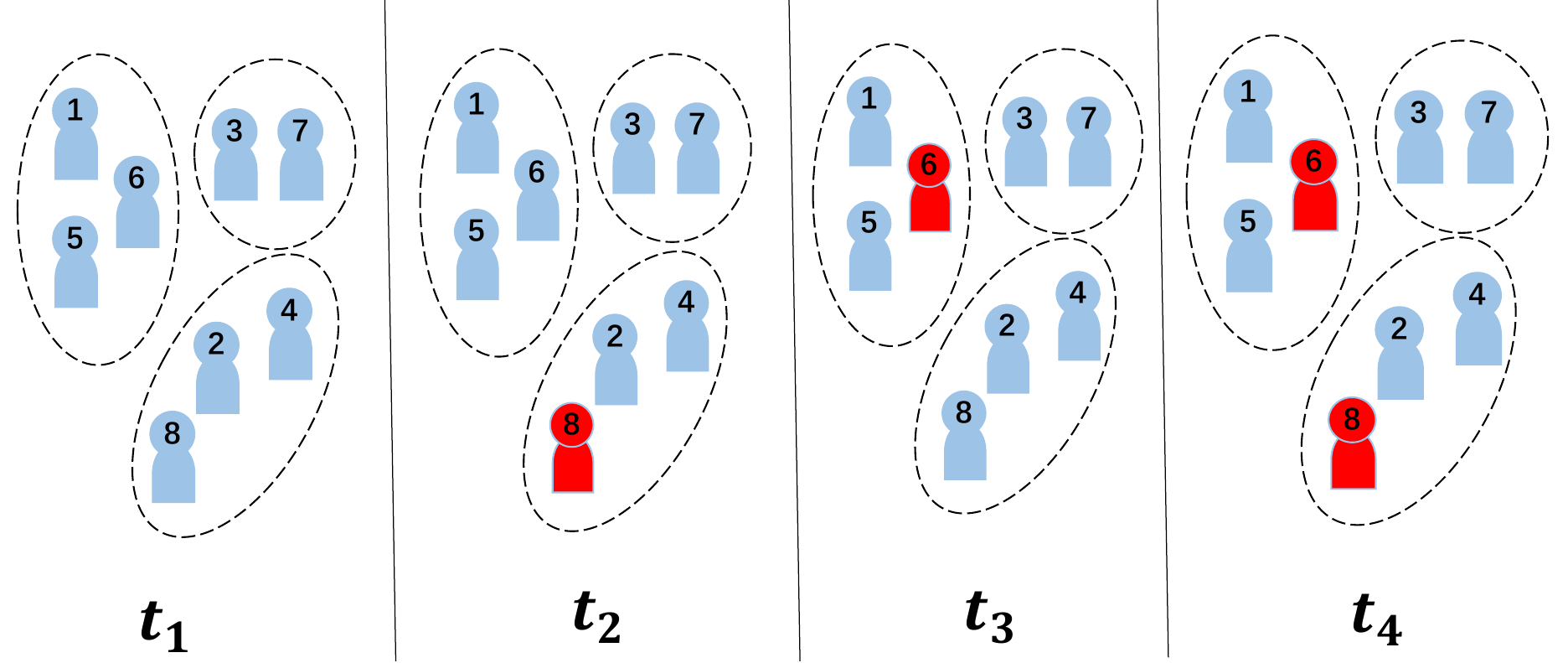}
\caption{Illustration of LOCUD. The \textit{unknown} user relations are represented by dotted circles, \emph{e.g.}, user 3, 7 have similar preferences and thus can be in the same user segment (\emph{i.e.}, cluster). Users 6 and 8 are corrupted users with dynamic behaviors over time (\emph{e.g.}, for user 8, the behaviors are normal at $t_1$ and $t_3$ (blue), but are adversarially corrupted at $t_2$ and $t_4$ (red)\cite{lykouris2018stochastic,he2022nearly}), making them hard to be detected online. The agent needs to learn user relations to utilize information among similar users to speed up learning, and detect corrupted users 6, 8 online from bandit feedback.}
   \label{fig: use case}
  
\end{wrapfigure}

This section formulates the problem of “\textit{Learning and Online Corrupted Users Detection from bandit feedback}” (LOCUD) (illustrated in Fig.\ref{fig: use case}).  
% We use boldface \textbf{lowercase} and boldface
% \textbf{CAPITALIZED} letters for vectors and matrices, respectively. 
% We use $\norm{\boldsymbol{x}}_{\boldsymbol{M}}=\sqrt{\boldsymbol{x}^{\top}\boldsymbol{M}\boldsymbol{x}}$ to denote the
% matrix norm of vector $\bx$ regarding the positive semi-definite (PSD) matrix $\boldsymbol{M}$, $\left|\mathcal{A}\right|$ to denote the number of elements in set $\mathcal{A}$, and $[m]$ to denote the set $\{1,\ldots,m\}$. 
We denote $\norm{\boldsymbol{x}}_{\boldsymbol{M}}=\sqrt{\boldsymbol{x}^{\top}\boldsymbol{M}\boldsymbol{x}}$, $[m]=\{1,\ldots,m\}$, number of elements in set $\mathcal{A}$ as $\left|\mathcal{A}\right|$.

In LOCUD, there are $u$ users, which we denote by set $\mathcal{U}=\{1,2,\ldots,u\}$. Some of them are corrupted users, denoted by set $\tilde{\mathcal{U}}\subseteq\mathcal{U}$. These corrupted users, on the one hand, try to mimic normal users to make themselves hard to detect; on the other hand, they can occasionally perform corrupted behaviors to fool the agent into making sub-optimal decisions. Each user $i\in \mathcal{U}$, no matter a normal one or corrupted one, is associated with a (possibly mimicked for corrupted users) preference feature vector $\btheta_i\in\RR^d$ that is \textit{unknown} and $\norm{\btheta_i}_2\leq1$. There is an underlying clustering structure among all the users representing the similarity of their preferences, but it is \textit{unknown} to the agent and needs to be learned via interactions. Specifically, the set of users $\mathcal{U}$ can be partitioned into $m$ ($m\ll u$) clusters, $V_1, V_2,\ldots V_m$, where $\cup_{j \in [m]}V_j=\mathcal{U},$ and $V_j \cap V_{j'}=\emptyset,$ for $j\neq j'$. Users in the same cluster have the same preference feature vector, while users in different clusters have different preference vectors. We use $\btheta^j$ to denote the common preference vector shared by users in the $j$-th cluster $V_j$, and use $j(i)$ to denote the index of cluster user $i$ belongs to (\emph{i.e.}, $i\in V_{j(i)}$). For any two users $k,i\in\mathcal{U}$, if $k\in V_{j(i)}$, then $\btheta_k=\btheta^{j(i)}=\btheta_i$; otherwise $\btheta_k\neq\btheta_i$. We assume the arm set $\mathcal{A}\subseteq \RR^d$ is finite. Each arm $a\in\mathcal{A}$ is associated with a feature vector $\bx_a\in\RR^d$ with $\norm{\bx_a}_2\leq1$.

The learning process of the agent is as follows. At each round $t\in[T]$, a user $i_t\in\mathcal{U}$ comes to be served, and the learning agent receives a set of arms $\mathcal{A}_t\subseteq\mathcal{A}$ to choose from. The agent infers the cluster $V_t$ that user $i_t$ belongs to based on the interaction history, and recommends an arm $a_t\in\mathcal{A}_t$ according to the aggregated information gathered in the cluster $V_t$. After receiving the recommended arm $a_t$, a normal user $i_t$ will give a random reward with expectation $\bx_{a_t}^\top\btheta_{i_t}$ to the agent. 

To model the behaviors of corrupted users, following \cite{lykouris2018stochastic,gupta2019better,ding2022robust,he2022nearly}, we assume that they can occasionally corrupt the rewards to mislead the agent into recommending sub-optimal arms. Specifically, at each round $t$, if the current served user is a corrupted user (i.e., $i_t\in\Tilde{\mathcal{U}}$), the user can corrupt the reward by $c_t$. 
In summary, we model the reward received by the agent at round $t$ as
% $$r_t=\bx_{a_t}^\top\btheta_{i_t}+\eta_t+c_t\,,$$
\begin{equation*}
    r_t=\bx_{a_t}^\top\btheta_{i_t}+\eta_t+c_t\,,
\end{equation*}
where $c_t=0$ if $i_t$ is a normal user, (\emph{i.e.}, $i_t\notin \Tilde{\mathcal{U}}$), and $\eta_t$ is 1-sub-Gaussian random noise.

As the number of corrupted users is usually small, and they only corrupt the rewards occasionally with small magnitudes to make themselves hard to detect, we assume the sum of corruption magnitudes in all rounds is upper bounded by the \textit{corruption level} $C$, \emph{i.e.}, $\sum_{t=1}^T \left|c_t\right|\leq C$ \cite{lykouris2018stochastic,gupta2019better,ding2022robust,he2022nearly}. 

We assume the clusters, users, and items satisfy the following assumptions. Note that all these assumptions basically follow the settings from classical works on clustering of bandits \cite{gentile2014online,li2018online,liu2022federated,wang2023online}.

\begin{assumption}[Gap between different clusters]
\label{assumption1}
The gap between any two preference vectors for different clusters is at least an \textit{unknown} positive constant $\gamma$
\begin{equation*}
    \norm{\btheta^{j}-\btheta^{j^{\prime}}}_2\geq \gamma>0\,, \forall{j,j^{\prime}\in [m]\,, j\neq j^{\prime}}\,.
\end{equation*}
\end{assumption}

\begin{assumption}[Uniform arrival of users]
\label{assumption2}
At each round $t$, a user $i_t$ comes uniformly at random from $\mathcal{U}$ with probability $1/u$, independent of the past rounds.
\end{assumption}

\begin{assumption}[Item regularity]
\label{assumption3}
At each round $t$, the feature vector $\bx_a$ of each arm $a\in \mathcal{A}_t$ is drawn independently from a fixed \textit{unknown} distribution $\rho$ over $\{\bx\in\RR^d:\norm{\bx}_2\leq1\}$, where $\EE_{\bx\sim \rho}[\bx \bx^{\top}]$'s minimal eigenvalue $\lambda_x > 0$. At $\forall t$, for any fixed unit vector $\boldsymbol{z} \in \RR^d$, $(\btheta^{\top}\boldsymbol{z})^2$ has sub-Gaussian tail with variance no greater than $\sigma^2$.
\end{assumption}

Let $a_t^*\in{\arg\max}_{a\in{\mathcal{A}_t}}\bx_a^{\top}\btheta_{i_t}$ denote an optimal arm with the highest expected reward at round $t$. One objective of the learning agent is to minimize the expected cumulative regret
\begin{equation}\textstyle
R(T)=\EE[\sum_{t=1}^T(\bx_{a_t^*}^{\top}\btheta_{i_t}-\bx_{a_t}^{\top}\btheta_{i_t})]\,.
\label{regret def}
\end{equation}
% where the expectation is taken over the possible randomness of the algorithm and randomness of the environment
% regarding the users $i_1,\ldots, i_T$ and the arm sets $\mathcal{A}_1,\ldots,\mathcal{A}_T$. 
Another objective is to detect corrupted users online accurately. Specifically, at round $t$, the agent will give a set of users $\Tilde{\mathcal{U}}_t$ as the detected corrupted users, and we want $\Tilde{\mathcal{U}}_t$ to be as close to the ground-truth set of corrupted users $\Tilde{\mathcal{U}}$ as possible.

\section{Algorithms}

This section introduces our algorithms RCLUB-WCU (Algo.\ref{club-rac}) and OCCUD (Algo.\ref{occud}). RCLUB-WCU robustly learns the unknown user clustering structure and preferences from corrupted feedback, and leverages the cluster-based information to accelerate learning. Based on the clustering structure learned by RCLUB-WCU, OCCUD can accurately detect corrupted users online.
% \tong{the motivations of the two algorithms can be mentioned again here}
\begin{algorithm}[tbh!]
    \caption{RCLUB-WCU}
    \label{club-rac}
\resizebox{0.956\columnwidth}{!}{
\begin{minipage}{\columnwidth}
\begin{algorithmic}[1]
    \STATE{{\textbf{Input:}} Regularization parameter $\lambda$, confidence radius parameter $\beta$, threshold parameter $\alpha$, edge deletion parameter $\alpha_1$}, $f(T)=\sqrt{{(1 + \ln(1+T))}/{(1 + T)}}$.
    \STATE {{\textbf{Initialization:}}\label{complete graph}
$\boldsymbol{{M}}_{i,0} = \boldsymbol{0}_{d\times d}, \boldsymbol{{b}}_{i,0} = \boldsymbol{0}_{d \times 1},$
    $ \tilde{\boldsymbol{{M}}}_{i,0} = \boldsymbol{0}_{d\times d}, \tilde{\boldsymbol{{b}}}_{i,0} = \boldsymbol{0}_{d \times 1}, T_{i,0}=0$ , $\forall{i \in \mathcal{U}}$;\\
    A complete graph $G_0 = (\mathcal{U},E_0)$ over $\mathcal{U}$.
  }
    \FORALL{$t=1,2,\ldots, T$}
        % \STATE{set $\hat{\boldsymbol{\theta}} = \boldsymbol{M}_{i, t-1}^{-1}\boldsymbol{b}_{i, t-1}$, $\forall{i = 1, \ldots, n} $}
        \STATE{Receive the index of the current served user $i_{t} \in \mathcal{U}$, get the feasible arm set at this round $\mathcal{A}_t$\label{alg:club-rac:receive user index and arms}}.
        \STATE{Determine the connected components ${V}_{t}$ in the current maintained graph $G_{t-1}=(\mathcal{U},E_{t-1})$ such that $i_{t} \in {V}_{t}$.\label{find V_t}}
        \STATE{Calculate the robustly estimated statistics for 
        the cluster $V_t$:\\
        $\boldsymbol{M}_{{V}_{t},t-1} = \lambda\boldsymbol{I}+ \sum_{i \in {V}_{t}}\boldsymbol{M}_{i, t-1}\,,
        \boldsymbol{b}_{{V}_{t},t-1} = \sum_{i \in 
         {V}_{t}}\boldsymbol{b}_{i, t-1}\,,
        \hat{\boldsymbol{\theta}}_{{V}_{t},t-1} = \boldsymbol{M}_{{V}_{t},t-1}^{-1}\boldsymbol{b}_{{V}_{t},t-1}\,;$}
        \STATE {\label{recommend arm}Select an arm $a_t$ with largest UCB index in Eq.(\ref{UCB})
and receive the corresponding reward $r_t$;} 
\STATE  {Update the statistics for robust estimation of user $i_t$:\\
$\boldsymbol{M}_{i_t,t} = \boldsymbol{M}_{i_t,t-1} + w_{i_{t}, t-1}\boldsymbol{x}_{a_t}\boldsymbol{x}_{a_t}^{\top}\,,
\boldsymbol{b}_{i_t,t} =\boldsymbol{b}_{i_t,t-1} + w_{i_{t}, t-1}r_t\boldsymbol{x}_{a_t}\,,
T_{i_t,t} = T_{i_t,t-1} + 1\,,$\\
$\boldsymbol{M}_{i_t,t}^{\prime}=\lambda \boldsymbol{I}+\boldsymbol{M}_{i_t,t}$, $\hat{\boldsymbol{\theta}}_{i_t,t} = {\boldsymbol{M}_{i_t,t}^{\prime-1}}\boldsymbol{b}_{i_t,t}\,,
w_{i_{t}, t} = \min\{1, \alpha/{{\norm{\boldsymbol{x}_{a_t}}_{{\boldsymbol{M}_{i_t,t}^{\prime-1}}}}}\}\,;$\label{alg: update1}}

\STATE {Keep robust estimation statistics of other users unchanged\label{alg: update2}:\\
$
\boldsymbol{M}_{\ell,t} = \boldsymbol{M}_{\ell,t-1}, \boldsymbol{b}_{\ell,t} = \boldsymbol{b}_{\ell,t-1}, T_{\ell,t} = T_{\ell,t-1}     
$\,, $\hat{\boldsymbol{\theta}}_{\ell,t} = \hat{\boldsymbol{\theta}}_{\ell,t-1}$, for all $\ell\in\mathcal{U}, \ell \ne i_t$;}

\STATE {\label{alg:delete} Delete the edge $(i_t, \ell) \in E_{t-1}$, if

\begin{equation*}
    \norm{\hat{\boldsymbol{\theta}}_{i_t,t} - \hat{\boldsymbol{\theta}}_{\ell,t}}_2 \ge \alpha_1\big(f(T_{i_t,t})+f(T_{\ell,t}) + \alpha C\big)\,,
\end{equation*}
and get an updated graph $G_t = (\mathcal{U}, E_t)$;} 
\STATE {\label{use occud line} Use the OCCUD Algorithm (Algo.\ref{occud}) to detect the corrupted users.}
\ENDFOR
\end{algorithmic}
\end{minipage}
}
\end{algorithm}
\subsection{RCLUB-WCU}\label{section: rclub-wcu}

The corrupted behaviors may cause inaccurate preference estimations, leading to erroneous relation
inference and sub-optimal decisions. In this case, how to learn and utilize unknown user relations to accelerate learning becomes non-trivial. Motivated by this, we design RCLUB-WCU as follows.

%The design ideas and process of RCLUB-WCU are as follows.

\noindent\textbf{Assign the inferred cluster $V_t$ for user $i_t$.} RCLUB-WCU maintains a dynamic undirected graph $G_t=(\mathcal{U}, E_t)$ over users, which is initialized to be a complete graph (Algo.\ref{club-rac} Line \ref{complete graph}). Users with similar learned preferences will be connected with edges in $E_t$. The connected components in the graph represent the inferred clusters by the algorithm. At round $t$, user $i_t$ comes to be served with a feasible arm set $\mathcal{A}_t$ for the agent to choose from (Line \ref{alg:club-rac:receive user index and arms}). In Line \ref{find V_t}, RCLUB-WCU  detects the connected component $V_t$ in the graph containing user $i_t$ to be the current inferred cluster for $i_t$.

\noindent\textbf{Robust preference estimation of cluster $V_t$.} After determining the cluster $V_t$, RCLUB-WCU estimates the common preferences for users in $V_t$ using the historical feedback of all users in $V_t$ and recommends an arm accordingly.
The corrupted behaviors could cause inaccurate preference estimates, which can easily mislead the agent. To address this, inspired by \cite{zhao2021linear,he2022nearly}, we use weighted ridge regression to make corruption-robust estimations. Specifically, RCLUB-WCU robustly estimates the common preference vector of cluster $V_t$ by solving the following weighted ridge regression
\begin{equation}
    \textstyle{\hat{\btheta}_{{V}_t,t-1}=\mathop{\arg\min}\limits_{\btheta\in\RR^d}\sum_{s\in[t-1]\atop i_s\in {V}_t}w_{i_s,s}(r_s-\bx_{a_s}^{\top}\btheta)^2+\lambda\norm{\btheta}_2^2\,,}
    \label{optimization problem for cluster}
\end{equation}
where $\lambda>0$ is a regularization coefficient. Its closed-form solution is
% \begin{equation}
%     \hat{\btheta}_{{V}_t,t-1}={\bM}_{{V}_t,t-1}^{-1}{\bb}_{{V}_t,t-1}\,,\nonumber
% \end{equation}
% where
%     $${\bM}_{{V}_t,t-1}=\lambda\bI+\sum_{s\in[t-1]\atop i_s\in {V}_t}w_{i_s,s}\bx_{a_s}\bx_{a_s}^{\top},
%     {\bb}_{{V}_t,t-1}=\sum_{s\in[t-1]\atop i_s\in {V}_t}w_{i_s,s}r_{a_s}\bx_{a_s}.$$
    $\hat{\btheta}_{{V}_t,t-1}={\bM}_{{V}_t,t-1}^{-1}{\bb}_{{V}_t,t-1}\,,$
where
    ${\bM}_{{V}_t,t-1}=\lambda\bI+\sum_{s\in[t-1]\atop i_s\in {V}_t}w_{i_s,s}\bx_{a_s}\bx_{a_s}^{\top}$,
   ${\bb}_{{V}_t,t-1}=\sum_{s\in[t-1]\atop i_s\in {V}_t}w_{i_s,s}r_{a_s}\bx_{a_s}.$

We set the weight of sample for user $i_s$ in $V_t$ at round $s$ as $w_{i_s,s}=\min\{1,\alpha/\norm{\bx_{a_s}}_{M_{i_s,s}^{\prime-1}}\}$, where $\alpha$ is a coefficient to be determined later. The intuitions of designing these weights are as follows. The term $\norm{\bx_{a_s}}_{M_{i_s,s}^{\prime-1}}$ is the confidence radius of arm $a_s$ for user $i_s$ at $s$, reflecting how confident the algorithm is about the estimation of $i_s$'s preference on $a_s$ at $s$. If $\norm{\bx_{a_s}}_{M_{i_s,s}^{\prime-1}}$ is large, it means the agent is uncertain of user $i_s$'s preference on $a_s$, indicating this sample is probably corrupted. Therefore, we use the inverse of confidence radius to assign a small weight to this round's sample if it is potentially corrupted. In this way, uncertain information for users in cluster $V_t$ is assigned with less importance when estimating the $V_t$'s preference vector, which could help relieve the estimation inaccuracy caused by corruption. For technical details, please refer to Section \ref{section: theory rclub-wcu} and Appendix.

\noindent\textbf{Recommend $a_t$ with estimated preference of cluster $V_t$.} Based on the corruption-robust preference estimation $\hat{\btheta}_{{V}_t,t-1}$ of cluster $V_t$, in Line \ref{recommend arm}, the agent recommends an arm using the upper confidence bound (UCB)
strategy to balance exploration and exploitation
% \begin{equation}
% \label{UCB}
%     \begin{aligned}
%     a_t&=\argmax_{a\in \mathcal{A}_t} \text{UCB}_t(a)\\
%     &= \argmax_{a\in \mathcal{A}_t} \underbrace{\bx_a^{\top}\hat{\btheta}_{{V}_t,t-1}}_{\hat{R}_{a,t}}
%     + \underbrace{\beta \norm{\bx_a}_{{\bM}_{{V}_t,t-1}^{-1}}
%     }_{C_{a,t}}\,,
%     \end{aligned}
% \end{equation}    
% \begin{equation}
% \label{UCB}
%     a_t= \argmax_{a\in \mathcal{A}_t} \underbrace{\bx_a^{\top}\hat{\btheta}_{{V}_t,t-1}}_{\hat{R}_{a,t}}
%     + \underbrace{\beta \norm{\bx_a}_{{\bM}_{{V}_t,t-1}^{-1}}
%     }_{C_{a,t}}\,,
% \end{equation}
\begin{equation}
\label{UCB}
    a_t= \argmax_{a\in \mathcal{A}_t} \bx_a^{\top}\hat{\btheta}_{{V}_t,t-1}
    + \beta \norm{\bx_a}_{{\bM}_{{V}_t,t-1}^{-1}
    }\triangleq {\hat{R}_{a,t}}+{C_{a,t}}\,,
\end{equation}
where $\beta=\sqrt{\lambda}+\sqrt{2\log(\frac{1}{\delta})+d\log(1+\frac{T}{\lambda d})}+\alpha C$ is the confidence radius parameter,  $\hat{R}_{a,t}$ denotes the estimated reward of arm $a$ at $t$, $C_{a,t}$ denotes the confidence radius of arm $a$ at $t$. The design of $C_{a,t}$ theoretically relies on Lemma \ref{concentration bound} that will be given in Section \ref{section: theory}.

\noindent\textbf{Update the robust estimation of user $i_t$.}
After receiving $r_t$, the algorithm updates the estimation statistics of user $i_t$, while keeping the statistics of others unchanged (Line \ref{alg: update1} and Line \ref{alg: update2}). Specifically, RCLUB-WCU estimates the preference vector of user $i_t$ by solving a weighted ridge regression
\begin{equation}
\textstyle{\hat{\btheta}_{i_t,t}=\mathop{\arg\min}\limits_{\btheta\in\RR^d}\sum_{s\in[t]\atop i_s=i_t}w_{i_s,s}(r_s-\bx_{a_s}^{\top}\btheta)^2+\lambda\norm{\btheta}_2^2}
\end{equation}
with closed-form solution
       $\hat{\btheta}_{i_t,t}={(\lambda\bI+\bM_{i_t,t})}^{-1}\bb_{i_t,t}\,,$
where 
    $\bM_{i_t,t}=\sum_{s\in[t]\atop i_s=i_t}w_{i_s,s}\bx_{a_s}\bx_{a_s}^{\top}$,
    $ \bb_{i_t,t}=\sum_{s\in[t]\atop i_s=i_t}w_{i_s,s}r_{a_s}\bx_{a_s}\,,$
and we design the weights in the same way by the same reasoning. 
\begin{algorithm}[tbh!]
    \caption{OCCUD (At round $t$, used in Line \ref{use occud line} in Algo.\ref{club-rac})}
    \label{occud}
\resizebox{1\columnwidth}{!}{
\begin{minipage}{\columnwidth}
\begin{algorithmic}[1]
\STATE{Initialize $\tilde{\mathcal{U}}_t=\emptyset$; input probability parameter $\delta$.}
\STATE  {Update the statistics for non-robust estimation of user $i_t$\\
$\tilde{\boldsymbol{M}}_{i_t,t} = \Tilde{\boldsymbol{M}}_{i_t,t-1} + \boldsymbol{x}_{a_t}\boldsymbol{x}_{a_t}^{\top}$\,, 
$\tilde{\boldsymbol{b}}_{i_t,t} =\tilde{\boldsymbol{b}}_{i_t,t-1} + r_t\boldsymbol{x}_{a_t}\,,
\tilde{\boldsymbol{\theta}}_{i_t,t}= (\lambda \boldsymbol{I}+\tilde{\boldsymbol{M}}_{i_t,t})^{-1} \tilde{\boldsymbol{b}}_{i_t,t}$\,,\label{occud: update1}}

\STATE {Keep non-robust estimation statistics of other users unchanged\label{occud: update2} 

$\tilde{\boldsymbol{M}}_{\ell,t} = \tilde{\boldsymbol{M}}_{\ell,t-1}, \tilde{\boldsymbol{b}}_{\ell,t} = \tilde{\boldsymbol{b}}_{\ell,t-1}, \tilde{\boldsymbol{\theta}}_{\ell,t} = \tilde{\boldsymbol{\theta}}_{\ell,t-1}$, for all $\ell\in\mathcal{U}, \ell \ne i_t$\,.}
   \FORALL{connected component $V_{j,t} \in G_{t}$}
   % \STATE{Find the associated connected component $V_{i,t}$ in the current graph $G_{t-1}$ (i.e., $i\in V_{i,t}$)\,;\label{occud: find cluster}}
   \STATE{Calculate the robust estimation statistics for the cluster $V_{j,t}$:\\
  $\boldsymbol{M}_{{V}_{j,t},t} = \lambda\boldsymbol{I}+ \sum_{\ell \in {V}_{j,t}}\boldsymbol{M}_{\ell, t}\,, T_{V_{j,t},t}=\sum_{\ell\in V_{j,t}} T_{\ell,t}\,, $\\
  $\boldsymbol{b}_{{V}_{j,t},t} = \sum_{\ell \in 
         {V}_{j,t}}\boldsymbol{b}_{\ell, t}\,,
        \hat{\boldsymbol{\theta}}_{{V}_{j,t},t} = \boldsymbol{M}_{{V}_{j,t},t}^{-1}\boldsymbol{b}_{{V}_{j,t},t}\,;$
   \label{occud: cluster estimation}}
    \FORALL{user $i\in V_{j,t}$}
   \STATE{Detect user $i$ to be a corrupted user and add user $i$ to the set $\tilde{\mathcal{U}}_t$ if the following holds:
        \begin{align}
         \norm{\Tilde{\btheta}_{i,t}-\hat{\btheta}_{V_{i,t},t}}_2&>\frac{\sqrt{d\log(1+\frac{T_{i,t}}{\lambda d})+2\log(\frac{1}{\delta})}\sqrt{\lambda}}{\sqrt{\lambda_{\text{min}}(\tilde{\boldsymbol{M}}_{i,t})+\lambda}}\notag\\
         &+\frac{\sqrt{d\log(1+\frac{T_{V_{i,t},t}}{\lambda d})+2\log(\frac{1}{\delta})}+\sqrt{\lambda}+\alpha C}{\sqrt{\lambda_{\text{min}}(\boldsymbol{M}_{{V}_{i,t},t})}}\,,   
   \end{align}  

   where $\lambda_{\text{min}}(\cdot)$ denotes the minimum eigenvalue of the matrix argument. \label{detect line}
   }
\ENDFOR
\ENDFOR
\end{algorithmic}
\end{minipage}
}
\end{algorithm}
\noindent\textbf{Update the dynamic graph.} Finally, with the updated statistics of user $i_t$, RCLUB-WCU checks whether the inferred $i_t$'s preference similarities with other users are still true, and updates the graph accordingly. Precisely, if gap between the updated estimation $\hat{\btheta}_{i_t,t}$ of $i_t$ and the estimation $\hat{\btheta}_{\ell,t}$ of user $\ell$ exceeds a threshold in Line \ref{alg:delete}, RCLUB-WCU will delete the edge $(i_t,\ell)$ in $G_{t-1}$ to split them apart. The threshold is carefully designed to handle the estimation uncertainty from both stochastic noises and potential corruptions. The updated graph $G_t=(\mathcal{U},E_t)$ will be used in the next round.

% A simple illustration of our RCLUB-WCU can be found in Fig.\ref{figure1}.
\subsection{OCCUD}
Based on the inferred clustering structure of RCLUB-WCU, we devise a novel online detection algorithm OCCUD (Algo.\ref{occud}). The design ideas and process of OCCUD are as follows.

Besides the robust preference estimations (with weighted regression) of users and clusters kept by RCLUB-WCU, OCCUD also maintains the non-robust estimations for each user by online ridge regression without weights (Line \ref{occud: update1} and Line \ref{occud: update2}). Specifically, at round $t$, OCCUD updates the non-robust estimation of user $i_t$ by solving the following online ridge regression:
\begin{equation}
\textstyle{\tilde{\btheta}_{i_t,t}=\mathop{\arg\min}\limits_{\btheta\in\RR^d}\sum_{s\in[t]\atop i_s=i_t}(r_s-\bx_{a_s}^{\top}\btheta)^2+\lambda\norm{\btheta}_2^2\,,}
\end{equation}
with solution
\begin{small}
    $\tilde{\btheta}_{i_t,t}={(\lambda\bI+\tilde{\bM}_{i_t,t})}^{-1}\tilde{\bb}_{i_t,t}\,,$
where 
$\tilde{\bM}_{i_t,t}=\sum_{s\in[t]\atop i_s=i_t}\bx_{a_s}\bx_{a_s}^{\top}\,, \tilde{\bb}_{i_t,t}=\sum_{s\in[t]\atop i_s=i_t}r_{a_s}\bx_{a_s}\,.$
\end{small}
With the robust and non-robust preference estimations, OCCUD does the following to detect corrupted users based on the clustering structure inferred by RCLUB-WCU. 
First, OCCUD finds the connected components in the graph kept by RCLUB-WCU, which represent the inferred clusters. Then, for each inferred cluster $V_{j,t}\in G_t$: (1) OCCUD computes its robustly estimated preferences vector $\hat{\btheta}_{V_{i,t},t}$ (Line \ref{occud: cluster estimation}). (2) For each user $i$ whose inferred cluster is $V_{j,t}$ (\emph{i.e.,}$i\in V_{j,t}$), OCCUD computes the gap between user $i$'s non-robustly estimated preference vector $\Tilde{\btheta}_{i,t}$ and the robust estimation $\hat{\btheta}_{V_{i,t},t}$ for user $i$'s inferred cluster $V_{j,t}$. If the gap exceeds a carefully-designed threshold, OCCUD will detect user $i$ as corrupted and add $i$ to the detected corrupted user set $\Tilde{\mathcal{U}}_t$ (Line \ref{detect line}).

\begin{figure*}[htp]
\resizebox{1\columnwidth}{!}{
\begin{minipage}{\columnwidth}
    \subfigure[RCLUB-WCU]{
    \includegraphics[scale=0.2]{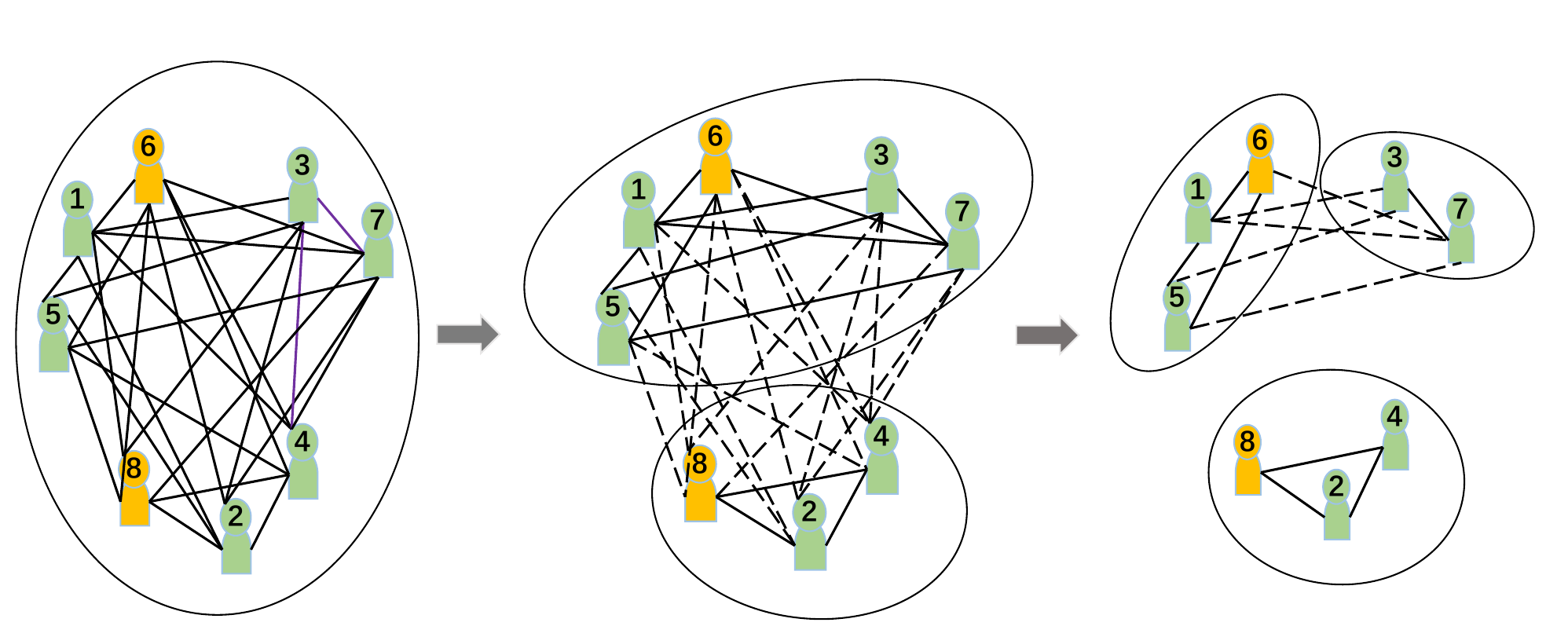}
    }
    \vrule
    \subfigure[OCCUD]{
    \includegraphics[scale=0.2]{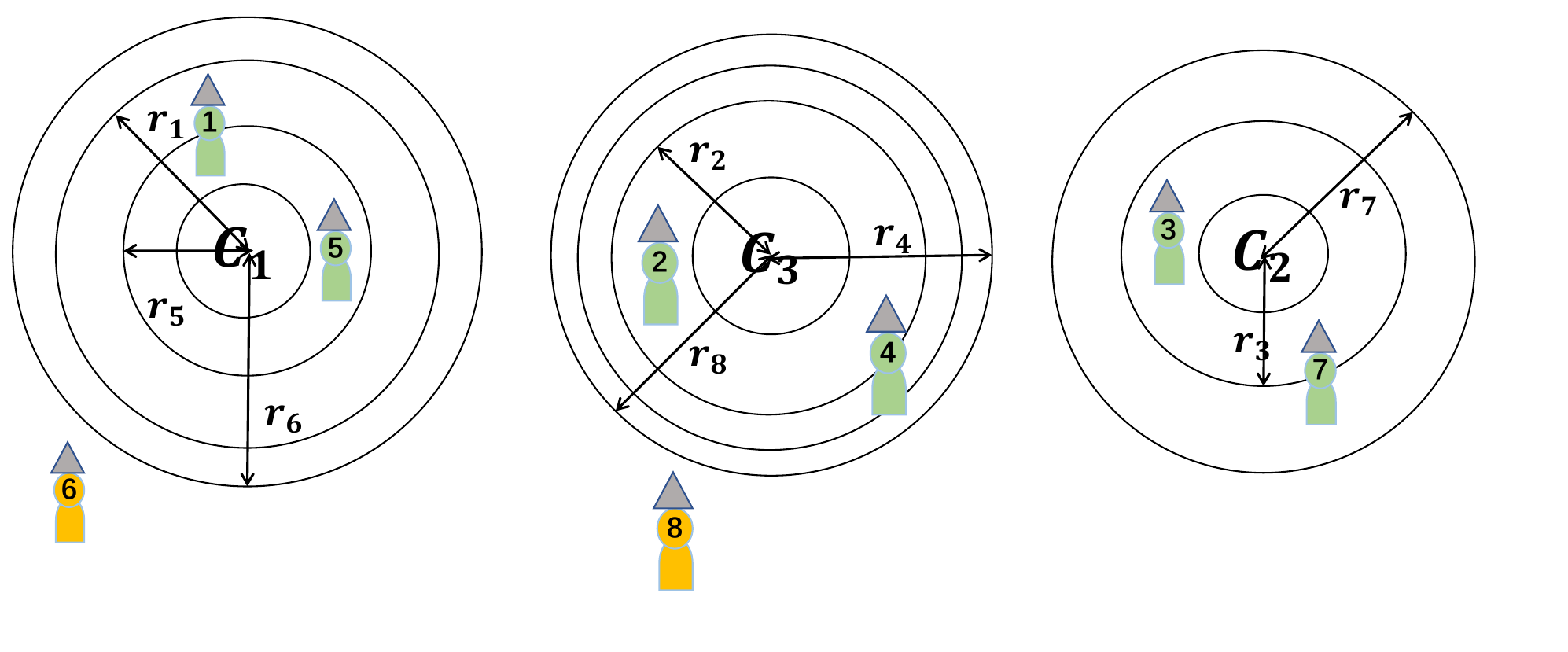}
    }
    \caption{Algorithm illustrations. Users 6 and 8 are corrupted users (orange), and the others are normal (green). (a) illustrates RCLUB-WCU, which starts with a complete user graph, and adaptively deletes edges between users (dashed lines) with dissimilar robustly learned preferences. The corrupted behaviors of users 6 and 8 may cause inaccurate preference estimations, leading to erroneous relation inference. In this case, how to delete edges correctly is non-trivial, and RCLUB-WCU addresses this challenge (detailed in Section \ref{section: rclub-wcu}).
        (b) illustrates OCCUD at some round $t$, where person icons with triangle hats represent the non-robust user preference estimations. 
        The gap between the non-robust estimation of user 6 and the robust estimation of user 6's inferred cluster (circle $C_1$) exceeds the threshold $r_6$ at this round (Line \ref{detect line} in Algo.\ref{occud}), so OCCUD detects user 6 to be corrupted.}
    \label{fig: algorithm illustration} \end{minipage}}
\end{figure*}
The intuitions of OCCUD are as follows. On the one hand, after some interactions, RCLUB-WCU will infer the user clustering structure accurately. Thus, at round $t$, the robust estimation $\hat{\btheta}_{V_{i,t},t}$ for user $i$'s inferred cluster should be pretty close to user $i$'s ground-truth preference vector $\btheta_i$. On the other hand, since the feedback of normal users are always regular, at round $t$, if user $i$ is a normal user, the non-robust estimation $\Tilde{\btheta}_{i,t}$ should also be close to the ground-truth $\btheta_i$. However, the non-robust estimation of a corrupted user should be quite far from the ground truth due to corruptions. Based on this reasoning, OCCUD compares each user's non-robust estimation and the robust estimation of the user's inferred cluster to detect the corrupted users. For technical details, please refer to Section \ref{section occud theory} and Appendix.
Simple illustrations of our proposed algorithms can be found in Fig.\ref{fig: algorithm illustration}.

\section{Theoretical Analysis}
\label{section: theory}

In this section, we theoretically analyze the performances of our proposed algorithms, RCLUB-WCU and OCCUD. 
% by giving an upper bound of its expected regret defined in Eq.(\ref{regret def}). 
Due to the page limit, we put the proofs in the Appendix.

\subsection{Regret Analysis of RCLUB-WCU}
\label{section: theory rclub-wcu}
This section gives an upper bound of the expected regret (defined in Eq.(\ref{regret def})) for RCLUB-WCU. 

The following lemma provides a sufficient time $T_0(\delta)$, after which RCLUB-WCU can cluster all the users correctly with high probability.
\begin{lemma}
\label{sufficient time}
With probability at least $1-3\delta$, RCLUB-WCU will cluster all the users correctly after 
\begin{small}
 \begin{equation*}
    \begin{aligned}
        T_0(\delta) &\triangleq 16u\log(\frac{u}{\delta})+4u\max\{\frac{288d}{\gamma^2\alpha\sqrt{\lambda}\tilde{\lambda}_x}\log(\frac{u}{\delta}), \frac{16}{\tilde{\lambda}_x^2}\log(\frac{8d}{\tilde{\lambda}_x^2\delta}),\frac{72\sqrt{\lambda}}{\alpha\gamma^2\tilde{\lambda}_x},\frac{72\alpha C^2}{\gamma^2\sqrt{\lambda}\tilde{\lambda}_x}\}\,
    \end{aligned}
\end{equation*}   
\end{small}
for some $\delta\in(0,\frac{1}{3})$, where $\tilde{\lambda}_x\triangleq\int_{0}^{\lambda_x} (1-e^{-\frac{(\lambda_x-x)^2}{2\sigma^2}})^{K} dx$, $\left|\mathcal{A}_t\right|\leq K,\forall{t\in[T]}$.
\end{lemma}

After $T_0(\delta)$, the following lemma gives a bound of the gap between $\hat{\btheta}_{{V}_t,t-1}$ and the ground-truth $\btheta_{i_t}$
in direction of action vector $\bx_a$ for RCLUB-WCU, which supports the design in Eq.(\ref{UCB}).

\begin{lemma}
\label{concentration bound}
With probability at least $1-4\delta$ for some $\delta\in(0,\frac{1}{4})$, $\forall{t\geq T_0(\delta)}$, we have:
\begin{equation*}
\begin{aligned}
        \left|\boldsymbol{x}_{a}^{\mathrm{T}}(\hat{\boldsymbol{\theta}}_{V_{t},t-1} - \boldsymbol{\theta}_{i_t})\right| \le \beta\norm{\boldsymbol{x_{a}}}_{\boldsymbol{M}^{-1}_{V_t, t-1}}\triangleq C_{a,t}\,.
\end{aligned}
\end{equation*}
% where $\beta=\sqrt{\lambda}+\sqrt{2\log(\frac{1}{\delta})+d\log(1+\frac{T}{\lambda d})}+\alpha C$.
\end{lemma}

With Lemma \ref{sufficient time} and \ref{concentration bound}, we prove the following theorem on the regret upper bound of RCLUB-WCU.

% \begin{theorem}[\textbf{Regret Upper Bound of RCLUB-WCU}]
% \label{thm:main}
% With the assumptions in Section \ref{setup}, the expected regret of the RCLUB-WCU algorithm for $T$ rounds satisfies

% \begin{equation}
% \begin{aligned}
%         R(T)&\leq T_0(\frac{1}{T})+2\bigg(\sqrt{2mdT\log(1+\frac{T}{\lambda d})}+\frac{2md}{\alpha\sqrt{\lambda}}\log(1+\frac{T}{\lambda d})\bigg)\\
%         &\times\bigg(\sqrt{\lambda}+\alpha C+\sqrt{d\log(1+\frac{T}{\lambda d})+2\log T}\bigg)\,,
%     \label{regret upper bound initial}
% \end{aligned}
% \end{equation}
% picking $\alpha=\frac{\sqrt{d}+\sqrt{\lambda}}{C}$ for optimal trade-off, we have
%     \begin{align}
%         R(T)&\le O\bigg(\big(\frac{C\sqrt{d}}{\gamma^{2}\tilde{\lambda}_{x}}+\frac{1}{\tilde{\lambda}_{x}^{2}}\big)u\log(T)\bigg)+O\bigg(d\sqrt{mT}\log(T)\bigg)\notag\\
%     &\quad + O\bigg(mCd\log^{1.5}(T)\bigg)\,.    \label{regret bound 3 parts}
%     \end{align}
% \end{theorem}
\begin{theorem}[\textbf{Regret Upper Bound of RCLUB-WCU}]
\label{thm:main}
With the assumptions in Section \ref{setup}, and picking $\alpha=\frac{\sqrt{d}+\sqrt{\lambda}}{C}$, the expected regret of the RCLUB-WCU algorithm for $T$ rounds satisfies
\begin{small}
     \begin{align}
        R(T)&\le O\big((\frac{C\sqrt{d}}{\gamma^{2}\tilde{\lambda}_{x}}+\frac{1}{\tilde{\lambda}_{x}^{2}})u\log(T)\big)+O\big(d\sqrt{mT}\log(T)\big)+ O\big(mCd\log^{1.5}(T)\big)\,.    \label{regret bound 3 parts}
    \end{align}   
\end{small}
\end{theorem}

\noindent\textbf{Discussion and Comparison.} The regret bound in Eq.(\ref{regret bound 3 parts}) has three terms. The first term is the time needed to get enough information for accurate robust estimations such that RCLUB-WCU could cluster all users correctly afterward with high probability. This term is related to the \textit{corruption level} $C$, which is inevitable since, if there are more corrupted user feedback, it will be harder for the algorithm to learn the clustering structure correctly. The last two terms correspond to the regret after $T_0$ with the correct clustering. Specifically, the second term is caused by stochastic noises when leveraging the aggregated information within clusters to make recommendations; the third term associated with the \textit{corruption level} $C$ is the regret caused by the disruption of corrupted behaviors.

When the \textit{corruption level} $C$ is \textit{unknown}, we can use its estimated upper bound $\hat{C}\triangleq\sqrt{T}$ to replace $C$ in the algorithm. In this way, if $C\leq\hat{C}$, the bound will be replacing $C$ with $\hat{C}$ in Eq.(\ref{regret bound 3 parts}); when $C>\sqrt{T}$, $R(T)=O(T)$, which is already optimal for a large class of bandit algorithms \cite{he2022nearly}.

The following theorem gives a regret lower bound of the LOCUD problem.
\begin{theorem}[Regret lower bound for LOCUD]
\label{thm: lower bound}
There exists a problem instance for the LOCUD problem such that for any algorithm
\begin{equation*}
    R(T)\geq \Omega(d\sqrt{mT}+dC)\,.
\end{equation*}
\end{theorem}

Its proof and discussions can be found in Appendix \ref{appendix: proof of lower bound}. The upper bound in Theorem \ref{thm:main}
asymptotically matches this lower bound in $T$ up to logarithmic factors, showing our regret bound is nearly optimal.
 
 We then compare our regret upper bound with several degenerated 
cases. First, when $C=0$, \emph{i.e.}, all users are normal, our setting degenerates to the classic CB problem \cite{gentile2014online}. In this case the bound in Theorem \ref{thm:main} becomes $O({1}/{\tilde{\lambda}_{x}^{2}}\cdot u\log(T))+O(d\sqrt{mT}\log(T))$, perfectly matching the state-of-the-art results in CB \cite{gentile2014online,li2018online,li2019improved}. Second, when $m=1$ and $u=1$, \emph{i.e.}, there is only one user, our setting degenerates to linear bandits with adversarial corruptions \cite{li2019stochastic,he2022nearly}, and the bound in Theorem \ref{thm:main} becomes $O(d\sqrt{T}\log(T))+ O(Cd\log^{1.5}(T))$, it also perfectly matches the nearly optimal result in \cite{he2022nearly}.
The above comparisons also show the tightness of
the regret bound of RCLUB-WCU.

\subsection{Theoretical Performance Guarantee for OCCUD}
\label{section occud theory}

The following theorem gives a performance guarantee of the online detection algorithm OCCUD.

\begin{theorem}[\textbf{Theoretical Guarantee for OCCUD}]
\label{thm:occud}
With assumptions in Section \ref{setup}, at $\forall{t\geq T_0(\delta)}$, for any detected corrupted user $i\in \Tilde{\mathcal{U}}_t$, with probability at least $1-5\delta$, $i$ is indeed a corrupted user.
\end{theorem}

This theorem guarantees that after RCLUB-WCU learns the clustering structure accurately, with high probability, the corrupted users detected by OCCUD are indeed corrupted, showing the high detection accuracy of OCCUD. The proof of Theorem \ref{thm:occud} can be found in Appendix \ref{appendix: proof of lower bound}.

\section{Experiments}
\label{section:experiments}
This section shows experimental results on synthetic and real data to evaluate RCLUB-WCU's recommendation quality and OCCUD's detection accuracy. We compare RCLUB-WCU to LinUCB \cite{abbasi2011improved} with a single non-robust estimated vector for all users, LinUCB-Ind with separate non-robust estimated vectors for each user, CW-OFUL \cite{he2022nearly} with a single robust estimated vector for all users, CW-OFUL-Ind with separate robust estimated vectors for each user, CLUB\cite{gentile2014online}, and SCLUB\cite{li2019improved}. More description of these baselines are in Appendix \ref{appendix: baselines}. To show that the design of OCCUD is non-trivial, we develop a straightforward detection
algorithm GCUD, which leverages the same cluster structure as OCCUD but detects corrupted users by selecting users
with highest \begin{small}
$\norm{\hat{\boldsymbol{\theta}}_{i,t} - \hat{\boldsymbol{\theta}}_{V_{i,t},t-1}}_2$    
\end{small}
in each inferred cluster. GCUD selects users according to the underlying percentage of corrupted
users, which is unrealistic in practice, but OCCUD still performs better in this unfair condition.

\noindent\textbf{Remark.} 
The offline detection methods \cite{zhang2021fraudre, dou2020enhancing, li2022dual, qin2022explainable} need to know all the user information in advance to derive the user embedding for classification, so they cannot be directly applied in online detection with bandit feedback thus cannot be directly compared to OCCUD. However, we observe the AUC achieved by
OCCUD on Amazon and Yelp (in Tab.\ref{tab:real AUC}) is similar to recent offline methods \cite{li2022dual, qin2022explainable}. Additionally, OCCUD has rigorous theoretical performance guarantee (Section \ref{section occud theory}). 

\subsection{Experiments on Synthetic Dataset}
\label{exp:synthetic}

We use $u = 1,000$ users and $m = 10$ clusters, where each cluster contains $100$ users. We randomly select $100$ users as the corrupted users. The preference and arm (item) vectors are drawn in $d-1$ ($d=50$) dimensions with each entry a standard Gaussian variable and then normalized, added
one more dimension with constant 1, and divided by $\sqrt{2}$ \cite{li2019improved}. We fix an arm set with $
\left|\mathcal{A}\right|= 1000$ items, at each round, 20 items are randomly selected to form a set $\mathcal{A}_{t}$ to choose from. Following \cite{zhao2021linear,bogunovic2021stochastic}, in the first $k$ rounds, we always flip the reward of corrupted users by setting $r_{t} = -\boldsymbol{x}_{a_{t}}^{\mathrm{T}}\boldsymbol{\theta}_{i_{t},t} + \eta_{t}$. And we leave the remaining $T-k$ rounds intact. Here we set $T=1,000,000$ and $k=20,000$.

\begin{figure*}[htp]
\resizebox{0.96\columnwidth}{!}{
\begin{minipage}{\columnwidth}
    \subfigure[Synthetic]{
    \includegraphics[scale=0.25]{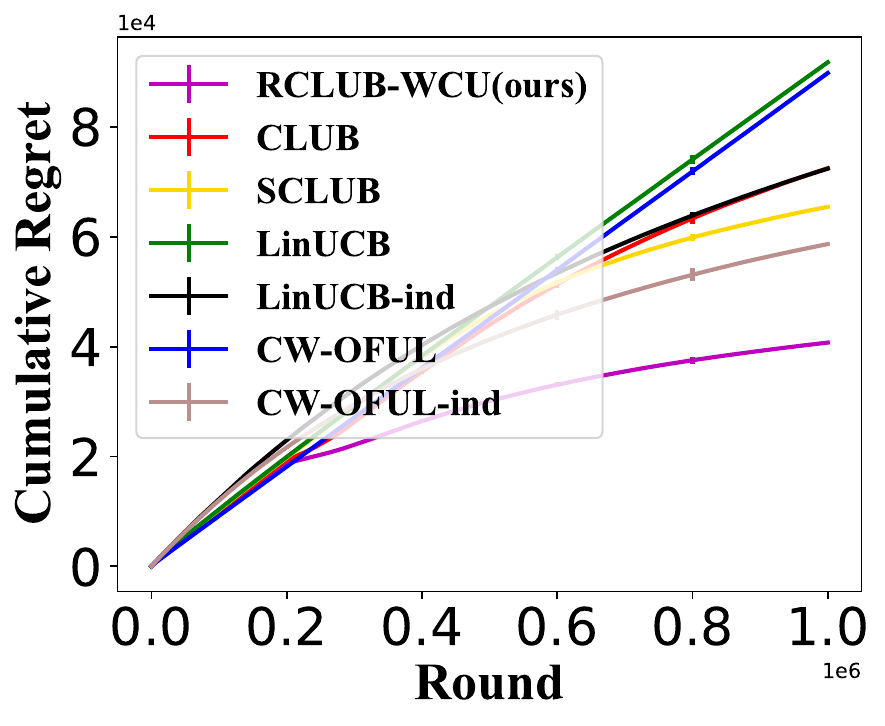}
    }
    \subfigure[Movielens ]{
    \includegraphics[scale=0.25]{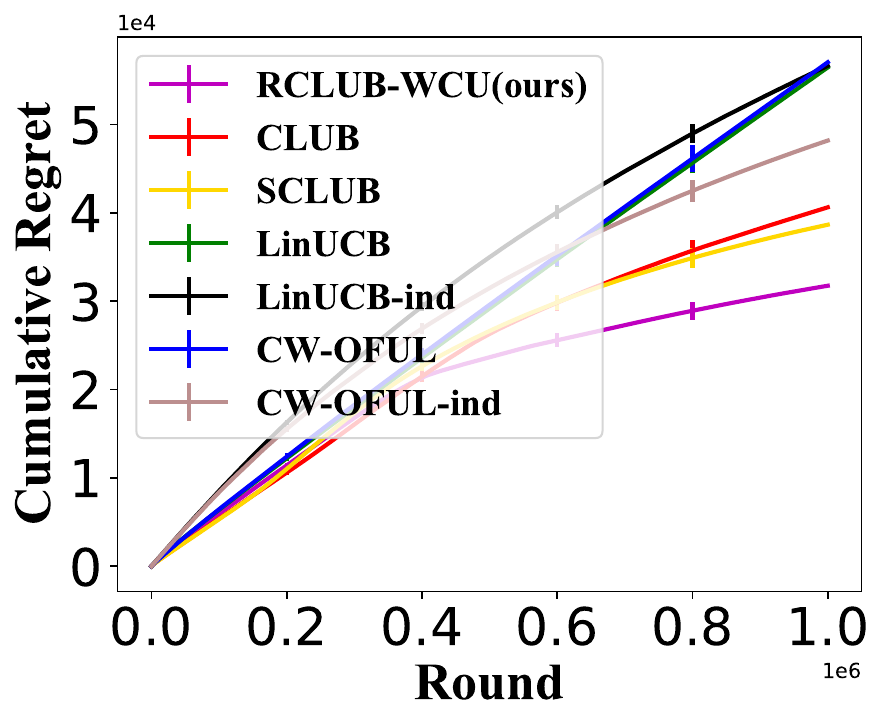}
    }
    \subfigure[Amazon ]{
    \includegraphics[scale=0.25]{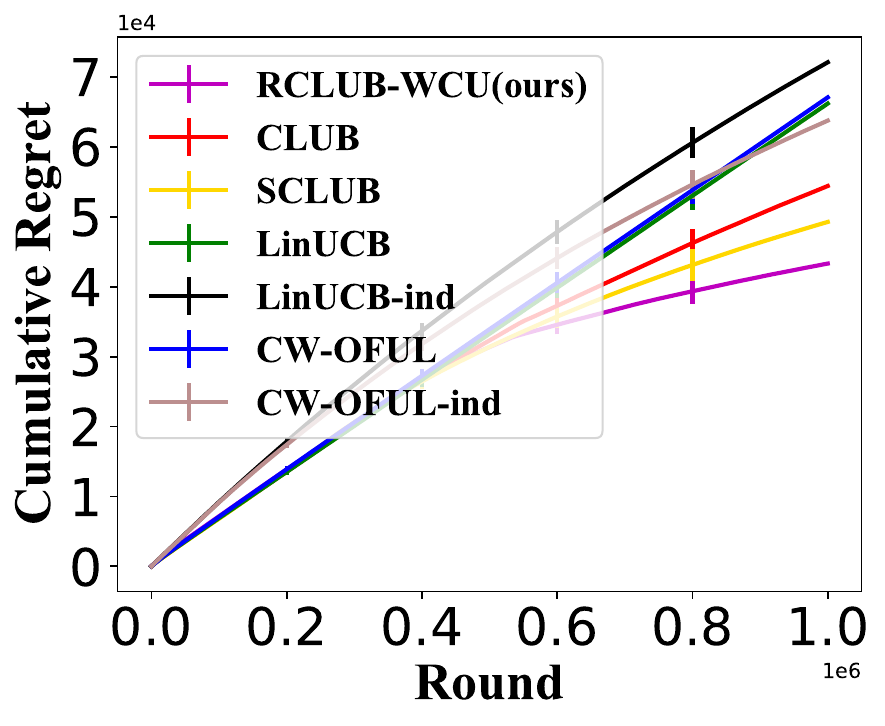}
    }
    \subfigure[Yelp ]{
    \includegraphics[scale=0.25]{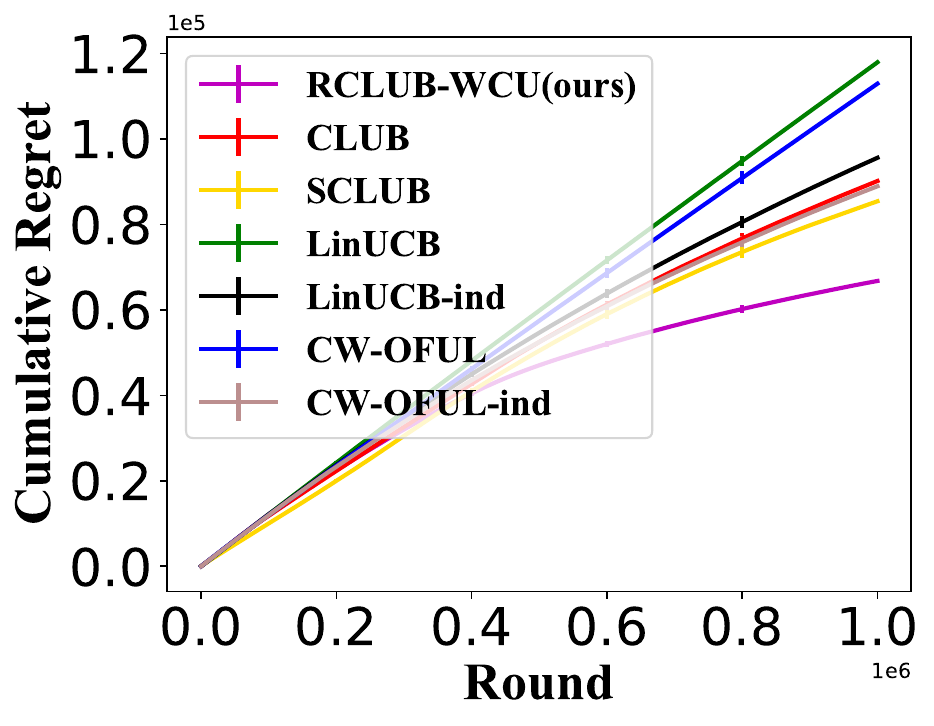}
    }

    \caption{Recommendation results on the synthetic and real-world datasets}
    \label{fig: real regret}\end{minipage}}

\end{figure*}

Fig.\ref{fig: real regret}(a) shows the  recommendation results. RCLUB-WCU outperforms all baselines and achieves a sub-linear regret. LinUCB and CW-OFUL perform worst as they ignore the preference differences among users. CW-OFUL-Ind outperforms LinUCB-Ind because it considers the corruption, but worse than RCLUB-WCU since it does not consider leveraging user relations to speed up learning. 
%  \begin{figure}[hbtp]
%  \centering
%     \includegraphics[scale=0.25]{synthetic.pdf}
%     \caption{Recommendation results on the synthetic dataset}
%     \label{fig:synthetic}
% \end{figure}

The detection results are shown in Tab.\ref{tab:real AUC}. We test the AUC of OCCUD and GCUD in every $200,000$ rounds. OCCUD's performance improves over time with more interactions, while GCUD's performance is much worse as it detects corrupted users only relying on the robust estimations. OCCUD finally achieves an AUC of 0.855, indicating it can identify most of the corrupted users.

\subsection{Experiments on Real-world Datasets}

We use three real-world data Movielens \cite{harper2015movielens}, Amazon\cite{mcauley2013amateurs}, and Yelp \cite{rayana2015collective}. 
% The Movielens dataset contains 2,113 users and 10,197 movies. The Amazon dataset, which is a subset of the Amazon musical instrument review dataset, contains 1,429 users and 900 items. The Yelp dataset contains 1,987,929 users and 150,346 items. 
The Movielens data does not have the corrupted users' labels, so following \cite{liu2017holoscope}, we manually select the corrupted users. On Amazon data, following \cite{zhang2021fraudre}, we label the users with more than 80\% helpful votes as normal users, and label users with less than 20\% helpful votes as corrupted users. The Yelp data contains users and their comments on restaurants with true labels of the normal users and corrupted users.

% According to the different sizes of the datasets, we select different quantities of top users who rate most and items with the most ratings, which is a common method in the previous bandit works \cite{li2018online,li2019improved,xie2021comparison}.
We select 1,000 users and 1,000 items for Movielens; 1,400 users and 800 items for Amazon; 2,000 users and 2,000 items for Yelp. The ratios of corrupted users on these data are 10\%, 3.5\%, and 30.9\%, respectively.
We generate the preference and item vectors following \cite{wu2021clustering,li2019improved}. We first construct the binary feedback matrix through the users' ratings: if the rating is greater than 3, the feedback is 1; otherwise, the feedback is 0. Then we use SVD to decompose the extracted binary feedback matrix $R_{u \times m} = \boldsymbol{\theta}SX^{\mathrm{T}}$, where $\boldsymbol{\theta} = (\boldsymbol{\theta}_{i}), i \in [u]$ and $X = (\boldsymbol{x}_{j}), j \in [m]$, and select $d=50$ dimensions. We have 10 clusters on Movielens and Amazon, and 20 clusters on Yelp. We use the same corruption mechanism as the synthetic data with $T=1,000,000$ and $k=20,000$. We conduct more experiments in different environments to show our algorithms' robustness in Appendix.\ref{sec:more experiments}.
% \begin{figure*}
%     \subfigure[Amazon Corruption Level]{
%     \includegraphics[scale=0.2]{amazon_corruption.pdf}
%     }
%     \subfigure[Yelp Corruption Level]{
%     \includegraphics[scale=0.2]{yelp_corruption.pdf}
%     }
%      \subfigure[Amazon Cluster Number]{
%     \includegraphics[scale=0.2]{amazon_cluster.pdf}
%     }
%     \subfigure[Yelp Cluster Number]{
%     \includegraphics[scale=0.2]{yelp_cluster.pdf}
%     }
%      %  \vspace{-0.28cm}
%      \caption{Cumulative regret in different environments. (a)(b) show the results under different corruption levels, (c)(d) show the results with different cluster numbers }
%     \label{fig:corruption level}
% \end{figure*}
\begin{wraptable}{r}{0.56\textwidth}
\resizebox{0.56\columnwidth}{!}{
    \begin{tabular}{|c|c|c|c|c|c|c|}
    \hline 
    Dataset  & \makebox[0.13\textwidth][c]{\diagbox{Alg}{Time}}  & 0.2M  &0.4M  &0.6M  &0.8M  & 1M \\
    \hline  
   \multirow{2}*{Synthetic} & OCCUD  & 0.599  & 0.651  & 0.777  &0.812  &\textbf{0.855} \\
   \cline{2-7}
   & GCUD  & 0.477  & 0.478  & 0.483  &0.484  &0.502\\
   \hline
   \multirow{2}*{ Movielens} & OCCUD  & 0.65  & 0.750  & 0.785  &0.83  &\textbf{0.85} \\
   \cline{2-7}
   & GCUD  & 0.450  & 0.474   & 0.485  &0.489  &0.492\\
   \hline
   \multirow{2}*{Amazon} & OCCUD  & 0.639  & 0.735  & 0.761  & 0.802 & \textbf{0.840} \\
   \cline{2-7}
   & GCUD  & 0.480  & 0.480 & 0.486  & 0.500  & 0.518 \\
   \hline
   \multirow{2}*{Yelp} & OCCUD  & 0.452  & 0.489 & 0.502  & 0.578  & \textbf{0.628} \\
   \cline{2-7}
   &GCUD  & 0.473  & 0.481  & 0.496  & 0.500  & 0.510 \\
   \hline
    \end{tabular}}
    \caption{Detection results on synthetic and real datasets}
    \label{tab:real AUC}

\end{wraptable}
The recommendation results are shown in Fig.\ref{fig: real regret}(b)-(d). RCLUB-WCU outperforms all baselines. 
% On the Movielens dataset and the Amazon dataset where the gaps between users are smaller, LinUCB performs better than LinUCB-Ind, but eventually LinUCB-Ind shows the trend to surpass LinUCB. 
On the Amazon dataset, the percentage of corrupted users is lowest, RCLUB-WCU's advantages over baselines decrease because of the weakened corruption. The corrupted user detection results are provided in Tab.\ref{tab:real AUC}. OCCUD's performance improves over time and is much better than GCUD. On the Movielens dataset, OCCUD achieves an AUC of 0.85; on the Amazon dataset, OCCUD achieves an AUC of 0.84; and on the Yelp dataset, OCCUD achieves an AUC of 0.628. According to recent works on offline settings \cite{li2022dual, qin2022explainable}, our results are relatively high.

\chapter{Efficient Explorative Key-term Selection Strategies for Conversational Contextual
Bandits}\label{chapter: aaai}

Conversational contextual bandits elicit user preferences by occasionally querying for explicit feedback on key-terms to accelerate learning. 
However, there are aspects of existing approaches which limit their performance.
First, information gained from key-term-level conversations and arm-level recommendations is not appropriately incorporated to speed up learning.
Second, it is important to ask explorative key-terms to quickly elicit the user's potential interests in various domains to accelerate the convergence of user preference estimation, which has never been considered in existing works. 
To tackle these issues, we first propose ``ConLinUCB", a general framework for conversational bandits with better information incorporation, combining arm-level and key-term-level feedback to estimate user preference in one step at each time. Based on this framework, we further design two bandit algorithms with explorative key-term selection strategies, ConLinUCB-BS and ConLinUCB-MCR. We prove tighter regret upper bounds of our proposed algorithms. Particularly, ConLinUCB-BS achieves a regret bound of $O(d\sqrt{T\log T})$, better than the previous result $O(d\sqrt{T}\log T)$. Extensive experiments on synthetic and real-world data show significant advantages of our algorithms in learning accuracy (up to 54\% improvement) and computational efficiency (up to 72\% improvement), compared to the classic ConUCB algorithm, showing the potential benefit to recommender systems. This chapter is based on our publication \cite{wang2023efficient}.

\section{Introduction}
Nowadays, recommender systems are widely used in various areas. The learning speed for traditional online recommender systems is usually slow since extensive exploration is needed to discover user preferences. To accelerate the learning process and provide more personalized recommendations, the conversational recommender system (CRS) has been proposed \cite{christakopoulou2016towards,christakopoulou2018q,sun2018conversational,zhang2018towards,li2021seamlessly,gao2021advances}. In a CRS,  a learning agent occasionally asks for the user's explicit feedback on some ``key-terms", and leverages this additional conversational information to better elicit the user's preferences \cite{zhang2020conversational,xie2021comparison}. 
    
Despite the recent success of CRS, there are crucial limitations in using conversational contextual bandit approaches to design recommender systems. These limitations include: (a) The information gained from key-term-level conversations and arm-level recommendations is not incorporated properly to speed up learning, as the user preferences are essentially assumed to be the same in these two stages but are estimated separately \cite{zhang2020conversational,xie2021comparison,wu2021clustering}; 
(b) Queries using traditional key-terms were restrictive and not explorative enough. Specifically, we say a key-term is ``explorative" if it is under-explored so far and the system is uncertain about the user's preferences in its associated items. Asking for the user's feedback on explorative key-terms can efficiently elicit her potential interests in various domains (e.g., sports, science), which means we can quickly estimate the user preference vector in all directions of the feature space, thus accelerating the learning speed. Therefore, it is crucial to design explorative key-term selection strategies, which existing works have not considered.
    
Motivated by the above considerations, we propose to design conversational bandit algorithms that (i) estimate the user's preferences utilizing both arm-level and key-term-level interactions simultaneously to properly incorporate the information gained from both two levels and (ii) use effective strategies to choose explorative key-terms when conducting conversations for quick user preference inference.
    
To better utilize the interactive feedback from both recommendations and conversations, we propose ConLinUCB, a general framework for conversational bandits with possible flexible key-term selection strategies. ConLinUCB estimates the user preference vector by solving \textit{one single optimization problem} that minimizes the mean squared error of both arm-level estimated rewards and key-term-level estimated feedback simultaneously, instead of separately estimating at different levels as in previous works. In this manner, the information gathered from these two levels can be better combined to guide the learning. 

Based on this ConLinUCB framework, we design two new algorithms with explorative key-term selection strategies, ConLinUCB-BS and ConLinUCB-MCR.
\begin{itemize}
    \item ConLinUCB-BS makes use of a barycentric spanner containing linearly independent vectors, which can be an efficient exploration basis in bandit problems \cite{amballa2021computing}. Whenever a conversation is allowed, ConLinUCB-BS selects an explorative key-term uniformly at random from a precomputed barycentric spanner $\mathcal{B}$ of the given key-term set $\mathcal{K}$. 
    \item
    ConLinUCB-MCR applies in a more general setting when the key-term set can be time-varying, and it can leverage interactive histories to choose explorative key-terms adaptively. Note that in the bandit setting, we often use confidence radius to adaptively evaluate whether an arm has been sufficiently explored, and the confidence radius of an arm will shrink whenever it is selected  \cite{lattimore2020bandit}. This implies that an explorative key-term should have a large confidence radius. Based on this reasoning, ConLinUCB-MCR selects the most explorative key-terms with maximal confidence radius when conducting conversations. 
\end{itemize}

Equipped with explorative conversations, our algorithms can quickly elicit user preferences for better recommendations. For example, if the key-term \textit{sports} is explorative at round $t$, indicating that so far the agent is not sure whether the user favors items associated with \textit{sports} (e.g., basketball, volleyball), it will ask for the user's feedback on \textit{sports} directly and conduct recommendations
accordingly. In this manner, the agent can quickly find suitable items for the user. We prove the regret upper bounds of our algorithms, which are better than the classic ConUCB algorithm.
    
In summary, our paper makes the following contributions:
\begin{itemize}
    \item We propose a new and general framework for conversational contextual bandits, ConLinUCB, which can efficiently incorporate the interactive information gained from both recommendations and conversations.
    \item  Based on ConLinUCB, we design two new algorithms with explorative key-term selection strategies, ConLinUCB-BS and ConUCB-MCR, which can accelerate the convergence of user preference estimation. 
    \item  We prove that our algorithms achieve tight regret upper bounds. Particularly, ConLinUCB-BS achieves a bound of $O(d\sqrt{T\log T})$, better than the previous $O(d\sqrt{T}\log T)$ in the conversational bandits literature.
    \item Experiments on both synthetic and real-world data validate the advantages of our algorithms in both learning accuracy (up to 54\% improvement) and computational efficiency (up to 72\% improvement)\footnote{Codes are available at \url{https://github.com/ZhiyongWangWzy/ConLinUCB.}}.
\end{itemize}

\section{Problem Settings}\label{section2}

This section states the problem setting of conversational contextual bandits. Suppose there is a finite set $\mathcal{A}$ of arms. Each arm $a \in \mathcal{A}$ represents an item to be recommended and is associated with a feature vector $\boldsymbol{x}_a \in \mathbb{R}^d$. Without loss of generality, the feature vectors are assumed to be normalized, i.e., $\norm{\boldsymbol{x}_a}_2=1$, ${\forall} a\in\mathcal{A}$. The agent interacts with a user in $T\in\mathbb{N}_{+}$ rounds, whose preference of items is represented by an \textit{unknown} vector $\boldsymbol{\theta}^* \in \mathbb{R}^d$, $\norm{\boldsymbol{\theta}^*}_2\leq 1$. 

At each round $t=1,2,...,T$, a subset of arms $\mathcal{A}_t \subseteq \mathcal{A}$ are available to the agent to choose from. Based on historical interactions, the agent selects an arm $a_t\in \mathcal{A}_t$, and receives a corresponding reward $r_{a_t,t}\in[0,1]$. The reward is assumed to be a linear function of the contextual vectors 
\begin{equation}
r_{a_t,t}=\boldsymbol{x}_{a_t}^{\top}\boldsymbol{\theta}^*+\epsilon_t\,,
\label{equation1}
\end{equation}
where $\epsilon_t$ is 1-sub-Gaussian random noise with zero mean.

Let $a_t^*\in{\arg\max}_{a\in\mathcal{A}_t}\boldsymbol{x}_{a}^{\top}\boldsymbol{\theta}^*$ denote an optimal arm with the largest expected reward at $t$. The learning objective is to minimize the cumulative regret
\begin{equation}
R(T)=\sum_{t=1}^{T}\boldsymbol{x}_{a_t^*}^{\top}\boldsymbol{\theta}^*-\sum_{t=1}^{T}\boldsymbol{x}_{a_t}^{\top}\boldsymbol{\theta}^*.
\label{equation2}
\end{equation}

The agent can also occasionally query the user's feedback on some conversational key-terms to help elicit user preferences. In particular, a ``key-term" is a keyword or topic related to a subset of arms. For example, the key-term \textit{sports} is related to the arms like basketball, football, swimming, etc.

Suppose there is a finite set $\cK$ of key-terms. The relationship between arms and key-terms is given by a weighted bipartite graph $(\mathcal{A},\mathcal{K},\boldsymbol{W})$, where $\boldsymbol{W}\triangleq\left[w_{a,k}\right]_{a\in\mathcal{A},k\in\mathcal{K}}$ represents the relationship between arms and key-terms, i.e., a key-term $k\in \cK$ is associated to an arm $a\in\cA$ with weight $w_{a,k}\ge 0$. We assume that each key-term $k$ has positive weights with some related arms (i.e., $\sum_{a\in\mathcal{A}}
w_{a,k}>0$, $\forall{k\in\mathcal{K}}$), and the weights associated with each arm sum up to 1, i.e., $\sum_{k\in\mathcal{K}} w_{a,k}=1$, $a\in \mathcal{A}$. The feature vector of a key-term $k$ is given by $\tilde{\boldsymbol{x}}_{k}=\sum_{a\in\mathcal{A}}\frac{w_{a,k}}{\sum_{a^{\prime}\in\mathcal{A}}w_{a^{\prime},k}}\boldsymbol{x}_a$. The key-term-level feedback on the key-term $k$ at $t$ is defined as
\begin{equation}
    \tilde{r}_{k,t}=\tilde{\boldsymbol{x}}_{k}^{\top}\boldsymbol{\theta}^*+\tilde{\epsilon}_t\,,
    \label{equation3}
\end{equation}
where $\tilde{\epsilon}_t$ is assumed to be 1-sub-Gaussian random noise. One thing to stress is that in the previous works \cite{zhang2020conversational,wu2021clustering,xie2021comparison,zhao2022knowledge}, the \textit{unknown} user preference vector $\boldsymbol{\theta}^*$ is essentially assumed to be the same at both the arm level and the key-term level.

To avoid affecting the user experience, the agent should not conduct conversations too frequently. Following \cite{zhang2020conversational}, we define a function $b: \mathbb{N}_+ \rightarrow \mathbb{R}_+$, where $b(t)$ is increasing in $t$, to control the conversation frequency of the agent. At each round $t$, if $b(t)-b(t-1)>0$, the agent is allowed to conduct $q(t)=\lfloor b(t)-b(t-1)\rfloor$ conversations by asking for user's feedback on $q(t)$ key-terms. Using this modeling arrangement, the agent will have $b(t)$ conversational interactions with the user up
to round $t$.

\begin{algorithm}[t]
    \caption{General ConLinUCB framework}
    \label{algorithm1}
    \LinesNumbered
    \KwIn {graph$(\mathcal{A},\mathcal{K},\boldsymbol{W})$, conversation frequency function $b(t)$, key-term selection strategy $\boldsymbol{\pi}$.}
    \textbf{Initialization}: $\boldsymbol{M}_0=\beta\boldsymbol{I}$, $\boldsymbol{b}_0=\boldsymbol{0}$.\\
    \For{t = 1 to T}{   
        \eIf{$b(t)-b(t-1)>0$} {
            $q(t)=\lfloor b(t)-b(t-1)\rfloor$;\\
            \While{$q(t)>0$}{
            Select a key-term $k\in \mathcal{K}$ using the specified key-term selection strategy $\boldsymbol{\pi}$ (e.g., Eq. (\ref{BS}) for ConLinUCB-BS and Eq. (\ref{MCR}) for ConLinUCB-MCR), and query the user's preference over it;\\
            Receive the user's feedback $\tilde{r}_{k,t}$;\\
            $\boldsymbol{M}_t=\boldsymbol{M}_{t-1}+\tilde{\boldsymbol{x}}_{k}\tilde{\boldsymbol{x}}^{\top}_{k}$;\\
            $\boldsymbol{b}_t=\boldsymbol{b}_{t-1}+\tilde{\boldsymbol{x}}_{k}\tilde{r}_{k,t}$;\\
            $q(t)\mathrel{-}=1$;
            }
        } {
            $\boldsymbol{M}_t=\boldsymbol{M}_{t-1}$,
            $\boldsymbol{b}_t=\boldsymbol{b}_{t-1}$;
        }
$\boldsymbol{\theta}_t=\boldsymbol{M}_t^{-1}\boldsymbol{b}_t$;\\
Select $a_t = \mathop{\arg\max}\limits_{a \in \mathcal{A}_t} \boldsymbol{x}^{\top}_{a} \boldsymbol{\theta}_t +\alpha_t\norm{\boldsymbol{x}_{a}}_{\boldsymbol{M}_t^{-1}}$;\\
Ask the user's preference on arm $a_t$ and receive the reward $r_{a_t,t}$
;\\
$\boldsymbol{M}_t=\boldsymbol{M}_{t-1}+\boldsymbol{x}_{a_t}\boldsymbol{x}^{\top}_{a_t}$;\\
$\boldsymbol{b}_t=\boldsymbol{b}_{t-1}+\boldsymbol{x}_{a_t}r_{a_t,t}$;\\
    }
\end{algorithm}
\section{Algorithms and Theoretical Analysis} \label{section3}

This section first introduces ConLinUCB, a framework for conversational bandits with better information incorporation, which is general for ``any" key-term selection strategies. Based on ConLinUCB, we further propose two bandit algorithms, ConLinUCB-BS and ConLinUCB-MCR, with explorative key-term selection strategies.

To simplify the exposition, we merge the ConLinUCB framework, ConLinUCB-BS and ConLinUCB-MCR in Algorithm \ref{algorithm1}. We also theoretically give regret bounds of our proposed algorithms.

\subsection{General ConLinUCB Algorithm Framework}
In conversational bandits, it is common that the \textit{unknown} preference vector $\boldsymbol{\theta}^*$ is essentially assumed to be the same at both arm level and key-term level  \cite{zhang2020conversational,xie2021comparison,wu2021clustering}. However, all existing works treat $\boldsymbol{\theta}^*$ differently at these two levels.
Specifically, they take two different steps
to estimate user preference vectors at the arm level and key-term
level, and use a discounting parameter $\lambda\in(0,1)$ to balance learning from these two levels' interactions. In this manner, the contributions of the arm-level and key-term-level information to the convergence of estimation are discounted by $\lambda$ and $1-\lambda$, respectively. Therefore, such discounting will cause waste of observations, indicating that information at these two levels can not be fully leveraged to accelerate the learning process.

To handle the above issues, we propose a general framework called ConLinUCB, for conversational contextual
bandits. In this new framework, in order to fully leverage interactive information from two levels, we simultaneously estimate the user preference vector by solving \textit{one
single optimization problem} that minimizes the mean squared error of both arm-level estimated rewards and key-term-level estimated feedback. Specifically, in ConLinUCB, at round $t$, the user preference vector is estimated by solving the following linear regression
   
% {\small     \begin{equation}
%             \boldsymbol{\theta}_t =\mathop{\arg\min}\limits_{\boldsymbol{\theta}\in\mathbb{R}^d} \sum_{\tau=1}^{t-1} (\boldsymbol{x}^{\top}_{a_\tau}\boldsymbol{\theta}-r_{a_\tau,\tau})^2+\sum_{\tau=1}^{t} \sum_{k\in \mathcal{K}_\tau} ( \boldsymbol{\tilde{x}}^{\top}_{k}\boldsymbol{\theta} -\tilde{r}_{k,\tau})^2 + \beta\norm{\boldsymbol{\theta}}_2^2\,,
%         \label{equation4}
%         \end{equation}}
        \begin{align}
            \boldsymbol{\theta}_t &=\mathop{\arg\min}\limits_{\boldsymbol{\theta}\in\mathbb{R}^d} \sum_{\tau=1}^{t-1} (\boldsymbol{x}^{\top}_{a_\tau}\boldsymbol{\theta}-r_{a_\tau,\tau})^2+\sum_{\tau=1}^{t} \sum_{k\in \mathcal{K}_\tau} ( \boldsymbol{\tilde{x}}^{\top}_{k}\boldsymbol{\theta} -\tilde{r}_{k,\tau})^2\nonumber\\\
            &\quad\quad+ \beta\norm{\boldsymbol{\theta}}_2^2\,,
        \label{equation4}
        \end{align}
where $\mathcal{K}_\tau$ denotes the set of key-terms asked at round $\tau$, and the coefficient $\beta>0$ controls regularization. The closed-form solution of this optimization problem is
    \begin{equation}
        \boldsymbol{\theta}_t=\boldsymbol{M}_t^{-1}\boldsymbol{b}_t\,,
    \label{equation5}
    \end{equation}
	where\begin{align}\begin{aligned}
        \boldsymbol{M}_t&=\sum_{\tau=1}^{t-1} \boldsymbol{x}_{a_\tau}\boldsymbol{x}^{\top}_{a_\tau}+\sum_{\tau=1}^{t} \sum_{k\in \mathcal{K}_\tau} \boldsymbol{\tilde{x}}_{k} \boldsymbol{\tilde{x}}^{\top}_{k}+\beta\boldsymbol{I}\,, \\
        		\boldsymbol{b}_t&=\sum_{\tau=1}^{t-1} \boldsymbol{x}_{a_\tau} r_{a_\tau,\tau}+\sum_{\tau=1}^{t} \sum_{k\in \mathcal{K}_\tau} \boldsymbol{\tilde{x}}_{k} \tilde{r}_{k,\tau}.
    \label{equation6}
    \end{aligned}\end{align}

To balance exploration and exploitation, ConLinUCB selects arms using the upper confidence bound (UCB) strategy
        \begin{equation}
                a_t = \mathop{\arg\max}\limits_{a \in \mathcal{A}_t} \underbrace{\boldsymbol{x}^{\top}_{a} \boldsymbol{\theta}_t}_{\hat{R}_{a,t}} +\underbrace{\alpha_t\norm{\boldsymbol{x}_{a}}_{\boldsymbol{M}_t^{-1}}}_{C_{a,t}}\,,
            \label{equation7}
        \end{equation}
        where $\norm{\boldsymbol{x}}_{\boldsymbol{M}}=\sqrt{\boldsymbol{x}^{\top}\boldsymbol{M}\boldsymbol{x}}$, $\hat{R}_{a,t}$ and $C_{a,t}$ denote the estimated reward and  confidence radius of arm $a$ at round $t$, and 
         \begin{equation}
        \alpha_t=\sqrt{2\log(\frac{1}{\delta})+d\log(1+\frac{t+b(t)}{\beta d})} +\sqrt{\beta}\,,
        \label{equation8}
        \end{equation}
        which comes from the following Lemma \ref{lemma2}.

The ConLinUCB algorithm framework is shown in Alg. \ref{algorithm1}. The key-term-level interactions take place in line 3-14. At round $t$, the agent first determines whether conversations are allowed using $b(t)$. When conducting conversations, the agent asks for the user's feedback on $q(t)$ key-terms and uses the feedback to update the parameters. Line 15-20 summarise the arm-level interactions. Based on historical interactions, the agent calculates the estimated $\boldsymbol{\theta}^*$, selects an arm with the largest UCB index, receives the corresponding reward, and updates the parameters accordingly. ConLinUCB only maintains one set of covariance matrix $\boldsymbol{M}_t$ and regressand vector $\boldsymbol{b}_t$, containing the feedback from both arm-level and key-term-level interactions. By doing so, ConLinUCB better leverages the feedback information than ConUCB. Note that ConLinUCB is a general framework with the specified key-term selection strategy $\boldsymbol{\pi}$ to be determined.

\subsection{ConLinUCB with key-terms from Barycentric Spanner (ConLinUCB-BS)}\label{section3.2}
Based on the ConLinUCB framework, we propose the ConLinUCB-BS algorithm with an explorative key-term selection strategy. Specifically, 
ConLinUCB-BS selects key-terms from the barycentric spanner $\mathcal{B}$ of the key-term set $\mathcal{K}$, which is an efficient exploration basis in online learning \cite{amballa2021computing}, to conduct explorative conversations. Below is the formal definition of the barycentric spanner for the key-term set $\mathcal{K}$.
\begin{definition}[Barycentric Spanner of $\mathcal{K}$]\label{def1}
$\mathcal{B}=\{k_1,k_2,...,k_d\}
\subseteq\mathcal{K}$ is a barycentric spanner for $\mathcal{K}$ if for any $k\in\mathcal{K}$, there exists a set of coefficients $\boldsymbol{c}\in[-1,1]^d$, such that $\tilde{\boldsymbol{x}}_k=\sum_{i=1}^d\boldsymbol{c}_i\tilde{\boldsymbol{x}}_{k_i}$.
\end{definition}

We assume that the key-term set $\mathcal{K}$ is finite and $\{\tilde{\bx}_{k}\}_{k \in \cK}$ span $\mathbb{R}^d$, thus the existence of a barycentric spanner $\mathcal{B}$ of $\mathcal{K}$ is guaranteed \cite{awerbuch2008online}. 

Corresponding vectors in the barycentric spanner are linearly independent. By choosing key-terms from the barycentric spanner, we can quickly explore the \textit{unknown} user preference vector $\boldsymbol{\theta}^*$ in various directions. Based on this reasoning, whenever a conversation is allowed, ConLinUCB-BS selects a key-term
\begin{equation}
\label{BS}
    {k}\sim \text{unif}(\mathcal{B}),
\end{equation}
which means sampling a key-term $k$ uniformly at random from the barycentric spanner $\mathcal{B}$ of $\mathcal{K}$. ConLinUCB-BS is completed using the above strategy as $\boldsymbol{\pi}$ in the ConLinUCB framework (Alg. \ref{algorithm1}). As shown in the following Lemma \ref{lemma2} and Lemma \ref{lemma3}, in ConLinUCB-BS, the statistical estimation uncertainty shrinks quickly. Additionally, since the barycentric spanner $\mathcal{B}$ of the key-term set $\mathcal{K}$ can be precomputed offline, ConLinUCB-BS is computationally efficient, which is vital for real-time recommendations.
% \begin{algorithm}[t]
%     \caption{Key-term selection strategy for ConLinUCB-BS}
%     \label{algorithm2}
%     \LinesNumbered
%     \KwIn {precomputed barycentric spanner $\mathcal{B}$ of the given key-term set $\mathcal{K}$.} 
%     Whenever a conversation is allowed:\\
%     \quad Pick a key-term $k$ uniformly at random from $\mathcal{B}$.
% \end{algorithm}

\subsection{ConLinUCB with key-terms having Max Confidence Radius (ConLinUCB-MCR)}\label{section3.3}
  We can further improve ConLinUCB-BS in the following aspects. First, ConLinUCB-BS does not apply in a more general setting where the key-term set $\mathcal{K}$ varies over time since it needs a precomputed barycentric spanner $\mathcal{B}$ of $\mathcal{K}$. Second, as the selection of key-terms is independent of past observations, ConLinUCB-BS does not fully leverage the historical information. For example, suppose the agent is already certain about whether the user favors \textit{sports} based on previous interactions. In that case, it does not need to ask for the user's feedback on the key-term \textit{sports} anymore. To address these issues, we propose the ConLinUCB-MCR algorithm that (i) is applicable when the key-term set $\cK$ varies with $t$ and (ii) can adaptively conduct explorative conversations based on historical interactions. 

In multi-armed bandits, confidence radius is used to capture whether an arm has been well explored in the interactive history, and it will shrink whenever the arm is selected. Motivated by this, if a key-term has a large confidence radius, it means the system has not sufficiently explored the user's preferences in its related items, indicating that this key-term is explorative. 
Based on this reasoning, ConLinUCB-MCR selects key-terms with maximal confidence radius to conduct explorative conversations apdaptively. Specifically, when a conversation is allowed at $t$, ConLinUCB-MCR chooses a key-term as follow
\begin{equation}
    \label{MCR}
    k\in\mathop{\arg\max}\limits_{k \in \mathcal{K}_t}\alpha_t\norm{\boldsymbol{\tilde{x}}_{k}}_{\boldsymbol{M}_t^{-1}}\,,
\end{equation}
where $\alpha_t$ is defined in Eq. (\ref{equation8}) and $\mathcal{K}_t\subseteq\mathcal{K}$ denotes the possibly time-varying  key-terms set available at round $t$. ConLinUCB-MCR is completed using the above strategy (Eq. (\ref{MCR})) as $\boldsymbol{\pi}$ in ConLinUCB (Alg. \ref{algorithm1}).
% By choosing key-terms with maximal confidence radius, ConLinUCB-MCR can adaptively ask for the user's feedback on key-terms whose related areas have been explored the least so far, thus can leverage interactive information to guide the choice of key-terms for better user preference exploration.

\subsection{Theoretical Analysis}
We give upper bounds of the regret for our algorithms. As a convention, the conversation frequency satisfies $b(t)\leq t$, so we assume $b(t)= b\cdot t$, $b\in(0,1)$. We leave the proofs of Lemma \ref{lemma2}-\ref{lemma3} and Theorem \ref{theorem1}-\ref{theorem4} to the Appendix due to the space limitation.

The following lemma shows a high probability upper bound of the difference between $\boldsymbol{\theta}_t$ and $\boldsymbol{\theta}^*$ in the direction of the action vector $\boldsymbol{x}_{a}$ for algorithms based on ConLinUCB.
\begin{figure*}[htbp]
	\centering
	\begin{subfigure}
		\centering
		\includegraphics[width=0.75\linewidth]{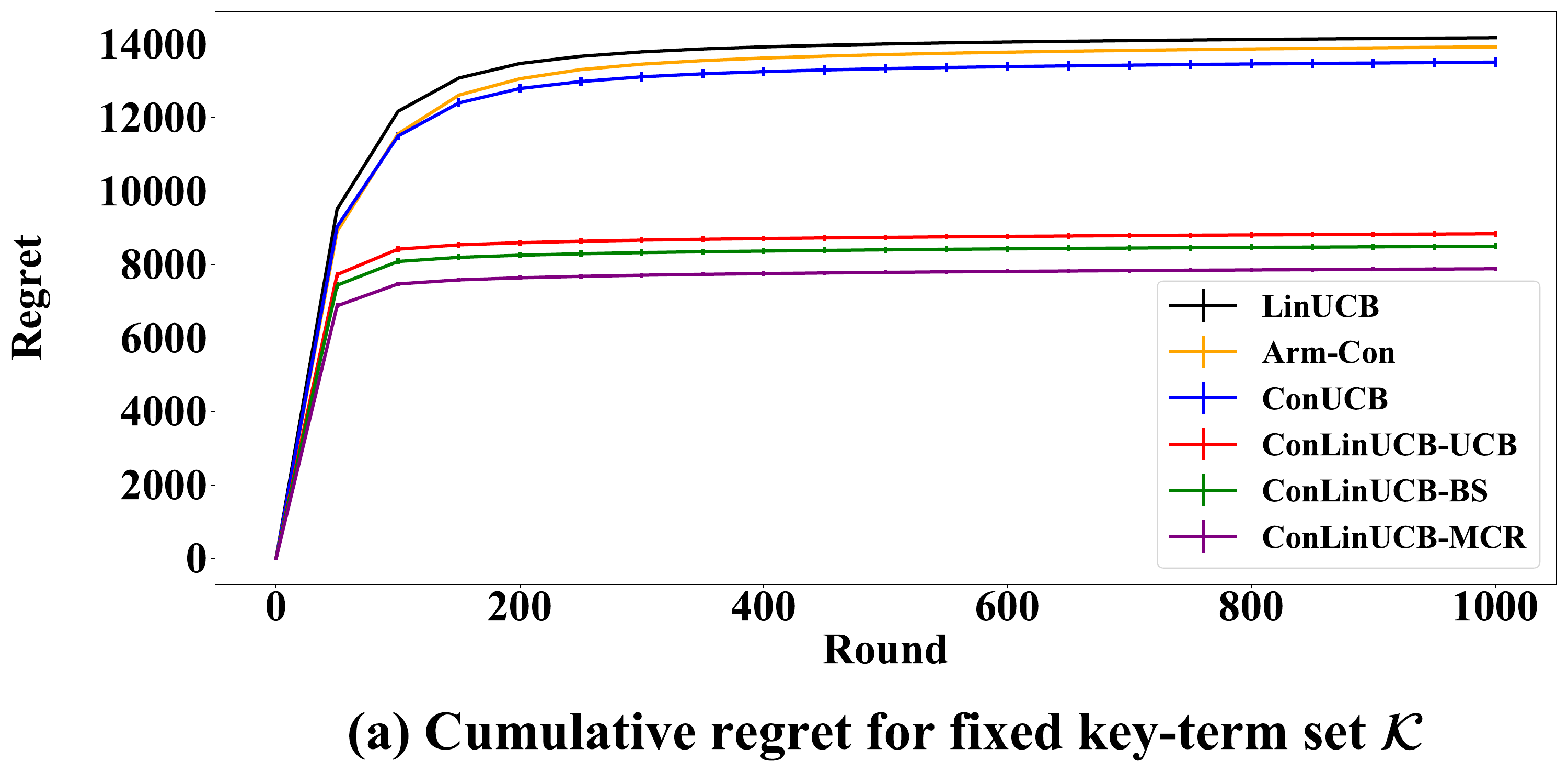}

	\end{subfigure}
	\centering
	\begin{subfigure}
		\centering
		\includegraphics[width=0.75\linewidth]{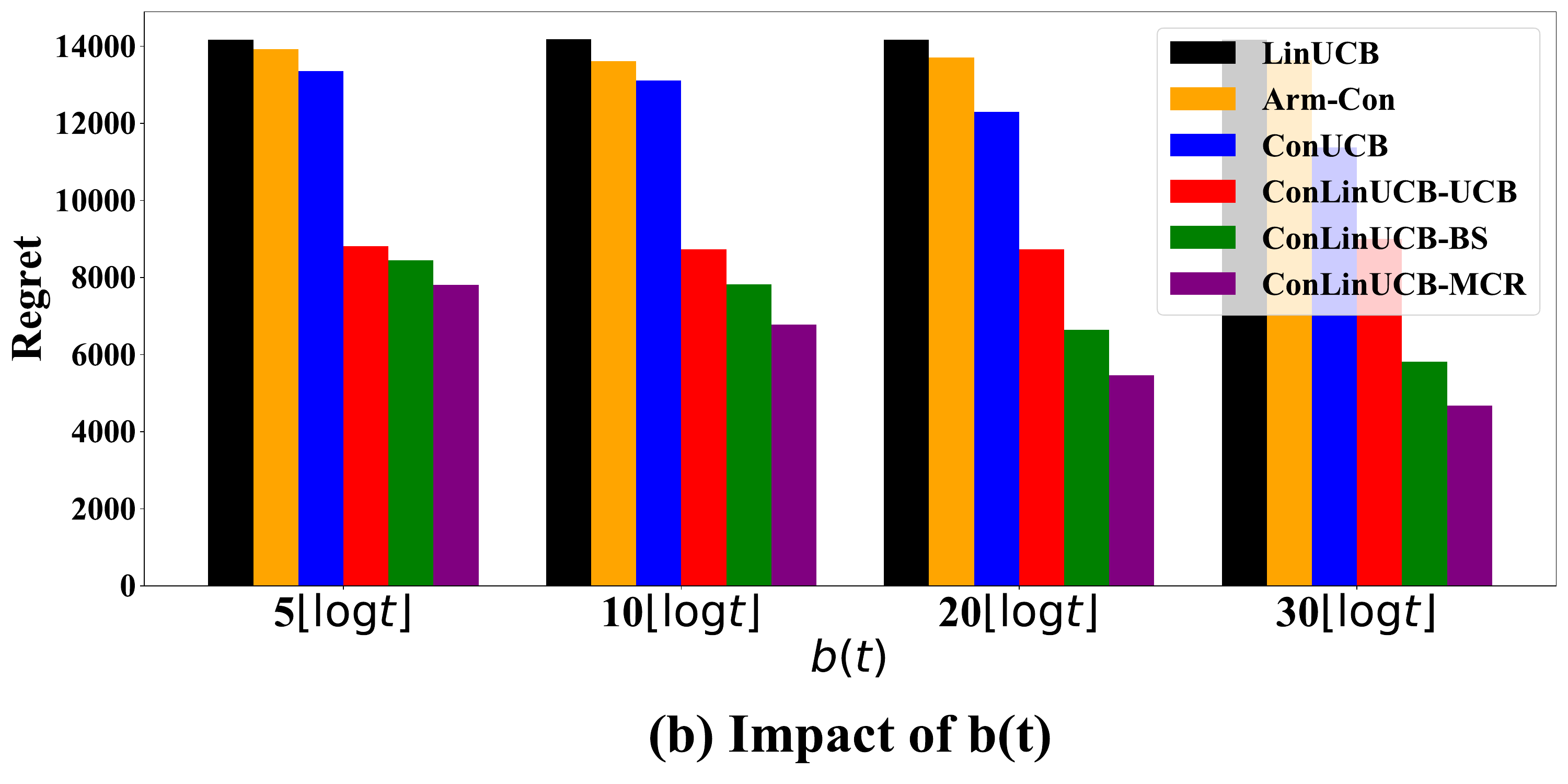}
% 		\caption{Effect of $b(t)$}
        % \caption{}
% 		\label{fig2}%文中引用该图片代号
	\end{subfigure}

	\centering
	\begin{subfigure}
		\centering
		\includegraphics[width=0.75\linewidth]{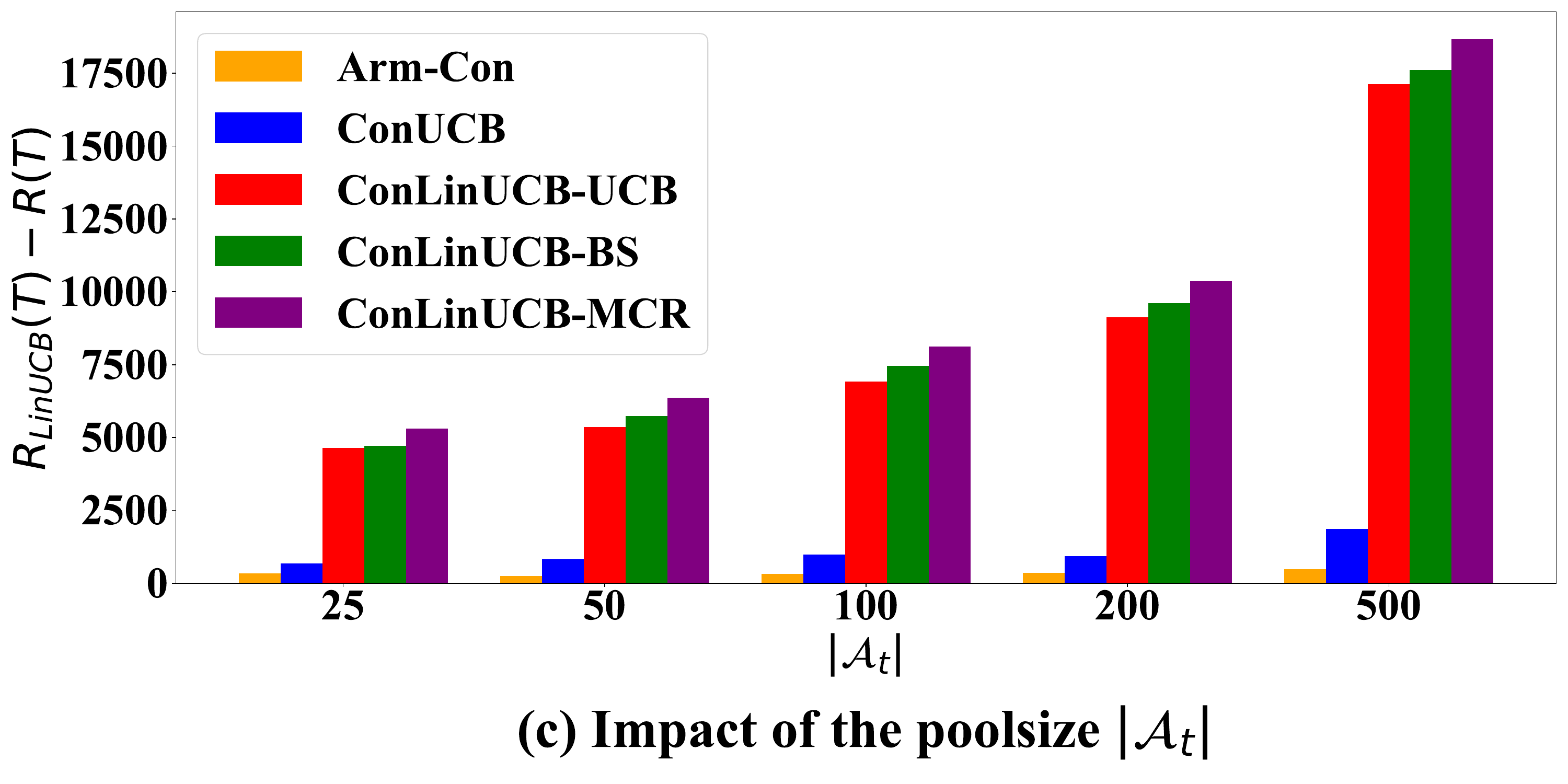}
% 		\caption{Effect of the poolsize $|\mathcal{A}_t|$}
        % \caption{}
% 		\label{fig3}%文中引用该图片代号
	\end{subfigure}
		\begin{subfigure}
		\centering
		\includegraphics[width=0.75\linewidth]{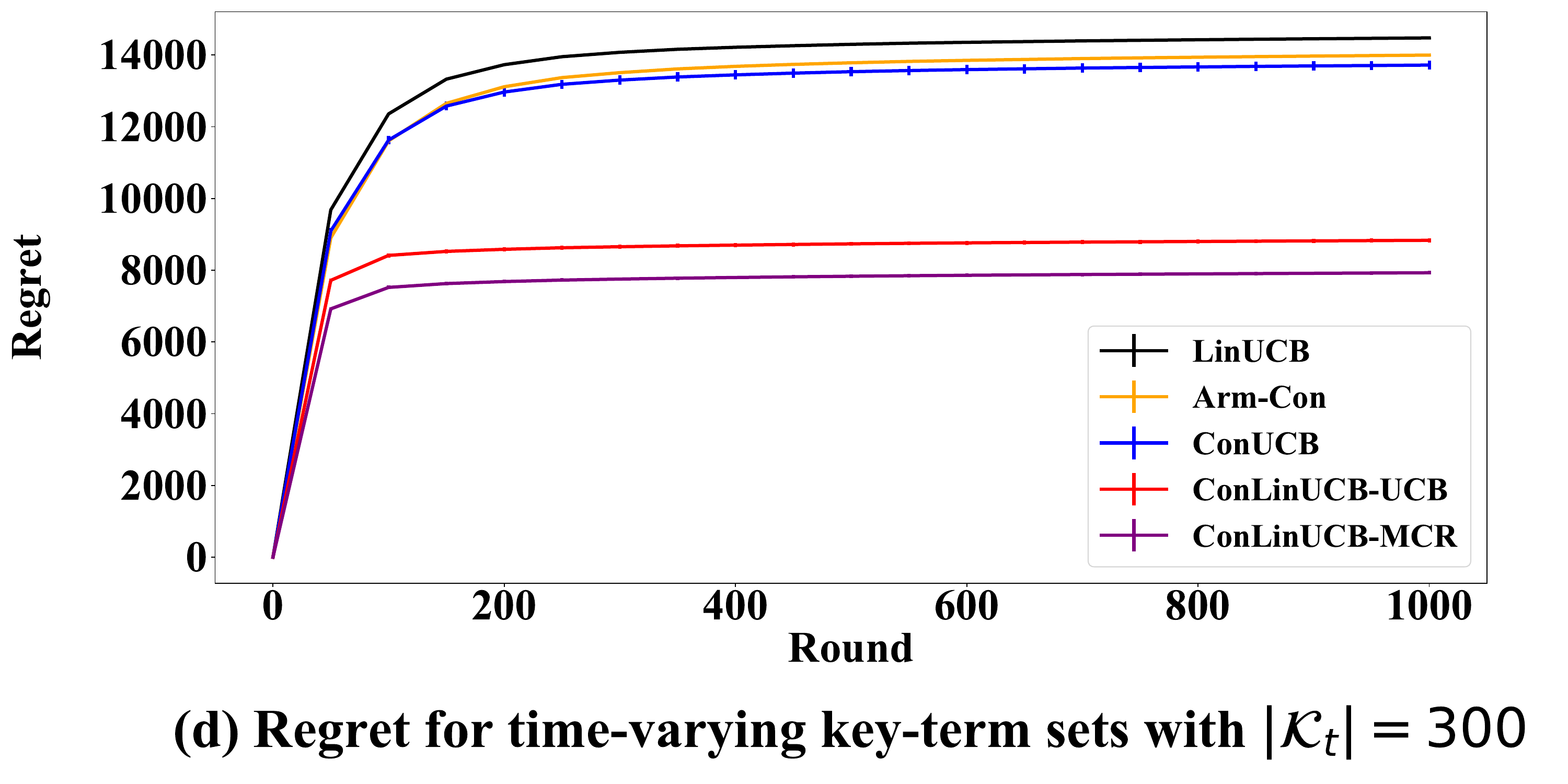}
% 		\caption{Cumulative regret for the case $|\mathcal{K}_t|=300$}
        % \caption{}
% 		\label{fig4}%文中引用该图片代号
	\end{subfigure}
	\caption{Experimental results on synthetic dataset} 
	\label{fig1}
\end{figure*}
\begin{lemma}\label{lemma2}
At ${\forall}t$, for any $a\in \mathcal{A}$, with probability at least $1-\delta$ for some $\delta \in (0,1)$
$$\left|\boldsymbol{x}_{a}^{\top}(\boldsymbol{\theta}_t-\boldsymbol{\theta}^*)\right|\leq\alpha_t\norm{\boldsymbol{x}_{a}}_{\boldsymbol{M}_t^{-1}}=C_{a,t},$$
where $\alpha_t=\sqrt{2\log(\frac{1}{\delta})+d\log(1+\frac{t+b(t)}{\beta d})} +\sqrt{\beta}$.
\end{lemma}

For a  barycentric spanner $\mathcal{B}$ of the key-term set $\mathcal{K}$, let
\begin{equation}
        \lambda_{\mathcal{B}}\coloneqq\lambda_{\min}(\boldsymbol{E}_{k\sim \text{unif}(\mathcal{B})}[\tilde{\boldsymbol{x}}_k\tilde{\boldsymbol{x}}_k^{\top}])>0\,,
        \label{equation11}
\end{equation}
where $\lambda_{\min}(\cdot)$ denotes the minimum eigenvalue of the augment. 
We can get the following Lemma that gives a high probability upper bound of $\norm{\boldsymbol{x}_{a}}_{\boldsymbol{M}_t^{-1}}$ for ConLinUCB-BS.
\begin{lemma}\label{lemma3}

For ConLinUCB-BS, ${\forall}a\in \mathcal{A}$, at ${\forall}t\geq t_0=\frac{256}{b\lambda_{\mathcal{B}}^2}\log(\frac{128d}{\lambda_{\mathcal{B}}^2\delta})$, with probability at least $1-\delta$ for $\delta \in (0,\frac{1}{8}]$
$$\norm{\boldsymbol{x}_{a}}_{\boldsymbol{M}_t^{-1}}\leq\sqrt{\frac{2}{\lambda_{\mathcal{B}} bt}}\,.$$
\end{lemma}

The following theorem gives a high probability regret upper bound of our ConLinUCB-BS.

\begin{theorem}\label{theorem1}
With probability at least $1-\delta$ for some $\delta \in (0,\frac{1}{4}]$, the regret $R(T)$ of ConLinUCB-BS satisfies
% \begin{equation*}
%     \begin{aligned}
%     R(T)&\leq \frac{256}{b\lambda_{\mathcal{B}}^2}\log(\frac{256d}{\lambda_{\mathcal{B}}^2\delta})+1\\
%     &+4\sqrt{\frac{2}{b\lambda_{\mathcal{B}}}}\sqrt{T}\Bigg(\sqrt{2\log(\frac{2}{\delta})+d\log(1+\frac{(1+b)T}{\beta d})}\\
%     &\quad+\sqrt{\beta}\norm{\boldsymbol{\theta}^*}_2\Bigg).
%     \end{aligned}
% \end{equation*}
\begin{equation*}
    \begin{aligned}
    R(T)
    &\leq4\sqrt{\frac{2}{b\lambda_{\mathcal{B}}}}\sqrt{T}\Bigg(\sqrt{2\log(\frac{2}{\delta})+d\log(1+\frac{(1+b)T}{\beta d})}\\
    &\quad+\sqrt{\beta}\Bigg)+ \frac{256}{b\lambda_{\mathcal{B}}^2}\log(\frac{256d}{\lambda_{\mathcal{B}}^2\delta})+1\,.
    \end{aligned}
\end{equation*}
\end{theorem}

Recall that the regret upper bound of ConUCB \cite{zhang2020conversational} is
\begin{equation*}
    \begin{aligned}
    R(T) &\leq 2\sqrt{2Td\log(1+\frac{\lambda(T+1)}{(1-\lambda)d})}\Bigg(\sqrt{\frac{1-\lambda}{\lambda}}\\
        &\quad\ +\sqrt{\frac{1-\lambda}{\lambda\beta}}\sqrt{2\log(\frac{2}{\delta})+d\log(1+\frac{bT}{\beta d})}\\
        &\quad\ +\sqrt{2\log(\frac{2}{\delta})+d\log(1+\frac{\lambda T}{(1-\lambda) d})}\Bigg)\,,
    \end{aligned}
\end{equation*}
which is of $O(d\sqrt{T}\log T)$. The regret bound of ConLinUCB-BS given in Theorem \ref{theorem1} is of $O(d\sqrt{T\log T})$ (as  $\lambda_{\cB}$ is of order $O(\frac{1}{d})$), better than ConUCB by reducing a multiplicative $\sqrt{\log T}$ term.

Next, the following theorem gives a high-probability regret upper bound of ConLinUCB-MCR.
\begin{theorem}\label{theorem4}
With probability at least $1-\delta$ for some $\delta \in (0,1)$, the regret $R(T)$ of ConLinUCB-MCR satisfies
\begin{equation*}
    \begin{aligned}
    R(T)&\leq2\sqrt{2Td\log(1+\frac{T+1}{\beta d})}\\
    &\times\Bigg(\sqrt{\beta}+\sqrt{2\log(\frac{1}{\delta})+d\log(1+\frac{(b+1)T}{\beta d})}\Bigg)\,.
    \end{aligned}
\end{equation*}
\end{theorem}

Note that the regret upper bound of ConLinUCB-MCR is smaller than ConUCB by reducing some additive terms.

\section{Experiments on Synthetic Dataset} \label{section5}
In this section, we show the experimental results on synthetic data. To obtain the offline-precomputed barycentric spanner $\cB$, we use the method proposed in \cite{awerbuch2008online}.
\begin{figure*}[htbp]
	\centering
	\begin{subfigure}
		\centering
		\includegraphics[width=0.8\linewidth]{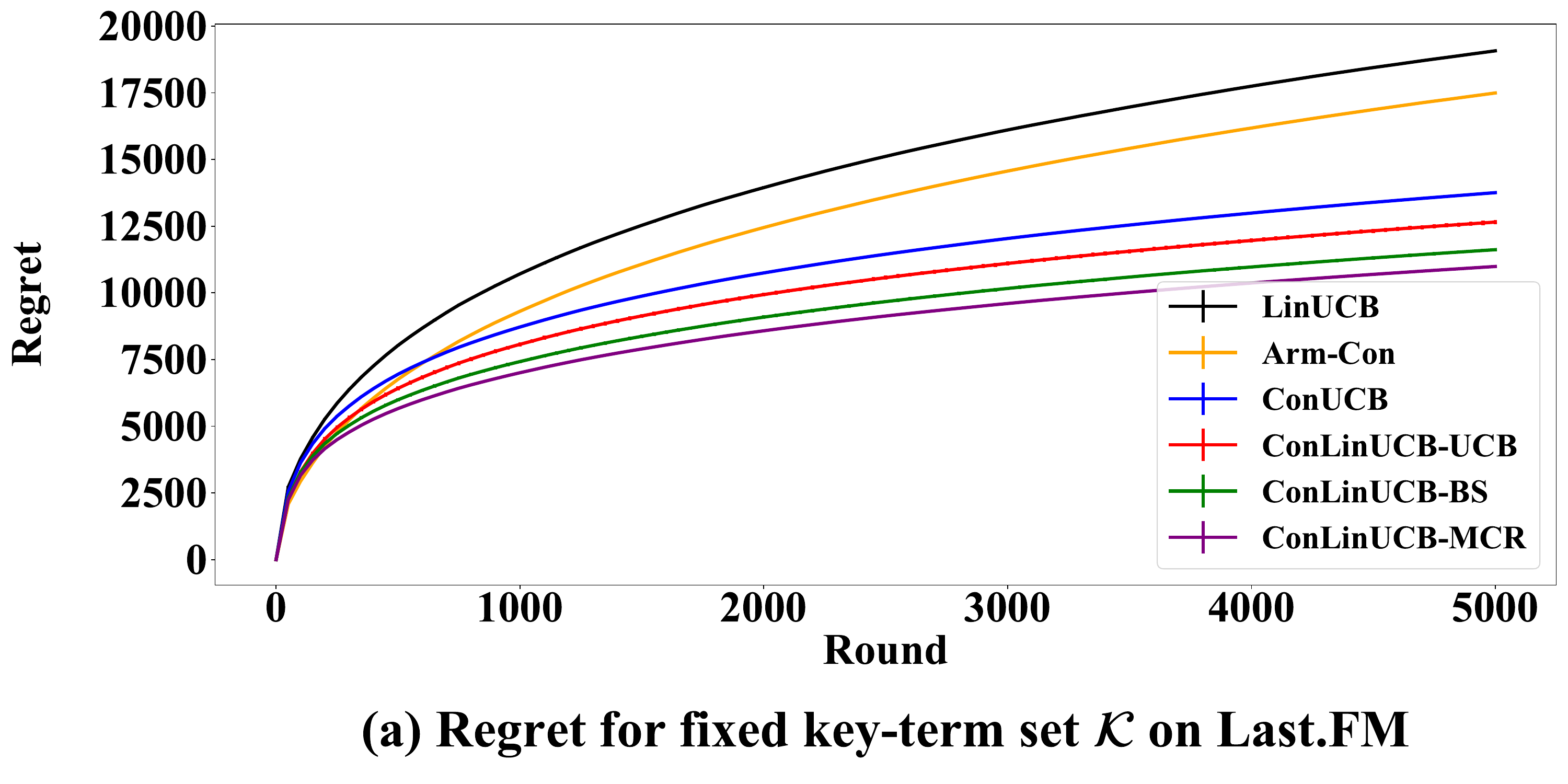}
% 		\caption{Cumulative regret for fixed key-term set $\mathcal{K}$ on Last.FM}
  		% \caption{Cumulative regret for fixed key-term set $\mathcal{K}$}
% 		\label{fig6}%文中引用该图片代号
	\end{subfigure}
	\centering
	\begin{subfigure}
		\centering
		\includegraphics[width=0.75\linewidth]{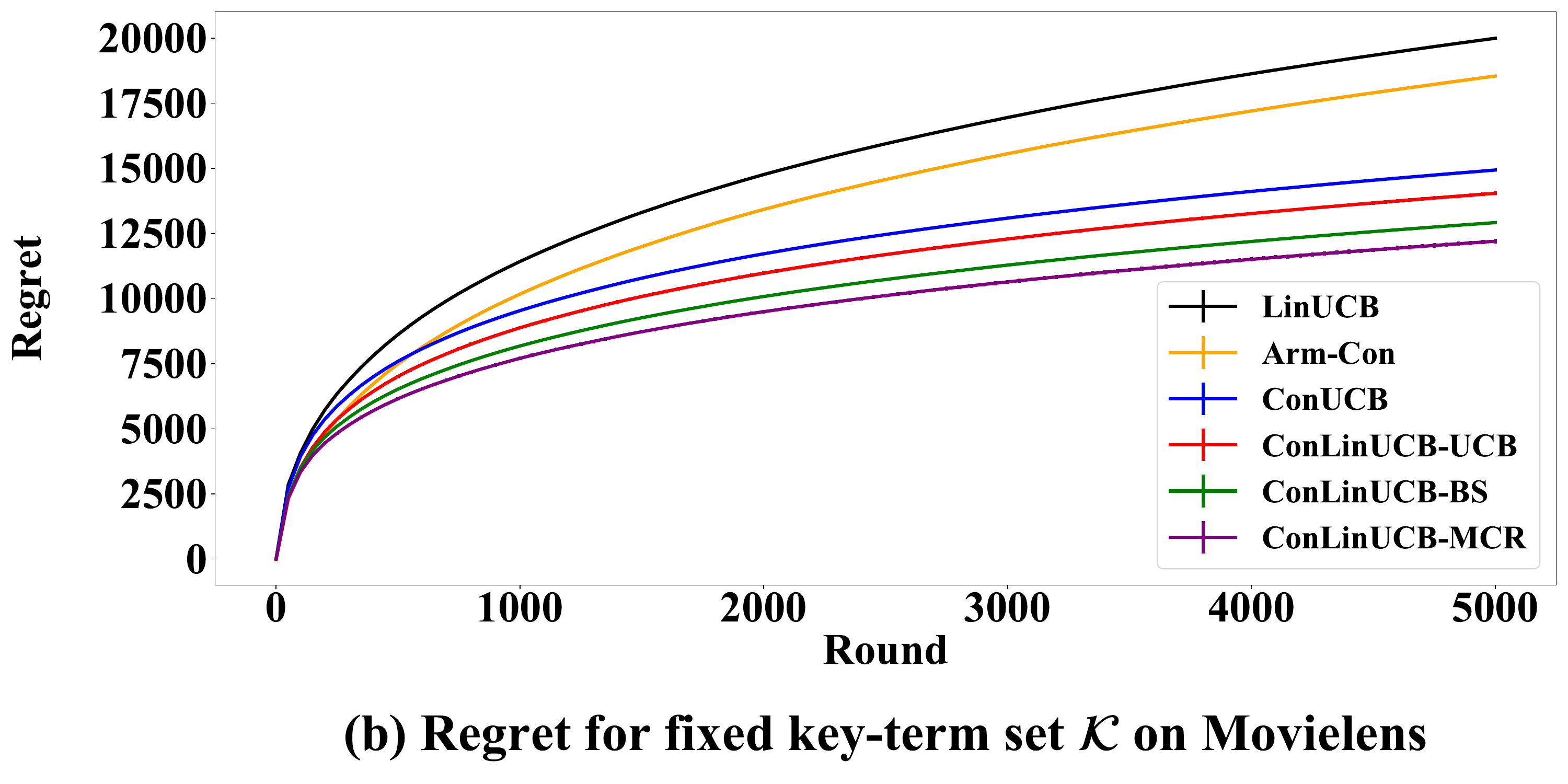}
% 		\caption{Cumulative regret for fixed key-term set $\mathcal{K}$ on Movielens}
  		% \caption{Cumulative regret for fixed key-term set $\mathcal{K}$ on Movielens}
% 		\label{fig7}%文中引用该图片代号
	\end{subfigure}

	\centering
	\begin{subfigure}
		\centering
		\includegraphics[width=0.75\linewidth]{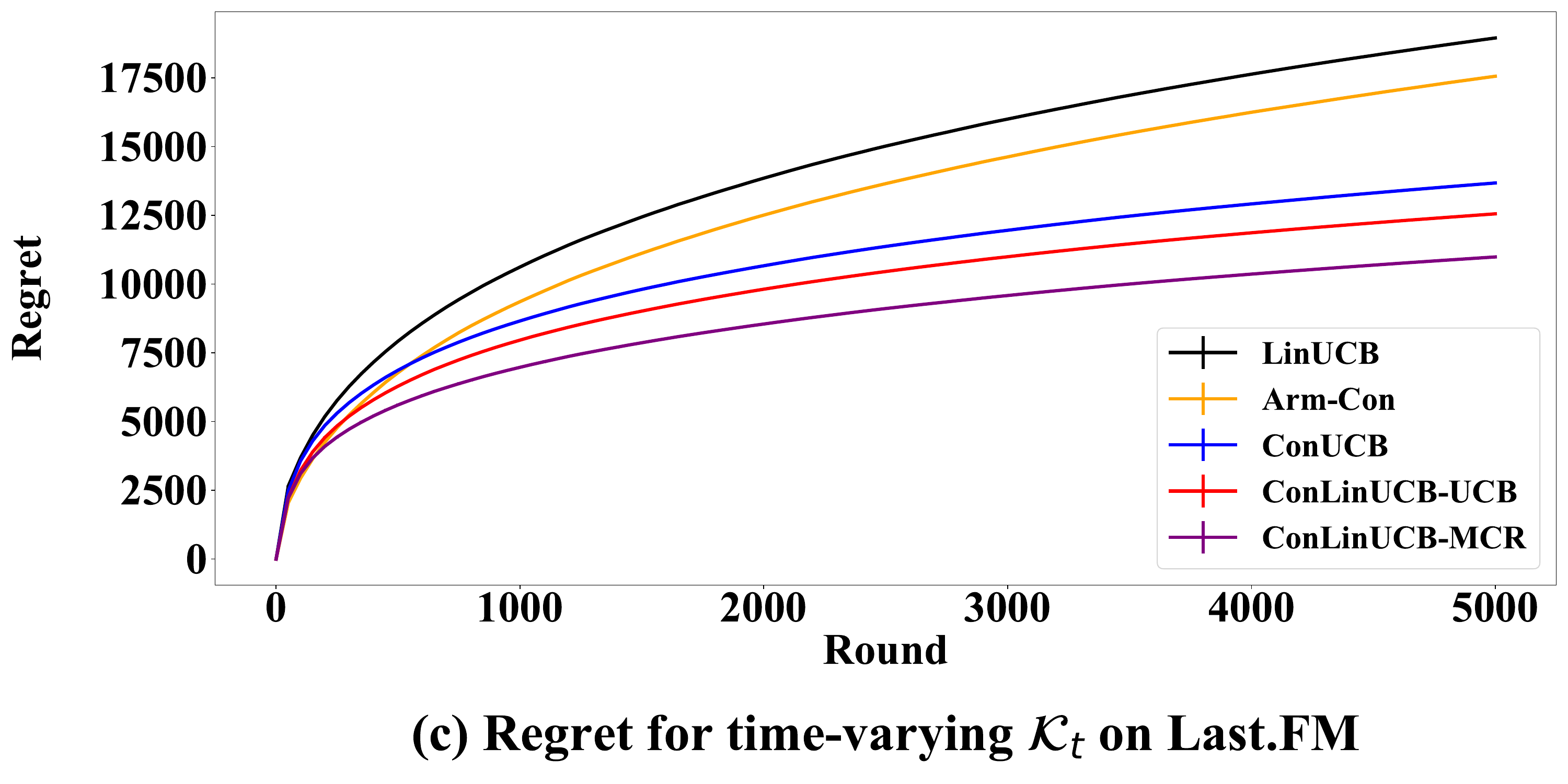}
% 		\caption{Cumulative regret for varying key-term set $\mathcal{K}_t=1000$ \\on Last.FM}
  		% \caption{Cumulative regret for the case when $\left|\mathcal{K}_t\right|=1000$ on Last.FM}
% 		\label{fig8}%文中引用该图片代号
	\end{subfigure}
		\begin{subfigure}
		\centering
		\includegraphics[width=0.75\linewidth]{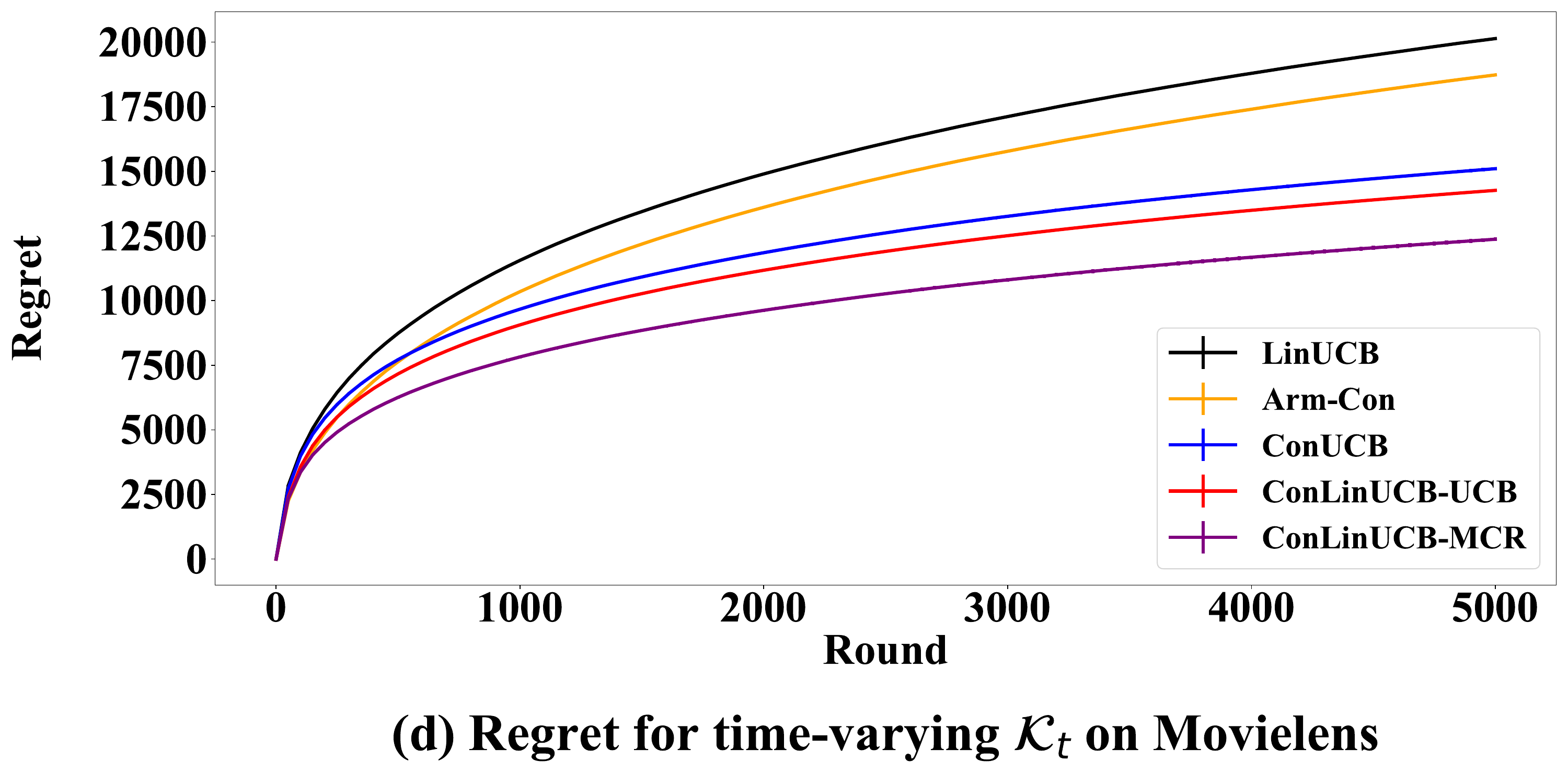}
%  		\caption{Cumulative regret for varying key-term set $\mathcal{K}_t=1000$ on Movielens}
  		% \caption{Cumulative regret for the case when $\left|\mathcal{K}_t\right|=1000$}
% 		\label{fig9}%文中引用该图片代号
	\end{subfigure}
	\caption{Experimental results on real-word datasets}
	\label{fig2}
\end{figure*}
\subsection{Experimental Settings}
\subsubsection{Generation of the synthetic dataset.}We create a set of arms $\mathcal{A}$ with $\left|\mathcal{A}\right|=5,000$ arms, and a set of key-terms $\mathcal{K}$ with $\left|\mathcal{K}\right|=500$. We set the dimension of the feature space to be $d=50$ and the number of users $N_u=200$.

% For each user preference vector $\boldsymbol{\theta}_u^*$, each entry is generated by independently drawing from the standard normal distribution $\mathcal{N}(0,1)$, and it is normalized such that $\norm{\boldsymbol{\theta}_u^*}_2=1$. Each dimension of the feature vector $\boldsymbol{x}_a$ of each arm $a$ is also independently drawn from $\mathcal{N}(0,1)$ and all the feature vectors are normalized such that $\norm{\boldsymbol{x}_a}_2=1$, $\forall a\in\mathcal{A}$. 
For each user preference vector $\boldsymbol{\theta}_u^*$ and each arm feature vector $\boldsymbol{x}_a$, each entry is generated by independently drawing from the standard normal distribution $\mathcal{N}(0,1)$, and all these vectors are normalized such that $\norm{\boldsymbol{\theta}_u^*}_2=1$, $\norm{\boldsymbol{x}_a}_2=1$.
The weight matrix $\boldsymbol{W}\triangleq\left[w_{a,k}\right]$ is generated as follows: First, for each key-term $k$, we select an integer $n_k\in[1,10]$ uniformly at random, then randomly select a subset of $n_k$ arms $\mathcal{A}_k$ to be the related arms for key-term $k$; second, for each arm $a$, if it is related to a set of $n_a$ key-terms $\mathcal{K}_a$, we assign equal weights $w_{a,k}=\frac{1}{n_a}$, $\forall{k\in\mathcal{K}_a}$. Following \cite{zhang2020conversational}, the feature vector for each key-term $k$ is computed using
$\tilde{\boldsymbol{x}}_{k}=\sum_{a\in\mathcal{A}}\frac{w_{a,k}}{\sum_{a^{\prime}\in\mathcal{A}}w_{a^{\prime},k}}\boldsymbol{x}_a$. The arm-level rewards and key-term-level feedback are generated following Eq. (\ref{equation1}) and Eq. (\ref{equation3}). 
\subsubsection{Baselines.}\label{section5.1.2}We compare our algorithms with the following baselines:
\begin{itemize}
    \item LinUCB \cite{li2010contextual}: A state-of-the-art contextual linear bandit algorithm that selects arms only based on the arm-level feedback without using conversational feedback.
    \item Arm-Con \cite{christakopoulou2016towards}: A conversational bandit algorithm that conducts conversations on arms without considering key-terms, and uses LinUCB for arm selection. 
    \item ConUCB \cite{zhang2020conversational}: The core conversational bandit algorithm that selects a key-term to minimize some estimation error whenever a conversation is allowed.
    \item ConLinUCB-UCB: An algorithm using a LinUCB-alike method as the key-term selection strategy in our proposed ConLinUCB framework, i.e., choose key-term $    k\in\mathop{\arg\max}\limits_{k \in \mathcal{K}_t}\boldsymbol{\tilde{x}}_{k}^{\top}\boldsymbol{\theta}_t+\alpha_t\norm{\boldsymbol{\tilde{x}}_{k}}_{\boldsymbol{M}_t^{-1}}$ at round $t$.
\end{itemize}

\subsection{Evaluation Results}
This section first shows the results when the key-term set $\mathcal{K}$ is fixed. In this case, we evaluate the regret $R(T)$ for all algorithms, and we study the impact of the conversation frequency function $b(t)$ and the number of arms $\left|\mathcal{A}_t\right|$ available at each round $t$. When $\mathcal{K}$ varies with time, ConLinUCB-BS does not apply, and we compare the regret of other algorithms. Following \cite{zhang2020conversational}, we set $T=1,000$, $b(t)=5\lfloor\log (t+1)\rfloor$ and $\left|\mathcal{A}_t\right|=50$, unless otherwise stated. 
% The reason we use $b=\Theta(\log t)$ instead of $b=\Theta(t)$ is to show our advantages over ConUCB in a setting that is more favorable to their algorithm. Specifically, in \cite{zhang2020conversational}, their ConUCB algorithm achieves better performance when $b(t)=\Theta(\log t)$, whereas theoretically our algorithms will perform worse with less conversations $b(t)=\Theta(\log t)$. We also conduct experiments over $b(t)=\Theta(t)$ and our algorithms still significantly outperform ConUCB.
% Note that in \cite{zhang2020conversational}, they show that experimentally ConUCB performs better when $b(t)$ is $O(\log t)$ than $O(t)$, whereas theoretically our algorithms are supposed to perform better when $b(t)=O(t)$. Therefore, we set $b(t)=O(\log t)$ to show our advantages over ConUCB in its more favorable situation. When $b(t)=O(t)$, our algorithms also significantly outperform ConUCB, which we do not show for ease of exposition.

\begin{table*}
    \centering
    \resizebox{1\columnwidth}{!}{
	\begin{tabular}{ | c | c | c | c |}
		\hline
		Alogrithm 	& Total time 	& Total time for selecting arms & Total time for selecting key-terms \\ 
		\hline
		ConUCB 	& 11,297		& 5,217		 & 6,080  \\ 
		\hline 
		ConLinUCB-UCB & 5,738  		& 3,060		 & 2,678  \\ 
		\hline
		ConLinUCB-MCR & 4,821 		& 3,030		 & 1,791  \\
		\hline
		ConLinUCB-BS &  3,127		& 3,120		 & 6  \\ 
		\hline
	\end{tabular}
        }
            \caption{Total runninng time (in seconds) of algorithms on Movielens with $T=5,000$.}\label{table1}
\end{table*}
\subsubsection{Cumulative regret}
We run the experiments 10 times and calculate the average regret of all the users for each algorithm. We include $\pm std$ as the error bar, where $std$ stands for the \textit{standard deviation}. The results are given in Figure \ref{fig1} (a). First, all other algorithms outperform LinUCB, showing the advantage of conversations. Further, with our proposed ConLinUCB framework, even if we use ConLinUCB-UCB with a simple LinUCB-alike key-term selection strategy, the performance is already better than ConUCB (34.91\% improvement), showing more efficient information incorporation. With explorative conversations, ConLinUCB-BS and ConLinUCB-MCR achieve much lower regrets (37.00\% and 43.10\% improvement over ConUCB respectively), indicating better learning accuracy. ConLinUCB-MCR further leverages historical information to conduct explorative conversations adaptively, thus achieving the lowest regret.

\subsubsection{Impact of conversation frequency function $b(t)$} A larger $b(t)$ means the agent can conduct more conversations. We set $b(t)=f_q\cdot\lfloor\log t\rfloor$ and vary $f_q$ to change the conversation frequencies, i.e., $f_q\in\{5, 10, 20, 30\}$. The results are shown in Figure \ref{fig1} (b). With larger $b(t)$, our algorithms have less regret, showing the power of conversations. In all cases, ConLinUCB-BS and ConLinUCB-MCR have lower regrets than ConUCB, and ConLinUCB-MCR performs the best.

\subsubsection{Impact of $\left|\mathcal{A}_t\right|$} We vary $\left|\mathcal{A}_t\right|$ to be 25, 50, 100, 200, 500. To clearly show the advantage of our algorithms, we evaluate the difference in regrets between LinUCB and other algorithms, i.e., $R_{\text{LinUCB}}(T)-R(T)$, representing the improved accuracy of the conversational bandit algorithms as compared with LinUCB. Note that the larger $|\mathcal{A}_t|$ is, the harder it is for the algorithm to identify the best arm. Results in Figure \ref{fig1} (c) show that as $\left|\mathcal{A}_t\right|$ increases, the advantages of ConLinUCB-BS and ConLinUCB-MCR become more significant. Particularly, when $|\mathcal{A}_t|$=25, ConLinUCB-BS and ConLinUCB-MCR achieve 34.99\% and 40.21\% improvement over ConUCB respectively; when $|\mathcal{A}_t|$=500, ConLinUCB-BS and ConLinUCB-MCR achieve 50.36\% and 53.77\% improvement over ConUCB, respectively. In real applications, the size of arm set $\left|\mathcal{A}_t\right|$ is usually very large. Therefore, our proposed algorithms are expected to significantly outperform ConUCB in practice. 
\subsubsection{Cumulative regret for time-varying $\mathcal{K}$} 
% In many real-world applications, the key-term set $\mathcal{K}$ varies with time. In this case, ConLinUCB-BS is not applicable since a precomputed barycentric spanner of the key-term set could not be achieved. This section studies the case when only a subset of key-terms $\mathcal{K}_t\subseteq\mathcal{K}$ are available to the agent at each round $t$. The number of key-terms available at each time $t$ is set to be $\left|\mathcal{K}_t\right|=300$. At each round $t$, 300 key-terms are chosen uniformly at random from $\mathcal{K}$ to form $\mathcal{K}_t$. We evaluate the regret of all algorithms except ConLinUCB-BS. The results are shown in Figure \ref{fig1} (d). We can observe that ConLinUCB-MCR outperforms all baselines and achieves 43.02\% improvement over ConUCB.
This section studies the case when only a subset of key-terms $\mathcal{K}_t\subseteq\mathcal{K}$ are available to the agent at each round $t$, where ConLinUCB-BS is not applicable as mentioned before. The number of key-terms available at each time $t$ is set to be $\left|\mathcal{K}_t\right|=300$. At round $t$, 300 key-terms are chosen uniformly at random from $\mathcal{K}$ to form $\mathcal{K}_t$. We evaluate the regret of all algorithms except ConLinUCB-BS. The results are shown in Figure \ref{fig1} (d). We can observe that ConLinUCB-MCR outperforms all baselines and achieves 43.02\% improvement over ConUCB.

\section{Experiments on Real-world Datasets} \label{section6}
This section shows the experimental results on two real-world datasets, Last.FM and Movielens. The baselines, generations of arm-level rewards and key-term-level feedback, and the computation method of the barycentric spanner are the same as in the last section. Following the experiments on real data of \cite{zhang2020conversational}, we set $T=5,000$, $b(t)=5\lfloor\log (t+1)\rfloor$ and $\left|\mathcal{A}_t\right|=50$, unless otherwise stated.

\subsection{Experiment Settings}
\subsubsection{Last.FM and Movielens datasets \cite{Cantador:RecSys2011}}Last.FM is a dataset for music artist recommendations containing 186,479 interaction records between 1,892 users and 17,632 artists. Movielens is a dataset for movie recommendation containing 47,957 interaction records between 2,113 users and 10,197 movies.
\subsubsection{Generation of the data} The data is generated following \cite{10.5555/3367243.3367445,zhang2020conversational,wu2021clustering}. We treat each music artist and each movie as an arm. For both datasets, we extract $\left|\mathcal{A}\right|=2,000$ arms with the most assigned tags by users and $N_u=500$ users who have assigned the most tags. For each arm, we keep at most 20 tags that are related to the most arms, and consider them as the associated key-terms of the arm. All the kept key-terms associated with the arms form the key-term set $\mathcal{K}$. The number of key-terms for Last.FM is $\left|\mathcal{K}\right|=2,726$ and that for Movielens is $5,585$. The weights of all key-terms related to the same arm are set to be equal. Based on the interactive recordings, the user feedback is constructed as follows: if the user has assigned tags to the item, the feedback is 1, otherwise the feedback is 0. To generate the feature vectors of users and arms, following \cite{10.5555/3367243.3367445}, we construct a feedback matrix $\boldsymbol{F}\in \mathbb{R}^{N_u\times N}$ based on the above user feedback, and decompose it using the singular-value decomposition (SVD): $\boldsymbol{F}=\boldsymbol{\Theta}\boldsymbol{S}\boldsymbol{X}^{\top}$, where $\boldsymbol{\Theta}=(\boldsymbol{\theta_u^*})$, $u\in[N_u]$ and $\boldsymbol{X}=(\boldsymbol{x}_a)$, $a\in[N]$. We select $d=50$ dimensions with highest singular values in $\boldsymbol{S}$. Following \cite{zhang2020conversational}, feature vectors of key-terms are calculated using $\tilde{\boldsymbol{x}}_{k}=\sum_{a\in\mathcal{A}}\frac{w_{a,k}}{\sum_{a^{\prime}\in\mathcal{A}}w_{a^{\prime},k}}\boldsymbol{x}_a$.  The arm-level rewards and key-term-level feedback are then generated following Eq. (\ref{equation1}) and Eq. (\ref{equation3}).

\subsection{Evaluation Results}
This section first shows the results on both datasets in two cases: $\mathcal{K}$ is fixed and $\mathcal{K}$ is varying with time $t$. We also compare the running time of all algorithms on the Movielens dataset, since it has more key-terms than Last.FM.
\subsubsection{Cumulative regret}We run the experiments 10 times and calculate the average regret of all the users over $T=5,000$ rounds on the fixed generated datasets. The randomness of experiments comes from the randomly chosen $\mathcal{A}_t$ (also $\mathcal{K}_t$ in the varying key-term set case) and the randomness in the ConLinUCB-BS algorithm. We also include $\pm std$ as the error bar. For the time-varying key-term sets case, we set $\left|\mathcal{K}_t\right|=1,000$ and randomly select $\left|\mathcal{K}_t\right|$ key-terms from $\mathcal{K}$ to form $\mathcal{K}_t$ at round $t$. Results on Last.FM and Movielens for fixed key-term set are shown in Figure \ref{fig2} (a) and Figure \ref{fig2} (b). On both datasets, the regrets of ConLinUCB-BS and ConLinUCB-MCR are much smaller than ConUCB (13.28\% and 17.12\% improvement on Last.FM, 13.08\% and 16.93\% improvement on Movielens, respectively) and even the simple ConLinUCB-UCB based on our ConLinUCB framework outperforms ConUCB. Results on Last.FM and Movielens for varying key-term sets are given in Figure \ref{fig2} (c) and Figure \ref{fig2} (d). ConLinUCB-MCR performs much better than ConUCB on both datasets (19.66\% and 17.85\% improvement on Last.FM and Movielens respectively).

\subsubsection{Running time} We evaluate the running time of all the conversational bandit algorithms on the representative Movielens dataset to compare their computational efficiency. For clarity, we report the total running time for selecting arms and key-terms. We set $T=5,000$ and the results are summarized in Table \ref{table1}. It is clear that our algorithms cost much less time in both key-term selection and arm selection than ConUCB. Specifically, the improvements of total running time over ConUCB are 72.32\% for ConLinUCB-BS and 57.32\% for ConLinUCB-MCR. The main reason is that our algorithms estimate the \textit{unknown} user preference vector in one single step, whereas ConUCB does it in two separate steps as mentioned before. For ConLinUCB-BS, the time costed in the key-term selection is almost negligible, since it just randomly chooses a key-term from the precomputed barycentric spanner whenever a conversation is allowed.

\chapter{Variance-Dependent Regret Bounds for Non-stationary Linear Bandits}\label{chapter: aistats}

  We investigate the non-stationary stochastic linear bandit problem where the reward distribution evolves each round. Existing algorithms characterize the non-stationarity by the \emph{total variation budget} $B_K$, which is the summation of the change of the consecutive feature vectors of the linear bandits over $K$ rounds. However, such a quantity only measures the non-stationarity with respect to the expectation of the reward distribution, which makes existing algorithms sub-optimal under the general non-stationary distribution setting. In this work, we propose algorithms that utilize the 
\emph{variance} of the reward distribution as well as the $B_K$, and show that they can achieve tighter regret upper bounds. Specifically, we introduce two novel algorithms: Restarted Weighted$\text{OFUL}^+$ and Restarted $\text{SAVE}^+$. These algorithms address cases where the variance information of the rewards is known and unknown, respectively. Notably, when the total variance $V_K$ is much smaller than $K$, our algorithms outperform previous state-of-the-art results on non-stationary stochastic linear bandits under different settings. Experimental evaluations further validate the superior performance of our proposed algorithms over existing works. This chapter is based on our publication \cite{wang2024variance}.

\section{Introduction}
In this work, we study non-stationary stochastic bandits, which is a generalization of the classical stationary stochastic bandits, where the reward distribution is non-stationary. The intuition about the non-stationary setting comes from real-world applications such as dynamic pricing and ads allocation, where the environment changes rapidly and deviates significantly from stationarity \cite{auer2002nonstochastic,cheung2018hedging}. Most of the existing works in stochastic bandits consider a stationary setting where the goal of the agent is to minimize the \emph{static regret}, \emph{i.e.}, the summation of suboptimality gaps between the agent's selected arm and the fixed, time-independent best arm that maximizes the expectation of the reward distribution. In contrast, for the non-stationary setting, the emphasis shifts to minimizing the \emph{dynamic regret}, which represents the gap between the cumulative reward of selecting the time-dependent optimal arm at each time and that of the learner. As we can always treat a stationary bandit instance as a special case of the non-stationary bandit instance, designing algorithms that work well under the non-stationary setting is significantly more challenging.  

There have been a series of works aiming to minimize the \emph{dynamic regret} for non-stationary stochastic bandits, such as Multi-Armed Bandits (MAB) \cite{auer2002nonstochastic,10.1007/978-3-642-24412-4_16,besbes2014stochastic,wei2016tracking}, linear bandits \cite{cheung2018hedging,cheung2019learning,zhao2020simple,wei2021non,wang2023revisiting}, general function approximation \cite{faury2021regret,russac2020algorithms,russac2021self}, and the even more challenging reinforcement learning (RL) setting \cite{pmlr-v139-mao21b,touati2020efficient,gajane2018sliding,cheung2020reinforcement,wei2021non}. In this work, we mainly consider the linear bandit setting, where each arm is a contextual vector, and the expected reward of each arm is assumed to be the linear product of the arm with an unknown feature vector. Most existing \emph{dynamic regret} results for non-stationary linear bandits depend on both the \emph{non-stationarity measurement} and the number of interaction rounds. Specifically, assume $K$ is the total number of rounds, and for each $k\in[K]$, $\xb$ is one of the arms, $\btheta_k$ and $\btheta_{k+1}$ are the feature vectors at $k$ and $k+1$ rounds, satisfying $\|\xb\|_2 \leq 1$. Then, the non-stationarity measurement is often defined as the summation of the changes in the mean of the reward distribution, which is
\begin{align}
    B_K:=\sum_{k=1}^K \max_{\xb \in \RR^d}|\la\xb, \btheta_k - \btheta_{k+1}\ra|= \sum_{k=1}^K\|\btheta_k - \btheta_{k+1}\|_2\,.
\end{align}
Existing works for non-stationary linear bandits \cite{russac2019weighted, kim2019randomized, pmlr-v108-zhao20a, touati2020efficient,cheung2018hedging,zhao2020simple} achieved a regret upper bound of $\Tilde{O}(d^{7/8}B_K^{\frac{1}{4}}K^{\frac{3}{4}})$, where $d$ is the problem dimension. A recent work by \cite{wei2021non} proposed a black-box reduction method that can achieve a regret upper bound of $\Tilde{O}(dB_K^{\frac{1}{3}}K^{\frac{2}{3}})$
in the setting with a fixed arm set across all rounds. Such regret bounds clearly demonstrate that regret grows as long as the non-stationarity grows, which is aligned with intuition. 

Although existing works clearly demonstrate the relationship between the $B_K$ and the regret, we claim that it is not sufficient for us to fully characterize the non-stationary level of the reward distributions. Consider applications such as hyperparameter tuning in physical systems, the noise distribution may highly depend on the evaluation point since the measurement noise often
largely varies with the chosen parameter settings \cite{kirschner2018information}. 
For linear bandits, such examples suggest that the non-stationarity not only consists of the change of the mean of the distribution, but also the variance of the distribution. However, none of the previous works on non-stationary linear bandits considered how to leverage the variance information to improve regret bounds in the above heteroscedastic noise setting. 
Therefore, an open question arises:\begin{center}
    \emph{Can we design even better algorithms for non-stationary linear bandits by considering its variance information?}
    \end{center}

In this paper, we answer this question affirmatively. We assume that at the $k$-th round, the reward distribution of an arm $\xb$ satisfies $r_k \sim \la \btheta_k, \xb\ra + \epsilon_k$, where $\epsilon_k$ is a zero-mean noise variable with variance $\sigma_k^2$. Our contributions are:
\begin{itemize}[leftmargin =*]
\item We establish the first variance-dependent regret lower bound for non-stationary linear bandits. This result captures the interplay between non-stationarity and variance, which is not addressed in existing literature for non-stationary linear bandits.
        \item For the case where the reward variance $\sigma_k^2$ at round $k$ can be observed and the \emph{total variation budget} $B_K$ is known, we propose the Restarted-$\text{WeightedOFUL}^+$ algorithm, which uses variance-based weighted linear regression to deal with heteroscedastic noises \cite{zhou2021nearly, zhou2022computationally}    
        and a restarted scheme to forget some historical data to hedge against the non-stationarity. We prove that the regret upper bound of Restarted-$\text{WeightedOFUL}^+$ is $\Tilde{O}(d^{7/8}(B_KV_K)^{1/4}\sqrt{K} + d^{5/6}B_K^{1/3}K^{2/3})$. Our regret surpasses the best result for non-stationary linear bandits $\tilde O(dB_K^{1/3}K^{2/3})$ \cite{wei2021non} when the total variance $V_K = \tilde O(1)$ is small, which indicates that additional variance information benefits non-stationary linear bandit algorithms. 
         \item For the case where the reward variance $\sigma_k^2$ is unknown but the total variance $V_K$ and variation budget $B_K$ are known, we propose the Restarted-$\text{SAVE}^+$ algorithm. It maintains a multi-layer weighted linear regression structure with carefully-designed weight within each layer to handle the unknown variances \cite{zhao2023variance}. We prove that Restarted-$\text{SAVE}^+$ can achieve a regret upper bound of $\tilde O(d^{\frac{4}{5}}V_K^{\frac{2}{5}}B_K^{\frac{1}{5}}K^{\frac{2}{5}}+ d^{\frac{2}{3}}B_K^{\frac{1}{3}}K^{\frac{2}{3}})$. Specifically, when $V_K = \tilde O(1)$, our regret is also better than the existing best result $\tilde O(dB_K^{1/3}K^{2/3})$ \cite{wei2021non}, which again verifies the effect of the variance information.
        \item Lastly, we propose Restarted-$\text{SAVE}^+$-BOB for the case where both the reward variance $\sigma_k^2$ and $B_K$ are unknown. Restarted-$\text{SAVE}^+$-BOB equips a \emph{bandit-over-bandit} (BOB) framework to handle the unknown $B_K$ \cite{cheung2019learning}, and also maintains a multi-layer structure as Restarted-$\text{SAVE}^+$. We show that Restarted-$\text{SAVE}^+$-BOB achieves a regret upper bound of $\tilde O(d^{\frac{4}{5}}V_K^{\frac{2}{5}}B_K^{\frac{1}{5}}K^{\frac{2}{5}}+ d^{\frac{2}{3}}B_K^{\frac{1}{3}}K^{\frac{2}{3}}+d^{\frac{1}{5}}K^{\frac{7}{10}})$, and it behaves the same as Restarted-$\text{SAVE}^+$ when $V_K = \tilde O(1)$ and $B_K = \Omega(d^{-14}K^{1/10})$.
        % \item Notably, our proposed algorithms apply to the contextual setting where the arm set varies over time. In the worst case with $\sum_{k=1}^K \sigma_k^2=O(K)$, our results are $\Tilde{O}(B_K^{\frac{1}{4}}K^{\frac{3}{4}})$, matching the state-of-the-art results in the contextual setting. In the case where $\sum_{k=1}^K \sigma_k^2$ is small, our results are $\Tilde{O}(B_K^{\frac{1}{3}}K^{\frac{2}{3}})$, better than previous results in contextual setting in terms of the dependence on $K$.
        \item We also conduct experimental evaluations to validate the outperformance of our proposed algorithms over existing works.
    \end{itemize}
    \paragraph{Notation} 
    We use lower case letters to denote scalars, and use lower and upper case bold face letters to denote vectors and matrices respectively. We denote by $[n]$ the set $\{1,\dots, n\}$. For a vector $\xb\in \RR^d$ and a positive semi-definite matrix $\bSigma\in \RR^{d\times d}$, we denote by $\|\xb\|_2$ the vector's Euclidean norm and define $\|\xb\|_{\bSigma}=\sqrt{\xb^\top\bSigma\xb}$. For two positive sequences $\{a_n\}$ and $\{b_n\}$ with $n=1,2,\dots$, 
    we write $a_n=O(b_n)$ if there exists an absolute constant $C>0$ such that $a_n\leq Cb_n$ holds for all $n\ge 1$ and write $a_n=\Omega(b_n)$ if there exists an absolute constant $C>0$ such that $a_n\geq Cb_n$ holds for all $n\ge 1$. We use $\tilde O(\cdot)$ to further hide the polylogarithmic factors. 

\begin{table*}[t!]
\centering
\resizebox{0.95\textwidth}{!}{%
\begin{tabular}{cggggg}
\toprule
\rowcolor{white}
&  &  & Variance & Varying &  \\
\rowcolor{white}
\multirow{-2}{*}{Model} & \multirow{-2}{*}{Algorithm} & \multirow{-2}{*}{Regret} & -Dependent & Arm Set & \multirow{-2}{*}{Require $B_K$} \\
\midrule
\rowcolor{white}
& SW-UCB &  \\
\rowcolor{white}
 &\cite{cheung2018hedging} & \multirow{-2}{*}{$\tilde  O\big(d^{\frac{7}{8}}B_K^{\frac{1}{4}}K^{\frac{3}{4}}\big)$}&\multirow{-2}{*}{No}&\multirow{-2}{*}{Yes}&\multirow{-2}{*}{Yes}\\
\rowcolor{white}
\multirow{-4}{*}{Linear Bandit}
 & BOB &\\
\rowcolor{white}
 &\cite{cheung2018hedging}&\multirow{-2}{*}{ $\tilde  O\Big(d^{\frac{7}{8}}B_K^{\frac{1}{4}}K^{\frac{3}{4}}\Big)$} &\multirow{-2}{*}{No}&\multirow{-2}{*}{Yes}&\multirow{-2}{*}{No} \\
\rowcolor{white}
 & RestartUCB &  \\
\rowcolor{white}
 &\cite{zhao2020simple}&\multirow{-2}{*}{ $\tilde  O\Big(d^{\frac{7}{8}}B_K^{\frac{1}{4}}K^{\frac{3}{4}}\Big)$} &\multirow{-2}{*}{No}&\multirow{-2}{*}{Yes}&\multirow{-2}{*}{Yes}\\
\rowcolor{white}
 & RestartUCB-BOB &  \\
\rowcolor{white}
 &\cite{zhao2020simple}&\multirow{-2}{*}{ $\tilde  O\Big(d^{\frac{7}{8}}B_K^{\frac{1}{4}}K^{\frac{3}{4}}\Big)$} &\multirow{-2}{*}{No}&\multirow{-2}{*}{Yes}&\multirow{-2}{*}{No}\\
\rowcolor{white}
 & LB-WeightUCB &  \\
\rowcolor{white}
 &\cite{wang2023revisiting}&\multirow{-2}{*}{ $\tilde  O\Big(d^{\frac{3}{4}}B_K^{\frac{1}{4}}K^{\frac{3}{4}}\Big)$} &\multirow{-2}{*}{No}&\multirow{-2}{*}{Yes}&\multirow{-2}{*}{Yes}\\
\rowcolor{white}
  & MASTER + OFUL &  \\
\rowcolor{white}
 &\cite{wei2021non}&\multirow{-2}{*}{ $\tilde  O\Big(d B_K^{\frac{1}{3}}K^{\frac{2}{3}}\Big)$} &\multirow{-2}{*}{No}&\multirow{-2}{*}{No}&\multirow{-2}{*}{No}\\
 &\small{Restarted-$\algbandit$}&$\tilde O\Big(d^{\frac{7}{8}}(B_KV_K)^{\frac{1}{4}}K^{\frac{1}{2}}$&&&\\
  &\textbf{(Ours)}&$+ d^{\frac{5}{6}}B_K^{\frac{1}{3}}K^{\frac{2}{3}}\Big)$
&\multirow{-2}{*}{Yes}&\multirow{-2}{*}{Yes}&\multirow{-2}{*}{Yes}\\
&Restarted $\text{SAVE}^+$&$\tilde O\Big(d^{\frac{4}{5}}V_K^{\frac{2}{5}}B_K^{\frac{1}{5}}K^{\frac{2}{5}}$&&&\\
  &\textbf{(Ours)}&$+ d^{\frac{2}{3}}B_K^{\frac{1}{3}}K^{\frac{2}{3}}\Big)$
&\multirow{-2}{*}{Yes}&\multirow{-2}{*}{Yes}&\multirow{-2}{*}{Yes}\\
 &Restarted $\text{SAVE}^+\text{-BOB}$&$\tilde O\Big(d^{\frac{4}{5}}V_K^{\frac{2}{5}}B_K^{\frac{1}{5}}K^{\frac{2}{5}}$&&&\\
  &\textbf{(Ours)}&$+ d^{\frac{2}{3}}B_K^{\frac{1}{3}}K^{\frac{2}{3}}+d^{\frac{1}{5}}K^{\frac{7}{10}}\Big)$
&\multirow{-2}{*}{Yes}&\multirow{-2}{*}{Yes}&\multirow{-2}{*}{No}\\
 &Lower Bound&$\tilde \Omega\Big(d^{2/3}B_K^{1/3}V_K^{1/3}K^{1/3}$&&&\\
  &\textbf{(Ours)}&$ \land V_K+\sqrt{B_K K}\Big)$
&\multirow{-2}{*}{Yes}&\multirow{-2}{*}{Yes}&\multirow{-2}{*}{-}\\
 \midrule
 \rowcolor{white}
MAB 
  &Rerun-UCB-V&$\tilde O\Big(\left|\mathcal{A}\right|^{\frac{2}{3}}B_K^{\frac{1}{3}}V_K^{\frac{1}{3}}K^{\frac{1}{3}}$&&&\\
  \rowcolor{white}
  &\cite{wei2016tracking}&$+ \left|\mathcal{A}\right|^{\frac{1}{2}}B_K^{\frac{1}{2}}K^{\frac{1}{2}}\Big)$
&\multirow{-2}{*}{Yes}&\multirow{-2}{*}{No}&\multirow{-2}{*}{Yes}\\
\rowcolor{white}
& Lower Bound &  \\
\rowcolor{white}
 &\cite{wei2016tracking} & \multirow{-2}{*}{$\tilde  \Omega\Big(B_K^{\frac{1}{3}}V_K^{\frac{1}{3}}K^{\frac{1}{3}}+B_K^{\frac{1}{2}}K^{\frac{1}{2}}\Big)$}&\multirow{-2}{*}{Yes}&\multirow{-2}{*}{No}&\multirow{-2}{*}{-}\\
\bottomrule
\end{tabular}
}
\caption{Comparison of non-stationary bandits in terms of regret guarantee. $K$ is the total rounds, $d$ is the problem dimension for linear bandits, $B_K$ is the \emph{total variation budget} defined in Section \ref{sec:setting} (for the MAB setting, $B_K=\sum_{k=1}^K \|\mu_k-\mu_{k+1}\|_{\infty}$, where $\mu_k$ is the mean of the reward distribution at round $k$), $V_K$ is the \emph{total variance} defined in Section \ref{sec:setting}, $\left|\mathcal{A}\right|$ is the number of arms for MAB.}
\label{table:11}
\end{table*}

\section{Problem Setting}\label{sec:setting}
    We consider a heteroscedastic variant of the classic non-stationary linear contextual bandit problem. Let $K$ be the total number of rounds. At each round $k\in[K]$, the learner interacts with the environment as follows: 
   (1) the environment generates an arbitrary arm set $\cD_k \subseteq \RR^d$ where each element represents a feasible arm for the learner to choose, and also generates an \emph{unknown} feature vector $\btheta_k$;
        (2) the leaner observes $\cD_k$ and selects $\ab_k \in \cD_k$;
        (3) the environment generates the stochastic noise $\epsilon_k$ and reveals the stochastic reward $r_k = \la \btheta_k, \ab_k \ra + \epsilon_k$ to the leaner. 
        % (1) the environment generates an arbitrary arm set $\cD_k \subseteq \RR^d$ where each element represents a feasible arm for the learner to choose, and also generates an \emph{unknown} feature vector $\btheta_k$; (2) the leaner observes $\cD_k$ and selects $\ab_k \in \cD_k$; and (3) the environment generates the stochastic noise $\epsilon_k$ at round $k$ and reveal the stochastic reward $r_k = \la \btheta_k, \ab_k \ra + \epsilon_k$ to the leaner. 
 We assume that for all $k \geq 1$ and all $\ab \in \cD_k$, $\la \ab, \btheta_k\ra \in [-1, 1]$, $\|\btheta_k\|_2 \leq \pnorm$, $\|\ab\|_2 \leq A$.
        
        Following \cite{zhou2021nearly,zhao2023variance}, we assume the following condition on the random noise $\epsilon_k$ at each round $k$: 
    \begin{align} 
        \PP\left(|\epsilon_k| \le R\right) = 1&,\quad \EE[\epsilon_k | \ab_{1:k}, \epsilon_{1:k - 1}] = 0, \notag\\
        \quad\EE[\epsilon_k^2 | \ab_{1:k}, \epsilon_{1:k - 1}] &\leq \sigma_k^2. \label{eq:noise:condition}
    \end{align}

    Following \cite{cheung2018hedging,cheung2019learning,russac2019weighted,zhao2020simple}, we assume the summation of $\ell_2$ differences of consecutive $\btheta_k$'s is upper bounded by the \emph{total variation budget} $B_K$, \emph{i.e.}, $\sum_{k=1}^{K-1}\|\btheta_{k+1}-\btheta_k\|_2\leq B_K$, where the $\btheta_k$'s can be adversarially chosen by an oblivious adversary. We also assume that the \emph{total variance} is upper bounded by  $V_K$, which is $\sum_{k=1}^K \sigma_k^2\leq V_K$. The goal of the agent is to minimize the \emph{dynamic regret} defined as follows: $\text{Regret}(K)=\sum_{k\in[K]} \big(\la\ab_k^*,\btheta_k\ra-\la\ab_k,\btheta_k\ra\big)$,
where $\ab_k^*=\argmax_{\ab\in\cD_k}\la\ab,\btheta_k\ra$ is the optimal arm at round $k$ with the highest expected reward.

\section{Lower Bound}

In this section, we establish a novel variance-dependent regret lower bound for non-stationary linear bandits, which reveals new insights into the problem structure. 
\begin{theorem}\label{thm:lowerbound}
    Given $K>0$. For any bandit algorithm there exists $\btheta_1,\dots, \btheta_K$ satisfying the problem setting denoted in Section \ref{sec:setting}, such that
    \begin{align}
        &\text{Regret}(K) \notag \\
        &\geq \Omega(\min\{d^{2/3}B_K^{1/3}V_K^{1/3}K^{1/3}, V_K\} + \sqrt{B_K K}).\notag
    \end{align}
\end{theorem}
\begin{proof}
    See Appendix \ref{app:lowerbound}. 
\end{proof}
\begin{remark}
Note that \cite{cheung2019learning} proposed a lower bound of $\Omega(d^{2/3}B_K^{1/3}K^{2/3})$ for general non-stationary linear bandits. However, their result applies only to cases without the variance restriction $V_K$, making it inapplicable to our setting.
\end{remark}
Theorem \ref{thm:lowerbound} represents the first variance-dependent regret lower bound specifically tailored for non-stationary linear bandits. The bound highlights the inherent complexity of balancing variance and non-stationarity, offering a foundation for future work aimed at designing algorithms with matching upper bounds. Notably, our result improves the existing variance-dependent lower bound $\Omega(B_K^{1/3}V_K^{1/3}K^{1/3} + B_K^{1/2}K^{1/2})$  \cite{wei2016tracking} by a factor of $d^{2/3}$ for the linear bandits setting.

\section{Non-stationary Linear Contextual Bandit with Known Variance}
\begin{algorithm}[t]	\caption{Restarted-$\algbandit$}\label{algorithm:reweightbandit}
	\begin{algorithmic}[1]
	\REQUIRE Regularization parameter $\lambda>0$; 
	$\pnorm$, 
	an upper bound on the $\ell_2$-norm of $\btheta_k$ for all $k\in[K]$; confidence radius $\hat\beta_k$, variance parameters $\alpha, \gamma$; restart window size $w$.
	\STATE $\hat\bSigma_1 \leftarrow \lambda \Ib$, $\hat\bbb_1 \leftarrow \zero$, $\hat\btheta_1 \leftarrow \zero$, $\hat\beta_1 = \sqrt{\lambda}\pnorm$
	\FOR{$k=1,\ldots, K$}
        \IF{$k\% w == 0$}
\STATE $\hat\bSigma_k \leftarrow \lambda \Ib$, $\hat\bbb_k \leftarrow \zero$, $\hat\btheta_k \leftarrow \zero$, $\hat\beta_k = \sqrt{\lambda}\pnorm$
        \ENDIF
	\STATE Observe $\cD_k$ and choose $\ab_k\leftarrow\argmax_{\ab\in\cD_k} \la\ab,\btheta_k\ra+\hat\beta_k\|\ab_k\|_{\hat\bSigma_k^{-1}}$ 
	\STATE Observe $(r_k,\sigma_k)$, set $\bar\sigma_k$ as 
	\begin{align}
	    &\bar\sigma_k \leftarrow \max\{\sigma_k, \alpha, \gamma\|\ab_k\|_{\hat\bSigma_k^{-1}}^{1/2}\}\label{def:banditvar}
	\end{align}
	\STATE $\hat\bSigma_{k+1} \leftarrow \hat\bSigma_k + \ab_k\ab_k^\top/\bar\sigma_k^2$, $\hat\bbb_{k+1} \leftarrow \hat\bbb_k + r_k\ab_k/\bar\sigma_k^2$, $\hat\btheta_{k+1}\leftarrow \hat\bSigma_{k+1}^{-1}\hat\bbb_{k+1}$\label{alg:reweightbandit}
	\ENDFOR
	\end{algorithmic}
\end{algorithm}
In this section, we introduce our Algorithm \ref{algorithm:reweightbandit} under the setting where the variance $\sigma_k^2$ at $k$-th iteration is known to the agent in prior. We start from WeightedOFUL$^+$ \cite{zhou2022computationally}, an \emph{weighted ridge regression}-based algorithm for heteroscedastic linear bandits under the stationary reward assumption. For our non-stationary linear bandit setting where $\btheta_k$ is changing over the round $k$, WeightedOFUL$^+$ aims to build an $\hat\btheta_k$ which estimates the feature vector $\btheta_k$ by using the solution to the following regression problem:
\begin{align}
    \hat\btheta_k\leftarrow \arg\min_{\btheta}\sum_{t=1}^{k-1}\bar\sigma_t^{-2}(\la \btheta, \ab_t\ra - r_t)^2 + \lambda \|\btheta\|_2^2,\label{hhh}
\end{align}
where the weight is defined as in \eqref{def:banditvar}. After obtaining $\hat\btheta_k$, WeightedOFUL$^+$ chooses arm $\ab_k$ by maximizing the upper confidence bound (UCB) of $\la \ab, \hat\btheta\ra$, with an exploration bonus $\hat\beta_k\|\ab_k\|_{\hat\bSigma_k^{-1}}$, where $\hat\bSigma_k$ is the covariance matrix over $\ab_k$. 
The weight $\bar\sigma_k^2$ is introduced to balance the different past examples based on their reward variance $\sigma_k^2$, and such a strategy has been proved as a state-of-the-art algorithm for the stationary heteroscedastic linear bandits \cite{zhou2022computationally}. However, the non-stationary nature of our setting prevents us from directly using $\hat\btheta_k$ defined in \eqref{hhh} as an estimate to $\btheta$. Therefore, inspired by the \emph{restarting} strategy which has been adopted by previous algorithms for non-stationary linear bandits \cite{zhao2020simple}, we propose Restarted-WeightedOFUL$^+$, which periodically restarts itself and runs WeightedOFUL$^+$ as its submodule. The restart window size is set as $w$, which is used to balance the nonstationarity and the total regret and will be fine-tuned in the next steps. Combined with the restart window size $w$, we set $\{\hat\beta_k\}_{k\geq 1}$ to
\begin{align}
    \hat\beta_k& = 12\sqrt{d\log(1+\frac{(k\%w)A^2}{\alpha^2d\lambda})\log(32(\log(\frac{\gamma^2}{\alpha}+1)\frac{(k\%w)^2}{\delta})} \notag \\
    &\quad + 30\log(32(\log(\frac{\gamma^2}{\alpha})+1)\frac{(k\%w)^2}{\delta})\frac{R}{\gamma^2}+ \sqrt{\lambda}\pnorm.
\label{eq:defbanditbeta}
\end{align}

We now propose the theoretical guarantee for Algorithm \ref{alg:reweightbandit}. The following key lemma shows how nonstationarity affects our estimation of the reward of each arm.

\begin{lemma}\label{lemma:key}
Let $0<\delta<1$. Then with probability at least $1-\delta$, for any action $\ab \in \RR^d$, we have
    \begin{align}
    |\ab^\top(\hat\btheta_k-\btheta_k)|&\leq \underbrace{\frac{A^2}{\alpha}\sqrt{\frac{dw}{\lambda}}\sum_{t=w\cdot \lfloor k/ w\rfloor+1}^{k-1} \|\btheta_t-\btheta_{t+1}\|_2}_{\text{Drifting term}}\notag \\
    &\quad +\underbrace{\hat\beta_k\|\ab\|_{\hat\bSigma_k^{-1}}}_{\text{Stochastic term}}.\notag
\end{align}
\end{lemma}
\begin{proof}
    See Appendix \ref{app:keylemma} for the full proof. 
\end{proof}
Here we provide a proof sketch of Lemma \ref{lemma:key} to show the technical challenge we need to overcome. Without loss of generality, we prove the lemma for $k\in[1,w]$. We have 

    \begin{align}
    |\ab^\top(\hat\btheta_k-\btheta_k)|&\leq\left|\ab^\top\hat\bSigma_k^{-1}\sum_{t=1}^{k-1}\frac{\ab_t\ab_t^\top}{\bar\sigma_t^2}(\btheta_t-\btheta_k)\right|\notag\\  &+\|\ab\|_{\hat\bSigma_k^{-1}}\|\sum_{t=1}^{k-1}\frac{\ab_t\epsilon_t}{\bar\sigma_t^2}\|_{\hat\bSigma_k^{-1}}+\sqrt{\lambda}B \|\ab\|_{\hat\bSigma_k^{-1}}\,,\label{lemma 4.1 sketch}
\end{align}

For the first term, it gets involved by the nonstationarity of $\btheta_k$. By rearranging the summation orders and several calculation steps, we have

    \begin{align*}
&\left|\ab^\top\hat\bSigma_k^{-1}\sum_{t=1}^k\frac{\ab_t\ab_t^\top}{\bar\sigma_t^2}(\btheta_t-\btheta_k)\right|
        \leq \sum_{t=1}^{k-1}  |  \ab^{\top} \hat\bSigma_k^{-1} \frac{\ab_t}{\bar\sigma_t} | \cdot \| \frac{\ab_t}{\bar\sigma_t}\|_2 \\
        &\cdot \| \sum_{s = t}^{k-1} (\btheta_{s} - \btheta_{s+1})\|_2
     % & \leq \frac{A}{\alpha}\sum_{t=1}^{k-1}   |  \ab^{\top} \hat\bSigma_k^{-1} \frac{\ab_t}{\bar\sigma_t} | \cdot   \| \sum_{s = t}^{k-1} (\btheta_{s} - \btheta_{s+1})\|_2 \\
     % & \leq \frac{A}{\alpha}\sum_{s =1}^{k-1} \sum_{t = 1}^{s} | \ab^{\top}\hat\bSigma_k^{-1} \frac{\ab_t}{\bar\sigma_t} | \cdot \| \btheta_{s} - \btheta_{s+1}\|_2 \\
     % & \leq \frac{A}{\alpha}\sum_{s = 1}^{k-1} \sqrt{ \bigg[ \sum_{t =1}^{s} \ab^\top \hat\bSigma_k^{-1} \ab \bigg]  \cdot \biggl [ \sum_{t =1}^{s} \frac{\ab_t}{\bar\sigma_t}^\top \hat\bSigma_k^{-1}  \frac{\ab_t}{\bar\sigma_t}\bigg] }
     % \cdot \norm{ \btheta_{s} - \btheta_{s+1}}_2  \\ 
     % & \leq \frac{A}{\alpha}\sum_{s =1}^{k-1} \sqrt{ \bigg[ \sum_{t =1}^{s} \ab^\top \hat\bSigma_k^{-1}\ab \bigg] \cdot d }
     % \cdot \norm{ \btheta_{s} - \btheta_{s+1}}_2 \\
     % & \leq \frac{A\|\ab\|_2}{\alpha}\sqrt{d} \sum_{s =1}^{k-1} \sqrt{ \frac{\sum_{t =1}^{k-1} 1 }{ \lambda }} \cdot \norm{ \btheta_{s} - \btheta_{s+1}}_2 \\
      \leq \frac{A^2}{\alpha}\sqrt{\frac{d w}{\lambda }} \sum_{s =1}^{k-1} \norm{ \btheta_{s} - \btheta_{s+1}}_2\,,
\end{align*}

We would like to highlight the subtleties in both our algorithm design and analysis to get the desired improvement. First, from here, we can see the necessity of introducing $\alpha$ in the design of $\bar\sigma_k$ in Eq.(\ref{def:banditvar}), which makes it possible to upper bound $\bar\sigma_k^{-1}$ and get a tunable $\alpha$ in the drifting term, which can subsequently be used to optimize the regret bound. Second, we show that it is essential to split the term $\bar\sigma_t^{-2}$ as how we did. Only by doing that can we bound the $\sum_{t =1}^{s} \frac{\ab_t}{\bar\sigma_t}^\top \hat\bSigma_k^{-1}  \frac{\ab_t}{\bar\sigma_t}$ term by $d$ with the elliptical potential lemma. Otherwise, we can get a $1/\alpha^2$ term rather than the $A/\alpha$ term, which will hurt the final regret bound. For the second term in Eq.(\ref{lemma 4.1 sketch}), a vanilla way to control it is adopting a self-normalized concentration inequality from \cite{abbasi2011improved}. However, it can not utilize variance information, but just the magnitude of the noise, which fails to get a tight bound with the variance information. Inspired by \cite{zhou2022computationally,zhou2021nearly,zhao2023variance}, we adapt a variance-adaptive concentration inequality in Theorem \ref{lemma:concentration_variance} to get a tighter bound. Similar arguments also hold for the proof of Theorem \ref{thm:regret1} for the unknown variance case. We refer to Appendix \ref{app:keylemma} for the full proof.
Lemma \ref{lemma:key} suggests that under the non-stationary setting, the difference between the true expected reward and our estimated reward will be upper bounded by two separate terms. The first drifting term characterizes the error caused by the non-stationary environment, and the second stochastic term characterizes the error caused by the estimation of the stochastic environment. Note thata  similar bound has also been discovered in \cite{touati2020efficient}. We want to emphasize that our bound differs from existing ones in 1) an additional variance parameter $\alpha$ in the drifting term, and 2) a weighted convariance matrix $\hat\bSigma$ rather than a vanilla convariance matrix.

Next we present our main theorem.
\begin{theorem}\label{thm: regret for algo1 final}
Let $0<\delta<1$. By treating $A,\lambda, B, R$ as constants and setting $\gamma^2 = R/\sqrt{d}$, with probability at least $1-\delta$, the regret of Restarted-$\algbandit$ is bounded by
\begin{align}
     \text{Regret}(K) &= \tilde O(B_Kw^{3/2}d^{1/2}\alpha^{-1}+ dK\alpha/\sqrt{w} \notag \\
     &\quad  + d\sqrt{K V_K/w} + dK/w).\label{eq:ttt}
\end{align}
% \begin{align}
%     \text{Regret}(K)
%     & \leq \frac{2A^2B_Kw^{\frac{3}{2}}}{\alpha}\sqrt{\frac{d}{\lambda}}+ 4\hat\beta\sqrt{V_K + K\alpha^2}\sqrt{\frac{Kd\iota}{w}}\\
%     &+ \frac{4d\iota K\hat\beta\gamma^2}{w}+\frac{4d\iota K}{w},\label{eq:cororegret}
% \end{align}
% where $\iota = \log(1+\frac{wA^2}{d\lambda\alpha^2})$, and $\hat\beta= \tilde O(\sqrt{d} + R/\gamma^2 + \sqrt{\lambda}\pnorm)$. 
\end{theorem}
\begin{proof}
    See Appendix \ref{app:firstalg}.
\end{proof}
\begin{remark}\label{rmk:1}
    For the stationary linear bandit case where $B_K = 0$, we can set the restart window size $w = K$ and the variance parameter $\alpha = 1/\sqrt{K}$, then we obtain an $\tilde O(d\sqrt{V_K} + d)$ regret for Algorithm \ref{alg:reweightbandit}, which is identical to the one in \cite{zhou2022computationally}. 
\end{remark}

Next, we aim to select parameters $\alpha$ and $w$ in order to optimize \eqref{eq:ttt}. 
\begin{corollary}
    Assume that $B_K, V_K \in [\Omega(1), O(K)]$. Then by selecting
    \begin{align}
        &w=d^{1/4}\sqrt{V_K/B_K}, &dV_K^6\geq K^4B_K^2,\notag\\
        &w=d^{1/6}(K/B_K)^{1/3}&\text{otherwise}.\notag
    \end{align}
    and $\alpha = d^{-1/4}B_K^{1/2}wK^{-1/2}$, the regret is in the order
    \begin{align}
        \text{Regret}(K) &=\tilde O(d^{7/8}(B_KV_K)^{1/4}\sqrt{K} + d^{5/6}B_K^{1/3}K^{2/3}). 
    \end{align}\label{corollary for algo1 regret}
\end{corollary}

\begin{remark}
We compare the regret of Algo.\ref{alg:reweightbandit} in Corollary \ref{corollary for algo1 regret} with previous results in the special cases below.
\begin{itemize}
    \item In the worst case where $V_K=O(K)$, our result becomes $\Tilde{O}(d^{7/8}B_K^{1/4}K^{3/4})$, matching the state-of-the-art results for restarting and sliding window strategies \cite{cheung2018hedging,zhao2020simple}.
    \item In the case where the \emph{total variance} is small, \emph{i.e.}, $V_K=\Tilde{O}(1)$,  assuming that $K^4> d$, our result becomes $\Tilde{O}(d^{5/6}B_K^{1/3}K^{2/3})$, better than all the previous results \cite{cheung2018hedging,zhao2020simple,wang2023revisiting,wei2021non}.
\end{itemize}
\end{remark}

\begin{remark}
    \cite{wei2016tracking} has studied non-stationary MAB with dynamic variance. With the knowledge of $V_K$ and $B_K$, \cite{wei2016tracking} proposed a restart-based Rerun-UCB-V algorithm with a $\tilde O(\left|\mathcal{A}\right|^{\frac{2}{3}}B_K^{\frac{1}{3}}V_K^{\frac{1}{3}}K^{\frac{1}{3}}+ \left|\mathcal{A}\right|^{\frac{1}{2}}B_K^{\frac{1}{2}}K^{\frac{1}{2}})$ regret, where $\cA$ is the action set. Reduced to the MAB setting, our Restarted-$\algbandit$ achieves an\\ $\tilde O(|\cA|^{7/8}(B_KV_K)^{1/4}\sqrt{K} + |\cA|^{5/6}B_K^{1/3}K^{2/3})$ regret, which is worse than \cite{wei2016tracking}. We claim that this is due to the generality of the linear bandits, which brings us a looser bound to the drifting term in Lemma \ref{lemma:key}. When restricting to the MAB setting, our drifting term enjoys a tighter bound, which could further tighten our final regret. To develop an algorithm achieving the same regret as \cite{wei2016tracking} is beyond the scope of this work.
\end{remark}

\begin{remark}
    \cite{wei2016tracking} has established a lower bound $\tilde  \Omega(B_K^{\frac{1}{3}}V_K^{\frac{1}{3}}K^{\frac{1}{3}}+B_K^{\frac{1}{2}}K^{\frac{1}{2}})$ for MAB with total variance $V_K$ and total variation budget $B_K$. There still exist gaps between our regret and their lower bound regarding the dependence of $K, V_K, B_K$, and we leave to fix the gaps as future work.
\end{remark}

\section{Non-stationary Linear Contextual Bandit with Unknown Variance and Total Variation Budget}
By Theorem \ref{thm: regret for algo1 final}, we know that Algorithm \ref{alg:reweightbandit} is able to utilize the total variance $V_K$ and obtain a better regret result compared with existing algorithms which do not utilize $V_K$. However, the success of Algorithm \ref{alg:reweightbandit} depends on the knowledge of the per-round variance $\sigma_k$, and it also depends on a good selection of restart window size $w$, whose optimal selection depends on both $V_K$ and $B_K$. In this section, we aim to relax these two requirements with still better regret results.

\subsection{Unknown Per-round Variance, Known $V_K$ and $B_K$}

\begin{algorithm*}[t!] 
    \caption{$\text{Restarted SAVE}^+$} \label{alg:1}
    \begin{algorithmic}[1]
        \REQUIRE $\alpha > 0$; the upper bound on the $\ell_2$-norm of $\ab$ in $\cD_k (k\ge 1)$, i.e., $A$; the upper bound on the $\ell_2$-norm of $\btheta_k$ $(k\ge 1)$, i.e., $\pnorm$; restart window size $w$.
    \STATE Initialize $L \leftarrow \lceil \log_2 (1 / \alpha) \rceil$. 
    \STATE Initialize the estimators for all layers: $\hat\bSigma_{1, \ell} \leftarrow 2^{-2\ell} \cdot \Ib$, $\hat\bbb_{1, \ell} \leftarrow \zero$, $\hat\btheta_{1, \ell} \leftarrow \zero$, $\hat \beta_{1, \ell} \leftarrow 2^{-\ell + 1}$, $\hat\Psi_{1,\ell}\leftarrow \emptyset$ for all $\ell \in [L]$. \label{alg1:line:2}
    \FOR{$k=1,\ldots, K$}
    \IF{$k\%w==0$}
\STATE Set $\hat\bSigma_{k, \ell} \leftarrow 2^{-2\ell} \cdot \Ib$, $\hat\bbb_{k, \ell} \leftarrow \zero$, $\hat\btheta_{k, \ell} \leftarrow \zero$, $\hat \beta_{1, \ell} \leftarrow 2^{-\ell + 1}$, $\hat\Psi_{k,\ell}\leftarrow \emptyset$ for all $\ell \in [L]$.\label{alg:restart}
\ENDIF
\STATE Observe $\cD_k$, choose $\ab_k \leftarrow \argmax_{\ab \in \cD_k} \min_{\ell \in [L]}\la \ab, \hat{\btheta}_{k, \ell}\ra + \hat\beta_{k, \ell} \|\ab\|_{\hat\bSigma_{k, \ell}^{-1}}$ and observe $r_k$. \label{line:selection} 
\STATE Set $\ell_k\leftarrow L+1$
\STATE Let $\cL_k \leftarrow \{\ell \in [L]: \|\ab_k\|_{\hat\bSigma_{k,\ell}^{-1}} \geq 2^{-\ell}\}$, set $\ell_k\leftarrow \min(\cL_k)$ if $\cL_k \neq \emptyset$\label{alg1:line:9}
\STATE $\hat\Psi_{k,\ell_k} \leftarrow \hat\Psi_{k,\ell_k}\cup \{k\}$
\IF{$\cL_k \neq \emptyset$}
\STATE Set $w_k \leftarrow \frac{2^{-\ell_k}}{{\|\ab_k\|_{\hat\bSigma_{k, \ell_k}^{-1}}}}$ and update \label{alg1:line:12} 
\begin{align} \hat\bSigma_{k + 1, \ell_k} \leftarrow \hat\bSigma_{k, \ell_k} + w_k^2 \ab_k \ab_k^{\top}, \hat{\bbb}_{k + 1, \ell} \leftarrow \hat\bbb_{k, \ell_k} + w_k^2 \cdot r_k \ab_k, \hat\btheta_{k + 1, \ell_k} \leftarrow \hat\bSigma_{k + 1, \ell_k}^{-1} \hat{\bbb}_{k + 1, \ell_k}. \notag \end{align}
\STATE Compute the adaptive confidence radius $\hat\beta_{k+1, l}$for the next round according to \eqref{eq:def:beta}. \label{alg1:line:18}
\ENDIF
     \STATE For $\ell\neq \ell_k$ let $\hat\bSigma_{k + 1, \ell} \leftarrow \hat\bSigma_{k, \ell}, \hat\bbb_{k + 1, \ell} \leftarrow \hat\bbb_{k, \ell}, \hat\btheta_{k + 1, \ell} \leftarrow \hat\btheta_{k, \ell}, \hat\beta_{k + 1, \ell} \leftarrow \hat\beta_{k, \ell}.$
    \ENDFOR
    \end{algorithmic}
\end{algorithm*}

We first aim to relax the requirement that each $\sigma_k^2$ is known to the agent at the beginning of $k$-th round. We follow the SAVE algorithm \cite{zhao2023variance} which introduces a multi-layer structure \cite{chu2011contextual, he2021uniform} to deal with unknown $\sigma_k^2$. In detail, SAVE maintains multiple estimates to the current feature vector $\theta_k$, which we denote them as $\hat\btheta_{k,1},...,\hat\btheta_{k,L}$ in line \ref{alg1:line:2}. Each $\hat\btheta_{k,\ell}$ is calculated based on a subset $\hat\Psi_{k, \ell} \subseteq [k-1]$ of samples $\{(\ab_t, r_t)\}$. The rule that whether to add the current $k$ to some $\hat\Psi_{k, \ell}$ is based on the uncertainty of $\ab_k$ with the sample set $\{(\ab_t, r_t)\}_{t \in \hat\Psi_{k, \ell}}$. As long as $\ab_k$ is too uncertain w.r.t. some level $\ell_k$ (line \ref{alg1:line:9}), we add $k$ to $\hat\Psi_{k, \ell}$ and update the estimate $\hat\btheta_{k,\ell_k}$ accordingly (line \ref{alg1:line:12}). Each $\hat\btheta_{k,\ell_k}$ is calculated as the solution of a weighted regression problem, where the weight $w_k$ is selected as the inverse of the uncertainty of the arm $\ab_k$ w.r.t. the samples in the $\ell$-th layer. Maintaining $L$ different $\hat\btheta_{k,\ell}, \ell \in [L]$, Algorithm \ref{alg:1} then calculates $L$ number of UCB for each arm $\ab$ w.r.t. $L$ different $\hat\btheta_{k,\ell}$, and selects the arm which maximizes the minimization of $L$ UCBs (line \ref{line:selection}). It has been shown in \cite{zhao2023variance} that such a multilayer structure is able to utilize the $V_K$ information without knowing the per-round variance $\sigma_k^2$. Similar to Algorithm \ref{alg:reweightbandit}, in order to deal with the nonstationarity issue, we introduce a restarting scheme that Algorithm \ref{alg:1} restarts itself by a restart window size $w$ (line \ref{alg:restart}).

Next we show the theoretical guarantee of Algorithm \ref{alg:1}. We call the restart time rounds \emph{grids} and denote them by $g_1, g_2, \ldots g_{\lceil \frac{K}{w}\rceil-1}$, where $g_i\%w=0$ for all $i\in[\lceil \frac{K}{w}\rceil-1]$. Let $i_k$ be the grid index of time round $k$, \emph{i.e.}, $g_{i_k}\leq k<g_{i_k+1}$. We denote $\hat\Psi_{k, \ell}:=\{t: t\in[g_{i_k},k-1], \ell_t=\ell\}$.
We define the confidence radius $\hat\beta_{k,\ell}$ at round $k$ and layer %\todoq{layer, or level, need to be consistent} 
$\ell$ as

    \begin{align} 
    \hat \beta_{k, \ell} &:= 16 \cdot 2^{-\ell} \sqrt{\left(8\hat\Var_{k, \ell}  + 6R^2 \log(\frac{4(w + 1)^2 L}{\delta}) + 2^{-2\ell + 4}\right)} \notag\\&\quad \times\sqrt{\log(\frac{4w^2 L }{\delta} )}+ 6 \cdot 2^{-\ell} R \log(\frac{4w^2 L}{\delta}) + 2^{-\ell}B\label{eq:def:beta}, 
\end{align}

where we set $\hat\Var_{k, \ell}$ as $\sum_{i \in \hat\Psi_{k, \ell}} w_i^2 \big(r_i - \la \hat\btheta_{k, \ell}, \ab_i \ra \big)^2$, if $2^\ell \ge 64 \sqrt{\log\left(\frac{4(w + 1)^2 L}{\delta}\right)}$, or $R^2 \left|\hat\Psi_{k, \ell}\right|$ for the remaining cases.

% \begin{small}
%     $\hat\Var_{k, \ell} := \begin{cases}\sum_{i \in \hat\Psi_{k, \ell}} w_i^2 \big(r_i - \la \hat\btheta_{k, \ell}, \ab_i \ra \big)^2, \\\quad\text{If} \quad 2^\ell \ge 64 \sqrt{\log\left(\frac{4(w + 1)^2 L}{\delta}\right)} \\
% R^2 \left|\hat\Psi_{k, \ell}\right|,  \\ \quad\text{otherwise}.
% \end{cases} $
% \end{small}

% \begin{figure*}[t!]
%     \centering
%      \subfigure[$B_K=1$]{
%     \includegraphics[scale=0.25]{2Regret1.pdf}
%     }
%     \hfill
%      \subfigure[$B_K=10$]{
%     \includegraphics[scale=0.25]{2Regret10.pdf}
%     }
%     \hfill
%     \subfigure[$B_K=20$]{
%     \includegraphics[scale=0.25]
%     {2Regret20.pdf}}
%     \hfill
%     \subfigure[$B_K=K^{1/3}$]{
%     \includegraphics[scale=0.25]{2RegretT3.pdf}
%     }
%     \caption{The regret of Restarted-$\algbandit$, $\text{Restarted SAVE}^+$, SW-UCB and Modified EXP3.S under different total rounds.}
%     \label{fig:1}
% \end{figure*}

Note that our selection of the confidence radius $\hat\beta_{k, \ell}$ only depends on $\hat\Var_{k, \ell}$, which serves as an estimate of the total variance of samples at $\ell$-th layer without knowing $\sigma_k^2$. 

We build the theoretical guarantee of Algorithm \ref{alg:1} as follows.

 \begin{theorem} \label{thm:regret1}
 Let $0<\delta<1$. Define $\{\beta_{k, \ell}\}_{k\ge 1, \ell \in [L]}$ as in \eqref{eq:def:beta}, regarding $A, R$ as constants, we have
 % and $\alpha = 1 / (R \cdot K^{3 / 2})$,
 % then the cumulative regret of Algorithm~\ref{alg:1} is bounded as follows with probability at least $1 - 3\delta$: 
 %        \begin{align} 
 %            \text{Regret}(K)&= \tilde{O}\bigg(\frac{A^2\sqrt{d}w^{\frac{3}{2}}B_K}{\alpha}+\big(w\alpha^2+d\big)\cdot\sqrt{\frac{K}{w}V_K}\notag\\
 %            &\quad+\big(1+R\big)\cdot\big(K\alpha^2+\frac{Kd}{w}\big)\bigg)
 %        \end{align}
        \begin{align}
            \text{Regret}(K)&=\tilde O(\sqrt{d}w^{1.5}B_K/\alpha + \alpha^2(K+\sqrt{wKV_K}) \notag \\
            &\quad + d\sqrt{KV_K/w} + dK/w).\notag
        \end{align}
    \end{theorem}
    \begin{proof}
    See Appendix \ref{app:secalg} for the full proof.
\end{proof}
    \begin{remark}
    Like Remark \ref{rmk:1}, we consider the case where $B_K = 0$. We set $w = K$ and $\alpha^2 = 1/K\sqrt{V_K}$, then we obtain a regret $\tilde O(d\sqrt{V_K} +d)$, which matches the regret of the SAVE algorithm in \cite{zhao2023variance}. 
\end{remark}
% We have the following corollary to optimize the regret if we have no knowledge of $B_K$ and $V_K$. 
% \begin{corollary}
% By selecting 
% \begin{align}
%     w=(dK)^{1/3}, \alpha=d^{1/6}\sqrt{w}/K^{1/3}=d^{1/3}/K^{1/6},
% \end{align}
% we have
% \begin{align}
% &\sqrt{d}\sqrt{dK}B_KK^{2/3}/d^{1/3} + d^{2/3}K^{2/3} + d^{2/3}/K^{1/3}d^{1/6}K^{1/6}\sqrt{KV_K} + d^{5/6}K^{1/3}\sqrt{V_K} + d^{2/3}K^{2/3}\notag \\
% &=d^{2/3}K^{2/3}B_K + d^{5/6}K^{1/3}\sqrt{V_K}
% \end{align}
% \end{corollary}
\begin{figure*}[h]
    \centering
     \subfigure[$B_K=1$]{
    \includegraphics[scale=0.24]{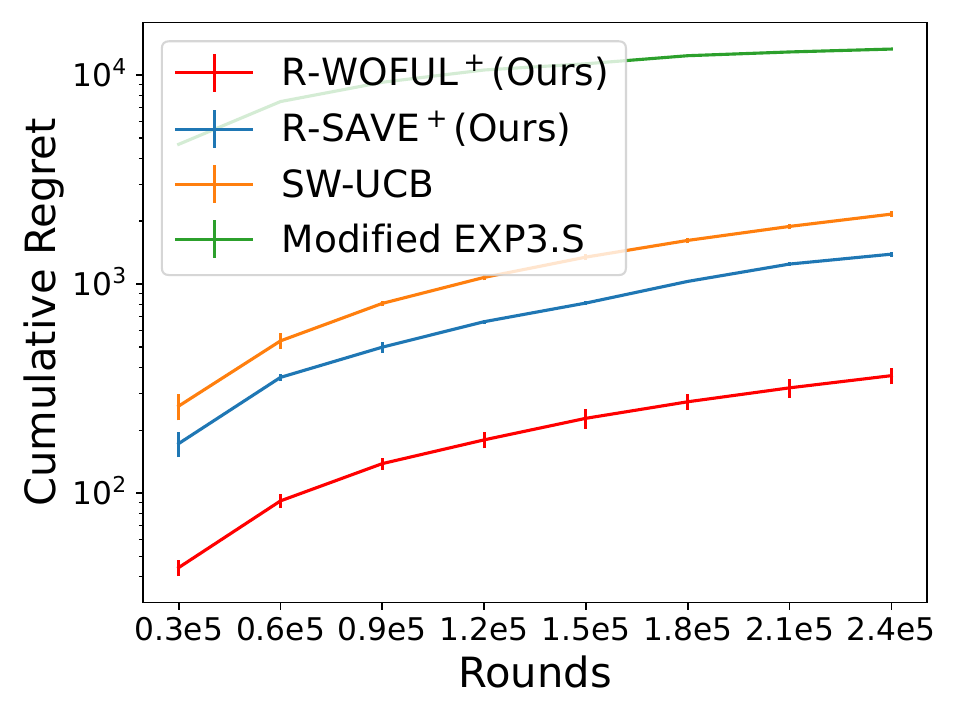}
    }
    \hfill
     \subfigure[$B_K=10$]{
    \includegraphics[scale=0.24]{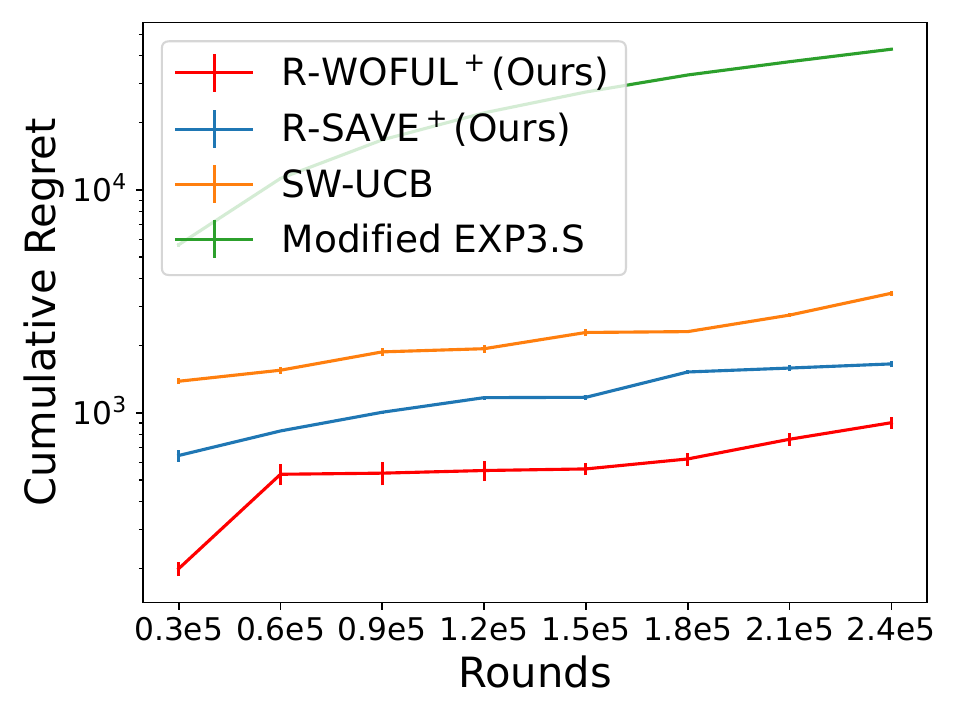}
    }
    \hfill
    \subfigure[$B_K=20$]{
    \includegraphics[scale=0.24]
    {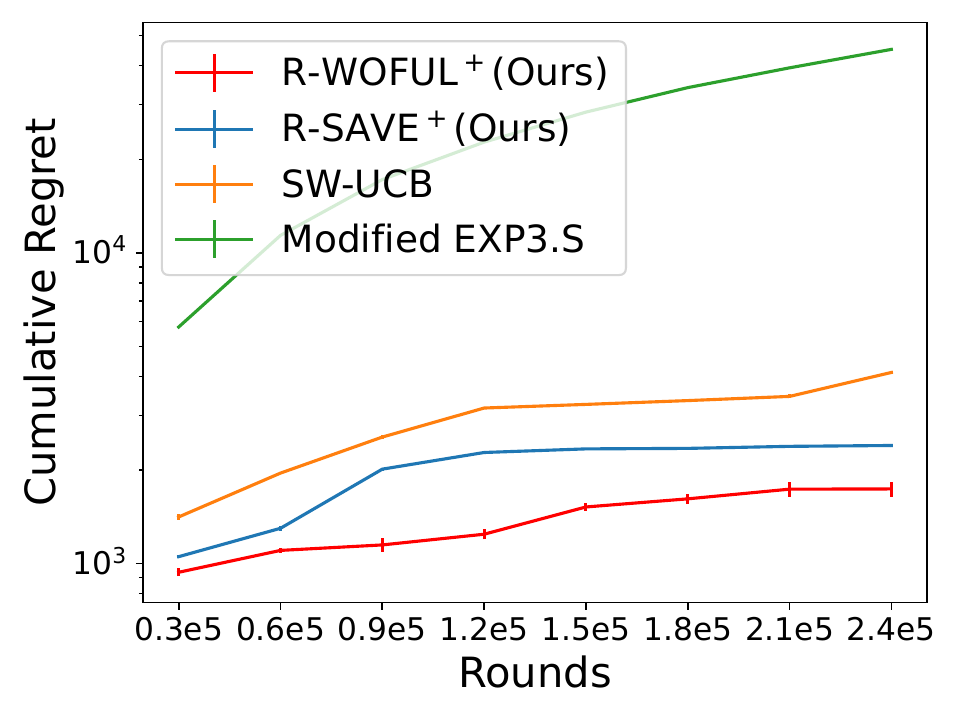}}
    \hfill
    \subfigure[$B_K=K^{1/3}$]{
    \includegraphics[scale=0.24]{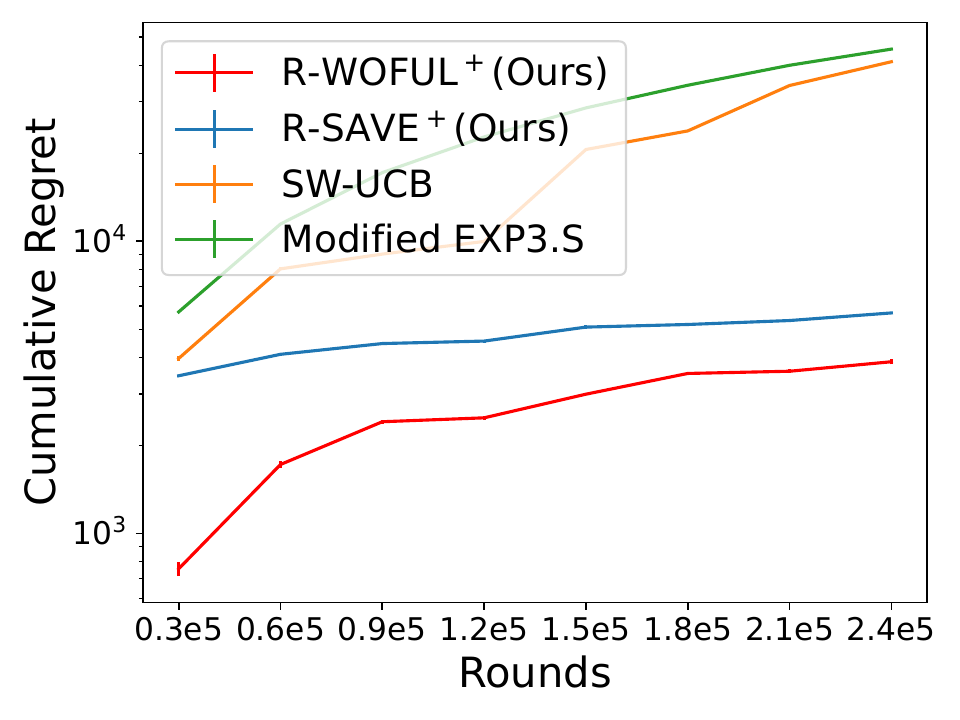}
    }
    \caption{The regret of Restarted-$\algbandit$, $\text{Restarted SAVE}^+$, SW-UCB and Modified EXP3.S under different total rounds.}
    \label{fig:1}
\end{figure*}
\begin{corollary}
    Assume that $B_K, V_K \in [\Omega(1), O(K)]$, then by selecting
\begin{align}
        &w=d^{1/3}(K/B_K)^{1/3}, &K^2\geq V_K^3d/B_K,\notag\\
        &w=d^{2/5}(KV_K)^{1/5}/B_K^{2/5}&\text{otherwise}.\notag
    \end{align}
    and $\alpha = d^{1/6}\sqrt{w}B_K^{1/3}/(K^{1/3} + (V_KKw)^{1/6})$, we have
    \begin{align*}
        \text{Regret}(K) = \tilde O(d^{4/5}V_K^{2/5}B_K^{1/5}K^{2/5} + d^{2/3}B_K^{1/3}K^{2/3}).
    \end{align*}\label{corollary w alpha opt}
\end{corollary}
\begin{remark}
We discuss the regret of Algo.\ref{alg:1} in Corollary \ref{corollary w alpha opt} in the following special cases. In the case where the \emph{total variance} is small, \emph{i.e.}, $V_K=\Tilde{O}(1)$,  assuming that $K^2> d$, our result becomes $\Tilde{O}(d^{2/3}B_K^{1/3}K^{2/3})$, better than all the previous results \cite{cheung2018hedging,zhao2020simple,wang2023revisiting,wei2021non}. In the worst case where $V_K=O(K)$, our result becomes $\Tilde{O}(d^{4/5}B_K^{1/5}K^{4/5})$.
\end{remark}

% \zhiyong{If we use BOB, we can get     \begin{align}
%         \text{Regret}(K) = \tilde O(d^{4/5}V_K^{2/5}B_K^{1/5}K^{2/5} + d^{2/3}B_K^{1/3}K^{2/3}+\sqrt{wK})
%     \end{align}}

\noindent\textbf{Unknown Per-round Variance, Unknown $V_K$ and $B_K$} In Corollary \ref{corollary w alpha opt}, we need to know the \emph{total variance} $V_K$ and \emph{total variation budget} $B_K$ to select the optimal $w$ and $\alpha$. To deal with the more general case where $V_K$ and $B_K$ are unknown, we can employ the \emph{Bandits-over-Bandits} (BOB) mechanism (\cite{cheung2019learning,wang2023revisiting,zhao2020simple}). We name the Restarted $\text{SAVE}^+$ algorithm with BOB mechanism as “Restarted $\text{SAVE}^+$-BOB”. Due to the space limit, we put the algorithm design, descriptions, and theoretical analysis of Restarted $\text{SAVE}^+$-BOB (Algo.\ref{alg:bob algo}) in Appendix \ref{app:bob}.

\section{Experiments}
\label{sec:experiments}

% \begin{figure}[htp]
%     \centering
%     \includegraphics[scale=0.4]{icml2024/Regret.pdf}
%     \caption{Comparison of Restarted-$\algbandit$ with SW-UCB and Modified EXP3.S when $B_K=1$.} 
%     \label{fig:1}
% \end{figure}
% \begin{figure}[htp]
%     \centering
%     \includegraphics[scale=0.4]{icml2024/RegretT3.pdf}
%     \caption{Comparison of Restarted-$\algbandit$ with SW-UCB and Modified EXP3.S when $B_K=K^{1/3}$.} 
%     \label{fig:2}
% \end{figure}
% \begin{figure}[htp]
%     \centering
%     \includegraphics[scale=0.4]{icml2024/Regret10.pdf}
%     \caption{Comparison of Restarted-$\algbandit$ with SW-UCB and Modified EXP3.S when $B_K=10$.} 
%     \label{fig:3}
% \end{figure}
% \begin{figure}[htp]
%     \centering
%     \includegraphics[scale=0.4]{icml2024/Regret20.pdf}
%     \caption{Comparison of Restarted-$\algbandit$ with SW-UCB and Modified EXP3.S when $B_K=20$.} 
%     \label{fig:4}
% \end{figure}
To validate the effectiveness of our methods, we conduct a series of experiments on the synthetic data.

\noindent\textbf{Problem Setting and Baselines}
Following the experimental set up in \cite{cheung2019learning}, we consider the 2-armed bandits setting, where the action set $\cD_k = \{(1,0),(0,1)\}$, and 
\begin{align}
\btheta_{k} = \begin{pmatrix}
    0.5 + \frac{3}{10}\sin(5B_{K}\pi k/K)\\
    0.5 + \frac{3}{10}\sin(\pi + 5B_{K}\pi k/K)
\end{pmatrix}.\notag
\end{align}
It is easy to see that the total variation budget can be bounded as $B_K$. At each round $k$, the $\epsilon_{k}$ satisfies the following distribution:
\begin{align}
    \epsilon_k \sim \text{Bernoulli}(0.5/k)-0.5/k.\notag
\end{align}
We can verify that under such a distribution for $\epsilon_k$, the variance of the reward distribution at $k$-th round is $(1-0.5/k)\cdot 0.5/k$, and the total variance $V_K \sim \log K$. 
% The results are shown in Figure.\ref{fig:1}, which are consistent with our theoretical findings. And it's obvious that our algorithms significantly outperform SW-UCB and Modified EXP3.S. The Restarted-$\algbandit$ performs best because it knows the variance and can make more informative decision , the performance of $\text{Restarted SAVE}^+$ is slightly worse than Restarted-$\algbandit$ but still better than the baselines, especially when $B_{K}=K^{1/3}$. 

We compare the proposed Restarted-$\algbandit$ and\\ $\text{Restarted SAVE}^+$ with SW-UCB \cite{cheung2019learning} and Modified EXP3.S \cite{besbes2014optimal}. We leave the detailed setup for the baselines in Appendix \ref{app:addexp}.

\noindent\textbf{Result}
We plot the results in Figure.\ref{fig:1}, where all the empirical results are averaged over ten independent trials and the error bar is the standard error divided by $\sqrt{10}$. The results are consistent with our theoretical findings. It is evident that our algorithms significantly outperform both SW-UCB and Modified EXP3.S. Among our proposed algorithms, \\Restarted-$\algbandit$ achieves the best performance. This can be attributed to the fact that it knows the variance and can make more informed decisions. Although $\text{Restarted SAVE}^+$ performed slightly worse than Restarted-$\algbandit$, it still outperforms the baseline algorithms, particularly when $B_{K}=K^{1/3}$. These results highlight the superiority of our methods.

\chapter{Online Clustering of Dueling Bandits}\label{chapter: dueling}
The contextual multi-armed bandit (MAB) is a widely used framework for problems requiring sequential decision-making under uncertainty, such as recommendation systems. In applications involving a large number of users, the performance of contextual MAB can be significantly improved by facilitating collaboration among multiple users. This has been achieved by the clustering of bandits (CB) methods, which adaptively group the users into different clusters and achieve collaboration by allowing the users in the same cluster to share data. However, classical CB algorithms typically rely on numerical reward feedback, which may not be practical in certain real-world applications.  For instance, in recommendation systems, it is more realistic and reliable to solicit \textit{preference feedback} between pairs of recommended items rather than absolute rewards. To address this limitation, we introduce the first "clustering of dueling bandit algorithms" to enable collaborative decision-making based on preference feedback. We propose two novel algorithms: (1) Clustering of Linear Dueling Bandits (COLDB) which models the user reward functions as linear functions of the context vectors, and (2) Clustering of Neural Dueling Bandits (CONDB) which uses a neural network to model complex, non-linear user reward functions. Both algorithms are supported by rigorous theoretical analyses, demonstrating that user collaboration leads to improved regret bounds. Extensive empirical evaluations on synthetic and real-world datasets further validate the effectiveness of our methods, establishing their potential in real-world applications involving multiple users with preference-based feedback.

\section{Introduction}

The contextual multi-armed bandit (MAB) is a widely used method in real-world applications requiring sequential decision-making under uncertainty, such as recommendation systems, computer networks, among others \cite{li2010contextual}.
In a contextual MAB problem, a user faces a set of $K$ arms (i.e., context vectors) in every round, selects one of these $K$ arms, and then observes a corresponding numerical reward \cite{lattimore2020bandit}.
In order to select the arms to maximize the cumulative reward (or equivalently minimize the cumulative regret), we often need to consider the trade-off between the \emph{exploration} of the arms whose unknown rewards are associated with large uncertainty and \emph{exploitation} of the available observations collected so far.
% In order to intelligently select the arms in maximize the cumulative rewards,
To carefully handle this trade-off, we often model the reward function using a surrogate model, such as a linear model \cite{chu2011contextual} or a neural network \cite{zhou2020neural}.

Some important applications of contextual MAB, such as recommendation systems, often involve a large number (e.g., in the scale of millions) of users, which opens up the possibility of further improving the performance of contextual MAB via user collaboration.
To this end, the method of \emph{online Clustering of Bandits} (CB) has been proposed, which adaptively partitions the users into a number of clusters and leverages the collaborative effect of the users in the same cluster to achieve improved performance \cite{gentile2014online,wang2024onlinea,10.5555/3367243.3367445}.

Classical CB algorithms usually require an absolute real-valued numerical reward as feedback for each arm \cite{wang2024onlinea}. However, in some crucial applications of contextual MAB, it is often more realistic and reliable to request the users for \emph{preference feedback}.
For example, in recommendation systems, it is often preferable to recommend a pair of items to a user and then ask the user for relative feedback (i.e., which item is preferred) \cite{JCSS12_yue2012k}.
As another example, contextual MAB has been successfully adopted to optimize the input prompt for large language models (LLMs), which is often referred to as \emph{prompt optimization} \cite{lin2024prompt,lin2023instinct}.
In this application, instead of requesting an LLM user for a numerical score as feedback, it is more practical to show the user a pair of LLM responses generated by two candidate prompts and ask the user which response is preferred \cite{lin2024prompt,verma2024neural}.

% To incorporate preference feedback into MAB, the framework of \emph{contextual dueling bandits} has been developed
A classical and principled approach to account for preference feedback in contextual MAB is the framework of contextual \emph{dueling bandit} 
\cite{NeurIPS21_saha2021optimal,ICML22_bengs2022stochastic,ALT22_saha2022efficient,arXiv24_li2024feelgood}.
In every round of contextual dueling bandits, a pair of arms are selected, after which a binary observation is collected reflecting which arm is preferred.
However, classical dueling bandit algorithms are not able to leverage the collaboration of multiple users, which leaves significant untapped potential to further improve the performance in these applications involving preference feedback.
In this work, we bring together the merits of both approaches, and hence introduce the first \emph{clustering of dueling bandit} algorithms, enabling multi-user collaboration in scenarios involving preference feedback.

We firstly proposed our \emph{Clustering Of Linear Dueling Bandits} (COLDB) algorithm (Sec.~\ref{subsec:algo:coldb}), which assumes that the latent reward function of each user is a linear function of the context vectors (i.e., the arm features). In addition, to handle challenging real-world scenarios with complicated non-linear reward functions, we extend our COLDB algorithm to use a \emph{neural network to model the reward function}, hence introducing our \emph{Clustering Of Neural Dueling Bandits} (CONDB) algorithm (Sec.~\ref{subsec:algo:condb}).
Both algorithms adopt a graph to represent the estimated clustering structure of all users, and adaptively update the graph to iteratively refine the estimate. After receiving a user in every round, our both algorithms firstly assign the user to its estimated cluster, and then leverage the data from all users in the estimated cluster to learn a linear model (COLDB) or a neural network (CONDB), which is then used to select a pair of arms for the user to query for preference feedback. After that, we update the reward function estimate for the user based on the newly observed feedback, and then update the graph to remove its connection with users who are estimated to belong to a different cluster.

We conduct rigorous theoretical analysis for both our COLDB and CONDB algorithms, and our theoretical results demonstrate that the regret upper bounds of both algorithms are sub-linear and that a larger degree of user collaboration (i.e., when a larger number of users belong to the same cluster on average) leads to theoretically guaranteed improvement (Sec.~\ref{sec:theory}).
In addition, we also perform both synthetic and real-world experiments to demonstrate the practical advantage of our algorithms and the benefit of user collaboration in contextual MAB problems with preference feedback (Sec.~\ref{sec:experiments}).

% {\color{blue}
% \begin{itemize}
%     \item Contextual bandits is a widely use paradigm to solve sequential decision-making tasks, such as recommdender systems
%     \item To facilitate colalboration among multiple users, clustering of bandits has been proposed
%     \item Classical works on clustering of bandits all require the users to provide numerical reward feedback. However, in some important applications of contextual bandits, it is more common to let users provide relative preference feedback, e.g., which recommended item is preferred by the user or which prompt generated the better response.
%     \item To this end, in this work, we propose the problem of \emph{clustering of dueling bandits}.
%     \item We firstly proposed the clustering of linear bandits (COLDB) algorithm, which assumes that the (latent) reward functions of each user is a linear function. Then, to account for complicated non-linear reward functions, we extend our COLDB algorithm to acocunt for arbitrary non-linear reward functions, and hence propose our clustering of neural dueling bandits (CONDB) algorithm.
% \end{itemize}
% }

\section{Problem Setting}\label{sec: setting}

This section formulates the problem of \emph{clustering of dueling bandits}. In the following, we use boldface lowercase letters for vectors and boldface uppercase letters for matrices. The number of elements in a set \( \mathcal{A} \) is denoted as \( |\mathcal{A}| \), while \( [m] \) refers to the index set \( \{1, 2, \dots, m\} \), and \( \norm{\boldsymbol{x}}_{\boldsymbol{M}} = \sqrt{\boldsymbol{x}^{\top}\boldsymbol{M}\boldsymbol{x}} \) represents the matrix norm of vector \( \boldsymbol{x} \) with respect to the positive semi-definite (PSD) matrix \( \boldsymbol{M} \).

\textbf{Clustering Structure.}
Consider a scenario with \( u \) users, indexed by \( \mathcal{U} = \{1, 2, \dots, u\} \), where each user \( i \in \mathcal{U} \) 
% has an unknown reward function 
% Here we assume that every user $i\in\mathcal{U}$ 
is associated with a unknown 
% non-linear 
reward function $f_i: \mathbb{R}^{d'} \rightarrow \mathbb{R}$ which maps an arm $\bx \in \mathcal{X}\subset\mathbb{R}^{d'}$ to its corresponding reward value $f_i(\bx)$.
% s.t.~$|f_i(\bx)|\leq 1,\forall \bx\in\mathcal{X}$.
% preference vector \( \btheta_i \in \mathbb{R}^d \) satisfying \( \norm{\btheta_i}_2 \leq 1 \). 
We assume that there exists an underlying, yet unknown, clustering structure over the users reflecting their behavior similarities. Specifically, the set of users \( \mathcal{U} \) is partitioned into \( m \) clusters \( C_1, C_2, \dots, C_m \), where \( m \ll u \), and the clusters are mutually disjoint: \( \cup_{j \in [m]} C_j = \mathcal{U} \) and \( C_j \cap C_{j'} = \emptyset \) for \( j \neq j' \). These clusters are referred to as \gtclusters{}, and the set of clusters is denoted by \( \mathcal{C} = \{C_1, C_2, \dots, C_m\} \). 
Let $f^j$ denote the common reward function of all users in cluster $j$ and let \( j(i) \in [m] \) be the index of the cluster to which user \( i \) belongs.
If two users $i$ and $l$ belong to the same cluster, they have the same reward function.
% All users in the same cluster share the 
% % same preference vector, 
% reward function, 
That is, for any $\ell \in \mathcal{U}$, if $\ell \in C_{j(i)}$, then $f_\ell = f_i = f^{j(i)}$.
Meanwhile, users from different clusters have distinct 
% preference vectors. 
reward functions.
% Denote \( \btheta^j \) as the common preference vector of users in cluster \( C_j \), and 
% let \( j(i) \in [m] \) be the index of the cluster to which user \( i \) belongs. 
% Therefore, for any \( \ell \in \mathcal{U} \), if \( \ell \in C_{j(i)} \), then \( \btheta_\ell = \btheta_i = \btheta^{j(i)} \).

\textbf{Modeling Preference Feedback.}
At each time step \( t \in [T] \), a user \( i_t \in \mathcal{U} \) is served. The learning agent observes a set of context vectors (i.e., arms) \( \cX_t \subseteq \cX \subset \mathbb{R}^{d'} \), where \( \left|\cX_t\right| = K \leq C \) for all \( t \).
Each arm \( \bx \in \cX_t \) is a feature vector in \( \mathbb{R}^{d'} \) with \( \norm{\bx}_2 \leq 1 \). The agent assigns the cluster \( \overline{C}_t \) to user \( i_t \) and recommends two arms \( \bx_{t,1}, \bx_{t,2} \in \cX_t \) based on the aggregated historical data from cluster \( \overline{C}_t \). 
After receiving the recommended pair of arms, the user provides a binary preference feedback \( y_t \in \{0, 1\} \), in which $y_t=1$ if $\bx_{t,1}$ is preferred over $\bx_{t,2}$ and $y_t=0$ otherwise.
% Each user \( i \in \mathcal{U} \) is associated with a latent reward function \( f_i: \mathbb{R}^{d'} \to \mathbb{R} \), which maps an arm \( \bx \) to a corresponding reward value \( f_i(\bx) \).
We model the binary preference feedback following the widely used Bradley-Terry-Luce (BTL) model \cite{AS04_hunter2004mm,Book_luce2005individual}.
% The reward feedback \( y_t \) for the pair of arms \( \bx_{t,1} \) and \( \bx_{t,2} \) at round \( t \) is generated according to the standard Bradley-Terry-Luce (BTL) model. 
Specifically, the BTL model assumes that for user $i_t$,
% the binary feedback \( y_t \) is sampled from a Bernoulli distribution with probability:
the probability that the first arm $\bx_{t,1}$ is preferred over the second arm $\bx_{t,2}$ is given by
\[
\mathbb{P}_t(\bx_{t,1} \succ \bx_{t,2}) = \mu(f_{i_t}(\bx_{t,1}) - f_{i_t}(\bx_{t,2})),
\]
where \( \mu: \mathbb{R} \to [0, 1] \) is the logistic function: \( \mu(z) = \frac{1}{1+e^{-z}} \). 
In other words, the binary feedback $y_t$ is sampled from the Bernoulli distribution with the probability $\mathbb{P}_t(\bx_{t,1} \succ \bx_{t,2})$.
% The observed reward feedback also contains noise, denoted \( \epsilon_t \), where \( \epsilon_t \) is 1-Sub-Gaussian, resulting in:
% \[
% y_t = \mu(f_{i_t}(\bx_{t,1}) - f_{i_t}(\bx_{t,2})) + \epsilon_t.
% \]

% Following the common practice in the theoretical analysis of generalized linear bandits and dueling bandits \cite{ICML17_li2017provably,ICML22_bengs2022stochastic}, 
We make the following assumption about the preference model:
\begin{assumption}[Standard Dueling Bandits Assumptions]
\label{assumption4}
1. $|\mu(f(\bx)) - \mu(g(\bx))| \le L_\mu|f(\bx) - g(\bx)|, \forall x\in\mathcal{X}$ , for any functions $f,g: \mathbb R^{d'} \rightarrow \mathbb R$.\\
% 2. $\forall \bx \in \mathcal{X}: \min \nabla\mu(f(\bx)) \ge \kappa_\mu > 0.$
2. $\min_{\bx \in \mathcal{X}} \nabla\mu(f(\bx)) \ge \kappa_\mu > 0.$
\end{assumption}
Assumption \ref{assumption4} is the standard assumption in the analysis of linear bandits and dueling bandits \cite{ICML17_li2017provably,ICML22_bengs2022stochastic}, and when $\mu$ is the logistic function, $L_\mu = 1/4$.
% We analyze the following commonly adopted notion of regret for dueling bandits:
The regret incurred by the learning agent is defined as:
\[
R_T = \sum_{t=1}^{T} r_t = \sum_{t=1}^{T} \left( 2 f_{i_t}(\bx^*_t) - f_{i_t}(\bx_{t,1}) - f_{i_t}(\bx_{t,2}) \right),
\]
where \( \bx^*_t = \arg\max_{\bx \in \mathcal{X}_t} f_{i_t}(\bx) \) represents the optimal arm at round \( t \).
This is a commonly adopted notion of regret in the analysis of dueling bandits \cite{ICML22_bengs2022stochastic,ALT22_saha2022efficient}.

\subsection{Clustering of Linear Dueling Bandits}
\label{subsec:problem:setting:linear}
For the linear setting, we assume that each reward function \( f_i \) is linear in a fixed feature space \( \phi(\cdot) \), such that \( f_i(\bx) = \btheta_i^{\top} \phi(\bx),\forall \bx\in\mathcal{X} \). 
The feature mapping \( \phi: \mathbb{R}^{d'} \to \mathbb{R}^d \) is a fixed mapping with \( \norm{\phi(\bx)}_2 \leq 1 \) for all \( \bx \in \cX \). In the special case of classical linear dueling bandits, we have that \( \phi(\bx) = \bx \), i.e., $\phi(\cdot)$ is the identity mapping. The use of \( \phi(\bx) \) enables us to potentially model non-linear reward functions given an appropriate feature mapping.

In this case, the reward function of every user $i$ is represented by its corresponding \emph{preference vector} $\btheta_i$, and all users in the same cluster share the same preference vector while users from different clusters have distinct preference vectors. 
Denote \( \btheta^j \) as the common preference vector of users in cluster \( C_j \), and let \( j(i) \in [m] \) be the index of the cluster to which user \( i \) belongs. Therefore, for any \( \ell \in \mathcal{U} \), if \( \ell \in C_{j(i)} \), then \( \btheta_\ell = \btheta_i = \btheta^{j(i)} \).

The following assumptions are made regarding the clustering structure, users, and items:
% , and the dueling bandit model:
\begin{assumption}[Cluster Separation]
\label{assumption1}
The preference vectors of users from different clusters are at least separated by a constant gap \( \gamma > 0 \), i.e.,
\[
\norm{\btheta^{j} - \btheta^{j'} }_2 \geq \gamma \quad \text{for all} \quad j \neq j' \in [m].
\]
\end{assumption}

\begin{assumption}[Uniform User Arrival]
\label{assumption2}
At each time step \( t \), the user \( i_t \) is selected uniformly at random from \( \mathcal{U} \), with probability \( 1/u \), independent of previous rounds.
\end{assumption}

\begin{assumption}[Item regularity]
\label{assumption3}
At each time step $t$, the feature vector $\phi(\bx)$ of each arm $\bx\in \mathcal{X}_t$ is drawn independently from a fixed but unknown distribution $\rho$ over $\{\phi(\bx)\in\RR^d:\norm{\phi(\bx)}_2\leq1\}$, where 
$\EE_{\bx\sim \rho}[\phi(\bx) \phi(\bx)^{\top}]$ 
is full rank with minimal eigenvalue $\lambda_x > 0$. Additionally, at any time $t$, for any fixed unit vector $\btheta \in \RR^d$, $(\btheta^{\top}\phi(\bx))^2$ has sub-Gaussian tail with variance upper bounded by $\sigma^2$.
\end{assumption}

\noindent\textbf{Remark 1.} All these assumptions above
% above standard assumptions 
follow the previous works on clustering of bandits \cite{gentile2014online,gentile2017context,
li2018online,
ban2021local,
liu2022federated,wang2024onlinea,wang2024onlineb}.
% and dueling bandits \zhiyong{References to be added}. 
For Assumption \ref{assumption2}, our results can easily generalize to the case where the user arrival follows any distribution with minimum arrival probability 
% greater than 
$\geq p_{min}$. 

% \zhiyong{The following part is to be added after finishing the algorithm and proof of the neural case. We might need to rewrite this section in a unified way later.}

\subsection{Clustering of Neural Dueling Bandits}
\label{subsec:problem:setting:neural}
Here we allow the reward functions $f_i$'s 
% for every user $i$ 
to be non-linear functions.
% assume that every user $i\in\mathcal{U}$ is associated with a unknown non-linear reward function $f_i: \mathbb{R}^{d'} \rightarrow \mathbb{R}$ 
% as long as $|f_i(\bx)|\leq 1,\forall \bx\in\mathcal{X}$.
% Let $f^j$ denote the common reward function of all users in cluster $j$.
% For any $\ell \in \mathcal{U}$, if $\ell \in C_{j(i)}$, then $f_\ell = f_i= f^{j(i)}$.
% \begin{assumption}[Item regularity (neural bandits)]
% ...
% % At each time step $t$, the feature vector $\phi(\bx)$ of each arm $\bx\in \mathcal{X}_t$ is drawn independently from a fixed but unknown distribution $\rho$ over $\{\phi(\bx)\in\RR^d:\norm{\phi(\bx)}_2\leq1\}$, where $\EE_{\phi(\bx)\sim \rho}[\phi(\bx) \phi(\bx)^{\top}]$ is full rank with minimal eigenvalue $\lambda_x > 0$. Additionally, at any time $t$, for any fixed unit vector $\btheta \in \RR^d$, $(\btheta^{\top}\phi(\bx))^2$ has sub-Gaussian tail with variance upper bounded by $\sigma^2$.
% \end{assumption}
To estimate the unknown reward functions $f_i$'s, we use fully connected neural networks (NNs) with 
% $L \ge 2$ layers, the width of hidden layer $m_{\text{NN}}$, and 
ReLU activations, and denote the depth and width (of every layer) of the NN by $L\geq 2$ and $m_{\text{NN}}$, respectively \cite{zhou2020neural,zhang2020neural}.
Let $h(\bx;\theta)$ represent the output of an NN with parameters $\btheta$ and input vector $\bx$, which is defined as follows:
\[
    h(\bx;\btheta) = \mathbf{W}_L \text{ReLU}\left( \mathbf{W}_{L-1} \text{ReLU}\left( \cdots \text{ReLU}\left(\mathbf{W}_1 \bx\right) \right) \right),
\]
in which $\text{ReLU}(\bx) = \max\{ \bx, 0 \}$, $\mathbf{W}_1 \in \mathbb{R}^{m_{\text{NN}} \times d}$, $\mathbf{W}_l \in \mathbb{R}^{m_{\text{NN}} \times m_{\text{NN}}}$ for $2 \le l < L$, $\mathbf{W}_L \in \mathbb{R}^{1\times m_{\text{NN}}}$. 
We denote the parameters of NN by $\btheta = \left( \text{vec}\left( \mathbf{W}_1 \right);\cdots \text{vec}\left( \mathbf{W}_L \right) \right)$, where $\text{vec}\left( A \right)$ converts an $M \times N$ matrix $A$ into a $MN$-dimensional vector.
We 
% use $m$ to denote the width of every layer of the NN, 
use $p$ to denote the total number of NN parameters: $p = dm_{\text{NN}} + m_{\text{NN}}^2(L-1) + m_{\text{NN}}$, and use $g(\bx;\btheta)$ to denote the gradient of $h(\bx;\btheta)$ with respect to $\btheta$.
% In our CONDB algorithm (Sec.~\ref{subsec:algo:condb}), every user uses an NN to estimate its reward function $f_{i}$. 

The algorithmic design and analysis of neural bandit algorithms make use of the theory of the \emph{neural tangent kernel} (NTK) \cite{jacot2018neural}.
We let all $u$ users use the same initial NN parameters $\btheta_0$, and assume that the value of the \emph{empircal NTK} is bounded: $\frac{1}{m_{\text{NN}}}\langle g(\bx;\btheta_0), g(\bx;\btheta_0) \rangle \leq 1,\forall \bx \in \mathcal{X}$.
This is a commonly adopted assumption in the analysis of neural bandits \cite{ICLR23_dai2022federated,kassraie2021neural}. 
% Note that replacing $1$ by an absolute constant $c_0$ will only lead to an extra constant in the final regret bound.
Let $T^j$ denote total number of rounds in which the users in cluster $j$ is served. 
We use $\mathbf{H}_j$ to denote the \emph{NTK matrix} \cite{zhou2020neural} for cluster $j$, which is a $(T_j K) \times (T_j K)$-dimensional matrix.
Similarly, we define $\mathbf{h}_j$ as the $(T_j K)\times 1$-dimensional vector containing the reward function values of all $T_j K$ arm feature vectors for cluster $j$.
We provide the concrete definitions of $\mathbf{H}_j$ and $\mathbf{h}_j$ in App.~\ref{app:subsec:aux:defs}.
We make the following assumptions which are commonly adopted by previous works on neural bandits \cite{zhou2020neural,zhang2020neural},
for which we provide justifications in App.~\ref{app:subsec:aux:defs}.
\begin{assumption}
\label{assumption:main:neural}
The reward functions for all users are bounded: $|f_i(x)| \leq 1,\forall x\in\mathcal{X},\forall i\in\mathcal{U}$. 
There exists $\lambda_0 > 0$ s.t.~$\mathbf{H}_j \succeq \lambda_0 I, \forall j\in\mathcal{C}$. 
All 
% context-
arm feature vectors satisfy $\norm{x}_{2}=1$ and $x_{j}=x_{j+d/2}$, $\forall x\in\mathcal{X}_{t},\forall t\in[T]$.
\end{assumption}
% We provide justifications for these assumptions in App.~\ref{app:subsec:aux:defs}.

% {\color{blue}
% We now introduce the assumptions needed for our regret analysis, all of which are standard assumptions in neural bandits \cite{zhou2020neural,zhang2020neural}. 
% \begin{assumption}
% \label{assumption:main}
% The reward functions for all users are bounded: $|f_i(x)| \leq 1,\forall x\in\mathcal{X},\forall i\in\mathcal{U}$. 
% There exists $\lambda_0 > 0$ s.t.~$\mathbf{H}_j \succeq \lambda_0 I, \forall j\in\mathcal{C}$. 
% All context-arm feature vectors satisfy $\norm{x}_{2}=1$ and {\color{blue}$x_{j}=x_{j+d/2}$}, $\forall x\in\mathcal{X}_{t},\forall t\in[T]$.
% \end{assumption}
% The last assumption in \cref{assumption:main} above, together with the way we initialize $\theta_0$ (i.e., following standard practice in neural bandits \cite{zhou2020neural,zhang2020neural}), ensures that $h(x;\theta_0)=0,\forall x\in\mathcal{X}_{t},\forall t\in[T]$.
% The assumption of {\color{blue}$x_{j}=x_{j+d/2}$} is a mild assumption and commonly used in the neural bandits literature \cite{zhou2020neural,zhang2020neural}. This assumption is just for convenience in regret analysis: for any context $x$, $||x|| = 1$, we can always construct a new context {\color{blue}$x' = (x^\top,x^\top)^\top/\sqrt{2}$} that satisfies this assumption \cite{zhou2020neural}.
% }

Denote by \( f^j \) the common reward function of the users in cluster \( C_j \), and let \( j(i) \in [m] \) be the index of the cluster to which user \( i \) belongs. 
Same as Sec.~\ref{subsec:problem:setting:linear}, here all users in the same cluster share the same reawrd function.
Therefore, for any \( \ell \in \mathcal{U} \), if \( \ell \in C_{j(i)} \), then \( f_\ell(\bx) = f_i(\bx) = f^{j(i)}(\bx),\forall \bx\in\mathcal{X} \).
% We use 
The following lemma shows that when the NN is wide enough (i.e., $m_{\text{NN}}$ is large), the reward function of every cluster can be modeled by a linear function.
\begin{lemma}[Lemma B.3 of \cite{zhang2020neural}]
\label{lemma:linear:utility:function:informal}
As long as the width $m_{\text{NN}}$ of the NN is large: $m_{\text{NN}} \geq \text{poly}(T, L, K, 1/\kappa_\mu, L_\mu, 1/\lambda_0, 1/\lambda, \log(1/\delta))$,
then for all clusters $j\in[m]$,
with probability of at least $1-\delta$, there exits a $\btheta^j_{f}$ such that 
\begin{align*}
	f^j(\bx) &= \langle g(\bx;\btheta_0), \btheta^j_{f} - \btheta_0 \rangle, \\
    \sqrt{m_{\text{NN}}} \norm{\btheta^j_{f} - \btheta_0}_2 &\leq \sqrt{2\mathbf{h}_j^{\top} \mathbf{H}_j^{-1} \mathbf{h}_j} \leq B,
\end{align*}
for all $\bx\in\mathcal{X}_{t}$, $t\in[T]$ with $i_t\in C_{j}$.
% for all , $t\in[T]$ with $i_t=i$.
% \begin{align*}
% 	f_i(\bx) &= \langle g(\bx;\btheta_0), \btheta_{f,i} - \btheta_0 \rangle, \\
%     \sqrt{m_{\text{NN}}} \norm{\btheta_{f,i} - \btheta_0}_2 &\leq \sqrt{2\mathbf{h}_{i}^{\top} \mathbf{H}_{i}^{-1} \mathbf{h}_{i}} \leq B,
% \end{align*}
% for all $\bx\in\mathcal{X}_{t}$, $t\in[T]$ with $i_t=i$.
\end{lemma}
We provide the detailed statement of Lemma \ref{lemma:linear:utility:function:informal} in Lemma \ref{lemma:linear:utility:function} (App.~\ref{app:subsec:proof:neural:real:proof}).
For a user $i$ belonging to cluster $j(i)$, we let $\btheta_{f,i}=\btheta^{j(i)}_{f}$, then we have that $f_i(\bx) = \langle g(\bx;\btheta_0), \btheta_{f,i} - \btheta_0 \rangle,\forall \bx\in\mathcal{X}$.
As a result of Lemma \ref{lemma:linear:utility:function:informal}, for any \( \ell \in \mathcal{U} \), if \( \ell \in C_{j(i)} \), we have that \( \btheta_{f,\ell} = \btheta_{f,i} = \btheta^{j(i)},\forall \bx\in\mathcal{X} \).

The assumption below formalizes the gap between different clusters in a similar way to Assumption \ref{assumption1}.
\begin{assumption}[Cluster Separation]
\label{assumption:gap:neural:bandits}
% The gap between any two reward functions 
% % preference vectors 
% for different \gtclusters{} is at least an \textit{unknown} positive constant $\gamma'$
The reward functions of users from different clusters are separated by a constant gap $\gamma'$:

\begin{equation*}
    \norm{f^{j}(\bx)-f^{j^{\prime}}(\bx)}_2\geq \gamma'>0\,, \forall{j,j^{\prime}\in [m]\,, j\neq j^{\prime}}\,\forall \bx\in\mathcal{X}.
\end{equation*}  
\end{assumption}

In neural bandits, we adopt $(1 / \sqrt{m_{\text{NN}}})g(\bx;\btheta_0)$ as the feature mapping. Therefore, our item regularity assumption (Assumption \ref{assumption3}) is also applicable here after plugging in $\phi(\bx) = (1 / \sqrt{m_{\text{NN}}})g(\bx;\btheta_0)$.

\section{Algorithms}
\subsection{Clustering Of Linear Dueling Bandits (COLDB)}
\label{subsec:algo:coldb}
% We introduce the Clustering Of Linear Dueling Bandits (COLDB) algorithm tailored for the linear setting, as detailed in Algorithm~\ref{algo:linear:dueling:bandits}. We elucidate the underlying principles and operational workflow of COLDB as follows.
Our Clustering Of Linear Dueling Bandits (COLDB) algorithm is described in Algorithm~\ref{algo:linear:dueling:bandits}. Here we elucidate the underlying principles and operational workflow of COLDB.
COLDB maintains a dynamic graph $G_t = (\mathcal{U}, E_t)$ encompassing all users, whose connected components represent the inferred user clusters in round $t$. Throughout the learning process, COLDB adaptively removes edges to accurately cluster the users based on their estimated reward function parameters, thereby leveraging these clusters to enhance online learning efficiency. The operation of COLDB proceeds as follows:

\noindent\textbf{Cluster Inference $\overline{C}_t$ for User $i_t$ (Line \ref{algo line: init}-Line \ref{algo line: cluster detection}).} Initially, COLDB constructs a complete undirected graph $G_0 = (\mathcal{U}, E_0)$ over the user set (Line~\ref{algo line: init}). As learning progresses, edges are selectively removed to ensure that only users with similar preference profiles remain connected. At each round $t$, when a user $i_t$ comes to the system with a feasible arm set $\mathcal{X}_t$ (Line~\ref{algo line: user comes}), COLDB identifies the connected component $\overline{C}_t$ containing $i_t$ in the maintained graph $G_{t-1}$, which serves as the current estimated cluster for this user (Line~\ref{algo line: cluster detection}).

\noindent\textbf{Estimating Shared Statistics for Cluster $\overline{C}_t$ (Line \ref{algo line: common theta}-Line \ref{algo line: common matrix}).} Once the cluster $\overline{C}_t$ is identified, COLDB estimates a common preference vector $\overline{\btheta}_t$ for all users within this cluster by aggregating the historical feedback from all members of $\overline{C}_t$. 
% Specifically, in Line~\ref{algo line: common theta}, the common preference vector is determined via Maximum Likelihood Estimation (MLE) using the accumulated feedback:
Specifically, in Line~\ref{algo line: common theta}, the common preference vector is determined by minimizing the following loss function:
% via Maximum Likelihood Estimation (MLE) using the accumulated feedback:
\begin{align}
    &\overline{\btheta}_t=\arg\min_{\btheta} - \sum_{s\in[t-1]\atop i_s\in \overline C_t} \Big( y_s\log\mu\left({\btheta}^{\top}\left[\phi(\bx_{s,1}) - \phi(\bx_{s,2})\right]\right) \notag\\
    &+ (1-y_s)\log\mu\left({\btheta}^{\top}\left[\phi(\bx_{s,2}) - \phi(\bx_{s,1})\right]\right) \Big) + \frac{1}{2}\lambda\norm{\btheta}_2^2  \,,\label{eq: solve common theta}
\end{align}
which corresponds to the Maximum Likelihood Estimation (MLE) using the data from all users in the cluster $\overline{C}_t$.
% \begin{small}
% \begin{align}
%                 \overline{\btheta}_t&=\arg\min_{\btheta} \Bigg[ - \sum_{s\in[t-1]\atop i_s\in \overline C_t} \bigg( y_s\log\mu\left({\btheta}^{\top}\left[\phi(\bx_{s,1}) - \phi(\bx_{s,2})\right]\right) \notag\\
%                 &+ (1-y_s)\log\mu\left({\btheta}^{\top}\left[\phi(\bx_{s,2}) - \phi(\bx_{s,1})\right]\right) \bigg) + \frac{1}{2}\lambda\norm{\btheta}_2^2 \Bigg] \,.\label{eq: solve common theta}
% \end{align}
% \end{small}
Additionally, in Line~\ref{algo line: common matrix}, COLDB computes the aggregated information matrix for $\overline{C}_t$, which is subsequently utilized in selecting the second arm $\bx_{t,2}$:
\begin{equation}
    \bV_{t-1} = \bV_0 + \sum_{\substack{s \in [t-1] \\ i_s \in \overline{C}_t}} (\phi(\bx_{s,1}) - \phi(\bx_{s,2})) (\phi(\bx_{s,1}) - \phi(\bx_{s,2}))^\top
    \label{eq:update:info:matrix:v:linear}
\end{equation}
\noindent\textbf{Arm Recommendation Based on Cluster Statistics (Line \ref{algo line: choose x1}-Line \ref{algo line: choose x2}).} Leveraging the estimated common preference vector $\overline{\btheta}_t$ and the aggregated information matrix $\bV_{t-1}$, COLDB proceeds to recommend two arms as follows:

\begin{itemize}
%\squishlisttwo
    \item \textbf{First Arm Selection ($\bx_{t,1}$).} In Line~\ref{algo line: choose x1}, COLDB selects the first arm by greedily choosing the arm that maximizes the estimated reward according to $\overline{\btheta}_t$:
    \begin{equation}
        \bx_{t,1} = \arg\max_{\bx \in \mathcal{X}_t} \overline{\btheta}_t^\top \phi(\bx).
    \end{equation}
    \item \textbf{Second Arm Selection ($\bx_{t,2}$).} Following the selection of $\bx_{t,1}$, in Line~\ref{algo line: choose x2}, COLDB selects the second arm by maximizing an upper confidence bound (UCB):
    % relative to $\bx_{t,1}$:

    \begin{align}
    \bx_{t,2} &= \arg\max_{\bx\in\mathcal{X}_t} \overline\btheta_t^\top \phi(\bx) + \frac{\beta_t}{\kappa_\mu}\norm{\phi(\bx) - \phi(\bx_{t,1})}_{\bV_{t-1}^{-1}}\,.
\label{eq:linear:select:second:arm}
\end{align}
\end{itemize}

Intuitively, Eq.(\ref{eq:linear:select:second:arm}) encourages the selection of the arm which both (a) has a large predicted reward value and (b) is different from $\bx_{t,1}$ and the arms selected in the previous $t-1$ rounds when the served user belongs to the currently estimated cluster $\overline{C}_t$.
In other words, the second arm $\bx_{t,2}$ is chosen by balancing exploration and exploitation.
% \begin{align}
%     \bx_{t,2} &= \arg\max_{\bx\in\mathcal{X}_t} \overline\btheta_t^\top \left( \phi(\bx) - \phi(\bx_{t,1}) \right) \notag\\
%     &+ \frac{\beta_t}{\kappa_\mu}\norm{\phi(\bx) - \phi(\bx_{t,1})}_{\bV_{t-1}^{-1}}\,.
% \end{align}
% \squishend
% \end{itemize}

% Then, in Line \ref{algo line: choose x2}, based on this first greedily selected arm $\bx_{t,1}$, COLDB selects the second arm which has the maximum upper confidence bound (UCB) with respect to $\bx_{t,1}$

\noindent\textbf{Updating User Estimates and Interaction History (Line \ref{algo line: feedback}-Line \ref{algo line: update it}).} Upon recommending $\bx_{t,1}$ and $\bx_{t,2}$, the user receives binary feedback $y_t = \mathbbm{1}(\bx_{t,1} \succ \bx_{t,2})$ from user $i_t$, and then updates the interaction history $\mathcal{D}_t = \{i_s, \bx_{s,1}, \bx_{s,2}, y_s\}_{s=1}^t$ (Line~\ref{algo line: feedback}). Moreover, COLDB updates the preference vector estimate for user $i_t$ while keeping the estimates for the other users unchanged (Line~\ref{algo line: update it}). Specifically, the preference vector estimate $\hat{\btheta}_{i_t,t}$ is updated via MLE using the historical data from user $i_t$:
\begin{align}
    &\hat{\btheta}_{i_t,t} = \arg\min_{\btheta} - \sum_{\substack{s \in [t-1] \\ i_s = i_t}} \Big( y_s \log \mu\big(\btheta^\top [\phi(\bx_{s,1}) - \phi(\bx_{s,2})]\big) \notag \\
    &+ (1 - y_s) \log \mu\big(\btheta^\top [\phi(\bx_{s,2}) - \phi(\bx_{s,1})]\big) \Big) + \frac{\lambda}{2} \|\btheta\|_2^2\,.
\end{align}
% \begin{small}
%     \begin{align}
%         \hat{\btheta}_{i_t,t} &= \arg\min_{\btheta} \Bigg[ - \sum_{\substack{s \in [t-1] \\ i_s = i_t}} \bigg( y_s \log \mu\big(\btheta^\top [\phi(\bx_{s,1}) - \phi(\bx_{s,2})]\big) \notag \\
%         &\quad + (1 - y_s) \log \mu\big(\btheta^\top [\phi(\bx_{s,2}) - \phi(\bx_{s,1})]\big) \bigg) + \frac{\lambda}{2} \|\btheta\|_2^2 \Bigg].
%     \end{align}
% \end{small}

\noindent\textbf{Dynamic Graph Update (Line \ref{algo line: delete}).} Finally, based on the updated preference estimate $\hat{\btheta}_{i_t,t}$ for user $i_t$, COLDB reassesses the similarity between $i_t$ and the other users. If the discrepancy between $\hat{\btheta}_{i_t,t}$ and $\hat{\btheta}_{\ell,t}$ for any user $\ell$ surpasses a predefined threshold (Line~\ref{algo line: delete}), the edge $(i_t, \ell)$ is removed from the graph $G_{t-1}$, effectively separating them into distinct clusters. The resultant graph $G_t = (\mathcal{U}, E_t)$ is then utilized in the subsequent rounds.

\begin{algorithm*}[t!] 
\caption{Clustering Of Linear Dueling Bandits (COLDB)}
\label{algo:linear:dueling:bandits}
	\begin{algorithmic}[1]
    \STATE {\bf Input:} $f(T_{i,t})=\frac{\sqrt{\lambda/\kappa_\mu}+\sqrt{2\log(u/\delta)+d\log(1+4T_{i,t}\kappa_\mu/d\lambda)}}{\kappa_\mu{\sqrt{2\tilde{\lambda}_x T_{i,t}}}}$, regularization parameter $\lambda>0$, confidence parameter $\beta_t \triangleq \sqrt{2\log(1/\delta) + d\log\left( 1 + tL^2\kappa_{\mu}/(d\lambda) \right)}$, $\kappa_\mu>0$.
    \STATE {\bf Initialization:} 
$\bV_0=\bV_{i,0} = \frac{\lambda}{\kappa_\mu} \mathbf{I}$ , $\hat\btheta_{i,0}=\bzero$, $\forall{i \in \mathcal{U}}$, a complete Graph $G_0 = (\mathcal{U},E_0)$ over $\mathcal{U}$.\alglinelabel{algo line: init}
		\FOR{$t= 1, \ldots, T$}
            \STATE Receive the index of the current user $i_t\in\mathcal{U}$, and the current feasible arm set $\cX_t$;\alglinelabel{algo line: user comes}
            \STATE Find the connected component $\overline C_t$ for user $i_t$ in the current graph $G_{t-1}$ as the current cluster; \alglinelabel{algo line: cluster detection}
            
            \STATE Estimate the common preference vector $\overline{\btheta}_t$ for the current cluster $\overline C_t$:
            \begin{small}
                         \begin{align}
                \overline{\btheta}_t
                &=\text{argmin}_{\btheta} - \sum_{s\in[t-1]\atop i_s\in \overline C_t} \Big( y_s\log\mu\left({\btheta}^{\top}\left[\phi(\bx_{s,1}) - \phi(\bx_{s,2})\right]\right) + (1-y_s)\log\mu\left({\btheta}^{\top}\left[\phi(\bx_{s,2}) - \phi(\bx_{s,1})\right]\right) \Big) \notag\\
                &\quad+ \frac{\lambda}{2}\norm{\btheta}_2^2;
            \end{align}
            \end{small}
   \alglinelabel{algo line: common theta}
            
            \STATE Calculate aggregated information matrix for cluster $\overline C_t$: 
            $\bV_{t-1}=\bV_0+\sum_{s\in[t-1]\atop i_s\in \overline C_t}(\phi(\bx_{s,1}) - \phi(\bx_{s,2}))(\phi(\bx_{s,1}) - \phi(\bx_{s,2}))^\top$. \alglinelabel{algo line: common matrix}
            \STATE Choose the first arm  $\bx_{t,1} = \arg\max_{\bx\in\mathcal{X}_t}\overline\btheta_t^\top \phi(\bx)$; \alglinelabel{algo line: choose x1}
            \STATE Choose the second arm $\bx_{t,2} = \arg\max_{\bx\in\mathcal{X}_t} \overline\btheta_t^\top \left( \phi(\bx) - \phi(\bx_{t,1}) \right) + \frac{\beta_t}{\kappa_\mu}\norm{\phi(\bx) - \phi(\bx_{t,1})}_{\bV_{t-1}^{-1}}$; \alglinelabel{algo line: choose x2}
            % $x_{t,2}$ by maximizing the upper confidence bound in \eqref{eq:choose:arm:2}
		% \STATE Select $(x_{t,1}, x_{t,2})$ and observe human feedback $y_t$
		\STATE Observe the preference feedback: $y_t = \mathbbm{1}(\bx_{t,1}\succ \bx_{t,2})$, and update history: $\mathcal{D}_t=\{i_s, \bx_{s,1}, \bx_{s,2}, y_s\}_{s=1,\ldots,t}$;\alglinelabel{algo line: feedback}
        \STATE Update the estimation for the current served user $i_t$: \alglinelabel{algo line: update it}
        \begin{small}
                    \begin{align}
        \hat{\btheta}_{i_t,t}&=\arg\min_{\btheta} - \sum_{s\in[t-1]\atop i_s=i_t}\Big( y_s\log\mu\left({\btheta}^{\top}\left[\phi(\bx_{s,1}) - \phi(\bx_{s,2})\right]\right) + (1-y_s)\log\mu\left({\btheta}^{\top}\left[\phi(\bx_{s,2}) - \phi(\bx_{s,1})\right]\right) \Big)\notag\\
        &\quad+ \frac{\lambda}{2}\norm{\btheta}_2^2, 
            \end{align}
        \end{small}

            keep the estimations of other users unchanged;
            \STATE Delete the edge $(i_t,\ell)\in E_{t-1}$ if
            \begin{equation}
                \norm{\hat\btheta_{i_t,t}-\hat\btheta_{\ell,t}}_2>f(T_{i_t,t})+f(T_{\ell,t})
            \end{equation} \alglinelabel{algo line: delete}
		\ENDFOR
	\end{algorithmic}
\end{algorithm*}

\subsection{Clustering Of Neural Dueling Bandits (CONDB)}
\label{subsec:algo:condb}
Our Clustering Of Neural Dueling Bandits (CONDB) algorithm is illustrated in Algorithm~\ref{algo:neural:dueling:bandits} (App.~\ref{app:sec:condb:algo}), which adopts neural networks to model non-linear reward functions.
Similar to COLDB, our CONDB algorithm also maintains a dynamic graph $G_t = (\mathcal{U}, E_t)$ in which every connected component denotes an inferred cluster, and adaptively removes the edges between users who are estimated to belong to different clusters.

\noindent\textbf{Cluster Inference $\overline{C}_t$ for User $i_t$ (Line 5).} 
Similar to COLDB (Algo.~\ref{algo:linear:dueling:bandits}), when a new user $i_t$ arrives,  our CONDB firstly identifies the connected component $\overline{C}_t$ in the maintained graph $G_{t-1}$ which contains the user $i_t$ and then uses it as the estimated cluster for $i_t$ (Line 5).

\noindent\textbf{Estimating Shared Statistics for Cluster $\overline{C}_t$ (Line 6).}
After the cluster $\overline{C}_t$ is identified, our CONDB algorithm uses the history of preference feedback observations from all users in the cluster $\overline{C}_t$
% , and uses these data 
to train a neural network (NN) to minimize the following loss function (Line 6):
\begin{align}
    % \overline{\btheta}_t&=\arg\min_{\btheta} 
    &\mathcal{L}_t(\btheta)=
    - \frac{1}{m} \sum_{s\in[t-1]\atop i_s\in \overline C_t}\Big( y_s\log\mu\left( h(\bx_{s,1};\btheta) - h(\bx_{s,2};\btheta) \right) + \notag \\
    &(1-y_s)\log\mu\left(h(\bx_{s,2};\btheta) - h(\bx_{s,1};\btheta)\right) \Big) + \frac{\lambda}{2}\norm{\btheta - \btheta_0}_2^2
\end{align}
to yield parameters $\overline{\btheta}_t$.
In addition, similar to COLDB (Algorithm \ref{algo:linear:dueling:bandits}), our CONDB computes the aggregated information matrix for the cluster $\overline{C}_t$ following Eq.(\ref{eq:update:info:matrix:v:linear})
.
Note that here we replace $\phi(\bx)$ from Eq.(\ref{eq:update:info:matrix:v:linear}) by the NTK feature representation $\phi(\bx)=(1/\sqrt{m})g(\bx;\btheta_0)$, in which $\btheta_0$ represents the initial parameters of the NN (Sec.~\ref{subsec:problem:setting:neural}).

\noindent\textbf{Arm Recommendation Based on Cluster Statistics (Line 8-Line 9).} 
Next, our CONDB algorithm leverages the trained NN with parameters $\overline{\btheta}_t$ and the aggregated information matrix $\bV_{t-1}$ to select the pair of arms.
The first arm is selected by greedily maximizing the reward prediction of the NN with parameters $\overline{\btheta}_t$ (Line 8): 
\begin{equation}
\bx_{t,1} = \arg\max_{\bx\in\mathcal{X}_t} h(\bx;\overline{\btheta}_t).
\end{equation}
The second arm is then selected optimistically (Line 9):
\begin{equation}
\bx_{t,2} = \arg\max_{\bx\in\mathcal{X}_t} h(\bx;\overline{\btheta}_t) + \nu_T \norm{\left(\phi(\bx) - \phi(\bx_{t,1})\right)}_{\bV_{t-1}^{-1}},
\end{equation}
% \begin{equation}
% \bx_{t,2} = \arg\max_{\bx\in\mathcal{X}_t} h(\bx;\overline{\btheta}_t) + \left( \beta_T + B\sqrt{\frac{\lambda}{\kappa_\mu}} + 1 \right) \norm{\left(\phi(\bx) - \phi(\bx_{t,1})\right)}_{\bV_{t-1}^{-1}}.
% \end{equation}
in which $\nu_T \triangleq \beta_T + B\sqrt{\frac{\lambda}{\kappa_\mu}} + 1$, $\beta_T \triangleq \frac{1}{\kappa_\mu} \sqrt{ \widetilde{d} + 2\log(u/\delta)}$ and $B$ is defined in Lemma \ref{lemma:linear:utility:function:informal}.
Here $\widetilde{d}$ denotes the \emph{effective dimenision} which we will introduce in detail in Sec.~\ref{subsec:theory:neural}.

\noindent\textbf{Updating User Estimates and Interaction History (Line 10-Line 11).} 
After recommending the pair of arms $\bx_{t,1}$ and $\bx_{t,2}$, we collect the preference feedback $y_t = \mathbbm{1}(\bx_{t,1} \succ \bx_{t,2})$ and update interaction history: $\mathcal{D}_t = \{i_s, \bx_{s,1}, \bx_{s,2}, y_s\}_{s=1}^t$ (Line 10).
Next, we update the parameters of the NN used to predict the reward for user $i_t$ by minimizing the following loss function (Line 11):
% the reward function estimate for user $i_t$ by training an NN to minimize the following loss function:

   \begin{align}
    % \hat{\btheta}_{i_t,t}&=\arg\min_{\btheta} 
     &\mathcal{L}_{i_t,t}(\btheta)= 
    - \frac{1}{m_{\text{NN}}} \sum_{s\in[t-1]\atop i_s = i_t}\big( y_s\log\mu\left( h(\bx_{s,1};\btheta) - h(\bx_{s,2};\btheta) \right) + \notag\\
    &(1-y_s)\log\mu\left(h(\bx_{s,2};\btheta) - h(\bx_{s,1};\btheta)\right) \big) + \frac{\lambda}{2}\norm{\btheta - \btheta_0}_2^2
    % \label{eq:loss:func:individial}
\end{align} 

to yield parameters $\hat{\btheta}_{i_t,t}$.
The NN parameters for the other users remain unchanged.

\noindent\textbf{Dynamic Graph Update (Line 12).} 
Finally, we use the updated NN parameters $\hat{\btheta}_{i_t,t}$ for user $i_t$ to reassess the similarity between user $i_t$ and the other users.
We remove the edge between $(i_t, \ell)$ from the graph $G_{t-1}$ if the difference between $\hat{\btheta}_{i_t,t}$ and $\hat{\btheta}_{\ell,t}$ is large enough (Line 12). Intuitively, if the estimated reward functions (represented by the respective parameters of their NNs for reward prediction) between two users are significantly different, we separate these two users into different clusters.
The updated graph $G_t = (\mathcal{U}, E_t)$ is then used in the following rounds.

\section{Theoretical Analysis}
\label{sec:theory}
In this section, we present the theoretical results regarding the regret guarantees of our proposed algorithms and provide a detailed discussion of these findings.
\subsection{Clustering Of Linear Dueling Bandits (COLDB)}
\label{subsec:theory:linear}
The following theorem provides an upper bound on the expected regret achieved by the COLDB algorithm (Algo.~\ref{algo:linear:dueling:bandits}) under the linear setting.
\begin{theorem} \label{thm: linear regret bound}
    Suppose that 
    % the assumptions in Section \ref{sec: setting}
    Assumptions \ref{assumption4}, \ref{assumption1}, \ref{assumption2} and \ref{assumption3}
    are satisfied. Then the expected regret of the COLDB algorithm (Algo.~\ref{algo:linear:dueling:bandits}) for $T$ rounds satisfies
    \begin{align}
        R(T)&= O\Big(u\big(\frac{d}{\kappa_\mu^2\tilde\lambda_x \gamma^2}+\frac{1}{\tilde\lambda_x^2}\big)\log T+\frac{1}{\kappa_\mu}d\sqrt{mT}\Big)\label{bound linear 2 terms}\\
        &=O\Big(\frac{1}{\kappa_\mu}d\sqrt{mT}\Big)\,,
    \end{align}
    where $\tilde{\lambda}_x\triangleq\int_{0}^{\lambda_x} (1-e^{-\frac{(\lambda_x-x)^2}{2\sigma^2}})^{C} dx$ is the problem instance dependent constant \cite{wang2024onlinea,wang2024onlineb}.
\end{theorem}
The proof of this theorem can be found in Appendix \ref{app: proof linear}.
% \noindent\textbf{Discussion and Comparison.}  
The regret bound in Eq.(\ref{bound linear 2 terms}) consists of two terms. The first term accounts for the number of rounds required to accumulate sufficient information to correctly cluster all users with high probability, and it scales only logarithmically with the number of time steps \(T\). The second term captures the regret after successfully clustering the users, which depends on the number of clusters \(m\), rather than the potentially huge total number of users $u$. 
Notably, the regret upper bound is not only sub-linear in $T$, but also \emph{becomes tighter when there is a smaller number of clusters $m$}, i.e., when a larger number of users belong to the same cluster on average.
This provides a formal justification for the advantage of cross-user collaboration in our problem setting where only preference feedback is available.

In the special case where there is only one user (\(m = 1\)), the regret bound simplifies to \(O(d \sqrt{T} / \kappa_\mu)\), which aligns with the 
% state-of-the-art 
classical results in the single-user linear dueling bandit literature \cite{NeurIPS21_saha2021optimal,ICML22_bengs2022stochastic,arXiv24_li2024feelgood}.
% \zhiyong{@Zhongxiang, could you please add the reference here?}
Compared to the previous works on clustering of bandits with linear reward functions \cite{gentile2014online,wang2024onlinea,10.5555/3367243.3367445}, our regret upper bound has an extra dependency on $1 / \kappa_\mu$. Since $\kappa_\mu < 0.25$ for the logistic function, this dependency makes our regret upper bound larger and hence captures the more challenging nature of the preference feedback compared to the numerical feedback in classical clustering of linear bandits.

\subsection{Clustering Of Neural Dueling Bandits (CONDB)}
\label{subsec:theory:neural}
% \zhiyong{@Zhongxiang, please add a cite here.} 
Let $\mathbf{H}' = \sum_{t=1}^T \sum_{(i, j) \in C_K^2} z^i_j(t)z^i_j(t)^\top  \frac{1}{m_{\text{NN}}}$, in which 
% $z^i_j(s) = \phi(x_{s,i}) - \phi(x_{s,j})$ 
$z^i_j(t) = g(\bx_{t,i};\btheta_0) - g(\bx_{t,j};\btheta_0)$ 
and $C_K^2$ denotes all pairwise combinations of $K$ arms. 
% We now define the \emph{effective dimension} as follows:
Then, the effective dimension $\widetilde{d}$ is defined as follows \cite{verma2024neural}:
\begin{equation}
    \widetilde{d} = \log \det  \left(\frac{\kappa_\mu}{\lambda}  \mathbf{H}' + \mathbf{I}\right).
\label{eq:eff:dimension}
\end{equation}
The definition of $\widetilde{d}$ considers the contexts from all users and in all $T$ rounds.
The theorem below gives an upper bound on the expected regret of our CONDB algorithm (Algo.~\ref{algo:neural:dueling:bandits}).
\begin{theorem} \label{thm: neural regret bound}
Suppose that Assumptions \ref{assumption4}, \ref{assumption3}, \ref{assumption:main:neural} and \ref{assumption:gap:neural:bandits} are satisfied (let $\phi(\bx) = (1 / \sqrt{m_{\text{NN}}})g(\bx;\btheta_0)$ in Assumption \ref{assumption3}). 
As long as $m_{\text{NN}} \geq \text{poly}(T, L, K, 1/\kappa_\mu, L_\mu, 1/\lambda_0, 1/\lambda, \log(1/\delta))$,
then the expected regret of the CONDB algorithm (Algo.~\ref{algo:neural:dueling:bandits}) for $T$ rounds satisfies
\begin{align}
R_T &= O\bigg(u\big(\frac{\widetilde{d}}{\kappa_\mu^2\tilde\lambda_x \gamma^2}+\frac{1}{\tilde\lambda_x^2}\big)\log T+\big(\frac{\sqrt{\widetilde{d}}}{\kappa_\mu} + B\sqrt{\frac{\lambda}{\kappa_\mu}}\big)\sqrt{\widetilde{d}mT} \bigg) \label{regret:bound:condb}\\
&=O\Big(\big(\frac{\sqrt{\widetilde{d}}}{\kappa_\mu} + B\sqrt{\frac{\lambda}{\kappa_\mu}}\big)\sqrt{\widetilde{d}mT} \Big)\,.\label{regret:bound:condb:second}
\end{align}
\end{theorem}
% \zhiyong{@Zhongxiang, do you think we can add the $T_0$ term here as well?}
The proof of this theorem can be found in Appendix \ref{app: proof neural}.
% \textbf{Discussion and Comparison.}
The first term in the regret bound in Eq.~\ref{regret:bound:condb} has the same form as the first term in the regret bound of COLDB in Eq.(\ref{bound linear 2 terms}), except that the input dimension $d$ for COLDB (Eq.(\ref{bound linear 2 terms})) is replaced by the effective dimension $\widetilde{d}$ for CONDB (Eq.(\ref{regret:bound:condb})).
As discussed in \cite{verma2024neural}, $\widetilde{d}$ is usually larger than the effective dimension in classical neural bandits \cite{zhou2020neural,zhang2020neural}.
This dependency, together with the extra dependency on $1/\kappa_\mu$, reflects the added difficulty from the preference feedback compared to the more informative numerical feedback in classical neural bandits.

Similar to COLDB (Theorem \ref{thm: linear regret bound}), the first term in the regret upper bound of CONDB (Theorem \ref{thm: neural regret bound}) results from the number of rounds needed to collect enough observations to correctly identify the clustering structure. The second term corresponds to the regret of all users after the correct clustering structure is identified, which depends on the number of clusters $m$ instead of the number of users $u$.
Theorem \ref{thm: neural regret bound} also shows that the regret upper bound of CONDB is sub-linear in $T$, and becomes improved as the number of users belonging to the same cluster is increased on average (i.e., when the number of clusters $m$ is smaller).
Moreover, in the special case where the number of clusters is $m=1$, the regret upper bound in Eq.(\ref{regret:bound:condb:second}) becomes the same as that of the standard neural dueling bandits \cite{verma2024neural}.

\section{Experimental Results}
\label{sec:experiments}

We use both synthetic and real-world experiments to evaluate the performance of our COLDB and CONDB algorithms.
For both algorithms, we compare them with their corresponding single-user variant as the baseline. Specifically, for COLDB, we compare it with the baseline of LDB\_IND, which refers to Linear Dueling Bandit (Independent) \cite{ICML22_bengs2022stochastic}, meaning running independent classic linear dueling bandit algorithms for each user separately; similarly, for CONDB, we compare it with NDB\_IND, which stands for Neural Dueling Bandit (Independent) \cite{verma2024neural}.

% {\color{blue}
% Some details of the synthetic and MoviewLens dataset to be added.}

\paragraph{COLDB.}
Our experimental settings mostly follow the designs from the works on clustering of bandits \cite{wang2024onlinea,10.5555/3367243.3367445}.
In our synthetic experiment for COLDB, we design a setting with linear reward functions: $f_i(\bx)=\btheta_i^{\top} \bx$.
We choose $u=200$ users, $K=20$ arms and a feature dimension of $d=20$, and construct two settings with $m=2$ and $m=5$ groundtruth clusters, respectively.
In the experiment with the MovieLens dataset \cite{harper2015movielens}, we follow the experimental setting from \cite{wang2024onlinea}, a setting with $200$ users.
Same as the synthetic experiment, we choose the number of arms in every round to be $K=20$ and let the input feature dimension be $d=20$. We construct a setting with $m=5$ clusters.
We repeat each experiment for three independent trials and report the mean $\pm$ standard error.

Fig.~\ref{fig:exp:linear} plots the cumulative regret of our COLDB and the baseline of LDB\_IND.
% shows the results for COLDB.
% in a synthetic experiment and an experiment constructed using the MovieLens dataset.
The results show that our COLDB algorithm significantly outperforms the baseline of LDB\_IND in both the synthetic and real-world experiments. Moreover, Fig.~\ref{fig:exp:linear} (a) demonstrates that when $m=2$ (i.e., when a larger number of users belong to the same cluster on average), the performance of our COLDB is improved, which is \emph{consisent with our theoretical results} (Sec.~\ref{subsec:theory:linear}).
\begin{figure}[t]
% \vspace{-1.4mm}
     \centering
     \begin{tabular}{cc}
        \hspace{-3mm} \includegraphics[width=0.52\linewidth]{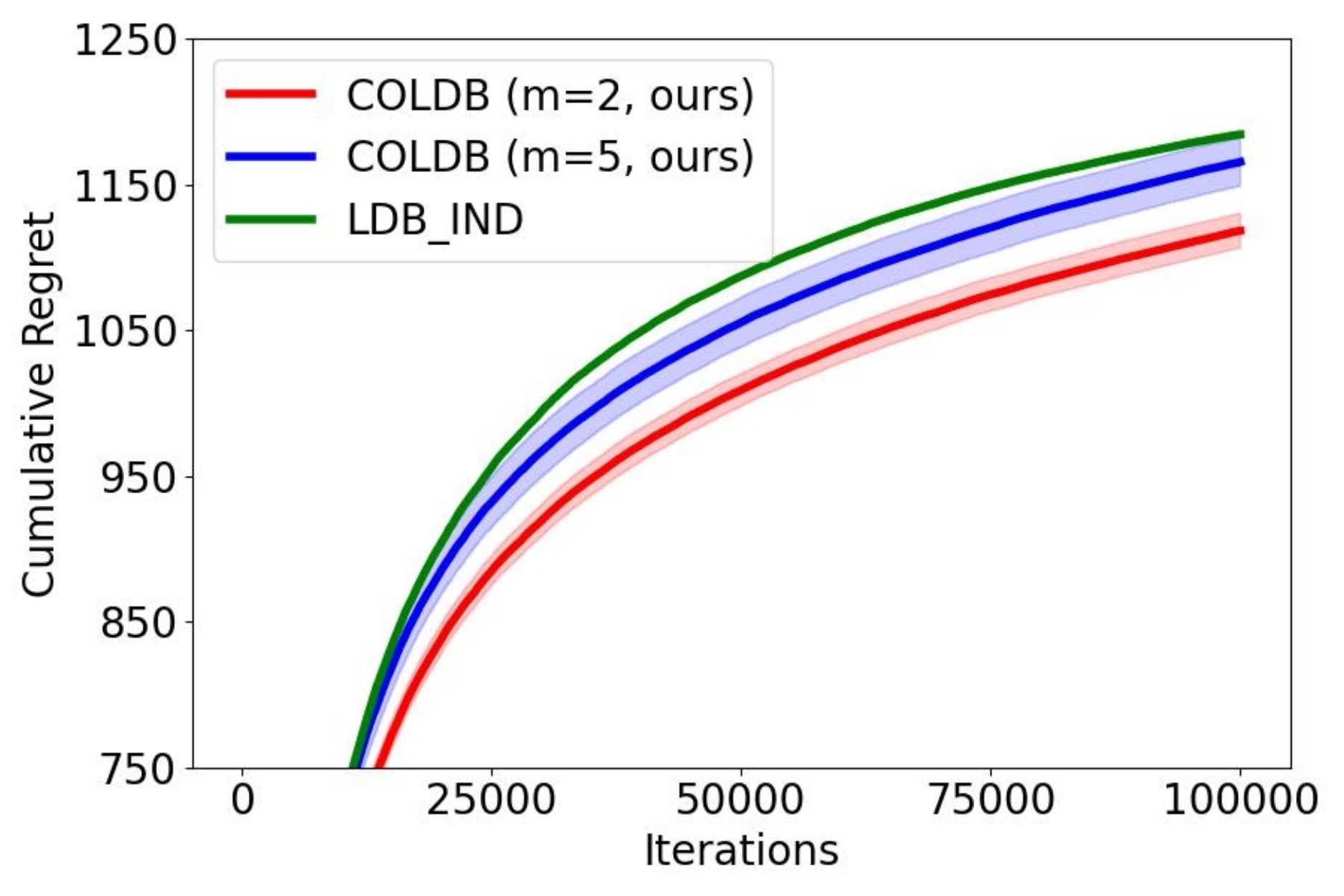} &\hspace{-5mm} 
         \includegraphics[width=0.52\linewidth]{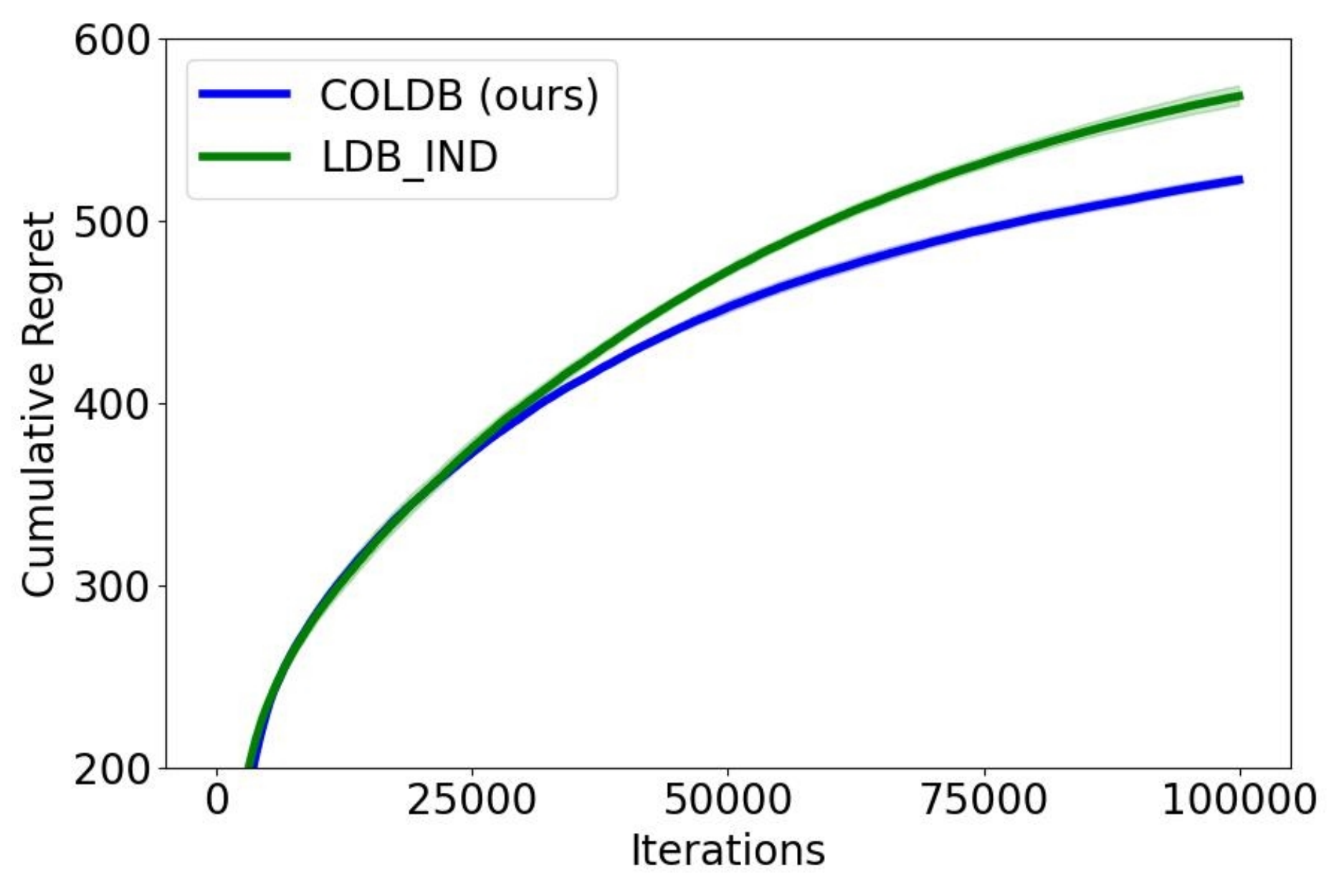} \\
         {\hspace{-3mm}(a) Synthetic} & {(b) MovieLens} 
     \end{tabular}
     \caption{
     Experimental results for our COLDB algorithm with a linear reward function.
     }
     \label{fig:exp:linear}
\end{figure}
\begin{figure}[h]
% \vspace{-1.4mm}
     \centering
     \begin{tabular}{cc}
        \hspace{-3mm} \includegraphics[width=0.52\linewidth]{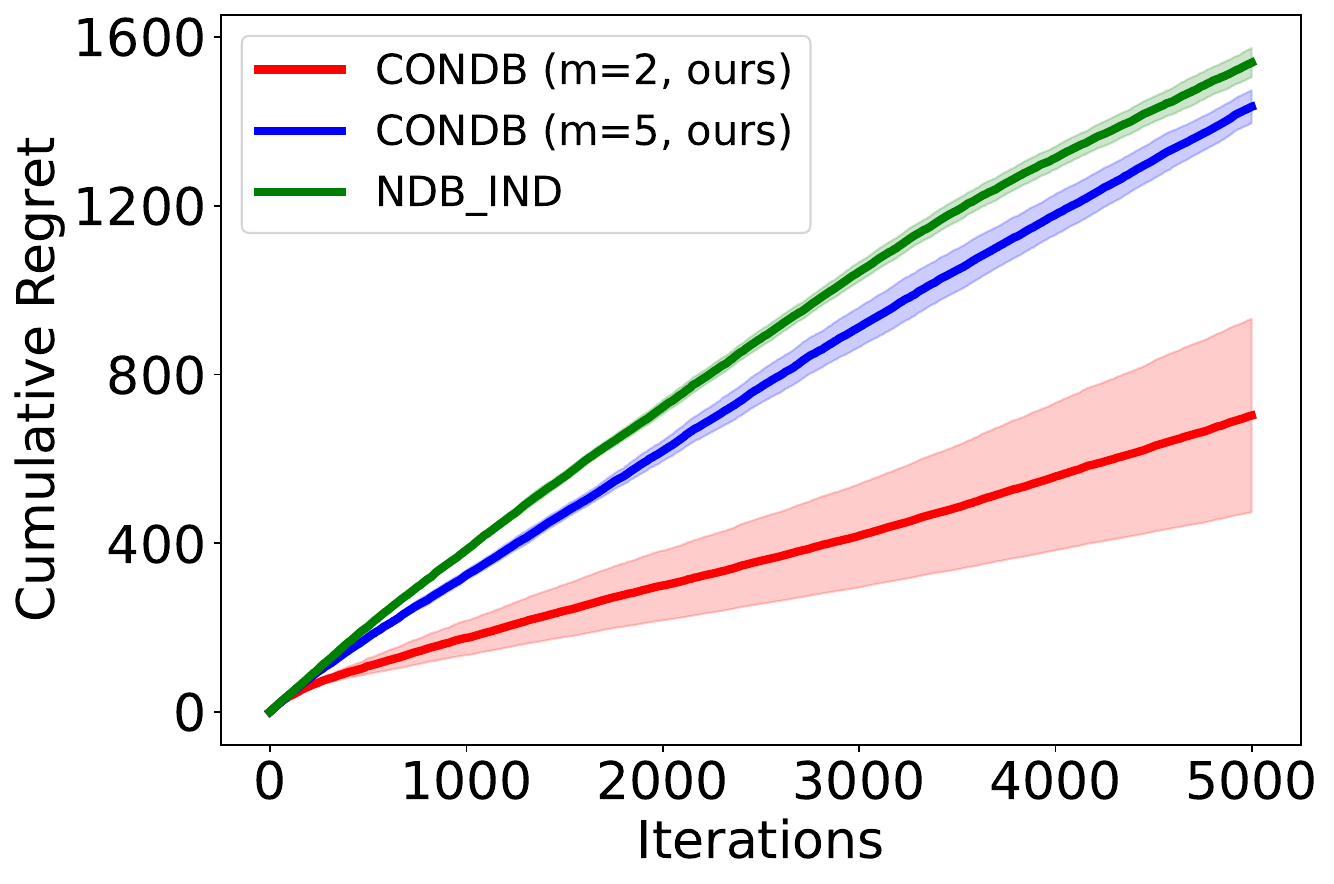} &\hspace{-5mm} 
         \includegraphics[width=0.52\linewidth]{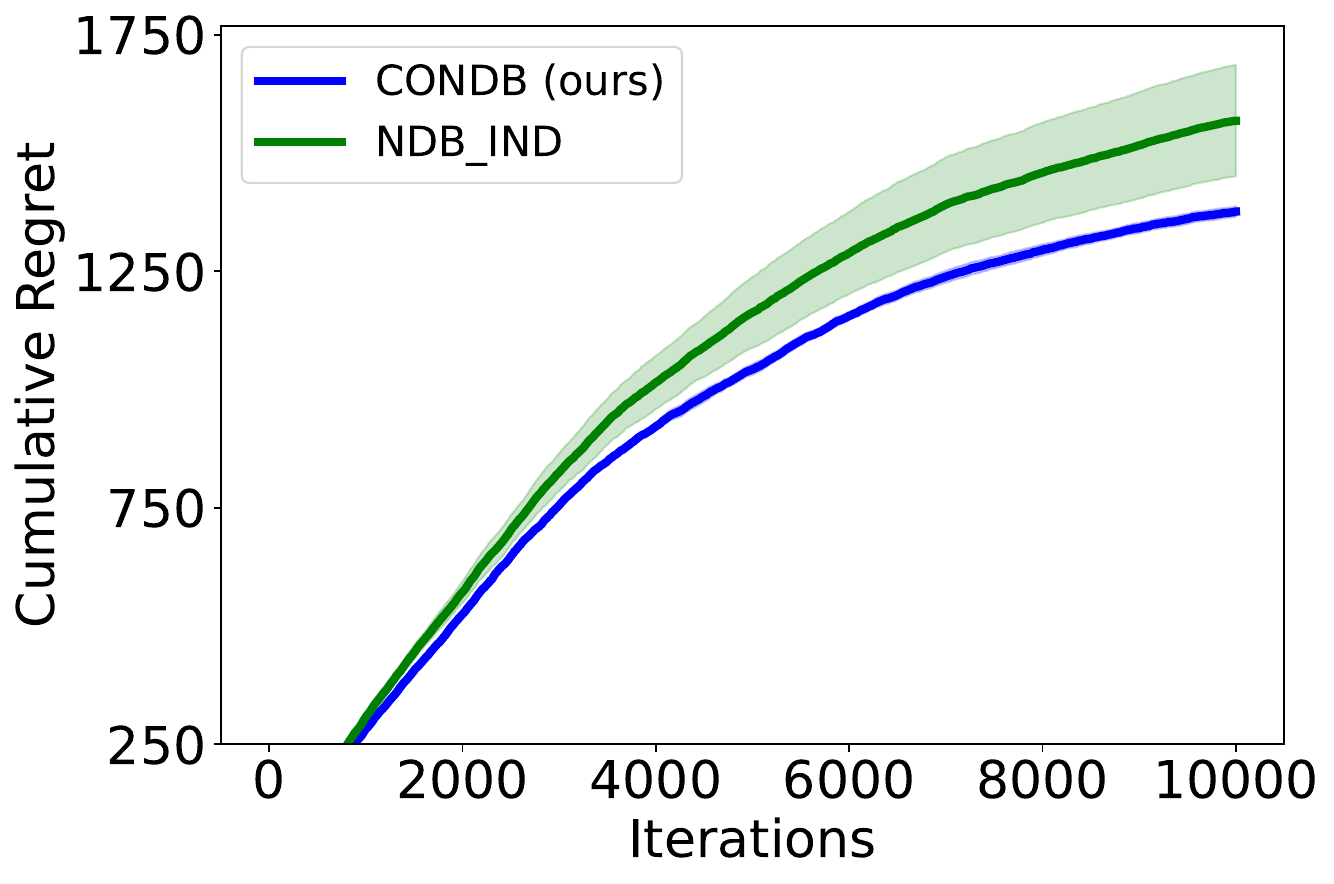} \\
         {\hspace{-3mm}(a) Synthetic} & {(b) MovieLens} 
     \end{tabular}
     \caption{
     Experimental results for our CONDB algorithm with a non-linear (square) reward function.
     }
     \label{fig:exp:neural}
\end{figure}

\paragraph{CONDB.}
We also construct both a synthetic and real-world experiment to evaluate our CONDB algorithm.
Most of the experimental settings are the same as those of the COLDB algorithm described above. The major difference is that instead of using linear reward functions, here we adopt a non-linear reward function, i.e., a square function: $f_i(\bx)=(\btheta_i^{\top} \bx)^2$. 
The results in this setting are plotted in Fig.~\ref{fig:exp:neural}.
Our CONDB algorithm achieves significantly smaller cumulative regrets than the baseline algorithm of NDB\_IND in both the synthetic and real-world experiments. Moreover, Fig.~\ref{fig:exp:neural} (a) shows that the performance of our CONDB is improved when a larger number of users are in the same cluster on average, i.e., when $m=2$.
These results demonstrate the potential of our CONDB algorithm to excel in problems with complicated non-linear reward functions.

\chapter{Conclusion and Future Work}
\label{chapter:conclusion}
In this chapter, we summarize the thesis and list some future directions that could inspire the follow-up works.

\textbf{In Chapter~\ref{chapter:iclr2025}}, we presented a minimalist approach for achieving horizon-free and second-order regret bounds in RL: simply train transition models via Maximum Likelihood Estimation followed by optimistic or pessimistic planning, depending on whether we operate in the online or offline learning mode.  Our horizon-free bounds for general function approximation look quite similar to the bounds in Contextual bandits, indicating that the need for long-horizon planning does not make RL harder than CB from a statistical perspective. 

Our work has some limitations. First, when extending our result to continuous function class, we pay $\ln(H)$. This $\ln(H)$ is coming from a naive application of the $\epsilon$-net/bracket argument to the generalization bounds of MLE. We conjecture that this $\ln(H)$ can be elimiated by using a more careful analysis that uses techiniques such as peeling/chaining \cite{dudley1978central,zhang2006varepsilon}. We leave this as an important future direction.
Second, while our model-based framework is quite general, it cannot capture problems that need to be solved via model-free approaches such as linear MDPs \cite{jin2020provably}. An interesting future work is to see if we can develop the corresponding model-free approaches that can achieve horizon-free and instance-dependent bounds for RL with general function approximation. Finally, the algorithms studied in this work are not computationally tractable. This is due to the need of performing optimism/pessimism planning for exploration. Deriving computationally tractacle RL algorithms for the rich function approximation setting is a long-standing question.

\textbf{In Chapter~\ref{chapter ZSG}}, we study the zero-shot generalization (ZSG) performance of offline reinforcement learning (RL). We propose two offline RL frameworks, pessimistic empirical risk minimization and pessimistic proximal policy optimization, and show that both of them can find the optimal policy with ZSG ability. We also show that such a generalization property does not hold for offline RL without knowing the context information of the environment, which demonstrates the necessity of our proposed new algorithms. Currently, our theorems and algorithm design depend on the i.i.d. assumption of the environment selection. How to relax such an assumption remains an interesting future direction. 

\textbf{In Chapter~\ref{chapter:mis}}, we present a new problem of clustering of bandits with misspecified user models (CBMUM), where the agent
has to adaptively assign appropriate clusters for users under model misspecifications. 
We propose two robust CB algorithms, RCLUMB and RSCLUMB. Under milder assumptions than previous CB works, we prove the regret bounds of our algorithms, which match the lower bound asymptotically in $T$ up
to logarithmic factors, and match the state-of-the-art results in several degenerate cases. It is challenging to bound the regret caused by \textit{misclustering} users with close but not the same preference vectors and use inaccurate cluster-based information to select arms. Our analysis to bound this part of the regret is quite general and may be of independent interest. Experiments on synthetic and real-world data demonstrate the advantage of our algorithms. 
We would like to state some interesting future works: (1) Prove a tighter regret lower bound for CBMUM, (2) Incorporate recent model selection methods into our fundamental framework to design robust algorithms for CBMUM with unknown exact maximum model misspecification level, and (3) Consider the setting with misspecifications in the underlying user clustering structure rather than user models.

\textbf{In Chapter~\ref{chapter: detect}}, we are the first to propose the novel LOCUD problem, where there are many users with \textit{unknown} preferences and \textit{unknown} relations, and some corrupted users can occasionally perform disrupted actions to fool the agent. Hence, the agent not only needs to learn the \textit{unknown} user preferences and relations robustly from potentially disrupted bandit feedback, balance the exploration-exploitation trade-off to minimize regret, but also needs to detect the corrupted users over time. To robustly learn and leverage the \textit{unknown} user preferences and relations from corrupted behaviors, we propose a novel bandit algorithm RCLUB-WCU. To detect the corrupted users in the online bandit setting, based on the learned user relations of RCLUB-WCU, we propose a novel detection algorithm OCCUD. We prove a regret upper bound for RCLUB-WCU, which matches the lower bound asymptotically in $T$ up to logarithmic factors and matches the state-of-the-art results in degenerate
cases. We also give a theoretical guarantee for the detection accuracy of OCCUD. Extensive experiments show that our proposed algorithms achieve superior performance over previous bandit algorithms and high corrupted user detection accuracy.

\textbf{In Chapter~\ref{chapter: aaai}}, we introduce ConLinUCB, a general framework for conversational bandits with efficient information incorporation. Based on this framework, we propose ConLinUCB-BS and ConLinUCB-MCR,
with explorative key-term selection strategies that can quickly elicit the user's potential interests. We prove tight regret bounds of our algorithms. Particularly, ConLinUCB-BS achieves a bound
of $O(d\sqrt{T\log T})$, much better than $O(d\sqrt{T}\log T)$ of the classic ConUCB. In the empirical evaluations, our algorithms dramatically outperform the classic ConUCB. For future work, it would be interesting to consider the settings with knowledge graphs \cite{zhao2022knowledge}, hierarchy item trees \cite{song2022show}, relative feedback \cite{xie2021comparison} or different feedback selection strategies \cite{letard2020partial,letard2022mabs}, and use our framework and principles to improve the performance of existing algorithms.

\textbf{In Chapter~\ref{chapter: aistats}}, we study non-stationary stochastic linear bandits in this work. We establish the first variance-dependent regret lower bound for non-stationary linear bandits, which captures the interplay between variance, non-stationarity, and dimensionality in the linear bandit setting, offering new insights into the complexity of this problem. We propose Restarted-$\algbandit$ and Restarted SAVE$^+$, two algorithms that utilize the dynamic variance information of the dynamic reward distribution. We show that both of our algorithms are able to achieve better dynamic regret compared with best existing results \cite{wei2021non} under several parameter regimes, \emph{e.g.}, when the total variance $V_K$ is small. Experiment results backup our theoretical claim. It is worth noting there still exist gaps between our current obtained regret and the lower bound, and to fix such a gap leaves as our future work.

\textbf{In Chapter~\ref{chapter: dueling}}, we introduce the first clustering of dueling bandit algorithms for both linear and non-linear latent reward functions, which enhance the performance of MAB with preference feedback via cross-user collaboraiton.
Our algorithms estimates the clustering structure online 
based on the
estimated reward function parameters, and employs the data from all users within the same cluster to select the pair of arms to query for preference feedback.
We derive upper bounds on the cumulative regret of our algorithms, which show that our algorithms enjoy theoretically guaranteed improvement when a larger number of users belong to the same cluster on average. We also use synthetic and real-world experiments to validate our theoretical findings.

% =======================================================

% % appendix
\appendix
\chapter{Appendices}
\newpage
\section{Appendix for Chapter \ref{chapter:iclr2025}}
\subsection{Summary of Contents in the Appendix}
The Appendix is organized as follows. 

In \pref{app: eluder new}, we provide some new analyses for Eluder dimension, which we will use for proving the regret bounds for the online RL setting. 

In \pref{app: supporting lemmas}, we provide some other supporting lemmas that will be used in our proofs. 

In \pref{app: online full}, we provide the detailed proofs for the online RL setting (\pref{sec:online}). Specifically, in \pref{app:online} we give the proof of \pref{thm:online_theorem}; in \pref{app:online_coro_faster}, we show the proof of \pref{corr:online_coro_faster}; in \pref{app:online_coro_infinite}, we give the proof of \pref{corr:online_coro_infinite}. 

In \pref{app: offline full}, we provide the detailed proofs for the offline RL setting (\pref{sec:offline}). Specifically, in \pref{app:offline} we give the proof of \pref{thm:mleoffline}; in \pref{app:offline_coro_faster}, we show the proof of \pref{corr:coro_faster}; in \pref{app:offline_coro_infinite}, we give the proof of \pref{corr:offline_coro_infinite}; in \pref{app: example proof}, we show the proof of the claim in \pref{ex: offline coverage}.

\subsection{Analysis regarding the Eluder Dimension}\label{app: eluder new}
For simplicity, we denote $x_h^k =  (s_h^k, a_h^k)$.

First we have two technical lemma. The first lemma bounds the summation of ``self-normalization" terms by the Eluder dimension. Our result generalizes the previous result by \cite{zhao2023nearly} from the $\ell_2$-Eluder dimension to the $\ell_1$ case. 
\begin{lemma}\label{lem:normal}
Suppose for all $g \in \Psi, |g| \leq 1$ and $\lambda>1$, then we have
    \begin{align}
        &\sum_{k=1}^K\sum_{h=1}^H \min\bigg\{1,\sup_{g \in \Psi}\frac{|g(x_{h}^{k})|}{\sum_{k'=1}^{k-1}\sum_{h'=1}^H |g(x_{h'}^{k'})| + \sum_{h'=1}^{h-1} |g(x_{h'}^{k})| + \lambda}\bigg\} \notag \\
        &\leq 12\log^2(4\lambda KH)\cdot DE_1(\Psi, \Scal \times \Acal, 1/(8\lambda KH)) + \lambda^{-1}.\notag
    \end{align}
\end{lemma}

\noindent\textbf{Proof} [Proof of \pref{lem:normal}]
We follow the proof steps of Theorem 4.6 in \cite{zhao2023nearly}. 
    For simplicity, we use $n = KH$, $i = kH+h$ to denote the indices and denote $x_h^k$ by $x_i$. Then we need to prove
    \begin{align}
        &\sum_{i=1}^n\min\bigg\{1,\sup_{g \in \Psi}\frac{|g(x_i)|}{\sum_{t=1}^{i-1}|g(x_t)| + \lambda}\bigg\}\notag\\
        &\leq 12\log^2(4\lambda n)\cdot DE_1(\Psi, \Scal \times \Acal, 1/(8\lambda n)) + \lambda^{-1}\label{eqn:help999}
    \end{align}
    Let 
        \begin{align}
        g_i = \argmax_{g \in \Psi}\frac{|g(x_i)|}{\sum_{j=1}^{i-1}|g(x_j)| + \lambda}
    \end{align}
    For any $1/(\lambda n)\leq \rho\leq 1$ and $1\leq j \leq \lceil \log(4\lambda n)\rceil$, we define
    \begin{align}
       &A_\rho^j= \bigg\{i\in [n] : 2^{-j} < |g_i(x_i)|\leq 2^{-j+1}, \frac{|g_i(x_i)|}{\sum_{t=1}^{i-1}|g_i(x_t)| + \lambda}\geq \rho/2\bigg\},\notag\\
       &d_j:=DE_1(\Psi, \Scal \times \Acal, 2^{-j}).
    \end{align}
Next we only consider the set $A_\rho^j$ where $|A_\rho^j|>d_j$. We denote $A_\rho^j = \{a_1,\dots, a_A\}$, where $A = |A_\rho^j|$ and $\{a_i\}$ keeps the same order as $\{x_i\}$. Next we do the following constructions. We maintain $k = \lfloor (A-1)/d_j\rfloor$ number of queues $Q_1,\dots, Q_k$, all of them initialized as emptysets. We put $a_1$ into $Q_1$. For $a_i, i\geq 2$, we put $a_i$ into $Q_l$, where $Q_l$ is the first queue where $a_i$ is $2^{-j}$-independent of all elements in $Q_l$. Let $i_{\max}$ be the smallest $i$ when we can not put $a_i$ into any existing queue. 

We claim that $i_{\max}$ indeed exists, i.e., our construction will stop before we put all elements in $A^j_\rho$ into $Q_1,\dots, Q_k$. In fact, note the fact that the length of each $Q_l$ is always no more than $d_j$, which is due to the fact that any $2^{-j}$-independent sequence's length is at most $d_j$. Meanwhile, since we only have $k = \lfloor (A-1)/d_j\rfloor$, then the amount of elements in $Q_1\cup\dots\cup Q_k$ will be upper bounded by $k\cdot d_j<A$. That suggests at least one element in $A_\rho^j$ is not contained by $Q_1\cup\dots\cup Q_k$, i.e., $i_{\max}$ exists. 

By the definition of $i_{\max}$, we know that $a_{i_{\max}}$ is $2^{-j}$-dependent to each $Q_l$. Next we give a bound of $A$. First, note
\begin{align}
    \sum_{t=1}^{i_{\max}-1}|g_{i_{\max}}(x_t)| \geq \sum_{t\in Q_1\cup\dots\cup Q_k}|g_{i_{\max}}(a_t)| = \sum_{l=1}^k \sum_{t\in Q_l}|g_{i_{\max}}(a_t)| > k\cdot 2^{-j},\label{eqn: help221}
\end{align}
where the first inequality holds since $Q_l$ are the elements that appear before $a_{i_{\max}}$, the second one holds due to the following induction of Eluder dimension: since $a_{i_{\max}}$ is $2^{-j}$-dependent to $Q_l$, then we have
\begin{align}
    \forall g \in \Psi,\ \sum_{t\in Q_l}|g(a_t)| \leq 2^{-j} \Rightarrow |g(a_{i_{\max}})| \leq 2^{-j}.\label{eqn: help:222}
\end{align}
Therefore, given the fact $|g_{i_{\max}}(a_{i_{\max}})| > 2^{-j}$ (recall the definition of $A^j_\rho$), we must have $\sum_{t\in Q_l}|g_{i_{\max}}(a_t)|>2^{-j}$ as well, which suggests the second inequality of \pref{eqn: help221} holds. Second, we have
\begin{align}
    \sum_{t=1}^{i_{\max}-1}|g_{i_{\max}}(x_t)| \leq 2/\rho\cdot |g_{i_{\max}}(a_{i_{\max}})| \leq 4\cdot 2^{-j}/\rho,\label{eqn:help223}
\end{align}
where both inequalities hold due to the definition of $A^j_\rho$. Combining \pref{eqn: help221} and \pref{eqn:help223}, we have
\begin{align}
    k<4/\rho \Rightarrow A \leq 4d_j/\rho +d_j \leq 5d_j/\rho. 
\end{align}
Therefore, we have that for all $\rho, j$, $|A^j_\rho| \leq 5d_j/\rho$. 

Finally we prove \pref{eqn:help999}. $1/(\lambda n)\leq \rho\leq 1$ and $1\leq j \leq \lceil \log(4\lambda n)\rceil = J$. Denote 
   \begin{align}
       A_\rho= \bigg\{i\in [n] :\frac{|g_i(x_i)|}{\sum_{t=1}^{i-1}|g_i(x_t)| + \lambda}\geq \rho/2\bigg\}.
    \end{align}
Then it is easy to notice that $|A_\rho| = \sum_j |A^j_\rho| \leq \lceil \log(4\lambda n)\rceil\cdot 5d_J/\rho$, where we use the fact that the Eluder dimension $d_j$ is increasing. Therefore, by the standard peeling technique, we have
\begin{align}
    &\sum_{i=1}^n\min\bigg\{1,\sup_{g \in \Psi}\frac{|g(x_i)|}{\sum_{t=1}^{i-1}|g(x_t)| + \lambda}\bigg\}\notag\\
    &= \sum_{j\in [\lceil\log(\lambda n) \rceil]}\sum_{i\in A_{2^{-j}}\setminus A_{2^{-j+1}}} + \sum_{j= \lceil\log(\lambda n) \rceil }\sum_{i\notin A_{2^{-j+1}}} \notag \\
    & \leq  \sum_{j\in [\lceil\log(\lambda n) \rceil]}\sum_{i\in A_{2^{-j}}\setminus A_{2^{-j+1}}} 2^{-j-1} + n\cdot 1/(\lambda n)\notag \\
    & \leq \sum_{j\in [\lceil\log(\lambda n) \rceil]}\sum_{i\in A_{2^{-j}}} 2^{-j-1} + n\cdot 1/(\lambda n)\notag \\
    & \leq \lceil\log(\lambda n) \rceil\cdot \lceil \log(4\lambda n)\rceil\cdot 3d_J + \lambda^{-1},\notag
\end{align}
which concludes our proof.

Next lemma gives a bound to bound the number of episodes where the behavior along these episodes are ``bad". Intuitively speaking, our lemma suggests we only have limited number of bad episodes, therefore won't affect the final performance of our algorithm.
   \begin{lemma}\label{lem:zhou_1}
        Given $\lambda>1$. There exists at most 
        \begin{align}
            13\log^2(4\lambda KH)\cdot DE_1(\Psi, \Scal \times \Acal, 1/(8\lambda KH))
        \end{align}
        number of $k \in [K]$ satisfying the following claim
        \begin{align}
        \sup_{g \in \Psi}\frac{\lambda + \sum_{k'=1}^{k}\sum_{h'=1}^H |g(x_{h'}^{k'})|}{\lambda + \sum_{k'=1}^{k-1}\sum_{h'=1}^H |g(x_{h'}^{k'})|} > 4.
\label{help:1}
        \end{align}
    \end{lemma}

\noindent\textbf{Proof}[Proof of \pref{lem:zhou_1}]

Note that 
\begin{align}
   &\sum_{k=1}^K\min\bigg\{2, \log \sup_{g \in \Psi}\frac{\lambda + \sum_{k'=1}^{k}\sum_{h'=1}^H |g(x_{h'}^{k'})|}{\lambda + \sum_{k'=1}^{k-1}\sum_{h'=1}^H |g(x_{h'}^{k'})|} \bigg\}\notag \\
   &\leq \sum_{k=1}^K\min\bigg\{2,\log \prod_{h=1}^H \sup_{g \in \Psi}\frac{\lambda + \sum_{k'=1}^{k-1}\sum_{h'=1}^H |g(x_{h'}^{k'})| + \sum_{h'=1}^{h} |g(x_{h'}^{k})|}{\lambda + \sum_{k'=1}^{k-1}\sum_{h'=1}^H |g(x_{h'}^{k'})| + \sum_{h'=1}^{h-1} |g(x_{h'}^{k})|}\bigg\}\notag \\
   & = \sum_{k=1}^K\min\bigg\{2,\sum_{h=1}^H \log\bigg(1+ \sup_{g \in \Psi}\frac{|g(x_{h}^{k})|}{\lambda + \sum_{k'=1}^{k-1}\sum_{h'=1}^H |g(x_{h'}^{k'})| + \sum_{h'=1}^{h-1} |g(x_{h'}^{k})|}\bigg)\bigg\}\notag \\
   & \leq \sum_{k=1}^K\sum_{h=1}^H \min\bigg\{2,\sup_{g \in \Psi}\frac{|g(x_{h}^{k})|}{\sum_{k'=1}^{k-1}\sum_{h'=1}^H |g(x_{h'}^{k'})| + \sum_{h'=1}^{h-1} |g(x_{h'}^{k})| + \lambda}\bigg\},\notag \\
   & \leq 2\sum_{k=1}^K\sum_{h=1}^H \min\bigg\{1,\sup_{g \in \Psi}\frac{|g(x_{h}^{k})|}{\sum_{k'=1}^{k-1}\sum_{h'=1}^H |g(x_{h'}^{k'})| + \sum_{h'=1}^{h-1} |g(x_{h'}^{k})| + \lambda}\bigg\}\notag\\
    &\leq 24\log^2(4\lambda KH)\cdot DE_1(\Psi, \Scal \times \Acal, 1/(8\lambda KH)) + 2\lambda^{-1}\notag \\
    & \leq 26\log^2(4\lambda KH)\cdot DE_1(\Psi, \Scal \times \Acal, 1/(8\lambda KH)).\label{eq:xxx}
\end{align}
where the first inequality holds since $\sup_{g}\prod f(g) \leq \prod \sup_g f(g)$, the second one holds since $\log(1+x) \leq x$, the fourth one holds due to \pref{lem:normal}. Therefore, there are at most 
\begin{align}
    26\log^2(4\lambda KH)\cdot DE_1(\Psi, \Scal \times \Acal, 1/(8\lambda KH))/2\notag
\end{align}
number of $k$ satisfying
\begin{align}
    \log \sup_{g \in \Psi}\frac{\lambda + \sum_{k'=1}^{k}\sum_{h'=1}^H |g(x_{h'}^{k'})|}{\lambda + \sum_{k'=1}^{k-1}\sum_{h'=1}^H |g(x_{h'}^{k'})|}> 2,\notag
\end{align}
which concludes the proof.

We next have the following lemma, which bounds the regret by the Eluder dimension. 
\begin{lemma}[Theorem 5.3, \cite{wang2023benefits}]\label{lem:kaiweneluder}
Let $C := \sup_{(s,a)\in\Scal\times\Acal,f\in\Psi}\abs{ f((s,a))}$ be the envelope.
For any sequences $f^{(1)},\dots,f^{(N)}\subseteq \Psi$, $(s,a)^{(1)},\dots,(s,a)^{(N)}\subseteq\Scal\times\Acal$, let $\beta$ be a constant such that for all $n\in[N]$ we have,
$
    \sum_{i=1}^{n-1}\abs{f^{(n)}((s,a)^i)} \leq \beta.
$
Then, for all $n\in[N]$, we have
\begin{equation*}
    \sum_{t=1}^n\abs{f^{(t)}((s,a)^t)}\leq \inf_{0<\epsilon\leq 1}\braces{ \text{DE}_1(\Psi,\Scal \times \Acal,\epsilon)(2C + \beta\log(C/\epsilon)) + n\epsilon }.
\end{equation*}
\end{lemma}

Given \pref{lem:zhou_1} and Lemma \pref{lem:kaiweneluder}, we are able to prove the following key lemma. 
\begin{lemma}[New Eluder Pigeon Lemma]\label{lem: new eluder pigeon lemma}
Let the event $\Ecal$ be 
\begin{align}
    \Ecal: \forall k \in [K],\ \sum_{i=1}^{k-1} \sum_{h=1}^H   \mathbb H^2(\widehat P^k(s_h^i ,a_h^i )||P^*(s_h^i ,a_h^i ))\leq \eta.
\end{align}
Then under event $\Ecal$, there exists a set $\Kcal \in [K]$ such that
\begin{itemize}[leftmargin = *]
    \item We have $|\Kcal| \leq 13\log^2(4\eta KH)\cdot DE_1(\Psi, \Scal \times \Acal, 1/(8\eta KH))$.
    \item We have
    \begin{align}
        &\sum_{k \in [K]\setminus \Kcal}\sum_{h=1}^H \mathbb H^2\Big(P^\star(s_h^k,a_h^k)\Mid  \widehat P^k\big(s_h^k,a_h^k)\Big)\notag \\
        &\leq \inf_{0<\epsilon\leq 1}\braces{ \text{DE}_1(\Psi,\Scal \times \Acal,\epsilon)(2 + 7\eta\log(1/\epsilon)) + KH\epsilon }\notag \\
        &\leq \text{DE}_1(\Psi,\Scal \times \Acal,1/KH)(2 + 7\eta\log(KH)) + 1 \,,
    \end{align}
\end{itemize}
where the function class $\Psi =\{(s,a)\mapsto \mathbb H^2(P^\star(s,a)\Mid P(s,a)):P\in\Pcal\}$.
\end{lemma}
\noindent\textbf{Proof}[Proof of \pref{lem: new eluder pigeon lemma}]
We interchangeably use $n = kH+h$ to denote the indices of $s_h^k, a_h^k$. We set $f^{(n)}((s,a))$ in \pref{lem:kaiweneluder} as $H^2(P^k(s,a )||P^*(s,a))$.

First, we prove that the $\beta$ in \pref{lem:kaiweneluder} can be selected as $7\eta$ under event $\Ecal$. To show that, let $\Kcal$ denote all the $k$ stated in \pref{lem:zhou_1}. Then for all $k$ such that $k+1\notin \Kcal$, $h = 2,...,H$, let $n = kH+h$, we have  
\begin{align}
    \sum_{i=0}^{n-1}\abs{f^{(n)} ((s,a)^i)} &\leq \sum_{i=0}^{kH+H}\abs{f^{(n)}((s,a)^i)} \notag \\
    & = \bigg(\lambda+\sum_{i=0}^{kH}\abs{f^{(n)}((s,a)^i)}\bigg)\cdot \frac{\sum_{i=0}^{kH+H}\abs{f^{(n)}((s,a)^i)}+\lambda}{\sum_{i=0}^{kH}\abs{f^{(n)}((s,a)^i)}+\lambda}-\lambda\notag \\
    & \leq \bigg(\lambda+\sum_{i=0}^{kH}\abs{f^{(n)}((s,a)^i)}\bigg)\cdot 4-\lambda\notag \\
    &\leq 7\eta,
\end{align}
where the second inequality holds due to \pref{lem:zhou_1}, the last one holds due to the definition of $\Ecal$. Therefore, we prove our lemma by the conclusion of \pref{lem:kaiweneluder} with $\beta = 7\eta$.

\subsection{Other Supporting Lemmas}\label{app: supporting lemmas}
\begin{lemma}[Simulation Lemma (\cite{agarwal2019reinforcement})]\label{lem:simulation} 
% Consider two finite-horizon MDPs, $\hat{\Mcal} = \{\Scal,\Acal, r, \hat{P}, H\}$ and $\Mcal = \{\Scal,\Acal, r, P, H\}$. Let $V_{0,P^\star}^\pi(s_0)=\EE\left[\sum_{h=1}^H r_h(s_h,a_h)|\pi,P\right]$, $\hat{V}_1^\pi(s_1)=\EE\left[\sum_{h=1}^H r_h(s_h,a_h)|\pi,\hat{P}\right]$, and let $ d^{\pi}_h(s, a; s_1)$ be the probability of $\pi$ reaching $(s, a)$ at time step $h$ starting from $s_1$. 
We have
\begin{align}
    V_{0;P^\star}^\pi-{V}_{0;\hat P}^\pi\leq \sum_{h=0}^{H-1} \EE_{s,a\sim d^{\pi}_h} \left[ \left|\EE_{s'\sim P^\star(s,a)} V^{\pi}_{h+1;\widehat P}(s') -\EE_{s'\sim \widehat P(s,a)} V^{\pi}_{h+1;\widehat P}(s')    \right|\right].\notag
\end{align}
    
\end{lemma}

% \wen{I don't think we define this short notation $\VV$? should be $\VV_{P^\star}$ at least? but maybe define $\VV$ in short of $\VV_{P^\star}$ is good and is consistant with our definition of $V$, so we can do that, but just need to define it carefully in the preliminary. Also i believe we start index at $h=0$, not $h=1$; similarly for $k$, i believe we start at $k=0$}
\begin{lemma}[Change of Variance Lemma (Lemma C.5 in \cite{jin2018q})]\label{lem:variance_lemma} 
\begin{align}
    \sum_{h=0}^{H-1} \EE_{s,a\sim d^{\pi}_h}\left[\big(\VV_{P^\star} V_{h+1;P^\star}^{\pi}\big)(s,a)\right]=\var_{\pi }.\notag
\end{align} 
%\wen{cite the paper for the lemmas?}
    
\end{lemma}

\begin{lemma}[Generalization bounds of MLE for finite model class (Theorem E.4 in \cite{wang2023benefits})]\label{lem:mle_generalization_offline}
Let $\Xcal$ be the context/feature space and $\Ycal$ be the label space, and we are given a dataset $D = \braces{ (x_i,y_i) }_{i\in[n]}$ from a martingale process:
for $i=1,2,...,n$, sample $x_i \sim \Dcal_i(x_{1:i-1},y_{1:i-1})$ and $y_i\sim p(\cdot \mid x_i)$.
Let $f^\star(x,y) = p(y\mid x)$ and we are given a realizable, \emph{i.e.}, $f^\star\in\Fcal$, function class $\Fcal:\Xcal\times\Ycal\to\Delta(\RR)$ of distributions.
Suppose $\Fcal$ is finite.
Fix any $\delta\in(0,1)$, set $\beta=\log(|\Fcal|/\delta)$ and define
\begin{align*}
    \widehat\Fcal = \braces{ f\in\Fcal: \sum_{i=1}^n\log f(x_i,y_i)\geq \max_{\widetilde f\in\Fcal}\sum_{i=1}^n \log\widetilde f(x_i,y_i) - 4\beta }.
\end{align*}
Then w.p. at least $1-\delta$, the following holds:
\begin{enumerate}
    \item[(1)] The true distribution is in the version space, \emph{i.e.}, $f^\star\in\widehat{\mathcal{F}}$.
    \item[(2)] Any function in the version space is close to the ground truth data-generating distribution, \emph{i.e.}, for all $f\in\widehat\Fcal$
    \begin{align*}
        \sum_{i=1}^n \EE_{x\sim\Dcal_i}\left[\mathbb H^2( f(x,\cdot) \Mid f^\star(x,\cdot) ) \right]\leq 22\beta.
    \end{align*}
\end{enumerate}
\end{lemma}

\begin{lemma}[Generalization bounds of MLE for infinite model class (Theorem E.5 in \cite{wang2023benefits})]\label{lem:mle_generalization infinite}
Let $\Xcal$ be the context/feature space and $\Ycal$ be the label space, and we are given a dataset $D = \braces{ (x_i,y_i) }_{i\in[n]}$ from a martingale process:
for $i=1,2,...,n$, sample $x_i \sim \Dcal_i(x_{1:i-1},y_{1:i-1})$ and $y_i\sim p(\cdot \mid x_i)$.
Let $f^\star(x,y) = p(y\mid x)$ and we are given a realizable, \emph{i.e.}, $f^\star\in\Fcal$, function class $\Fcal:\Xcal\times\Ycal\to\Delta(\RR)$ of distributions.
Suppose $\Fcal$ is finite.
Fix any $\delta\in(0,1)$, set $\beta=\log(\Ncal_{[]}((n|\Ycal|)^{-1},\Fcal,\|\cdot\|_\infty)/\delta)$ (where $\Ncal_{[]}((n|\Ycal|)^{-1},\Fcal,\|\cdot\|_\infty)$ is the bracketing number defined in \pref{def: bracketing number}) and define
\begin{align*}
    \widehat\Fcal = \braces{ f\in\Fcal: \sum_{i=1}^n\log f(x_i,y_i)\geq \max_{\widetilde f\in\Fcal}\sum_{i=1}^n \log\widetilde f(x_i,y_i) - 7\beta }.
\end{align*}
Then w.p. at least $1-\delta$, the following holds:
\begin{enumerate}
    \item[(1)] The true distribution is in the version space, \emph{i.e.}, $f^\star\in\widehat{\mathcal{F}}$.
    \item[(2)] Any function in the version space is close to the ground truth data-generating distribution, \emph{i.e.}, for all $f\in\widehat\Fcal$
    \begin{align*}
        \sum_{i=1}^n \EE_{x\sim\Dcal_i}\left[\mathbb H^2( f(x,\cdot) \Mid f^\star(x,\cdot) ) \right]\leq 28\beta.
    \end{align*}
\end{enumerate}
\end{lemma}
 \begin{lemma}[Recursion Lemma]\label{lem:recursion bound C_m}
     Let $G >0$ be a positive constant, $a<G/2$ is also a positive constant, and let $\{C_m\}_{m=0}^{N=\lceil\log_2(\frac{KH}{G})\rceil}$ be a sequence of positive real numbers satisfying:
\begin{enumerate}
    \item $C_m \leq 2^{m} G + \sqrt{a C_{m+1}}+a$ for all $m \geq 0$,
    \item $C_m \leq K H$ for all $m \geq 0$, where $K > 0$ and $H > 0$ are positive constants.
\end{enumerate}
Then, it holds that:
\[
C_0 \leq 4 G.
\]
 \end{lemma}
\noindent\textbf{Proof}[Proof of \pref{lem:recursion bound C_m}]
     We will prove by induction that for all $m \geq 0$,
\[
C_m \leq 2^{m+2} G.
\]
Then, for $m = 0$, this would immediately show $C_0 \leq 4 G$.

\textbf{1. The base case $m = N$}: 

Since $N=\lceil\log_2(\frac{KH}{G})\rceil$, it is obvious that $2^{N+2} G \geq K H$.  
Thus, $C_N\leq KH\leq 2^{N+2} G$, the inequality holds for $m = N$.

\textbf{2. The induction step:}

Assume that for some $m \geq 0$, for $C_{m+1}$, we have:
\[
C_{m+1} \leq 2^{m+1+2} G = 2^{m+3} G.
\]

Then, we have
\begin{align}
    C_m&\leq 2^{m} G + \sqrt{aC_{m+1}}+a\notag\\
    &\leq 2^{m}G+\sqrt{a2^{m+3}G}+a\notag\\
    &\leq 2^{m}G+\sqrt{\frac{G}{2}\cdot2^{m+3}G}+\frac{G}{2}\notag\\
    &=G\cdot(2^m+2^{m/2+1}+2^{-1})\notag\\
    &\leq G\cdot(2^m+2^{m+1}+2^m)\notag\\
    &=2^{m+2}G\,.
\end{align}
Therefore, by induction, we have 
for all $m \geq 0$,
\[
C_m \leq 2^{m+2} G.
\]
And the proof follows by setting $m=0$.
  
\subsection{Detailed Proofs for the online setting in \pref{sec:online}}\label{app: online full}
\subsection{Proof of \pref{thm:online_theorem}}
\label{app:online}
The following is the full proof of \pref{thm:online_theorem}. 

For notational simplicity, throughout this whole section, we denote 
\begin{small}
  \begin{align}
    A&:=\sum_{k\in[K-1]\setminus \Kcal}\sum_{h=0}^{H-1} \left[\big(\VV_{P^\star} V_{h+1; \widehat P^k}^{\pi^k}\big)(s_h^k,a_h^k)\right]\notag\\
    B&:=\sum_{k\in[K-1]\setminus \Kcal}\sum_{h=0}^{H-1}\left[\big(\VV_{P^\star} V_{h+1}^{\pi^k}\big)(s_h^k,a_h^k)\right],\notag\\
       C_m&:=\sum_{k\in[K-1]\setminus \Kcal}\sum_{h=0}^{H-1} \left[\big(\VV_{P^\star}( V_{h+1; \widehat P^k}^{\pi^k}-V_{h+1}^{\pi^k})^{2^m}\big)(s_h^k,a_h^k)\right]\notag\\
    G&:= \sqrt{\sum_{k\in[K-1]\setminus \Kcal}\sum_{h=0}^{H-1}  \left[\big(\VV_{P^\star} V_{h+1; \widehat P^k}^{\pi^k}\big) (s_h^k, a_h^k)\right]\cdot \text{DE}_1(\Psi,\Scal \times \Acal,1/KH)\cdot\log(K\left|\Pcal\right|/\delta)\log(KH)}\notag\\
&\quad+\text{DE}_1(\Psi,\Scal \times \Acal,1/KH)\cdot\log(K\left|\Pcal\right|/\delta)\log(KH)
\notag\\
I_h^k&:=\EE_{s'\sim P^*(s_h^k, a_h^k)}V^{\pi^k}_{h+1;\widehat P^k}(s')  - V^{\pi^k}_{h+1;\widehat P^k}(s_{h+1}^k)
\end{align}  
\end{small}
We use $\mathbb{I}\{\cdot\}$ to denote the indicator function. We define the following events which we will later show that they happen with high probability.\begin{small}
    \begin{align}
    \Ecal_1&:=\{\forall k\in[K-1]:P^\star\in\widehat\Pcal^k, \text{and} \sum_{i=0}^{k-1}\sum_{h=0}^{H-1}\mathbb H^2(P^\star(s_h^i,a_h^i)||\widehat P^k(s_h^i,a_h^i))\leq22\log(K\left|\Pcal\right|/\delta).\}\,,\label{eqn:event 1}\\
    \Ecal_2&:=\Big\{\sum_{k\in[K-1]\setminus \Kcal}\sum_{h=0}^{H-1} I_h^k\lesssim \sqrt{\sum_{k\in[K-1]\setminus \Kcal}\sum_{h=0}^{H-1} \big(\VV_{P^\star} V_{h+1; \widehat P^k}^{\pi^k}\big)(s_h^k,a_h^k)\log(1/\delta)} +  \log(1/\delta) \Big\}\label{eqn:event 2}\,,\\
    \Ecal_3&:=\Ecal_1\cap\{\forall m\in[0,\lceil\log_2(\frac{KH}{G})\rceil]: C_m\lesssim 2^{m}G+\sqrt{\log(1/\delta)\cdot C_{m+1}}+\log(1/\delta)\}\,,\label{eqn:event 3}\\
    \Ecal_4&:=    \{\sum_{k=0}^{K-1}\sum_{h=0}^{H-1}\left[\big(\VV_{P^\star} V_{h+1}^{\pi^k}\big)(s_h^k,a_h^k)\right]\lesssim\sum_{k=0}^{K-1}\var_{\pi^k}+\log(1/\delta)\}\,,\label{eqn:event 4}\\
    \Ecal_5&:=\{\sum_{k=0}^{K-1}\sum_{h=1}^H r(s_h^k, a_h^k)-\sum_{k=0}^{K-1} V^{\pi^k}_{0;P^*}\lesssim \sqrt{\sum_{k=0}^{K-1}\var_{\pi^k}\log(1/\delta)}+\log(1/\delta)\}\,,\label{eqn:event 5}\\
    \Ecal&:=\Ecal_2\cap\Ecal_3\cap\Ecal_4\cap\Ecal_5\,.
\end{align}
\end{small}

   First, by the realizability assumption, the standard generalization bound for MLE (\pref{lem:mle_generalization_offline}) with simply setting $D_i$ to be the delta distribution on the realized $(s_h^k, a_h^k)$ pairs, and a union bound over $K$ episodes, we have that w.p. at least $1-\delta$, for any $k\in[0,K-1]$: \\
\begin{enumerate}
    \item[(1)] $P^\star\in\widehat\Pcal^k$; 
    \item[(2)]         \begin{equation}
            \sum_{i=0}^{k-1}\sum_{h=0}^{H-1}\mathbb H^2(P^\star(s_h^i,a_h^i)||\widehat P^k(s_h^i,a_h^i))\leq22\log(K\left|\Pcal\right|/\delta).\label{eqn: generalization online}
        \end{equation}
        \end{enumerate}
    This directly indicates that 
    \begin{equation}
        P(\mathbb{I}\{\Ecal_1\})\geq 1-\delta\,.\label{eqn: event 1 prob}
    \end{equation}
    Under event $\Ecal_1$, with the realizability in above (1), and by the optimistic algorithm design $(\pi^k,\widehat{P}^k)\leftarrow \argmax_{\pi\in\Pi, P\in\widehat\Pcal^k} V_{0; P}^\pi(s_0)$, for any $k\in[0,K-1]$, we have the following optimism guarantee
\begin{align}
    V^\star_{0;P^\star}&\leq \max_{\pi\in\Pi,P\in\widehat \Pcal^k} V^\pi_{0;P}=V^{\pi^k}_{0;\widehat P^k}.\notag
\end{align}

Then, under event $\Ecal_1$, we use \pref{lem: new eluder pigeon lemma} and \pref{eqn: generalization online} to get the following:

There exists a set $\Kcal\subseteq[K-1]$ such that 
\begin{itemize}
    \item $|\Kcal| \leq 13\log^2(88\log(K\left|\Pcal\right|/\delta) KH)\cdot DE_1(\Psi, \Scal \times \Acal, 1/(176\log(K\left|\Pcal\right|/\delta)KH))$
    \item And
    \begin{align}
        &\sum_{k \in [K-1]\setminus \Kcal}\sum_{h=0}^{H-1} \mathbb H^2\Big(P^\star(s_h^k,a_h^k)\Mid \widehat P^k\big(s_h^k,a_h^k)\Big)\notag \\
        &\leq \text{DE}_1(\Psi,\Scal \times \Acal,1/KH)\cdot(2 + 154\log(K\left|\Pcal\right|/\delta)\log(KH)) + 1\notag\\
        &\lesssim \text{DE}_1(\Psi,\Scal \times \Acal,1/KH)\cdot\log(K\left|\Pcal\right|/\delta)\log(KH)\,.\label{eqn: mle}
    \end{align}
\end{itemize}
We upper bound the regret with optimism, and by dividing $k\in[K-1]$ into $\Kcal$ and $[K-1]\setminus \Kcal$ with the assumption that the trajectory-wise cumulative reward is normalized in [0,1], as follows

\begin{align}
    &\sum_{k=0}^{K-1} V^\star_{0;P^\star}-\sum_{k=0}^{K-1}\sum_{h=1}^H r(s_h^k, a_h^k) \notag\\
    &\leq |\Kcal| +\sum_{k\in[K-1]\setminus \Kcal}\left(V^{\pi^k}_{0;\widehat P^k}-\sum_{h=0}^{H-1}r(s_h^k, a_h^k)\right)\notag\\
    &\lesssim  \log^2( \log(K\left|\Pcal\right|/\delta) KH)\cdot DE_1(\Psi, \Scal \times \Acal, 1/( \log(K\left|\Pcal\right|/\delta)KH))\notag\\
    &+\sum_{k\in[K-1]\setminus \Kcal}\left(V^{\pi^k}_{0;\widehat P^k}-\sum_{h=0}^{H-1}r(s_h^k, a_h^k)\right)\label{eqn: regret 1}\,.
\end{align}
We then do the following decomposition. Note that for any $k\in[K-1]$, policy $\pi^k$ is deterministic. We have that for any $k\in[K-1]$
\begin{align}
    & V^{\pi^k}_{0;\widehat P^k}(s_0^k)-\sum_{h=0}^{H-1} r(s_h^k, a_h^k) \notag \\
    &= Q^{\pi^k}_{0;\widehat P^k}(s_0^k, a_0^k)-\sum_{h=0}^{H-1} r(s_h^k, a_h^k)\notag \\
        &= r(s_0^k,a_0^k)+\EE_{s'\sim \widehat P^k(s_0^k, a_0^k)}V^{\pi^k}_{1;\widehat P^k}(s')-\sum_{h=0}^{H-1} r(s_h^k, a_h^k)\notag \\
    & = \EE_{s'\sim \widehat P^k(s_0^k, a_0^k)}V^{\pi^k}_{1;\widehat P^k}(s') - \sum_{h=1}^H r(s_h^k, a_h^k)\notag \\
    & = \EE_{s'\sim P^*(s_0^k, a_0^k)}V^{\pi^k}_{1;\widehat P^k}(s') - \sum_{h=1}^H r(s_h^k, a_h^k) + \EE_{s'\sim \widehat P^k(s_0^k, a_0^k)}V^{\pi^k}_{1;\widehat P^k}(s') - \EE_{s'\sim P^*(s_0^k, a_0^k)}V^{\pi^k}_{1;\widehat P^k}(s')\notag \\
    & = V^{\pi^k}_{1;\widehat P^k}(s_1^k)-\sum_{h=1}^{H-1} r(s_h^k, a_h^k) + \underbrace{\EE_{s'\sim P^*(s_0^k, a_0^k)}V^{\pi^k}_{1;\widehat P^k}(s')  - V^{\pi^k}_{1;\widehat P^k}(s_1^k)}_{I_0^k}\notag \\
    &\quad + \EE_{s'\sim \widehat P^k(s_0^k, a_0^k)}V^{\pi^k}_{1;\widehat P^k}(s') - \EE_{s'\sim P^*(s_0^k, a_0^k)}V^{\pi^k}_{1;\widehat P^k}(s')\notag\,,
\end{align}
where we use the Bellman equation for several times.

Then, by doing this recursively, we can get for any $k\in[K-1]$
\begin{align}
    &V^{\pi^k}_{0;\widehat P^k}(s_h^k)-\sum_{h=0}^{H-1} r(s_h^k, a_h^k) \notag\\
    &\leq \sum_{h=0}^{H-1} I_h^k + \sum_{h=0}^{H-1}  \left|\EE_{s'\sim \widehat P^k(s_h^k, a_h^k)}V^{\pi^k}_{h+1;\widehat P^k}(s') - \EE_{s'\sim P^*(s_h^k, a_h^k)}V^{\pi^k}_{{h+1};\widehat P^k}(s')\right|\label{eqn:regret2}
\end{align}
Therefore, 
\begin{align}
    &\sum_{k\in[K-1]\setminus \Kcal}(V^{\pi^k}_{0;\widehat P^k}(s_h^k)-\sum_{h=0}^{H-1} r(s_h^k, a_h^k)) \notag\\
    &\leq\sum_{k\in[K-1]\setminus \Kcal}\sum_{h=0}^{H-1} I_h^k + \sum_{k\in[K-1]\setminus \Kcal}\sum_{h=0}^{H-1}  \left|\EE_{s'\sim \widehat P^k(s_h^k, a_h^k)}V^{\pi^k}_{h+1;\widehat P^k}(s') - \EE_{s'\sim P^*(s_h^k, a_h^k)}V^{\pi^k}_{{h+1};\widehat P^k}(s')\right|\label{eqn:regret2}
\end{align}
Next we bound $\sum_{k\in[K-1]\setminus \Kcal}\sum_{h=0}^{H-1} I_h^k $. Note that by Azuma Bernstein's inequality, with probability at least $1-\delta$
\begin{align}
    \sum_{k\in[K-1]\setminus \Kcal}\sum_{h=0}^{H-1} I_h^k \leq \sqrt{2\sum_{k\in[K-1]\setminus \Kcal}\sum_{h=0}^{H-1} \big(\VV_{P^\star} V_{h+1; \widehat P^k}^{\pi^k}\big)(s_h^k,a_h^k)\log(1/\delta)} +  \frac{2}{3}\log(1/\delta) \label{eqn:regret 3}
\end{align}
This directly indicates that 
\begin{equation}
    P(\mathbb I\{\Ecal_2\})\geq 1-\delta\,. \label{eqn:event 2 probab}
\end{equation}
Then, we propose the following lemma.
\begin{lemma}[Bound of sum of mean value differences for online RL]\label{lem: online sum mean value difference}
    Under event $\Ecal_1$, we have 
    \begin{small}
           \begin{align}
        &\sum_{k\in[K-1]\setminus \Kcal}\sum_{h=0}^{H-1}    \left|\EE_{s'\sim \widehat P^k(s_h^k, a_h^k)}V^{\pi^k}_{h+1;\widehat P^k}(s') - \EE_{s'\sim P^*(s_h^k, a_h^k)}V^{\pi^k}_{{h+1};\widehat P^k}(s')\right| \notag\\
        &\lesssim\sqrt{\sum_{k\in[K-1]\setminus \Kcal}\sum_{h=0}^{H-1}  \left[\big(\VV_{P^\star} V_{h+1; \widehat P^k}^{\pi^k}\big) (s_h^k, a_h^k)\right]\cdot\text{DE}_1(\Psi,\Scal \times \Acal,1/KH)\cdot\log(K\left|\Pcal\right|/\delta)\log(KH)}\notag\\
        &\quad+\text{DE}_1(\Psi,\Scal \times \Acal,1/KH)\cdot\log(K\left|\Pcal\right|/\delta)\log(KH).\notag
    \end{align}
    \end{small}
\end{lemma}
\noindent\textbf{Proof} [Proof of \pref{lem: online sum mean value difference}]  
% To simplify the presentation, we use $ $ to denote $ $.
Under event $\Ecal_1$, we have\begin{small}
    \begin{align}
    &\sum_{k\in[K-1]\setminus \Kcal}\sum_{h=0}^{H-1} \left|\EE_{s'\sim \widehat P^k(s_h^k, a_h^k)}V^{\pi^k}_{h+1;\widehat P^k}(s') - \EE_{s'\sim P^*(s_h^k, a_h^k)}V^{\pi^k}_{{h+1};\widehat P^k}(s')\right|\notag\\
    &\leq 4\sum_{k\in[K-1]\setminus \Kcal}\sum_{h=0}^{H-1} \left[\sqrt{\big(\VV_{P^\star} V_{h+1; \widehat P^k}^{\pi^k}\big)(s_h^k, a_h^k)D_\triangle\Big(V_{h+1; \widehat P^k}^{\pi^k}\big(s'\sim P^\star(s_h^k, a_h^k)\big)\Mid  V_{h+1; \widehat P^k}^{\pi^k}(s'\sim \widehat P^k\big(s_h^k, a_h^k)\big)\Big)}\right]\notag\\
&\quad+5\sum_{k\in[K-1]\setminus \Kcal}\sum_{h=0}^{H-1}\left[D_\triangle\Big(V_{h+1; \widehat P^k}^{\pi^k}\big(s'\sim P^\star(s_h^k, a_h^k)\big)\Mid V_{h+1; \widehat P^k}^{\pi^k}(s'\sim \widehat P^k\big(s_h^k, a_h^k)\big)\Big)\right]\notag\\
&\leq 8\sum_{k\in[K-1]\setminus \Kcal}\sum_{h=0}^{H-1} \left[\sqrt{\big(\VV_{P^\star} V_{h+1; \widehat P^k}^{\pi^k}\big)(s_h^k, a_h^k)\mathbb H^2\Big(V_{h+1; \widehat P^k}^{\pi^k}\big(s'\sim P^\star(s_h^k, a_h^k)\big)\Mid  V_{h+1; \widehat P^k}^{\pi^k}(s'\sim \widehat P^k\big(s_h^k, a_h^k)\big)\Big)}\right]\notag\\
&\quad+20\sum_{k\in[K-1]\setminus \Kcal}\sum_{h=0}^{H-1} \left[\mathbb H^2\Big(V_{h+1; \widehat P^k}^{\pi^k}\big(s'\sim  P^\star(s_h^k, a_h^k)\big)\Mid V_{h+1; \widehat P^k}^{\pi^k}(s'\sim \widehat P^k\big(s_h^k, a_h^k)\big)\Big)\right]\notag\\
            &\leq 8\sum_{k\in[K-1]\setminus \Kcal}\sum_{h=0}^{H-1}   \left[\sqrt{\big(\VV_{P^\star} V_{h+1; \widehat P^k}^{\pi^k}\big)(s_h^k, a_h^k)\mathbb H^2\Big(P^\star(s_h^k, a_h^k)\Mid \widehat P^k\big(s_h^k, a_h^k)\Big)}\right]\notag \\
            &\quad +20\sum_{k\in[K-1]\setminus \Kcal}\sum_{h=0}^{H-1} \left[\mathbb H^2\Big(P^\star(s_h^k, a_h^k)\Mid \widehat P^k\big(s_h^k, a_h^k)\Big)\right]\notag\\
                        &\leq 8 \sqrt{\sum_{k\in[K-1]\setminus \Kcal}\sum_{h=0}^{H-1} \left[\big(\VV_{P^\star} V_{h+1; \widehat P^k}^{\pi^k}\big)(s_h^k, a_h^k)\right]\cdot\sum_{k\in[K-1]\setminus \Kcal}\sum_{h=0}^{H-1}  \left[\mathbb H^2\Big(P^\star(s_h^k, a_h^k)\Mid \widehat P^k\big(s_h^k, a_h^k)\Big)\right]}\notag\\
&\quad+20\sum_{k\in[K-1]\setminus \Kcal}\sum_{h=0}^{H-1} \left[\mathbb H^2\Big(P^\star (s_h^k, a_h^k)\Mid \widehat P^k\big (s_h^k, a_h^k)\Big)\right]\notag\\
&\lesssim   \sqrt{\sum_{k\in[K-1]\setminus \Kcal}\sum_{h=0}^{H-1}  \left[\big(\VV_{P^\star} V_{h+1; \widehat P^k}^{\pi^k}\big) (s_h^k, a_h^k)\right]\cdot \text{DE}_1(\Psi,\Scal \times \Acal,1/KH)\cdot\log(K\left|\Pcal\right|/\delta)\log(KH)}\notag\\
&\quad+\text{DE}_1(\Psi,\Scal \times \Acal,1/KH)\cdot\log(K\left|\Pcal\right|/\delta)\log(KH)\label{eqn: mean to variance bound}\,,
\end{align}
\end{small}
    
where in the first inequality, we use \pref{lem: mean to variance} to bound the difference of two means $\EE_{s'\sim P^\star (s_h^k, a_h^k)} V^{\pi^k}_{h+1;\widehat P^k}(s') - \EE_{s'\sim \widehat P^k (s_h^k, a_h^k)} V^{\pi^*}_{h+1;\widehat P^k}(s')$ using variances and the triangle discrimination; in the second inequality we use the fact that that triangle discrimination is equivalent to squared Hellinger distance, i.e., $D_\triangle \leq 4 \mathbb H^2$; the third inequality is via data processing inequality on the squared Hellinger distance; the fourth inequality is by the Cauchy–Schwarz inequality; the last inequality holds under $\Ecal_1$ by \pref{eqn: mle}.

The next lemma shows that the event $\mathcal{E}_3$ happens with high probability.

\begin{lemma}[Recursion Event Lemma]\label{lem: event 3 lemma}
Event $\Ecal_3$ happens with high probability. Specifically, we have
\begin{equation}
    P(\mathbb I\{\Ecal_3\})\geq 1-(1+\lceil\log_2(\frac{KH}{G})\rceil)\delta.
\end{equation}
\end{lemma}
\noindent\textbf{Proof}[Proof of \pref{lem: event 3 lemma}]
    Let $\Delta_{h+1}^{\pi^k}:=V_{h+1; \widehat P^k}^{\pi^k}-V_{h+1}^{\pi^k}$. First, under event $\Ecal_1$, with happens with probability at least $1-\delta$ by \pref{eqn: event 1 prob}, and also note that $\pi^k$ is deterministic for any $k\in[K-1]$, we can prove the following \begin{small}
        \begin{align}
& \sum_{k\in[K-1]\setminus \Kcal}\sum_{h=0}^{H-1} \left[\left|(\Delta_{h}^{\pi^k})(s_{h}^k)-\big(P^\star\Delta_{h+1}^{\pi^k}\big)(s_h^k,a_h^k)\right|\right] \\
    &= \sum_{k\in[K-1]\setminus \Kcal}\sum_{h=0}^{H-1} \left[\left|({V}_{h;\widehat P^k}^{\pi^k})(s_{h}^k)-\big(P^\star{V}_{h+1;\widehat P^k}^{\pi^k}\big)(s_h^k,a_h^k)-\Big(({V}_{h}^{\pi^k})(s_{h}^k)-\big(P^\star{V}_{h+1}^{\pi^k}\big)(s_h^k,a_h^k)\Big)\right|\right]\notag\\
    &=\sum_{k\in[K-1]\setminus \Kcal}\sum_{h=0}^{H-1} \left[\left|r(s_h^k,a_h^k)+\big(\widehat P^k{V}_{h+1;\widehat P^k}^{\pi^k}\big)(s_h^k,a_h^k)-\big(P^\star{V}_{h+1;\widehat P^k}^{\pi^k}\big)(s_h^k,a_h^k)-r(s_h^k,a_h^k)\right|\right]\notag\\
        &=\sum_{k\in[K-1]\setminus \Kcal}\sum_{h=0}^{H-1} \left[\left|\big(\widehat P^k{V}_{h+1;\widehat P^k}^{\pi^k}\big)(s_h^k,a_h^k)-\big(P^\star{V}_{h+1;\widehat P^k}^{\pi^k}\big)(s_h^k,a_h^k)\right|\right]\notag\\
        &= \sum_{k\in[K-1]\setminus \Kcal}\sum_{h=0}^{H-1} \left[\left|\EE_{s'\sim P^\star(s_h^k,a_h^k)}\left[V^{\pi^k}_{h+1;\widehat P^k}(s')\right]-\EE_{s'\sim\widehat P^k(s_h^k,a_h^k)}\left[V^{\pi^k}_{h+1;\widehat P^k}(s')\right]\right|\right]\notag\\
        &\lesssim   \sqrt{\sum_{k\in[K-1]\setminus \Kcal}\sum_{h=0}^{H-1}  \left[\big(\VV_{P^\star} V_{h+1; \widehat P^k}^{\pi^k}\big) (s_h^k, a_h^k)\right]\cdot \text{DE}_1(\Psi,\Scal \times \Acal,1/KH)\cdot\log(K\left|\Pcal\right|/\delta)\log(KH)}\notag\\
&\quad+\text{DE}_1(\Psi,\Scal \times \Acal,1/KH)\cdot\log(K\left|\Pcal\right|/\delta)\log(KH)\notag\\
        &=G\label{eqn: help 11}
\end{align}
    \end{small}

 where the first equality is by the definition of $\Delta_{h+1}^{\pi^k}$, the inequality holds under $\Ecal_1$ by \pref{lem: online sum mean value difference}, and the last equality is by definition of $A$ and $G$.

Under event $\Ecal_1$, with probability at least $1-\lceil\log_2(\frac{KH}{G})\rceil\delta$, for any $m\in[0,\lceil\log_2(\frac{KH}{G})\rceil]$
\begin{small}
    \begin{align}
    C_m&=\sum_{k\in[K-1]\setminus \Kcal}\sum_{h=0}^{H-1} \left[\big(\VV_{P^\star}( V_{h+1; \widehat P^k}^{\pi^k}-V_{h+1}^{\pi^k})^{2^m}\big)(s_h^k,a_h^k)\right]\notag\\
    &=\sum_{k\in[K-1]\setminus \Kcal}\sum_{h=0}^{H-1} \left[\big(P^\star(\Delta_{h+1}^{\pi^k})^{2^{m+1}}\big)(s_h^k,a_h^k)-\big((P^\star(\Delta_{h+1}^{\pi^k})^{2^{m}})(s_h^k,a_h^k)\big)^2\right]\notag\\
    &=\sum_{k\in[K-1]\setminus \Kcal}\sum_{h=0}^{H-1}\left[(\Delta_{h+1}^{\pi^k})^{2^{m+1}}(s_{h+1}^k)\right]-\sum_{k\in[K-1]\setminus \Kcal}\sum_{h=0}^{H-1} \left[\big((P^\star(\Delta_{h+1}^{\pi^k})^{2^{m}})(s_h^k,a_h^k)\big)^2\right]\notag\\
    &\quad+\sum_{k\in[K-1]\setminus \Kcal}\sum_{h=0}^{H-1}\left(\EE_{s\sim P^*(s_h^k,a_h^k)}\left[(\Delta_{h+1}^{\pi^k})^{2^{m+1}}(s)\right]-(\Delta_{h+1}^{\pi^k})^{2^{m+1}}(s_{h+1}^k)\right)\notag\\
    &\leq \sum_{k\in[K-1]\setminus \Kcal}\sum_{h=0}^{H-1} \left[(\Delta_{h}^{\pi^k})^{2^{m+1}}(s_h^k)-\big((P^\star(\Delta_{h+1}^{\pi^k})^{2^{m}})(s_h^k,a_h^k)\big)^2\right]\notag\\
    &\quad+\sum_{k\in[K-1]\setminus \Kcal}\sum_{h=0}^{H-1}\left(\EE_{s\sim P^*(s_h^k,a_h^k)}\left[(\Delta_{h+1}^{\pi^k})^{2^{m+1}}(s)\right]-(\Delta_{h+1}^{\pi^k})^{2^{m+1}}(s_h^k)\right)\notag\\
    &\lesssim \sum_{k\in[K-1]\setminus \Kcal}\sum_{h=0}^{H-1} \left[(\Delta_{h}^{\pi^k})^{2^{m+1}}(s_h^k)-\big((P^\star(\Delta_{h+1}^{\pi^k})^{2^{m}})(s_h^k,a_h^k)\big)^2\right]+\log(1/\delta)\notag\\
    &\quad+\sqrt{\sum_{k\in[K-1]\setminus \Kcal}\sum_{h=0}^{H-1} \VV_{P^*}\left((V_{h+1; \widehat P^k}^{\pi^k}-V_{h+1}^{\pi^k})^{2^{m+1}}\right)(s_h^k,a_h^k)\log(1/\delta)}\notag\\
            &=\sum_{k\in[K-1]\setminus \Kcal}\sum_{h=0}^{H-1} \left[\Big((\Delta_{h}^{\pi^k})^{2^m}(s_h^k)+(P^\star(\Delta_{h+1}^{\pi^k})^{2^{m}})(s_h^k,a_h^k)\Big)\cdot\Big((\Delta_{h}^{\pi^k})^{2^m}(s_h^k)-(P^\star(\Delta_{h+1}^{\pi^k})^{2^{m}})(s_h^k,a_h^k)\Big)\right] \notag\\
    &\quad+\sqrt{\log(1/\delta)\cdot C_{m+1}}+\log(1/\delta)\notag\\
    &=\sum_{k\in[K-1]\setminus \Kcal}\sum_{h=0}^{H-1} \left[\Big((\Delta_{h}^{\pi^k})^{2^m}(s_h^k)+(P^\star(\Delta_{h+1}^{\pi^k})^{2^{m}})(s_h^k,a_h^k)\Big)\cdot\Big((\Delta_{h}^{\pi^k})^{2^m}(s_h^k)-(P^\star((\Delta_{h+1}^{\pi^k})^2)^{2^{m-1}})(s_h^k,a_h^k)\Big)\right] \notag\\
    &\quad+\sqrt{\log(1/\delta)\cdot C_{m+1}}+\log(1/\delta)\notag\\
    &\leq\sum_{k\in[K-1]\setminus \Kcal}\sum_{h=0}^{H-1} \left[\Big((\Delta_{h}^{\pi^k})^{2^m}(s_h^k)+(P^\star(\Delta_{h+1}^{\pi^k})^{2^{m}})(s_h^k,a_h^k)\Big)\cdot\Big((\Delta_{h}^{\pi^k})^{2^m}(s_h^k)-((P^\star (\Delta_{h+1}^{\pi^k})^2)(s_h^k,a_h^k))^{2^{m-1}}\Big)\right] \notag\\
    &\quad+\sqrt{\log(1/\delta)\cdot C_{m+1}}+\log(1/\delta)\notag\\
 &\leq 2^{m} \sum_{k\in[K-1]\setminus \Kcal}\sum_{h=0}^{H-1} \left[\left|(\Delta_{h}^{\pi^k})^2(s_h^k)-((P^\star \Delta_{h+1}^{\pi^k})(s_h^k,a_h^k))^2\right|\right]+ \sqrt{\log(1/\delta)\cdot C_{m+1}}+\log(1/\delta)\notag\\
 &=2^{m} \sum_{k\in[K-1]\setminus \Kcal}\sum_{h=0}^{H-1} \left[\left|\big((\Delta_{h}^{\pi^k})(s_h^k)+(P^\star \Delta_{h+1}^{\pi^k})(s_h^k,a_h^k)\big)\cdot\big((\Delta_{h}^{\pi^k})(s_h^k)-(P^\star \Delta_{h+1}^{\pi^k})(s_h^k,a_h^k)\big)\right|\right]\notag\\
 &\quad+ \sqrt{\log(1/\delta)\cdot C_{m+1}}+\log(1/\delta)\notag\\
    &\leq 2\cdot2^{m}\sum_{k\in[K-1]\setminus \Kcal}\sum_{h=0}^{H-1} \left[\left|(\Delta_{h}^{\pi^k}) (s_h^k)-\big(P^\star\Delta_{h+1}^{\pi^k}\big) (s_h^k,a_h^k)\right|\right]+\sqrt{\log(1/\delta)\cdot C_{m+1}}+\log(1/\delta)\notag\\
    &\lesssim 2^{m}G+\sqrt{\log(1/\delta)\cdot C_{m+1}}+\log(1/\delta)
    \label{eqn: help 10}\,,
\end{align}
\end{small}

where in the first inequality we change the index, the second inequality holds with probability at least $1-\delta$ by Azuma Bernstain's inequality, the third inequality holds because that $E[X^{2^{m-1}}]\geq (E[X])^{2^{m-1}}$ for $m\geq 1$ and $X\geq 0$, the fourth inequality holds by keep using $a^2-b^2=(a+b)(a-b)$, then with $E[X^2]\geq E[X]^2$, and the assumption that the trajectory-wise total reward is normalized in $[0,1]$, the last inequality holds under $\Ecal_1$ by \pref{eqn: help 11}, and we take a union bound to get this hold for all $m\in[0,\lceil\log_2(\frac{KH}{G})\rceil]$ with probability at least $1-\lceil\log_2(\frac{KH}{G})\rceil\delta$ (because for each $m\in[0,\lceil\log_2(\frac{KH}{G})\rceil]$ we need to apply the Azuma Bernstain's inequality once).

The above reasoning directly implies that
\begin{equation}
    P(\mathbb I\{\Ecal_3\})\geq 1-(1+\lceil\log_2(\frac{KH}{G})\rceil)\delta.
\end{equation}

Under the event $\mathcal{E}_3$, we prove the following lemma to bound $\sum_{k\in[K-1]\setminus \Kcal}\sum_{h=0}^{H-1}  \left[\big(\VV_{P^\star} V_{h+1; \widehat P^k}^{\pi^k}\big) (s_h^k, a_h^k)\right]$.
% with $\widetilde O(\sum_{k=0}^{K-1}\sum_{h=0}^{H-1} \EE_{s,a\sim d^{\pi^*}_h}\left[\big(\VV_{P^\star} V_{h+1}^{\pi^k}\big)(s,a)\right]+\dimRL\log(K\left|\Pcal\right|/\delta)\log(K)$).

\begin{lemma}[Variance Conversion Lemma for online RL]\label{lem: variance recusion online}
Under event $\Ecal_3$, we have \begin{small}
    \begin{align}
    &\sum_{k\in[K-1]\setminus \Kcal}\sum_{h=0}^{H-1}  \left[\big(\VV_{P^\star} V_{h+1; \widehat P^k}^{\pi^k}\big) (s_h^k, a_h^k)\right]\notag\\
    &\leq O\Big(\sum_{k\in[K-1]\setminus \Kcal}\sum_{h=0}^{H-1} \left[\big(\VV_{P^\star} V_{h+1}^{\pi^k}\big)(s_h^k, a_h^k)\right]+\text{DE}_1(\Psi,\Scal \times \Acal,1/KH)\cdot\log(K\left|\Pcal\right|/\delta)\log(KH)\Big).\notag
\end{align}
\end{small}

\end{lemma}
\noindent\textbf{Proof}[Proof of \pref{lem: variance recusion online}] 
Under $\Ecal_3$, we have for any $m\in[0,\lceil\log_2(\frac{KH}{G})\rceil]$
\begin{equation}
    C_m\lesssim 2^{m}G+\sqrt{\log(1/\delta)\cdot C_{m+1}}+\log(1/\delta)\,.
\end{equation}
Then, by \pref{lem:recursion bound C_m}, we have 
\begin{equation}
    C_0\lesssim G\,.
\end{equation}
Also note that we have $A\leq 2B+2C_0$ since $\VV_{P^\star}(a+b)\leq 2\VV_{P^\star}(a)+2\VV_{P^\star}(b)$.
Therefore, we have
\begin{align}
    A&\leq 2B+2C_0\notag\\
    &\lesssim  B+ G\notag\\
 &=  B+\sqrt{A\cdot \text{DE}_1(\Psi,\Scal \times \Acal,1/KH)\cdot\log(K\left|\Pcal\right|/\delta)\log(KH)}\notag\\
&\quad+\text{DE}_1(\Psi,\Scal \times \Acal,1/KH)\cdot\log(K\left|\Pcal\right|/\delta)\log(KH)
\end{align}

Then, with the fact that $x\leq 2a+b^2$ if $x\leq a+b\sqrt{x}$, we have
\begin{align}
    A &\leq O\Bigg(B+\text{DE}_1(\Psi,\Scal \times \Acal,1/KH)\cdot\log(K\left|\Pcal\right|/\delta)\log(KH)\Bigg)\,,
\end{align}
which is \begin{small}
    \begin{align}
    &\sum_{k\in[K-1]\setminus \Kcal}\sum_{h=0}^{H-1} \left[\big(\VV_{P^\star} V_{h+1; \widehat P^k}^{\pi^k}\big)(s_h^k,a_h^k)\right]\notag\\
    &\leq O\bigg(\sum_{k\in[K-1]\setminus \Kcal}\sum_{h=0}^{H-1}\left[\big(\VV_{P^\star} V_{h+1}^{\pi^k}\big)(s_h^k,a_h^k)\right]+\text{DE}_1(\Psi,\Scal \times \Acal,1/KH)\cdot\log(K\left|\Pcal\right|/\delta)\log(KH)\bigg)
\end{align}
\end{small}

% For notational simplicity, we denote 
% \begin{align}
%     &A:=\sum_{k\in[K-1]\setminus \Kcal}\sum_{h=0}^{H-1} \left[\big(\VV_{P^\star} V_{h+1; \widehat P^k}^{\pi^k}\big)(s_h^k,a_h^k)\right]\notag\\
%     &B:=\sum_{k\in[K-1]\setminus \Kcal}\sum_{h=0}^{H-1}\left[\big(\VV_{P^\star} V_{h+1}^{\pi^k}\big)(s_h^k,a_h^k)\right],\notag\\
%        C_m&:=\sum_{k\in[K-1]\setminus \Kcal}\sum_{h=0}^{H-1} \left[\big(\VV_{P^\star}( V_{h+1; \widehat P^k}^{\pi^k}-V_{h+1}^{\pi^k})^{2^m}\big)(s_h^k,a_h^k)\right]\notag\\
%     G&:= 8 \sqrt{A\cdot 198\text{DELUERlogterm}}
% \notag\\
% &\quad+3960\text{DELUERlogterm}
% \end{align}

% $A_1:=\sum_{k=0}^{K-1}\sum_{h=0}^{H-1} \EE_{s,a \sim d^{\pi^k}_h}\left[\big(\VV_{P^\star} V_{h+1; \widehat P^k}^{\pi^k}\big)(s,a)\right]$, and we denote
% $B_1:=\sum_{k=0}^{K-1}\sum_{h=0}^{H-1} \EE_{s,a \sim d^{\pi^k}_h}\left[\big(\VV_{P^\star} V_{h+1}^{\pi^k}\big)(s,a)\right]$, $C_1:=\sum_{k=0}^{K-1}\sum_{h=0}^{H-1} \EE_{s,a \sim d^{\pi^k}_h}\left[\big(\VV_{P^\star}( V_{h+1; \widehat P^k}^{\pi^k}-V_{h+1}^{\pi^k})\big)(s,a)\right]$, 

By the same reasoning in Lemma 26 of \cite{zhou2023sharp}, we have that with probability at least $1-\delta$
\begin{align}
    \sum_{k\in[K-1]\setminus \Kcal}\sum_{h=0}^{H-1}\left[\big(\VV_{P^\star} V_{h+1}^{\pi^k}\big)(s_h^k,a_h^k)\right]&\leq \sum_{k=0}^{K-1}\sum_{h=0}^{H-1}\left[\big(\VV_{P^\star} V_{h+1}^{\pi^k}\big)(s_h^k,a_h^k)\right]\notag \\
    &\leq O(\sum_{k=0}^{K-1}\var_{\pi^k}+\log(1/\delta))\,.
\end{align}
This indicates that 
\begin{equation}
    P(\mathbb I\{\Ecal_4\})\geq 1-\delta\,.
\end{equation}
We can use the Azuma Bernstain's inequality to get that with probability at least $1-\delta$:
\begin{align}
    \sum_{k=0}^{K-1}\sum_{h=1}^H r(s_h^k, a_h^k)-\sum_{k=0}^{K-1} V^{\pi^k}_{0;P^*}\lesssim \sqrt{\sum_{k=0}^{K-1}\var_{\pi^k}\log(1/\delta)}+\log(1/\delta)\,.
\end{align}
This indicates that 
\begin{equation}
    P(\mathbb I\{\Ecal_5\})\geq 1-\delta\,.
\end{equation}
Then, together with \pref{lem: event 3 lemma},  \pref{eqn: event 1 prob} and \pref{eqn:event 2 probab}, we have
\begin{equation}
    P(\mathbb I\{\Ecal\})\geq 1-(5+\lceil\log_2(\frac{KH}{G})\rceil)\delta\geq 1-5KH\delta\,. \label{eqn: event prob bound}
\end{equation}

Finally, under event $\Ecal$, with all the things above (\pref{eqn: regret 1}, \pref{eqn:regret2}, \pref{eqn:regret 3},\pref{lem: online sum mean value difference}, \pref{lem: variance recusion online}), we have \begin{small}
    \begin{align}
    &\sum_{k=0}^{K-1} V^\star_{0;P^\star}-\sum_{k=0}^{K-1} V^{\pi^k}_{0;P^*}\notag\\
    &=\sum_{k=0}^{K-1} V^\star_{0;P^\star}-\sum_{k=0}^{K-1}\sum_{h=1}^H r(s_h^k, a_h^k) +\sum_{k=0}^{K-1}\sum_{h=1}^H r(s_h^k, a_h^k)-\sum_{k=0}^{K-1} V^{\pi^k}_{0;P^*}\notag\\
    &\lesssim |\Kcal| +\sum_{k\in[K-1]\setminus \Kcal}\left(V^{\pi^k}_{0;\widehat P^k}-\sum_{h=0}^{H-1}r(s_h^k, a_h^k)\right)+\sqrt{\sum_{k=0}^{K-1}\var_{\pi^k}\log(1/\delta)}+\log(1/\delta)\notag\\
    &\lesssim \log^2( \log(K\left|\Pcal\right|/\delta) KH)\cdot DE_1(\Psi, \Scal \times \Acal, 1/( \log(K\left|\Pcal\right|/\delta)KH))+\sum_{k\in[K-1]\setminus \Kcal}\left(V^{\pi^k}_{0;\widehat P^k}-\sum_{h=0}^{H-1}r(s_h^k, a_h^k)\right)\notag\\
    &\quad+\sqrt{\sum_{k=0}^{K-1}\var_{\pi^k}\log(1/\delta)}+\log(1/\delta)\notag\\
    &\lesssim \log^2( \log(K\left|\Pcal\right|/\delta) KH)\cdot DE_1(\Psi, \Scal \times \Acal, 1/( \log(K\left|\Pcal\right|/\delta)KH)) +\log(1/\delta)\notag\\
    &\quad+\sqrt{\sum_{k\in[K-1]\setminus \Kcal}\sum_{h=0}^{H-1} \big(\VV_{P^\star} V_{h+1; \widehat P^k}^{\pi^k}\big)(s_h^k,a_h^k)\log(1/\delta)}+\sqrt{\sum_{k=0}^{K-1}\var_{\pi^k}\log(1/\delta)}\notag\\
    &\quad+ \sum_{k\in[K-1]\setminus \Kcal}\sum_{h=0}^{H-1}    \left|\EE_{s'\sim \widehat P^k(s_1^k, A^k)}V^{\pi^k}_{1;\widehat P^k}(s') - \EE_{s'\sim P^*(s_1^k, A^k)}V^{\pi^k}_{1;\widehat P^k}(s')\right| \notag\\
        &\lesssim \log^2( \log(K\left|\Pcal\right|/\delta) KH)\cdot DE_1(\Psi, \Scal \times \Acal, 1/( \log(K\left|\Pcal\right|/\delta)KH))+\sqrt{\sum_{k=0}^{K-1}\var_{\pi^k}\log(1/\delta)}+\log(1/\delta)\notag\\
        &+\sqrt{(\sum_{k\in[K-1]\setminus \Kcal}\sum_{h=0}^{H-1} \left[\big(\VV_{P^\star} V_{h+1}^{\pi^k}\big)(s_h^k, a_h^k)\right]+\text{DE}_1(\Psi,\Scal \times \Acal,1/KH)\cdot\log(K\left|\Pcal\right|/\delta)\log(KH))\cdot\log(1/\delta) }\notag\\
        &\quad+\sqrt{\sum_{k\in[K-1]\setminus \Kcal}\sum_{h=0}^{H-1}  \left[\big(\VV_{P^\star} V_{h+1; \widehat P^k}^{\pi^k}\big) (s_h^k, a_h^k)\right]\cdot\text{DE}_1(\Psi,\Scal \times \Acal,1/KH)\cdot\log(K\left|\Pcal\right|/\delta)\log(KH)}\notag\\
        &\lesssim \log^2( \log(K\left|\Pcal\right|/\delta) KH)\cdot DE_1(\Psi, \Scal \times \Acal, 1/( \log(K\left|\Pcal\right|/\delta)KH))+\sqrt{\sum_{k=0}^{K-1}\var_{\pi^k}\log(1/\delta)}+\log(1/\delta)\notag\\
        &+\sqrt{(\sum_{k\in[K-1]\setminus \Kcal}\sum_{h=0}^{H-1} \left[\big(\VV_{P^\star} V_{h+1}^{\pi^k}\big)(s_h^k, a_h^k)\right]+\text{DE}_1(\Psi,\Scal \times \Acal,1/KH)\cdot\log(K\left|\Pcal\right|/\delta)\log(KH))\cdot\log(1/\delta) }\notag\\
        &+\sqrt{(\sum_{k\in[K-1]\setminus \Kcal}\sum_{h=0}^{H-1} \left[\big(\VV_{P^\star} V_{h+1}^{\pi^k}\big)(s_h^k, a_h^k)\right]+\text{DE}_1(\Psi,\Scal \times \Acal,1/KH)\cdot\log(K\left|\Pcal\right|/\delta)\log(KH))}\notag\\
        &\quad\times\sqrt{\text{DE}_1(\Psi,\Scal \times \Acal,1/KH)\cdot\log(K\left|\Pcal\right|/\delta)\log(KH))}\notag\\
        &\lesssim \log^2( \log(K\left|\Pcal\right|/\delta) KH)\cdot DE_1(\Psi, \Scal \times \Acal, 1/( \log(K\left|\Pcal\right|/\delta)KH))+\sqrt{\sum_{k=0}^{K-1}\var_{\pi^k}\log(1/\delta)}+\log(1/\delta)\notag\\
        &+\sqrt{(\sum_{k=0}^{K-1}\var_{\pi^k}+\log(1/\delta)+\text{DE}_1(\Psi,\Scal \times \Acal,1/KH)\cdot\log(K\left|\Pcal\right|/\delta)\log(KH))\cdot\log(1/\delta) }\notag\\
        &+\sqrt{\left( \sum_{k=0}^{K-1}\var_{\pi^k} + \log\left(\frac{1}{\delta}\right) + \text{DE}_1\left(\Psi, \Scal \times \Acal, \frac{1}{KH}\right) \cdot \log\left(\frac{K \left|\Pcal\right|}{\delta}\right) \log(KH) \right)}
\notag \\
&\quad\times \sqrt{\text{DE}_1\left(\Psi, \Scal \times \Acal, \frac{1}{KH}\right) \log\left(\frac{K \left|\Pcal\right|}{\delta}\right) \log(KH)}\notag\\
        &\leq O\Big(\sqrt{\sum_{k=0}^{K-1}\var_{\pi^k}\cdot\text{DE}_1(\Psi,\Scal \times \Acal,1/KH)\cdot\log(K\left|\Pcal\right|/\delta)\log(KH)}\notag\\
        &\quad+\text{DE}_1(\Psi,\Scal \times \Acal,1/KH)\cdot\log(K\left|\Pcal\right|/\delta)\log(KH)\Big)\,.\label{eqn: regret final 1}
\end{align}

\end{small}

The final result follows by replacing $\delta$ to be $\delta/(5KH)$ to make the event $\Ecal$ happen with probability at least $1-\delta$.
\subsection{Proof of  \pref{corr:online_coro_faster}}\label{app:online_coro_faster}
\noindent\textbf{Proof}[Proof of  \pref{corr:online_coro_faster}]
By \pref{lem:variance_lemma}, we have
\begin{equation}
    \var_{\pi^k}=\sum_{h=0}^{H-1}\EE_{s,a \sim d^{\pi^k}_h}\left[\big(\VV_{P^\star} V_{h+1}^{\pi^k}\big)(s,a)\right]
\end{equation}
Therefore, when $P^\star$ is deterministic, 
the $\EE_{s,a \sim d^{\pi^k}_h}\left[\big(\VV_{P^\star} V_{h+1}^{\pi^k}\big)(s,a)\right]$ terms are all 0 for any $k\in[K-1]$ and $h\in[H-1]$, and then the $\sum_{k=0}^{K-1}\var_{\pi^k}$ term in the higher order term in \pref{thm:online_theorem} is 0.

\subsection{Proof of  \pref{corr:online_coro_infinite}}\label{app:online_coro_infinite}
\noindent\textbf{Proof}[Proof of  \pref{corr:online_coro_infinite}]
    We follow the MLE guarantee for the infinite model class in \pref{lem:mle_generalization infinite} and the same proof steps in the proof of \pref{thm:online_theorem} in Appendix \ref{app:online}.

\subsection{Detailed Proofs for the Offline RL setting in \pref{sec:offline}}\label{app: offline full}

% \wen{give the MLE bound below in an independent lemma.}

% \zhiyong{Should we put it here? I think it is also used in online part} \wen{i see, let's maybe put a very general version in the supporting lemma section, and later just call it in each section when needed? }

\subsection{Proof of \pref{thm:mleoffline}}\label{app:offline}
The following is the full proof of \pref{thm:mleoffline}. 

% \wen{let's use the pref thing that i was using for consistency? i.e., \pref{thm:}}
\noindent\textbf{Proof}[Proof of \pref{thm:mleoffline}]
    First, by the realizability assumption, the standard generalization bound for MLE (\pref{lem:mle_generalization_offline}) with simply setting $D_i$ to be the delta distribution on the  $(s_h^k, a_h^k)$ pairs in the offline dataset $\mathcal{D}$, we have that w.p. at least $1-\delta$ : \\
    \begin{enumerate}
    \item[(1)] $P^\star\in\widehat\Pcal$; 
    \item[(2)]         \begin{equation}
            \frac{1}{K}\sum_{k=1}^K\sum_{h=0}^{H-1}\mathbb H^2(P^\star(s_h^k,a_h^k)||\widehat P(s_h^k,a_h^k))\leq\frac{22\log(\left|\Pcal\right|/\delta)}{K}.\label{eqn: generalization offline}
        \end{equation}
        \end{enumerate}
        % (1) $P^\star\in\widehat\Pcal$; \\
        % (2) 
        % \begin{equation}
        %     \sum_{h=0}^{H-1}\EE_{s,a\sim d^{\pi^b}_h}[H^2(P^\star(s,a)||\widehat P(s,a))]\leq\frac{22\log(\left|\Pcal\right|/\delta)}{N}.\label{eqn: generalization offline}
        % \end{equation}
        
        % \wen{I changed $\hat P$ to $\widehat P $, @zhiyong, can you make them consistent in the proof?}
        
Then, with the above realizability in (1), and by the pessimistic algorithm design $\hat\pi\leftarrow \argmax_{\pi\in\Pi}\min_{P\in\widehat\Pcal} V_{0; P}^\pi(s_0)$, $\widehat P\leftarrow \argmin_{P \in \widehat \Pcal} V_{0; P}^{\widehat \pi}(s_0)$, we have that for any $\pi^\star\in\Pi$
\begin{align}
    V_{0; P^\star}^{\pi^\star}-V_{0; P^\star}^{\hat{\pi}}&=V_{0; P^\star}^{\pi^\star}-V_{0;\widehat P}^{\pi^\star}+V_{0;\widehat P}^{\pi^\star}-V_{0; P^\star}^{\hat{\pi}}\notag\\
    &\leq V_{0; P^\star}^{\pi^\star}-V_{0;\widehat P}^{\pi^\star}+V_{0;\widehat P}^{\hat{\pi}}-V_{0; P^\star}^{\hat{\pi}}\notag\\
    &\leq V_{0; P^\star}^{\pi^\star}-V_{0;\widehat P}^{\pi^\star}\,. \label{offline gap 1}
\end{align}
We can then bound $V_{0; P^\star}^{\pi^\star}-V_{0;\widehat P}^{\pi^\star}$ using the simulation lemma (\pref{lem:simulation}):
\begin{align}
    V_{0; P^\star}^{\pi^\star}-V_{0;\widehat P}^{\pi^\star}&\leq \sum_{h=0}^{H-1} \EE_{s,a\sim d^{\pi^\star}_h} \left[ \left|\EE_{s'\sim P^\star(s,a)} V^{\pi^\star}_{h+1;\widehat P}(s') -\EE_{s'\sim \widehat P(s,a)} V^{\pi^\star}_{h+1;\widehat P}(s')    \right|\right]\,. \label{eqn: offline gap 2}
\end{align}
Then, we prove the following lemma to bound the RHS of \pref{eqn: offline gap 2}.
\begin{lemma}[Bound of sum of mean value differences for offline RL]\label{lem: offline sum mean value difference}
    With probability at least $1-\delta$, we have 
    \begin{align}
        &\sum_{h=0}^{H-1} \EE_{s,a \sim d^{\pi^\star}_h} \left[ \left|\EE_{s'\sim P^\star(s,a)} V^{\pi^\star}_{h+1;\widehat P}(s') -\EE_{s'\sim \widehat P(s,a)} V^{\pi^\star}_{h+1;\widehat P}(s')    \right|\right]\notag\\
        &\leq 8\sqrt{\sum_{h=0}^{H-1} \EE_{s,a\sim d^{\pi^\star}_h}\left[\big(\VV_{P^\star} V_{h+1; \widehat P}^{\pi^\star}\big)(s,a)\right]\cdot\frac{22C^{\pi^\star}\log(\left|\Pcal\right|/\delta)}{K}}+\frac{440C^{\pi^\star}\log(\left|\Pcal\right|/\delta)}{K}.\notag
    \end{align}
\end{lemma}
\noindent\textbf{Proof}[Proof of \pref{lem: offline sum mean value difference}]  
% To simplify the presentation, we use $\EE_{s,a \sim d^{\pi^\star}_h}$ to denote $\EE_{s,a \sim d^{\pi^\star}_h}$.
We have \begin{small}
     \begin{align}
    &\sum_{h=0}^{H-1} \EE_{s,a \sim d^{\pi^\star}_h} \left[ \left|\EE_{s'\sim P^\star(s,a)} V^{\pi^\star}_{h+1;\widehat P}(s') -\EE_{s'\sim \widehat P(s,a)} V^{\pi^\star}_{h+1;\widehat P}(s')    \right|\right]\notag\\
    &\leq 4\sum_{h=0}^{H-1} \EE_{s,a \sim d^{\pi^\star}_h} \left[\sqrt{\big(\VV_{P^\star} V_{h+1; \widehat P}^{\pi^\star}\big)(s,a)D_\triangle\Big(V_{h+1; \widehat P}^{\pi^\star}\big(s'\sim P^\star(s,a)\big)\Mid  V_{h+1; \widehat P}^{\pi^\star}(s'\sim \widehat P\big(s,a)\big)\Big)}\right]\notag\\
&\quad+5\sum_{h=0}^{H-1}\EE_{s,a \sim d^{\pi^\star}_h}\left[D_\triangle\Big(V_{h+1; \widehat P}^{\pi^\star}\big(s'\sim P^\star(s,a)\big)\Mid V_{h+1; \widehat P}^{\pi^\star}(s'\sim \widehat P\big(s,a)\big)\Big)\right]\notag\\&\leq 8\sum_{h=0}^{H-1} \EE_{s,a \sim d^{\pi^\star}_h} \left[\sqrt{\big(\VV_{P^\star} V_{h+1; \widehat P}^{\pi^\star}\big)(s,a)\mathbb H^2\Big(V_{h+1; \widehat P}^{\pi^\star}\big(s'\sim P^\star(s,a)\big)\Mid  V_{h+1; \widehat P}^{\pi^\star}(s'\sim \widehat P\big(s,a)\big)\Big)}\right]\notag\\
&\quad+20\sum_{h=0}^{H-1}\EE_{s,a \sim d^{\pi^\star}_h}\left[\mathbb H^2\Big(V_{h+1; \widehat P}^{\pi^\star}\big(s'\sim  P^\star(s,a)\big)\Mid V_{h+1; \widehat P}^{\pi^\star}(s'\sim \widehat P\big(s,a)\big)\Big)\right]\notag\\
            &\leq 8\sum_{h=0}^{H-1} \EE_{s,a \sim d^{\pi^\star}_h} \left[\sqrt{\big(\VV_{P^\star} V_{h+1; \widehat P}^{\pi^\star}\big)(s,a)\mathbb H^2\Big(P^\star(s,a)\Mid \widehat P\big(s,a)\Big)}\right]+20\sum_{h=0}^{H-1}\EE_{s,a \sim d^{\pi^\star}_h}\left[\mathbb H^2\Big(P^\star(s,a)\Mid \widehat P\big(s,a)\Big)\right]\label{eqn:help 1}
\end{align}
\end{small}
   
where in the first inequality, we use \pref{lem: mean to variance} to bound the difference of two means $\EE_{s'\sim P^\star(s,a)} V^{\pi^\star}_{h+1;\widehat P}(s') - \EE_{s'\sim \widehat P(s,a)} V^{\pi^*}_{h+1;\widehat P}(s')$ using variances and the triangle discrimination; in the second inequality we use the fact that that triangle discrimination is equivalent to squared Hellinger distance, i.e., $D_\triangle \leq 4 \mathbb H^2$; the third inequality is via data processing inequality on the squared Hellinger distance. Next, starting from \pref{eqn:help 1}, with probability at least $1-\delta$, we have

\begin{small}
    \begin{align}
&8\sum_{h=0}^{H-1} \EE_{s,a \sim d^{\pi^\star}_h} \left[\sqrt{\big(\VV_{P^\star} V_{h+1; \widehat P}^{\pi^\star}\big)(s,a)\mathbb H^2\Big(P^\star(s,a)\Mid \widehat P\big(s,a)\Big)}\right]+20\sum_{h=0}^{H-1}\EE_{s,a \sim d^{\pi^\star}_h}\left[\mathbb H^2\Big(P^\star(s,a)\Mid \widehat P\big(s,a)\Big)\right]\notag \\
     &\leq 8 \sqrt{\sum_{h=0}^{H-1} \EE_{s,a \sim d^{\pi^\star}_h}\left[\big(\VV_{P^\star} V_{h+1; \widehat P}^{\pi^\star}\big)(s,a)\right]\cdot\sum_{h=0}^{H-1} \EE_{s,a \sim d^{\pi^\star}_h}\left[\mathbb H^2\Big(P^\star(s,a)\Mid \widehat P\big(s,a)\Big)\right]}\notag\\
&\quad+20\sum_{h=0}^{H-1}\EE_{s,a \sim d^{\pi^\star}_h}\left[\mathbb H^2\Big(P^\star(s,a)\Mid \widehat P\big(s,a)\Big)\right]\notag\\
&\leq 8 \sqrt{\sum_{h=0}^{H-1} \EE_{s,a \sim d^{\pi^\star}_h}\left[\big(\VV_{P^\star} V_{h+1; \widehat P}^{\pi^\star}\big)(s,a)\right]\cdot C^{\pi^\star}\frac{1}{K}\sum_{k=1}^K\sum_{h=0}^{H-1}\mathbb H^2(P^\star(s_h^k,a_h^k)||\widehat P(s_h^k,a_h^k))}\notag\\
&\quad+20C^{\pi^\star}\frac{1}{K}\sum_{k=1}^K\sum_{h=0}^{H-1}\mathbb H^2(P^\star(s_h^k,a_h^k)||\widehat P(s_h^k,a_h^k))\notag\\
&\leq 8 \sqrt{\sum_{h=0}^{H-1} \EE_{s,a \sim d^{\pi^\star}_h}\left[\big(\VV_{P^\star} V_{h+1; \widehat P}^{\pi^\star}\big)(s,a)\right]\cdot\frac{22C^{\pi^\star}\log(\left|\Pcal\right|/\delta)}{K}}+\frac{440C^{\pi^\star}\log(\left|\Pcal\right|/\delta)}{K}\label{eqn: proof mle offline 1}\,,
\end{align}
\end{small}

where the first inequality is by the Cauchy–Schwarz inequality; the second inequality is by the definition of single policy coverage (\pref{def: coverage offline}); the last inequality holds with probability at least $1-\delta$ with \pref{eqn: generalization offline}. Substituting \pref{eqn: proof mle offline 1} into \pref{eqn:help 1} ends our proof.

% \wen{can we abstract the below long derivation into a lemma? so that later when bounding $C$ we can just call the lemma?}

% \wen{can we abstract the derivation for upper bounding $A$ below into a lemma as well? want to call this in the proof sketch in the main body.}

We denote $\mathcal{E}$ as the event that \pref{lem: offline sum mean value difference} holds. Under the event $\mathcal{E}$, we prove the following lemma to bound $\sum_{h=0}^{H-1} \EE_{s,a \sim d^{\pi^\star}_h}\left[\big(\VV_{P^\star} V_{h+1; \widehat P}^{\pi^\star}\big)(s,a)\right]$ with $\widetilde O(\sum_{h=0}^{H-1} \EE_{s,a\sim d^{\pi^*}_h}\left[\big(\VV_{P^\star} V_{h+1}^{\pi^*}\big)(s,a)\right]+C^{\pi^*}\log(\left|\Pcal\right|/\delta)/{K}$).

\begin{lemma}[Variance Conversion Lemma for offline RL]\label{lem: variance recusion}
Under event $\mathcal{E}$, we have \begin{small}
 \begin{align}
    \sum_{h=0}^{H-1} \EE_{s,a \sim d^{\pi^\star}_h}\left[\big(\VV_{P^\star} V_{h+1; \widehat P}^{\pi^\star}\big)(s,a)\right]\leq O\Big(\sum_{h=0}^{H-1} \EE_{s,a\sim d^{\pi^*}_h}\left[\big(\VV_{P^\star} V_{h+1}^{\pi^*}\big)(s,a)\right]+C^{\pi^*}\frac{\log(\left|\Pcal\right|/\delta)}{K}\Big).\notag
\end{align}   
\end{small}

\end{lemma}
\noindent\textbf{Proof}[Proof of \pref{lem: variance recusion}]
For notational simplicity, we denote $A:=\sum_{h=0}^{H-1} \EE_{s,a \sim d^{\pi^\star}_h}\left[\big(\VV_{P^\star} V_{h+1; \widehat P}^{\pi^\star}\big)(s,a)\right]$, and we denote \\
$B:=\sum_{h=0}^{H-1} \EE_{s,a \sim d^{\pi^\star}_h}\left[\big(\VV_{P^\star} V_{h+1}^{\pi^\star}\big)(s,a)\right]$, \\$C:=\sum_{h=0}^{H-1} \EE_{s,a \sim d^{\pi^\star}_h}\left[\big(\VV_{P^\star}( V_{h+1; \widehat P}^{\pi^\star}-V_{h+1}^{\pi^\star})\big)(s,a)\right]$, then we have 
\begin{align}
    A\leq 2B+2C,\notag
\end{align}
since $\VV_{P^\star}(a+b)\leq 2\VV_{P^\star}(a)+2\VV_{P^\star}(b)$.

Let $\Delta_{h+1}^{\pi^\star}:=V_{h+1; \widehat P}^{\pi^\star}-V_{h+1}^{\pi^\star}$. Then, w.p. at least $1-\delta$, we have
\begin{align}
    C&=\sum_{h=0}^{H-1} \EE_{s,a\sim d^{ \pi^\star}_h}\left[\big(P^\star(\Delta_{h+1}^{\pi^\star})^2\big)(s,a)-\big(P^\star\Delta_{h+1}^{\pi^\star}\big)^2(s,a)\right]\notag\\
    &=\sum_{h=0}^{H-1} \EE_{s\sim d^{\pi^\star}_{h+1}}\left[(\Delta_{h+1}^{\pi^\star})^2(s)\right]-\sum_{h=0}^{H-1} \EE_{s,a\sim d^{ \pi^\star}_h}\left[\big(P^\star\Delta_{h+1}^{\pi^\star}\big)^2(s,a)\right]\notag\\
    &\leq \sum_{h=0}^{H-1} \EE_{s,a\sim d^{ \pi^\star}_h}\left[(\Delta_{h}^{\pi^\star})^2(s)-\big(P^\star\Delta_{h+1}^{\pi^\star}\big)^2(s,a)\right]\notag\\
    &=\sum_{h=0}^{H-1} \EE_{s,a\sim d^{\pi^\star}_h}\left[\Big((\Delta_{h}^{\pi^\star})(s)+\big(P^\star\Delta_{h+1}^{\pi^\star}\big)(s,a)\Big)\cdot\Big((\Delta_{h}^{\pi^\star})(s)-\big(P^\star\Delta_{h+1}^{\pi^\star}\big)(s,a)\Big)\right],\label{eqn: help 2}
\end{align}
where the first equality is by the definition of variance, the second equality holds as $d^{\pi^\star}_h$ is the occupancy measure also generated under $P^\star$, the first inequality is just changing the index, the third equality holds as $a^2-b^2=(a+b)\cdot (a-b)$. Starting from \pref{eqn: help 2}, we have
\begin{align}
&\sum_{h=0}^{H-1} \EE_{s,a\sim d^{\pi^\star}_h}\left[\Big((\Delta_{h}^{\pi^\star})(s)+\big(P^\star\Delta_{h+1}^{\pi^\star}\big)(s,a)\Big)\cdot\Big((\Delta_{h}^{\pi^\star})(s)-\big(P^\star\Delta_{h+1}^{\pi^\star}\big)(s,a)\Big)\right]\notag\\
     &\leq 2 \sum_{h=0}^{H-1} \EE_{s,a\sim d^{\pi^\star}_h}\left[\left|(\Delta_{h}^{\pi^\star})(s)-\big(P^\star\Delta_{h+1}^{\pi^\star}\big)(s,a)\right|\right]\notag\\
    &=2 \sum_{h=0}^{H-1} \EE_{s,a\sim d^{ \pi^\star}_h}\left[\left|({V}_{h;\widehat P}^{\pi^\star})(s)-\big(P^\star{V}_{h+1;\widehat P}^{\pi^\star}\big)(s,a)-\Big(({V}_{h}^{\pi^\star})(s)-\big(P^\star{V}_{h+1}^{\pi^\star}\big)(s,a)\Big)\right|\right]\notag\\
    &=2 \sum_{h=0}^{H-1} \EE_{s,a\sim d^{ \pi^\star}_h}\left[\left|r(s,a)+\big(\widehat P{V}_{h+1;\widehat P}^{\pi^\star}\big)(s,a)-\big(P^\star{V}_{h+1;\widehat P}^{\pi^\star}\big)(s,a)-r(s,a)\right|\right]\notag\\
        &=2 \sum_{h=0}^{H-1} \EE_{s,a\sim d^{ \pi^\star}_h}\left[\left|\big(\widehat P{V}_{h+1;\widehat P}^{\pi^\star}\big)(s,a)-\big(P^\star{V}_{h+1;\widehat P}^{\pi^\star}\big)(s,a)\right|\right],\label{eqn: help 3}
\end{align}
where the inequality holds as the value functions are all bounded by 1 by the assumption that the total reward over any trajectory is bounded by 1, the first equality is by the definition of $\Delta_{h+1}^{\pi^\star}$, the second equality is because $a$ is drawn from $\pi^\star$. Starting from \pref{eqn: help 3}, we have
\begin{align}
&2 \sum_{h=0}^{H-1} \EE_{s,a\sim d^{ \pi^\star}_h}\left[\left|\big(\widehat P{V}_{h+1;\widehat P}^{\pi^\star}\big)(s,a)-\big(P^\star{V}_{h+1;\widehat P}^{\pi^\star}\big)(s,a)\right|\right]\notag\\
    &= 2 \sum_{h=0}^{H-1} \EE_{s,a\sim d^{ \pi^\star}_h}\left[\left|\EE_{s'\sim P^\star(s,a)}\left[V^{\pi^\star}_{h+1;\widehat P}(s')\right]-\EE_{s'\sim\widehat P(\cdot|s,a)}\left[V^{\pi^\star}_{h+1;\widehat P}(s')\right]\right|\right]\notag\\
        &\leq 16 \sqrt{\sum_{h=0}^{H-1} \EE_{s,a \sim d^{\pi^\star}_h}\left[\big(\VV_{P^\star} V_{h+1; \widehat P}^{\pi^\star}\big)(s,a)\right]\cdot\frac{22C^{\pi^\star}\log(\left|\Pcal\right|/\delta)}{K}}+\frac{880C^{\pi^\star}\log(\left|\Pcal\right|/\delta)}{K}\notag\\
        &=16 \sqrt{A\cdot\frac{22C^{\pi^\star}\log(\left|\Pcal\right|/\delta)}{K}}+\frac{880C^{\pi^\star}\log(\left|\Pcal\right|/\delta)}{K}\label{eqn: help 4}
\end{align}
where the inequality holds with probability at least $1-\delta$ by \pref{lem: offline sum mean value difference}, and the second equality is by definition of $A$.

Then combining \pref{eqn: help 2}, \pref{eqn: help 3} and \pref{eqn: help 4}, we obtain an upper bound for $C$, which suggests
\begin{align}
    A&\leq 2B+2C\notag\\
    &\leq 2B+\frac{1760C^{\pi^\star}\log(\left|\Pcal\right|/\delta)}{K}+32\sqrt{\frac{22C^{\pi^\star}\log(\left|\Pcal\right|/\delta)}{K}}\cdot\sqrt{A}.\notag
\end{align}
Then, with the fact that $x\leq 2a+b^2$ if $x\leq a+b\sqrt{x}$, we have
\begin{align}
    A\leq 4B+\frac{3520C^{\pi^\star}\log(\left|\Pcal\right|/\delta)}{K}+\frac{22528C^{\pi^\star}\log(\left|\Pcal\right|/\delta)}{K}\leq O(B+\frac{C^{\pi^\star}\log(\left|\Pcal\right|/\delta)}{K}).\notag
\end{align}

With the above lemmas, we can now prove the final results of \pref{thm:mleoffline}. We have that w.p. at least $1-\delta$
\begin{small}
    \begin{align}
    V_{0;P^\star}^{\pi^\star}-V_{0;P^\star}^{\hat{\pi}}&\leq O\Big(\sqrt{A\cdot\frac{C^{\pi^\star}\log(\left|\Pcal\right|/\delta)}{K}}+\frac{C^{\pi^\star}\log(\left|\Pcal\right|/\delta)}{K}\Big)\notag\\
    &\leq O\Big(\sqrt{(B+\frac{C^{\pi^\star}\log(\left|\Pcal\right|/\delta)}{K})\cdot\frac{C^{\pi^\star}\log(\left|\Pcal\right|/\delta)}{K}}+\frac{C^{\pi^\star}\log(\left|\Pcal\right|/\delta)}{K}\Big)\notag\\
    &\leq  O\Big(\sqrt{B\cdot\frac{C^{\pi^\star}\log(\left|\Pcal\right|/\delta)}{K}}+\sqrt{\frac{C^{\pi^\star}\log(\left|\Pcal\right|/\delta)}{K}\cdot\frac{C^{\pi^\star}\log(\left|\Pcal\right|/\delta)}{K}}+\frac{C^{\pi^\star}\log(\left|\Pcal\right|/\delta)}{K}\Big)\notag\\
    &= O\Big(\sqrt{\sum_{h=0}^{H-1} \EE_{s,a\sim d^{ \pi^\star}_h}\left[\big(\VV_{P^\star} V_{h+1}^{\pi^\star}\big)(s,a)\right]\cdot\frac{C^{\pi^\star}\log(\left|\Pcal\right|/\delta)}{K}}+ \frac{C^{\pi^\star}\log(\left|\Pcal\right|/\delta)}{K}\Big)\label{eqn: intermediate step proof offline}\\
    &=O\Big(\sqrt{\frac{\var_{\pi^\star}C^{\pi^\star}\log(\left|\Pcal\right|/\delta)}{K}}+ \frac{C^{\pi^\star}\log(\left|\Pcal\right|/\delta)}{K}\Big)\notag\,,
\end{align}
\end{small}

where in the last equation we use \pref{lem:variance_lemma}, and $\var_{\pi^\star}:=\EE\left[\bigg(\sum_{h=0}^{H-1}r(s_h,\pi^\star(s_h))-V_0^{\pi^\star}\bigg)^2\right]$.

\subsection{Proof of \pref{corr:coro_faster}}\label{app:offline_coro_faster}
\noindent\textbf{Proof}[Proof of \pref{corr:coro_faster}]
    By \pref{lem:variance_lemma}, we have
\begin{equation}
    \var_{\pi^*}=\sum_{h=0}^{H-1}\EE_{s,a \sim d^{\pi^*}_h}\left[\big(\VV_{P^\star} V_{h+1}^{\pi^*}\big)(s,a)\right]
\end{equation}
Therefore, when $P^\star$ is deterministic, 
the $\EE_{s,a \sim d^{\pi^*}_h}\left[\big(\VV_{P^\star} V_{h+1}^{\pi^*}\big)(s,a)\right]$ terms are all 0 for any $k\in[K-1]$ and $h\in[H-1]$, and then the $\var_{\pi^*}$ term in the higher order term in \pref{thm:mleoffline} is 0.

\subsection{Proof of  \pref{corr:offline_coro_infinite}}\label{app:offline_coro_infinite}
\noindent\textbf{Proof}[Proof of \pref{corr:offline_coro_infinite}]
This claim follows the proof of \pref{thm:mleoffline}, while we take a different choice of $\beta$ that depends on the bracketing number and follow the MLE guarantee in \pref{lem:mle_generalization infinite} for infinite model class.

\subsection{Proof of the claim in \pref{ex: offline coverage}}\label{app: example proof}
\noindent\textbf{Proof}
    Recall that in \pref{def: coverage offline}, we have
    \begin{align*}
C^{\pi^*}_{\Dcal} := \max_{h, P \in \Pcal}  \frac{ \EE_{s,a\sim d^{\pi^*}_h} \mathbb{H}^2\left( P(s,a) \Mid P^\star(s,a) \right)   }{ 1/K \sum_{k=1}^K \mathbb H^2\left(   P(s_h^k,a_h^k)  \Mid   P^\star(s_h^k,a_h^k)  \right)    }\,.
\end{align*} 
For each step $h$, define two distributions, $p_h,q_h$, where $p_h(s,a)=d^{\pi^*}(s,a)$, $q_h(s,a)=\frac{1}{K}\sum_{k=1}^K \mathbb I \{(s,a)=(s_h^k,a_h^k)\}$, and we define $f(s,a,P)=\mathbb H^2(P(s,a)\Mid P^\star(s,a))$, then we have
\begin{align}
    C^{\pi^*}_{\Dcal} &= \max_{h, P \in \Pcal}  \frac{ \EE_{s,a\sim p_h} f(s,a,P)   }{ \EE_{s,a\sim q_h} f(s,a,P)     }\notag\\
    &= \max_{h, P \in \Pcal}  \frac{ \EE_{s,a\sim q_h} \frac{p_h(s,a)}{q_h(s,a)}f(s,a,P)  }{ \EE_{s,a\sim q_h} f(s,a,P)   }\notag\\
    &\leq \max_{h,s,a} \frac{p_h(s,a)}{q_h(s,a)}\notag\\
    &\leq \max_{h,s,a} \frac{1}{q_h(s,a)}\,.
\end{align}
Note that for all $h$, $\{(s_h^k,a_h^k)\}_{k=1}^K$ are i.i.d. samples drawn from $d_h^{\pi^b}$, therefore, $\mathbb E [\mathbb I\{(s_h^k,a_h^k)=(s,a)\}]=d^{\pi^b}_h(s,a)$. By Hoeffding's inequality and with a union bound over $s,a,h$, and for $K\geq \frac{2\log(\frac{|\Scal||\Acal|H}{\delta})}{\rho_{\min}^2}$, w.p. at least $1-\delta$, we have
\begin{align}
    q_h(s,a)&=\frac{1}{K}\sum_{k=1}^K \mathbb I \{(s_h^k,a_h^k)=(s,a)\}\notag\\
    &\geq d^{\pi^b}_h(s,a)-\sqrt{\frac{\log(\frac{|\Scal||\Acal|H}{\delta})}{2K}}\notag\\
    &\geq \frac{d^{\pi^b}_h(s,a)}{2}\,,
\end{align}
where in the last inequality we use the assumption that $d^{\pi^b}_h(s,a) \geq \rho_{\min}, \forall s,a, h$, which gives us $K \geq \frac{2\log(\frac{|\Scal||\Acal|H}{\delta})}{\rho_{\min}^2}\geq\max_{s,a,h}\frac{2\log(\frac{|\Scal||\Acal|H}{\delta})}{(d^{\pi^b}_h(s,a))^2}$, so $K\geq \frac{2\log(\frac{|\Scal||\Acal|H}{\delta})}{(d^{\pi^b}_h(s,a))^2}$ for any $s,a,h$.

Therefore, with $K\geq \frac{2\log(\frac{|\Scal||\Acal|H}{\delta})}{\rho_{\min}^2}$, we have that w.p. at least $1-\delta$
\begin{align}
    C^{\pi^*}_{\Dcal}&\leq \max_{h,s,a} \frac{1}{q_h(s,a)}\leq \max_{h,s,a} \frac{2}{d^{\pi^b}_h(s,a)}\leq \frac{2}{\rho_{min}}\,.
\end{align}

\section{Appendix for Chapter \ref{chapter ZSG}}
We provide missing proofs and theoretical results of our paper in the Appendix sections:
\begin{itemize}[leftmargin = *]
   \item In Appendix \ref{app:notext}, we provide the missing results of Section \ref{sec:nocontext}. We first provide the proof of Proposition \ref{thm: nodistinguish}, then we analyze the suboptimality gap of the Pessimistic Value Iteration (PEVI) (\cite{jin2021pessimism}) in the contextual linear MDP setting without context information.
   \item In Appendix \ref{app:mainthm}, we provide the proofs of our main theorems on the suboptimality bounds of PERM and PPPO in Section \ref{sec:withcontext}.
   \item In Appendix \ref{appendix:merge}, we state and prove the suboptimality bounds we promised in Remarks \ref{rmk:merge} and \ref{rmk:mergefree}, where we merge the sampled contexts into $m$ groups ($m<n$) to reduce the computational complexity in practical settings.
   \item In Appendix \ref{proof:linear}, we provide the proofs of results in Section \ref{sec:linear} on linear MDPs. Namely, we provide proof of Theorem \ref{thm:regret_upper_linear}, proof of Corollary \ref{cor:well_explore}.
\end{itemize}

\subsection{Results in Section \ref{sec:nocontext}}\label{app:notext}

\subsubsection{Proof of Proposition \ref{thm: nodistinguish}}\label{app:nodistin}

%\begin{proof}[Proof of Proposition \ref{thm: nodistinguish}]
    % We have multiple MDPs $M_i = (\cS, \cA, (\cP_{i,h})_h, (r_{i,h})_h)$. We consider the probability, where
    % \begin{align}
    %     &\PP_{\cD_i}(r_{i,h}^\tau = r', x_{i,h+1}^\tau = x'|\{(x_{i,h}^j, a_{i,h}^j)\}_{j=1}^\tau, \{r_{i,h}^j, x_{i,h+1}^j\}_{j=1}^{\tau-1}) \notag \\
    %     &\quad = \PP_{i}(r_{i,h}(s_h)=r',s_{h+1} = x'|s_h = x_{i,h}^\tau, a_h = a_{i,h}^\tau).\notag
    % \end{align}
    Let $\cD' = \{(x_{c_\tau,h}^\tau, a_{c_\tau,h}^\tau, r_{c_\tau,h}^\tau)\}_{h=1, \tau = 1}^{H,K}$ denote the merged dataset, where each trajectory belongs to a context $c_\tau$. For simplicity, let $\cD_c$ denote the collection of trajectories that belong to MDP $\cM_c$.   
    Then each trajectory in $\cD'$ is generated by the following steps:
    \begin{itemize}
        \item The experimenter randomly samples an environment $c \sim C$.
        \item The experimenter collect a trajectory from the episodic MDP $\cM_c$.
    \end{itemize}
  Then for any $x', r', \tau$ we have 
    \begin{align}
        &\PP_{\cD'}(r_{c_\tau,h}^\tau = r', x_{c_\tau,h+1}^\tau = x'|\{(x_{c_j,h}^j, a_{c_j,h}^j)\}_{j=1}^{\tau}, \{r_{c_j,h}^j, x_{c_j,h+1}^j\}_{j=1}^{\tau-1})\notag \\
        & = \frac{\PP_{\cD'}(r_{c_\tau,h}^\tau = r', x_{c_\tau,h+1}^\tau = x',\{(x_{c_j,h}^j, a_{c_j,h}^j)\}_{j=1}^{\tau}, \{r_{c_j,h}^j, x_{c_j,h+1}^j\}_{j=1}^{\tau-1})}{\PP_{\cD'}(\{(x_{c_j,h}^j, a_{c_j,h}^j)\}_{j=1}^{\tau}, \{r_{c_j,h}^j, x_{c_j,h+1}^j\}_{j=1}^{\tau-1})}\notag \\
        & = \sum_{c\in C}\PP_{\cD'}(r_{c_\tau,h}^\tau = r', x_{c_\tau,h+1}^\tau = x'|\{(x_{c_j,h}^j, a_{c_j,h}^j)\}_{j=1}^{\tau}, \{r_{c_j,h}^j, x_{c_j,h+1}^j\}_{j=1}^{\tau-1}, c_\tau = c)q(c),\label{www:1}
    \end{align}
where 
\begin{align}
    q(c'):=\frac{\PP_{\cD'}(\{(x_{c_j,h}^j, a_{c_j,h}^j)\}_{j=1}^{\tau}, \{r_{c_j,h}^j, x_{c_j,h+1}^j\}_{j=1}^{\tau-1}, c_\tau = c')}{\sum_{c\in C}\PP_{\cD'}(\{(x_{c_j,h}^j, a_{c_j,h}^j)\}_{j=1}^{\tau}, \{r_{c_j,h}^j, x_{c_j,h+1}^j\}_{j=1}^{\tau-1}, c_\tau = c)}.\notag
\end{align}
Next, we further have
\begin{small}
    \begin{align}
&\eqref{www:1}\notag \\
&=\sum_{c\in C}\PP_{c}(r_{c,h}(s_h) = r', s_{h+1} = x'|s_h = x_{c_\tau, h}^\tau, a_h = a_{c_\tau, h}^\tau)q(c)\notag \\
        & =\sum_{c\in C}\frac{\PP_{c}(r_{c,h}(s_h) = r', s_{h+1} = x'|s_h = x_{c_\tau, h}^\tau, a_h = a_{c_\tau, h}^\tau)\PP_{\cD'}(s_h = x_{c_\tau, h}^\tau, a_h = a_{c_\tau, h}^\tau, c_\tau = c)}{\sum_{c\in C}\PP_{\cD'}(s_h = x_{c_\tau, h}^\tau, a_h = a_{c_\tau, h}^\tau, c_\tau = c)}\notag \\
        &=\sum_{c\in C}p(c)\cdot \frac{\PP_{c}(r_{c,h}(s_h) = r', s_{h+1} = x'|s_h = x_{c_\tau, h}^\tau, a_h = a_{c_\tau, h}^\tau)\PP_{c}(s_h = x_{c_\tau, h}^\tau, a_h = a_{c_\tau, h}^\tau)}{\sum_{c\in C}p(c)\cdot \PP_{c}(s_h = x_{c_\tau, h}^\tau, a_h = a_{c_\tau, h}^\tau)}\notag\\
        &=\EE_{c \sim C} \frac{\PP_{c}(r_{c,h}(s_h) = r', s_{h+1} = x'|s_h = x_{c_\tau, h}^\tau, a_h = a_{c_\tau, h}^\tau) \mu_{c,h}(x_{c_\tau, h}^\tau, a_{c_\tau, h}^\tau)}{\EE_{c \sim C} \mu_{c,h}(x_{c_\tau, h}^\tau, a_{c_\tau, h}^\tau)},\notag
\end{align}
\end{small}

where the first equality holds since for all trajectories $\tau$ satisfying $c_\tau = c$, they are compliant with $\cM_c$, the second one holds since all trajectories are independent of each other, the third and fourth ones hold due to the definition of $\mu_{c,h}(\cdot, \cdot)$. 
%\end{proof}

\subsubsection{PEVI algorithm}\label{app:pevi}

\begin{algorithm}[H]
\caption{\cite{jin2021pessimism} Pessimistic Value Iteration (PEVI)}
\begin{algorithmic}[1]
\label{alg:no context}
\REQUIRE Dataset $\cD=\{(x_{c_\tau, h}^\tau,a_{c_\tau, h}^\tau,r_{c_\tau, h}^\tau)_{h=1}^H\}_{\tau=1}^{K}$, confidence probability $\delta\in(0,1)$.
\STATE Initialization: Set $\hat{V}_{H+1}(\cdot) \leftarrow 0$.
\FOR{step $h=H,H-1,\ldots,1$}
\STATE Set $\Lambda_h \leftarrow \sum_{\tau=1}^K \phi(x_h^\tau,a_h^\tau)  \phi(x_h^\tau,a_h^\tau) ^\top + \lambda\cdot I$. %\hfill  {//Estimation}
\STATE Set $\hat{w}_h\leftarrow  \Lambda_h ^{-1}( \sum_{\tau=1}^{K} \phi(x_h^\tau,a_h^\tau) \cdot (r_h^\tau + \hat{V}_{h+1}(x_{h+1}^\tau)) ) $. 
\STATE Set $\Gamma_h(\cdot,\cdot) \leftarrow \beta(\delta)\cdot ( \phi(\cdot,\cdot)^\top  \Lambda_h ^{-1} \phi(\cdot,\cdot) )^{1/2}$. 
\STATE Set $\hat{Q}_h(\cdot,\cdot) \leftarrow \min\{\phi(\cdot,\cdot)^\top \hat{w}_h - \Gamma_h(\cdot,\cdot),H-h+1\}^+$. 
\STATE Set $\hat{\pi}_h (\cdot \given \cdot) \leftarrow \argmax_{\pi_h}\langle \hat{Q}_h(\cdot, \cdot),\pi_h(\cdot\given \cdot)\rangle_{\cA}$.
\STATE Set $\hat{V}_h(\cdot) \leftarrow \langle \hat{Q}_h(\cdot,\cdot),\hat\pi_h(\cdot \given \cdot)\rangle_{\cA}$.
\ENDFOR 
\RETURN $\pi^{\text{PEVI}}= \{\hat{\pi}_h\}_{h=1}^H$.
\end{algorithmic}
\end{algorithm}

We analyze the suboptimality gap of the Pessimistic Value Iteration (PEVI) (\cite{jin2021pessimism}) in the contextual linear MDP setting without context information to demonstrate that by finding the optimal policy for $\bar\cM$ is not enough to find the policy that performs well on MDPs with context information. 

\noindent\textbf{Pessimistic Value Iteration (PEVI)}.
Let $\overline{\pi}^*$ be the optimal policy w.r.t. the average MDP $\bar \cM$. We analyze the performance of the Pessimistic Value Iteration (PEVI) \cite{jin2021pessimism} under the unknown context information setting. The details of PEVI is in Algo.\ref{alg:no context}.

Suppose that $\bar\cD$ consists of $K$ number of trajectories generated i.i.d. following by a fixed behavior policy $\bar\pi$. Then the following theorem shows the suboptimality gap for Algo.\ref{alg:no context} does not converge to 0 even when the data size grows to infinity.

\begin{theorem}\label{thm:pevi}
Assume that $\bar\pi$
In Algo.\ref{alg:linear mdp model based}, we set 
\begin{equation}
    \lambda=1,\quad \beta(\delta) = c'\cdot dH\sqrt{\log(4dHK/\delta)}\,,
\end{equation}
where $c' >0$ is a positive constant. 
Suppose we have 
$K \geq \tilde c\cdot d \log (4 dH/ \xi)$, where $\tilde c > 0$ is a sufficiently large positive constant that depends on $c$. Then we have: w.p. at least $1-\delta$, for the output policy $\pi^{\text{PEVI}}$ of Algo.\ref{alg:no context},
\begin{align}
    \sup_{\pi}V_{\bar\cM,1}^\pi - V_{\bar\cM,1}^{\pi^{\text{PEVI}}} \leq c'' \cdot d^{3/2} H^2 K^{-1/2} \sqrt{\log(4dHK/\delta)}, 
\end{align}
and the suboptimality gap satisfies
\begin{align}\normalfont
\text{SubOpt}(\pi^{\text{PEVI}})  &\leq c'' \cdot d^{3/2} H^2 K^{-1/2} \sqrt{\log(4dHK/\delta)}\notag\\
&+ 2\sup_\pi |V_{\bar{\cM}, 1}^{\pi}(x_1)-\EE_{c\sim C} V_{c, 1}^{\pi}(x_1)|\,,
\label{eq:model based gap no context}
\end{align}
where $c''>0$ is a positive constant that only depends on $c$ and $c'$.
\label{cor:well_explore2}
\end{theorem}

\begin{proof}[Proof of Theorem \ref{cor:well_explore2}]
    First, we define the value function on the average MDP $\Bar{\cM}$ as follows.

\begin{equation}
    \overline{V}^\pi_{h}(x)=\EE_{\pi,\Bar{\cM}}\Big[ \sum_{i=h}^H r_i(s_i, a_i)\biggiven s_h=x  \Big]\,.
\end{equation}

We then decompose the suboptimality gap as follows.
\begin{align}\normalfont
&\text{SubOpt}(\pi^{\text{PEVI}}) \notag \\
&= \EE_{c\sim C}\big[V_{c,1}^{\pi^*}(x_1)\big] - \EE_{c\sim C}\big[V_{c,1}^{\pi^{\text{PEVI}}}(x_1)\big]\notag\\
&=\overline{V}^{\overline{\pi}^*}_{1}(x_1)-\overline{V}^{\pi^{\text{PEVI}}}_1(x_1)+\big(\EE_{c\sim C}\big[V_{c,1}^{\pi^*}(x_1)\big]-\overline{V}^{\overline{\pi}^*}_{1}(x_1)\big)\notag\\
&+\big(\overline{V}^{\pi^{\text{PEVI}}}_1(x_1)-\EE_{c\sim C}\big[V_{c,1}^{\pi^{\text{PEVI}}}(x_1)\big]\big)\notag\\
&\leq\overline{V}^{\overline{\pi}^*}_{1}(x_1)-\overline{V}^{\pi^{\text{PEVI}}}_1(x_1)+2\sup_\pi |V_{\bar{\cM}, 1}^{\pi}(x_1)-\EE_{c\sim C} V_{c, 1}^{\pi}(x_1)|\,.
\label{eq:model based gap no context}
\end{align}

Then, applying Corollary 4.6 in \cite{jin2021pessimism}, we can get that w.p. at least $1-\delta$
\begin{equation}
    \overline{V}^{\overline{\pi}^*}_{1}(x_1)-\overline{V}^{\pi^{\text{PEVI}}}_1(x_1)\leq  c'' \cdot d^{3/2} H^2 K^{-1/2} \sqrt{\log(4dHK/\delta)}\,,
\end{equation}
which, together with Eq.(\ref{eq:model based gap no context}) completes the proof.

\end{proof}

Theorem \ref{thm:pevi} shows that by adapting the standard pessimistic offline RL algorithm over the offline dataset without context information, the learned policy $\pi^{\text{PEVI}}$ converges to the optimal policy $\bar\pi^*$ over the average MDP $\bar\cM$.

\subsection{Proof of Theorems in Section \ref{sec:withcontext}}\label{app:mainthm}

\subsubsection{Proof of Theorem \ref{thm:model based regret_upper_bound_general}}\label{appendix model based}

We define the model estimation error as 
\begin{equation}
\iota_{i,h}^\pi(x,a) = (\BB_{i,h} \hat{V}^\pi_{i,h+1})(x,a) - \hat{Q}^\pi_{i,h}(x,a).%\quad \forall (x,a)\in \cS\times \cA,
\label{eq:def_iota with context}
\end{equation}
And we define the following condition 
\begin{align}
    &\big|(\hat\BB_{i,h} \hat{V}^\pi_{i,h+1})(x,a) - (\BB_{i,h} \hat{V}^\pi_{i,h+1})(x,a)\big|\leq \Gamma_{i,h}(x,a)\notag\\
    & \text{for all}~i\in[n], \pi\in\Pi, (x,a)\in \cS\times \cA, h\in[H]\,. \label{high prob 2}
\end{align}
We introduce the following lemma to bound the model estimation error.
\begin{lemma}[Model estimation error bound (Adapted from Lemma 5.1 in \cite{jin2021pessimism}]
Under the condition of Eq.(\ref{high prob 2}), we have
\label{lem:model_eval_err}
\begin{equation}
    0\leq \iota_{i,h}^\pi(x,a) \leq 2\Gamma_{i,h}(x,a),\quad \text{for all}~~i\in[n], ~\pi\in\Pi, ~(x,a)\in \cS\times \cA,~ h\in [H]. 
\end{equation}\label{eq:model_eval_err_bound}
\end{lemma}

Then, we prove the following lemma for pessimism in V values.
\begin{lemma}[Pessimism for Estimated V Values]
   Under the condition of Eq.(\ref{high prob 2}), for any $i\in[n], \pi\in\Pi, x \in \cS$, we have
    \begin{equation}        V_{i,h}^\pi(x)\geq\hat{V}_{i,h}^\pi(x)\,.
    \end{equation}
    \label{pessimism lemma1}
\end{lemma}

\begin{proof}
    For any $i\in[n], \pi\in\Pi, x\in\cS, a\in \cA$, we have 
    \begin{align}
        &Q_{i,h}^\pi(x,a)-\hat{Q}_{i,h}^\pi(x,a)\notag \\
        &\geq r_{i,h}(x,a)+(\BB_{i,h}V^\pi_{i,h+1})(x,a)-\big(r_{i,h}(s,a)+(\hat{\BB}_{i,h}\hat{V}^\pi_{i,h+1})(x,a)-\Gamma_{i,h}(x,a)\big)\notag\\
        &=(\BB_{i,h}V^\pi_{i,h+1})(x,a)-({\BB}_{i,h}\hat{V}^\pi_{i,h+1})(x,a)+\Gamma_{i,h}(x,a)\notag\\
        &\quad-\big((\hat{\BB}_{i,h}\hat{V}^\pi_{i,h+1})(x,a)-{\BB}_{i,h}\hat{V}^\pi_{i,h+1})(x,a)\big)\notag\\
        &\geq (\BB_{i,h}V^\pi_{i,h+1})(x,a)-({\BB}_{i,h}\hat{V}^\pi_{i,h+1})(x,a)\notag\\
        &=\big(P_{i,h}(V_{i,h+1}^\pi-\hat{V}_{i,h+1}^\pi)\big)(x,a)\notag\,,
    \end{align}
    where the second inequality is because of Eq.(\ref{high prob 2}). And since in the $H+1$ step we have $V_{i,H+1}^\pi=\hat{V}_{i,h+1}^\pi=0$, we can get $Q_{i,H}^\pi(x,a)-\hat{Q}_{i,H}^\pi(x,a)$. Then we use induction to prove $Q_{i,h}^\pi(x,a)\geq\hat{Q}_{i,h}^\pi(x,a)$ for all $h$. Given $Q_{i,h+1}^\pi(x,a)\geq\hat{Q}_{i,h+1}^\pi(x,a)$, we have
    \begin{align}
        &Q_{i,h}^\pi(x,a)-\hat{Q}_{i,h}^\pi(x,a)\notag\\
        &\geq \big(P_{i,h}(V_{i,h+1}^\pi-\hat{V}_{i,h+1}^\pi)\big)(x,a)\notag\\
        &=\E{}{\langle Q_{i,h+1}^\pi(s_{h+1},\cdot)-\hat{Q}_{i,h+1}^\pi(s_{h+1},\cdot),\pi_{h+1}(\cdot|s_{h+1})\rangle_\cA|s_h=x, a_h=a}\notag\\
        &\geq 0\,.
    \end{align}
    Then we have
    \begin{align}
        V_{i,h}^\pi(x)-\hat{V}_{i,h}^\pi(x)&=\langle Q_{i,h}^\pi(x, \cdot)-\hat{Q}_{i,h}^\pi(x, \cdot),\pi_h(\cdot\given x)\rangle_{\cA}\geq 0\notag\,.
    \end{align}
\end{proof}

Then we start our proof. 
\begin{proof}[Proof of Theorem \ref{thm:model based regret_upper_bound_general}]

First, we decompose the suboptimality gap as follows
\begin{align}
    &\text{SubOpt}(\pi^{\text{PERM}})\notag \\
    &=\EE_{c \sim C}{V_{c,1}^{\pi^*}(x_1)-V_{c,1}^{\hatpistar}(x_1)}\notag \\
    &=\EE_{c \sim C}{V_{c,1}^{\pi^*}(x_1)}-\frac{1}{n}\sum_{i=1}^n V^{\pi^*}_{i,1}(x_1) +\frac{1}{n}\sum_{i=1}^n V_{i,1}^{\pi^{\text{PERM}}}(x_1)-\EE_{c \sim C}{V_{c,1}^{\pi^{\text{PERM}}}(x_1)}\notag\\
    &\quad +\frac{1}{n}\sum_{i=1}^n \big(V_{i,1}^{\pi^*}(x_1)-V_{i,1}^{\pi^{\text{PERM}}}(x_1)\big)\label{decomposition}\,.
\end{align}

For the first two terms, we can bound them following the standard generalization techniques (\cite{ye2023power}), \emph{i.e.}, we use the covering argument, Chernoff bound,and union bound.

Define the distance between policies $d(\pi^1,\pi^2) \triangleq \max_{s\in \mathcal{S}, h \in [H]} \|\pi^1_h(\cdot|s) - \pi^2_h (\cdot|s)\|_{1}$. 
We construct the $\epsilon$-covering set $\tilde{\Pi}$ w.r.t. $ d$ such that
\begin{align}
    \label{inq: definition of tildePi}
    \forall \pi \in \Pi, \exists \tilde{\pi} \in \tilde{\Pi}, s.t. \quad d(\pi,\tilde{\pi}) \leq \epsilon.
\end{align}
Then we have
\begin{align}
    \label{inq: pi tildepi value gap}
    \forall i\in[n], \pi \in \Pi, \exists \tilde{\pi} \in \tilde{\Pi}, s.t. V_{i,1}^{\pi} (x_1) - V_{i,1}^{\tilde{\pi}} (x_1) \leq H \epsilon.
\end{align}
By the definition of the covering number, $\left|\tilde{\Pi}\right| =\mathcal{N}_\epsilon^\Pi$. By Chernoff bound and union bound over the policy set $\tilde{\Pi}$, we have with prob. at least $1-\frac{\delta}{3}$, for any $\tilde{\pi} \in \tilde{\Pi}$,
\begin{align}
    \label{inq: concentration in tildePi}
    \left|\frac{1}{n} \sum_{i=1}^{n} V^{\tilde{\pi}}_{i,1} (x_1) - \EE_{c \sim C} {V^{\tilde{\pi}}_{c,1} (x_1)} \right| \leq \sqrt{\frac{2\log(6\mathcal{N}_\epsilon^\Pi/\delta)}{n}}.
\end{align}

By Eq.(\ref{inq: pi tildepi value gap}) and Eq.(\ref{inq: concentration in tildePi}), $\forall i\in[n], \pi \in \Pi, \exists \tilde{\pi} \in \tilde{\Pi}$ with $\left|\tilde{\Pi}\right| =\mathcal{N}_\epsilon^\Pi,~ s.t. V_{i,1}^{\pi} (x_1) - V_{i,1}^{\tilde{\pi}} (x_1) \leq H \epsilon$, and with probability at least $1-\delta/3$, we have
\begin{align}
    &\left|\frac{1}{n} \sum_{i=1}^{n} V^{{\pi}}_{i,1} (x_1) - \EE_{c \sim C} {V^{{\pi}}_{c,1} (x_1) }\right| \notag \\
    &\leq  \left|\frac{1}{n} \sum_{i=1}^{n} V^{\tilde{\pi}}_{i,1} (s_1) - \EE_{c \sim C} {V^{{\tilde{\pi}}}_{c,1} (x_1) } \right| \notag\\
    & + \left|\frac{1}{n} \sum_{i=1}^{n} V^{{\pi}}_{i,1} (s_1) - \frac{1}{n}\sum_{i=1}^{n} V^{\tilde{\pi}}_{i,1} (s_1) \right|
    + \left|\EE_{c \sim C} {V^{{\tilde{\pi}}}_{c,1} (x_1) } - \EE_{c \sim C} {V^{{\pi}}_{c,1} (x_1) } \right| \notag\\
    &\leq  \sqrt{\frac{2\log(6\mathcal{N}_\epsilon^\Pi/\delta)}{n}}+ 2H\epsilon\,.\label{eq: generalization gap}
\end{align}
Therefore, we have for the first two terms, w.p. at least $1-\frac{2}{3}\delta$ we can upper bound them with $4H\epsilon+2\sqrt{\frac{2\log(6\mathcal{N}_\epsilon^\Pi/\delta)}{n}}$.
% \begin{lemma}[Model estimation error bound (Adapted from Lemma 5.1 in \cite{jin2021pessimism}]
% Suppose that $\{\Gamma_h\}_{h=1}^H$ in Algo.\ref{alg:model based general} are $\xi$-uncertainty quantifiers. Under $\cE$ defined in Equation \eqref{eq:def_event_eval_err_general}, which satisfies $\PP_{\cD}(\cE)\geq 1-\xi$, we have
% \label{lem:model_eval_err}
% \begin{equation}
%     0\leq \iota_{i,h}^\pi(x,a) \leq 2\Gamma_h(x,a),\quad \text{for all}~~(x,a)\in \cS\times \cA,~ h\in [H]. 
% \end{equation}\label{eq:model_eval_err_bound}

% %Suppose $\{ \Gamma_h\}_{h = 1 } ^ H $ in Algorithm \ref{alg:pess_greedy_general} are a $\xi$-uncertainty quantifier specified in Definition \ref{def:uncertainty_quantifier}. Then, on event $\cE$  defined in Equation \eqref{eq:def_event_eval_err_general}, the model evaluation errors $\{ \iota_h \}_{h = 1}^H $ satisfies that
% %\#\label{eq:model_eval_err_bound}
% %0\leq \iota_h(x,a) \leq 2\Gamma_h(x,a),\qquad \forall (x,a)\in \cS\times \cA, \forall h \in [H]. 
% %\#
% %Recall that $\cE$  holds with probability at least  $1-\xi$ under $\PP_{\cD}$.
% \end{lemma}
Then, what remains is to bound the term $\frac{1}{n}\sum_{i=1}^n \big(V_{i,1}^{\pi^*}(x_1)-V_{i,1}^{\pi^{\text{PERM}}}(x_1)\big)$.

First, by similar arguments, we have 
\begin{align}
    V_{i,1}^{\pi^*}(x_1)-V_{i,1}^{\pi^{\text{PERM}}}(x_1)&\leq \big(V_{i,1}^{{\pi}^*}(x_1)-V_{i,1}^{\tilde{\pi}^{\text{PERM}}}(x_1)\big)+|V_{i,1}^{\tilde{\pi}^{\text{PERM}}}(x_1)-V_{i,1}^{\pi^{\text{PERM}}}(x_1)|\notag\\
    &\leq H\epsilon+V_{i,1}^{{\pi}^*}(x_1)-V_{i,1}^{\tilde{\pi}^{\text{PERM}}}(x_1)\label{cover 1}\,,
\end{align}
where $ \tilde{\pi}^{\text{PERM}}\in\tilde{\Pi}$ such that $|V_{i,1}^{\tilde{\pi}^{\text{PERM}}}(x_1)-V_{i,1}^{\pi^{\text{PERM}}}(x_1)|\leq H\epsilon$.

By the definition of the oracle in Definition.\ref{def:oracle}, the algorithm design of Algo.\ref{alg:model based general} (e.g., we call oracle $\mathbb{O}(\cD_h, \hat V_{h+1}, \delta/(3nH\mathcal{N}_{(Hn)^{-1}}^\Pi))$), and use a union bound over $H$ steps, $n$ contexts, and $\mathcal{N}_{(Hn)^{-1}}^\Pi$ policies, we have: with probability at least $1-\delta/3$, the condition in Eq.(\ref{high prob 2}) holds (with the policy class $\Pi$ replaced by $\tilde{\Pi}$ (and $\epsilon=1/(Hn))$.

Then, we have
\begin{align}
    &\frac{1}{n}\sum_{i=1}^n\big(V_{i,1}^{\pi^*}(x_1)-V_{i,1}^{\tilde\pi^{\text{PERM}}}(x_1)\big)\notag \\
    &\leq \frac{1}{n}\sum_{i=1}^n\big(V_{i,1}^{\pi^*}(x_1)-\hat{V}_{i,1}^{\tilde\pi^{\text{PERM}}}(x_1)\big)\notag\\
    &=\frac{1}{n}\sum_{i=1}^n\big(V_{i,1}^{\pi^*}(x_1)-\hat{V}_{i,1}^{\pi^{\text{PERM}}}(x_1)\big)+\frac{1}{n}\sum_{i=1}^n\big(\hat{V}_{i,1}^{\pi^{\text{PERM}}}(x_1)-\hat{V}_{i,1}^{\tilde\pi^{\text{PERM}}}(x_1)\big)\notag\\
    &\leq\frac{1}{n}\sum_{i=1}^n\big(V_{i,1}^{\pi^*}(x_1)-\hat{V}_{i,1}^{\pi^{\text{PERM}}}(x_1)\big)+H\cdot\frac{1}{Hn}\notag\\
    &\leq \frac{1}{n}\sum_{i=1}^n\big(V_{i,1}^{\pi^*}(x_1)-\hat{V}_{i,1}^{\pi^*}(x_1)\big)+1/n\label{last term}\,,
\end{align}
where the first inequality holds because of the pessimism in Lemma \ref{pessimism lemma1}, the second inequality holds because $|\hat V_{i,1}^{\tilde{\pi}^{\text{PERM}}}(x_1)-\hat V_{i,1}^{\pi^{\text{PERM}}}(x_1)|\leq H\epsilon$ with $\epsilon$ here specified as $1/(Hn)$, and the last inequality holds because that in the algorithm design of Algo.\ref{alg:erm} we set $\pi^{\text{PERM}}=\argmax_{\pi\in\Pi}\frac{1}{n}\sum_{i=1}^n \hat{V}^\pi_{i,1}(x_1)$. 

Then what left is to bound $V_{i,1}^{\pi^*}(x_1)-\hat{V}_{i,1}^{\pi^*}(x_1)$. 

And using Lemma A.1 in \cite{jin2021pessimism}, we have
\begin{align}
    &V_{i,1}^{\pi^*}(x_1)-\hat{V}_{i,1}^{\pi^*}(x_1)\notag\\
    &=-\sum_{h=1}^H   \EE_{\hatpistar, \cM_i}\big[  \iota_{i,h}^{\pi^*}  (s_h,a_h) \biggiven s_1=x\big]+\sum_{h=1}^H   \EE_{\pi^*, \cM_i}\big[  \iota_{i,h}^{\pi^*}  (s_h,a_h) \biggiven s_1=x\big]\notag \\
    &\quad+ \sum_{h=1}^H \EE_{\pi^*, \cM_i}\big[ \langle \hat{Q}^{\pi^*}_{i,h}(s_h,\cdot) , \pi^*_h(\cdot\given s_h) - \pi^*_h(\cdot\given s_h) \rangle_{\cA} \biggiven s_1=x\big] \notag\\
    &\leq 2 \sum_{h=1}^H   \EE_{\pi^*, \cM_i}\big[  \Gamma_{i,h}   (s_h,a_h) \biggiven s_1=x\big]\label{final step}\,,
\end{align}
where in the last inequality we use Lemma \ref{lem:model_eval_err}.

Finally, with Eq.(\ref{decomposition}), Eq.(\ref{eq: generalization gap}), Eq.(\ref{cover 1}), Eq.(\ref{last term}), and Eq.(\ref{final step}), with $\epsilon$ set as $\frac{1}{nH}$, we can get w.p. at least $1-\delta$

\begin{align}\normalfont
&\EE_{c \sim C}{V_{c,1}^{\pi^*}(x_1)-V_{c,1}^{\pi^{\text{PERM}}}(x_1)}\notag \\
&\leq  \frac{5}{n} +2\sqrt{\frac{2\log(6\mathcal{N}_{(Hn)^{-1}}^\Pi/\delta)}{n}}+\frac{2}{n}\sum_{i=1}^n\sum_{h=1}^H\E{\pi^*,\cM_i}{\Gamma_{i,h}(s_h,a_h)|s_1=x_1}\notag\\
&\leq 7\sqrt{\frac{2\log(6\mathcal{N}_{(Hn)^{-1}}^\Pi/\delta)}{n}}+\frac{2}{n}\sum_{i=1}^n\sum_{h=1}^H\E{\pi^*,\cM_i}{\Gamma_{i,h}(s_h,a_h)|s_1=x_1}\,.
\notag
\end{align}
\end{proof}

\subsubsection{Proof of Theorem \ref{thm:model-free}}\label{proof:modelfree}
Our proof has two steps. First, we define that
\begin{align}
    \iota_{i,h}(x,a):=\mathbb{B}_{i,h} V_{i, h+1}(x,a) - Q_{i,h}(x,a)
\end{align}
Then we have the following lemma from \cite{jin2021pessimism}:
\begin{lemma}\label{lemma:pess}
Define the event $\cE$ as\begin{small}
    \begin{align}
    \cE = \bigg\{\big|(\hat\BB \hat V^{\pi_i}_{i,h+1})(x,a) - (\BB_{i,h} \hat V^{\pi_i}_{i,h+1})(x,a)\big|\leq \Gamma_{i,h}(x,a)~ \forall (x,a)\in \cS\times \cA ,\forall h\in[H], \forall i\in[n]  \bigg\},\notag
\end{align}
\end{small}

    Then by selecting the input parameter $\xi = \delta/(Hn)$ in $\mathbb{O}$, we have $\PP(\cE)\geq 1-\delta$ and 
    \begin{align}
        0 \leq \iota_{i,h}(x,a) \leq 2\Gamma_{i,h}(x,a).\notag
    \end{align}
\end{lemma}
\begin{proof}
    The proof is the same as [Lemma 5.1, \cite{jin2021pessimism}] with the probability assigned as $\delta/(Hn)$ and a union bound over $h\in[H], i\in[n]$.   
\end{proof}
Next lemma shows the difference between the value of the optimal policy $\pi^*$ and number $n$ of different policies $\pi_i$ for $n$ MDPs.

\begin{lemma}\label{lemma:3.1}
Let $\pi$ be an arbitrary policy. Then we have 
    \begin{align}
        \sum_{i=1}^n[V_{i,1}^{\pi}(x_1) - V_{i,1}^{\pi^i}(x_1)] &=\sum_{i=1}^n\sum_{h=1}^H \EE_{i,\pi}[\la Q_{i,h}(\cdot, \cdot), \pi_h(\cdot|\cdot) - \pi_{i,h}(\cdot| \cdot)\ra_{\cA}] \notag \\
        &\quad + \sum_{i=1}^n \sum_{h=1}^H (\EE_{i,\pi}[\iota_{i,h}(x_h, a_h)] - \EE_{i,\pi_i}[\iota_{i,h}(x_h, a_h)])
    \end{align}
\end{lemma}
\begin{proof}
    The proof is the same as Lemma 3.1 in \cite{jin2021pessimism} except substituting $\pi$ into the lemma. 
\end{proof}

We also have the following one-step lemma:
\begin{lemma}[Lemma 3.3, \cite{cai2020provably}]
For any distribution $p^*, p \in \Delta(\cA)$, if $p'(\cdot)\propto p(\cdot)\cdot \exp(\alpha\cdot Q(x,\cdot))$, then
\begin{align}
    \la Q(x, \cdot), p^*(\cdot) - p(\cdot)\ra \leq \alpha H^2/2 + \alpha^{-1}\cdot \bigg(\text{KL}(p^*(\cdot)\|p(\cdot)) -\text{KL}(p^*(\cdot)\| p'(\cdot))  \bigg).\notag
\end{align}
\end{lemma}

Given the above lemmas, we begin our proof of Theorem \ref{thm:model-free}. 
\begin{proof}[Proof of Theorem \ref{thm:model-free}]
Combining Lemma \ref{lemma:pess} and Lemma \ref{lemma:3.1}, we have
\begin{align}
    &\sum_{i=1}^n[V_{i,1}^{\pi^*}(x_1) - V_{i,1}^{\pi^i}(x_1)] \notag \\
    &\leq \sum_{i=1}^n\sum_{h=1}^H \EE_{i,\pi^*}[\la Q_{i,h}, \pi_h^* - \pi_{i,h}\ra] + 2\sum_{i=1}^n \sum_{h=1}^H\EE_{i,\pi^*}[\Gamma_{i,h}(x_h,a_h)]\notag \\
    & \leq \sum_{i=1}^n\sum_{h=1}^H \alpha H^2/2 + \alpha^{-1}\EE_{i,\pi^*}[\text{KL}(\pi_h^*(\cdot|x_h)\|\pi_{i,h}(\cdot|x_h)) - \text{KL}(\pi_h^*(\cdot|x_h)\|\pi_{i+1,h}(\cdot|x_h))]\notag \\
    &\quad + 2\sum_{i=1}^n \sum_{h=1}^H\EE_{i,\pi^*}[\Gamma_{i,h}(x_h,a_h)]\notag \\
    & \leq \alpha H^3 n/2 + \alpha^{-1}\cdot \sum_{h=1}^H\EE_{i,\pi^*}[\text{KL}(\pi_h^*(\cdot|x_h)\|\pi_{1,h}(\cdot|x_h))] + 2\sum_{i=1}^n \sum_{h=1}^H\EE_{i,\pi^*}[\Gamma_{i,h}(x_h,a_h)]\notag \\
    & \leq \alpha H^3 n/2 + \alpha^{-1}H \log|A| + 2\sum_{i=1}^n \sum_{h=1}^H\EE_{i,\pi^*}[\Gamma_{i,h}(x_h,a_h)],\notag
\end{align}
where the last inequality holds since $\pi_{1,h}$ is the uniform distribution over $\cA$. Then, selecting $\alpha = 1/\sqrt{H^2n}$, we have
\begin{align}
    \sum_{i=1}^n[V_{i,1}^{\pi^*}(x_1) - V_{i,1}^{\pi^i}(x_1)] \leq 2\sqrt{n\log|A|H^2} + 2\sum_{i=1}^n \sum_{h=1}^H\EE_{i,\pi^*}[\Gamma_{i,h}(s_h,a_h)],\notag
\end{align}
which holds for the random selection of $\cD$ with probability at least $1-\delta$. Meanwhile, note that each MDP $M_i$ is drawn i.i.d. from $C$. Meanwhile, note that $\pi_i$ only depends on MDP $M_1, ..., M_{i-1}$. Therefore, by the standard online-to-batch conversion, we have
\begin{small}
    \begin{align}
    \PP\bigg(\frac{1}{n}\sum_{i=1}^n[V_{i,1}^{\pi^*}(x_1) - V_{i,1}^{\pi_i}(x_1)] + \bigg(\frac{1}{n} \sum_{i=1}^n\EE_{c \sim C} V_{c,1}^{\pi_i}(x_1)- \EE_{c \sim C} V_{c,1}^{\pi^*}(x_1)\bigg) \leq 2H\sqrt{\frac{2\log1/\delta}{n}}\bigg) \geq 1-\delta,\notag
\end{align}
\end{small}
which suggests that with probability at least $1-2\delta$, 
\begin{align}
    &\EE_{c \sim C} V_{c,1}^{\pi^*}(x_1) - \frac{1}{n} \sum_{i=1}^n\EE_{c \sim C} V_{c,1}^{\pi_i}(x_1) \notag\\
    &\leq 2\sqrt{\frac{\log|A|H^2}{n}} + \frac{2}{n}\sum_{i=1}^n \sum_{h=1}^H\EE_{\pi^*}[\Gamma_{i,h}(x_h,a_h)] + 2\sqrt{\frac{2H\log1/\delta}{n}}.\notag
\end{align}
Therefore, by selecting $\pi^{\text{PPPO}}:=\text{random}(\pi_1, ..., \pi_n)$ and applying the Markov inequality, setting $\delta = 1/8$, we have our bound holds. 
\end{proof}

\subsection{Suboptimality bounds for real-world setups}\label{appendix:merge}
In this section we state and prove the suboptimality bounds we promised in Remarks \ref{rmk:merge} and \ref{rmk:mergefree}, where we merge the sampled contexts into $m$ groups (generally, $m<<n$) to reduce the computational complexity in practical settings. 

Assume $m|n$ and the $n$ contexts from offline dataset are equally partitioned into $m$ groups. We write the resulting average MDPs (see Proposition \ref{thm: nodistinguish}) for each group as $\bar\cM_1,\ldots,\bar\cM_m$. For each $\bar\cM_j$, we regard it as an individual context in the sense of (\ref{high prob 2}) and denote the resulting uncertainty quantifier and value function as ${\Gamma'}_{j,h}, {V'}^\pi_{j,h}$.

\begin{theorem}[Suboptimality bound for Remark \ref{rmk:merge}]\label{thm:permv}
Assume the same setting as Theorem \ref{thm:model based regret_upper_bound_general} with the original $n$ contexts grouped as $m$ contexts, and denote the resulting algorithm as PERM-$m$V. Then w.p. at least $1-\delta$, the output $\pi'$ of PERM-$m$V satisfies 
\begin{small}
    \begin{align}\normalfont
\text{SubOpt}(\pi')&\leq \underbrace{2\sqrt{\frac{2\log(6\mathcal{N}_{(Hm)^{-1}}^\Pi/\delta)}{n}}}_{I_1: \text{Supervised learning (SL) error}}+\underbrace{\frac{2}{m}\sum_{j=1}^m\sum_{h=1}^H\E{\pi^*,\bar\cM_j}{{\Gamma'}_{j,h}(s_h,a_h)|s_1=x_1}}_{I_2: \text{Reinforcement learning (RL) error}}\notag \\
&+ \underbrace{\frac{5}{m}+2 \sup_\pi \left|  \frac{1}{n}\sum_{i=1}^n V^{\pi}_{i,1}(x_1)-\frac{1}{m}\sum_{j=1}^m {V'}^{\pi}_{j,1}(x_1)\right|}_{\text{Additional approximation error}},\notag
\end{align}
\end{small}
where $\EE_{j,\pi^*}$ is w.r.t. the trajectory induced by $\pi^*$ with the transition $\bar\cP_j$ in the underlying average MDP $\bar\cM_j$.
% $\mathcal{N}_\epsilon^\Pi$ is the $\epsilon$-covering number of the policy space $\Pi$ w.r.t. distance $\mathrm d(\pi^1,\pi^2) = \max_{s\in \mathcal{S}, h \in [H]} \|\pi^1_h(\cdot|s) - \pi^2_h (\cdot|s)\|_{1}$. 
\end{theorem}

\begin{proof}[Proof of Theorem \ref{thm:permv}]

Similar to the proof of Theorem \ref{thm:model based regret_upper_bound_general}, we decompose the suboptimality gap as follows
\begin{align}
    &\text{SubOpt}(\pi')\notag \\
    &=\EE_{c \sim C}{V_{c,1}^{\pi^*}(x_1)-V_{c,1}^{\pi'}(x_1)}\notag \\
    &=\EE_{c \sim C}{V_{c,1}^{\pi^*}(x_1)}-\frac{1}{n}\sum_{i=1}^n V^{\pi^*}_{i,1}(x_1) +\frac{1}{n}\sum_{i=1}^n V_{i,1}^{\pi'}(x_1)-\EE_{c \sim C}{V_{c,1}^{\pi'}(x_1)}\notag\\
    &\quad +\frac{1}{n}\sum_{i=1}^n V^{\pi^*}_{i,1}(x_1)-\frac{1}{m}\sum_{j=1}^m {V'}^{\pi^*}_{j,1}(x_1)+\frac{1}{m}\sum_{j=1}^m {V'}_{j,1}^{\pi'}(x_1)-\frac{1}{n}\sum_{i=1}^n V_{i,1}^{\pi'}(x_1)\notag\\
    &\quad +\frac{1}{m}\sum_{j=1}^m \big({V'}_{j,1}^{\pi^*}(x_1)-{V'}_{j,1}^{\pi'}(x_1)\big)\label{decomposition1}\,.
\end{align}

Note that we can bound the first and third lines of (\ref{decomposition1}) with the exactly same arguments as the proof of Theorem \ref{thm:model based regret_upper_bound_general}, the only notation-wise difference is that the uncertainty quantifier becomes $\Gamma'$ as we are operating on the level of average MDP $\bar\cM_j$.

The only thing left is to bound the second line of (\ref{decomposition1}). This is the same in spirit of the bound (\ref{eq:model based gap no context}), so that we can express the bound as follows
\begin{align}
    &\frac{1}{n}\sum_{i=1}^n V^{\pi^*}_{i,1}(x_1)-\frac{1}{m}\sum_{j=1}^m {V'}^{\pi^*}_{j,1}(x_1)+\frac{1}{m}\sum_{j=1}^m {V'}_{j,1}^{\pi'}(x_1)-\frac{1}{n}\sum_{i=1}^n V_{i,1}^{\pi'}(x_1)\notag\\
    &\leq 2 \sup_\pi \left|  \frac{1}{n}\sum_{i=1}^n V^{\pi}_{i,1}(x_1)-\frac{1}{m}\sum_{j=1}^m {V'}^{\pi}_{j,1}(x_1)\right|.\notag
\end{align}

To conclude, our final bound can be expressed as: with $\epsilon$ set as $\frac{1}{mH}$, we can get w.p. at least $1-\delta$
\begin{align}
    &\text{SubOpt}(\pi')\notag \\
    &\leq 2\sqrt{\frac{2\log(6\mathcal{N}_{(Hm)^{-1}}^\Pi/\delta)}{n}}+\frac{2}{m}\sum_{j=1}^m\sum_{h=1}^H\E{\pi^*,\bar\cM_j}{{\Gamma'}_{j,h}(s_h,a_h)|s_1=x_1}\notag\\
    &\quad +\frac{5}{m}+2 \sup_\pi \left|  \frac{1}{n}\sum_{i=1}^n V^{\pi}_{i,1}(x_1)-\frac{1}{m}\sum_{j=1}^m {V'}^{\pi}_{j,1}(x_1)\right|.\notag
\end{align}
\end{proof}

To prove the suboptimality bound for Remark \ref{rmk:mergefree}, we denote that the policies produced by PPPO after merging dataset to $m$ groups to be $\pi_1,\ldots,\pi_m$, and the original PPPO algorithm would produce the policies as $\pi'_1,\ldots,\pi'_n$. We assume that the merging of dataset from $n$ to $m$ groups is only to combine the consecutive $n/m$ terms from $\pi'_1,\ldots,\pi'_n$ and preserves the order.

\begin{theorem}[Suboptimality bound for Remark \ref{rmk:mergefree}]\label{thm:pppov}
    Assume the same setting as Theorem \ref{thm:model-free} with the original $n$ contexts grouped as $m$ contexts, and denote the resulting algorithm as PPPO-$m$V. Let ${\Gamma'}_{j,h}$ be the uncertainty quantifier returned by $\mathbb{O}$ through the PPPO-$m$V algorithm. Selecting $\alpha = 1/\sqrt{H^2m}$. Then selecting $\delta = 1/8$, w.p. at least $2/3$, we have
    \begin{small}
        \begin{align}
    \text{SubOpt}(\pi^{\text{PPPO}-mV}) &\leq 10\bigg(\underbrace{\sqrt{\frac{\log|\actions|H^2}{m}}}_{I_1: \text{SL error}} + \underbrace{\frac{1}{m}\sum_{j=1}^m \sum_{h=1}^H\E{j,\pi^*}{{\Gamma'}_{j,h}(s_h,a_h)|s_1=x_1}}_{I_2: \text{RL error}}\notag \\
    &+\sup_\pi\left| \frac{1}{n}\sum_{i=1}^nV_{i,1}^{\pi}(x_1)-\frac{1}{m}\sum_{j=1}^m{V'}_{j,1}^{\pi}(x_1)\right|+\frac{1}{n}\sum_{i=1}^n\sup_\pi\left| \mathbb{E}_c[V_{c,1}^{\pi}(x_1)]-V_{i,1}^{\pi}(x_1)\right|\notag \\
    &+\frac{1}{m}\sum_{j=1}^m\sup_\pi\left| \mathbb{E}_c[{V'}_{c,1}^{\pi}(x_1)]-{V'}_{j,1}^{\pi}(x_1)\right|\bigg).\notag
\end{align}
    \end{small}
where $\EE_{j,\pi^*}$ is w.r.t. the trajectory induced by $\pi^*$ with the transition $\bar\cP_j$ in the underlying MDP $\bar\cM_j$.
\end{theorem}

\begin{proof}[Proof of Theorem \ref{thm:pppov}]

Using the same arguments as in the proof of Theorem \ref{thm:model-free} with $\alpha=1/\sqrt{H^2m}$, we can derive the bound

$$\sum_{j=1}^m[{V'}_{j,1}^{\pi^*}(x_1) - {V'}_{j,1}^{\pi_j}(x_1)] \leq 2\sqrt{m\log|A|H^2} + 2\sum_{j=1}^m \sum_{h=1}^H\EE_{j,\pi^*}[{\Gamma'}_{j,h}(s_h,a_h)].$$

Leveraging this bound and online-to-batch, we obtain the following estimation

\begin{align}
    &\mathbb{E}_{c}[V^{\pi^\ast}_{c,1}(x_1)]-\frac{1}{m}\sum_{j=1}^{m}\mathbb{E}_{c}[V^{\pi_j}_{c,1}(x_1)]\notag \\
    =& \mathbb{E}_{c}[V^{\pi^\ast}_{c,1}(x_1)]-\frac{1}{n}\sum_{i=1}^{n}\mathbb{E}_{c}[V^{\pi'_i}_{c,1}(x_1)]+\frac{1}{n}\sum_{i=1}^{n}\mathbb{E}_{c}[V^{\pi'_i}_{c,1}(x_1)]-\frac{1}{m}\sum_{j=1}^{m}\mathbb{E}_{c}[V^{\pi_j}_{c,1}(x_1)]\notag\\
    \leq &2H\sqrt{\frac{2\log 1/\delta}{n}}+ \frac{1}{n}\sum_{i=1}^n\left( \mathbb{E}_c[V_{c,1}^{\pi'_i}(x_1)]-V_{i,1}^{\pi'_i}(x_1)\right)\notag\\
    &\quad+ \frac{1}{n}\sum_{i=1}^nV_{i,1}^{\pi^*}(x_1)-\frac{1}{m}\sum_{j=1}^{m}\mathbb{E}_{c}[V^{\pi_j}_{c,1}(x_1)]\notag\\
    = & 2H\sqrt{\frac{2\log 1/\delta}{n}}+\frac{1}{n}\sum_{i=1}^nV_{i,1}^{\pi^*}(x_1)-\frac{1}{m}\sum_{j=1}^m{V'}_{j,1}^{\pi^*}(x_1)\notag \\
    &+\frac{1}{m}\sum_{j=1}^m{V'}_{j,1}^{\pi^*}(x_1)-\frac{1}{m}\sum_{j=1}^m{V'}_{j,1}^{\pi_j}(x_1)\notag \\
    &+ \frac{1}{n}\sum_{i=1}^n\left( \mathbb{E}_c[V_{c,1}^{\pi'_i}(x_1)]-V_{i,1}^{\pi'_i}(x_1)\right)+\frac{1}{m}\sum_{j=1}^m{V'}_{j,1}^{\pi_j}(x_1)-\frac{1}{m}\sum_{j=1}^{m}\mathbb{E}_{c}[V^{\pi_j}_{c,1}(x_1)]\notag \\
    \leq &2H\sqrt{\frac{2\log 1/\delta}{n}}+\sup_\pi\left| \frac{1}{n}\sum_{i=1}^nV_{i,1}^{\pi}(x_1)-\frac{1}{m}\sum_{j=1}^m{V'}_{j,1}^{\pi}(x_1)\right|\notag \\
    &+2\sqrt{\frac{\log|A|H^2}{m}} + \frac{2}{m}\sum_{j=1}^m \sum_{h=1}^H\EE_{j,\pi^*}[{\Gamma'}_{j,h}(s_h,a_h)]\notag \\
    &+ \frac{1}{n}\sum_{i=1}^n\sup_\pi\left| \mathbb{E}_c[V_{c,1}^{\pi}(x_1)]-V_{i,1}^{\pi}(x_1)\right|+ \frac{1}{m}\sum_{j=1}^m\sup_\pi\left| \mathbb{E}_c[{V'}_{c,1}^{\pi}(x_1)]-{V'}_{j,1}^{\pi}(x_1)\right|.\notag
\end{align}
Finally we apply Markov inequality and take $\delta=1/8$ as in the proof of Theorem \ref{thm:model-free}.
\end{proof}
\subsection{Results in Section \ref{sec:linear}}\label{proof:linear}
\subsubsection{Proof of Theorem \ref{thm:regret_upper_linear}}
    By \cite{jin2021pessimism}, the parameters specified as $\lambda=1,\quad \beta(\delta) = c\cdot dH\sqrt{\log(2dHK/\delta)}$, and applying union bound, we can get:
for Algo.\ref{alg:linear mdp model based}, with probability at least $1-\delta/3$
\begin{align}
    &\big|(\hat\BB_{i,h} \hat{V}^\pi_{i,h+1})(x,a) - (\BB_{i,h} \hat{V}^\pi_{i,h+1})(x,a)\big|\leq  \beta\big(\frac{\delta}{3nH\mathcal{N}_{(Hn)^{-1}}^\Pi}\big) \bigl(\phi(x,a)^\top \Lambda_{i,h}^{-1}\phi(x,a)\bigr)^{1/2}\,,\notag\\
    &\quad\text{for all}~i\in[n], \pi\in\tilde\Pi, (x,a)\in \cS\times \cA, h\in[H]\,,\label{high prob 3}
\end{align}
where $\tilde\Pi$ is the $\frac{1}{Hn}$-covering set of the policy space $\Pi$ w.r.t. distance $\mathrm d(\pi^1,\pi^2) = \max_{s\in \mathcal{S}, h \in [H]} \|\pi^1_h(\cdot|s) - \pi^2_h (\cdot|s)\|_{1}$.

Therefore, we can specify the $\Gamma_{i,h}(\cdot,\cdot)$ in Theorem \ref{thm:model based regret_upper_bound_general} with $\beta\big(\frac{\delta}{3nH\mathcal{N}_{(Hn)^{-1}}^\Pi}\big) \bigl(\phi(x,a)^\top \Lambda_{i,h}^{-1}\phi(x,a)\bigr)^{1/2}$, and follow the same process as the proof of Theorem \ref{thm:model based regret_upper_bound_general} to get the result for Algo.\ref{alg:erm} with subroutine Algo.\ref{alg:linear mdp model based}.

Similarly, we can get: we can get:
for Algo.\ref{alg:linear mdp model based}, with probability at least $1-1/4$
\begin{align}
    &\big|(\hat\BB_{i,h} \hat{V}_{i,h+1})(x,a) - (\BB_{i,h} \hat{V}_{i,h+1})(x,a)\big|\leq  \beta\big(\frac{\delta}{4nH}\big) \bigl(\phi(x,a)^\top \Lambda_{i,h}^{-1}\phi(x,a)\bigr)^{1/2}\,,\notag\\
    &\quad\text{for all}~i\in[n], (x,a)\in \cS\times \cA, h\in[H]\,.\label{high prob 3}
\end{align}
Therefore, we can specify the $\Gamma_{i,h}(\cdot,\cdot)$ in Theorem \ref{thm:model-free} with $\beta\big(\frac{\delta}{4nH}\big) \bigl(\phi(x,a)^\top \Lambda_{i,h}^{-1}\phi(x,a)\bigr)^{1/2}$ and follow the same process as the proof of Theorem \ref{thm:model-free} to get the result for Algo.\ref{alg:modelfree} with subroutine Algo.\ref{alg:linear mdp model based}.
\subsubsection{Proof of Corollary \ref{cor:well_explore}}
   By the assumption that $\cD_i$ is generated by behavior policy $\bar\pi_i$ which well-explores MDP $\cM_i$ with constant $c_i$ (where the well-explore is defined in Def.\ref{ass:wellexp}), the proof of Corollary 4.6 in \cite{jin2021pessimism}, and applying a union bound over $n$ contexts, we have that for Algo.\ref{alg:erm} with subroutine Algo.\ref{alg:linear mdp model based} w.p. at least $1-\delta/2$
\begin{align}
    &\|\phi(x,a)\|_{\Lambda_{i,h}^{-1}}\leq \sqrt{\frac{2d}{c_i K}}\notag\\&~\textrm{for all}~i\in[n], ~(x,a)\in \cS\times \cA \text{ and all } h\in[H]  \,,
\end{align}
and for Algo.\ref{alg:erm} with subroutine Algo.\ref{alg:linear mdp model based} w.p. at least $1-\delta/2$
\begin{align}
    &\|\phi(x,a)\|_{\Lambda_{i,h}^{-1}}\leq \sqrt{\frac{2dH}{c_i K}}\notag\\&~\textrm{for all}~i\in[n], ~(x,a)\in \cS\times \cA \text{ and all } h\in[H]  \,,
\end{align}
because we use the data splitting technique and we only utilize each trajectory once for one data tuple at some stage $h$, so we replace $K$ with $K/H$.

Then, the result follows by plugging the results above into Theorem \ref{thm:regret_upper_linear}.
\section{Appendix for Chapter \ref{chapter:mis}}
\subsection{More Discussions on Related Work}\label{appendix: related work}
In this section, we will give more comparisons and discussions on some previous works that are related to our work to some extent.

There are some other works on bandits leveraging user (or task) relations, which have some relations with the clustering of bandits (CB) works to some extent, but are in different lines of research from CB, and are quite different from our work. First, besides CB, the work \cite{wu2016contextual} also leverages user relations. Specifically, it utilizes a \textit{known} user adjacency graph to share context and payoffs among neighbors, whereas in CB, the user relations are \textit{unknown} and need to be learnt, thus the setting differs a lot from CB. Second, there are lines of works on multi-task learning \cite{cella2021multi,deshmukh2017multi,soare2014multi,cella2023multi,wang2022thompson,wang2021multitask}, meta-learning \cite{wan2023towards,hong2022hierarchical,cella2020meta} and federated learning \cite{shi2021federated,huang2021federated}, where multiple different tasks are solved jointly and share information. Note that all of these works do not assume an underlying \textit{unknown} user clustering structure which needs to be inferred by the agent to speed up learning. For works on multi-task learning \cite{cella2021multi,deshmukh2017multi,soare2014multi,cella2023multi,wang2022thompson,wang2021multitask}, they assume the tasks are related but no user clustering structures, and to the best of our knowledge, none of them consider model misspefications, thus differing a lot from ours. For some recent works on meta-learning \cite{wan2023towards,hong2022hierarchical,wan2021metadata}, they propose general Bayesian hierarchical models to share knowledge across tasks, and design Thompson-Sampling-based algorithms to optimize the Bayes regret, which are quite different from the line of CB works, and differ a lot from ours.
And additionally, as supported by the discussions in the works \cite{cella2020meta,wang2021multitask}, multi-task learning and meta-learning are different lines of research from CB. For the works on federated learning \cite{shi2021federated,huang2021federated}, they consider the privacy and communication costs among multiple servers, whose setting is also very different from the previous CB works and our work. 

\noindent\textbf{Remark.} Again, we emphasize that the goal of this work is to initialize the study of the important CBMUM problem, and propose general design ideas for dealing with model misspecifications in CB problems. Therefore, our study is based on fundamental models on CB \cite{gentile2014online, 10.5555/3367243.3367445} and MLB \cite{lattimore2020learning}, and the algorithm design ideas and theoretical analysis are pretty general. We leave incorporating the more recent model selection methods \cite{pacchiano2020model,foster2020adapting} into our framework to address the unknown exact maximum model misspecification level as an interesting future work. It would also be interesting to consider incorporating our methods and ideas of tackling model misspecifications into the studies of multi-task learning, meta learning and federated learning.

\subsection{More Discussions on Assumptions}\label{appendix: assumptions}
All the assumptions (Assumptions \ref{assumption1},\ref{assumption2},\ref{assumption3},\ref{assumption4})in this work are natural and basically follow (or less strigent than) previous works on CB and MLB \cite{gentile2014online,li2018online,10.5555/3367243.3367445,liu2022federated,lattimore2020learning}. 
\subsubsection{Less Strigent Assumption on on the Generating Distribution of Arm Vectors}
We also make some contributions to relax a widely-used but stringent assumption on the generating distribution of arm vectors. Specifically, our Assumption \ref{assumption3} on item regularity relaxes the previous one used in previous CB works \cite{gentile2014online,li2018online,10.5555/3367243.3367445,liu2022federated} by removing the condition that the variance should be upper bounded by $\frac{\lambda^2}{8\log (4\left|\mathcal{A}_t\right|)}$. For technical details on this, please refer to the theoretical analysis and discussions in Appendix \ref{appendix: technical}.
\subsubsection{Discussions on Assumption \ref{assumption4} about Bounded Misspecification Level}

This assumption follows \cite{lattimore2020learning}. Note that this $\epsilon_*$ can be an upper bound on the maximum misspecification
level, not the exact maximum itself. In real-world applications, the deviations are usually small \cite{ghosh2017misspecified}, and we can set a relatively big $\epsilon_{*}$ (e.g., 0.2) to be the upper bound. Our experimental results support this claim. As shown in our experimental results on real-data case 2, even when $\epsilon_{*}$ is unknown, our algorithms still perform well by setting $\epsilon_{*} = 0.2$. 
Some recent studies \cite{pacchiano2020model,foster2020adapting} use model selection methods to theoretically
deal with unknown exact maximum misspecification level in the single-user case, which is not the emphasis of this work. Additionally, the work \cite{foster2020adapting} assumes that the learning agent has access to a regression oracle. And for the work \cite{pacchiano2020model}, though their regret bound is dependent on the exact maximum misspecification level that needs not to be known by the agent, an upper bound of the exact maximum misspecification level is still needed.
We leave incorporating their methods to deal with unknown exact maximum misspecification level as an interesting future work.
\subsubsection{Discussions on Assumption \ref{assumption2} about the Theoretical Results under General User Arrival Distributions}
The uniform arrival in Assumption \ref{assumption2} follows previous CB works \cite{gentile2014online,li2018online,liu2022federated}, it only affects the $T_0$ term, which is the time after which the algorithm maintains a “good partition” and is of $O(u\log T)$. For an arbitrary arrival distribution, $T_0$ becomes $O(1/p_{min} \log T)$, where $p_{min}$ is the minimal arrival probability of a user. And since it is a lower-order term (of $O(\log T)$), it will not affect the main order of our regret upper bound which is of $O(\epsilon_*T\sqrt{md\log T}  + d\sqrt{mT}\log T)$. The work \cite{10.5555/3367243.3367445} studies arbitrary arrivals and aims to remove the $1/p_{min}$ factor in this term, but their setting is different. They make an additional assumption that users in the same cluster not only have the same preference vector, but also the same arrival probability, which is different from our setting and other classic CB works \cite{gentile2014online,li2018online,liu2022federated} where we only assume users in the same cluster share the same preference vector.

\subsection{Highlight of the Theoretical Analysis\label{appendix: theory}}
Our proof flow and methodologies are novel in clustering of bandits (CB), which are expected to inspire future works on model misspecifications and CB. The main challenge of the regret analysis in CBMUM is that due to the estimation inaccuracy caused by misspecifications, it is impossible to cluster all users exactly correctly, and it is highly non-trivial to 
  bound the regret caused by \textbf{``misclustering" $\zeta$-close users}. 
  
 % The two key challenges are: (1) how to deal with the \textbf{dynamic clustering structure} after ``good partition"; (2) how to bound the regret caused by \textbf{``misclustering" $\zeta$-close users}. 

To the best of our knowledge, the common proof flow of previous CB works (e.g., \cite{gentile2014online,li2018online,liu2022federated}) can be summarized in two steps: The first is to prove a sufficient time $T^{\prime}_0$ after which the algorithms can cluster all users \textbf{exactly correctly} with high probability. Note that the inferred clustering structure remains static after $T^{\prime}_0$, making the analysis easy. Second, after the \textbf{correct static clustering}, the regret can be trivially bounded by bounding $m$ (number of underlying clusters) independent linear bandit algorithms, resulting in a $O(d\sqrt{mT}\log T)$ regret. 

The above common proof flow is straightforward in CB with perfectly linear models, but it would fail to get a non-vacuous regret bound for CBMUM. 
In CBMUM, it is impossible to learn an exactly correct static clustering structure with model misspecifications. In particular, we prove that we can only expect the algorithm to cluster $\zeta$-close users together rather than cluster all users exactly correctly. Therefore, the previous flow can not be applied to the more challenging CBMUM problem.

We do the following to address the challenges in obtaining a tight regret bound for CBMUM. With the carefully-designed novel key components of RCLUMB, we can prove a sufficient time $T_0$ after which RCLUMB can get a ``good partition" (Definition \ref{def:good partition}) with high probability, which means the cluster $\overline{V_t}$ assigned to $i_t$ contains all users in the same ground-truth cluster as $i_t$, and possibly some other $i_t$'s $\zeta$-close users. Intuitively, after $T_0$, the algorithm can leverage all the information from the users' ground-truth clusters but may misuse some information from other  $\zeta$-close users with preference gaps up to $\zeta$, causing a regret of \textbf{``misclustering" $\zeta$-close users}. It is highly non-trivial to bound this part of regret, and the proof methods would be beneficial for future studies in CB in challenging cases when it is impossible to cluster all users exactly correctly. For details, please refer to the discussions ``(ii) Bounding the term of misclustering it's $\zeta$-close users" in Section \ref{section: theoretical}, the key Lemma \ref{bound for mis} (Bound of error caused by misclustering), its proof and tightness discussion in Appendix \ref{key lemma appendix}. Also, a more subtle analysis is needed to handle
the time-varying inferred clustering structure since the ``good partition" may change over time, whereas in the previous CB works, the clustering structure remains static after $T_0^{\prime}$. For theoretical details on this, please refer to 
Appendix \ref{appendix: proof of main thm}.

\subsection{Discussions on why Trivially Combining Existing CB and MLB Works Could Not Achieve a Non-vacuous Regret Upper Bound}\label{appendix: why trivial not apply}

We consider discussing regret upper bounds for CB without considering misspecifications for three cases: (1) neither the clustering process nor the decision process considers misspecifications (previous CB algorithms); (2) the decision process does not consider misspecifications; (3) the clustering process does not consider misspecifications.

For cases (1) and (2), the decision process could contribute to the leading regret. We consider the case where there are $m$ underlying clusters, with each cluster's arrival being $T/m$, and the agent knows the underlying clustering structure. For this case, there exist some instances where the regret upper bound $R(T)$ is strictly larger than $\epsilon_{*}T\sqrt{m\log T}$ asymptotically in $T$. Formally, in the discussion of ``Failure of unmodified algorithm" in Appendix E in \cite{lattimore2020learning},
 they give an example to show that in the single-user case, the regret $R_1(T)$ of the classic linear bandit algorithms without considering misspecifications will have: $\displaystyle\lim_{T \rightarrow + \infty}\frac{R_1(T)}{\epsilon_*T\sqrt{m\log T}}=+ \infty$. In our problem with multiple users and $m$ underlying clusters, even if we know the underlying clustering structure and keep $m$ independent linear bandit algorithms with $T_i$ for the cluster $i \in [m]$ to leverage the common information of clusters, the best we can get is $R_2(T)=\sum_{i \in [m]}R_1(T_i)$. By the above results, if the decision process does not consider misspecifications, we have $\displaystyle\lim_{T \rightarrow + \infty}\frac{R_2(T)}{\epsilon_*T\sqrt{m\log T}}=\displaystyle\lim_{T \rightarrow + \infty}\frac{mR_1(T/m)}{\epsilon_*T\sqrt{m\log T}}=+ \infty$. Recall that the regret upper bound $R(T)$ of our proposed algorithms is of $O(\epsilon_*T\sqrt{md\log T}  + d\sqrt{mT}\log T)$ (thus, we have $\displaystyle\lim_{T \rightarrow + \infty}\frac{R(T)}{\epsilon_*T\sqrt{m\log T}}<+ \infty$), which gives a proof that that the regret upper bound of our proposed algorithms is asymptotically much better than CB algorithms in cases (1)(2).

For case (3), if the clustering process does not use the more tolerant deletion rule in Line \ref{alg:RCLUMB:delete}
 of Algo.\ref{alg:RCLUMB}, the gap between users linked by edges would possibly exceed $\zeta$ ($\zeta= 2\epsilon_*\sqrt{\frac{2}{\tilde{\lambda}_x}}$) even after $T_0$, which will result in a regret upper bound no better than $O(\epsilon_*u\sqrt{d}T)$. As the number of users $u$ is usually huge in practice, this result is vacuous. The reasons for getting the above claim are as follows. Even if the clustering process further uses our deletion rule considering misspecifications, and the users linked by edges are within $\zeta$ distance, failing to extract $1$-hop users (Line \ref{alg:RCLUMB:find V} in Algo.\ref{alg:RCLUMB}) would cause the leading $O(\epsilon_*u\sqrt{d}T)$ regret term, as in the worst case, the preference vector $\theta$ of the user in $\tilde{V}_t$ who is $h$-hop away from user $i_t$ could deviate by $h\zeta$ from $\theta_{i_t}$, where $h$ can be as large as $u$, and it would make the second term in Eq.(\ref{bigO}) a $O(\epsilon_*u\sqrt{d}T)$ term. If we completely do not consider the misspecifications in the clustering process, the above user gap between users linked by edges would possibly exceed $\zeta$, which will cause a regret upper bound worse than $O(\epsilon_*u\sqrt{d}T)$.
\subsection{Proof of Theorem \ref{thm:main}}
\label{appendix: proof of main thm}
We first prove the result in the case when $\gamma_1$ defined in Definition \ref{def:gap} is not infinity, i.e., $4\epsilon_*\sqrt{\frac{2}{\tilde{\lambda}_x}}<\gamma_1<\infty$. The proof of the special case when $\gamma_1=\infty$ will directly follow the proof of this case.

For the instantaneous regret $R_t$ at round $t$, with probability at least $1-5\delta$ for some $\delta\in(0,\frac{1}{5})$, at $\forall{t\geq T_0}$:
    \begin{equation}
    \begin{aligned}
    R_t&=(\bx_{a_t^*}^{\top}\btheta_{i_t}+\bepsilon^{i_t,t}_{a_t^*})-(\bx_{a_t}^{\top}\btheta_{i_t}+\bepsilon^{i_t,t}_{a_t})\\
    &=\bx_{a_t^*}^{\top}(\btheta_{i_t}-\hat{\btheta}_{\overline{V}_t,t-1})+(\bx_{a_t^*}^{\top}\hat{\btheta}_{\overline{V}_t,t-1}+C_{a_t^*,t})-(\bx_{a_t}^{\top}\hat{\btheta}_{\overline{V}_t,t-1}+C_{a_t,t})\\
    &\quad+\bx_{a_t}^{\top}(\hat{\btheta}_{\overline{V}_t,t-1}-\btheta_{i_t})+C_{a_t,t}-C_{a_t^*,t}+(\bepsilon^{i_t,t}_{a_t^*}-\bepsilon^{i_t,t}_{a_t})\\
    &\leq 2C_{a_t,t}+\frac{2\epsilon_*\sqrt{2d}}{\tilde{\lambda}_x^{\frac{3}{2}}}\mathbb{I}\{\overline{V}_t\notin \mathcal{V}\}+2\epsilon_*\,,
    \end{aligned}
    \label{bound r_t}
\end{equation}
where the last inequality holds by the UCB arm selection strategy in Eq.(\ref{UCB}), the concentration bound given in Lemma \ref{concentration bound}, and the fact that $\norm{\bepsilon^{i,t}}_{\infty}\leq\epsilon_*, \forall{i\in\mathcal{U}},\forall{t}$.

We define the following events. Let
  \begin{align*}
\cE_0 &= \{R_t\leq 2C_{a_t,t}+\frac{2\epsilon_*\sqrt{2d}}{\tilde{\lambda}_x^{\frac{3}{2}}}\mathbb{I}\{\overline{V}_t\notin \mathcal{V}\}+2\epsilon_*, \notag\\
&\text{for all }   \{t: t\geq T_0, \text{and the algorithm maintains a ``good partition" at $t$}\}\}\,,\\
\cE_1 &= \{ \text{the algorithm maintains a ``good partition" for all } t \geq T_0\}\,,\\
\cE &=\cE_0 \cap \cE_1
\,.
\end{align*} 

$\PP (\cE_0)\geq 1-2\delta$. According to Lemma \ref{sufficient time}, $\PP (\cE_1)\geq 1-3\delta$. Thus, $\PP (\cE)\geq 1-5\delta$ for some $\delta\in(0,\frac{1}{5})$. Take $\delta=\frac{1}{T}$, we can get that
    \begin{equation}
\begin{aligned}
    \EE[R(T)]&=\PP(\cE)\mathbb{I}\{\cE\}R(T)+\PP(\overline{\cE})\mathbb{I}\{\overline{\cE}\}R(T)\\
    &\leq \mathbb{I}\{\cE\}R(T)+ 5\times\frac{1}{T}\times T\\
    &=\mathbb{I}\{\cE\}R(T)+5\,,
\end{aligned}
\end{equation}

where $\overline{\cE}$ denotes the complementary event of $\cE$, $\mathbb{I}\{\cE\}R(T)$ denotes $R(T)$ under event $\cE$, $\mathbb{I}\{\overline{\cE}\}R(T)$ denotes $R(T)$ under event $\overline{\cE}$, and we use $R(T)\leq T$ to bound $R(T)$ under event $\overline{\cE}$.

Then it remains to bound $\mathbb{I}\{\cE\}R(T)$:
    \begin{align}
        \mathbb{I}\{\cE\}R(T)&\leq R(T_0)+\EE [\mathbb{I}\{\cE\}\sum_{t=T_0+1}^{T}R_t]\nonumber\\
    &\leq T_0+2\EE [\mathbb{I}\{\cE\}\sum_{t=T_0+1}^{T}C_{a_t,t}]+\frac{2\epsilon_*\sqrt{2d}}{\tilde{\lambda}_x^{\frac{3}{2}}}\sum_{t=T_0+1}^{T}\EE[\mathbb{I}\{\cE, \overline{V}_t\notin \mathcal{V}\}]+2\epsilon_*T\label{general bound}\\
    &=T_0+2\EE [\mathbb{I}\{\cE\}\sum_{t=T_0+1}^{T}C_{a_t,t}]+\frac{2\epsilon_*\sqrt{2d}}{\tilde{\lambda}_x^{\frac{3}{2}}}\sum_{t=T_0+1}^{T}\PP(\mathbb{I}\{\cE, \overline{V}_t\notin \mathcal{V}\})+2\epsilon_*T\nonumber\\
    &\leq T_0+2\EE [\mathbb{I}\{\cE\}\sum_{t=T_0+1}^{T}C_{a_t,t}]+\frac{2\epsilon_*\sqrt{2d}}{\tilde{\lambda}_x^{\frac{3}{2}}}\times\frac{\tilde{u}}{u}T+2\epsilon_*T\label{regret parts}\,,
\end{align}
where Eq.(\ref{general bound}) follows from Eq.(\ref{bound r_t}). Eq.(\ref{regret parts}) holds since under Assumption \ref{assumption2} about user arrival uniformness and by Definition \ref{def:good partition} of ``good partition", $\PP(\mathbb{I}\{\cE, \overline{V}_t\notin \mathcal{V}\})\leq\frac{\tilde{u}}{u}, \forall{t\geq T_0}$, where $\tilde{u}$ is defined in Definition \ref{def:user number}.

Then we need to bound $\EE [\mathbb{I}\{\cE\}\sum_{t=T_0+1}^{T}C_{a_t,t}]$:
\begin{align}
    \mathbb{I}\{\cE\}\sum_{t=T_0+1}^{T}C_{a_t,t}&=\Big(\sqrt{\lambda}+\sqrt{2\log(\frac{1}{\delta})+d\log(1+\frac{T}{\lambda d})}\Big)\mathbb{I}\{\cE\}\sum_{t=T_0+1}^{T}\norm{\bx_{a_t}}_{\overline{\bM}_{\overline{V}_t,t-1}^{-1}}\nonumber\\
    \quad&+\mathbb{I}\{\cE\}\epsilon_*\sum_{t=T_0+1}^{T}\sum_{s\in[t-1]\atop i_s\in \overline{V}_t}\left|\bx_{a_t}^{\top}\overline{\bM}_{\overline{V}_t,t-1}^{-1}\bx_{a_s}\right|\,.
    \label{sum C}
\end{align}

Next, we bound the $\mathbb{I}\{\cE\}\sum_{t=T_0+1}^{T}\norm{\bx_{a_t}}_{\overline{\bM}_{\overline{V}_t,t-1}^{-1}}$ term in Eq.(\ref{sum C}):
\begin{align}
    &\mathbb{I}\{\cE\}\sum_{t=T_0+1}^{T}\norm{\bx_{a_t}}_{\overline{\bM}_{\overline{V}_t,t-1}^{-1}}\notag\\
    &=\mathbb{I}\{\cE\}\sum_{t=T_0+1}^{T}\sum_{k=1}^{m_t}\mathbb{I}\{i_t\in \tilde{V}^{\prime}_{t,k}\}\norm{\bx_{a_t}}_{\overline{\bM}_{\overline{V}^{\prime}_{t,k},t-1}^{-1}}\nonumber\\
    &\leq\mathbb{I}\{\cE\}\sum_{t=T_0+1}^{T}\sum_{j=1}^{m}\mathbb{I}\{i_t\in V_j\}\norm{\bx_{a_t}}_{\overline{\bM}_{V_j,t-1}^{-1}}\label{cluster inequality}\\
    &\leq \mathbb{I}\{\cE\}\sum_{j=1}^{m}\sqrt{\sum_{t=T_0+1}^{T}\mathbb{I}\{i_t\in V_j\}\sum_{t=T_0+1}^{T}\mathbb{I}\{i_t\in V_j\}\norm{\bx_{a_t}}_{\overline{\bM}_{V_j,t-1}^{-1}}^2}\label{cauchy1}\\
    &\leq \mathbb{I}\{\cE\}\sum_{j=1}^{m}\sqrt{2T_{V_j,T}d\log (1+\frac{T}{\lambda d})}\label{lemma7 use1}\\
    &\leq \mathbb{I}\{\cE\}\sqrt{2\sum_{j=1}^{m}1\sum_{j=1}^{m}T_{V_j,T}d\log (1+\frac{T}{\lambda d})}   \notag\\& = \mathbb{I}\{\cE\}\sqrt{2mdT\log(1+\frac{T}{\lambda d})}\label{cauchy2}\,,
\end{align}

where we use $m_t$ to denote the number of connected components partitioned by the algorithm at $t$, $\tilde{V}^{\prime}_{t,k}, k\in [m_t]$ to denote the connected components partitioned by the algorithm at $t$, $\overline{V}^{\prime}_{t,k}\subseteq \tilde{V}^{\prime}_{t,k}$ to denote the subset extracted to be the cluster $\overline{V}_t$ for $i_t$ from $\tilde{V}^{\prime}_{t,k}$ conditioned on $i_t\in \tilde{V}^{\prime}_{t,k}$, and $T_{V_j,T}$ to denote the number of times that the served users lie in the \gtcluster{} $V_j$ up to time $T$, i.e., $T_{V_j,T}=\sum_{t\in[T]}\mathbb{I}\{i_t\in V_j\}$.

The reasons for having Eq.(\ref{cluster inequality}) are as follows. Under event $\cE$, the algorithm will always have a ``good partition" after $T_0$. By Definition \ref{def:good partition} and the proof process of Lemma \ref{sufficient time} about the edge deletion conditions, we can get $m_t\leq m$ and if $i_t\in \tilde{V}^{\prime}_{t,k}, i_t\in V_j$, then $V_j \subseteq \overline{V}^{\prime}_{t,k}$ since $\overline{V}^{\prime}_{t,k}$ contains $V_j$ and possibly other \gtclusters{} $V_n,n\in[m]$, whose preference vectors are $\zeta$-close to $\btheta^{j}$. Therefore, by the definition of the regularized Gramian matrix, we can get $M_{\overline{V}^{\prime}_{t,k},t-1}\succeq M_{V_j,t-1},\forall{t\geq T_0+1}$. Thus by the above reasoning, $\sum_{k=1}^{m_t}\mathbb{I}\{i_t\in \tilde{V}^{\prime}_{t,k}\}\norm{\bx_{a_t}}_{\overline{\bM}_{\overline{V}^{\prime}_{t,k},t-1}^{-1}}\leq \sum_{j=1}^{m}\mathbb{I}\{i_t\in V_j\}\norm{\bx_{a_t}}_{\overline{\bM}_{V_j,t-1}^{-1}}, \forall{t\geq T_0+1}$.
Eq.(\ref{cauchy1}) holds by the Cauchy–Schwarz inequality; Eq.(\ref{lemma7 use1}) follows by the following technical Lemma \ref{technical lemma6}. Eq.(\ref{cauchy2}) is from the Cauchy–Schwarz inequality and the fact that $\sum_{j=1}^{m}T_{V_j,T}=T$.

We then bound the last term in Eq.(\ref{sum C}):

    \begin{align}
    &\mathbb{I}\{\cE\}\epsilon_*\sum_{t=T_0+1}^{T}\sum_{s\in[t-1]\atop i_s\in \overline{V}_t}\left|\bx_{a_t}^{\top}\overline{\bM}_{\overline{V}_t,t-1}^{-1}\bx_{a_s}\right|
    \notag\\&=\mathbb{I}\{\cE\}\epsilon_*\sum_{t=T_0+1}^{T}\sum_{k=1}^{m_t}\mathbb{I}\{i_t\in \tilde{V}^{\prime}_{t,k}\}\sum_{s\in[t-1]\atop i_s\in \overline{V}^{\prime}_{t,k}}\left|\bx_{a_t}^{\top}\overline{\bM}_{\overline{V}^{\prime}_{t,k},t-1}^{-1}\bx_{a_s}\right|\nonumber\\
    &\leq \mathbb{I}\{\cE\}\epsilon_*\sum_{t=T_0+1}^{T}\sum_{k=1}^{m_t}\mathbb{I}\{i_t\in \tilde{V}^{\prime}_{t,k}\}\sqrt{\sum_{s\in[t-1]\atop i_s\in \overline{V}^{\prime}_{t,k}}1\sum_{s\in[t-1]\atop i_s\in \overline{V}^{\prime}_{t,k}}\left|\bx_{a_t}^{\top}\overline{\bM}_{\overline{V}^{\prime}_{t,k},t-1}^{-1}\bx_{a_s}\right|^2}\label{cauchy 3}\\
    &\leq \mathbb{I}\{\cE\}\epsilon_*\sum_{t=T_0+1}^{T}\sum_{k=1}^{m_t}\mathbb{I}\{i_t\in \tilde{V}^{\prime}_{t,k}\}\sqrt{T_{\overline{V}^{\prime}_{t,k},t-1}\norm{\bx_{a_t}}_{\overline{\bM}_{\overline{V}^{\prime}_{t,k},t-1}^{-1}}^2}\label{psd}\\
    &\leq \mathbb{I}\{\cE\}\epsilon_*\sum_{t=T_0+1}^{T}\sqrt{\sum_{k=1}^{m_t}\mathbb{I}\{i_t\in \tilde{V}^{\prime}_{t,k}\}\sum_{k=1}^{m_t}\mathbb{I}\{i_t\in \tilde{V}^{\prime}_{t,k}\}T_{\overline{V}^{\prime}_{t,k},t-1}\norm{\bx_{a_t}}_{\overline{\bM}_{\overline{V}^{\prime}_{t,k},t-1}^{-1}}^2}\label{cauchy 4}\\
    % &\leq \mathbb{I}\{\cE\}\epsilon_*\sum_{t=T_0+1}^{T} \sqrt{1\times (t-1)\sum_{k=1}^{m_t}\mathbb{I}\{i_t\in \overline{V}^{\prime}_{t,k}\}\norm{\bx_{a_t}}_{\overline{\bM}_{\overline{V}^{\prime}_{t,k},t-1}^{-1}}^2}\label{bound T_V,t by t}\\
    &\leq \mathbb{I}\{\cE\}\epsilon_*\sqrt{T}\sum_{t=T_0+1}^{T} \sqrt{\sum_{k=1}^{m_t}\mathbb{I}\{i_t\in \tilde{V}^{\prime}_{t,k}\}\norm{\bx_{a_t}}_{\overline{\bM}_{\overline{V}^{\prime}_{t,k},t-1}^{-1}}^2}\label{bound t by T}\\
    &\leq \mathbb{I}\{\cE\}\epsilon_*\sqrt{T}\sqrt{\sum_{t=T_0+1}^{T}1\sum_{t=T_0+1}^{T}\sum_{k=1}^{m_t}\mathbb{I}\{i_t\in \tilde{V}^{\prime}_{t,k}\}\norm{\bx_{a_t}}_{\overline{\bM}_{\overline{V}^{\prime}_{t,k},t-1}^{-1}}^2}\label{cauchy 5}\\
    &\leq \mathbb{I}\{\cE\}\epsilon_*\sqrt{T}\sqrt{T\sum_{t=T_0+1}^{T}\sum_{j=1}^{m}\mathbb{I}\{i_t\in V_j\}\norm{\bx_{a_t}}_{\overline{\bM}_{V_j,t-1}^{-1}}^2}\label{cluster inequality 1}\\
    &=\mathbb{I}\{\cE\}\epsilon_*T\sqrt{\sum_{j=1}^{m}\sum_{t=T_0+1}^{T}\mathbb{I}\{i_t\in V_j\}\norm{\bx_{a_t}}_{\overline{\bM}_{V_j,t-1}^{-1}}^2}\notag\\
    &\leq \mathbb{I}\{\cE\}\epsilon_*T\sqrt{2md\log (1+\frac{T}{\lambda d})}\label{lemma7 use2}\,,
\end{align}

where Eq.(\ref{cauchy 3}), Eq.(\ref{cauchy 4}) and Eq.(\ref{cauchy 5}) hold because of the Cauchy–Schwarz inequality, Eq.(\ref{psd}) holds since $\overline{\bM}_{\overline{V}^{\prime}_{t,k},t-1}\succeq\sum_{s\in[t-1]\atop i_s\in \overline{V}^{\prime}_{t,k}}\bx_{a_s}\bx_{a_s}^{\top}$, Eq.(\ref{bound t by T}) is because $T_{\overline{V}^{\prime}_{t,k},t-1}\leq T$, Eq. (\ref{cluster inequality 1}) follows from the same reasoning as Eq.(\ref{cluster inequality}), and Eq.(\ref{lemma7 use2}) comes from the following technical Lemma \ref{technical lemma6}.

Finally, plugging Eq.(\ref{cauchy2}) and Eq.(\ref{lemma7 use2}) into Eq.(\ref{sum C}), take expectation and plug it into Eq.(\ref{regret parts}), we can get:
\begin{align}
R(T) \le & 5+T_0+\frac{\tilde{u}}{u}\times\frac{2\epsilon_*\sqrt{2d}T}{\tilde{\lambda}_x^{\frac{3}{2}}}+2\epsilon_*T\bigg(1+\sqrt{2md\log(1+\frac{T}{\lambda d})}\bigg)\notag\\
&+2\bigg(\sqrt{\lambda}+\sqrt{2\log(T)+d\log(1+\frac{T}{\lambda d})}\bigg)\times\sqrt{2mdT\log(1+\frac{T}{\lambda d})}\,,
\end{align}
where
\begin{equation*}
        T_0= 16u\log(\frac{u}{\delta})+4u\max\max\{
        \frac{8d}{\tilde{\lambda}_x(\frac{\gamma_1}{4}-\epsilon_*\sqrt{\frac{1}{2\tilde{\lambda}_x}})^2}\log(\frac{u}{\delta}),\frac{16}{\tilde{\lambda}_x^2}\log(\frac{8d}{\tilde{\lambda}_x^2\delta})\}
\end{equation*}
is given in the following Lemma \ref{sufficient time} in Appendix \ref{section T0}.
\subsection{Proof and Discussions of Theorem \ref{thm: lower bound}}\label{appendix: lower bound}
In the work \cite{lattimore2020learning}, they give a lower bound for misspecified linear bandits with a single user. The lower bound of $R(T)$ is given by:
$R_3(T)\geq \epsilon_*T\sqrt{d}$. Therefore, suppose our problem with multiple users and $m$ underlying clusters where the arrival times are $T_i$ for each cluster, then for any algorithms, even if they know the underlying clustering structure and keep $m$ independent linear bandit algorithms to leverage the common information of clusters, the best they can get is $R(T)=\sum_{i \in [m]}R_3(T_i)\geq \epsilon_*\sum_{i \in [m]}T_i\sqrt{d}=\epsilon_*T\sqrt{d}$, which gives a lower bound of $O(\epsilon_*T\sqrt{d})$ for the CBMUM problem. Recall that the regret upper bound of our algorithms is of $O(\epsilon_*T\sqrt{md\log T}  + d\sqrt{mT}\log T)$, asymptotically matching this lower bound with respect to $T$ up to logarithmic factors and with respect to $m$ up to $O(\sqrt{m})$ factors, showing the tightness of our theoretical results (where 
 $m$ are typically very small for real applications).
 
 We conjecture that the gap for the $m$ factor is due to the strong assumption that cluster structures are known to prove our lower bound, and whether there exists a tighter lower bound will be left for future work.

\subsection{Proof of the key Lemma \ref{bound for mis}}\label{key lemma appendix}
% \subsubsection{Proof of Lemma \ref{bound for mis}}\label{proof of key lemma}
In Lemma \ref{bound for mis}, we want to bound $\left|\bx_a^{\top}\overline{\bM}_{\overline{V}_t,t-1}^{-1}\sum_{s\in[t-1]\atop i_s\in \overline{V}_t}\bx_{a_s}\bx_{a_s}^{\top}(\btheta_{i_s}-\btheta_{i_t})\right|$. By the definition of ``good partition", we have  $\norm{\btheta_{i_s}-\btheta_{i_t}}_2\leq\zeta\,,\forall{i_s\in \overline{V}_t}$. It is an easy-to-be-made mistake to directly drag $\norm{\btheta_{i_s}-\btheta_{i_t}}_2$ out to upper bound it by $\norm{\bx_a^{\top}\overline{\bM}_{\overline{V}_t,t-1}^{-1}\sum_{s\in[t-1]\atop i_s\in \overline{V}_t}\bx_{a_s}\bx_{a_s}^{\top}}_2\times\zeta$ and then proceed. We need more careful analysis.

We first prove the following general lemma.
\begin{lemma}
\label{technical lemma 5}
For vectors $\bx_1,\bx_2,\ldots,\bx_k\in\RR^d$,$\norm{\bx_i}_2\leq1,\forall{i\in[k]}$, and vectors $\btheta_1,\btheta_2,\ldots,\btheta_k\in\RR^d,\norm{\btheta_i}_2\leq C,\forall{i\in[k]}$, where $C>0$ is a constant, we have:
\begin{equation}
    \norm{\sum_{i=1}^{k}\bx_i\bx_i^{\top}\btheta_i}_2\leq C\sqrt{d}\norm{\sum_{i=1}^{k}\bx_i\bx_i^{\top}}_2\,.\nonumber
\end{equation}
\end{lemma}
\newpage
\begin{proof}
Let $\bX\in\RR^{d\times k}$ be a matrix such that it has $\bx_i$ s as its columns, i.e., $\bX=[\bx_1,\ldots,\bx_k]=\begin{bmatrix}  
  \bx_{11}& x_{21}& \cdots  & \bx_{k1} \\  
  \bx_{12}& x_{22}& \cdots  & \bx_{k2} \\  
  \vdots & \vdots & \ddots & \vdots \\  
  \bx_{1d}& x_{2d}& \cdots  & \bx_{kd}  
\end{bmatrix}.$

Let $\by\in\RR^{k\times1}$ be a vector that has $\bx_i^{\top}\btheta_i$ s as its elements, i.e., $\by=[\bx_1^{\top}\btheta_1,\ldots,\bx_k^{\top}\btheta_k]^{\top}$. Then we have:
\begin{align}
    \norm{\sum_{i=1}^{k}\bx_i\bx_i^{\top}\btheta_i}_2^2=\norm{\bX\by}_2^2
    &\leq \norm{\bX}_2^2\norm{\by}_2^2\label{operator norm 4}\\
    &=\norm{\bX}_2^2\sum_{i=1}^{k}(\bx_i^{\top}\btheta_i)^2\nonumber\\
    &\leq\norm{\bX}_2^2\sum_{i=1}^{k}\norm{\bx_i}_2^2\norm{\btheta_i}_2^2\label{cauchy 11}\\
    &\leq C^2\norm{\bX}_2^2\sum_{i=1}^{k}\norm{\bx_i}_2^2\nonumber\\
    &=C^2\norm{\bX}_2^2\norm{\bX}_F^2\nonumber\\
    &\leq C^2d\norm{\bX}_2^4\label{matrix norm inequality}\\
    &=C^2d\norm{\bX\bX^{\top}}_2^2\label{matrix norm equality}\\
    &=C^2d\norm{\sum_{i=1}^{k}\bx_i\bx_i^{\top}}_2^2\,,
\end{align}
where Eq. (\ref{operator norm 4}) follows by the matrix operator norm inequality, Eq. (\ref{cauchy 11}) follows by the Cauchy–Schwarz inequality, Eq. (\ref{matrix norm inequality}) follows by $\norm{\bX}_F\leq\sqrt{d}\norm{\bX}_2$, Eq. (\ref{matrix norm equality}) follows from $\norm{\bX}_2^2=\norm{\bX\bX^{\top}}_2$.
\end{proof}
The above result is tight. We can show that the lower bound of $\norm{\sum_{i=1}^{k}\bx_i\bx_i^{\top}\btheta_i}_2$ under the conditions in the lemma is exactly $C\sqrt{d}\norm{\sum_{i=1}^{k}\bx_i\bx_i^{\top}}_2$. Specifically, let $k=2$, $C=1$, $d=2$, $\bx_1=[0,1]^\top$, $\bx_2=[1,0]^\top$, $\btheta_1=[1,0]^\top$, $\btheta_2=[0,1]^\top$, then we have $\norm{\sum_{i=1}^{2}\bx_i\bx_i^{\top}\btheta_i}_2=\norm{[1,1]^\top}_2=\sqrt{2}$, and $C\sqrt{d}\norm{\sum_{i=1}^{2}\bx_i\bx_i^{\top}}_2=1\times\sqrt{2}\times\norm{\begin{bmatrix}  
  1 & 0 \\  
  0 & 1  
\end{bmatrix}}_2=\sqrt{2}$. Therefore, we have that the upper bound given in Lemma \ref{technical lemma 5} matches the lower bound.

We are now ready to prove the key Lemma \ref{bound for mis} with the above Lemma \ref{technical lemma 5}.
% \begin{lemma}
% \label{bound for mis}
% $\forall{t\geq T_0}$, if the current partition is a ``good partition", and $\overline{V}_t\notin V$, then for all $\bx_a\in \RR^d, \norm{\bx_a}_2\leq 1$, with probability at least $1-\delta$:
% \begin{equation}
%     \left|\bx_a^{\top}\overline{\bM}_{\overline{V}_t,t-1}^{-1}\sum_{s\in[t-1]\atop i_s\in \overline{V}_t}\bx_{a_s}\bx_{a_s}^{\top}(\btheta_{i_s}-\btheta_{i_t})\right|\leq \frac{8\epsilon_*\sqrt{2d}}{\lambda_x^{\frac{3}{2}}}\nonumber\,.
% \end{equation}

% \end{lemma}

At any $t\geq T_0$, if the current partition is a ``good partition", and $\overline{V}_t\notin\mathcal{V}$, then for all $\bx_a\in \RR^d, \norm{\bx_a}_2\leq 1$, with probability at least $1-\delta$:
\begin{align}
    &\left|\bx_a^{\top}\overline{\bM}_{\overline{V}_t,t-1}^{-1}\sum_{s\in[t-1]\atop i_s\in \overline{V}_t}\bx_{a_s}\bx_{a_s}^{\top}(\btheta_{i_s}-\btheta_{i_t})\right|\notag\\
    &\leq \norm{\bx_a}_2\norm{\overline{\bM}_{\overline{V}_t,t-1}^{-1}\sum_{s\in[t-1]\atop i_s\in \overline{V}_t}\bx_{a_s}\bx_{a_s}^{\top}(\btheta_{i_s}-\btheta_{i_t})}_2\label{cauchy 10}\\
    &\leq \norm{\overline{\bM}_{\overline{V}_t,t-1}^{-1}}_2\norm{\sum_{s\in[t-1]\atop i_s\in \overline{V}_t}\bx_{a_s}\bx_{a_s}^{\top}(\btheta_{i_s}-\btheta_{i_t})}_2\label{operator norm inequality 2}\\
    &\leq 2\epsilon_*\sqrt{\frac{2d}{\tilde{\lambda}_x}}\times\norm{\overline{\bM}_{\overline{V}_t,t-1}^{-1}}_2\norm{\sum_{s\in[t-1]\atop i_s\in \overline{V}_t}\bx_{a_s}\bx_{a_s}^{\top}}_2\label{technical lemma}\\
    &\leq2\epsilon_*\sqrt{\frac{2d}{\tilde{\lambda}_x}}\times\frac{\lambda_{max}(\sum_{s\in[t-1]\atop i_s\in \overline{V}_t}\bx_{a_s}\bx_{a_s}^{\top})}{\lambda_{\text{min}}(\overline{\bM}_{\overline{V}_t,t-1})}\nonumber\\
    &\leq 2\epsilon_*\sqrt{\frac{2d}{\tilde{\lambda}_x}}\times\frac{T_{\overline{V}_t,t-1}}{2T_{\overline{V}_t,t-1}\tilde{\lambda}_x+\lambda}\label{min eigen assumption}\\
    &\leq \frac{\epsilon_*\sqrt{2d}}{\tilde{\lambda}_x^{\frac{3}{2}}}\nonumber\,,
\end{align}
where Eq.(\ref{cauchy 10}) follows by the Cauchy–Schwarz inequality, Eq.(\ref{operator norm inequality 2}) follows from the inequality of matrix's operator norm, Eq.(\ref{technical lemma}) follows from the fact that in a ``good partition", $\norm{\btheta_{i_t}-\btheta_l}_2\leq2\epsilon_*\sqrt{\frac{2}{\tilde{\lambda}_x}},\forall{l\in \overline{V}_t}$ and Lemma \ref{technical lemma 5}, Eq.(\ref{min eigen assumption}) follows by Eq.(\ref{min eigen}) with probability $\geq 1-\delta$.

\subsection{Lemma \ref{sufficient time} of the sufficient time $T_0$ and its proof}
The following lemma gives a sufficient time $T_0$ for the algorithm to get a ``good partition".\label{section T0}
\begin{lemma}
\label{sufficient time}
With the carefully designed edge deletion rule, after 
\begin{equation*}
    \begin{aligned}
        T_0&\triangleq 16u\log(\frac{u}{\delta})+4u\max\max\{
        \frac{8d}{\tilde{\lambda}_x(\frac{\gamma_1}{4}-\epsilon_*\sqrt{\frac{1}{2\tilde{\lambda}_x}})^2}\log(\frac{u}{\delta}),\frac{16}{\tilde{\lambda}_x^2}\log(\frac{8d}{\tilde{\lambda}_x^2\delta})\}\\
        &=O\bigg(u\left( \frac{d}{\tilde{\lambda}_x (\gamma_1-\zeta)^2}+\frac{1}{\tilde{\lambda}_x^2}\right)\log \frac{1}{\delta}\bigg)
    \end{aligned}
\end{equation*}
rounds, with probability at least $1-3\delta$ for some $\delta\in(0,\frac{1}{3})$, RCLUMB can always get a ``good partition".
\end{lemma} 

Below is the detailed proof of Lemma \ref{sufficient time}.
\begin{proof}
We first prove the following result:\\
With probability at least $1-\delta$ for some $\delta\in(0,1)$, at any $t\in[T]$:
\begin{equation}
    \norm{\hat{\btheta}_{i,t}-\btheta^{j(i)}}_2\leq\frac{\beta(T_{i,t},\frac{\delta}{u})+\epsilon_*\sqrt{T_{i,t}}}{\sqrt{\lambda+\lambda_{\text{min}}(\bM_{i,t})}}, \forall{i\in\mathcal{U}}\label{norm difference bound}\,,
\end{equation}
where $\beta(T_{i,t},\frac{\delta}{u})\triangleq\sqrt{\lambda}+\sqrt{2\log(\frac{u}{\delta})+d\log(1+\frac{T_{i,t}}{\lambda d})}$.
\begin{align}
    \hat{\btheta}_{i,t}-\btheta^{j(i)}&=(\sum_{s\in[t]\atop i_s=i}\bx_{a_s}\bx_{a_s}^{\top}+\lambda \bI)^{-1}\bigg(\sum_{s\in[t]\atop i_s=i}\bx_{a_s}(\bx_{a_s}^{\top}\btheta^{j(i)}+\bepsilon_{a_s}^{i_s,s}+\eta_s)\bigg)-\btheta^{j(i)}\label{by definition}\\
    &=(\sum_{s\in[t]\atop i_s=i}\bx_{a_s}\bx_{a_s}^{\top}+\lambda \bI)^{-1}[(\sum_{s\in[t]\atop i_s=i}\bx_{a_s}\bx_{a_s}^{\top}+\lambda\bI)\btheta^{j(i)}-\lambda\btheta^{j(i)}+\sum_{s\in[t]\atop i_s=i}\bx_{a_s}\bepsilon_{a_s}^{i_s,s}\notag\\
    &+\sum_{s\in[t]\atop i_s=i}\bx_{a_s}\eta_s]-\btheta^{j(i)}\nonumber\\
    &=-\lambda\tilde{\bM}_{i,t}^{-1}\btheta^{j(i)}+\tilde{\bM}_{i,t}^{-1}\sum_{s\in[t]\atop i_s=i}\bx_{a_s}\bepsilon_{a_s}^{i_s,s}+\tilde{\bM}_{i,t}^{-1}\sum_{s\in[t]\atop i_s=i}\bx_{a_s}\eta_s\nonumber\,,
\end{align}
where we denote $\tilde{\bM}_{i,t}=\bM_{i,t}+\lambda\bI$, and Eq.(\ref{by definition}) holds by definition.

Therefore, \begin{equation}
    \norm{\hat{\btheta}_{i,t}-\btheta^{j(i)}}_2\leq\lambda\norm{\tilde{\bM}_{i,t}^{-1}\btheta^{j(i)}}_2+\norm{\tilde{\bM}_{i,t}^{-1}\sum_{s\in[t]\atop i_s=i}\bx_{a_s}\bepsilon_{a_s}^{i_s,s}}_2+\norm{\tilde{\bM}_{i,t}^{-1}\sum_{s\in[t]\atop i_s=i}\bx_{a_s}\eta_s}_2\,.\label{difference bound parts}
\end{equation}

We then bound the three terms in Eq.(\ref{difference bound parts}) one by one. For the first term:
\begin{equation}
    \lambda\norm{\tilde{\bM}_{i,t}^{-1}\btheta^{j(i)}}_2\leq \lambda\norm{\tilde{\bM}_{i,t}^{-\frac{1}{2}}}_2^2\norm{\btheta^{j(i)}}_2\leq\frac{\sqrt{\lambda}}{\sqrt{\lambda_{\text{min}}(\tilde{\bM}_{i,t})}}\label{first term}\,,
\end{equation}
where we use the Cauchy–Schwarz inequality, the inequality for the operator norm of matrices, and the fact that $\lambda_{\text{min}}(\tilde{\bM}_{i,t})\geq\lambda$.

For the second term in Eq.(\ref{difference bound parts}):
\begin{align}
    \norm{\tilde{\bM}_{i,t}^{-1}\sum_{s\in[t]\atop i_s=i}\bx_{a_s}\bepsilon_{a_s}^{i_s,s}}_2
    &=\max_{\bx\in S^{d-1}}\sum_{s\in[t]\atop i_s=i}\bx^{\top}\tilde{\bM}_{i,t}^{-1}\bx_{a_s}\bepsilon_{a_s}^{i_s,s}\nonumber\\
    &\leq \max_{\bx\in S^{d-1}}\sum_{s\in[t]\atop i_s=i}\left|\bx^{\top}\tilde{\bM}_{i,t}^{-1}\bx_{a_s}\bepsilon_{a_s}^{i_s,s}\right|\nonumber\\
    &\leq \max_{\bx\in S^{d-1}}\sum_{s\in[t]\atop i_s=i}\left|\bx^{\top}\tilde{\bM}_{i,t}^{-1}\bx_{a_s}\right|\norm{\bepsilon_{a_s}^{i_s,s}}_{\infty}\label{holders1}\\
    &\leq \epsilon_*\max_{\bx\in S^{d-1}}\sum_{s\in[t]\atop i_s=i}\left|\bx^{\top}\tilde{\bM}_{i,t}^{-1}\bx_{a_s}\right|\nonumber\\
    &\leq\epsilon_*\max_{\bx\in S^{d-1}}\sqrt{\sum_{s\in[t]\atop i_s=i}1\sum_{s\in[t]\atop i_s=i}\left|\bx^{\top}\tilde{\bM}_{i,t}^{-1}\bx_{a_s}\right|^2}\label{cauchy 6}\\
    &\leq \epsilon_*\sqrt{T_{i,t}}\sqrt{\max_{\bx\in S^{d-1}} \bx^{\top}\tilde{\bM}_{i,t}^{-1}\bx}\label{psd1}\\
    &=\frac{\epsilon_*\sqrt{T_{i,t}}}{\sqrt{\lambda_{\text{min}}(\tilde{\bM}_{i,t})}}\label{Courant–Fischer}\,,
\end{align}
where we denote $S^{d-1}=\{\bx\in\RR^d:\norm{\bx}_2=1\}$, Eq.(\ref{holders1}) follows from Holder's inequality, Eq.(\ref{cauchy 6}) follows by the Cauchy–Schwarz inequality, Eq.(\ref{psd1}) holds because $\tilde{\bM}_{i,t}\succeq\sum_{s\in[t]\atop i_s=i}\bx_{a_s}\bx_{a_s}^{\top}$, Eq.(\ref{Courant–Fischer}) follows from the Courant-Fischer theorem.

For the last term in Eq.(\ref{difference bound parts})
\begin{align}
    \norm{\tilde{\bM}_{i,t}^{-1}\sum_{s\in[t]\atop i_s=i}\bx_{a_s}\eta_s}_2
    &\leq\norm{\tilde{\bM}_{i,t}^{-\frac{1}{2}}\sum_{s\in[t]\atop i_s=i}\bx_{a_s}\eta_s}_2\norm{\tilde{\bM}_{i,t}^{-\frac{1}{2}}}_2\label{operator norm}\\
    &=\frac{\norm{\sum_{s\in[t]\atop i_s=i}\bx_{a_s}\eta_s}_{\tilde{\bM}_{i,t}^{-1}}}{\sqrt{\lambda_{\text{min}}(\tilde{\bM}_{i,t})}}\label{Courant–Fischer1}\,,
\end{align}
where Eq.(\ref{operator norm}) follows by the Cauchy–Schwarz inequality and the inequality for the operator norm of matrices, and Eq.(\ref{Courant–Fischer1}) follows by the Courant-Fischer theorem.

Following Theorem 1 in \cite{abbasi2011improved}, with probability at least $1-\delta$ for some $\delta\in(0,1)$, for any $i\in\mathcal{U}$, we have:
\begin{align}
    \norm{\sum_{s\in[t]\atop i_s=i}\bx_{a_s}\eta_s}_{\tilde{\bM}_{i,t}^{-1}}
    &\leq\sqrt{2\log(\frac{u}{\delta})+\log(\frac{\text{det}(\tilde{\bM}_{i,t})}{\text{det}(\lambda\bI)})}\nonumber\\
    &\leq \sqrt{2\log(\frac{u}{\delta})+d\log(1+\frac{T_{i,t}}{\lambda d})}\label{det inequality}\,,
\end{align}
where $\text{det}(\bM)$ denotes the determinant of matrix $\bM$, Eq.(\ref{det inequality}) is because $\text{det}(\tilde{\bM}_{i,t})\leq\Bigg(\frac{\text{trace}(\lambda\bI+\sum_{s\in[t]\atop i_s=i}\bx_{a_s}\bx_{a_s}^{\top})}{d}\Bigg)^d\leq\big(\frac{\lambda d+T_{i,t}}{d}\big)^d$, and $\text{det}(\lambda\bI)=\lambda^d$.

Plugging Eq.(\ref{det inequality}) into Eq. (\ref{Courant–Fischer1}), then plugging Eq. (\ref{first term}), Eq.(\ref{Courant–Fischer}) and Eq.(\ref{Courant–Fischer1}) into Eq.(\ref{difference bound parts}), we can get that Eq.(\ref{norm difference bound}) holds with probability $\geq 1-\delta$.

Then, with the item regularity assumption stated in Assumption \ref{assumption3}, the technical Lemma \ref{assumption}, together with Lemma 7 in \cite{li2018online}, with probability at least $1-\delta$, for a particular user $i$, at any $t$ such that $T_{i,t}\geq\frac{16}{\tilde{\lambda}_x^2}\log(\frac{8d}{\tilde{\lambda}_x^2\delta})$, we have:
\begin{equation}
    \lambda_{\text{min}}(\tilde{\bM}_{i,t})\geq2\tilde{\lambda}_x T_{i,t}+\lambda\,.
    \label{min eigen}
\end{equation}

Based on the above reasoning, we have: if $T_{i,t}\geq\frac{16}{\tilde{\lambda}_x^2}\log(\frac{8d}{\tilde{\lambda}_x^2\delta})$, then with probability $\geq 1-2\delta$, we have:
\begin{align}
    \norm{\hat{\btheta}_{i,t}-\btheta^{j(i)}}_2
    &\leq\frac{\beta(T_{i,t},\frac{\delta}{u})+\epsilon_*\sqrt{T_{i,t}}}{\sqrt{\lambda_{\text{min}}(\tilde{\bM}_{i,t})}}\nonumber\\
    &\leq\frac{\beta(T_{i,t},\frac{\delta}{u})+\epsilon_*\sqrt{T_{i,t}}}{\sqrt{2\tilde{\lambda}_x T_{i,t}+\lambda}}\nonumber\\
    &\leq\frac{\sqrt{\lambda}+\sqrt{2\log(\frac{u}{\delta})+d\log(1+\frac{T_{i,t}}{\lambda d})}}{\sqrt{2\tilde{\lambda}_x T_{i,t}+\lambda}}+\epsilon_*\sqrt{\frac{1}{2\tilde{\lambda}_x}}\,,
\end{align}
for any $i\in\mathcal{U}$.

Let 
\begin{equation}
    \frac{\sqrt{\lambda}+\sqrt{2\log(\frac{u}{\delta})+d\log(1+\frac{T_{i,t}}{\lambda d})}}{\sqrt{2\tilde{\lambda}_x T_{i,t}+\lambda}}+\epsilon_*\sqrt{\frac{1}{2\tilde{\lambda}_x}}<\frac{\gamma_1}{4}\,,
    \label{init condition}
\end{equation}
which is equivalent to
\begin{equation}
    \frac{\sqrt{\lambda}+\sqrt{2\log(\frac{u}{\delta})+d\log(1+\frac{T_{i,t}}{\lambda d})}}{\sqrt{2\tilde{\lambda}_x T_{i,t}+\lambda}}<\frac{\gamma_1}{4}-\epsilon_*\sqrt{\frac{1}{2\tilde{\lambda}_x}}\label{condition gamma/4}\,,
\end{equation}
where $\gamma_1$ is given in Definition \ref{def:gap}.

Assume $\lambda\leq2\log(\frac{u}{\delta})+d\log(1+\frac{T_{i,t}}{\lambda d})$, which is typically held, then a sufficient condition for Eq. (\ref{condition gamma/4}) is:
\begin{equation}
    \frac{2\log(\frac{u}{\delta})+d\log(1+\frac{T_{i,t}}{\lambda d})}{2\tilde{\lambda}_x T_{i,t}}<\frac{1}{4}(\frac{\gamma_1}{4}-\epsilon_*\sqrt{\frac{1}{2\tilde{\lambda}_x}})^2\,.
    \label{condition}
\end{equation}
To satisfy the condition in Eq.(\ref{condition}), it is sufficient to show
\begin{equation}
    \frac{2\log(\frac{u}{\delta})}{2\tilde{\lambda}_x T_{i,t}}
    <\frac{1}{8}(\frac{\gamma_1}{4}-\epsilon_*\sqrt{\frac{1}{2\tilde{\lambda}_x}})^2\label{condition1}
\end{equation}
and \begin{equation}
    \frac{d\log(1+\frac{T_{i,t}}{\lambda d})}{2\tilde{\lambda}_x T_{i,t}}<\frac{1}{8}(\frac{\gamma_1}{4}-\epsilon_*\sqrt{\frac{1}{2\tilde{\lambda}_x}})^2\label{condition2}\,.
\end{equation}

From Eq.(\ref{condition1}), we can get:
\begin{equation}
    T_{i,t}\geq\frac{8\log(\frac{u}{\delta})}{\tilde{\lambda}_x(\frac{\gamma_1}{4}-\epsilon_*\sqrt{\frac{1}{2\tilde{\lambda}_x}})^2}\,.
\end{equation}

Following Lemma 9 in \cite{li2018online}, we can get the following sufficient condition for Eq.(\ref{condition2}):
\begin{equation}
    T_{i,t}\geq\frac{8d\log(\frac{4}{\lambda\tilde{\lambda}_x(\frac{\gamma_1}{4}-\epsilon_*\sqrt{\frac{1}{2\tilde{\lambda}_x}})^2})}{\tilde{\lambda}_x(\frac{\gamma_1}{4}-\epsilon_*\sqrt{\frac{1}{2\tilde{\lambda}_x}})^2}\,.
\end{equation}

Assume $\frac{u}{\delta}\geq\frac{4}{\lambda\tilde{\lambda}_x(\frac{\gamma_1}{4}-\epsilon_*\sqrt{\frac{1}{2\tilde{\lambda}_x}})^2}$, which is typically held, we can get that
\begin{equation}
    T_{i,t}\geq\frac{8d}{\tilde{\lambda}_x(\frac{\gamma_1}{4}-\epsilon_*\sqrt{\frac{1}{2\tilde{\lambda}_x}})^2}\log(\frac{u}{\delta})
\end{equation}
is a sufficient condition for Eq.(\ref{init condition}). Together with the condition that $T_{i,t}\geq\frac{16}{\tilde{\lambda}_x^2}\log(\frac{8d}{\tilde{\lambda}_x^2\delta})$, we can get that if
\begin{equation}
  T_{i,t}\geq\max\{
        \frac{8d}{\tilde{\lambda}_x(\frac{\gamma_1}{4}-\epsilon_*\sqrt{\frac{1}{2\tilde{\lambda}_x}})^2}\log(\frac{u}{\delta}),\frac{16}{\tilde{\lambda}_x^2}\log(\frac{8d}{\tilde{\lambda}_x^2\delta})\},\forall{i\in\mathcal{U}}\label{second last condition}\,, 
\end{equation}
then with probability $\geq 1-2\delta$:
\begin{equation*}
    \norm{\hat{\btheta}_{i,t}-\btheta^{j(i)}}_2<\frac{\gamma_1}{4}\,,\forall{i\in\mathcal{U}}\,.
\end{equation*}

By Lemma 8 in \cite{li2018online}, and Assumption \ref{assumption2} of user arrival uniformness, we have that for all
\begin{equation}
    t\geq T_0\triangleq16u\log(\frac{u}{\delta})+4u\max\{
        \frac{8d}{\tilde{\lambda}_x(\frac{\gamma_1}{4}-\epsilon_*\sqrt{\frac{1}{2\tilde{\lambda}_x}})^2}\log(\frac{u}{\delta}),\frac{16}{\tilde{\lambda}_x^2}\log(\frac{8d}{\tilde{\lambda}_x^2\delta})\}\,,
\end{equation}
with probability at least $1-\delta$, condition in Eq.(\ref{second last condition}) is satisfied.

Therefore we have that for all $t\geq T_0$, with probability $\geq 1-3\delta$:
\begin{equation}
    \norm{\hat{\btheta}_{i,t}-\btheta^{j(i)}}_2<\frac{\gamma_1}{4}\,,\forall{i\in\mathcal{U}}\,.
    \label{final condition}
\end{equation}

Next, we show that with Eq.(\ref{final condition}), we can get that the RCLUMB keeps a ``good partition". First, if we delete the edge $(i,l)$, then user $i$ and user $j$ belong to different \gtclusters{}, i.e., $\norm{\btheta_i-\btheta_l}_2>0$. This is because by the deletion rule of the algorithm, the concentration bound, and triangle inequality,
$\norm{\btheta_i-\btheta_l}_2=\norm{\btheta^{j(i)}-\btheta^{j(l)}}_2\geq\norm{\hat{\btheta}_{i,t}-\hat{\btheta}_{l,t}}_2-\norm{\btheta^{j(l)}-\btheta_{l,t}}_2-\norm{\btheta^{j(i)}-\btheta_{i,t}}_2>0$. Second, we show that if $\norm{\btheta_i-\btheta_l}\geq\gamma_1>2\epsilon_*\sqrt{\frac{2}{\tilde{\lambda}_x}}$, the RCLUMB algorithm will delete the edge $(i,l)$. This is because if $\norm{\btheta_i-\btheta_l}\geq\gamma_1$, then by the triangle inequality, and $\norm{\hat{\btheta}_{i,t}-\btheta^{j(i)}}_2<\frac{\gamma_1}{4}$, $\norm{\hat{\btheta}_{l,t}-\btheta^{j(l)}}_2<\frac{\gamma_1}{4}$, $\btheta_i=\btheta^{j(i)}$, $\btheta_l=\btheta^{j(l)}$, we have $\norm{\hat{\btheta}_{i,t}-\hat{\btheta}_{l,t}}_2\geq\norm{\btheta_i-\btheta_l}-\norm{\hat{\btheta}_{i,t}-\btheta^{j(i)}}_2-\norm{\hat{\btheta}_{l,t}-\btheta^{j(l)}}_2>\gamma_1-\frac{\gamma_1}{4}-\frac{\gamma_1}{4}=\frac{\gamma_1}{2}>\frac{\sqrt{\lambda}+\sqrt{2\log(\frac{u}{\delta})+d\log(1+\frac{T_{i,t}}{\lambda d})}}{\sqrt{\lambda+2\tilde{\lambda}_x T_{i,t}}}+\epsilon_*\sqrt{\frac{1}{2\tilde{\lambda}_x}}+\frac{\sqrt{\lambda}+\sqrt{2\log(\frac{u}{\delta})+d\log(1+\frac{T_{l,t}}{\lambda d})}}{\sqrt{\lambda+2\tilde{\lambda}_x T_{l,t}}}+\epsilon_*\sqrt{\frac{1}{2\tilde{\lambda}_x}}$, which will trigger the deletion condition Line \ref{alg:RCLUMB:delete} in Algo.\ref{alg:RCLUMB}.

From the above reasoning, we can get that at round $t$, any user within $\overline{V}_t$ is $\zeta$-close to $i_t$, and all the users belonging to $V_{j(i)}$ are contained in $\overline{V}_t$, which means the algorithm has done a ``good partition" at $t$ by Definition \ref{def:good partition}.
\end{proof}
\subsection{Proof of Lemma \ref{concentration bound}}
We prove the result in two situations: when $\overline{V}_t\in \mathcal{V}$ and when $\overline{V}_t\notin \mathcal{V}$.

(1) Situation 1: for any $t\geq T_0$ and $\overline{V}_t\in \mathcal{V}$, which means that the current user $i_t$ is clustered completely correctly, i.e., $\overline{V}_t=V_{j(i_t)}$, therefore $\btheta_l=\btheta_{i_t},\forall{l\in \overline{V}_t}$, then we have:
\begin{align}
    \hat{\btheta}_{\overline{V}_t,t-1}-\btheta_{i_t}
    &=(\sum_{s\in[t-1]\atop i_s\in \overline{V}_t}\bx_{a_s}\
    \bx_{a_s}^{\top}+\lambda\bI)^{-1}(\sum_{s\in[t-1]\atop i_s\in \overline{V}_t}\bx_{a_s}r_s)-\btheta_{i_t}\nonumber\\
    &=(\sum_{s\in[t-1]\atop i_s\in \overline{V}_t}\bx_{a_s}\
    \bx_{a_s}^{\top}+\lambda\bI)^{-1}\bigg(\sum_{s\in[t-1]\atop i_s\in \overline{V}_t}\bx_{a_s}(\bx_{a_s}^{\top}\btheta_{i_s}+\bepsilon_{a_s}^{i_s,s}+\eta_s)\bigg)-\btheta_{i_t}\nonumber\\
    &=(\sum_{s\in[t-1]\atop i_s\in \overline{V}_t}\bx_{a_s}\
    \bx_{a_s}^{\top}+\lambda\bI)^{-1}\bigg(\sum_{s\in[t-1]\atop i_s\in \overline{V}_t}\bx_{a_s}(\bx_{a_s}^{\top}\btheta_{i_t}+\bepsilon_{a_s}^{i_s,s}+\eta_s)\bigg)-\btheta_{i_t}\nonumber\\
    &=(\sum_{s\in[t-1]\atop i_s\in \overline{V}_t}\bx_{a_s}\
    \bx_{a_s}^{\top}+\lambda\bI)^{-1}[(\sum_{s\in[t-1]\atop i_s\in \overline{V}_t}\bx_{a_s}\bx_{a_s}^{\top}+\lambda\bI)\btheta_{i_t}-\lambda\btheta_{i_t}\notag\\
    &+\sum_{s\in[t-1]\atop i_s\in \overline{V}_t}\bx_{a_s}\bepsilon_{a_s}^{i_s,s}+\sum_{s\in[t-1]\atop i_s\in \overline{V}_t}\bx_{a_s}\eta_s]-\btheta_{i_t}\nonumber\\
    &=-\lambda\overline{\bM}_{\overline{V}_t,t-1}^{-1}\btheta_{i_t}+\sum_{s\in[t-1]\atop i_s\in \overline{V}_t}\overline{\bM}_{\overline{V}_t,t-1}^{-1}\bx_{a_s}\bepsilon_{a_s}^{i_s,s}+\sum_{s\in[t-1]\atop i_s\in \overline{V}_t}\overline{\bM}_{\overline{V}_t,t-1}^{-1}\bx_{a_s}\eta_s\nonumber\,.
\end{align}

Therefore we have 
\begin{align}
        \left|\bx_a^{\top}(\hat{\btheta}_{\overline{V}_t,t-1}-\btheta_{i_t})\right|&\leq\lambda\left|\bx_a^{\top}\overline{\bM}_{\overline{V}_t,t-1}^{-1}\btheta_{i_t}\right|+\left|\sum_{s\in[t-1]\atop i_s\in \overline{V}_t}\bx_a^{\top}\overline{\bM}_{\overline{V}_t,t-1}^{-1}\bx_{a_s}\bepsilon_{a_s}^{i_s,s}\right|\notag\\
        &+\left|\bx_a^{\top}\overline{\bM}_{\overline{V}_t,t-1}^{-1}\sum_{s\in[t-1]\atop i_s\in \overline{V}_t}\bx_{a_s}\eta_s\right|\label{parts 3}\,.
\end{align}

Next, we bound the three terms in Eq.(\ref{parts 3}). For the first term:
\begin{align}
     \lambda\left|\bx_a^{\top}\overline{\bM}_{\overline{V}_t,t-1}^{-1}\btheta_{i_t}\right|
     &\leq \lambda \norm{\bx_a}_{\overline{\bM}_{\overline{V}_t,t-1}^{-1}}\sqrt{\lambda_{\text{max}}(\overline{\bM}_{\overline{V}_t,t-1}^{-1})}\norm{\btheta_{i_t}}_2\leq \sqrt{\lambda}\norm{\bx_a}_{\overline{\bM}_{\overline{V}_t,t-1}^{-1}}\,,\label{first term bound}
\end{align}
where we use the inequality of matrix norm, the Cauchy–Schwarz inequality, $\norm{\btheta_{i_t}}_2\leq 1$, and the fact that $\lambda_{\text{max}}(\overline{\bM}_{\overline{V}_t,t-1}^{-1})=\frac{1}{\lambda_{\text{min}}(\overline{\bM}_{\overline{V}_t,t-1})}\leq\frac{1}{\lambda}$.

For the second term in Eq.(\ref{parts 3}):
\begin{align}
    \left|\sum_{s\in[t-1]\atop i_s\in \overline{V}_t}\bx_a^{\top}\overline{\bM}_{\overline{V}_t,t-1}^{-1}\bx_{a_s}\bepsilon_{a_s}^{i_s,s}\right|
    &\leq \sum_{s\in[t-1]\atop i_s\in \overline{V}_t}\left|\bx_a^{\top}\overline{\bM}_{\overline{V}_t,t-1}^{-1}\bx_{a_s}\bepsilon_{a_s}^{i_s,s}\right|\nonumber\\
    &\leq \sum_{s\in[t-1]\atop i_s\in \overline{V}_t}\norm{\bepsilon_{a_s}^{i_s,s}}_{\infty}\left|\bx_a^{\top}\overline{\bM}_{\overline{V}_t,t-1}^{-1}\bx_{a_s}\right|\notag\\
    &\leq \epsilon_*\sum_{s\in[t-1]\atop i_s\in \overline{V}_t}\left|\bx_a^{\top}\overline{\bM}_{\overline{V}_t,t-1}^{-1}\bx_{a_s}\right|\label{second term bound}\,,
\end{align}
where in the second inequality we use the Holder's inequality.

For the last term, with probability at least $1-\delta$:
\begin{align}
    \left|\bx_a^{\top}\overline{\bM}_{\overline{V}_t,t-1}^{-1}\sum_{s\in[t-1]\atop i_s\in \overline{V}_t}\bx_{a_s}\eta_s\right|
    &\leq \norm{\bx_a}_{\overline{\bM}_{\overline{V}_t,t-1}^{-1}}\norm{\sum_{s\in[t-1]\atop i_s\in \overline{V}_t}\bx_{a_s}\eta_s}_{\overline{\bM}_{\overline{V}_t,t-1}^{-1}}\label{cauchy 8}\\
    &\leq \norm{\bx_a}_{\overline{\bM}_{\overline{V}_t,t-1}^{-1}}\sqrt{2\log(\frac{1}{\delta})+\log(\frac{\text{det}(\overline{\bM}_{\overline{V}_t,t-1})}{\text{det}(\lambda\bI)})}\notag\\
    &\leq \norm{\bx_a}_{\overline{\bM}_{\overline{V}_t,t-1}^{-1}}\sqrt{2\log(\frac{1}{\delta})+d\log(1+\frac{T}{\lambda d})}\label{matrix det inequality}\,,
\end{align}
where the second inequality follows by Theorem 1 in \cite{abbasi2011improved}, Eq.(\ref{matrix det inequality}) is because $\text{det}(\overline{\bM}_{\overline{V}_t,t-1})\leq\Bigg(\frac{\text{trace}(\lambda\bI+\sum_{s\in[t]\atop i_s\in \overline{V}_t}\bx_{a_s}\bx_{a_s}^{\top})}{d}\Bigg)^d\leq\big(\frac{\lambda d+T_{\overline{V}_t,t}}{d}\big)^d\leq\big(\frac{\lambda d+T}{d}\big)^d$, and $\text{det}(\lambda\bI)=\lambda^d$.

Plugging Eq.(\ref{first term bound}), Eq.(\ref{second term bound}) and Eq.(\ref{matrix det inequality}) into Eq.(\ref{parts 3}), we can prove Lemma \ref{concentration bound} in situation 1, i.e., for any $t\geq T_0$ and $\overline{V}_t\in V$, with probability at least $1-\delta$:
\begin{align}
  \left|\bx_a^{\top}(\hat{\btheta}_{\overline{V}_t,t-1}-\btheta_{i_t})\right|&\leq\epsilon_*\sum_{s\in[t-1]\atop i_s\in \overline{V}_t}\left|\bx_a^{\top}\overline{\bM}_{\overline{V}_t,t-1}^{-1}\bx_{a_s}\right|\notag\\
  &+\norm{\bx_a}_{\overline{\bM}_{\overline{V}_t,t-1}^{-1}}\bigg(\sqrt{\lambda}+\sqrt{2\log(\frac{1}{\delta})+d\log(1+\frac{T}{\lambda d})}\bigg)\,.  
\end{align}

(2) Situation 2: for any $t\geq T_0$ and $\overline{V}_t\notin \mathcal{V}$, which means that the current user is \textit{misclustered} by the algorithm, i.e., $\overline{V}_t\neq V_{j(i_t)}$, but with Lemma \ref{sufficient time}, with probability at least $1-3\delta$, the current partition is a ``good partition", i.e., $\norm{\btheta_l-\btheta_{i_t}}_2\leq2\epsilon_*\sqrt{\frac{2}{\tilde{\lambda}_x}},\forall{l\in \overline{V}_t}$, we have:
\begin{small}\begin{align}
        &\hat{\btheta}_{\overline{V}_t,t-1}-\btheta_{i_t}\notag\\
    &=(\sum_{s\in[t-1]\atop i_s\in \overline{V}_t}\bx_{a_s}\
    \bx_{a_s}^{\top}+\lambda\bI)^{-1}(\sum_{s\in[t-1]\atop i_s\in \overline{V}_t}\bx_{a_s}r_s)-\btheta_{i_t}\nonumber\\
    &=(\sum_{s\in[t-1]\atop i_s\in \overline{V}_t}\bx_{a_s}\
    \bx_{a_s}^{\top}+\lambda\bI)^{-1}\bigg(\sum_{s\in[t-1]\atop i_s\in \overline{V}_t}\bx_{a_s}(\bx_{a_s}^{\top}\btheta_{i_s}+\bepsilon_{a_s}^{i_s,s}+\eta_s)\bigg)-\btheta_{i_t}\nonumber\\
    &=\overline{\bM}_{\overline{V}_t,t-1}^{-1}\sum_{s\in[t-1]\atop i_s\in \overline{V}_t}\bx_{a_s}\bepsilon_{a_s}^{i_s,s}+\overline{\bM}_{\overline{V}_t,t-1}^{-1}\sum_{s\in[t-1]\atop i_s\in \overline{V}_t}\bx_{a_s}\eta_s+\overline{\bM}_{\overline{V}_t,t-1}^{-1}\sum_{s\in[t-1]\atop i_s\in \overline{V}_t}\bx_{a_s}\bx_{a_s}^{\top}\btheta_{i_s}-\btheta_{i_t}\nonumber\\
    &=\overline{\bM}_{\overline{V}_t,t-1}^{-1}\sum_{s\in[t-1]\atop i_s\in \overline{V}_t}\bx_{a_s}\bepsilon_{a_s}^{i_s,s}+\overline{\bM}_{\overline{V}_t,t-1}^{-1}\sum_{s\in[t-1]\atop i_s\in \overline{V}_t}\bx_{a_s}\eta_s+\overline{\bM}_{\overline{V}_t,t-1}^{-1}\sum_{s\in[t-1]\atop i_s\in \overline{V}_t}\bx_{a_s}\bx_{a_s}^{\top}(\btheta_{i_s}-\btheta_{i_t})\nonumber\\
    &\quad+\overline{\bM}_{\overline{V}_t,t-1}^{-1}(\sum_{s\in[t-1]\atop i_s\in \overline{V}_t}\bx_{a_s}\bx_{a_s}^{\top}+\lambda\bI)\btheta_{i_t}-\lambda\overline{\bM}_{\overline{V}_t,t-1}^{-1}\btheta_{i_t}-\btheta_{i_t}\nonumber\\
    &=\overline{\bM}_{\overline{V}_t,t-1}^{-1}\sum_{s\in[t-1]\atop i_s\in \overline{V}_t}\bx_{a_s}\bepsilon_{a_s}^{i_s,s}+\overline{\bM}_{\overline{V}_t,t-1}^{-1}\sum_{s\in[t-1]\atop i_s\in \overline{V}_t}\bx_{a_s}\eta_s+\overline{\bM}_{\overline{V}_t,t-1}^{-1}\sum_{s\in[t-1]\atop i_s\in \overline{V}_t}\bx_{a_s}\bx_{a_s}^{\top}(\btheta_{i_s}-\btheta_{i_t})\nonumber\\
    &\quad-\lambda\overline{\bM}_{\overline{V}_t,t-1}^{-1}\btheta_{i_t}\nonumber\,.
\end{align}
\end{small}

Thus, with Lemma \ref{bound for mis} and with the previous reasoning, with probability at least $1-5\delta$, we have:\begin{small}
    \begin{align}
    &\left|\bx_a^{\top}(\hat{\btheta}_{\overline{V}_t,t-1}-\btheta_{i_t})\right|\notag\\
    &\leq\lambda\left|\bx_a^{\top}\overline{\bM}_{\overline{V}_t,t-1}^{-1}\btheta_{i_t}\right|+\left|\sum_{s\in[t-1]\atop i_s\in \overline{V}_t}\bx_a^{\top}\overline{\bM}_{\overline{V}_t,t-1}^{-1}\bx_{a_s}\bepsilon_{a_s}^{i_s,s}\right|+\left|\bx_a^{\top}\overline{\bM}_{\overline{V}_t,t-1}^{-1}\sum_{s\in[t-1]\atop i_s\in \overline{V}_t}\bx_{a_s}\eta_s\right|\nonumber\\
    &\quad+\left|\bx_a^{\top}\overline{\bM}_{\overline{V}_t,t-1}^{-1}\sum_{s\in[t-1]\atop i_s\in \overline{V}_t}\bx_{a_s}\bx_{a_s}^{\top}(\btheta_{i_s}-\btheta_{i_t})\right|\nonumber\\
    &\leq\epsilon_*\sum_{s\in[t-1]\atop i_s\in \overline{V}_t}\left|\bx_a^{\top}\overline{\bM}_{\overline{V}_t,t-1}^{-1}\bx_{a_s}\right|+\norm{\bx_a}_{\overline{\bM}_{\overline{V}_t,t-1}^{-1}}\bigg(\sqrt{\lambda}+\sqrt{2\log(\frac{1}{\delta})+d\log(1+\frac{T}{\lambda d})}\bigg)\nonumber\\
    &\quad+\frac{\epsilon_*\sqrt{2d}}{\tilde{\lambda}_x^{\frac{3}{2}}}\nonumber\,.
\end{align}
\end{small}

Therefore, combining situation 1 and situation 2, the result of Lemma \ref{concentration bound} then follows.

\subsection{Technical Lemmas and Their Proofs}\label{appendix: technical}
We first prove the following technical lemma which is used to prove Lemma \ref{sufficient time}.
\begin{lemma}
    \label{assumption}
    Under Assumption \ref{assumption3}, at any time $t$, for any fixed unit vector $\btheta \in \RR^d$
    \begin{equation}
        \mathbb{E}_t[(\btheta^{\top}\bx_{a_t})^2|\left|\mathcal{A}_t\right|]\geq\tilde{\lambda}_x\triangleq\int_{0}^{\lambda_x} (1-e^{-\frac{(\lambda_x-x)^2}{2\sigma^2}})^{C} dx\,.
    \end{equation}
\end{lemma}
\begin{proof}
The proof of this lemma mainly follows the proof of Claim 1 in \cite{gentile2014online}, but with more careful analysis, since their assumption is more stringent than ours.

Denote the feasible arms at round $t$ by ${\mathcal{A}_t=\{\bx_{t,1},\bx_{t,2},\ldots,\bx_{t,\left|\mathcal{A}_t\right|}\}}$. 
Consider the corresponding i.i.d. random variables $\theta_i=(\btheta^{\top}\bx_{t,i})^2-\mathbb{E}_t[(\btheta^{\top}\bx_{t,i})^2|\left|\mathcal{A}_t\right|], i=1,2,\ldots,\left|\mathcal{A}_t\right|$. By Assumption \ref{assumption3}, $\theta_i$ s are sub-Gaussian random variables with variance bounded by $\sigma^2$. Therefore, we have that for any $\alpha>0$ and any $i\in[\left|\mathcal{A}_t\right|]$:
\begin{equation*}
    \mathbb{P}_t(\theta_i<-\alpha|\left|\mathcal{A}_t\right|)\leq e^{-\frac{\alpha^2}{2\sigma^2}}\,,
\end{equation*}
where $\mathbb{P}_t(\cdot)$ is the shorthand for the conditional probability\\ $\mathbb{P}(\cdot|(i_1,\mathcal{A}_1,r_1),\ldots,(i_{t-1},\mathcal{A}_{t-1},r_{t-1}),i_t)$.

We also have that
$\mathbb{E}_t[(\btheta^{\top}\bx_{t,i})^2|\left|\mathcal{A}_t\right|=\mathbb{E}_t[\btheta^{\top}\bx_{t,i}\bx_{t,i}^{\top}\btheta|\left|\mathcal{A}_t\right|]\geq\lambda_{\text{min}}(\mathbb{E}_{\bx\sim \rho}[\bx\bx^{\top}])\geq\lambda_x$ by Assumption \ref{assumption3}.
With the above inequalities, we can get
\begin{equation*}
    \mathbb{P}_t(\min_{i=1,\ldots,\left|\mathcal{A}_t\right|}(\btheta^{\top}\bx_{t,i})^2\geq \lambda_x-\alpha|\left|\mathcal{A}_t\right|)\geq (1-e^{-\frac{\alpha^2}{2\sigma^2}})^C\,,
\end{equation*}
where $C$ is the upper bound of $\left|\mathcal{A}_t\right|$.

Therefore, we have
\begin{align}
\mathbb{E}_t[(\btheta^{\top}\bx_{a_t})^2|\left|\mathcal{A}_t\right|]
&\geq\mathbb{E}_t[\min_{i=1,\ldots,\left|\mathcal{A}_t\right|}(\btheta^{\top}\bx_{t,i})^2|\left|\mathcal{A}_t\right|]\notag\\
&\geq\int_{0}^{\infty} \mathbb{P}_t (\min_{i=1,\ldots,\left|\mathcal{A}_t\right|}(\btheta^{\top}\bx_{t,i})^2\geq x|\left|\mathcal{A}_t\right|) dx\notag\\
&\geq \int_{0}^{\lambda_x} (1-e^{-\frac{(\lambda_x-x)^2}{2\sigma^2}})^{C} dx\triangleq\tilde{\lambda}_x\notag
\end{align}
\end{proof}

Finally, we prove the following lemma which is used in the proof of Theorem \ref{thm:main}.
\begin{lemma}
\label{technical lemma6}
\begin{equation}
    \sum_{t=T_0+1}^{T}\min\{.\norm{\bx_{a_t}}_{\overline{\bM}_{V_j,t-1}^{-1}}^2,1\}\leq2d\log(1+\frac{T}{\lambda d}), \forall{j\in[m]}\,.
\end{equation}
\end{lemma}
\begin{proof}
\begin{align}
    \text{det}(\overline{\bM}_{V_j,T})
    &=\text{det}\bigg(\overline{\bM}_{V_j,T-1}+\mathbb{I}\{i_T\in V_j\}\bx_{a_T}\bx_{a_T}^{\top}\bigg)\nonumber\\
    &=\text{det}(\overline{\bM}_{V_j,T-1})\text{det}\bigg(\bI+\mathbb{I}\{i_T\in V_j\}\overline{\bM}_{V_j,T-1}^{-\frac{1}{2}}\bx_{a_T}\bx_{a_T}^{\top}\overline{\bM}_{V_j,T-1}^{-\frac{1}{2}}\bigg)\nonumber\\
    &=\text{det}(\overline{\bM}_{V_j,T-1})\bigg(1+\mathbb{I}\{i_T\in V_j\}\norm{\bx_{a_T}}_{\overline{\bM}_{V_j,T-1}^{-1}}^2\bigg)\nonumber\\
    &=\text{det}(\overline{\bM}_{V_j,T_0})\prod_{t=T_0+1}^{T}\bigg(1+\mathbb{I}\{i_t\in V_j\}\norm{\bx_{a_t}}_{\overline{\bM}_{V_j,t-1}^{-1}}^2\bigg)\nonumber\\
    &\geq \text{det}(\lambda\bI)\prod_{t=T_0+1}^{T}\bigg(1+\mathbb{I}\{i_t\in V_j\}\norm{\bx_{a_t}}_{\overline{\bM}_{V_j,t-1}^{-1}}^2\bigg)\label{det recursive}\,.
\end{align}

$\forall{x\in[0,1]}$, we have $x\leq 2\log(1+x)$. Therefore
\begin{align}
    &\sum_{t=T_0+1}^{T}\min\{\mathbb{I}\{i_t\in V_j\}\norm{\bx_{a_t}}_{\overline{\bM}_{V_j,t-1}^{-1}}^2,1\}\notag\\
    &\leq 2\sum_{t=T_0+1}^{T} \log\bigg(1+\mathbb{I}\{i_t\in V_j\}\norm{\bx_{a_t}}_{\overline{\bM}_{V_j,t-1}^{-1}}^2\bigg)\nonumber\\
    &=2\log\bigg(\prod_{t=T_0+1}^{T}\big(1+\mathbb{I}\{i_t\in V_j\}\norm{\bx_{a_t}}_{\overline{\bM}_{V_j,t-1}^{-1}}^2\big)\bigg)\nonumber\\
    &\leq 2[\log(\text{det}(\overline{\bM}_{V_j,T}))-\log(\text{det}(\lambda\bI))]\nonumber\\
    &\leq 2\log\bigg(\frac{\text{trace}(\lambda\bI+\sum_{t=1}^T\mathbb{I}\{i_t\in V_j\}\bx_{a_t}\bx_{a_t}^{\top})}{\lambda d}\bigg)^d\nonumber\\
    &\leq 2d \log(1+\frac{T}{\lambda d})\,.
\end{align}
\end{proof}
\subsection{Algorithms of RSCLUMB} \label{RSCLUMB section}
This section introduces the Robust Set-based Clustering of Misspecified Bandits Algorithm (RSCLUMB). Unlike RCLUMB, which maintains a graph-based clustering structure, RSCLUMB maintains a set-based clustering structure. Besides, RCLUMB only splits clusters during the learning process, while RSCLUMB allows both split and merge operations. A brief illustration is that the agent will split a user out of its current set(cluster) if it finds an inconsistency between the user and its set, and if there are two clusters whose estimated preferences are close enough, the agent will merge them. A detailed discussion of the connection between the graph structure and the set structure can be found in~\cite{10.5555/3367243.3367445}.

Now we introduce the details of RSCLUMB. The algorithm first initializes a single set $\boldsymbol{S}_{1}$ containing all users and updates it during the learning process. The whole learning process consists of phases (Algo. \ref{alg:RSCLUMB} Line 3), where the $s-th$ phase contains $2^{s}$ rounds. At the beginning of each phase, the agent marks all users as "unchecked", and if a user comes later, it will be marked as "checked". If all users in a cluster are checked, then this cluster will be marked as "checked" meaning it is an accurate cluster in the current phase. With this mechanism, every phase can maintain an accuracy level, and the agent can put the accurate clusters aside and focus on exploring the inaccurate ones. For each cluster $V_{j}$, the algorithm maintains two estimated vectors $\hat{\boldsymbol{\theta}}_{V_j}$ and $\Tilde{\boldsymbol{\theta}}_{V_{j}}$, where the $\hat{\boldsymbol{\theta}}_{V_j}$ is similar to the $\hat{\boldsymbol{\theta}}_{\overline{V}_j}$ in RCLUMB and is used for the recommendation, while the $\Tilde{\boldsymbol{\theta}}_{V_{j}}$ is the average of all the estimated user preference vectors in this cluster and is used for the split and merge operations.

At time $t$ in phase $s$, the user $i_{\tau}$ comes with the item set $\mathcal{D}_{\tau}$, where $\tau$ represents the index of total time steps. Then the algorithm determines the cluster and makes a cluster-based recommendation. This process is similar to RCLUMB. After updating the information (Algo. \ref{alg:RSCLUMB} Line12), the agent checks if a split or a merge is possible (Algo. \ref{alg:RSCLUMB} Line13-17).

By our assumption, users in the same cluster have the same vectors. So a cluster can be regarded as a good cluster only when all the estimated user vectors are close to the estimated cluster vector. We call a user is consistent with the cluster if their estimated vectors are close enough. If a user is inconsistent with its current cluster, the agent will split it out. Two clusters are consistent when their estimated vectors are close, and the agent will merge them. 

RSCLUMB maintains two sets of estimated cluster vectors: (i) cluster-level estimation with integrated user information, which is for recommendations (Line \ref{rsclumb:update} and Line \ref{rsclumb:recommend} in Algo.\ref{alg:RSCLUMB}); (ii) the average of estimated user vectors, which is used for robust clustering (Line \ref{rsclumb:split line} in Algo.\ref{alg: split} and Line \ref{rsclumb:merge line} in Algo.\ref{alg: merge}). The previous set-based CB work \cite{10.5555/3367243.3367445} only uses (i) for both recommendations and clustering, which would lead to erroneous clustering under misspecifications, and cannot get any non-vacuous regret bound in CBMUM.

\begin{algorithm}[htb!]
    \caption{Robust Set-based Clustering of Misspecified Bandits Algorithm (RSCLUMB)}
    \label{alg:RSCLUMB}
    \begin{algorithmic}[1]
       \STATE {{\bf Input:}  Deletion parameter $\alpha_1,\alpha_2>0$, $f(T)=\sqrt{\frac{1 + \ln(1+T)}{1 + T}}$, $\lambda, \beta, \epsilon_*>0$.}
    \STATE {{\bf Initialization:} 
     \begin{itemize}
      \item $\bM_{i,0} = 0_{d\times d}, \bb_{i,0} = 0_{d \times 1}, T_{i,0}=0$ , $\forall{i \in \mathcal{U}}$;
      \item Initialize the set of cluster indexes by $J = \{1\}$ and the single cluster $\boldsymbol{S}_1$ by $\boldsymbol{M}_{1} =0_{d \times d}$, $\boldsymbol{b}_{1} = 0_{d \times 1}$, $T_{1} = 0$, $C_{1} = \mathcal{U}$, $j(i)=1$, $\forall i$.
      \end{itemize}}
     \FORALL {$s=1,2,\ldots$}
      \STATE{Mark every user unchecked for each cluster.}
      \STATE{For each cluster $V_j$, compute $\Tilde{T}_{V_j} = T_{V_j}$, $\hat{\boldsymbol{\theta}}_{V_j} = (\lambda \boldsymbol{I} +\boldsymbol{M}_{V_j})^{-1}\boldsymbol{b}_{V_j}$}, $\Tilde{\boldsymbol{\theta}}_{V_j}=\frac{\sum_{i \in V_{j}}\hat{\boldsymbol{\theta}}_{i}}{[V_{j}]}$
      \FORALL {$t=1,2,\ldots, T$}
      \STATE{Compute $\tau = 2^{s}-2+t$}
      \STATE{Receive the user $i_{\tau}$ and the decision set $\mathcal{D}_{\tau}$}
      \STATE{Determine the cluster index $j = j(i_{\tau})$}
      \STATE{Recommend item $a_{\tau}$ with the largest UCB index as shown in Eq. (\ref{UCB})} \label{rsclumb:recommend}
      \STATE{Received the feedback $r_{\tau}$.}
      \STATE{Update the information:}\label{rsclumb:update}
     \begin{align*}
         \boldsymbol{M}_{i_{\tau}, \tau} &= \boldsymbol{M}_{i_{\tau}, \tau-1} + \boldsymbol{x}_{a_\tau}\boldsymbol{x}_{a_\tau}^{\mathrm{T}},  \boldsymbol{b}_{i_{\tau}, \tau} = \boldsymbol{b}_{i_{\tau}, \tau-1} + r_{\tau}\boldsymbol{x}_{a_\tau}, \\ T_{i_{\tau, \tau}} &= T_{i_{\tau}, \tau-1} +1, \hat{\boldsymbol{\theta}}_{i_{\tau}, \tau} = (\lambda \boldsymbol{I}+\boldsymbol{M}_{i_{\tau}, \tau})^{-1}\boldsymbol{b}_{i_{\tau}, \tau}\\
          \boldsymbol{M}_{V_j, \tau} &= \boldsymbol{M}_{V_j, \tau-1} + \boldsymbol{x}_{a_\tau}\boldsymbol{x}_{a_\tau}^{\mathrm{T}},  \boldsymbol{b}_{V_j, \tau} = \boldsymbol{b}_{V_j, \tau-1} + r_{\tau}\boldsymbol{x}_{\tau}, \\ T_{V_j, \tau} &= T_{V_j, \tau-1} +1,  \hat{\boldsymbol{\theta}}_{V_j, \tau} = (\lambda \boldsymbol{I} + \boldsymbol{M}_{V_j, \tau})^{-1}\boldsymbol{b}_{V_j, \tau},  \\
          \Tilde{\boldsymbol{\theta}}_{V_j, \tau}&=\frac{\sum_{i \in V_{j}}\hat{\boldsymbol{\theta}}_{i},\tau}{[V_{j}]}
     \end{align*}
      \IF {$i_{\tau}$ is unchecked}
      \STATE{Run \textbf{Split}}
      \STATE{Mark user $i_{\tau}$ has been checked}
      \STATE{Run \textbf{Merge}}
      \ENDIF
      \ENDFOR
      \ENDFOR
    \end{algorithmic}
\end{algorithm}
\begin{algorithm}[htb]
    \caption{Split}
    \label{alg: split}
    \begin{algorithmic}[1]
        \STATE{Define $F(T)=\sqrt{\frac{1+\ln(1+T)}{1+T}}$}
        \IF{$\norm{\hat{\boldsymbol{\theta}}_{i_{\tau}, \tau} - \Tilde{\boldsymbol{\theta}}_{V_{j}, \tau}} >\alpha_{1}(F(T_{i_\tau, \tau})+F(T_{V_{j}, \tau})) + \alpha_{2}\epsilon_{*} $}
        \STATE{Split user $i_{\tau}$ from cluster $V_j$ and form a new cluster $V_j^{'}$ of user $i_{\tau}$}\label{rsclumb:split line}
        \begin{align*}
            \boldsymbol{M}_{V_j, \tau}& = \boldsymbol{M}_{V_j, \tau} - \boldsymbol{M}_{i_{\tau}, \tau}, \boldsymbol{b}_{V_j} = \boldsymbol{b}_{V_j} - \boldsymbol{b}_{i_{\tau}, \tau}, \\ T_{V_j, \tau} &= T_{V_j, \tau} - T_{i_{\tau}, \tau}, C_{j, \tau} = C_{j, \tau} - \{i_{\tau}\}, \\
             \boldsymbol{M}_{V_j', \tau} &= \boldsymbol{M}_{i_{\tau}, \tau}, \boldsymbol{b}_{V_j', \tau} = \boldsymbol{b}_{i_{\tau}, \tau},\\T_{V_j', \tau} &= T_{i_{\tau}, \tau},C_{j', \tau}=\{i_{\tau}\}
        \end{align*}
        \ENDIF
    \end{algorithmic}
\end{algorithm}
\begin{algorithm}[htb]
    \caption{Merge}
    \label{alg: merge}
    \begin{algorithmic}[1]
        \FOR{any two checked clusters$V_{j_1},V_{j_2}$ satisfying
    \[\norm{\Tilde{\boldsymbol{\theta}}_{j_1}- \Tilde{\boldsymbol{\theta}}_{j_2}} < \frac{\alpha_{1}}{2} (F(T_{V_{j_1}})+F(T_{V_{j_2}})) + \frac{\alpha_2}{2}\epsilon_{*} \]}
    \STATE{Merge them:}\label{rsclumb:merge line}
    \begin{align*}
        \boldsymbol{M}_{V_{j_1}} &= \boldsymbol{M}_{j_1} + \boldsymbol{M}_{j_2}, \boldsymbol{b}_{V_{j_1}}=\boldsymbol{b}_{V_{j_1}}+\boldsymbol{b}_{V_{j_2}},\\T_{V_{j_1}} &= T_{V_{j_1}} + T_{V_{j_2}},C_{V_{j_1}} = C_{V_{j_1}} \cup C_{V_{j_2}}
    \end{align*}
     \STATE{Set $ j(i) = j_{1}, \forall i \in j_{2}$, delete $V_{j_2}$}
        \ENDFOR
    \end{algorithmic}
\end{algorithm}
\subsection{Main Theorem and Lemmas of RSCLUMB}
\begin{theorem}[main result on regret bound for RSCLUMB] 
\label{thm: main2}
With the same assumptions in Theorem \ref{thm:main}, the expected regret of the RSCLUMB algorithm for T rounds satisfies:
\begin{small}
    \begin{align}
    R(T) &\le O \bigg(u\left( \frac{d}{\Tilde{\lambda}_x (\gamma_1-\zeta_1)^2}+\frac{1}{\Tilde{\lambda}_x^2}\right)\log T + \frac{\epsilon_*\sqrt{d}T}{\Tilde{\lambda}_x^{1.5}} 
   + \epsilon_*T \sqrt{md\log T}   + d\sqrt{mT}\log T + \epsilon_{*}\sqrt{\frac{1}{\Tilde{\lambda}_{x}}}T\bigg) \\
    &\le O(\epsilon_*T\sqrt{md\log T}  + d\sqrt{mT}\log T)
\end{align}
\end{small}

\end{theorem}
\begin{lemma}
    \label{sufficient time for RSCLUMB}
    For RSCLUMB, we use $T_{1}$ to represent the corresponding $T_{0}$ of RCLUMB. Then :
    \begin{equation*}
        \begin{aligned}
            T_{1} &\triangleq 16u\log(\frac{u}{\delta})+4u\max\{\frac{16}{\Tilde{\lambda}_x^2}\log(\frac{8d}{\Tilde{\lambda}_x^2\delta}),
        \frac{8d}{\Tilde{\lambda}_x(\frac{\gamma_1}{6}-\epsilon_*\sqrt{\frac{1}{2\Tilde{{\lambda}}_x}})^2}\log(\frac{u}{\delta})\}\\
        &=O\bigg(u\left( \frac{d}{\Tilde{\lambda}_x (\gamma_1-\zeta_1)^2}+\frac{1}{\Tilde{\lambda}_x^2}\right)\log \frac{1}{\delta}\bigg)
        \end{aligned}
    \end{equation*}
\end{lemma}
\begin{lemma}
\label{concentration bound rsclumb}
For RSCLUMB, after $2T_{1}+1$ rounds: in each phase, after the first $u$ rounds, with probability at least $1-5\delta$:
\begin{equation*}
    \begin{aligned}
     &\left|\bx_a^{\top}(\btheta_{i_t}-\hat{\btheta}_{\overline{V}_t,t-1})\right|\notag\\
        &\leq (\frac{3\epsilon_*\sqrt{2d}}{2\Tilde{\lambda}_x^{\frac{3}{2}}}+6\epsilon_{*}\sqrt{\frac{1}{2\Tilde{\lambda}_{x}}})\mathbb{I}\{\overline{V}_t\notin V\}
        +\beta
        \norm{\bx_a}_{\overline{\bM}_{\overline{V}_t,t-1}^{-1}}
        +\epsilon_*\sum_{s\in[t-1]\atop i_s\in \overline{V}_t}\left|\bx_a^{\top}\overline{\bM}_{\overline{V}_t,t-1}^{-1}\bx_{a_s}\right| \\  &\triangleq (\frac{3\epsilon_*\sqrt{2d}}{2\Tilde{\lambda}_x^{\frac{3}{2}}} +6\epsilon_{*}\sqrt{\frac{1}{2\Tilde{\lambda}_{x}}})\mathbb{I}\{\overline{V}_t\notin V\} +C_{a,t}\,
\end{aligned}
\end{equation*}
    
\end{lemma}

\subsection{Proof of Lemma \ref{concentration bound rsclumb}}

    \begin{equation}
        \begin{aligned}
            \lvert \boldsymbol{x}_{a}^{\mathrm{T}}(\boldsymbol{\theta}_{i} - \hat{\boldsymbol{\theta}}_{\overline{V}_{t},t}) \rvert &=  \lvert\boldsymbol{x}_{a}^{\mathrm{T}}(\boldsymbol{\theta}_{i} - \boldsymbol{\theta}_{V_{t}})  \rvert + \lvert\boldsymbol{x}_{a}^{\mathrm{T}}(\hat{\boldsymbol{\theta}}_{\overline{V}_{t},t} - \boldsymbol{\theta}_{V_{t}}) \rvert\\
            &\le \norm{\boldsymbol{x}_{a}^{\mathrm{T}}}\norm{\boldsymbol{\theta}_{i} - \boldsymbol{\theta}_{V_{t}}} + \lvert\boldsymbol{x}_{a}^{\mathrm{T}}(\hat{\boldsymbol{\theta}}_{\overline{V}_{t},t} - \boldsymbol{\theta}_{V_{t}}) \rvert\\
            & \le 6\epsilon_{*}\sqrt{\frac{1} 
  {2\Tilde{\lambda}_{x}}} + \lvert\boldsymbol{x}_{a}^{\mathrm{T}}(\hat{\boldsymbol{\theta}}_{\overline{V}_{t},t} - \boldsymbol{\theta}_{V_{t}}) \rvert
        \end{aligned}
    \end{equation}
where the last inequality holds due to the fact $\norm{\boldsymbol{x}_{a}} \le 1$ and the condition of "split" and "merge". For $\lvert\boldsymbol{x}_{a}^{\mathrm{T}}(\hat{\boldsymbol{\theta}}_{\overline{V}_{t},t} - \boldsymbol{\theta}_{V_{t}}) \rvert$:
\begin{equation}
\begin{aligned}
    &\hat{\btheta}_{\overline{V}_t,t-1}-\btheta_{V_t}\\
    &=(\sum_{s\in[t-1]\atop i_s\in \overline{V}_t}\bx_{a_s}\
    \bx_{a_s}^{\top}+\lambda\bI)^{-1}(\sum_{s\in[t-1]\atop i_s\in \overline{V}_t}\bx_{a_s}r_s)-\btheta_{V_t}\nonumber\\
    &=(\sum_{s\in[t-1]\atop i_s\in \overline{V}_t}\bx_{a_s}\
    \bx_{a_s}^{\top}+\lambda\bI)^{-1}\bigg(\sum_{s\in[t-1]\atop i_s\in \overline{V}_t}\bx_{a_s}(\bx_{a_s}^{\top}\btheta_{i_s}+\bepsilon_{a_s}^{i_s,s}+\eta_s)\bigg)-\btheta_{V_t}\nonumber\\
    &=\overline{\bM}_{\overline{V}_t,t-1}^{-1}\sum_{s\in[t-1]\atop i_s\in \overline{V}_t}\bx_{a_s}\bepsilon_{a_s}^{i_s,s}+\overline{\bM}_{\overline{V}_t,t-1}^{-1}\sum_{s\in[t-1]\atop i_s\in \overline{V}_t}\bx_{a_s}\eta_s+\overline{\bM}_{\overline{V}_t,t-1}^{-1}\sum_{s\in[t-1]\atop i_s\in \overline{V}_t}\bx_{a_s}\bx_{a_s}^{\top}\btheta_{i_s}-\btheta_{V_t}\nonumber\\
    &=\overline{\bM}_{\overline{V}_t,t-1}^{-1}\sum_{s\in[t-1]\atop i_s\in \overline{V}_t}\bx_{a_s}\bepsilon_{a_s}^{i_s,s}+\overline{\bM}_{\overline{V}_t,t-1}^{-1}\sum_{s\in[t-1]\atop i_s\in \overline{V}_t}\bx_{a_s}\eta_s+\overline{\bM}_{\overline{V}_t,t-1}^{-1}\sum_{s\in[t-1]\atop i_s\in \overline{V}_t}\bx_{a_s}\bx_{a_s}^{\top}(\btheta_{i_s}-\btheta_{V_t})\nonumber\\
    &\quad+\overline{\bM}_{\overline{V}_t,t-1}^{-1}(\sum_{s\in[t-1]\atop i_s\in \overline{V}_t}\bx_{a_s}\bx_{a_s}^{\top}+\lambda\bI)\btheta_{V_t}-\lambda\overline{\bM}_{\overline{V}_t,t-1}^{-1}\btheta_{V_t}-\btheta_{V_t}\nonumber\\
    &=\overline{\bM}_{\overline{V}_t,t-1}^{-1}\sum_{s\in[t-1]\atop i_s\in \overline{V}_t}\bx_{a_s}\bepsilon_{a_s}^{i_s,s}+\overline{\bM}_{\overline{V}_t,t-1}^{-1}\sum_{s\in[t-1]\atop i_s\in \overline{V}_t}\bx_{a_s}\eta_s+\overline{\bM}_{\overline{V}_t,t-1}^{-1}\sum_{s\in[t-1]\atop i_s\in \overline{V}_t}\bx_{a_s}\bx_{a_s}^{\top}(\btheta_{i_s}-\btheta_{V_t})\nonumber\\
    &\quad-\lambda\overline{\bM}_{\overline{V}_t,t-1}^{-1}\btheta_{V_t}\nonumber\,.
    \end{aligned}
\end{equation}
Thus, with the same method in Lemma \ref{bound for mis} but replace $\zeta = 4\epsilon_{*}\sqrt{\frac{1}{2\Tilde{\lambda}_{x}}}$ with $\zeta_1 = 6\epsilon_{*}\sqrt{\frac{1}{2\Tilde{\lambda}_{x}}}$, and with the previous reasoning, with probability at least $1-5\delta$, we have:
\begin{equation}
    \lvert\boldsymbol{x}_{a}^{\mathrm{T}}(\hat{\boldsymbol{\theta}}_{\overline{V}_{t},t} - \boldsymbol{\theta}_{V_{t}}) \rvert \le C_{a_t} + \frac{3\epsilon_{*}\sqrt{2d}}{2\Tilde{\lambda}_{x}^{\frac{3}{2}}}
\end{equation}
The lemma can be concluded.

\subsection{Proof of Lemma \ref{sufficient time for RSCLUMB}}

With the analysis in the proof of Lemma \ref{sufficient time}, with probability at least $1-\delta$:
\begin{equation}
    \norm{\hat{\btheta}_{i,t}-\btheta^{j(i)}}_2\leq\frac{\beta(T_{i,t},\frac{\delta}{u})+\epsilon_*\sqrt{T_{i,t}}}{\sqrt{\lambda+\lambda_{\text{min}}(\bM_{i,t})}}, \forall{i\in\mathcal{U}}\label{norm difference bound}\,,
\end{equation}
and the estimated error of the current cluster $\norm{\Tilde{\btheta}^{j(i)} - \btheta^{j(i)}}$ also satisfies this inequality. For set-based clustering structure, to ensure for each user there is only one $\zeta$-close cluster, we let:
\begin{equation}
    \frac{\beta(T_{i,t},\frac{\delta}{u})+\epsilon_*\sqrt{T_{i,t}}}{\sqrt{\lambda+\lambda_{\text{min}}(\bM_{i,t})}} \le \frac{\gamma_1}{6}
\end{equation}
By assuming $\lambda < 2\log(\frac{u}{\delta}) + d\log(1 + \frac{T_{i,t}}{\lambda d}) $, we can simplify it to 
\begin{equation}
    \frac{2\log(\frac{u}{\delta}) + d\log(1 + \frac{T_{i,t}}{\lambda d})}{2\Tilde{\lambda}_{x}T_{i,t}} < \frac{1}{4}(\frac{\gamma_1}{6} - \epsilon_{*}\sqrt{\frac{1}{2\Tilde{\lambda}_{x}}})^{2}
\end{equation}
which can be proved by $\frac{2\log(\frac{u}{\delta)}}{2\Tilde{\lambda}_{x}T_{i,t}} \le \frac{1}{8}(\frac{\gamma_1}{6} - \epsilon_{*}\sqrt{\frac{1}{2\Tilde{\lambda}_{x}}})^{2}$ and $\frac{d\log(1 + \frac{T_{i,t}}{\lambda d})}{2\Tilde{\lambda}_{x}T_{i,t}} \le \frac{1}{8}(\frac{\gamma_1}{6} - \epsilon_{*}\sqrt{\frac{1}{2\Tilde{\lambda}_{x}}})^{2}$. It's obvious that the former one can be satisfied by $T_{i,t} \ge \frac{8\log(u/\delta)}{\Tilde{\lambda}_{x}(\frac{\gamma_1}{6}-\epsilon_{*}\sqrt{1/2\Tilde{\lambda}_{x}})^{2}}$. As for the latter one, by ~\cite{li2018online} Lemma 9, we can get $T_{i,t} \ge \frac{8d\log(\frac{16}{\Tilde{\lambda}_{x}\lambda(\frac{\gamma_1}{6}-\epsilon_{*}\sqrt{1/2\Tilde{\lambda}_{x}})^{2}}}{4\Tilde{\lambda}_{x}(\frac{\gamma_1}{6}-\epsilon_{*}\sqrt{1/2\Tilde{\lambda}_{x}})^{2}}$. By assuming $\frac{u}{\delta} \ge \frac{16}{4\Tilde{\lambda}_{x}\lambda(\frac{\gamma_1}{6}-\epsilon_{*}\sqrt{2/4\Tilde{\lambda}_{x}})^{2}}$, the lemma is proved.

\subsection{Proof of Theorem \ref{thm: main2}}

    After $2T_{1}$ rounds,in each phase, at most $u$ times split operations will happen, we use $u\log(T)$ to bound the regret generated in these rounds. Then in the remained rounds the cluster num will be no more than $m$.\\
For the instantaneous regret $R_{t}$ at round $t$, with probability at least $1-2\delta$ for some $\delta \in (0, \frac{1}{2})$:
\begin{equation}
        \begin{aligned}
            R_{t} &= (\boldsymbol{x}^{\mathrm{T}}_{a^{*}_t}\boldsymbol{\theta}_{i_t} + \boldsymbol{\epsilon}_{a^{*}_t}^{i_t,t}) - (\boldsymbol{x}^{\mathrm{T}}_{a_t}\boldsymbol{\theta}_{i_t} + \boldsymbol{\epsilon}_{a_t}^{i_t,t}) \\
            &= \bx_{a_t^*}^{\top}(\btheta_{i_t}-\hat{\btheta}_{\overline{V}_t,t-1})+(\bx_{a_t^*}^{\top}\hat{\btheta}_{\overline{V}_t,t-1}+C_{a_t^*,t})-(\bx_{a_t}^{\top}\hat{\btheta}_{\overline{V}_t,t-1}+C_{a_t,t})\\
    &\quad+\bx_{a_t}^{\top}(\hat{\btheta}_{\overline{V}_t,t-1}-\btheta_{i_t})+C_{a_t,t}-C_{a_t^*,t}+(\bepsilon^{i_t,t}_{a_t^*}-\bepsilon^{i_t,t}_{a_t})\\
            & \le 2C_{a_t} + 2\epsilon_{*} + (12\epsilon_{*}\sqrt{\frac{1}{2\Tilde{\lambda}_{x}}} + \frac{3\epsilon_*\sqrt{2d}}{\Tilde{\lambda}_x^{\frac{3}{2}}})\mathbb{I}(\overline{V}_{t} \notin V) 
        \end{aligned}
    \end{equation}
where the last inequality holds due to the UCB arm selection strategy, the concentration bound given in Lemma\ref{concentration bound rsclumb} and the fact that $\norm{\epsilon^{i,t}}_{\infty} \le \epsilon_*$.

Define such events. Let:
\begin{align}
   \mathcal{E}_{2}&=\{\text{All clusters $\overline{V}_{t}$ only contain users who satisfy} \notag\\
&\norm{\Tilde{\boldsymbol{\theta}}_{i} -\Tilde{\boldsymbol{\theta}}_{\overline{V}_{t}}}  \le \alpha_{1}(\sqrt{\frac{1+\log(1+T_{i,t})}{1+T_{i,t}}} + \sqrt{\frac{1+\log(1+T_{\overline{V}_{t},t})}{1+T_{\overline{V}_{t},t}}}) + \alpha_{2}\epsilon_{*}\} 
\end{align}

\[
\mathcal{E}_{3} = \{ r_{t} \le 2C_{a_t} + 2\epsilon_{*} + 12\epsilon_{*}\sqrt{\frac{1}{2\Tilde{\lambda}_{x}}} + \frac{3\epsilon_*\sqrt{2d}}{\Tilde{\lambda}_x^{\frac{3}{2}}}\}
\]
\[
\mathcal{E}^{'} = \mathcal{E}_{2} \cap \mathcal{E}_{3}
\]
From previous analysis, we can know that $\PP(\mathcal{E}_{2}) \ge 1-3\delta$ and $\PP(\mathcal{E}_{3}) \ge 1-2\delta$, thus $\PP(\mathcal{E}^{'} \ge 1-5\delta)$. By taking $\delta=\frac{1}{T}$, we can get:
\begin{equation}
\begin{aligned}
    E(R_{t}) &= P(\mathcal{E})\mathbb{I}\{\mathcal{E}\}R_{t} + P(\bar{\mathcal{E}})\mathbb{I}\{\Bar{\mathcal{E}}\}R_{t} \\ 
    &\le \mathbb{I}\{\mathcal{E}\}R_{t} + 5 \\
    &\le 2T_{1} + 2\epsilon_{*}T + (12\epsilon_{*}\sqrt{\frac{1}{2\Tilde{\lambda}_{x}}}+\frac{3\epsilon_*\sqrt{2d}}{\Tilde{\lambda}_x^{\frac{3}{2}}})T + 2\sum_{2T_{1}}^{T}C_{a_{t}} +  5
\end{aligned}
\end{equation}
Now we need to bound $2\sum_{2T_{1}}^{T}C_{a_{t}}$.
We already know that after $2T_{1}$ rounds, in each phase $k$ after the first $u$ rounds,there will be at most $m$ clusters\\
Consider phase $k$, for simplicity, ignore the fist $u$ rounds. For the first term in $C_{a_{t}}$:
\begin{equation}
    \begin{aligned}        
    \sum_{t=T_{k-1}}^{T_{k}}\norm{\boldsymbol{x}_{a_t}}_{\overline{\boldsymbol{M}}_{\overline{V}_{t}, t-1}}^{-1} &=  \sum_{t=T_{k-1}}^{T_{k}}\sum_{j=1}^{m_{t}}\mathbb{I}\{i \in \overline{V}_{t,j}\}\norm{\boldsymbol{x}_{a_t}}_{\boldsymbol{\overline{M}}_{\overline{V}_{t,j}}^{-1}} \\
    &\le \sum_{j=1}^{m_{t}}\sqrt{\sum_{t=T_{k-1}}^{T_{k}}\mathbb{I}\{i \in V_{t,j}\}\sum_{t=T_{k-1}}^{T_{k}}\mathbb{I}\{i \in V_{t,j}\}\norm{\boldsymbol{x}_{a_t}}_{\boldsymbol{M}_{\overline{V}_{t, j}}^{-1}}^{2}} \\ 
    &\le \sum_{j=1}^{m_{t}}\sqrt{2T_{k,j}d\log(1+\frac{T}{\lambda d})}\\
    &\le \sqrt{2m(T_{k}-T_{k-1})d\log(1+\frac{T}{\lambda d})}
    \end{aligned}
\end{equation}
For all phases:
\begin{equation}
    \begin{aligned}
        \sum_{k=1}^{s}\sqrt{2m(T_{k+1}-T_{k})d\log(1+\frac{T}{\lambda d})} &\le \sqrt{2\sum_{k=1}^{s}1\sum_{k=1}^{s}(T_{k+1}-T_{k})md\log(1+\frac{T}{\lambda d})} \\
        &\le \sqrt{2mdT\log(T)\log(1+\frac{T}{\lambda d})}
    \end{aligned}
\end{equation}
Similarly, for the second term in $C_{a_{t}}$:
\begin{equation}
    \begin{aligned}
       &\sum_{t=T_{k-1}}^{T_{k}}\sum_{s \in [t-1] \atop i_{s} \in \overline{V}_{t}}\epsilon_{*}\lvert\boldsymbol{x}_{a_t}^{\mathrm{T}}\boldsymbol{\overline{M}}_{\overline{V}_{t,t-1}}^{-1}\boldsymbol{x}_{a_s}\rvert \notag\\&=\sum_{t=T_{k-1}}^{T_{k}}\sum_{j=1}^{m_{t}}\mathbb{I}\{i \in \overline{V}_{t,j}\}\sum_{s \in [t-1] \atop i_{s} \in \overline{V}_{t,j}}\epsilon_{*} \lvert\boldsymbol{x}_{a_t}^{\mathrm{T}}\boldsymbol{\overline{M}}_{\overline{V}_{t,j}^{-1}}\boldsymbol{x}_{a_s}\rvert\\
        &\le \epsilon_{*}\sum_{t=T_{k-1}}^{T_{k}}\sum_{j=1}^{m_t}\mathbb{I}\{i \in \overline{V}_{t,j}\}\sqrt{\sum_{s \in [t-1] \atop i_{s} \in \overline{V}_{t,j}} 1\sum_{s \in [t-1] \atop i_{s} \in \overline{V}_{t,j}}\lvert\boldsymbol{x}_{a_t}^{T}\boldsymbol{\overline{M}}_{\overline{V}_{t,j}^{-1}}\boldsymbol{x}_{a_s}\rvert^{2} }\\
        &\le \epsilon_{*}\sum_{t=T_{k-1}}^{T_{k}}\sum_{j=1}^{m_{t}}\mathbb{I}\{i \in \overline{V}_{t,j}\}\sqrt{T_{k,j}\norm{\boldsymbol{x}_{a_t}}_{\boldsymbol{\overline{M}}_{\overline{V}_{t, j}}^{-1}}^{2}}\\
        &\le \epsilon_{*}\sum_{t=T_{k-1}}^{T_{k}}\sqrt{\sum_{j=1}^{m_{t}}\mathbb{I}\{i \in \overline{V}_{t,j}\}\sum_{j=1}^{m_t}\mathbb{I}\{i \in \overline{V}_{t,j}\}T_{k,j}\norm{\boldsymbol{x}_{a_t}}_{\boldsymbol{\overline{M}}_{\overline{V}_{t, j}}^{-1}}^{2}}\\
        &\le \epsilon_{*}\sqrt{(T_{k}-T_{k-1})}\sum_{t=T_{k-1}}^{T_{k}}\sqrt{\sum_{j=1}^{m_t}\mathbb{I}\{i \in \overline{V}_{t,j}\}\norm{\boldsymbol{x}_{a_t}}_{\boldsymbol{\overline{M}}_{\overline{V}_{t, j}}^{-1}}^{2}}\\
        &\le \epsilon_{*}(T_{k}-T_{k-1})\sqrt{2md\log(1+\frac{T}{\lambda d})}
    \end{aligned}
\end{equation}
Then for all phases this term can be bounded by $\epsilon_{*}T\sqrt{2md\log(1+\frac{T}{\lambda d})}$.\\
Thus the total regret can be bounded by:
\begin{small}
    \begin{align*}
     R_{t} &\le 2\sqrt{2mTd\log(T)\log(1+\frac{T}{\lambda d})}(\sqrt{2\log(T)+d\log(1+\frac{T}{\lambda d})}+2\sqrt{\lambda}) \\&+2\epsilon_{*}T\sqrt{2md\log(1+\frac{T}{\lambda d})} + 2\epsilon_{*}T + 12\epsilon_{*}\sqrt{\frac{1}{2\Tilde{\lambda}_{x}}}T + \frac{3\epsilon_{*}\sqrt{2d}}{\Tilde{\lambda}_{x}^{\frac{3}{2}}}T +2T_{1} + u\log(T)  + 5
\end{align*}
\end{small}

where $T_{1} = 16u\log(\frac{u}{\delta})+4u\max\{\frac{16}{\Tilde{\lambda}_x^2}\log(\frac{8d}{\Tilde{\lambda}_x^2\delta}),
        \frac{8d}{\Tilde{\lambda}_x(\frac{\gamma_1}{6}-\epsilon_*\sqrt{\frac{1}{2\Tilde{{\lambda}}_x}})^2}\log(\frac{u}{\delta})\}$

\subsection{More Experiments}\label{appendix: more experiments}
\begin{figure*}
\subfigure[Known Misspecification Level]{
\includegraphics[scale=0.35]{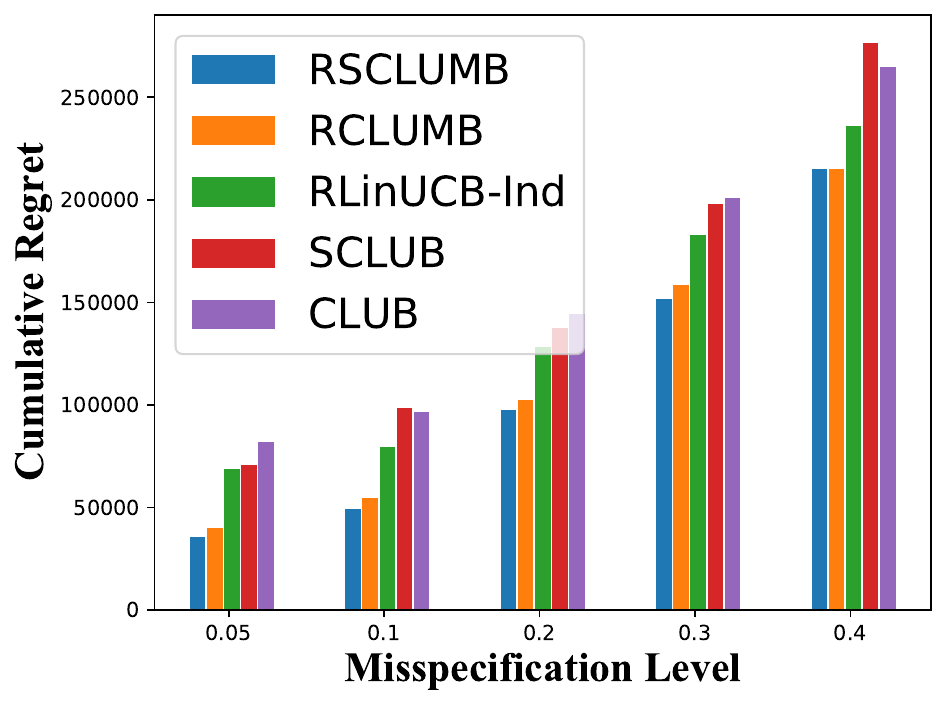}
}
    \subfigure[Unknown Misspecification Level]{
    \includegraphics[scale=0.35]{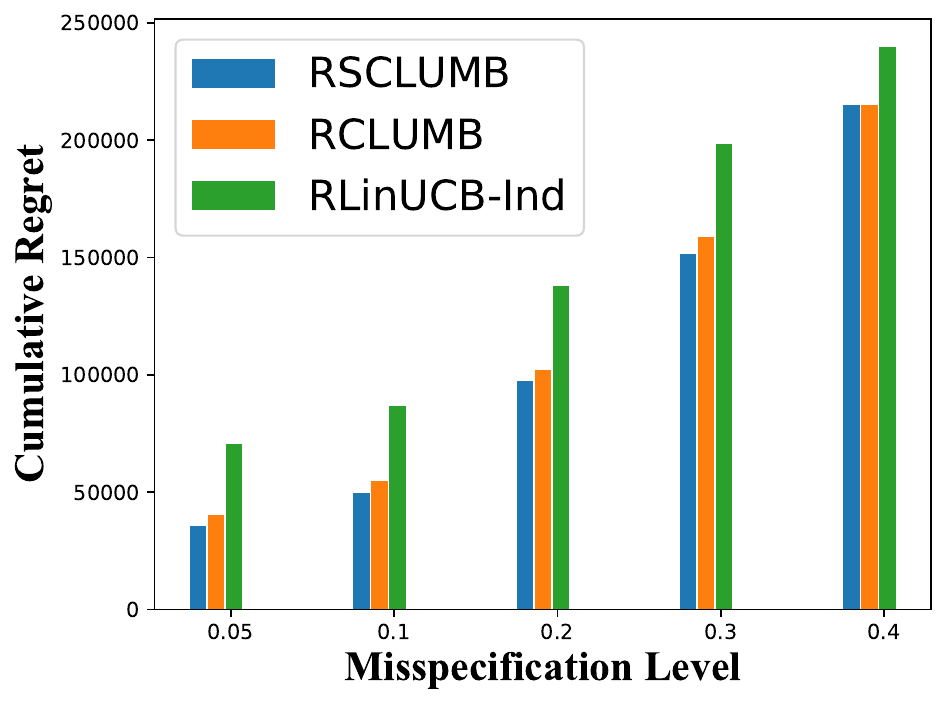}
    }
    \caption{The cumulative regret of the algorithms under different scales of misspecification level.}
    \label{fig:ablation}
\end{figure*}
For ablation study, we test our algorithms' performance under different scales of deviation. We test RCLUMB and RSCLUMB when $\epsilon^{*} = 0.05, 0.1, 0.2 , 0.3  \text{ and } 0.4$ in both misspecification level known and unknown cases. In the known case, we set $\epsilon^{*}$ according to the real misspecification level, and we compare our algorithms' performance to the baselines except LinUCB and CW-OFUL which perform worst; in the unknown case, we keep $\epsilon^{*}=0.2$, and we compare our algorithms to RLinUCB-Ind as only it has the pre-spicified parameter $\epsilon^{*}$ among the baselines. The results are shown in Fig.\ref{fig:ablation}. We plot each algorithm's final cumulative regret under different misspecification levels. All the algorithms' performances get worse when the deviation gets larger, and our two algorithms always perform better than the baselines. Besides, the regrets in the unknown case are only slightly larger than the known case. These results can match our theoretical results and again show our algorithms' effectiveness, as well as verify that our algorithm can handle the unknown misspecification level.
\section{Appendix for chapter \ref{chapter: detect}}
\label{appendix}

\subsection{Proof of Lemma \ref{sufficient time}}
\label{appendix: proof of sufficient time}
We first prove the following result:\\
With probability at least $1-\delta$ for some $\delta\in(0,1)$, at any $t\in[T]$:
\begin{equation}
    \norm{\hat{\btheta}_{i,t}-\btheta^{j(i)}}_2\leq\frac{\beta(T_{i,t},\frac{\delta}{u})}{\sqrt{\lambda+\lambda_{\text{min}}(\bM_{i,t})}}, \forall{i\in\mathcal{U}}\label{norm difference bound}\,,
\end{equation}
where $\beta(T_{i,t},\frac{\delta}{u})\triangleq\sqrt{2\log(\frac{u}{\delta})+d\log(1+\frac{T_{i,t}}{\lambda d})}+\sqrt{\lambda}+\alpha C$.
\begin{align}    \hat{\btheta}_{i,t}-\btheta^{j(i)}&=(\lambda \bI+\bM_{i,t})^{-1}\bb_{i,t}-\btheta^{j(i)}\nonumber\\&=(\lambda \bI+\sum_{s\in[t]\atop i_s=i}w_{i_s,s}\bx_{a_s}\bx_{a_s}^{\top})^{-1}\sum_{s\in[t]\atop i_s=i}w_{i_s,s}\bx_{a_s}r_{s}-\btheta^{j(i)}\nonumber\\
&=(\lambda \bI+\sum_{s\in[t]\atop i_s=i}w_{i_s,s}\bx_{a_s}\bx_{a_s}^{\top})^{-1}\bigg(\sum_{s\in[t]\atop i_s=i}w_{i_s,s}\bx_{a_s}(\bx_{a_s}^{\top}\btheta_{i_s}+\eta_s+c_s)\bigg)-\btheta^{j(i)}\nonumber\\
    &=(\lambda \bI+\sum_{s\in[t]\atop i_s=i}w_{i_s,s}\bx_{a_s}\bx_{a_s}^{\top})^{-1}\bigg[(\lambda \bI+\sum_{s\in[t]\atop i_s=i}w_{i_s,s}\bx_{a_s}\bx_{a_s}^{\top})\btheta^{j(i)}-\lambda\btheta^{j(i)}\nonumber\\
    &\quad\quad+\sum_{s\in[t]\atop i_s=i}w_{i_s,s}\bx_{a_s}\eta_s+\sum_{s\in[t]\atop i_s=i}w_{i_s,s}\bx_{a_s}c_s\bigg]-\btheta^{j(i)}\nonumber\\
    &=-\lambda\bM_{i,t}^{\prime-1}\btheta^{j(i)}+{\bM_{i,t}^{\prime-1}}\sum_{s\in[t]\atop i_s=i}w_{i_s,s}\bx_{a_s}\eta_s+\bM_{i,t}^{\prime-1}\sum_{s\in[t]\atop i_s=i}w_{i_s,s}\bx_{a_s}c_s\nonumber\,,
\end{align}
where we denote ${\bM_{i,t}^{\prime}}=\bM_{i,t}+\lambda\bI$, and the above equations hold by definition.

Therefore, we have \begin{equation}
    \norm{\hat{\btheta}_{i,t}-\btheta^{j(i)}}_2\leq\lambda\norm{\bM_{i,t}^{\prime-1}\btheta^{j(i)}}_2+\norm{\bM_{i,t}^{\prime-1}\sum_{s\in[t]\atop i_s=i}w_{i_s,s}\bx_{a_s}\eta_s}_2+\norm{\bM_{i,t}^{\prime-1}\sum_{s\in[t]\atop i_s=i}w_{i_s,s}\bx_{a_s}c_s}_2\,.\label{difference bound parts}
\end{equation}
We then bound the three terms in Eq.(\ref{difference bound parts}) one by one. For the first term:
\begin{equation}
    \lambda\norm{\bM_{i,t}^{\prime-1}\btheta^{j(i)}}_2\leq \lambda\norm{{\bM_{i,t}^{\prime-\frac{1}{2}}}}_2^2\norm{\btheta^{j(i)}}_2\leq\frac{\sqrt{\lambda}}{\sqrt{\lambda_{\text{min}}({\bM_{i,t}^{\prime}})}}\label{first term}\,,
\end{equation}
where we use the Cauchy–Schwarz inequality, the inequality for the operator norm of matrices, and the fact that $\lambda_{\text{min}}({\bM_{i,t}^{\prime}})\geq\lambda$.

For the second term in Eq.(\ref{difference bound parts}), we have 
\begin{align}
    \norm{\bM_{i,t}^{\prime-1}\sum_{s\in[t]\atop i_s=i}w_{i_s,s}\bx_{a_s}\eta_s}_2
    &\leq\norm{{\bM_{i,t}^{\prime-\frac{1}{2}}}\sum_{s\in[t]\atop i_s=i}w_{i_s,s}\bx_{a_s}\eta_s}_2\norm{{\bM_{i,t}^{\prime-\frac{1}{2}}}}_2\label{operator norm}\\
    &=\frac{\norm{\sum_{s\in[t]\atop i_s=i}w_{i_s,s}\bx_{a_s}\eta_s}_{\bM_{i,t}^{\prime-1}}}{\sqrt{\lambda_{\text{min}}({\bM_{i,t}^{\prime}})}}\label{Courant–Fischer1}\,,
\end{align}
where Eq.(\ref{operator norm}) follows by the Cauchy–Schwarz inequality and the inequality for the operator norm of matrices, and Eq.(\ref{Courant–Fischer1}) follows by the Courant-Fischer theorem.

Let $\Tilde{\bx}_{s}\triangleq \sqrt{w_{i_s,s}}\bx_{a_s}$, $\Tilde{\eta}_{s}\triangleq \sqrt{w_{i_s,s}}\eta_s$, then we have: $\norm{\Tilde{\bx}_{s}}_2\leq\norm{\sqrt{w_{i_s,s}}}_2\norm{\bx_{a_s}}_2\leq1$, $\Tilde{\eta}_{s}$ is still 1-sub-gaussian (since $\eta_s$ is 1-sub-gaussian and $\sqrt{w_{i_s,s}}\leq1$), ${\bM_{i,t}^{\prime}}=\lambda\bI+\sum_{s\in[t]\atop i_s=i}\Tilde{\bx}_{s}\Tilde{\bx}_{s}^{\top}$, and the nominator in Eq.(\ref{Courant–Fischer1}) becomes $\norm{\sum_{s\in[t]\atop i_s=i}\Tilde{\bx}_{s}\Tilde{\eta}_{s}}_{\bM_{i,t}^{\prime-1}}$. Then, following Theorem 1 in \cite{abbasi2011improved} and by union bound, with probability at least $1-\delta$ for some $\delta\in(0,1)$, for any $i\in\mathcal{U}$, we have:
\begin{align}
    \norm{\sum_{s\in[t]\atop i_s=i}w_{i_s,s}\bx_{a_s}\eta_s}_{\bM_{i,t}^{\prime-1}}&=\norm{\sum_{s\in[t]\atop i_s=i}\Tilde{\bx}_{s}\Tilde{\eta}_{s}}_{\bM_{i,t}^{\prime-1}}\nonumber\\
    &\leq\sqrt{2\log(\frac{u}{\delta})+\log(\frac{\text{det}({\bM_{i,t}^{\prime}})}{\text{det}(\lambda\bI)})}\nonumber\\
    &\leq \sqrt{2\log(\frac{u}{\delta})+d\log(1+\frac{T_{i,t}}{\lambda d})}\label{det inequality}\,,
\end{align}
where $\text{det}(\cdot)$ denotes the determinant of matrix arguement, Eq.(\ref{det inequality}) is because $\text{det}({\bM_{i,t}^{\prime}})\leq\Bigg(\frac{\text{trace}(\lambda\bI+\sum_{s\in[t]\atop i_s=i}w_{i_s,s}\bx_{a_s}\bx_{a_s}^{\top})}{d}\Bigg)^d\leq\big(\frac{\lambda d+T_{i,t}}{d}\big)^d$, and $\text{det}(\lambda\bI)=\lambda^d$.

For the third term in Eq.(\ref{difference bound parts}), we have 
\begin{align}
\norm{\bM_{i,t}^{\prime-1}\sum_{s\in[t]\atop i_s=i}w_{i_s,s}\bx_{a_s}c_s}_2&\leq \norm{{\bM_{i,t}^{\prime-\frac{1}{2}}}\sum_{s\in[t]\atop i_s=i}w_{i_s,s}\bx_{a_s}c_s}_2\norm{{\bM_{i,t}^{\prime-\frac{1}{2}}}^{}}_2\label{operator norm 2}\\
&=\frac{\norm{\sum_{s\in[t]\atop i_s=i}w_{i_s,s}\bx_{a_s}c_s}_{\bM_{i,t}^{\prime-1}}}{\sqrt{\lambda_{\text{min}}({\bM_{i,t}^{\prime}})}}\label{Courant–Fischer2}\\
&\leq \frac{\sum_{s\in[t]\atop i_s=i}\left|c_s\right|w_{i,s}\norm{\bx_{a_s}}_{\bM_{i,t}^{\prime-1}}}{\sqrt{\lambda_{\text{min}}({\bM_{i,t}^{\prime}})}}\nonumber\\
&\leq \frac{\alpha C}{\sqrt{\lambda_{\text{min}}({\bM_{i,t}^{\prime}})}}\label{definition of C and w}
\end{align}
where Eq.(\ref{operator norm 2}) follows by the Cauchy–Schwarz inequality and the inequality for the operator norm of matrices, Eq.(\ref{Courant–Fischer2}) follows by the Courant-Fischer theorem, and Eq.(\ref{definition of C and w}) is because by definition $w_{i,s}\leq \frac{\alpha}{\norm{\bx_{a_s}}_{\bM_{i,s}^{\prime-1}}}\leq\frac{\alpha}{\norm{\bx_{a_s}}_{\bM_{i,t}^{\prime-1}}}$ (since $\bM_{i,t}^{\prime}\succeq\bM_{i,s}^{\prime}$, $\bM_{i,s}^{\prime-1}\succeq\bM_{i,t}^{\prime-1}$, $\norm{\bx_{a_s}}_{\bM_{i,s}^{\prime-1}}\geq\norm{\bx_{a_s}}_{\bM_{i,t}^{\prime-1}}$), $\sum_{t=1}^{T} \left|c_t\right|\leq C$.

Combining the above bounds of these three terms, we can get that Eq.(\ref{norm difference bound}) holds.

We then prove the following technical lemma.
\begin{lemma}
    \label{assumption}
    Under Assumption \ref{assumption3}, at any time $t$, for any fixed unit vector $\btheta \in \RR^d$
    \begin{equation}
        \mathbb{E}_t[(\btheta^{\top}\bx_{a_t})^2|\left|\mathcal{A}_t\right|]\geq\tilde{\lambda}_x\triangleq\int_{0}^{\lambda_x} (1-e^{-\frac{(\lambda_x-x)^2}{2\sigma^2}})^{K} dx\,,
    \end{equation}
    where $K$ is the upper bound of $\left|\mathcal{A}_t\right|$ for any $t$.
\end{lemma}
\begin{proof}
The proof of this lemma mainly follows the proof of Claim 1 in \cite{gentile2014online}, but with more careful analysis, since their assumption on the arm generation distribution is more stringent than our Assumption \ref{assumption3} by putting more restrictions on the variance upper bound $\sigma^2$ (specifically, they require $\sigma^2\leq\frac{\lambda^2}{8\log (4K)}$).

Denote the feasible arms at round $t$ by ${\mathcal{A}_t=\{\bx_{t,1},\bx_{t,2},\ldots,\bx_{t,\left|\mathcal{A}_t\right|}\}}$. 
Consider the corresponding i.i.d. random variables $\theta_i=(\btheta^{\top}\bx_{t,i})^2-\mathbb{E}_t[(\btheta^{\top}\bx_{t,i})^2|\left|\mathcal{A}_t\right|], i=1,2,\ldots,\left|\mathcal{A}_t\right|$. By Assumption \ref{assumption3}, $\theta_i$ s are sub-Gaussian random variables with variance bounded by $\sigma^2$. Therefore, for any $\alpha>0$ and any $i\in[\left|\mathcal{A}_t\right|]$, we have:
\begin{equation*}
    \mathbb{P}_t(\theta_i<-\alpha|\left|\mathcal{A}_t\right|)\leq e^{-\frac{\alpha^2}{2\sigma^2}}\,,
\end{equation*}
where we use $\mathbb{P}_t(\cdot)$ to be the shorthand for the conditional probability $\mathbb{P}(\cdot|(i_1,\mathcal{A}_1,r_1),\ldots,(i_{t-1},\mathcal{A}_{t-1},r_{t-1}),i_t)$.

By Assumption \ref{assumption3}, we can also get that
$\mathbb{E}_t[(\btheta^{\top}\bx_{t,i})^2|\left|\mathcal{A}_t\right|=\mathbb{E}_t[\btheta^{\top}\bx_{t,i}\bx_{t,i}^{\top}\btheta|\left|\mathcal{A}_t\right|]\geq\lambda_{\text{min}}(\mathbb{E}_{\bx\sim \rho}[\bx\bx^{\top}])\geq\lambda_x$.
With these inequalities above, we can get
\begin{equation*}
    \mathbb{P}_t(\min_{i=1,\ldots,\left|\mathcal{A}_t\right|}(\btheta^{\top}\bx_{t,i})^2\geq \lambda_x-\alpha|\left|\mathcal{A}_t\right|)\geq (1-e^{-\frac{\alpha^2}{2\sigma^2}})^K\,.
\end{equation*}

Therefore, we can get
\begin{align}
\mathbb{E}_t[(\btheta^{\top}\bx_{a_t})^2|\left|\mathcal{A}_t\right|]
&\geq\mathbb{E}_t[\min_{i=1,\ldots,\left|\mathcal{A}_t\right|}(\btheta^{\top}\bx_{t,i})^2|\left|\mathcal{A}_t\right|]\notag\\
&\geq\int_{0}^{\infty} \mathbb{P}_t (\min_{i=1,\ldots,\left|\mathcal{A}_t\right|}(\btheta^{\top}\bx_{t,i})^2\geq x|\left|\mathcal{A}_t\right|) dx\notag\\
&\geq \int_{0}^{\lambda_x} (1-e^{-\frac{(\lambda_x-x)^2}{2\sigma^2}})^{K} dx\triangleq\tilde{\lambda}_x\notag
\end{align}
\end{proof}

Note that $w_{i,s}=\min\{1,\frac{\alpha}{\norm{\bx_{a_s}}_{\bM_{i,t}^{\prime-1}}}\}$, and we have $$\frac{\alpha}{\norm{\bx_{a_s}}_{\bM_{i,t}^{\prime-1}}}=\frac{\alpha}{\sqrt{\bx_{a_s}^{\top}\bM_{i,t}^{\prime-1}\bx_{a_s}}}\geq\frac{\alpha}{\sqrt{\lambda_{\text{min}}(\bM_{i,t}^{\prime-1})}}=\alpha\sqrt{\lambda_{\text{min}}(\bM_{i,t}^{\prime})}\geq\alpha\sqrt{\lambda}.$$
Since $\alpha\sqrt{\lambda}<1$ typically holds, we have $w_{i,s}\geq\alpha\sqrt{\lambda}$.

Then, with the item regularity assumption stated in Assumption \ref{assumption3}, the technical Lemma \ref{assumption}, together with Lemma 7 in \cite{li2018online}, with probability at least $1-\delta$, for a particular user $i$, at any $t$ such that $T_{i,t}\geq\frac{16}{\tilde{\lambda}_x^2}\log(\frac{8d}{\tilde{\lambda}_x^2\delta})$, we have:
\begin{equation}
    \lambda_{\text{min}}(\bM_{i,t}^{\prime})\geq2\alpha\sqrt{\lambda}\tilde{\lambda}_x T_{i,t}+\lambda\,.
    \label{min eigen}
\end{equation}
With this result, together with Eq.(\ref{norm difference bound}), we can get that for any $t$ such that $T_{i,t}\geq\frac{16}{\tilde{\lambda}_x^2}\log(\frac{8d}{\tilde{\lambda}_x^2\delta})$, with probability at least $1-\delta$ for some $\delta\in(0,1)$, $\forall{i\in\mathcal{U}}$, we have:
\begin{align}
   \norm{\hat{\btheta}_{i,t}-\btheta^{j(i)}}_2&\leq\frac{\beta(T_{i,t},\frac{\delta}{u})}{\sqrt{\lambda_{\text{min}}(\bM_{i,t}^{\prime})}}\nonumber\\
   & \leq\frac{\beta(T_{i,t},\frac{\delta}{u})}{\sqrt{2\alpha\sqrt{\lambda}\tilde{\lambda}_x T_{i,t}+\lambda}}\nonumber\\
   &\leq \frac{\beta(T_{i,t},\frac{\delta}{u})}{\sqrt{2\alpha\sqrt{\lambda}\tilde{\lambda}_x T_{i,t}}}\nonumber\\
   &=\frac{\sqrt{2\log(\frac{u}{\delta})+d\log(1+\frac{T_{i,t}}{\lambda d})}+\sqrt{\lambda}+\alpha C}{\sqrt{2\alpha\sqrt{\lambda}\tilde{\lambda}_x T_{i,t}}}\,.\label{norm bound with min eigen}
\end{align}

Then, we want to find a sufficient time $T_{i,t}$ for a fixed user $i$ such that 
\begin{equation}
 \norm{\hat{\btheta}_{i,t}-\btheta^{j(i)}}_2<\frac{\gamma}{4}\label{T_0 gamma/4} \,. 
\end{equation} To do this, with Eq.(\ref{norm bound with min eigen}), we can get it by letting
\begin{align}
    \frac{\sqrt{\lambda}}{\sqrt{2\alpha\sqrt{\lambda}\tilde{\lambda}_x T_{i,t}}}&<\frac{\gamma}{12}\,,\label{bound time 1}\\
    \frac{\alpha C}{\sqrt{2\alpha\sqrt{\lambda}\tilde{\lambda}_x T_{i,t}}}&<\frac{\gamma}{12}\,,\label{bound time 2}\\
        \frac{\sqrt{2\log(\frac{u}{\delta})+d\log(1+\frac{T_{i,t}}{\lambda d})}}{\sqrt{2\alpha\sqrt{\lambda}\tilde{\lambda}_x T_{i,t}}}&<\frac{\gamma}{12}\label{bound time 3}\,.
\end{align}
For Eq.(\ref{bound time 1}), we can get
\begin{equation}
    T_{i,t}>\frac{72\sqrt{\lambda}}{\alpha\gamma^2\tilde{\lambda}_x}\label{bound time result1}\,.
\end{equation}
For Eq.(\ref{bound time 2}), we can get
\begin{equation}
    T_{i,t}>\frac{72\alpha C^2}{\gamma^2\sqrt{\lambda}\tilde{\lambda}_x}\,.\label{bound time result2}
\end{equation}
For Eq.(\ref{bound time 3}), we have
\begin{equation}
\frac{2\log(\frac{u}{\delta})+d\log(1+\frac{T_{i,t}}{\lambda d})}{2\alpha\sqrt{\lambda}\tilde{\lambda}_x T_{i,t}}<\frac{\gamma^2}{144}\,. \label{time bound 3 equivalent}
\end{equation}
Then it is sufficient to get Eq.(\ref{time bound 3 equivalent}) if the following holds
\begin{align}
    \frac{2\log(\frac{u}{\delta})}{2\alpha\sqrt{\lambda}\tilde{\lambda}_x T_{i,t}}&<\frac{\gamma^2}{288}\label{time bound 4}\,,\\
    \frac{d\log(1+\frac{T_{i,t}}{\lambda d})}{2\alpha\sqrt{\lambda}\tilde{\lambda}_x T_{i,t}}&<\frac{\gamma^2}{288}\label{time bound 5}\,.
\end{align}
For Eq.(\ref{time bound 4}), we can get 
\begin{equation}
    T_{i,t}>\frac{288\log(\frac{u}{\delta})}{\gamma^2\alpha\sqrt{\lambda}\tilde{\lambda}_x }\label{time bound result 3.1}
\end{equation}
For Eq.(\ref{time bound 5}), we can get 
\begin{equation}
    T_{i,t}>\frac{144d}{\gamma^2\alpha\sqrt{\lambda}\tilde{\lambda}_x}\log(1+\frac{T_{i,t}}{\lambda d})\,.\label{bound time 6}
\end{equation}
Following Lemma 9 in \cite{li2018online}, we can get the following sufficient condition for Eq.(\ref{bound time 6}):
\begin{equation}
    T_{i,t}>\frac{288d}{\gamma^2\alpha\sqrt{\lambda}\tilde{\lambda}_x}\log(\frac{288}{\gamma^2\alpha\sqrt{\lambda}\tilde{\lambda}_x})\,.\label{bound time result 4}
\end{equation}
Then, since typically $\frac{u}{\delta}>\frac{288}{\gamma^2\alpha\sqrt{\lambda}\tilde{\lambda}_x}$, we can get the following sufficient condition for Eq.(\ref{time bound result 3.1}) and Eq.(\ref{bound time result 4})
\begin{equation}
    \label{final time bound 1}
    T_{i,t}>\frac{288d}{\gamma^2\alpha\sqrt{\lambda}\tilde{\lambda}_x}\log(\frac{u}{\delta})\,.
\end{equation}
Together with Eq.(\ref{bound time result1}), Eq.(\ref{bound time result2}), and the condition for Eq.(\ref{min eigen}) we can get the following sufficient condition for Eq.(\ref{T_0 gamma/4}) to hold
\begin{equation}
    T_{i,t}>\max\{\frac{288d}{\gamma^2\alpha\sqrt{\lambda}\tilde{\lambda}_x}\log(\frac{u}{\delta}), \frac{16}{\tilde{\lambda}_x^2}\log(\frac{8d}{\tilde{\lambda}_x^2\delta}),\frac{72\sqrt{\lambda}}{\alpha\gamma^2\tilde{\lambda}_x},\frac{72\alpha C^2}{\gamma^2\sqrt{\lambda}\tilde{\lambda}_x}\}\,.\label{final T_0 for a single user}
\end{equation}
Then, with Assumption \ref{assumption2} on the uniform arrival of users, following Lemma 8 in \cite{li2018online}, and by union bound, we can get that with probability at least $1-\delta$, for all
\begin{equation}
 t\geq T_0\triangleq16u\log(\frac{u}{\delta})+4u\max\{\frac{288d}{\gamma^2\alpha\sqrt{\lambda}\tilde{\lambda}_x}\log(\frac{u}{\delta}), \frac{16}{\tilde{\lambda}_x^2}\log(\frac{8d}{\tilde{\lambda}_x^2\delta}),\frac{72\sqrt{\lambda}}{\alpha\gamma^2\tilde{\lambda}_x},\frac{72\alpha C^2}{\gamma^2\sqrt{\lambda}\tilde{\lambda}_x}\}\,,
\end{equation}
Eq.(\ref{final time bound 1}) holds for all $i\in \mathcal{U}$, and therefore Eq.(\ref{T_0 gamma/4}) holds for all $i\in\mathcal{U}$. With this, we can show that RCLUB-WCU will cluster all the users correctly after $T_0$. First, if RCLUB-WCU deletes the edge $(i,l)$, then user $i$ and user $j$ belong to different ground-truth clusters, i.e., $\norm{\btheta_i-\btheta_l}_2>0$. This is because by the deletion rule of the algorithm, the concentration bound, and triangle inequality,
$\norm{\btheta_i-\btheta_l}_2=\norm{\btheta^{j(i)}-\btheta^{j(l)}}_2\geq\norm{\hat{\btheta}_{i,t}-\hat{\btheta}_{l,t}}_2-\norm{\btheta^{j(l)}-\btheta_{l,t}}_2-\norm{\btheta^{j(i)}-\btheta_{i,t}}_2>0$. Second, we show that if $\norm{\btheta_i-\btheta_l}\geq\gamma$, RCLUB-WCU will delete the edge $(i,l)$. This is because if $\norm{\btheta_i-\btheta_l}\geq\gamma$, then by the triangle inequality, and $\norm{\hat{\btheta}_{i,t}-\btheta^{j(i)}}_2<\frac{\gamma}{4}$, $\norm{\hat{\btheta}_{l,t}-\btheta^{j(l)}}_2<\frac{\gamma}{4}$, $\btheta_i=\btheta^{j(i)}$, $\btheta_l=\btheta^{j(l)}$, we have $\norm{\hat{\btheta}_{i,t}-\hat{\btheta}_{l,t}}_2\geq\norm{\btheta_i-\btheta_l}-\norm{\hat{\btheta}_{i,t}-\btheta^{j(i)}}_2-\norm{\hat{\btheta}_{l,t}-\btheta^{j(l)}}_2>\gamma-\frac{\gamma}{4}-\frac{\gamma}{4}=\frac{\gamma}{2}>\frac{\sqrt{\lambda}+\sqrt{2\log(\frac{u}{\delta})+d\log(1+\frac{T_{i,t}}{\lambda d})}}{\sqrt{\lambda+2\tilde{\lambda}_x T_{i,t}}}+\frac{\sqrt{\lambda}+\sqrt{2\log(\frac{u}{\delta})+d\log(1+\frac{T_{l,t}}{\lambda d})}}{\sqrt{\lambda+2\tilde{\lambda}_x T_{l,t}}}$, which will trigger the deletion condition Line \ref{alg:delete} in Algo.\ref{club-rac}.
\subsection{Proof of Lemma \ref{concentration bound}}
\label{appendix: proof of concentration bound}
After $T_0$, if the clustering structure is correct, \emph{i.e.}, $V_{t}=V_{j(i_t)}$, then we have
\begin{align}
    \hat{\boldsymbol{\theta}}_{V_{t},t-1} - \boldsymbol{\theta}_{i_t}&=\boldsymbol{M}_{{V}_{t},t-1}^{-1}\boldsymbol{b}_{{V}_{t},t-1}- \boldsymbol{\theta}_{i_t}\nonumber\\
    &=(\lambda\bI+\sum_{s\in[t-1]\atop i_s\in V_t}w_{i_s,s}\bx_{a_s}\bx_{a_s}^{\top})^{-1}(\sum_{s\in[t-1]\atop i_s\in V_t}w_{i_s,s}\bx_{a_s}r_s)-\btheta_{i_t}\nonumber\\
    &=(\lambda\bI+\sum_{s\in[t-1]\atop i_s\in V_t}w_{i_s,s}\bx_{a_s}\bx_{a_s}^{\top})^{-1}\big(\sum_{s\in[t-1]\atop i_s\in V_t}w_{i_s,s}\bx_{a_s}(\bx_{a_s}^{\top}\btheta_{i_t}+\eta_s+c_s)\big)-\btheta_{i_t}\label{true cluster}\\
    &=(\lambda\bI+\sum_{s\in[t-1]\atop i_s\in V_t}w_{i_s,s}\bx_{a_s}\bx_{a_s}^{\top})^{-1}\bigg(\sum_{s\in[t-1]\atop i_s\in V_t}(w_{i_s,s}\bx_{a_s}\bx_{a_s}^{\top}+\lambda\bI)\btheta_{i_t}-\lambda\btheta_{i_t}\nonumber\\
    &\quad\quad+\sum_{s\in[t-1]\atop i_s\in V_t}w_{i_s,s}\bx_{a_s}\eta_s+\sum_{s\in[t-1]\atop i_s\in V_t}w_{i_s,s}\bx_{a_s}c_s)\bigg)-\btheta_{i_t}\nonumber\\
    &=-\lambda\boldsymbol{M}_{{V}_{t},t-1}^{\prime-1}\btheta_{i_t}-\boldsymbol{M}_{{V}_{t},t-1}^{\prime-1}\sum_{s\in[t-1]\atop i_s\in V_t}w_{i_s,s}\bx_{a_s}\eta_s+\boldsymbol{M}_{{V}_{t},t-1}^{\prime-1}\sum_{s\in[t-1]\atop i_s\in V_t}w_{i_s,s}\bx_{a_s}c_s\nonumber\,,
\end{align}
where we denote $\boldsymbol{M}_{{V}_{t},t-1}^\prime=\boldsymbol{M}_{{V}_{t},t-1}+\lambda\bI$, and Eq.(\ref{true cluster}) is because $V_t=V_{j(i_t)}$ thus $\btheta_{i_s}=\btheta_{i_t}, \forall i_s\in V_t$.

Therefore, we have 
\begin{small}
    \begin{align}
&\left|\bx_{a}^{\top}(\hat{\boldsymbol{\theta}}_{V_{t},t-1} - \boldsymbol{\theta}_{i_t})\right|\notag\\&\leq \lambda\left|\bx_{a}^{\top}\boldsymbol{M}_{{V}_{t},t-1}^{\prime-1}\btheta_{i_t}\right|+\left|\bx_{a}^{\top}\boldsymbol{M}_{{V}_{t},t-1}^{\prime-1}\sum_{s\in[t-1]\atop i_s\in V_t}w_{i_s,s}\bx_{a_s}\eta_s\right|+\left|\bx_a^{\top}\boldsymbol{M}_{{V}_{t},t-1}^{\prime-1}\sum_{s\in[t-1]\atop i_s\in V_t}w_{i_s,s}\bx_{a_s}c_s\right|\nonumber\\
&\leq\norm{\bx_{a}}_{\boldsymbol{M}_{{V}_{t},t-1}^{\prime-1}} \bigg(\sqrt{\lambda}+\norm{\sum_{s\in[t-1]\atop i_s\in V_t}w_{i_s,s}\bx_{a_s}\eta_s}_{\boldsymbol{M}_{{V}_{t},t-1}^{\prime-1}}
+\norm{\sum_{s\in[t-1]\atop i_s\in V_t}w_{i_s,s}\bx_{a_s}c_s}_{\boldsymbol{M}_{{V}_{t},t-1}^{\prime-1}}\bigg)\label{operator norm cauchy again231}\,,
\end{align}
\end{small}
where Eq.(\ref{operator norm cauchy again231}) is by Cauchy–Schwarz inequality, matrix operator inequality, and $\left|\bx_{a}^{\top}\boldsymbol{M}_{{V}_{t},t-1}^{\prime-1}\btheta_{i_t}\right|\leq\lambda\norm{\boldsymbol{M}_{{V}_{t},t-1}^{\prime-\frac{1}{2}}}_2\norm{\btheta_{i_t}}_2=\lambda\frac{1}{\sqrt{\lambda_{\text{min}}(\boldsymbol{M}_{{V}_{t},t-1})}}\norm{\btheta_{i_t}}_2\leq\sqrt{\lambda}$ since $\lambda_{\text{min}}(\boldsymbol{M}_{{V}_{t},t-1})\geq\lambda$ and $\norm{\btheta_{i_t}}_2\leq1$.

Let $\Tilde{\bx}_{s}\triangleq \sqrt{w_{i_s,s}}\bx_{a_s}$, $\Tilde{\eta}_{s}\triangleq \sqrt{w_{i_s,s}}\eta_s$, then we have: $\norm{\Tilde{\bx}_{s}}_2\leq\norm{\sqrt{w_{i_s,s}}}_2\norm{\bx_{a_s}}_2\leq1$, $\Tilde{\eta}_{s}$ is still 1-sub-gaussian (since $\eta_s$ is 1-sub-gaussian and $\sqrt{w_{i_s,s}}\leq1$), ${\bM_{i,t}^{\prime}}=\lambda\bI+\sum_{s\in[t]\atop i_s=i}\Tilde{\bx}_{s}\Tilde{\bx}_{s}^{\top}$, and $\norm{\sum_{s\in[t-1]\atop i_s\in V_t}w_{i_s,s}\bx_{a_s}\eta_s}_{\boldsymbol{M}_{{V}_{t},t-1}^{\prime-1}}$ becomes $\norm{\sum_{s\in[t]\atop i_s=i}\Tilde{\bx}_{s}\Tilde{\eta}_{s}}_{\boldsymbol{M}_{{V}_{t},t-1}^{\prime-1}}$. Then, following Theorem 1 in \cite{abbasi2011improved}, with probability at least $1-\delta$ for some $\delta\in(0,1)$, we have:
\begin{align}
    \norm{\sum_{s\in[t-1]\atop i_s\in V_t}w_{i_s,s}\bx_{a_s}\eta_s}_{\boldsymbol{M}_{{V}_{t},t-1}^{\prime-1}}&=\norm{\sum_{s\in[t]\atop i_s=i}\Tilde{\bx}_{s}\Tilde{\eta}_{s}}_{\boldsymbol{M}_{{V}_{t},t-1}^{\prime-1}}\nonumber\\
    &\leq\sqrt{2\log(\frac{u}{\delta})+\log(\frac{\text{det}(\boldsymbol{M}_{{V}_{t},t-1}^{\prime})}{\text{det}(\lambda\bI)})}\nonumber\\
    &\leq \sqrt{2\log(\frac{u}{\delta})+d\log(1+\frac{T}{\lambda d})}\,,\label{result 231123123}
\end{align}
And for $\norm{\sum_{s\in[t-1]\atop i_s\in V_t}w_{i_s,s}\bx_{a_s}c_s}_{\boldsymbol{M}_{{V}_{t},t-1}^{\prime-1}}$, we have 
\begin{align}
    \norm{\sum_{s\in[t-1]\atop i_s\in V_t}w_{i_s,s}\bx_{a_s}c_s}_{\boldsymbol{M}_{{V}_{t},t-1}^{\prime-1}}\leq \sum_{s\in[t-1]\atop i_s\in V_t}w_{i_s,s}\left|c_s\right|\norm{\bx_{a_s}}_{\boldsymbol{M}_{{V}_{t},t-1}^{\prime-1}}\leq \alpha C\,\label{result 123},
\end{align}
where we use $\sum_{t=1}^{T}\left|c_t\right|\leq C$, $w_{i_s,s}\leq\frac{\alpha}{\norm{\bx_{a_s}}_{\boldsymbol{M}_{i_s,t-1}^{\prime-1}}}\leq\frac{\alpha}{\norm{\bx_{a_s}}_{\boldsymbol{M}_{{V}_{t},t-1}^{\prime-1}}}$\,.

Plugging Eq.(\ref{result 123}) and Eq.(\ref{result 231123123}) into Eq.(\ref{operator norm cauchy again231}), together with Lemma \ref{sufficient time}, we can complete the proof of Lemma \ref{concentration bound}.
\subsection{Proof of Theorem \ref{thm:main}}
\label{appendix: proof of upper bound}
After $T_0$, we define event
\begin{equation}
    \cE = \{ \text{the algorithm clusters all the users correctly for all } t \geq T_0\}\,.
\end{equation}

% Then denote $\tilde{R}(T)=\sum_{t=T_0}^T r_t$, where $r_t$ denotes the instantaneous regret at $t$, then we have 
Then, with Lemma \ref{sufficient time} and picking $\delta=\frac{1}{T}$, we have
    \begin{equation}
\begin{aligned}
    R(T)&=\PP(\cE)\mathbb{I}\{\cE\}R(T)+\PP(\overline{\cE})\mathbb{I}\{\overline{\cE}\}R(T)\\
    &\leq \mathbb{I}\{\cE\}R(T)+ 4\times\frac{1}{T}\times T\\
    &=\mathbb{I}\{\cE\}R(T)+4\,.\label{regret result 1}
\end{aligned}
\end{equation}
Then it remains to bound $\mathbb{I}\{\cE\}R(T)$. For the first $T_0$ rounds, we can upper bound the regret in the first $T_0$ rounds by $T_0$. After $T_0$, under event $\cE$ and by Lemma \ref{concentration bound}, we have that with probability at least $1-\delta$, for any $\bx_a$:
\begin{equation}
    \left|\boldsymbol{x}_{a}^{\mathrm{T}}(\hat{\boldsymbol{\theta}}_{V_{t},t-1} - \boldsymbol{\theta}_{i_t})\right| \le \beta\norm{\boldsymbol{x_{a}}}_{\boldsymbol{M}^{-1}_{V_t, t-1}}\triangleq C_{a,t}\,.\label{use of con}
\end{equation}

Therefore, for the instantaneous regret $R_t$ at round $t$, with $\cE$, with probability at least $1-\delta$, at $\forall{t\geq T_0}$:
    \begin{equation}
    \begin{aligned}
    R_t&=\bx_{a_t^*}^{\top}\btheta_{i_t}-\bx_{a_t}^{\top}\btheta_{i_t}\\
    &=\bx_{a_t^*}^{\top}(\btheta_{i_t}-\hat{\boldsymbol{\theta}}_{V_{t},t-1})+(\bx_{a_t^*}^{\top}\hat{\boldsymbol{\theta}}_{V_{t},t-1}+C_{a_t^*,t})-(\bx_{a_t}^{\top}\hat{\boldsymbol{\theta}}_{V_{t},t-1}+C_{a_t,t})\\
    &\quad+\bx_{a_t}^{\top}(\hat{\btheta}_{\overline{V}_t,t-1}-\btheta_{i_t})+C_{a_t,t}-C_{a_t^*,t}\\
    &\leq 2C_{a_t,t}\,,
    \end{aligned}
    \label{bound r_t}
\end{equation}
where the last inequality holds by the UCB arm selection strategy in Eq.(\ref{UCB}) and Eq.(\ref{use of con}).

Therefore, for $\mathbb{I}\{\cE\}R(T)$:
    \begin{align}
        \mathbb{I}\{\cE\}R(T)&\leq R(T_0)+\EE [\mathbb{I}\{\cE\}\sum_{t=T_0+1}^{T}R_t]\nonumber\\
    &\leq T_0+2\EE [\mathbb{I}\{\cE\}\sum_{t=T_0+1}^{T}C_{a_t,t}]\,.\label{regret result 2}
\end{align}
Then it remains to bound $\EE [\mathbb{I}\{\cE\}\sum_{t=T_0+1}^{T}C_{a_t,t}]$.
For $\sum_{t=T_0+1}^{T}C_{a_t,t}$, we can distinguish it into two cases:
\begin{align}
\sum_{t=T_0+1}^{T}C_{a_t,t}&\leq\beta\sum_{t=1}^{T}\norm{\boldsymbol{x_{a_t}}}_{\boldsymbol{M}^{-1}_{V_t, t-1}}\nonumber\\
&=\beta\sum_{t\in[T]:w_{i_t,t}=1}\norm{\boldsymbol{x_{a_t}}}_{\boldsymbol{M}^{-1}_{V_t, t-1}}+\beta\sum_{t\in[T]:w_{i_t,t}<1}\norm{\boldsymbol{x_{a_t}}}_{\boldsymbol{M}^{-1}_{V_t, t-1}}\,.\label{regret decompose 2 cases}
\end{align}
Then, we prove the following technical lemma.
\begin{lemma}
\label{technical lemma6}
\label{lemma:4}
    \begin{equation}
    \sum_{t=T_0+1}^{T}\min\{\mathbb{I}\{i_t\in V_j\}\norm{\boldsymbol{x}_{a_t}}_{\boldsymbol{M}_{V_j,t-1}^{-1}}^2,1\}\leq2d\log(1+\frac{T}{\lambda d}), \forall{j\in[m]}\,.
\end{equation}\end{lemma}
\begin{proof}
\begin{align}
    det(\boldsymbol{M}_{V_j,T})
    &=det\bigg(\boldsymbol{M}_{V_j,T-1}+\mathbb{I}\{i_T\in V_j\}\boldsymbol{x}_{a_T}\boldsymbol{x}_{a_T}^{\top}\bigg)\nonumber\\
    &=det(\boldsymbol{M}_{V_j,T-1})det\bigg(\boldsymbol{I}+\mathbb{I}\{i_T\in V_j\}\boldsymbol{M}_{V_j,T-1}^{-\frac{1}{2}}\boldsymbol{x}_{a_T}\boldsymbol{x}_{a_T}^{\top}\boldsymbol{M}_{V_j,T-1}^{-\frac{1}{2}}\bigg)\nonumber\\
    &=det(\boldsymbol{M}_{V_j,T-1})\bigg(1+\mathbb{I}\{i_T\in V_j\}\norm{\boldsymbol{x}_{a_T}}_{\boldsymbol{M}_{V_j,T-1}^{-1}}^2\bigg)\nonumber\\
    &=det(\boldsymbol{M}_{V_j,T_0})\prod_{t=T_0+1}^{T}\bigg(1+\mathbb{I}\{i_t\in V_j\}\norm{\boldsymbol{x}_{a_t}}_{\boldsymbol{M}_{V_j,t-1}^{-1}}^2\bigg)\nonumber\\
    &\geq det(\lambda\boldsymbol{I})\prod_{t=T_0+1}^{T}\bigg(1+\mathbb{I}\{i_t\in V_j\}\norm{\boldsymbol{x}_{a_t}}_{\boldsymbol{M}_{V_j,t-1}^{-1}}^2\bigg)\label{det recursive}\,.
\end{align}

$\forall{x\in[0,1]}$, we have $x\leq 2\log(1+x)$. Therefore
\begin{align}
    &\sum_{t=T_0+1}^{T}\min\{\mathbb{I}\{i_t\in V_j\}\norm{\boldsymbol{x}_{a_t}}_{\boldsymbol{M}_{V_j,t-1}^{-1}}^2,1\}\notag\\
    &\leq 2\sum_{t=T_0+1}^{T} \log\bigg(1+\mathbb{I}\{i_t\in V_j\}\norm{\boldsymbol{x}_{a_t}}_{\boldsymbol{M}_{V_j,t-1}^{-1}}^2\bigg)\nonumber\\
    &=2\log\bigg(\prod_{t=T_0+1}^{T}\big(1+\mathbb{I}\{i_t\in V_j\}\norm{\boldsymbol{x}_{a_t}}_{\boldsymbol{M}_{V_j,t-1}^{-1}}^2\big)\bigg)\nonumber\\
    &\leq 2[\log(det(\boldsymbol{M}_{V_j,T}))-\log(det(\lambda\boldsymbol{I}))]\nonumber\\
    &\leq 2\log\bigg(\frac{trace(\lambda\boldsymbol{I}+\sum_{t=1}^T\mathbb{I}\{i_t\in V_j\}\boldsymbol{x}_{a_t}\boldsymbol{x}_{a_t}^{\top})}{\lambda d}\bigg)^d\nonumber\\
    &\leq 2d \log(1+\frac{T}{\lambda d})\,.
\end{align}
\end{proof}

Denote the rounds with $w_{i_t,t}=1$ as $\{\tilde{t}_1,\ldots,\tilde{t}_{l_1}\}$, and gram matrix $\tilde{\bG}_{V_{\tilde{t}_{\tau}}, \tilde{t}_{\tau}-1}\triangleq\lambda\bI+\sum_{s\in[\tau]\atop i_s\in V_{\tilde{t}_{\tau}}}\bx_{a_{\tilde{t}_{s}}}\bx_{a_{\tilde{t}_{s}}}^{\top}$; denote the rounds with $w_{i_t,t}<1$ as $\{{t}^{\prime}_1,\ldots,{t}^{\prime}_{l_2}\}$, gram matrix ${\bG}^{\prime}_{V_{t^{\prime}_{\tau}},t^{\prime}_{\tau}-1}\triangleq\lambda\bI+\sum_{s\in[\tau]\atop i_s\in V_{{t}^{\prime}_{\tau}}}w_{i_{{t}^{\prime}_{s}},{t}^{\prime}_{s}}\bx_{a_{{t}^{\prime}_{s}}}\bx_{a_{{t}^{\prime}_{s}}}^{\top}$. 

Then we have
\begin{align}
&\sum_{t\in[T]:w_{i_t,t}=1}\norm{\boldsymbol{x_{a_t}}}_{\boldsymbol{M}^{-1}_{V_t, t-1}}\notag\\
&=\sum_{j=1}^m\sum_{\tau=1}^{l_1}\mathbb{I}\{i_{\tilde{t}_{\tau}}\in V_j\}\norm{\boldsymbol{x_{a_{\tilde{t}_{\tau}}}}}_{\boldsymbol{M}^{-1}_{V_{\tilde{t}_{\tau}, \tilde{t}_{\tau}-1}}}\notag\\&\leq\sum_{j=1}^m\sum_{\tau=1}^{l_1}\mathbb{I}\{i_{\tilde{t}_{\tau}}\in V_j\}\norm{\boldsymbol{x_{a_{\tilde{t}_{\tau}}}}}_{\tilde{\bG}^{-1}_{V_{\tilde{t}_{\tau}}, \tilde{t}_{\tau}-1}}\label{matrix psd 1}\\
&\leq\sum_{j=1}^m\sqrt{\sum_{\tau=1}^{l_1}\mathbb{I}\{i_{\tilde{t}_{\tau}}\in V_j\}\sum_{\tau=1}^{l_1}\min\{1,\mathbb{I}\{i_{\tilde{t}_{\tau}}\in V_j\} \norm{\boldsymbol{x_{a_{\tilde{t}_{\tau}}}}}^2_{\tilde{\bG}^{-1}_{V_{\tilde{t}_{\tau}}, \tilde{t}_{\tau}-1}}\}}\label{cauchy final 1}\\
    &\leq \sum_{j=1}^m\sqrt{T_{V_j,T}\times2d \log(1+\frac{T}{\lambda d})}\label{lemma technical apply 1}\\
    &\leq\sqrt{2m\sum_{j=1}^m T_{V_j,T}d \log(1+\frac{T}{\lambda d})}=\sqrt{2mdT\log(1+\frac{T}{\lambda d})}\label{last eq}\,,
\end{align}
where Eq.(\ref{matrix psd 1}) is because $\tilde{\bG}^{-1}_{V_{\tilde{t}_{\tau}}, \tilde{t}_{\tau}-1}\succeq\boldsymbol{M}^{-1}_{V_{\tilde{t}_{\tau}}, \tilde{t}_{\tau}-1}$ in Eq.(\ref{cauchy final 1}) we use Cauchy–Schwarz inequality, in Eq.(\ref{lemma technical apply 1}) we use Lemma \ref{technical lemma6} and $\sum_{\tau=1}^{l_1}\mathbb{I}\{i_{\tilde{t}_{\tau}}\in V_j\}\leq T_{V_j,T}$, in Eq.(\ref{last eq}) we use Cauchy–Schwarz inequality and $\sum_{j=1}^m T_{V_j,T}=T$.

For the second part in Eq.(\ref{regret decompose 2 cases}), Let ${\bx}_{a_{{t}^{\prime}_{\tau}}}^{\prime}\triangleq \sqrt{w_{i_{{t}^{\prime}_{\tau}},{t}^{\prime}_{\tau}}}\bx_{a_{{t}^{\prime}_{\tau}}}$, then
    \begin{align}
   \sum_{t:w_{i_t,t}<1}\norm{\boldsymbol{x_{a_t}}}_{\boldsymbol{M}^{-1}_{V_t, t-1}}&=\sum_{t:w_{i_t,t}<1}\frac{\norm{\boldsymbol{x_{a_t}}}_{\boldsymbol{M}^{-1}_{V_t, t-1}}^2}{\norm{\boldsymbol{x_{a_t}}}_{\boldsymbol{M}^{-1}_{V_t, t-1}}}=\sum_{t:w_{i_t,t}<1}\frac{w_{i_t,t}\norm{\boldsymbol{x_{a_t}}}_{\boldsymbol{M}^{-1}_{V_t, t-1}}^2}{\alpha}\label{def of weight use1}\\
   &=\sum_{j=1}^m\sum_{\tau=1}^{l_2} \mathbb{I}\{i_{{t}^{\prime}_{\tau}}\in V_j\}\frac{w_{i_{{t}^{\prime}_{\tau}},{t}^{\prime}_{\tau}}}{\alpha}\norm{\bx_{a_{{t}^{\prime}_{\tau}}}}_{{\bM}^{-1}_{V_{t^{\prime}_{\tau}},t^{\prime}_{\tau}-1}}^2\nonumber\\
   &\leq \sum_{j=1}^m \frac{\sum_{\tau=1}^{l_2}\min\{1,\mathbb{I}\{i_{{t}^{\prime}_{\tau}}\in V_j\}\norm{\bx_{a_{{t}^{\prime}_{\tau}}}^{\prime}}_{{\bG}^{\prime -1}_{V_{t^{\prime}_{\tau}},t^{\prime}_{\tau}-1}}^2\}}{\alpha}\label{matrix psd use 010}\\
   &\leq \sum_{j=1}^m \frac{2d\log(1+\frac{T}{\lambda d})}{\alpha}=\frac{2md\log(1+\frac{T}{\lambda d})}{\alpha}\label{use of tech 213}
   % &=\sum_{\tau=1}^{l_2}\sum_{j=1}^m \mathbb{I}\{i_{{t}^{\prime}_{\tau}}\in V_j\}\frac{{\bx}_{a_{{t}^{\prime}_{\tau}}}^{\prime\top}{\bG}^{\prime -1}_{V_{t^{\prime}_{\tau}},t^{\prime}_{\tau}-1}{\bx}_{a_{{t}^{\prime}_{\tau}}}^{\prime}}{w_{i_{{t}^{\prime}_{\tau}},{t}^{\prime}_{\tau}}}\notag\\
   % &=\sum_{\tau=1}^{l_2}\sum_{j=1}^m \mathbb{I}\{i_{{t}^{\prime}_{\tau}}\in V_j\}\frac{{\bx}_{a_{{t}^{\prime}_{\tau}}}^{\prime\top}{\bG}^{\prime -1}_{V_{t^{\prime}_{\tau}},t^{\prime}_{\tau}-1}{\bx}_{a_{{t}^{\prime}_{\tau}}}^{\prime}}{\alpha} \times\norm{\bx_{a_{{t}^{\prime}_{\tau}}}}_{{\bG}^{\prime -1}_{V_{t^{\prime}_{\tau}},t^{\prime}_{\tau}-1}}\notag\\
   % &\leq\frac{1}{\alpha\sqrt{\lambda}}(\sum_{j=1}^m \sum_{\tau=1}^{l_2}\mathbb{I}\{i_{{t}^{\prime}_{\tau}}\in V_j\}\norm{\bx^{\prime}_{a_{{t}^{\prime}_{\tau}}}}_{{\bG}^{\prime-1}_{V_{t^{\prime}_{\tau}},t^{\prime}_{\tau}-1}}^2)\leq\frac{2md}{\alpha\sqrt{\lambda}}\log(1+\frac{T}{\lambda d})\,,
\end{align}
where in Eq.(\ref{def of weight use1}) we use the definition of the weights, in Eq.(\ref{matrix psd use 010}) we use ${\bG}^{\prime -1}_{V_{t^{\prime}_{\tau}},t^{\prime}_{\tau}-1}\succeq{\bM}^{-1}_{V_{t^{\prime}_{\tau}},t^{\prime}_{\tau}-1}$, and Eq.(\ref{use of tech 213}) uses Lemma \ref{technical lemma6}.

Then, with Eq.(\ref{use of tech 213}), Eq.(\ref{last eq}), Eq.(\ref{regret decompose 2 cases}), Eq.(\ref{regret result 1}), Eq.(\ref{regret result 2}), $\delta=\frac{1}{T}$, and $\beta=\sqrt{\lambda}+\sqrt{2\log(T)+d\log(1+\frac{T}{\lambda d})}+\alpha C$, we can get 
\begin{small}
    \begin{align}
    R(T)&\leq 4+ T_0+\big(2\sqrt{\lambda}+\sqrt{2\log(T)+d\log(1+\frac{T}{\lambda d})}+\alpha C\big) \times \bigg(\sqrt{2mdT\log(1+\frac{T}{\lambda d})}\nonumber\\
    &\quad\quad+\frac{2md\log(1+\frac{T}{\lambda d})}{\alpha}\bigg)\nonumber\\
    &=4+16u\log(uT)+4u\max\{\frac{288d}{\gamma^2\alpha\sqrt{\lambda}\tilde{\lambda}_x}\log(uT), \frac{16}{\tilde{\lambda}_x^2}\log(\frac{8dT}{\tilde{\lambda}_x^2}),\frac{72\sqrt{\lambda}}{\alpha\gamma^2\tilde{\lambda}_x},\frac{72\alpha C^2}{\gamma^2\sqrt{\lambda}\tilde{\lambda}_x}\}\nonumber\\
&\quad\quad+\big(2\sqrt{\lambda}+\sqrt{2\log(T)+d\log(1+\frac{T}{\lambda d})}+\alpha C\big) \times \bigg(\sqrt{2mdT\log(1+\frac{T}{\lambda d})}\nonumber\\
    &\quad\quad+\frac{2md\log(1+\frac{T}{\lambda d})}{\alpha}\bigg)\,.\nonumber
\end{align}
\end{small}

Picking $\alpha=\frac{\sqrt{\lambda}+\sqrt{d}}{C}$, we can get
\begin{equation}
    R(T)\le O\big((\frac{C\sqrt{d}}{\gamma^{2}\tilde{\lambda}_{x}}+\frac{1}{\tilde{\lambda}_{x}^{2}})u\log(T)\big)+O\big(d\sqrt{mT}\log(T)\big)+ O\big(mCd\log^{1.5}(T)\big)\,.
\end{equation}
Thus we complete the proof of Theorem \ref{thm:main}.
\subsection{Proof and Discussions of Theorem \ref{thm: lower bound}}
\label{appendix: proof of lower bound}
Table 1 of the work \cite{he2022nearly} gives a lower bound for linear bandits with adversarial corruption for a single user. The lower bound of $R(T)$ is given by:
$R(T)\geq \Omega(d\sqrt{T}+dC)$. Therefore, suppose our problem with multiple users and $m$ underlying clusters where the arrival times are $T_i$ for each cluster, then for any algorithms, even if they know the underlying clustering structure and keep $m$ independent linear bandit algorithms to leverage the common information of clusters, the best they can get is $R(T)\geq dC+\sum_{i \in [m]}d\sqrt{T_i}$. For a special case where $T_i=\frac{T}{m}, \forall i\in[m]$, we can get $R(T)\geq dC+\sum_{i \in [m]}d\sqrt{\frac{T}{m}}=d\sqrt{mT}+dC$, which gives a lower bound of $\Omega(d\sqrt{mT}+dC)$ for the LOCUD problem. 

Recall that the regret upper bound of RCLUB-WCU shown in Theorem \ref{thm:main} is of $ O\bigg((\frac{C\sqrt{d}}{\gamma^{2}\tilde{\lambda}_{x}}+\frac{1}{\tilde{\lambda}_{x}^{2}})u\log(T)\bigg)+O\big(d\sqrt{mT}\log(T)\big)+ O\big(mCd\log^{1.5}(T)\big)$, asymptotically matching this lower bound with respect to $T$ up to logarithmic factors and with respect to $C$ up to $O(\sqrt{m})$ factors, showing the tightness of our theoretical results (where 
 $m$ are typically very small for real applications).
 
 We conjecture that the gap for the $m$ factor in the $mC$ term of the lower bound is due to the strong assumption that cluster structures are known to prove our lower bound, and whether there exists a tighter lower bound will be left for future work.
\subsection{Proof of Theorem \ref{thm:occud}}
\label{appendix: proof of occud}
We prove the theorem using the proof by contrapositive. Specifically, in Theorem \ref{thm:occud}, we need to prove that for any $t\geq T_0$, if the detection condition in Line \ref{detect line} of Algo.\ref{occud} for user $i$, then with probability at least $1-5\delta$, user $i$ is indeed a corrupted user. By the proof by contrapositive, we can prove Theorem \ref{thm:occud} by showing that: for any $t\geq T_0$, if user $i$ is a normal user, then with probability at least $1-5\delta$, the detection condition in Line \ref{detect line} of Algo.\ref{occud} will not be satisfied for user $i$.

If the clustering structure is correct at $t$, then for any normal user $i$
\begin{align}
    \Tilde{\btheta}_{i,t}-\hat{\btheta}_{V_{i,t},t}&=\Tilde{\btheta}_{i,t}-\btheta_i+\btheta_i-\hat{\btheta}_{V_{i,t},t}\label{decomposition}\,,
\end{align}
where $\Tilde{\btheta}_{i,t}$ is the non-robust estimation of the ground-truth $\theta_i$, and $\hat{\btheta}_{V_{i,t},t-1}$ is the robust estimation of the inferred cluster $V_{i,t}$ for user $i$ at round $t$. Since the clustering structure is correct at $t$, $\hat{\btheta}_{V_{i,t},t-1}$ is the robust estimation of user $i$'s ground-truth cluster's preference vector $\btheta^{j(i)}=\btheta_i$ at round $t$.

We have 
\begin{align}
    \Tilde{\btheta}_{i,t}-\btheta_i&=(\lambda \boldsymbol{I}+\tilde{\boldsymbol{M}}_{i,t})^{-1} \tilde{\boldsymbol{b}}_{i,t}-\btheta_i\nonumber\\
    &=(\lambda\bI+\sum_{s\in[t]\atop i_s=i}\bx_{a_s}\bx_{a_s}^{\top})^{-1}(\sum_{s\in[t]\atop i_s=i}\bx_{a_s}r_{s})-\btheta_i\nonumber\\
    &=(\lambda\bI+\sum_{s\in[t]\atop i_s=i}\bx_{a_s}\bx_{a_s}^{\top})^{-1}\big(\sum_{s\in[t]\atop i_s=i}\bx_{a_s}(\bx_{a_s}^{\top}\btheta_i+\eta_s)\big)-\btheta_i\label{because i is normal}\\
    &=(\lambda\bI+\sum_{s\in[t]\atop i_s=i}\bx_{a_s}\bx_{a_s}^{\top})^{-1}\big((\lambda\bI+\sum_{s\in[t]\atop i_s=i}\bx_{a_s}\bx_{a_s}^{\top})\btheta_i-\lambda\btheta_i+\sum_{s\in[t]\atop i_s=i}\bx_{a_s}\eta_s)\big)-\btheta_i\nonumber\\
    &=-\lambda\tilde{\boldsymbol{M}}_{i,t}^{\prime-1}\btheta_i+\tilde{\boldsymbol{M}}_{i,t}^{\prime-1}\sum_{s\in[t]\atop i_s=i}\bx_{a_s}\eta_s\,,\nonumber
\end{align}
where we denote $\tilde{\boldsymbol{M}}_{i,t}^{\prime}\triangleq \lambda\bI+\sum_{s\in[t]\atop i_s=i}\bx_{a_s}\bx_{a_s}^{\top}$, and Eq.(\ref{because i is normal}) is because since user $i$ is normal, we have $c_s=0, \forall s: i_s=i$.

Then, we have 
\begin{align}
    \norm{\Tilde{\btheta}_{i,t}-\btheta_i}_2&\leq \norm{\lambda\tilde{\boldsymbol{M}}_{i,t}^{\prime-1}\btheta_i}_2+\norm{\tilde{\boldsymbol{M}}_{i,t}^{\prime-1}\sum_{s\in[t]\atop i_s=i}\bx_{a_s}\eta_s}_2\nonumber\\
    &\leq \lambda\norm{{\Tilde{\bM}_{i,t}^{\prime-\frac{1}{2}}}}_2^2\norm{\btheta_i}_2+\norm{{\Tilde{\bM}_{i,t}^{\prime-\frac{1}{2}}}\sum_{s\in[t]\atop i_s=i}\bx_{a_s}\eta_s}_2\norm{{\Tilde{\bM}_{i,t}^{\prime-\frac{1}{2}}}}_2\label{matrice norm, cauchy again}\\
    &\leq\frac{\sqrt{\lambda}+\norm{\sum_{s\in[t]\atop i_s=i}\bx_{a_s}\eta_s}_{\Tilde{\bM}_{i,t}^{\prime-1}}}{\sqrt{\lambda_{\text{min}}({\Tilde{\bM}_{i,t}^{\prime}})}}\,,\label{min eigen value use 1, courant fisher}\,,
\end{align}
where Eq.(\ref{matrice norm, cauchy again}) follows by the Cauchy–Schwarz inequality and the inequality for the operator norm of matrices, and Eq.(\ref{min eigen value use 1, courant fisher}) follows by the Courant-Fischer theorem and the fact that $\lambda_{\text{min}}(\Tilde{\bM}_{i,t}^{\prime})\geq\lambda$.

Following Theorem 1 in \cite{abbasi2011improved}, for a fixed normal user $i$, with probability at least $1-\delta$ for some $\delta\in(0,1)$ we have:
\begin{align}
    \norm{\sum_{s\in[t]\atop i_s=i}\bx_{a_s}\eta_s}_{\Tilde{\bM}_{i,t}^{\prime-1}}&\leq\sqrt{2\log(\frac{1}{\delta})+\log(\frac{\text{det}(\Tilde{\bM}_{i,t}^{\prime})}{\text{det}(\lambda\bI)})}\nonumber\\
    &\leq \sqrt{2\log(\frac{1}{\delta})+d\log(1+\frac{T_{i,t}}{\lambda d})}\label{det inequality 2}\,,
\end{align}
where Eq.(\ref{det inequality 2}) is because $\text{det}(\Tilde{\bM}_{i,t}^{\prime})\leq\Bigg(\frac{\text{trace}(\lambda\bI+\sum_{s\in[t]\atop i_s=i}\bx_{a_s}\bx_{a_s}^{\top})}{d}\Bigg)^d\leq\big(\frac{\lambda d+T_{i,t}}{d}\big)^d$, and $\text{det}(\lambda\bI)=\lambda^d$.

Plugging this into Eq.(\ref{min eigen value use 1, courant fisher}), we can get 
\begin{equation}
    \norm{\Tilde{\btheta}_{i,t}-\btheta_i}_2\leq\frac{\sqrt{\lambda}+\sqrt{2\log(\frac{1}{\delta})+d\log(1+\frac{T_{i,t}}{\lambda d})}}{\sqrt{\lambda_{\text{min}}({\Tilde{\bM}_{i,t}^{\prime}})}}\label{occud proof result 1}\,.
\end{equation}

Then we need to bound $\norm{\btheta_i-\hat{\btheta}_{V_{i,t},t}}_2$. With the correct clustering, $V_{i,t}=V_{j(i)}$, we have
\begin{align}
    \hat{\btheta}_{V_{i,t},t}-\btheta_i&=\boldsymbol{M}_{{V}_{i,t},t}^{-1}\boldsymbol{b}_{{V}_{j,t},t}\nonumber\\
    &=(\lambda\bI+\sum_{s\in[t]\atop i_s\in V_{j(i)}}w_{i_s,s}\bx_{a_s}\bx_{a_s}^{\top})^{-1}(\sum_{s\in[t]\atop i_s\in V_{j(i)}}w_{i_s,s}\bx_{a_s}r_s)-\btheta_i\nonumber\\
    &=(\lambda\bI+\sum_{s\in[t]\atop i_s\in V_{j(i)}}w_{i_s,s}\bx_{a_s}\bx_{a_s}^{\top})^{-1}(\sum_{s\in[t]\atop i_s\in V_{j(i)}}w_{i_s,s}\bx_{a_s}(\bx_{a_s}^{\top}\btheta_i+\eta_s+c_s)))-\theta_i\label{cluster of i}\\
    &=(\lambda\bI+\sum_{s\in[t]\atop i_s\in V_{j(i)}}w_{i_s,s}\bx_{a_s}\bx_{a_s}^{\top})^{-1}\big((\lambda\bI+\sum_{s\in[t]\atop i_s\in V_{j(i)}}w_{i_s,s}\bx_{a_s}\bx_{a_s}^{\top})\btheta_i-\lambda\btheta_i\nonumber\\
    &\quad+\sum_{s\in[t]\atop i_s\in V_{j(i)}}w_{i_s,s}\bx_{a_s}
    \eta_s+\sum_{s\in[t]\atop i_s\in V_{j(i)}}w_{i_s,s}\bx_{a_s}c_s))\big)-\theta_i\nonumber\\
    &=-\lambda\boldsymbol{M}_{{V}_{i,t},t}^{-1}\btheta_i+\boldsymbol{M}_{{V}_{i,t},t}^{-1}\sum_{s\in[t]\atop i_s\in V_{j(i)}}w_{i_s,s}\bx_{a_s}
    \eta_s+\boldsymbol{M}_{{V}_{i,t},t}^{-1}\sum_{s\in[t]\atop i_s\in V_{j(i)}}w_{i_s,s}\bx_{a_s}c_s\label{difference cluster robust groud-truth}\,.
    \end{align}
Therefore, we have
\begin{align}
    &\norm{\btheta_i-\hat{\btheta}_{V_{i,t},t}}_2\notag\\&\leq \lambda\norm{\boldsymbol{M}_{{V}_{i,t},t}^{-1}\btheta_i}_2+\norm{\boldsymbol{M}_{{V}_{i,t},t}^{-1}\sum_{s\in[t]\atop i_s\in V_{j(i)}}w_{i_s,s}\bx_{a_s}
    \eta_s}_2+\norm{\boldsymbol{M}_{{V}_{i,t},t}^{-1}\sum_{s\in[t]\atop i_s\in V_{j(i)}}w_{i_s,s}\bx_{a_s}c_s}_2\nonumber\\
    &\leq\lambda\norm{{\boldsymbol{M}_{{V}_{i,t},t}^{-\frac{1}{2}}}}_2^2\norm{\btheta_i}_2+\norm{{\boldsymbol{M}_{{V}_{i,t},t}^{-\frac{1}{2}}}\sum_{s\in[t]\atop i_s\in V_{j(i)}}w_{i_s,s}\bx_{a_s}\eta_s}_2\norm{{\boldsymbol{M}_{{V}_{i,t},t}^{-\frac{1}{2}}}}_2\nonumber\\
    &\quad\quad+\norm{{\boldsymbol{M}_{{V}_{i,t},t}^{-\frac{1}{2}}}\sum_{s\in[t]\atop i_s\in V_{j(i)}}w_{i_s,s}\bx_{a_s}\eta_s}_2\norm{{\boldsymbol{M}_{{V}_{i,t},t}^{-\frac{1}{2}}}}_2\label{cauchy, operator norm 12}\\
    &\leq \frac{\sqrt{\lambda}+\norm{\sum_{s\in[t]\atop i_s\in V_{j(i)}}w_{i_s,s}\bx_{a_s}
\eta_s}_{\boldsymbol{M}_{{V}_{i,t},t}^{-1}}+\norm{\sum_{s\in[t]\atop i_s\in V_{j(i)}}w_{i_s,s}\bx_{a_s}
c_s}_{\boldsymbol{M}_{{V}_{i,t},t}^{-1}}}{\sqrt{\lambda_{\text{min}}(\boldsymbol{M}_{{V}_{i,t},t})}}
\end{align}
Let $\Tilde{\bx}_{s}\triangleq \sqrt{w_{i_s,s}}\bx_{a_s}$, $\Tilde{\eta}_{s}\triangleq \sqrt{w_{i_s,s}}\eta_s$, then we have: $\norm{\Tilde{\bx}_{s}}_2\leq\norm{\sqrt{w_{i_s,s}}}_2\norm{\bx_{a_s}}_2\leq1$, $\Tilde{\eta}_{s}$ is still 1-sub-gaussian (since $\eta_s$ is 1-sub-gaussian and $\sqrt{w_{i_s,s}}\leq1$), $\boldsymbol{M}_{{V}_{i,t},t}=\lambda\bI+\sum_{s\in[t]\atop i_s\in V_{j(i)}}\Tilde{\bx}_{s}\Tilde{\bx}_{s}^{\top}$, and $\norm{\sum_{s\in[t]\atop i_s\in V_{j(i)}}w_{i_s,s}\bx_{a_s}
\eta_s}_{\boldsymbol{M}_{{V}_{i,t},t}^{-1}}$ becomes $\norm{\sum_{s\in[t]\atop i_s\in V_{j(i)}}\tilde{\bx}_{s}
\tilde{\eta}_s}_{\boldsymbol{M}_{{V}_{i,t},t}^{-1}}$. Then, following Theorem 1 in \cite{abbasi2011improved}, with probability at least $1-\delta$ for some $\delta\in(0,1)$, for a fixed normal user $i$, we have
\begin{align}
   \norm{\sum_{s\in[t]\atop i_s\in V_{j(i)}}w_{i_s,s}\bx_{a_s}
\eta_s}_{\boldsymbol{M}_{{V}_{i,t},t}^{-1}}&\leq\sqrt{2\log(\frac{1}{\delta})+\log(\frac{\text{det}(\boldsymbol{M}_{{V}_{i,t},t})}{\text{det}(\lambda\bI)})}\nonumber\\
    &\leq \sqrt{2\log(\frac{1}{\delta})+d\log(1+\frac{T_{V_{i,t},t}}{\lambda d})}\label{det inequality 3}\,,
\end{align}
where Eq.(\ref{det inequality 2}) is because $\text{det}(\boldsymbol{M}_{{V}_{i,t},t})\leq\Bigg(\frac{\text{trace}(\lambda\bI+\sum_{s\in[t]\atop i_s\in V_{j(i)}}\bx_{a_s}\bx_{a_s}^{\top})}{d}\Bigg)^d\leq\big(\frac{\lambda d+T_{V_{i,t},t}}{d}\big)^d$, and $\text{det}(\lambda\bI)=\lambda^d$.

For $\norm{\sum_{s\in[t]\atop i_s\in V_{j(i)}}w_{i_s,s}\bx_{a_s}
c_s}_{\boldsymbol{M}_{{V}_{i,t},t}^{-1}}$, we have
\begin{align}
    \norm{\sum_{s\in[t]\atop i_s\in V_{j(i)}}w_{i_s,s}\bx_{a_s}
c_s}_{\boldsymbol{M}_{{V}_{i,t},t}^{-1}}&\leq \sum_{s\in[t]\atop i_s\in V_{j(i)}}\left|c_s\right| w_{i_s,s}\norm{\bx_{a_s}}_{\boldsymbol{M}_{{V}_{i,t},t}^{-1}}\nonumber\\
&\leq \alpha C\label{bound 21312}\,,
\end{align}
where Eq.(\ref{bound 21312}) is because $w_{i_s,s}\leq\frac{\alpha}{\norm{\bx_{a_s}}_{\boldsymbol{M}_{i_s,s}^{\prime-1}}}\leq\frac{\alpha}{\norm{\bx_{a_s}}_{\boldsymbol{M}_{i_s,t}^{\prime-1}}}\leq\frac{\alpha}{\norm{\bx_{a_s}}_{\boldsymbol{M}_{{V}_{i,t},t}^{-1}}}$ (since $\boldsymbol{M}_{{V}_{i,t},t}\succeq\bM_{i_s,t}^{\prime}\succeq\bM_{i_s,s}^{\prime}$, $\bM_{i_s,s}^{\prime-1}\succeq\bM_{i_s,t}^{\prime-1}\succeq\boldsymbol{M}_{{V}_{i,t},t}^{-1}$, $\norm{\bx_{a_s}}_{\bM_{i_s,s}^{\prime-1}}\geq\norm{\bx_{a_s}}_{\bM_{i_s,t}^{\prime-1}}\geq\norm{\bx_{a_s}}_{\boldsymbol{M}_{{V}_{i,t},t}^{-1}}$), and $\sum_{s\in[t]} \left|c_s\right|\leq C$.

Therefore, we have
\begin{equation}
    \norm{\btheta_i-\hat{\btheta}_{V_{i,t},t}}_2\leq \frac{\sqrt{\lambda}+\sqrt{2\log(\frac{1}{\delta})+d\log(1+\frac{T_{V_{i,t},t}}{\lambda d})}+\alpha C}{\sqrt{\lambda_{\text{min}}(\boldsymbol{M}_{{V}_{i,t},t})}}\label{bound norm difference for occud 2}\,.
\end{equation}

With Eq.(\ref{bound norm difference for occud 2}), Eq.(\ref{occud proof result 1}) and Eq.(\ref{decomposition}), together with Lemma \ref{sufficient time}, we have that for a normal user $i$, for any $t\geq T_0$, with probability at least $1-5\delta$ for some $\delta\in(0,\frac{1}{5})$
\begin{align}
    &\norm{\Tilde{\btheta}_{i,t}-\hat{\btheta}_{V_{i,t},t}}\notag\\&\leq\norm{\Tilde{\btheta}_{i,t}-\btheta_i}_2+\norm{\btheta_i-\hat{\btheta}_{V_{i,t},t}}_2  \nonumber\\
    &\leq \frac{\sqrt{\lambda}+\sqrt{2\log(\frac{1}{\delta})+d\log(1+\frac{T_{i,t}}{\lambda d})}}{\sqrt{\lambda_{\text{min}}({\Tilde{\bM}_{i,t}^{\prime}})}}+\frac{\sqrt{\lambda}+\sqrt{2\log(\frac{1}{\delta})+d\log(1+\frac{T_{V_{i,t},t}}{\lambda d})}+\alpha C}{\sqrt{\lambda_{\text{min}}(\boldsymbol{M}_{{V}_{i,t},t})}}\,,
\end{align}
which is exactly the detection condition in Line \ref{detect line} of Algo.\ref{occud}.

Therefore, by the proof by contrapositive, we complete the proof of Theorem \ref{thm:occud}.

\subsection{Description of Baselines}
\label{appendix: baselines}
We compare RCLUB-WCU to the following five baselines for recommendations.
\begin{itemize}
    \item LinUCB\cite{li2010contextual}: A state-of-the-art bandit approach for a single user without corruption. 
    \item LinUCB-Ind: Use a separate LinUCB for each user.
    \item CW-OFUL\cite{he2022nearly}: A state-of-the-art bandit approach for single user with corruption.
    \item CW-OFUL-Ind: Use a separate CW-OFUL for each user.
    \item CLUB\cite{gentile2014online}: A graph-based clustering of bandits approach for multiple users without corruption.
    \item SCLUB\cite{li2019improved}: A set-based clustering of bandits approach for multiple users without corruption.
\end{itemize}

\subsection{More Experiments}
\label{sec:more experiments}
\subsubsection{Different Corruption Levels}
To see our algorithm's performance under different corruption levels, we conduct the experiments under different corruption levels for RCLUB-WCU, CLUB, and SCLUB on Amazon and Yelp datasets. Recall the corruption mechanism in Section \ref{exp:synthetic}, we set $k$ as 1,000; 10,000; 100,000. The results are shown in Fig.\ref{fig:corruption level}. All the algorithms' performance becomes worse when the corruption level increases. But RCLUB-WCU is much robust than the baselines.
\begin{figure*}
    \subfigure[Amazon Corruption Level]{
    \includegraphics[scale=0.35]{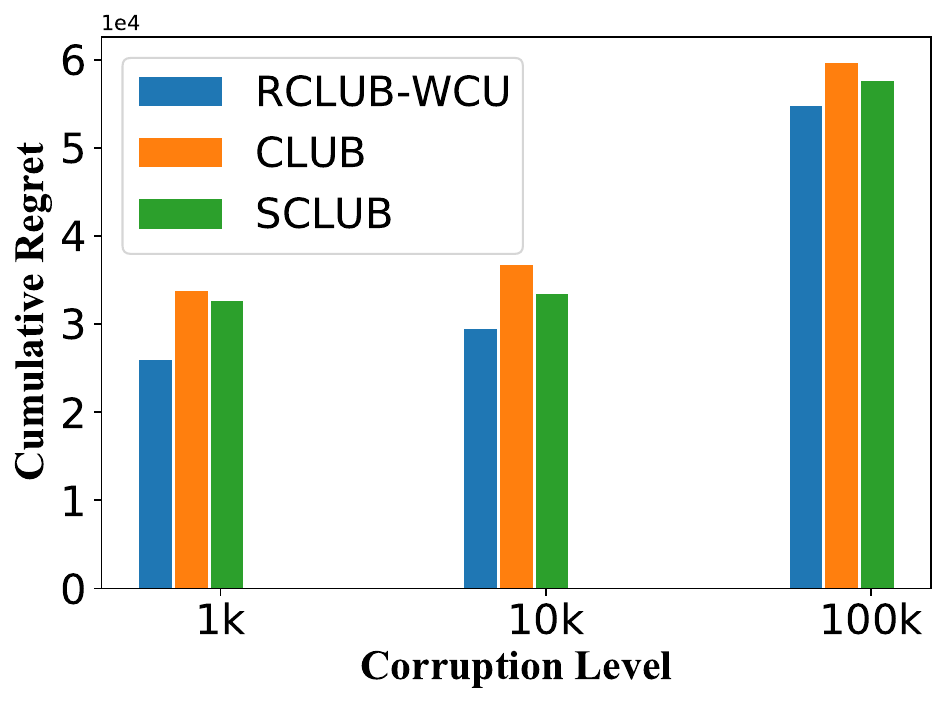}
    }
    \subfigure[Yelp Corruption Level]{
    \includegraphics[scale=0.35]{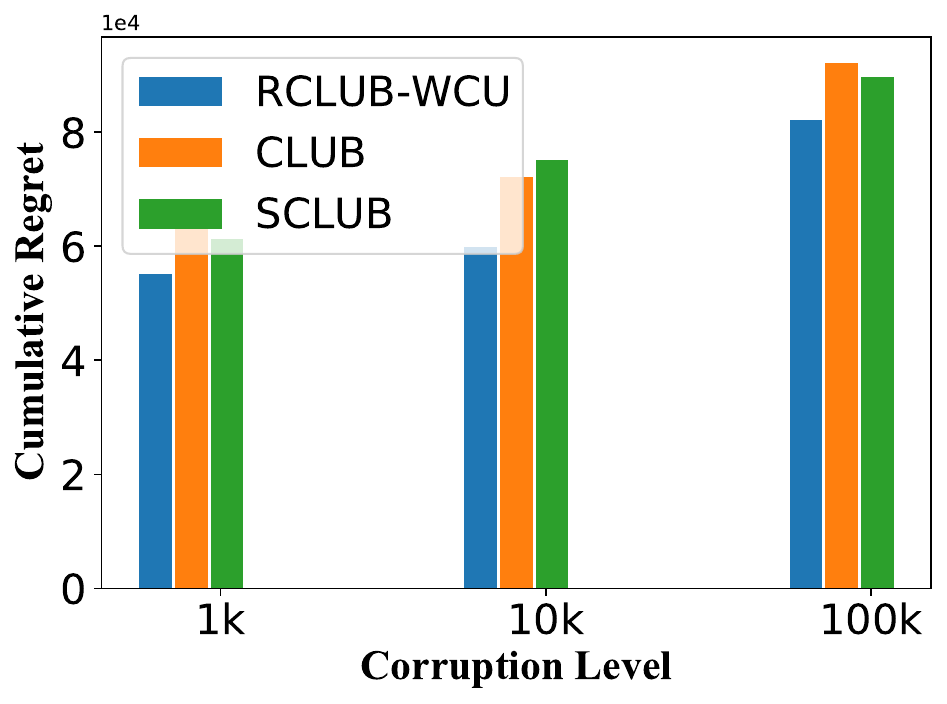}
    }
    \caption{Cumulative regret in different corruption levels}
    \label{fig:corruption level}
    
\end{figure*}
\subsubsection{Different Cluster numbers}
% \begin{figure}
%     \subfigure[Amazon]{
%     \includegraphics[scale=0.21]{amazon_cluster.pdf}
%     }
%     \subfigure[Yelp]{
%     \includegraphics[scale=0.21]{yelp_cluster.pdf}
%     }
%     \caption{Simialr to \cite{li2018online}, cumulative regret under different cluster numbers}
%     \label{fig:cluster number}
% \end{figure}
Following \cite{li2018online}, we test the performances of the cluster-based algorithms (RCLUB-WCU, CLUB, SCLUB) when the underlying cluster number changes. We set $m$ as 5, 10, 20, and 50. The results are shown in Fig.\ref{fig:cluster num}. All these algorithms' performances decrease when the cluster numbers increase, matching our theoretical results. The performances of CLUB and SCLUB decrease much faster than RCLUB-WCU, indicating that RCLUB-WCU is more robust when the underlying user cluster number changes.

\begin{figure*}
     \subfigure[Amazon Cluster Number]{
    \includegraphics[scale=0.35]{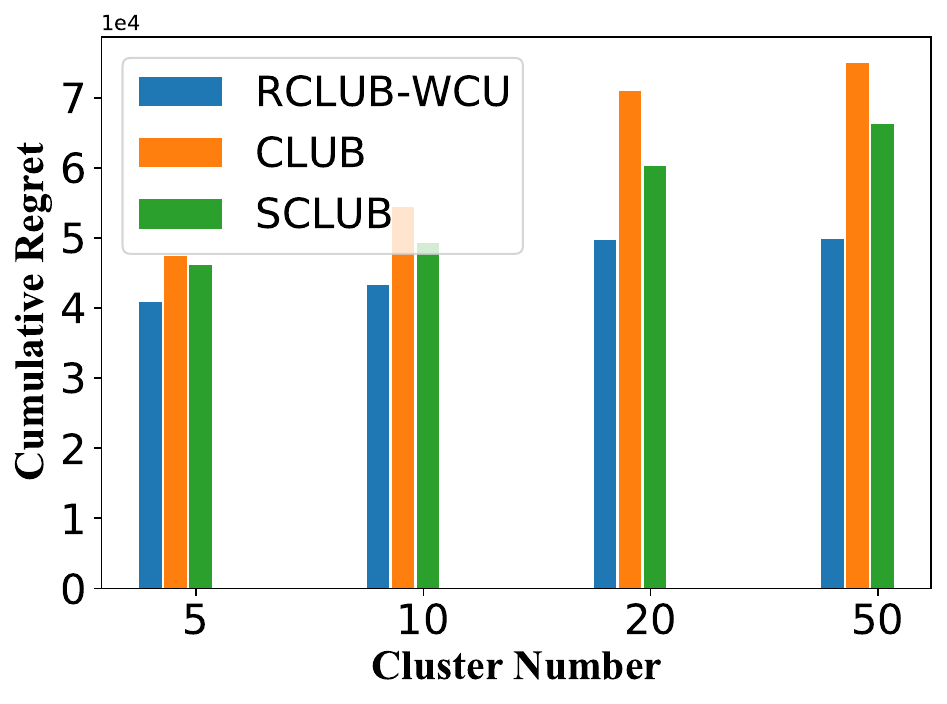}
    }
    \subfigure[Yelp Cluster Number]{
    \includegraphics[scale=0.35]{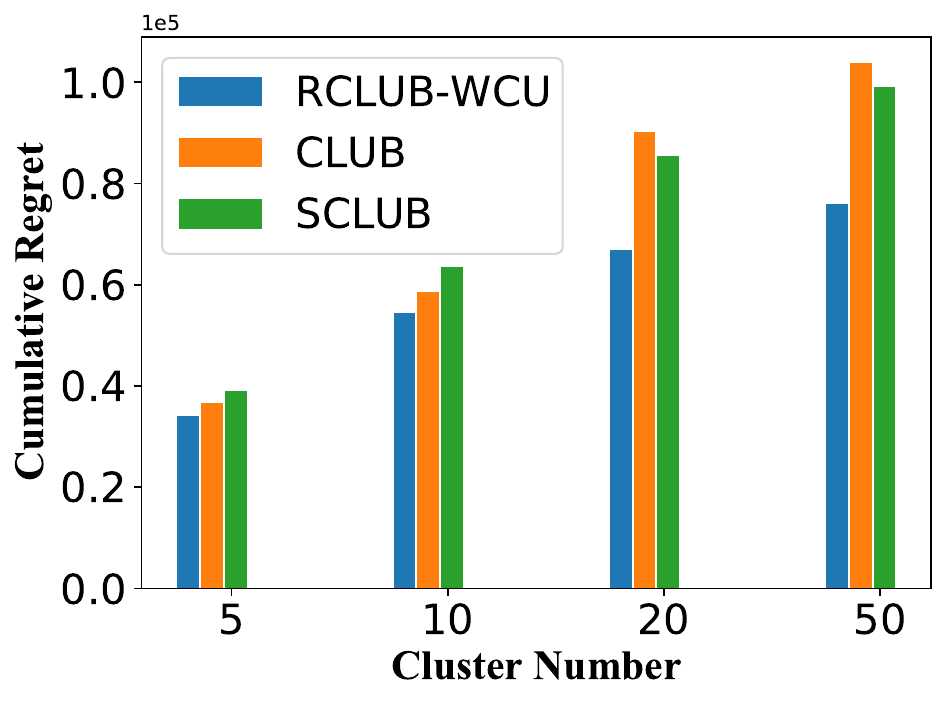}
    }
     %  \vspace{-0.28cm}
     \caption{Cumulative regret with different cluster numbers }
    \label{fig:cluster num}
\end{figure*}
\section{Appendix of Chapter \ref{chapter: aaai}}

\subsection{Proof of Lemma \ref{lemma2}}
\begin{proof}
According to the closed-form solution of $\boldsymbol{\theta}_t$ in Eq. (\ref{equation5}) (\ref{equation6}), we can calculate the estimation error as follows

\begin{small}\begin{equation*}
    \begin{aligned}
        \boldsymbol{\theta}_t-\boldsymbol{\theta}_*
        &=\boldsymbol{M}_t^{-1}\boldsymbol{b}_t-\boldsymbol{\theta}_*\\
        &=\Bigg(\sum_{\tau=1}^{t-1} \boldsymbol{x}_{a_\tau}\boldsymbol{x}^{\top}_{a_\tau}+\sum_{\tau=1}^{t} \sum_{k\in \mathcal{K}_\tau} \boldsymbol{\tilde{x}}_{k} \boldsymbol{\tilde{x}}^{\top}_{k}+\beta\boldsymbol{I}\Bigg)^{-1}\Bigg(\sum_{\tau=1}^{t-1} \boldsymbol{x}_{a_\tau} r_{a_\tau,\tau}+\sum_{\tau=1}^{t} \sum_{k\in \mathcal{K}_\tau} \boldsymbol{\tilde{x}}_{k} \tilde{r}_{k,\tau}\Bigg)-\boldsymbol{\theta}_*\\
        &=\Bigg(\sum_{\tau=1}^{t-1} \boldsymbol{x}_{a_\tau}\boldsymbol{x}^{\top}_{a_\tau}+\sum_{\tau=1}^{t} \sum_{k\in \mathcal{K}_\tau} \boldsymbol{\tilde{x}}_{k} \boldsymbol{\tilde{x}}^{\top}_{k}+\beta\boldsymbol{I}\Bigg)^{-1}\Bigg(\sum_{\tau=1}^{t-1} \boldsymbol{x}_{a_\tau} \bigg(\boldsymbol{x}^{\top}_{a_\tau}\boldsymbol{\theta}_*+\epsilon_\tau\bigg)
        +\sum_{\tau=1}^{t} \sum_{k\in \mathcal{K}_\tau} \boldsymbol{\tilde{x}}_{k} \bigg(\boldsymbol{\tilde{x}}_{k}^{\top}\boldsymbol{\theta}_*+\tilde{\epsilon}_{\tau}\bigg)\Bigg)\\
        &\quad-\boldsymbol{\theta}_*\\
        &=\Bigg(\sum_{\tau=1}^{t-1} \boldsymbol{x}_{a_\tau}\boldsymbol{x}^{\top}_{a_\tau}+\sum_{\tau=1}^{t} \sum_{k\in \mathcal{K}_\tau} \boldsymbol{\tilde{x}}_{k} \boldsymbol{\tilde{x}}^{\top}_{k}+\beta\boldsymbol{I}\Bigg)^{-1}\Bigg(\sum_{\tau=1}^{t-1} \boldsymbol{x}_{a_\tau}\boldsymbol{x}^{\top}_{a_\tau}+\sum_{\tau=1}^{t} \sum_{k\in \mathcal{K}_\tau} \boldsymbol{\tilde{x}}_{k} \boldsymbol{\tilde{x}}^{\top}_{k}+\beta\boldsymbol{I}-\beta\boldsymbol{I}\Bigg)\boldsymbol{\theta}_*-\boldsymbol{\theta}_*\\
        &\quad\quad+\boldsymbol{M}_t^{-1}(\sum_{\tau=1}^{t-1}\boldsymbol{x}_{a_\tau}\epsilon_\tau+\sum_{\tau=1}^{t} \sum_{k\in \mathcal{K}_\tau} \boldsymbol{\tilde{x}}_{k} \tilde{\epsilon}_{\tau})\\
        &=-\beta\boldsymbol{M}_t^{-1}\boldsymbol{\theta}_*+\boldsymbol{M}_t^{-1}(\sum_{\tau=1}^{t-1}\boldsymbol{x}_{a_\tau}\epsilon_\tau+\sum_{\tau=1}^{t} \sum_{k\in \mathcal{K}_\tau} \boldsymbol{\tilde{x}}_{k} \tilde{\epsilon}_{\tau})\,.
\end{aligned}\end{equation*}	
\end{small}

We can then bound the projection of the estimation error onto the direction of the action vector $\boldsymbol{x}_{a}$:
\begin{align}
        &\left|\boldsymbol{x}_{a}^{\top}(\boldsymbol{\theta}_t-\boldsymbol{\theta}_*)\right|\notag\\
&\leq\beta\left|\boldsymbol{x}_{a}^{\top}\boldsymbol{M}_t^{-1}\boldsymbol{\theta}_*\right| +\left|\boldsymbol{x}_{a}^{\top}\boldsymbol{M}_t^{-1}(\sum_{\tau=1}^{t-1}\boldsymbol{x}_{a_\tau}\epsilon_\tau+\sum_{\tau=1}^{t} \sum_{k\in \mathcal{K}_\tau} \boldsymbol{\tilde{x}}_{k} \tilde{\epsilon}_{\tau})\right|\nonumber
        \\
        &\leq\beta\norm{\boldsymbol{x}_{a}^{\top}\boldsymbol{M}_t^{-\frac{1}{2}}}_2\norm{\boldsymbol{M}_t^{-\frac{1}{2}}\boldsymbol{\theta}_*}_2 +\norm{\boldsymbol{x}_{a}^{\top}\boldsymbol{M}_t^{-\frac{1}{2}}}_2 \times\norm{\boldsymbol{M}_t^{-\frac{1}{2}}(\sum_{\tau=1}^{t-1}\boldsymbol{x}_{a_\tau}\epsilon_\tau+\sum_{\tau=1}^{t} \sum_{k\in \mathcal{K}_\tau} \boldsymbol{\tilde{x}}_{k} \tilde{\epsilon}_{\tau})}_2\label{cauchy1}\\
        &\leq\beta\norm{\boldsymbol{x}_{a}}_{\boldsymbol{M}_t^{-1}}\norm{\boldsymbol{M}_t^{-\frac{1}{2}}}_2\norm{\boldsymbol{\theta}_*}_2 +\norm{\boldsymbol{x}_{a}}_{\boldsymbol{M}_t^{-1}}\norm{\sum_{\tau=1}^{t-1}\boldsymbol{x}_{a_\tau,\tau}\epsilon_\tau+\sum_{\tau=1}^{t} \sum_{k\in \mathcal{K}_\tau} \boldsymbol{\tilde{x}}_{k} \tilde{\epsilon}_{\tau}}_{\boldsymbol{M}_t^{-1}}\label{operator norm1}\\
        &\leq\norm{\boldsymbol{x}_{a}}_{\boldsymbol{M}_t^{-1}}\Bigg(\sqrt{\beta}\norm{\boldsymbol{\theta}_*}_2+\norm{\sum_{\tau=1}^{t-1}\boldsymbol{x}_{a_\tau}\epsilon_\tau+\sum_{\tau=1}^{t} \sum_{k\in \mathcal{K}_\tau} \boldsymbol{\tilde{x}}_{k} \tilde{\epsilon}_{\tau}}_{\boldsymbol{M}_t^{-1}}\Bigg)\label{minimal eigen}\,,
        \end{align}
where Eq. (\ref{cauchy1}) is by the Cauchy–Schwarz inequality, Eq. (\ref{operator norm1}) is by the inequality of the matrix operator norm, and Eq. (\ref{minimal eigen}) is because $\lambda_{min}(\boldsymbol{M}_t)\geq\beta,\,\, \norm{\boldsymbol{M}_t^{-\frac{1}{2}}}_2=\sqrt{\lambda_{max}(\boldsymbol{M}_t^{-1})}=\sqrt{\frac{1}{\lambda_{min}(\boldsymbol{M}_t)}}\leq\sqrt{\frac{1}{\beta}}$.

Theorem 1 in \cite{abbasi2011improved} suggests that with probability at least $1-\delta$
\begin{equation}
      \norm{\sum_{\tau=1}^{t-1}\boldsymbol{x}_{a_\tau}\epsilon_\tau+\sum_{\tau=1}^{t} \sum_{k\in \mathcal{K}_\tau} \boldsymbol{\tilde{x}}_{k} \tilde{\epsilon}_{k}}_{\boldsymbol{M}_t^{-1}}
    %   \leq\sqrt{2\log(\frac{1}{\delta})+d\log(1+\frac{b(t)+t}{\beta d})}\,.
    \leq \sqrt{2\log\bigg(\frac{det(\boldsymbol{M}_t)^{\frac{1}{2}}det(\beta\boldsymbol{I})^{\frac{1}{2}}}{\delta}\bigg)}
      \label{equation need det}\,,
\end{equation}
where $det(\cdot)$ denotes the determinate of the argument.

We have
\begin{align}
        &det(\boldsymbol{M}_t)= \prod_{i=1}^{d}\lambda_i\notag\\
        &\leq\big(\frac{\sum_{i=1}^d \lambda_i}{d}\big)^d\label{mean inequality}\\
        &=\big(\frac{trace(\boldsymbol{M}_t)}{d}\big)^d \label{trace equality}\\
        &=\Bigg(\frac{trace(\sum_{\tau=1}^{t-1} \boldsymbol{x}_{a_\tau}\boldsymbol{x}^{\top}_{a_\tau}+\sum_{\tau=1}^{t} \sum_{k\in \mathcal{K}_\tau} \boldsymbol{\tilde{x}}_{k} \boldsymbol{\tilde{x}}^{\top}_{k}+\beta\boldsymbol{I})}{d}\Bigg)^d\nonumber\\&\leq \big(\frac{t+b(t)+\beta d}{d}\big)^d\nonumber\,,
\end{align}
where $\lambda_i,i= 1,2,\ldots,d$ denotes the eigenvalues of the matrix $\boldsymbol{M}_t$, $trace(\boldsymbol{M}_t)$ denotes the trace of $\boldsymbol{M}_t$, Eq. (\ref{mean inequality}) follows by the inequality of arithmetic and geometric means, Eq. (\ref{trace equality}) follows since the trace of a matrix is equal to the sum of its eigenvalues.

Plugging the above inequality and $det(\beta\bI)=\beta^d$ into Eq. (\ref{equation need det}), we can get
\begin{equation}
        \norm{\sum_{\tau=1}^{t-1}\boldsymbol{x}_{a_\tau}\epsilon_\tau+\sum_{\tau=1}^{t} \sum_{k\in \mathcal{K}_\tau} \boldsymbol{\tilde{x}}_{k} \tilde{\epsilon}_{k}}_{\boldsymbol{M}_t^{-1}}\leq\sqrt{2\log(\frac{1}{\delta})+d\log(1+\frac{b(t)+t}{\beta d})}\label{bound on a term}\,.
\end{equation}
The result then follows by plugging Eq. (\ref{bound on a term}) into Eq. (\ref{minimal eigen}), and the fact that $\norm{\boldsymbol{\theta}^*}_2\leq 1$.
\end{proof}

\subsection{Proof of Lemma \ref{lemma3}}
\begin{proof}
Recall that in ConLinUCB-BS, the key-terms are uniformly sampled from the pre-computed barycentric spanner $\cB$, i.e., ${k}\sim \text{unif}(\mathcal{B})$. Therefore we have
\begin{equation}
    \lambda_{\mathcal{B}}\coloneqq\lambda_{\min}(\boldsymbol{E}_{k\sim \text{unif}(\mathcal{B})}[\tilde{\boldsymbol{x}}_k\tilde{\boldsymbol{x}}_k^{\top}])>0\,.
\end{equation}

Using Eq. (\ref{equation11}) in the Lemma 7 in \cite{li2018online}, and the fact that $b_t=b\cdot t$, then with probability at least $1-\delta$ for $\delta\in(0,\frac{1}{8}]$, we have
\begin{equation}
    \label{equation12}
        \lambda_{\min}(\sum_{\tau=1}^t\sum_{k \in \mathcal{K}_{\tau}}\boldsymbol{\tilde{x}}_{k}\boldsymbol{\tilde{x}}_{k}^{\top})\geq \frac{\lambda_{\mathcal{B}} bt}{2}\,,
\end{equation}
for all $t\geq t_0=\frac{256}{b\lambda_{\mathcal{B}}^2}\log(\frac{128d}{\lambda_{\mathcal{B}}^2\delta})$.

Then, by Courant–Fischer theorem \cite{ikebe1987monotonicity}, the fact that $\norm{\boldsymbol{x}_{a}}_2=1$, together with Eq. (\ref{equation12}), we have that for any $t\geq t_0$,with probability at least $1-\delta$ for $\delta\in(0,\frac{1}{8}]$,
\begin{equation*}
    \begin{aligned}
    \norm{\boldsymbol{x}_{a}}_{\boldsymbol{M}_t^{-1}}
    &=\sqrt{\boldsymbol{x}_{a}^{\top}\boldsymbol{M}_t^{-1}\boldsymbol{x}_{a}}\\
    &\leq \max_{\boldsymbol{x}\in\mathbb{R}^d\atop\norm{\boldsymbol{x}}_2=1}\sqrt{\boldsymbol{x}^{\top}\boldsymbol{M}_t^{-1}\boldsymbol{x}}\\
    &=\sqrt{\lambda_{\max}(\boldsymbol{M}_t^{-1})}\\
    &=\sqrt{\lambda_{\max}\bigg((\sum_{\tau=1}^{t-1} \boldsymbol{x}_{a_\tau}\boldsymbol{x}^{\top}_{a_\tau}+\sum_{\tau=1}^{t} \sum_{k\in \mathcal{K}_\tau} \boldsymbol{\tilde{x}}_{k} \boldsymbol{\tilde{x}}^{\top}_{k}+\beta\boldsymbol{I})^{-1}\bigg)}\\
    &=\sqrt{\frac{1}{\lambda_{\min}\bigg(\sum_{\tau=1}^{t-1} \boldsymbol{x}_{a_\tau}\boldsymbol{x}^{\top}_{a_\tau}+\sum_{\tau=1}^{t} \sum_{k\in \mathcal{K}_\tau} \boldsymbol{\tilde{x}}_{k} \boldsymbol{\tilde{x}}^{\top}_{k}+\beta\boldsymbol{I}\bigg)}}\\
    &\leq\sqrt{\frac{1}{\lambda_{\min}(\sum_{\tau=1}^t\sum_{k \in \mathcal{K}_{\tau}}\boldsymbol{\tilde{x}}_{k}\boldsymbol{\tilde{x}}_{k}^{\top})}}\\
    &\leq \sqrt{\frac{2}{\lambda_{\mathcal{B}} bt}}\,.
    \end{aligned}
\end{equation*}
\end{proof}

\subsection{Proof of Theorem \ref{theorem1}}
\begin{proof}
We denote the instantaneous regret at round $t$ as $R_t$. With the definition of the cumulative regret given in Eq. (\ref{equation2}), the arm selection strategy shown in Eq. (\ref{equation7}) and Lemma \ref{lemma2}, we can bound the regret $R_t$ at each round $t=1,2,3,4,...,T$ as follows
\begin{equation}
        \begin{aligned}
        R_t &= \boldsymbol{x}_{a_t^*}^{\top}\boldsymbol{\theta}^*- \boldsymbol{x}_{a_t}^{\top}\boldsymbol{\theta}^*\\
        &=\boldsymbol{x}_{a_t^*}^{\top}(\boldsymbol{\theta}^*-\boldsymbol{\theta}_t)+(\boldsymbol{\theta}_t^{\top}\boldsymbol{x}_{a_t^*}+C_{a_t^*,t})-(\boldsymbol{\theta}_t^{\top}\boldsymbol{x}_{a_t}+C_{a_t,t})\\
        & \quad\ +\boldsymbol{x}_{a_t}^{\top}(\boldsymbol{\theta}_t-\boldsymbol{\theta}^*)+C_{a_t,t}-C_{a_t^*,t}\\
        &\leq 2C_{a_t,t}\,.
        \end{aligned} 
        \label{equation13}
\end{equation}

With Lemma \ref{lemma3}, together with the assumption that $r_t \leq 1$ for any $t$, with probability at least $1-\delta$ for some $\delta \in (0,\frac{1}{4}]$, we can get
\begin{align}
    R(T)&=R(\lceil t_0 \rceil)+\sum_{t=\lceil t_0 \rceil +1}^T R_t\nonumber\\
    &\leq t_0+1+2\sum_{t=\lceil t_0 \rceil}^{T} C_{a_t,t}\nonumber\\
    &\leq t_0+1+2\alpha_t\sum_{t=\lceil t_0 \rceil} ^{T}\norm{\boldsymbol{x}_{a_t}}_{\boldsymbol{M}_t^{-1}}\nonumber\\
    &\leq t_0+1+2\alpha_T\sum_{t=\lceil t_0 \rceil} ^{T}\norm{\boldsymbol{x}_{a_t}}_{\boldsymbol{M}_t^{-1}}\label{nondecreasing}\\
    &\leq t_0+1+2\alpha_T \sum_{t=\lceil t_0 \rceil} ^{T} \sqrt{\frac{2}{\lambda_{\mathcal{B}} bt}}\nonumber\\
    &\leq t_0+1+2\alpha_T \sqrt{\frac{2}{\lambda_{\mathcal{B}}b}}\int_{t_0}^{T} \sqrt{\frac{1}{t}} dt\nonumber\\
    &=\leq t_0+1+4\alpha_T \sqrt{\frac{2}{\lambda_{\mathcal{B}}b}}(\sqrt{T}-\sqrt{t_0})\nonumber\\
    &\leq t_0+1+4\alpha_T \sqrt{\frac{2}{\lambda_{\mathcal{B}}b}}\sqrt{T}\nonumber\,,
\end{align}
where Eq. (\ref{nondecreasing}) follows since $\alpha_t$ is non-decreasing in $t$.

The result follows by plugging in the definition of $t_0$ and $\alpha_T$.
\end{proof}

\subsection{Proof of Theorem \ref{theorem4}}
\begin{proof}
We first prove the following result:\\
For any two positive definite matrices $\boldsymbol{A},\boldsymbol{B}\in\mathbb{R}^{d\times d}$, and any vector $\boldsymbol{x}\in\mathbb{R}^d$, we have:
\begin{equation}
    \norm{\boldsymbol{x}}_{(\boldsymbol{A}+\boldsymbol{B})^{-1}}^2\leq \norm{\boldsymbol{x}}_{\boldsymbol{A}^{-1}}^2\,.
    \label{psd inequality}
\end{equation}

This result can be proved by the following arguments:
\begin{align}
    \norm{\boldsymbol{x}}_{(\boldsymbol{A}+\boldsymbol{B})^{-1}}^2
    &=\boldsymbol{x}^{\top}(\boldsymbol{A}+\boldsymbol{B})^{-1}\boldsymbol{x}\nonumber\\
    &=\boldsymbol{x}^{\top}\big(\boldsymbol{A}^{-1}-\boldsymbol{A}^{-1}(\boldsymbol{B}^{-1}+\boldsymbol{A}^{-1})^{-1}\boldsymbol{A}^{-1}\big)\boldsymbol{x}\label{woodbury matrix identity}\\
    &=\boldsymbol{x}^{\top}\boldsymbol{A}^{-1}\boldsymbol{x}-(\boldsymbol{A}^{-1}\boldsymbol{x})^{\top}(\boldsymbol{B}^{-1}+\boldsymbol{A}^{-1})^{-1}(\boldsymbol{A}^{-1}\boldsymbol{x})\nonumber\\
    &\leq \boldsymbol{x}^{\top}\boldsymbol{A}^{-1}\boldsymbol{x}=\norm{\boldsymbol{x}}_{\boldsymbol{A}^{-1}}^2\label{pd}\,,
\end{align}
where Eq. (\ref{woodbury matrix identity}) follows from the Woodbury matrix identity \cite{woodbury1950inverting}, and Eq. (\ref{pd}) is because $(\boldsymbol{B}^{-1}+\boldsymbol{A}^{-1})^{-1}$ is a positive definite matrix.

With the above result, then following Eq. (\ref{equation13}) and the Cauchy–Schwarz inequality, we can get
\begin{equation}
        \begin{aligned}
                R(T)&\leq 2\sum_{t=1}^{T}C_{a_t,t}\\
               &=2\alpha_t\sum_{t=1} ^{T}\norm{\boldsymbol{x}_{a_t}}_{\boldsymbol{M}_t^{-1}}\\
               &\leq2\alpha_T\sum_{t=1} ^{T}\norm{\boldsymbol{x}_{a_t}}_{\boldsymbol{M}_t^{-1}}\\
                &\leq 2\alpha_T\sqrt{T\sum_{t=1}^{T}\norm{\boldsymbol{x}_{a_t}}^2_{\boldsymbol{M}^{-1}_t}}\\
                &\leq 2\alpha_T\sqrt{T\sum_{t=1}^{T}\norm{\boldsymbol{x}_{a_t}}^2_{\boldsymbol{V}^{-1}_t}}\,.
        \end{aligned}
        \label{equation14}
\end{equation}

Using Lemma 11 in \cite{abbasi2011improved}, with probability at least $1-\delta$, we can get
\begin{equation}
        \sum_{t=1}^T\norm{\boldsymbol{x}_{a,t}}_{\boldsymbol{V}^{-1}_t}^2
        % \leq2d\log(1+\frac{T+1}{\beta d})\,.
        \leq 2\log\bigg(\frac{det(\boldsymbol{V}_T)}{det(\beta\boldsymbol{I})}\bigg)
        \label{equation15}\,.
\end{equation}

Following similar steps as in Eq. (\ref{trace equality}), we can get that
\begin{equation}
    \frac{det(\boldsymbol{V}_T)}{det(\beta\boldsymbol{I})}\leq \big(\frac{T+\beta d}{\beta d}\big)^d\,.
\end{equation}

Therefore we have
\begin{equation}
   \sum_{t=1}^T\norm{\boldsymbol{x}_{a,t}}_{\boldsymbol{V}^{-1}_t}^2 \leq2d\log(1+\frac{T+1}{\beta d})\,. 
   \label{final}
\end{equation}

The result then follows by plugging in the definition of $\alpha_T$ and Eq. (\ref{final}) into Eq. (\ref{equation14}).
\end{proof}
\section{Appendix for Chapter \ref{chapter: aistats}}
\subsection{$\text{Restarted SAVE}^+$-BOB}\label{app:bob}
In this section, we provide the details of our proposed $\text{Restarted SAVE}^+$-BOB algorithm. The $\text{Restarted SAVE}^+$-BOB algorithm is summarized in Algo.\ref{alg:bob algo}. 
We divide the $K$ rounds into $\lceil\frac{K}{H}\rceil$ blocks, with each block having $H$ rounds (except the last one may have less than $H$).
Within each block $i$, we use a fixed $(\alpha_i, w_i)$ pair to run the Restarted $\text{SAVE}^+$ algorithm. To adaptively learn the optimal $(\alpha, w)$ pair without the knowledge of $V_K$ and $B_K$, we employ an 
adversarial bandit algorithm (Exp3 in \cite{auer2002nonstochastic}) as the meta-learner to select $\alpha_i, w_i$ over time for $i\in\lceil\frac{K}{H}\rceil$ blocks. Specifically, in each block, the meta learner selects a $(\alpha, w)$ pair from the candidate pool to feed to Restarted $\text{SAVE}^+$, and the cumulative reward received by Restarted $\text{SAVE}^+$ within the block is fed to the meta-learner as the reward feedback to select a better pair for the next block.

We set $H$ to be $\lceil d^{\frac{2}{5}}K^{\frac{2}{5}}\rceil$, and set the candidate pool of $(\alpha, w)$ pairs for the Exp3 algorithm as:
\begin{align}
    \mathcal{P}=\{(w,\alpha): w\in\mathcal{W}, \alpha\in\mathcal{J}\}\,,\label{bob pool}
\end{align}
where
\begin{align}
\mathcal{W}&=\{w_i=d^{\frac{1}{3}}2^{i-1}|i\in\lceil\frac{1}{3}\log_2 K\rceil+1\}\cup\{w_i=d^{\frac{2}{5}}2^{i-1}|i\in\lceil\frac{2}{5}\log_2 K\rceil+1\}\,,\label{w set}
\end{align}
and 
\begin{align}
    \mathcal{J}&=\{\alpha_i=d^{\frac{1}{3}}2^{-i+1}|i\in\lceil\frac{1}{3}\log_2 K\rceil+1\}\cup\{\alpha_i=d^{\frac{11}{30}}2^{-i+1}|i\in\lceil\frac{11}{30}\log_2 K\rceil+1\}\,.\label{alpha set}
\end{align}
The algorithm also labels all the $\left|\mathcal{P}\right|=\big(\lceil\frac{1}{3}\log_2 K\rceil+\lceil\frac{2}{5}\log_2 K\rceil+2\big)\cdot\big(\lceil\frac{1}{3}\log_2 K\rceil+\lceil\frac{11}{30}\log_2 K\rceil+2\big)$ candidate pairs of parameters in $\mathcal{P}, $\emph{i.e.}, $\mathcal{P}=\{(w_i,\alpha_i)\}_{i=1}^{\left|\mathcal{P}\right|}$. The algorithm initializes $\{s_{j, 1}\}^{\left|\mathcal{P}\right|}_{j=1}$ to be $s_{j,1}=1,\quad\forall j=0,1,\ldots,\left|\mathcal{P}\right|$, which means that at the beginning, the algorithm selects a pair from $\mathcal{P}$ uniformly at random. At the beginning of each block $i\in[\lceil K/H\rceil]$, the meta-learner (Exp3) calculates the distribution $(p_{j, i})^{\left|\mathcal{P}\right|}_{j=1}$ over the candidate set $\mathcal{P}$ by 
\begin{align}
    p_{j, i}=(1-\gamma)\frac{s_{j,i}}{\sum_{u=1}^{\left|\mathcal{P}\right|}s_{u,i}}+\frac{\gamma}{\left|\mathcal{P}\right|+1},\quad\forall j=1,\ldots,\left|\mathcal{P}\right|\,,
\end{align}
where $\gamma$ is defined as 
\begin{align}
    \gamma=\min\left\{1,\sqrt{\frac{(\left|\mathcal{P}\right|+1)\ln(\left|\mathcal{P}\right|+1)}{(e-1)\lceil K/H\rceil}}\right\}\,.
\end{align}
Then, the meta-learner draws a $j_i$ from the distribution $(p_{j, i})^{\left|\mathcal{P}\right|}_{j=1}$, and sets the pair of parameters in block $i$ to be $(w_{j_i},\alpha_{j_i})$, and runs the base algorithm Algo.\ref{alg:1} from scratch in this block with $(w_{j_i},\alpha_{j_i})$, then feeds the cumulative reward in the block $\sum_{k=(i-1)H+1}^{\min\{i\cdot H, K\}}r_k$ to the meta-learner. The meta-learner rescales $\sum_{k=(i-1)H+1}^{\min\{i\cdot H, K\}}r_k$ to $\frac{\sum_{k=(i-1)H+1}^{\min\{i\cdot H, K\}}r_k}{ H+R\sqrt{\frac{H}{2} \log\big(K(\frac{K}{H}+1)\big)} + \frac{2}{3} \cdot R \log\big(K(\frac{K}{H}+1)\big)}$ to make it in the range $[0,1]$ with high probability (supported by Lemma \ref{lemma:bob}). The meta-learner updates the parameter $s_{j_i,i+1}$ to be
\begin{small}
    \begin{align}
    s_{j_i,i+1}=s_{j_i,i}\cdot\exp\left(\frac{\gamma}{(\left|\mathcal{P}\right|+1)p_{j_i,i}}\left(\frac{1}{2}+\frac{\sum_{k=(i-1)H+1}^{\min\{i\cdot H, K\}}r_k}{ H+R\sqrt{\frac{H}{2} \log\big(K(\frac{K}{H}+1)\big)} + \frac{2}{3} \cdot R \log\big(K(\frac{K}{H}+1)\big)}\right)\right)\,,
\end{align}
\end{small}

and keep others unchanged, \emph{i.e.}, $s_{u,i+1}=s_{u,i}, ~\forall u\neq j_i$. After that, the algorithm will go to the next block, and repeat the same process in block $i+1$.
\begin{algorithm*}[t!] 
    \caption{$\text{Restarted SAVE}^+$-BOB} \label{alg:bob algo}
    \begin{algorithmic}[1]
        \REQUIRE total time rounds $K$; problem dimension $d$; noise upper bound $R$; $\alpha > 0$; the upper bound on the $\ell_2$-norm of $\ab$ in $\cD_k (k\ge 1)$, i.e., $A$; the upper bound on the $\ell_2$-norm of $\btheta_k$ $(k\ge 1)$, i.e., $\pnorm$.
    \STATE Initialize $H=\lceil d^{\frac{2}{5}}K^{\frac{2}{5}}\rceil$; $\mathcal{P}$ as defined in Eq.(\ref{bob pool}), and index the $\left|\mathcal{P}\right|=\big(\lceil\frac{1}{3}\log_2 K\rceil+\lceil\frac{2}{5}\log_2 K\rceil+2\big)\cdot\big(\lceil\frac{1}{3}\log_2 K\rceil+\lceil\frac{11}{30}\log_2 K\rceil+2\big)$ items in $\mathcal{P}$, \emph{i.e.}, $\mathcal{P}=\{(w_i,\alpha_i)\}_{i=1}^{\left|\mathcal{P}\right|}$; $\gamma=\min\left\{1,\sqrt{\frac{(\left|\mathcal{P}\right|+1)\ln(\left|\mathcal{P}\right|+1)}{(e-1)\lceil K/H\rceil}}\right\}$; $\{s_{j, 1}\}^{\left|\mathcal{P}\right|}_{j=1}$ is set to $s_{j,1}=1,\quad\forall j=0,1,\ldots,\left|\mathcal{P}\right|$.
    \FOR{$i=1,2,\ldots,\lceil K/H\rceil$}
    \STATE Calculate the distribution $(p_{j, i})^{\left|\mathcal{P}\right|}_{j=1}$ by $p_{j, i}=(1-\gamma)\frac{s_{j,i}}{\sum_{u=1}^{\left|\mathcal{P}\right|}s_{u,i}}+\frac{\gamma}{\left|\mathcal{P}\right|+1},\quad\forall j=1,\ldots,\left|\mathcal{P}\right|$. 
\STATE Set $j_i\leftarrow j$ with probability $p_{j,i}$, and $(w_i, \alpha_i)\leftarrow (w_{i_i}, \alpha_{j_i})$.
\STATE Run Algo.\ref{alg:1} from scratch in block $i$ (\emph{i.e.}, in rounds $k=(i-1)H+1,\ldots,\min\{i\cdot H, K\}$) with $(w, \alpha)=(w_i, \alpha_i)$.
\STATE Update $s_{j_i,i+1}=s_{j_i,i}\cdot\exp\left(\frac{\gamma}{(\left|\mathcal{P}\right|+1)p_{j_i,i}}\left(\frac{1}{2}+\frac{\sum_{k=(i-1)H+1}^{\min\{i\cdot H, K\}}r_k}{ H+R\sqrt{\frac{H}{2} \log\big(K(\frac{K}{H}+1)\big)} + \frac{2}{3} \cdot R \log\big(K(\frac{K}{H}+1)\big)}\right)\right)$, and keep all the others unchanged, \emph{i.e.}, $s_{u,i+1}=s_{u,i}, ~\forall u\neq j_i$.
    \ENDFOR
    \end{algorithmic}
\end{algorithm*}

We have the following theorem to bound the regret of Restarted $\text{SAVE}^+$-BOB.
\begin{theorem}
    By using the BOB framework with Exp3 as the meta-algorithm and Restarted $\text{SAVE}^+$ as the base algorithm, with the candidate pool $\mathcal{P}$ for Exp3 specified as in Eq.(\ref{bob pool}), Eq.(\ref{w set}), Eq.(\ref{alpha set}), and $H=\lceil d^{\frac{2}{5}}K^{\frac{2}{5}}\rceil$, then the regret of Restarted $\text{SAVE}^+$-BOB (Algo.\ref{alg:bob algo}) satisfies
    \begin{align}
        \text{Regret}(K) &= \tilde O(d^{4/5}V_K^{2/5}B_K^{1/5}K^{2/5} + d^{2/3}B_K^{1/3}K^{2/3}+d^{2/5}K^{7/10}).
    \end{align}
    \label{theorem bob}
\end{theorem}
\begin{proof}
    See Appendix \ref{app:thirdalg} for the full proof.
\end{proof}
\begin{remark}
We discuss the regret of Algo.\ref{alg:bob algo} in Corollary \ref{corollary w alpha opt} in the following special cases.  In the case where the \emph{total variance} is small, \emph{i.e.}, $V_K=\Tilde{O}(1)$,  assuming $K^2> d$, our result becomes $\Tilde{O}(d^{2/3}B_K^{1/3}K^{2/3}+d^{1/5}K^{7/10})$, when $d^{14}B_K^{10}>K$, it becomes $\Tilde{O}(d^{2/3}B_K^{1/3}K^{2/3})$, better than all the previous results \cite{cheung2018hedging,zhao2020simple,wang2023revisiting,wei2021non}.
In the worst case where $V_K=O(K)$, our result becomes $\Tilde{O}(d^{4/5}B_K^{1/5}K^{4/5})$.
 
\end{remark}

\subsection{Additional Experiment Setup}\label{app:addexp}
For Restarted-$\algbandit$, we set $\lambda=1$, $\hat\beta_{k}=10$, $w=1000$, and we grid search the variance parameters $\alpha$ and $\gamma$, both among values [1, 1.5, 2, 2.5, 3]. Finally we set $\alpha=1$, and $\gamma=2$. For $\text{Restarted SAVE}^+$ we set $w=1000$, $\hat\beta_{k,\ell} = 2^{-\ell+1}$, and grid search $L$ from 1 to 10 with stepsize of 1 and finally choose $L=6$. For SW-UCB, we set $\lambda=1$, $w=1000$, $\beta_{k}=10$. The Modified EXP3.S requires two parameters $\Bar{\alpha}$ and $\Bar{\gamma}$, and we set $\Bar{\gamma}=0.01$ and $\Bar{\alpha}=\frac{1}{K}$.

To test the algorithms' performance under different total time horizons, we let $K$ vary from $3 \times 10^4$ to $2.4 \times 10^5$, with a stepsize of $3 \times 10^4$, and plot the cumulative regret $\text{Regret}(K)$ for these different total time step $K$. We set $B_K=1, 10, 20, \text{and } K^{1/3}$ to observe their performance in different levels of $B_{K}$.

\subsection{Proof of Theorem \ref{thm:lowerbound}}\label{app:lowerbound}
We prove the lower bound in Theorem \ref{thm:lowerbound} here. We need the following lemma from \cite{zhou2021nearly}. 
\begin{lemma}[Modification from Lemma 25, \cite{zhou2021nearly}]
    Fix a positive real $0 < \delta \leq 1/3$, and positive integers $T, d$ and assume that $T \geq d^2/(2\delta)$. Let $\Delta = \sqrt{d\delta /T} / (4 \sqrt{2})$ and consider the linear bandit problems $\mathcal{L}_{\bmu}$ parameterized with a parameter vector $\bmu \in \{-\Delta, \Delta\}^d$ and action set $\cA = \{-1/\sqrt{d}, 1/\sqrt{d}\}^d$ so that the reward distribution for taking action $\ab \in \cA$ is a Bernoulli distribution $B(\delta + \langle \bmu^*, \ab \rangle)$. Then for any bandit algorithm $\mathcal{B}$ such that
\begin{align}
    \mathbb{E}_{\bmu\sim\text{Unif} \{-\Delta, \Delta\}^d} [\text{Regret}(T, \cL_{\bmu})] \geq \frac{d \sqrt{T\delta} }{8 \sqrt{2}}.
\end{align}
Here $\text{Regret}(T, \cL_{\bmu})$ represents the regret under algorithm $\cB$ on the instance $\cL_{\bmu}$. 
\end{lemma}

Next we prove Theorem \ref{thm:lowerbound}.

\begin{proof}[Proof of Theorem \ref{thm:lowerbound}]
    Let $T<K$ be some constant to be defined. Let $\delta$ be a constant satisfying $2\delta\leq d^2/T$. We create $w = K/T$ number of linear bandit instances with the linear parameter $\bmu_1,\dots,\bmu_w$, where $\bmu_i \sim \{-\Delta, \Delta\}^d, \Delta = \sqrt{d\delta/ T}/4\sqrt{2}$. Our nonstationary instance $\cL_{\bmu_1,\dots,\bmu_w}$ consists of $\cL_{\bmu_1},\dots, \cL_{\bmu_w}$, where at the step $i\cdot T+1,\dots, i\cdot T+T$, $\cL_{\bmu_1,\dots,\bmu_w}$ follows $\cL_{\bmu_i}$. Then by the independence of $\mu_i$, we have 
    \begin{align}
&\mathbb{E}_{\bmu_1,\dots,\bmu_w\sim\text{Unif} \{-\Delta, \Delta\}^d}\text{Regret}(T, \cL_{\bmu_1,\dots,\bmu_w}) = \sum_{i=1}^w \mathbb{E}_{\bmu_i\sim\text{Unif} \{-\Delta, \Delta\}^d} [\text{Regret}(T, \cL_{\bmu_i})] \notag\\
&\geq \frac{d \sqrt{T\delta} }{8 \sqrt{2}}\cdot \frac{K}{T}. 
    \end{align}
    Next we calculate the total variation and total variance for instance $\cL_{\bmu_1,\dots,\bmu_w}$. For each step, the reward distribution is a Bernoulli distribution $B(\delta + \langle \bmu_i, \ab \rangle)$, whose variance is
    \begin{align}
        (\delta + \langle \bmu_i, \ab \rangle )(1-\delta - \langle \bmu_i, \ab \rangle) \leq  (\delta + \langle \bmu_i, \ab \rangle ) \leq 2\delta,
    \end{align}
    where we use the fact $\sqrt{d}\Delta \leq \delta$. Therefore, the total variance over $K$ steps is bounded by
    \begin{align}
        V \leq 2K\delta.
    \end{align}
    Next, for the total variation, we have for any $k,k+1$ belong to the same $\bmu_i$, the variation of $\bmu$ is 0. Note that for any two different $\bmu_i, \bmu_j$, their difference is at most $\|\bmu_i- \bmu_j\| \leq 2\sqrt{d\cdot \Delta^2}$, then the total variation is bounded by
    \begin{align}
        B \leq \frac{K}{T}\cdot 2\Delta\sqrt{d}= \sqrt{d\delta/ T} / (4 \sqrt{2})\frac{K}{T}\cdot 2\sqrt{d} = \frac{dK\sqrt{\delta}}{2\sqrt{2T^3}}.
    \end{align}
    Then we select $\delta$ and $T$ as
    \begin{align}
               & \delta = \frac{V_K}{2K},\ T = \max\{\bigg(\frac{KV_Kd^2}{16 B_K^2}\bigg)^{1/3}, d^2K/V_K\},\ \notag\\
               &\text{satisfying }2K\delta \leq V_K, \frac{dK\sqrt{\delta}}{2\sqrt{2T^3}} \leq B_K,\ T\geq \frac{d^2}{2\delta}.
    \end{align}
    We have the lower bound as  
    \begin{align}
\mathbb{E}_{\bmu_1,\dots,\bmu_w\sim\text{Unif} \{-\Delta, \Delta\}^d}\text{Regret}(T, \cL_{\bmu_1,\dots,\bmu_w}) \geq \Omega(d^{2/3}B_K^{1/3}V_K^{1/3}K^{1/3} \land V_K).
    \end{align}
    Therefore, there must exists $\bmu_1^*,\dots,\bmu_w^*$, satisfying
    \begin{align}
        \text{Regret}(T, \cL_{\bmu_1^*,\dots,\bmu_w^*})  \geq \Omega(d^{2/3}B_K^{1/3}V_K^{1/3}K^{1/3} \land  V_K).\label{help:rrr}
    \end{align}
    Finally, combining \eqref{help:rrr} with the lower bound result in \cite{wei2016tracking} concludes our proof. 
\end{proof}

\subsection{Proof of Lemma \ref{lemma:key}}\label{app:keylemma}
For simplicity, we denote 
\begin{align}
    \hat\beta& := 12\sqrt{d\log(1+\frac{wA^2}{\alpha^2d\lambda})\log(32(\log(\frac{\gamma^2}{\alpha}+1)\frac{w^2}{\delta})}\notag\\
    &+ 30\log(32(\log(\frac{\gamma^2}{\alpha})+1)\frac{w^2}{\delta})\frac{R}{\gamma^2}+ \sqrt{\lambda}\pnorm.
\label{eq:defbanditbeta1}
\end{align}
It is obvious that $\hat\beta\geq\hat\beta_k$ for all $k\in[K]$. We call the restart time rounds \emph{grids} and denote them by $g_1, g_2, \ldots g_{\lceil \frac{K}{w}\rceil-1}$, where $g_i\%w=0$ for all $i\in[\lceil \frac{K}{w}\rceil-1]$. Let $i_k$ be the grid index of time round $k$, \emph{i.e.}, $g_{i_k}\leq k<g_{i_k+1}$.  

For ease of exposition and without loss of generality, we prove the lemma for $k\in[1,w]$. We calculate the estimation difference $|\ba^\top(\hat\btheta_k-\btheta_k)|$ for any $\ab\in\RR^d$, $\|\ab\|_2\leq A$, $k\in[1,w]$. 
By definition:
\begin{equation}
\hat\btheta_k=\hat\bSigma_k^{-1}\bb_k=\hat\bSigma_k^{-1} (\sum_{t=1}^{k-1}\frac{r_t\ab_t}{\bar\sigma_t^2})=\hat\bSigma_k^{-1} (\sum_{t=1}^{k-1}\frac{\ab_t\ab_t^\top\btheta_t}{\bar\sigma_t^2}+\sum_{t=1}^{k-1}\frac{\ab_t\epsilon_t}{\bar\sigma_t^2})\,, 
\end{equation}
where $\hat\Sigma_k=\lambda\bI+\sum_{t=g_{i_k}}^{k-1}\frac{\ab_t\ab_t^\top}{\bar\sigma_t^2}$.

Then we have
\begin{equation}
    \hat\btheta_k-\btheta_k=\hat\bSigma_k^{-1} (\sum_{t=1}^{k-1}\frac{\ab_t\ab_t^\top}{\bar\sigma_t^2}(\btheta_t-\btheta_k)+\sum_{t=1}^{k-1}\frac{\ab_t\epsilon_t}{\bar\sigma_t^2})-\lambda \hat\bSigma_k^{-1}\btheta_k\,.
\end{equation}

Therefore
\begin{align}
    &|\ab^\top(\hat\btheta_k-\btheta_k)|\leq\left|\ab^\top\hat\bSigma_k^{-1}\sum_{t=1}^{k-1}\frac{\ab_t\ab_t^\top}{\bar\sigma_t^2}(\btheta_t-\btheta_k)\right|\notag\\
    &+\|\ab\|_{\hat\bSigma_k^{-1}}\|\sum_{t=1}^{k-1}\frac{\ab_t\epsilon_t}{\bar\sigma_t^2}\|_{\hat\bSigma_k^{-1}}+\lambda \|\ab\|_{\hat\bSigma_k^{-1}}\|\hat\bSigma_k^{-\frac{1}{2}}\btheta_k\|_2\,,\label{l2 bound difference 1 1}
\end{align}
where we use the Cauchy-Schwarz inequality.

% \noindent\textbf{some new thgouths} we have
% \begin{align}
%     \left|\ab^\top\hat\bSigma_k^{-1}\sum_{t=\max\{1,k-w\}}^k\frac{\ab_t\ab_t^\top}{\bar\sigma_t^2}(\btheta_t-\btheta_k)\right|
%          & =     \left|\ab^\top\hat\bSigma_k^{-1/2}\cdot\sum_{t=\max\{1,k-w\}}^k\hat\bSigma_k^{-1/2}\frac{\ab_t\ab_t^\top}{\bar\sigma_t^2}(\btheta_t-\btheta_k)\right|\notag \\
%          & \leq \|\ab^\top\hat\bSigma_k^{-1/2}\|_2\cdot\|\sum_{t=\max\{1,k-w\}}^k\hat\bSigma_k^{-1/2}\frac{\ab_t\ab_t^\top}{\bar\sigma_t^2}(\btheta_t-\btheta_k)\|\notag \\
%          & \leq A\cdot \|\ab^\top\hat\bSigma_k^{-1/2}\|_2\cdot \sum_{t=\max\{1,k-w\}}^k\|\hat\bSigma_k^{-1/2}\frac{\ab_t}{\bar\sigma_t^2}\|_2\cdot \|\btheta_t-\btheta_k\|_2\notag \\
%          & \leq A\|\ab^\top\hat\bSigma_k^{-1/2}\|_2\sum_{t=\max\{1,k-w\}}^k\|\hat\bSigma_t^{-1/2}\ab_t\|_2/{\bar\sigma_t^2}\cdot \|\btheta_t-\btheta_k\|_2\notag \\
%          & \leq A\|\ab^\top\hat\bSigma_k^{-1/2}\|_2\sum_{t=\max\{1,k-w\}}^k\|\hat\bSigma_t^{-1/2}\ab_t\|_2/(\gamma^2\|\hat\bSigma_t^{-1/2}\ab_t\|_2)\cdot \|\btheta_t-\btheta_k\|_2\notag \\
%          & = A/\gamma^2\cdot \|\ab\|_{\hat\bSigma_k^{-1}}\sum_{t=\max\{1,k-w\}}^k\|\btheta_t-\btheta_k\|_2
% \end{align}

For the first term, we have that for any $k\in[1,w]$
\begin{align}
    &\left|\ab^\top\hat\bSigma_k^{-1}\sum_{t=1}^k\frac{\ab_t\ab_t^\top}{\bar\sigma_t^2}(\btheta_t-\btheta_k)\right|\notag\\
         & \leq \sum_{t=1}^{k-1} |   \ab^{\top} \hat\bSigma_k^{-1} \frac{\ab_t}{\bar\sigma_t} | \cdot | \frac{\ab_t}{\bar\sigma_t}^\top (\sum_{s = t}^{k-1} (\btheta_{s} - \btheta_{s+1})) | \tag{triangle inequality } \\
     & \leq \sum_{t=1}^{k-1}  |  \ab^{\top} \hat\bSigma_k^{-1} \frac{\ab_t}{\bar\sigma_t} | \cdot \| \frac{\ab_t}{\bar\sigma_t}\|_2 \cdot \| \sum_{s = t}^{k-1} (\btheta_{s} - \btheta_{s+1})\|_2 \tag{Cauchy-Schwarz}
     \\
     & \leq \frac{A}{\alpha}\sum_{t=1}^{k-1}   |  \ab^{\top} \hat\bSigma_k^{-1} \frac{\ab_t}{\bar\sigma_t} | \cdot   \| \sum_{s = t}^{k-1} (\btheta_{s} - \btheta_{s+1})\|_2 \tag{$\| \ab_t\| \leq A$, $\bar\sigma_t\geq\alpha$} \\
    %  \end{align*} 
    %  \begin{align*}
     & \leq \frac{A}{\alpha}\sum_{s =1}^{k-1} \sum_{t = 1}^{s} | \ab^{\top}\hat\bSigma_k^{-1} \frac{\ab_t}{\bar\sigma_t} | \cdot \| \btheta_{s} - \btheta_{s+1}\|_2
     \tag{$\sum_{t =1}^{k-1} \sum_{s=t}^{k-1} =  \sum_{s =1}^{k-1} \sum_{t =1}^{s} $} \\
     & \leq \frac{A}{\alpha}\sum_{s = 1}^{k-1} \sqrt{ \bigg[ \sum_{t =1}^{s} \ab^\top \hat\bSigma_k^{-1} \ab \bigg]  \cdot \biggl [ \sum_{t =1}^{s} \frac{\ab_t}{\bar\sigma_t}^\top \hat\bSigma_k^{-1}  \frac{\ab_t}{\bar\sigma_t}\bigg] }
     \cdot \norm{ \btheta_{s} - \btheta_{s+1}}_2
     \tag{Cauchy-Schwarz}
     \\ & \leq \frac{A}{\alpha}\sum_{s =1}^{k-1} \sqrt{ \bigg[ \sum_{t =1}^{s} \ab^\top \hat\bSigma_k^{-1}\ab \bigg] \cdot d }
     \cdot \norm{ \btheta_{s} - \btheta_{s+1}}_2
     \tag{$(\star)$} \\
     & \leq \frac{A\|\ab\|_2}{\alpha}\sqrt{d} \sum_{s =1}^{k-1} \sqrt{ \frac{\sum_{t =1}^{k-1} 1 }{ \lambda }} \cdot \norm{ \btheta_{s} - \btheta_{s+1}}_2  \tag{$\lambda_{\max}(\hat\bSigma_k^{-1}) \leq \frac{1}{\lambda}$} \\
     & \leq \frac{A^2}{\alpha}\sqrt{\frac{d w}{\lambda }} \sum_{s =1}^{k-1} \norm{ \btheta_{s} - \btheta_{s+1}}_2\,,
\end{align}
where the inequality $(\star)$ follows from the fact that $\sum_{t =1}^{s} \frac{\ab_t}{\bar\sigma_t}^\top \hat\bSigma_k^{-1}  \frac{\ab_t}{\bar\sigma_t}\leq d$ that can be proved as follows. We have $\sum_{t =1}^{k-1} \frac{\ab_t}{\bar\sigma_t}^\top \hat\bSigma_k^{-1}  \frac{\ab_t}{\bar\sigma_t}= \sum_{t =1}^{k-1} \text{tr}\left( \frac{\ab_t}{\bar\sigma_t}^\top \hat\bSigma_k^{-1}  \frac{\ab_t}{\bar\sigma_t}\right) = \text{tr}\left( \hat\bSigma_k^{-1}\sum_{t =1}^{k-1} \frac{\ab_t}{\bar\sigma_t}   \frac{\ab_t}{\bar\sigma_t}^\top \right)$. Given the eigenvalue decomposition $\sum_{t =1}^{k-1}  \frac{\ab_t}{\bar\sigma_t}   \frac{\ab_t}{\bar\sigma_t}^\top = \text{diag}(\lambda_1, \ldots, \lambda_d)^\top$, we have $\hat\bSigma_k = \text{diag}(\lambda_1 + \lambda, \ldots, \lambda_d + \lambda)^\top$, and $\text{tr}\left( \hat\bSigma_k^{-1}\sum_{t =1}^{k-1} \frac{\ab_t}{\bar\sigma_t}   \frac{\ab_t}{\bar\sigma_t}^\top \right) = \sum_{i=1}^d \frac{\lambda_j}{\lambda_j + \lambda } \leq d$.
% \begin{equation}
%     \|\hat\bSigma_k^{-1} \sum_{t=\max\{1,k-w\}}^k\frac{\ba_t\ba_t^\top}{\bar\sigma_t^2}(\btheta_t-\btheta_k)\|_2\leq \sqrt{\frac{dw}{\lambda}}\sum_{t=\max\{1,k-w\}}^k \|\btheta_t-\btheta_{t+1}\|_2\,.
% \end{equation}

For the second term, by the assumption on $\epsilon_k$, we know that
\begin{align}
    &|\epsilon_k/\bar\sigma_k| \leq R/\alpha,\notag\\
    &|\epsilon_k/\bar\sigma_k|\cdot\min\{1,\|\ab_k/\bar\sigma_k\|_{\hat\bSigma_k^{-1}}\} \leq R\|\ab_k\|_{\hat\bSigma_k^{-1}}/\bar\sigma_k^2 \leq R/\gamma^2,\notag \\
    &\EE[\epsilon_k|\ab_{1:k}, \epsilon_{1:k-1}] = 0,\ \EE [(\epsilon_k/\bar\sigma_k)^2|\ab_{1:k}, \epsilon_{1:k-1}] \leq 1,\ \|\ab_k/\bar\sigma_k\|_2 \leq A/\alpha,\notag
\end{align}
Therefore, setting $\cG_k = \sigma(\ab_{1:k}, \epsilon_{1:k-1})$, 
and using that $\sigma_k$ is $\cG_k$-measurable, applying
Theorem \ref{lemma:concentration_variance} to $(\bx_k,\eta_k)=(\ba_k/\bar\sigma_k,\epsilon_k/\bar\sigma_k)$ with $\epsilon = R/\gamma^2$ , we get that
with probability at least $1-\delta$, for all $k \in[1,w]$, 
\begin{align}
    &\|\sum_{t=1}^{k-1}\frac{\ab_t\epsilon_t}{\bar\sigma_t^2}\|_{\hat\bSigma_k^{-1}}\leq 12\sqrt{d\log(1+\frac{(k\%w)A^2}{\alpha^2d\lambda})\log(32(\log(\frac{\gamma^2}{\alpha}+1)\frac{(k\%w)^2}{\delta})}\notag\\
    &+ 30\log(32(\log(\frac{\gamma^2}{\alpha})+1)\frac{(k\%w)^2}{\delta})\frac{R}{\gamma^2}.
\end{align}
For the last term
\begin{align}
&\lambda \|\ab\|_{\hat\bSigma_k^{-1}}\|\hat\bSigma_k^{-\frac{1}{2}}\btheta_k\|_2\leq\lambda \|\ab\|_{\hat\bSigma_k^{-1}}\|\hat\bSigma_k^{-\frac{1}{2}}\|_2\|\btheta_k\|_2\notag\\
&\leq \lambda \|\ab\|_{\hat\bSigma_k^{-1}}\frac{1}{\sqrt{\lambda_{\text{min}}(\hat\bSigma_k)}}\|\btheta_k\|_2\leq\sqrt{\lambda}B \|\ab\|_{\hat\bSigma_k^{-1}}\,,
\end{align}
where we use the fact that $\lambda_{\text{min}}(\hat\bSigma_k)\geq\lambda$.

Therefore, with probabilty at least $1-\delta$, we have
\begin{align}
    &|\ab^\top(\hat\btheta_k-\btheta_k)|\notag \\
    &\leq \frac{A^2}{\alpha}\sqrt{\frac{dw}{\lambda}}\sum_{t=1}^{k-1} \|\btheta_t-\btheta_{t+1}\|_2\notag\\
&\quad+\|\ab\|_{\hat\bSigma_k^{-1}}\bigg(12\sqrt{d\log(1+\frac{(k\%w)A^2}{\alpha^2d\lambda})\log(32(\log(\frac{\gamma^2}{\alpha}+1)\frac{(k\%w)^2}{\delta})} \notag\\
    &+ 30\log(32(\log(\frac{\gamma^2}{\alpha})+1)\frac{(k\%w)^2}{\delta})\frac{R}{\gamma^2}+ \sqrt{\lambda}\pnorm\bigg)\notag\\
    % 12\sqrt{d\log(1+\frac{wA^2}{\alpha^2d\lambda})\log(32(\log(\frac{\gamma^2}{\alpha}+1)\frac{w^2}{\delta})}+ 30\log(32(\log(\frac{\gamma^2}{\alpha})+1)\frac{w^2}{\delta})\frac{R}{\gamma^2}+\sqrt{\lambda}B\bigg)\notag\\
    &=\frac{A^2}{\alpha}\sqrt{\frac{dw}{\lambda}}\sum_{t=1}^{k-1} \|\btheta_t-\btheta_{t+1}\|_2+\hat\beta_k\|\ab\|_{\hat\bSigma_k^{-1}}\,,
\end{align}
where $\hat\beta_k$ is defined in Eq.(\ref{eq:defbanditbeta}).

\subsection{Proof for Theorem \ref{thm: regret for algo1 final}}\label{app:firstalg}

For simplicity of analysis, we only analyze the regret over the first grid, \emph{i.e.}, we try to analyze $\text{Regret}(\tilde K)$ for $\tilde K \in[1,w]$. Denote $\event_{1}$ as the event when Lemma \ref{lemma:key} holds. 
Therefore, under event $\event_{1}$, for any $\tilde K \in[1,w]$, the regret can be bounded by
\begin{align}
    \text{Regret}(\tilde K) &= \sum_{k=1}^{\tilde K}\big[\la \ab_k^*-\ab_k, \btheta_k\ra\big]\notag\\
    &=\sum_{k=1}^{\tilde K}\big[\la \ab_k^*, \btheta_k-\hat\btheta_k\ra+(\la\ab_k^*,\hat\btheta_k\ra+\hat\beta_k\|\ab_k^*\|_{\hat\bSigma_k^{-1}})-(\la\ab_k,\hat\btheta_k\ra+\hat\beta_k\|\ab_k\|_{\hat\bSigma_k^{-1}})\notag \\
    &\quad +\la \ab_k,\hat\btheta_k-\btheta_k\ra+\hat\beta_k\|\ab_k\|_{\hat\bSigma_k^{-1}}-\hat\beta_k\|\ab_k^*\|_{\hat\bSigma_k^{-1}}\big]\notag\\
    &\leq \frac{2A^2}{\alpha}\sqrt{\frac{dw}{\lambda}}\sum_{k=1}^{\tilde K}\sum_{t=1}^{k-1} \|\btheta_t-\btheta_{t+1}\|_2+2\sum_{k=1}^{\tilde K}\min\Big\{1, \hat\beta_k\|\ab_k\|_{\hat\bSigma_k^{-1}}\Big\}\,,\label{regret for restarted woful 1}
\end{align}
where in the last inequality we use the definition of event $\event_1$, the arm selection rule in Line 7 of Algo.\ref{alg:reweightbandit}, and $0 \leq \la \ab_k^*, \btheta^*\ra - \la \ab_k, \btheta^*\ra \leq 2$.

Then we will bound the two terms in Eq.(\ref{regret for restarted woful 1}). 

For the first term, we have
\begin{align}
    &\frac{2A^2}{\alpha}\sqrt{\frac{dw}{\lambda}}\sum_{k=1}^{\tilde K}\sum_{t=1}^{k-1} \|\btheta_t-\btheta_{t+1}\|_2\notag\\
    &=\frac{2A^2}{\alpha}\sqrt{\frac{dw}{\lambda}}\sum_{t=1}^{\tilde K-1} \sum_{k=t}^{\tilde K} \|\btheta_t-\btheta_{t+1}\|_2\notag\\
    &\leq \frac{2A^2}{\alpha}\sqrt{\frac{dw}{\lambda}}w\sum_{t=1}^{\tilde K-1} \|\btheta_t-\btheta_{t+1}\|_2\,.\label{regret bound sum theta 1}
\end{align}

To bound the second term in Eq.(\ref{regret for restarted woful 1}), we decompose the set $[\tilde K]$ into a union of two disjoint subsets $[K] = \cI_1 \cup \cI_2$. 
\begin{align}
    \cI_1 = \Big\{k \in [\tilde K]:\|\frac{\ab_k}{\bar\sigma_k}\|_{\hat\bSigma_k^{-1}} \geq 1 \Big\},\ \cI_2 = \Big\{k \in [\tilde K]:\|\frac{\ab_k}{\bar\sigma_k}\|_{\hat\bSigma_k^{-1}} < 1 \Big\}.
\end{align}
Then the following upper bound of $|\cI_1|$ holds:
\begin{align}
    |\cI_1| &= \sum_{k\in\cI_1}\min\Big\{1, \|\frac{\ab_k}{\bar\sigma_k}\|_{\hat\bSigma_k^{-1}}^2\Big\} \notag\\
    &\leq \sum_{k=1}^{\tilde K}\min\Big\{1, \|\frac{\ab_k}{\bar\sigma_k}\|_{\hat\bSigma_k^{-1}}^2\Big\} \notag\\
    % &=\sum_{i=0}^{\lceil \frac{K}{w}\rceil-1}\sum_{k=i\cdot w+1}^{(i+1)w} \min\Big\{1, \|\frac{\ab_k}{\bar\sigma_k}\|_{\hat\bSigma_k^{-1}}^2\Big\}\notag\\
    % &\leq \sum_{i=0}^{\lceil \frac{K}{w}\rceil-1}\sum_{k=i\cdot w+1}^{(i+1)w} \min\Big\{1, \|\frac{\ab_k}{\bar\sigma_k}\|_{\Tilde\bSigma_k^{-1}}^2\Big\}\label{lemma cheung2019 1 1}\\
    &\leq 2d\iota,\label{eq: bound l1 1}
\end{align}
where $\iota = \log(1+\frac{wA^2}{d\lambda\alpha^2})$, the first equality holds since $\|\frac{\xb_k}{\bar\sigma_k}\|_{\hat\bSigma_k^{-1}} \geq 1$ for $k \in \cI_1$, the last inequality holds due to Lemma \ref{Lemma:abba} together with the fact $\|\frac{\ab_k}{\bar\sigma_k}\|_2 \leq \frac{A}{\alpha}$ since $\bar\sigma_k\geq\alpha$ and $\|\ab_k\|_2\leq A$.

Then, we have
\begin{align}
    &\sum_{k=1}^{\tilde K}\min\Big\{1, \hat\beta_k\|\ab_k\|_{\hat\bSigma_k^{-1}}\Big\}\notag \\
    & =\sum_{k\in \cI_1}\min\Big\{1, \bar\sigma_k\hat\beta_k\|\frac{\ab_k}{\bar\sigma_k}\|_{\hat\bSigma_k^{-1}}\Big\} + \sum_{k\in \cI_2}\min\Big\{1, \bar\sigma_k\hat\beta_k\|\frac{\ab_k}{\bar\sigma_k}\|_{\hat\bSigma_k^{-1}}\Big\}\notag \\
    & \leq \bigg[\sum_{k\in \cI_1} 1\bigg] + \sum_{k\in \cI_2} \bar\sigma_k\hat\beta_k\|\frac{\ab_k}{\bar\sigma_k}\|_{\hat\bSigma_k^{-1}}\notag \\
    & \leq 2d\iota + \hat\beta\sum_{k\in \cI_2} \bar\sigma_k\|\frac{\ab_k}{\bar\sigma_k}\|_{\hat\bSigma_k^{-1}},\label{bound l1 l2}
\end{align}
where the first inequality holds since  $\min\{1,x\}\le 1$ and also $\min\{1,x\}\le x$, 
the second inequality holds by Eq.(\ref{eq: bound l1 1}), and the fact the $\hat\beta\geq\hat\beta_k$ for all $k\in[K]$ ($\hat\beta$ is defined in Eq.(\ref{eq:defbanditbeta1})). Next we further bound the second summation term in \eqref{bound l1 l2}. We decompose $\cI_2 = \cJ_1 \cup \cJ_2$, where
\begin{align}
    &\cJ_1 = \bigg\{k \in \cI_2:\bar\sigma_k = \sigma_k \cup\bar\sigma_k = \alpha \bigg\},\ \cJ_2 = \bigg\{k \in \cI_2:\bar\sigma_k = \gamma\sqrt{\|\ab_k\|_{\hat\bSigma_k^{-1}}}\bigg\}.\notag
\end{align}
Then $\sum_{k\in \cI_2} \bar\sigma_k\|\frac{\ab_k}{\bar\sigma_k}\|_{\hat\bSigma_k^{-1}}=\sum_{k\in \cJ_1} \bar\sigma_k\|\frac{\ab_k}{\bar\sigma_k}\|_{\hat\bSigma_k^{-1}}$+$\sum_{k\in \cJ_2} \bar\sigma_k\|\frac{\ab_k}{\bar\sigma_k}\|_{\hat\bSigma_k^{-1}}$. First, for $k\in \cJ_1$, we have
\begin{align}
    \sum_{k\in \cJ_1} \bar\sigma_k\|\frac{\ab_k}{\bar\sigma_k}\|_{\hat\bSigma_k^{-1}}&\leq \sum_{k\in \cJ_1} (\sigma_k+\alpha)\min\bigg\{1,\|\frac{\ab_k}{\bar\sigma_k}\|_{\hat\bSigma_k^{-1}}\bigg\}\notag\\
    &\leq \sqrt{\sum_{k=1}^{\tilde K}(\sigma_k + \alpha)^2}\sqrt{\sum_{k=1}^{\tilde K} \min\bigg\{1,\|\frac{\ab_k}{\bar\sigma_k}\|_{\hat\bSigma_k^{-1}}\bigg\}^2}\notag\\
        &\leq \sqrt{2\sum_{k=1}^{\tilde K}(\sigma_k^2 + \alpha^2)}\sqrt{\sum_{k=1}^{\tilde K} \min\bigg\{1,\|\frac{\ab_k}{\bar\sigma_k}\|_{\hat\bSigma_k^{-1}}^2\bigg\}}\notag\\
    % &=\sqrt{2\sum_{k=1}^K(\sigma_k^2 + \alpha^2)^2}\sqrt{\sum_{i=0}^{\lceil\frac{K}{w}\rceil-1}\sum_{k=i\cdot w}^{(i+1)w} \min\bigg\{1,\|\frac{\ab_k}{\bar\sigma_k}\|_{\hat\bSigma_k^{-1}}^2\bigg\}}\notag\\
    % &\leq \sqrt{2\sum_{k=1}^K(\sigma_k^2 + \alpha^2)^2}\sqrt{\sum_{i=0}^{\lceil\frac{K}{w}\rceil-1}\sum_{k=i\cdot w}^{(i+1)w} \min\bigg\{1,\|\frac{\ab_k}{\bar\sigma_k}\|_{\Tilde\bSigma_k^{-1}}\bigg\}^2}\label{lemma: chueng 2 1}\\
    &\leq 2\sqrt{\sum_{k=1}^{\tilde K}\sigma_k^2 + {\tilde K}\alpha^2}\sqrt{d\iota}\label{regret j1 1}\,,
\end{align}
where the first inequality holds since $\bar\sigma_k \leq \sigma_k + \alpha$ for $k \in \cJ_1$ and $\|\frac{\ab_k}{\bar\sigma_k}\|_{\hat\bSigma_k^{-1}} \leq 1$ since $k \in \cJ_1 \subseteq \cI_2$, the second inequality holds by Cauchy-Schwarz inequality, the third inequality holds due to $(a+b)^2\leq 2(a^2+b^2)$, and the last inequality holds due to Lemma \ref{Lemma:abba}.

Finally we bound the summation for $k\in\cJ_2$. When $k\in \cJ_2$, we have $\bar\sigma_k=\gamma^2\|\frac{\ab_k}{\bar\sigma_k}\|_{\hat\Sigma_k^{-1}}$. Therefore we have
\begin{align}
    \sum_{k\in \cJ_2} \bar\sigma_k\|\frac{\ab_k}{\bar\sigma_k}\|_{\hat\bSigma_k^{-1}}&=\sum_{k\in \cJ_2} \gamma^2\|\frac{\ab_k}{\bar\sigma_k}\|_{\hat\bSigma_k^{-1}}^2\notag\\
    &\leq\sum_{k=1}^{\tilde K} \gamma^2\min\bigg\{1,\|\frac{\ab_k}{\bar\sigma_k}\|^2_{\hat\bSigma_k^{-1}}\bigg\}\notag\\
    % &=\gamma^2\sum_{i=0}^{\lceil\frac{K}{w}\rceil-1}\sum_{k=i\cdot w}^{(i+1)w}   \min\bigg\{1,\|\frac{\ab_k}{\bar\sigma_k}\|_{\hat\bSigma_k^{-1}}^2\bigg\}\notag\\
    % &\leq \gamma^2\sum_{i=0}^{\lceil\frac{K}{w}\rceil-1}\sum_{k=i\cdot w}^{(i+1)w}   \min\bigg\{1,\|\frac{\ab_k}{\bar\sigma_k}\|_{\Tilde\bSigma_k^{-1}}^2\bigg\}\notag\\
    &\leq 2\gamma^2 d\iota\label{bound j2 1}\,,
\end{align}
where in the first inequality we use the fact that $\|\frac{\ab_k}{\bar\sigma_k}\|_{\hat\bSigma_k^{-1}} \leq 1$ since $k \in \cJ_2 \subseteq \cI_2$, and in the last inequality we use Lemma \ref{Lemma:abba}.

Therefore, with Eq.(\ref{regret for restarted woful 1}), Eq.(\ref{regret bound sum theta 1}), Eq.(\ref{bound l1 l2}), Eq.(\ref{regret j1 1}), Eq.(\ref{bound j2 1}), we can get the regret upper bound for $\tilde K\in[1,w]$
\begin{align}
    \text{Regret}(\tilde K)
    &\leq \frac{2A^2w^{\frac{3}{2}}}{\alpha}\sqrt{\frac{d}{\lambda}}\sum_{k=1}^{\tilde K-1} \|\btheta_k-\btheta_{k+1}\|_2 +4\hat\beta\sqrt{d\iota}\sqrt{\sum_{k\in[\tilde K]}\sigma_k^2 + w\alpha^2}\notag\\
    &+4d\iota\gamma^2\hat\beta+4d\iota\,.
\end{align}
Therefore, by the same deduction, we can get that
\begin{align}
    \text{Regret}([g_i,g_{i+1}])
    &\leq \frac{2A^2w^{\frac{3}{2}}}{\alpha}\sqrt{\frac{d}{\lambda}}\sum_{k=g_i}^{g_{i+1}-1} \|\btheta_k-\btheta_{k+1}\|_2 \notag\\  &+4\hat\beta\sqrt{d\iota}\sqrt{\sum_{k=g_i}^{g_{i+1}}\sigma_k^2 + w\alpha^2}+4d\iota\gamma^2\hat\beta+4d\iota\,,
\end{align}
where we use $\text{Regret}([g_i,g_{i+1}])$ to denote the regret accumulated in the time period $[g_i,g_{i+1}]$.

Finally, without loss of generality, we assume $K\%w=0$. Then we have
\begin{align}
        &\text{Regret}(\tilde K)\notag\\
        &=\sum_{i=0}^{\frac{K}{w}-1}\text{Regret}([g_i,g_{i+1}])\notag\\
            &\leq \frac{2A^2w^{\frac{3}{2}}}{\alpha}\sqrt{\frac{d}{\lambda}}\sum_{i=0}^{\frac{K}{w}-1}\sum_{k=g_i}^{g_{i+1}-1} \|\btheta_k-\btheta_{k+1}\|_2 +4\hat\beta\sqrt{d\iota}\sum_{i=0}^{\frac{K}{w}-1}\sqrt{\sum_{k=g_i}^{g_{i+1}}\sigma_k^2 + w\alpha^2}\notag\\
            &+\frac{4d\iota\gamma^2\hat\beta K}{w}+\frac{4dK\iota}{w}\notag\\
            &\leq \frac{2A^2w^{\frac{3}{2}}}{\alpha}\sqrt{\frac{d}{\lambda}}\sum_{k=1}^{K-1} \|\btheta_k-\btheta_{k+1}\|_2 \notag\\
            &+4\hat\beta\sqrt{d\iota}\sqrt{\frac{K}{w}\sum_{i=0}^{\frac{K}{w}-1}(\sum_{k=g_i}^{g_{i+1}}\sigma_k^2 + w\alpha^2)}+\frac{4d\iota\gamma^2\hat\beta K}{w}+\frac{4dK\iota}{w}\notag\\
            &\leq \frac{2A^2w^{\frac{3}{2}}B_K}{\alpha}\sqrt{\frac{d}{\lambda}} +4\hat\beta\sqrt{\frac{Kd\iota}{w}}\sqrt{\sum_{k=1}^{K}\sigma_k^2 + K\alpha^2}+\frac{4d\iota\gamma^2\hat\beta K}{w}+\frac{4dK\iota}{w}\notag\,,
\end{align}
where in the second inequality we use Cauchy-Schwarz inequality, and the last inequality holds due to $\sum_{k\in[K-1]}\|\btheta_k-\btheta_{k+1}\|_2\leq B_K$.

\subsection{Proof for Theorem \ref{thm:regret1}}\label{app:secalg}
Recall that we call the restart time rounds \emph{grids} and denote them by $g_1, g_2, \ldots g_{\lceil \frac{K}{w}\rceil-1}$, where $g_i\%w=0$ for all $i\in[\lceil \frac{K}{w}\rceil-1]$. Let $i_k$ be the grid index of time round $k$, \emph{i.e.}, $g_{i_k}\leq k<g_{i_k+1}$. We denote $\hat\Psi_{k, \ell}:=\{t: t\in[g_{i_k},k-1], \ell_t=\ell\}$.

For simplicity of analysis, we first try to bound the regret over the first grid, \emph{i.e.}, we try to analyze $\text{Regret}(\tilde K)$ for $\tilde K \in[1,w]$. Note that in this case, for any $k\in[\tilde K]$ with $\tilde K\in[1,w]$, we have $g_{i_k}=1$, so $\hat\Psi_{k, \ell}:=\{t: t\in[1,k-1], \ell_t=\ell\}$.

First, we calculate the estimation difference $|\ba^\top(\hat\btheta_{k,\ell}-\btheta_k)|$ for any $\ab\in\RR^d$, $\|\ab\|_2\leq A$. Recall that by definition, $\hat\Sigma_{k,\ell}=2^{-2\ell}\bI+\sum_{t\in\hat\Psi_{k,\ell}}w_t^2\ab_t\ab_t^\top$, $\hat\bb_{k,\ell}=\sum_{t\in\hat\Psi_{k,\ell}}w_t^2r_t\ab_t$, and
\begin{align}
\hat\btheta_{k,\ell}&=\hat\Sigma_{k,\ell}^{-1}\hat\bb_{k,\ell}=\hat\Sigma_{k,\ell}^{-1}(\sum_{t\in\hat\Psi_{k,\ell}}w_t^2r_t\ab_t)=\hat\Sigma_{k,\ell}^{-1}(\sum_{t\in\hat\Psi_{k,\ell}}w_t^2\ab_t\ab_t^\top\btheta_t+\sum_{t\in\hat\Psi_{k,\ell}}w_t^2\ab_t\epsilon_t)\notag\,.
\end{align}
Then we have 
\begin{align}
    \hat\btheta_{k,\ell}-\btheta_k=\hat\Sigma_{k,\ell}^{-1}(\sum_{t\in\hat\Psi_{k,\ell}}w_t^2\ab_t\ab_t^\top(\btheta_t-\btheta_k)+\sum_{t\in\hat\Psi_{k,\ell}}w_t^2\ab_t\epsilon_t)-2^{-2\ell}\hat\Sigma_{k,\ell}^{-1}\btheta_k\,.
\end{align}
Therefore, we can get
\begin{align}
    &|\ab^\top(\hat\btheta_{k,\ell}-\btheta_k)|\leq\left|\ab^\top\hat\bSigma_{k,\ell}^{-1}\sum_{t\in\hat\Psi_{k,\ell}}w_t^2{\ab_t\ab_t^\top}(\btheta_t-\btheta_k)\right|+\|\ab\|_{\hat\bSigma_{k,\ell}^{-1}}\|\sum_{t\in\hat\Psi_{k,\ell}}w_t^2{\ab_t\epsilon_t}\|_{\hat\bSigma_{k,\ell}^{-1}}\notag\\
    &+2^{-2\ell} \|\ab\|_{\hat\bSigma_{k,\ell}^{-1}}\|\hat\bSigma_{k,\ell}^{-\frac{1}{2}}\btheta_k\|_2\,,\label{l2 bound difference 12}
\end{align}
where we use the Cauchy-Schwarz inequality.

For the first term, we have that for any $k\in[1,w]$
\begin{align}  &\left|\ab^\top\hat\bSigma_{k,\ell}^{-1}\sum_{t\in\hat\Psi_{k,\ell}}w_t^2{\ab_t\ab_t^\top}(\btheta_t-\btheta_k)\right|\notag\\
         & \leq \sum_{t\in\hat\Psi_{k,\ell}} |   \ab^{\top} \bSigma_{k,\ell}^{-1} w_t\ab_t | \cdot | w_t\ab_t^\top (\sum_{s = t}^{k-1} (\btheta_{s} - \btheta_{s+1})) | \tag{triangle inequality } \\
     & \leq \sum_{t\in\hat\Psi_{k,\ell}}|  \ab^{\top} \bSigma_{k,\ell}^{-1} w_t{\ab_t} | \cdot \| w_t\ab_t\|_2 \cdot \| \sum_{s = t}^{k-1} (\btheta_{s} - \btheta_{s+1})\|_2 \tag{Cauchy-Schwarz}
     \\
     & \leq {A}\sum_{t\in\hat\Psi_{k,\ell}}   |  \ab^{\top} \hat\bSigma_{k,\ell}^{-1} w_t{\ab_t} | \cdot   \| \sum_{s = t}^{k-1} (\btheta_{s} - \btheta_{s+1})\|_2 \tag{$\| \ab_t\| \leq A$, $w_t=\frac{2^{-\ell_t}}{\|\ab_t\|_{\hat\Sigma_{t,\ell_t}^{-1}}}\leq 1$} \\
    %  \end{align*} 
    %  \begin{align*}
     & \leq {A}\sum_{s=1}^{k-1}\sum_{t\in\hat\Psi_{k,\ell}}   |  \ab^{\top} \hat\bSigma_{k,\ell}^{-1} w_t{\ab_t} | \cdot   \| \btheta_{s} - \btheta_{s+1}\|_2\notag \\
     & \leq {A}\sum_{s =1}^{k-1} \sqrt{ \bigg[ \sum_{t\in\hat\Psi_{k,\ell}} \ab^\top \hat\bSigma_{k,\ell}^{-1} \ab \bigg]  \cdot \biggl [ \sum_{t\in\hat\Psi_{k,\ell}}w_t{\ab_t}^\top \hat\bSigma_{k,\ell}^{-1}  w_t{\ab_t}\bigg] }
     \cdot \norm{ \btheta_{s} - \btheta_{s+1}}_2
     \tag{Cauchy-Schwarz}
     \\ & \leq {A}\sum_{s =1}^{k-1} \sqrt{ \bigg[ \sum_{t\in\hat\Psi_{k,\ell}} \ab^\top \hat\bSigma_{k,\ell}^{-1}\ab \bigg] \cdot d }
     \cdot \norm{ \btheta_{s} - \btheta_{s+1}}_2
     \tag{$(\star)$} \\
     & \leq {A}\|\ab\|_2\sqrt{d} \sum_{s =1}^{k-1} \sqrt{ 2^{2\ell}{\sum_{t \in\hat\Psi_{k,\ell}} 1 }} \cdot \norm{ \btheta_{s} - \btheta_{s+1}}_2  \tag{$\lambda_{\max}(\hat\bSigma_{k,\ell}^{-1}) \leq \frac{1}{2^{-2\ell}}=2^{2\ell}$} \\
     & \leq {A^2}2^{\ell}\sqrt{{d w}} \sum_{s =1}^{k-1} \norm{ \btheta_{s} - \btheta_{s+1}}_2\,,\label{eq: bound sum theta 1}
\end{align}
where the inequality $(\star)$ follows from the fact that $\sum_{t\in\hat\Psi_{k,\ell}}w_t{\ab_t}^\top \hat\bSigma_{k,\ell}^{-1}  w_t{\ab_t}\leq d$ that can be proved as follows. We have $\sum_{t\in\hat\Psi_{k,\ell}}w_t{\ab_t}^\top \hat\bSigma_{k,\ell}^{-1}  w_t{\ab_t}= \sum_{t\in\hat\Psi_{k,\ell}}\text{tr} \left(w_t{\ab_t}^\top \hat\bSigma_{k,\ell}^{-1}  w_t{\ab_t}\right)=\text{tr} \left(\hat\bSigma_{k,\ell}^{-1}  \sum_{t\in\hat\Psi_{k,\ell}}w_t^2{\ab_t} {\ab_t}^\top\right)$. Given the eigenvalue decomposition $\sum_{t\in\hat\Psi_{k,\ell}}w_t^2{\ab_t} {\ab_t}^\top = \text{diag}(\lambda_1, \ldots, \lambda_d)^\top$, we have $\hat\bSigma_{k,\ell} = \text{diag}(\lambda_1 + \lambda, \ldots, \lambda_d + \lambda)^\top$, and $\text{tr} \left(\hat\bSigma_{k,\ell}^{-1}  \sum_{t\in\hat\Psi_{k,\ell}}w_t^2{\ab_t} {\ab_t}^\top\right)= \sum_{i=1}^d \frac{\lambda_j}{\lambda_j + \lambda } \leq d$.

For the second term in Eq.(\ref{l2 bound difference 12}), we can apply Theorem \ref{thm:bernstein1} for the layer $\ell$. In detail, for any $k\in[K]$, for each $t \in \hat\Psi_{k, \ell}$, we have \begin{align} 
        \|w_t \ab_t\|_{\hat\bSigma_{t, \ell}^{-1}} = 2^{-\ell}, \quad
        \EE[w_t^2 \epsilon_t^2| \cF_{t}] \le w_t^2 \EE[\epsilon_t^2| \cF_{t}] \le w_t^2 \sigma_t^2, \quad
        |w_t \epsilon_t| \le |\epsilon_t| \le R, \notag
    \end{align}
    where the last inequality holds due to the fact that
$w_t = \frac{2^{-\ell_t}}{\|\ab_t\|_{\hat\bSigma_{t, \ell_t}^{-1}}}\leq 1$. According to Theorem \ref{thm:bernstein1}, and taking a union bound, we can deduce that with probability at least $1 - \delta$, for all $\ell\in[L]$, for all round $k \in \Psi_{K + 1, \ell}, \ $\begin{align}  \|\sum_{t\in\hat\Psi_{k,\ell}}w_t^2{\ab_t\epsilon_t}\|_{\hat\bSigma_{k,\ell}^{-1}}\leq 16 \cdot 2^{-\ell} \sqrt{\sum_{t \in \hat\Psi_{k, \ell}} w_t^2 \sigma_t^2\log(\frac{4w^2 L }{\delta} )} + 6 \cdot 2^{-\ell} R \log(\frac{4w^2 L }{\delta})\,. \label{event unknown 1}
    \end{align}
    For simplicity, we denote $\cE_{\conf}$ as the event such that Eq.(\ref{event unknown 1}) holds.

For the third term in Eq.(\ref{l2 bound difference 12}), we have 
\begin{align}
    &2^{-2\ell} \|\ab\|_{\hat\bSigma_{k,\ell}^{-1}}\|\hat\bSigma_{k,\ell}^{-\frac{1}{2}}\btheta_k\|_2\leq2^{-2\ell} \|\ab\|_{\hat\bSigma_{k,\ell}^{-1}}\|\hat\bSigma_k^{-\frac{1}{2}}\|_2\|\btheta_k\|_2\notag\\
    &\leq 2^{-2\ell} \|\ab\|_{\hat\bSigma_{k,\ell}^{-1}}\frac{1}{\sqrt{\lambda_{\text{min}}(\hat\bSigma_{k,\ell})}}\|\btheta_k\|_2\leq2^{-\ell}B \|\ab\|_{\hat\bSigma_k^{-1}}\,,\label{lambda bound}
\end{align}
where we use the fact that $\lambda_{\text{min}}(\hat\bSigma_{k,\ell})\geq2^{-2\ell}$.

For simplicity, we denote $\ell^* = \lceil \frac{1}{2} \log_2 \log\left(4(w + 1)^2 L /\delta\right) \rceil + 8$. 
Then, under $\event_{\conf}$, by the definition of $\hat \beta_{k, \ell}$ in Eq.(\ref{eq:def:beta}), Lemma \ref{lemma:var} and Lemma \ref{lemma:varest}, with probability at least $1-\delta$, we have for all $\ell^* + 1 \le \ell \le L$, \begin{align}
        \hat\beta_{k, \ell} \ge 16 \cdot 2^{-\ell} \sqrt{\sum_{t \in \hat\Psi_{k, \ell}} w_t^2 \sigma_t^2\log(\frac{4w^2 L }{\delta} )} + 6 \cdot 2^{-\ell} R \log(\frac{4w^2 L }{\delta})+2^{-\ell}B.\label{beta bound}
    \end{align}

Therefore, with Eq.(\ref{l2 bound difference 12}), Eq.(\ref{eq: bound sum theta 1}), Eq.(\ref{event unknown 1}), Eq.(\ref{lambda bound}), Eq.(\ref{beta bound}), with probability at least $1-3\delta$, for all $\ell^* + 1 \le \ell \le L$ we have
\begin{equation}
    |\ab^\top(\hat\btheta_{k,\ell}-\btheta_k)|\leq {A^2}2^{\ell}\sqrt{{d w}} \sum_{s =1}^{k-1} \norm{ \btheta_{s} - \btheta_{s+1}}_2+\hat\beta_{k,\ell}\|\ab\|_{\hat\Sigma_{k,\ell}^{-1}}\label{bound estimation for algo2}\,.
\end{equation}

Then for all $k\in[K]$ such that $\ell^* + 1 \le \ell_k \le L$, with probability at least $1-3\delta$ we have
\begin{align}
   \la\ab_k^*,\btheta_k\ra&\leq \min_{\ell\in[L]}\la \ab_k^*,\hat\btheta_{k,\ell}\ra+{A^2}2^{\ell}\sqrt{{d w}} \sum_{s =1}^{k-1} \norm{ \btheta_{s} - \btheta_{s+1}}_2+\hat\beta_{k,\ell}\|\ab_k^*\|_{\hat\Sigma_{k,\ell}^{-1}}\notag\\
   &\leq {A^2}2^{L}\sqrt{{d w}} \sum_{s =1}^{k-1} \norm{ \btheta_{s} - \btheta_{s+1}}_2+\min_{\ell\in[L]}\la \ab_k^*,\hat\btheta_{k,\ell}\ra+\hat\beta_{k,\ell}\|\ab_k^*\|_{\hat\Sigma_{k,\ell}^{-1}}\notag\\
   &\leq {A^2}2^{L}\sqrt{{d w}} \sum_{s =1}^{k-1} \norm{ \btheta_{s} - \btheta_{s+1}}_2+\min_{\ell\in[L]}\la \ab_k,\hat\btheta_{k,\ell}\ra+\hat\beta_{k,\ell}\|\ab_k\|_{\hat\Sigma_{k,\ell}^{-1}}\notag\\
   &\leq {A^2}2^{L}\sqrt{{d w}} \sum_{s =1}^{k-1} \norm{ \btheta_{s} - \btheta_{s+1}}_2+\la \ab_k,\hat\btheta_{k,\ell_k-1}\ra+\hat\beta_{k,\ell_k-1}\|\ab_k\|_{\hat\Sigma_{k,\ell_k-1}^{-1}}\,,\label{regret upper bound by max min}
\end{align}
where the first inequality holds because of Eq.(\ref{bound estimation for algo2}), the third inequality holds because of the arm selection rule in Line 8 of Algo.\ref{alg:1}.

% \begin{align}
%     regret_{[s,t]} = \sum_s^t \|\theta- \theta\|_2 + \sqrt{\sum_s^t \sigma^2}
% \end{align}

We decompose the regret for $\tilde K\in[1,w]$ as follows
\begin{align}
    \text{Regret}(\tilde K)&=\sum_{k\in[\tilde K]} (\la\ab_k^*,\btheta_k\ra-\la\ab_k,\btheta_k\ra)\notag\\
    &=\sum_{\ell \in [\ell^*]}\sum_{k\in\hat\Psi_{\tilde K+1,\ell}}(\la\ab_k^*,\btheta_k\ra-\la\ab_k,\btheta_k\ra)+\sum_{\ell \in [L]\backslash[\ell^*]}\sum_{k\in\hat\Psi_{\tilde K+1,\ell}}(\la\ab_k^*,\btheta_k\ra-\la\ab_k,\btheta_k\ra)\notag\\
    &\quad+\sum_{k\in\hat\Psi_{\tilde K+1,L+1}}(\la\ab_k^*,\btheta_k\ra-\la\ab_k,\btheta_k\ra)\label{regret decompose of algo2 1}\,.
\end{align}
We will bound the three terms separately. For the first term, we have
for layer $\ell \in [\ell^*]$ and round $k\in\hat\Psi_{\tilde K+1,\ell}$, we have \begin{align} 
            \sum_{k\in\hat\Psi_{\tilde K+1,\ell}} \big(\la \ab_{k}^*, \btheta^* \ra - \la \ab_{k}, \btheta^* \ra\big) &\le 2 \left|\Psi_{K + 1, \ell}\right| \notag\\
            &=  2^{2\ell+1} \sum_{k\in\hat\Psi_{\tilde K+1,\ell}} \|w_k\ab_k\|_{\hat \bSigma_{k, \ell}^{-1}}^2 \notag\\
            &\le 2\cdot128^2\log(\frac{4(w + 1)^2 L }{\delta}) \sum_{k\in\hat\Psi_{\tilde K+1,\ell}} \|w_k\ab_k\|_{\hat \bSigma_{k, \ell}^{-1}}^2\notag\\
            &\leq2\cdot128^2\log(\frac{4(w + 1)^2 L }{\delta})\cdot2d\log(1+\frac{2^{2\ell}wA^2}{d})\notag\\
            &=\Tilde{O}(d)\,,\label{bound small l11}
            \end{align}
       where the first inequality holds because the reward is in $[-1,1]$, the equation follows from the fact that $\|w_k\ab_k\|_{\hat \bSigma_{k, \ell}^{-1}}=2^{-\ell}$ holds for all $k \in \Psi_{K + 1, \ell}$,
        the second inequality holds due to the fact that $2^{\ell^*} \le 128\sqrt{\log(4(w + 1)^2 L /\delta)}$, and the last inequality holds due to Lemma \ref{Lemma:abba}.
        
% We then bound $\sum_{k \in \Psi_{K + 1, \ell}} \|w_k\ab_k\|_{\hat \bSigma_{k, \ell}^{-1}}^2$. We denote $\Psi_{[a,b],\ell}:=\{k:k\in[a,b], \ell_k=\ell\}$. Then we have
% \begin{align}
%     \sum_{k \in \Psi_{K + 1, \ell}} \|w_k\ab_k\|_{\hat \bSigma_{k, \ell}^{-1}}^2&=\sum_{i=0}^{\lceil\frac{K}{w}\rceil-1}\sum_{k\in\Psi_{[i\cdot w, (i+1)w],\ell}} \|w_k\ab_k\|_{\hat \bSigma_{k, \ell}^{-1}}^2\notag\\
%     % &\leq \sum_{i=0}^{\lceil\frac{K}{w}\rceil-1}\sum_{k\in\Psi_{[i\cdot w, (i+1)w],\ell}} \|w_k\ab_k\|_{\Tilde \bSigma_{k, \ell}}^2\notag\\
%     &\leq \frac{K}{w}\cdot 2d\log(1+\frac{2^{2\ell}wA^2}{d})\,,\label{bound sum of elliptical algo2 1}
% \end{align}
% where the last inequality holds due to Lemma \ref{Lemma:abba}.

Therefore
\begin{equation}
    \sum_{\ell \in [\ell^*]}\sum_{k\in\hat\Psi_{\tilde K+1,\ell}}(\la\ab_k^*,\btheta_k\ra-\la\ab_k,\btheta_k\ra)=\tilde{O}(d)\label{bound small l 2}\,.
\end{equation}
For the second part in Eq.(\ref{regret decompose of algo2 1}), we have
\begin{align}
    &\sum_{\ell \in [L]\backslash[\ell^*]}\sum_{k\in\hat\Psi_{\tilde K+1,\ell}}(\la\ab_k^*,\btheta_k\ra-\la\ab_k,\btheta_k\ra)\notag \\
    &\quad \leq  \sum_{\ell \in [L]\backslash[\ell^*]}\sum_{k\in\hat\Psi_{\tilde K+1,\ell}}\bigg(\la\ab_k,\hat\btheta_{k,\ell-1}\ra+\hat\beta_{k,\ell-1}\|\ab_k\|_{\hat\Sigma_{k,\ell-1}^{-1}}\notag\\
    &\quad+{A^2}2^{L}\sqrt{{d w}} \sum_{k\in\hat\Psi_{\tilde K+1,\ell}} \norm{ \btheta_{s} - \btheta_{s+1}}_2-\la\ab_k,\btheta_k\ra\bigg)\notag\\
       &\leq 2\sum_{\ell \in [L]\backslash[\ell^*]}\sum_{k\in\hat\Psi_{\tilde K+1,\ell}} \hat\beta_{k,\ell-1}\|\ab_k\|_{\hat\Sigma_{k,\ell-1}^{-1}}+{A^2}\sqrt{{d w}}\sum_{\ell \in [L]\backslash[\ell^*]}\sum_{k\in\hat\Psi_{\tilde K+1,\ell}}2^{L} \sum_{s=1}^{k-1} \norm{ \btheta_{s} - \btheta_{s+1}}_2\,,\label{bound of large L decom}
\end{align}
where the inequality holds due to Eq.(\ref{regret upper bound by max min}), the second inequality holds due to Eq.(\ref{bound estimation for algo2}). We then try to bound the two terms.

For the first term in Eq.(\ref{bound of large L decom}), we have
\begin{align}
    \sum_{\ell \in [L]\backslash[\ell^*]}\sum_{k\in\hat\Psi_{\tilde K+1,\ell}} \hat\beta_{k,\ell-1}\|\ab_k\|_{\hat\Sigma_{k,\ell-1}^{-1}}
    &\leq \sum_{\ell \in [L]\backslash[\ell^*]}\sum_{k\in\hat\Psi_{\tilde K+1,\ell}} \hat\beta_{k,\ell-1}\cdot2^{-\ell}\notag\\
    &\leq \sum_{\ell \in [L]\backslash[\ell^*]}\hat\beta_{\tilde K,\ell-1}\cdot2^{-\ell}\left|\hat\Psi_{\tilde K+1,\ell}\right|\notag\\
    &=\sum_{\ell \in [L]\backslash[\ell^*]}\hat\beta_{\tilde K,\ell-1}\cdot2^{\ell}\sum_{k\in\hat\Psi_{\tilde K+1,\ell}}\|w_k\ab_k\|_{\Sigma_{k,\ell}^{-1}}^2\notag\\
    &\leq \sum_{\ell \in [L]\backslash[\ell^*]}\hat\beta_{\tilde K,\ell-1}\cdot2^{\ell}\cdot 2d\log(1+\frac{2^{2\ell}\tilde KA^2}{d})\notag\\
    &=\Tilde{O}(d\cdot2^{\ell}\cdot\hat\beta_{\tilde K,\ell-1})\notag\\
    &=\Tilde{O}\bigg(d\big(\sqrt{\sum_{k=1}^{\tilde K}\sigma_k^2}+R+1\big)\bigg)\,,\label{regret algo2 final part2}
    % &\leq\sum_{\ell \in [L]\backslash[\ell^*]}2^{-\ell}\sqrt{\sum_{k\in\Psi_{K+1,\ell}} 1 \cdot\sum_{k\in\Psi_{K+1,\ell}} \hat\beta_{k,\ell-1}^2}\notag\\
    % &=\Tilde{O}\bigg(\sum_{\ell \in [L]\backslash[\ell^*]}2^{-2\ell} \sqrt{\sum_{k\in\Psi_{K+1,\ell}} 1\cdot \sum_{k\in\Psi_{K+1,\ell}} (\sum_{t\in\hat\Psi_{k,\ell}} \sigma_t^2+C_1)}\bigg)\notag\\
    % &\leq \Tilde{O}\bigg(\sum_{\ell \in [L]\backslash[\ell^*]}2^{-2\ell} \sqrt{2^{2\ell}\sum_{k\in\Psi_{K+1,\ell}} \|w_k\ab_k\|_{\hat\Sigma_{k,\ell}^{-1}}^2\cdot \sum_{k\in\Psi_{K+1,\ell}} (\sum_{t\in\hat\Psi_{k,\ell}} \sigma_t^2+C_1)}\bigg)\notag
\end{align}
where the first inequality holds because by the algorithm design, we have for all $k\in\hat\Psi_{\tilde K+1,\ell}$: $\|\ab_k\|_{\hat\Sigma_{k,\ell-1}^{-1}}\leq 2^{-\ell}$; the second inequality holds because for all $k\in\hat\Psi_{\tilde K+1,\ell}$, $\hat\beta_{k,\ell-1}\leq\hat\beta_{\Tilde K,\ell-1}$; the first equality holds because for all $k\in\hat\Psi_{\tilde K+1,\ell}$, $\|w_k\ab_k\|_{\Sigma_{k,\ell}^{-1}}^2=2^{-2\ell}$; the third inequality holds by Lemma \ref{Lemma:abba}; the last two equalities hold because by Lemma \ref{lemma:var} and Lemma \ref{lemma:varest}, we have $\hat\beta_{\Tilde K,\ell-1}=\Tilde{O}\bigg(2^{-\ell}(\sqrt{\sum_{k=1}^{\tilde K}\sigma_k^2}+R+1)\bigg)$.

For the second term in Eq.(\ref{bound of large L decom}), we have
\begin{align}
&{A^2}\sqrt{{d w}}\sum_{\ell \in [L]\backslash[\ell^*]}\sum_{k\in\hat\Psi_{\tilde K+1,\ell}}2^{L} \sum_{s=1}^{k-1} \norm{ \btheta_{s} - \btheta_{s+1}}_2\notag\\&\leq A^22^{L}\sqrt{dw}\sum_{k\in[\tilde K-1]}\sum_{s =1}^{k-1} \norm{ \btheta_{s} - \btheta_{s+1}}_2\notag\\
    &\leq \frac{A^2\sqrt{d}w^{\frac{3}{2}}}{\alpha}\sum_{k=1}^{\Tilde{K}-1}\norm{ \btheta_{k} - \btheta_{k+1}}_2
\end{align}
Therefore, with this, Eq.(\ref{bound of large L decom}), and Eq.(\ref{regret algo2 final part2}), we have
\begin{align}
   & \sum_{\ell \in [L]\backslash[\ell^*]}\sum_{k\in\hat\Psi_{\tilde K+1,\ell}}(\la\ab_k^*,\btheta_k\ra-\la\ab_k,\btheta_k\ra)\notag\\
   &\leq \frac{A^2\sqrt{d}w^{\frac{3}{2}}}{\alpha}\sum_{k=1}^{\Tilde{K}-1}\norm{ \btheta_{k} - \btheta_{k+1}}_2+\Tilde{O}\bigg(d\big(\sqrt{\sum_{k=1}^{\tilde K}\sigma_k^2}+R+1\big)\bigg)\,.\label{regret algo2 result part 2 1 final}
\end{align}
Finally, for the last term in Eq.(\ref{regret decompose of algo2 1}), we have
\begin{align}
    &\sum_{k\in\hat\Psi_{\tilde K+1,L+1}}(\la\ab_k^*,\btheta_k\ra-\la\ab_k,\btheta_k\ra)\notag\\
    &\leq \sum_{k\in\hat\Psi_{\tilde K+1,L+1}}\bigg(\la\ab_k,\hat\btheta_{k,L}\ra+\hat\beta_{k,L}\|\ab_k\|_{\hat\Sigma_{k,L}^{-1}}+A^2 2^L\sqrt{dw}\sum_{s=1}^{k-1} \norm{ \btheta_{s} - \btheta_{s+1}}_2-\la\ab_k,\btheta_k\ra\bigg)\notag\\
    &\leq \sum_{k\in\hat\Psi_{\tilde K+1,L+1}} \bigg(2\hat\beta_{k,L}\|\ab_k\|_{\hat\Sigma_{k,L}^{-1}}+A^2 2^{L+1}\sqrt{dw}\sum_{s=1}^{k-1} \norm{ \btheta_{s} - \btheta_{s+1}}_2\bigg)\notag\\
    &\leq \sum_{k\in\hat\Psi_{\tilde K+1,L+1}} \bigg(2^{-L+1}\hat\beta_{k,L}+A^2 2^{L+1}\sqrt{dw}\sum_{s=1}^{k-1} \norm{ \btheta_{s} - \btheta_{s+1}}_2\bigg)\notag\\
    &\leq \frac{2A^2\sqrt{d}w^{\frac{3}{2}}}{\alpha}\sum_{k=1}^{\tilde K-1} \norm{ \btheta_{k} - \btheta_{k+1}}_2+\sum_{k\in\hat\Psi_{\tilde K+1,L+1}} 2^{-L+1}\hat\beta_{\tilde K,L}\notag\\
    &\leq\frac{2A^2\sqrt{d}w^{\frac{3}{2}}}{\alpha}\sum_{k=1}^{\tilde K-1} \norm{ \btheta_{k} - \btheta_{k+1}}_2+w\cdot2\alpha\cdot\hat\beta_{\tilde K,L}\notag\\
    &=\frac{2A^2\sqrt{d}w^{\frac{3}{2}}}{\alpha}\sum_{k=1}^{\tilde K-1} \norm{ \btheta_{k} - \btheta_{k+1}}_2+\Tilde{O}\bigg(w\alpha^2\cdot\big(\sqrt{\sum_{k=1}^{\tilde K}\sigma_k^2}+R+1\big)\bigg)
    \,,\label{regret part 3 algo2}
\end{align}
where the first inequality holds due to Eq.(\ref{regret upper bound by max min}), the second inequality holds due to Eq.(\ref{bound estimation for algo2}), the third inequality holds because by the algorithm design, we have for all $k\in\hat\Psi_{\tilde K+1,L+1}$: $\|\ab_k\|_{\hat\Sigma_{k,L}^{-1}}\leq 2^{-L}$, the fourth inequality holds due to the same reasons as before, and the fact that $\hat\beta_{\tilde K,L}\geq\hat\beta_{k,L}$ for all $k\in\hat\beta_{\tilde K,L}$; the last inequality holds due to $\hat\beta_{\Tilde K,\ell-1}=\Tilde{O}\bigg(\alpha(\sqrt{\sum_{k=1}^{\tilde K}\sigma_k^2}+R+1)\bigg)$.

Plugging Eq.(\ref{regret algo2 result part 2 1 final}), Eq.(\ref{regret part 3 algo2}), and Eq.(\ref{bound small l 2}) into Eq.(\ref{regret decompose of algo2 1}), we can get that for $\tilde K\in[1,w]$

          \begin{align} 
            \text{Regret}(\tilde K)&= \tilde{O}\bigg(\frac{A^2\sqrt{d}w^{\frac{3}{2}}}{\alpha}\sum_{k=1}^{\Tilde{K}-1}\norm{ \btheta_{k} - \btheta_{k+1}}_2+\big(w\alpha^2+d\big)\cdot\big(\sqrt{\sum_{k=1}^{\tilde K}\sigma_k^2}+R+1\big)\bigg)\,.
        \end{align}
By the same deduction we can get 

          \begin{align} 
            &\text{Regret}([g_i,g_{i+1}])\notag\\&= \tilde{O}\bigg(\frac{A^2\sqrt{d}w^{\frac{3}{2}}}{\alpha}\sum_{k=g_i}^{g_{i+1}}\norm{ \btheta_{k} - \btheta_{k+1}}_2+\big(w\alpha^2+d\big)\cdot\big(\sqrt{\sum_{k=g_i}^{g_{i+1}}\sigma_k^2}+R+1\big)\bigg)\,.
        \end{align}

Finally, without loss of generality, we assume $K\%w=0$. Then we have
\begin{align}
        &\text{Regret}(K)\notag\\
        &=\sum_{i=0}^{\frac{K}{w}-1}\text{Regret}([g_i,g_{i+1}])\notag\\
        &=\tilde{O}\bigg(\frac{A^2\sqrt{d}w^{\frac{3}{2}}}{\alpha}\sum_{i=0}^{\frac{K}{w}-1}\sum_{k=g_i}^{g_{i+1}}\norm{ \btheta_{k} - \btheta_{k+1}}_2+\big(w\alpha^2+d\big)\cdot\sum_{i=0}^{\frac{K}{w}-1}\big(\sqrt{\sum_{k=g_i}^{g_{i+1}}\sigma_k^2}+R+1\big)\bigg)\notag\\
        &\leq\tilde{O}\bigg(\frac{A^2\sqrt{d}w^{\frac{3}{2}}}{\alpha}\sum_{k=1}^{K-1}\norm{ \btheta_{k} - \btheta_{k+1}}_2+\big(w\alpha^2+d\big)\cdot\big(\sqrt{\frac{K}{w}\sum_{i=0}^{\frac{K}{w}-1}\sum_{k=g_i}^{g_{i+1}}\sigma_k^2}+\frac{KR}{w}+\frac{K}{w}\big)\bigg)\notag\\
        &\leq\tilde{O}\bigg(\frac{A^2\sqrt{d}w^{\frac{3}{2}}B_K}{\alpha}+\big(w\alpha^2+d\big)\cdot\sqrt{\frac{K}{w}\sum_{k=1}^{K}\sigma_k^2}+\big(1+R\big)\cdot\big(K\alpha^2+\frac{Kd}{w}\big)\bigg)\notag\,,
        \end{align}
        where the first inequality holds due to the Cauchy-Schwarz inequality, the last inequality holds because $\sum_{k=1}^{K-1}\norm{ \btheta_{k} - \btheta_{k+1}}_2\leq B_K$.
\subsection{Proof of Theorem \ref{theorem bob}}\label{app:thirdalg}
With the candidate pool set $\mathcal{P}$ designed as in Eq.(\ref{bob pool}), Eq.(\ref{w set}), Eq.(\ref{alpha set}), and $H=\lceil d^{\frac{2}{5}}K^{\frac{2}{5}}\rceil$, we have $\left|\mathcal{P}\right|=O(\log K)$, and for any $w\in \mathcal{W}$, $w\leq H$.

We denote the optimal $(w,\alpha)$ with the knowledge of $V_K$ and $B_K$ in Corollary \ref{corollary w alpha opt} as $(w^*,\alpha^*)$. We denote the best approximation of $(w^*,\alpha^*)$ in the candidate set $\mathcal{P}$ as $(w^+,\alpha^+)$. Then we can decompose the regret as follows
\begin{align}
    \text{Regret}(K)&=\sum_{k=1}^{K}\la\ab_t^*,\btheta_k\ra-\la\ab_t,\btheta_k\ra\notag\\&=\underbrace{\sum_{k=1}^K\la\ab_t^*,\btheta_k\ra-\sum_{i=1}^{\lceil
    \frac{K}{H}\rceil}\sum_{k=(i-1)H+1}^{iH}\la\ab_t(w^+,\alpha^+),\btheta_k\ra}_{(1)}\notag\\
    &+\underbrace{\sum_{i=1}^{\lceil
    \frac{K}{H}\rceil}\sum_{k=(i-1)H+1}^{iH}\la\ab_t(w^+,\alpha^+),\btheta_k\ra-\la\ab_t(w_i,\alpha_i),\btheta_k\ra}_{(2)}\,.
\end{align}
The first term (1) is the dynamic regret of $\text{Restarted SAVE}^+$ with the best parameters in the candidate pool $\mathcal{P}$. The second term (2) is the regret overhead of meta-algorithm due to adaptive exploration of unknown optimal parameters.

By the design of the candidate pool set $\mathcal{P}$ in Eq.(\ref{bob pool}), Eq.(\ref{w set}), Eq.(\ref{alpha set}), we have that there exists a pair $(w^+,\alpha^+)\in\mathcal{P}$ such that $w^+<w^*<2w^+$, and $\alpha^+<\alpha^*<2\alpha^+$. Therefore, employing the regret bound in Theorem \ref{thm:regret1}, we can get
\begin{align}
    (1)&\leq \sum_{i=1}^{\lceil\frac{K}{H}\rceil}\Tilde O(\sqrt{d}w^{+1.5}B_i/\alpha^+ + \alpha^{+2}(H+\sqrt{w^+HV_i})
         + d\sqrt{HV_i/w^+} + dH/w^+)\notag\\
     &\leq \Tilde O(\sqrt{d}w^{+1.5}B_K/\alpha^+ + \alpha^{+2}(K+\sqrt{w^+H\frac{K}{H}\sum_{i=1}^{\lceil\frac{K}{H}\rceil}V_i})
         + d\sqrt{H\frac{K}{H}\sum_{i=1}^{\lceil\frac{K}{H}\rceil}V_i/w^+}\notag\\
         &\quad+ dK/w^+)\notag\\
         &=\Tilde O(\sqrt{d}w^{+1.5}B_K/\alpha^+ + \alpha^{+2}(K+\sqrt{w^+KV_K})
         + d\sqrt{KV_K/w^+} + dK/w^+)\notag\\
         &=\Tilde O(\sqrt{d}w^{*1.5}B_K/\alpha^* + \alpha^{*2}(K+\sqrt{w^*KV_K})
         + d\sqrt{KV_K/w^*} + dK/w^*)\notag\\
         &=\tilde O(d^{4/5}V_K^{2/5}B_K^{1/5}K^{2/5} + d^{2/3}B_K^{1/3}K^{2/3})\,,
\end{align}
where we denote $B_i$ as the total variation budget in block $i$, $V_i$ is the total variance in block $i$, the second inequality is by Cauchy–Schwarz inequality, the first equality holds due to $\sum_{i=1}^{\lceil\frac{K}{H}\rceil}B_i=B_K$, $\sum_{i=1}^{\lceil\frac{K}{H}\rceil} V_i=V_K$, the second equality holds due to $w^+<w^*<2w^+$ and $\alpha^+<\alpha^*<2\alpha^+$, the last equality holds by Corollary \ref{corollary w alpha opt}.

We then try to bound the second term (2). We denote by $\mathcal{E}$ the event such that Lemma \ref{lemma:bob} holds, and denote by $R_i:=\sum_{k=(i-1)H+1}^{iH}\la\ab_t(w^+,\alpha^+),\btheta_k\ra-\la\ab_t(w_i,\alpha_i),\btheta_k\ra$ the instantaneous regret of the meta learner in the block $i$. Then we have
\begin{align}
    (2)&=\EE\bigg[\sum_{i=1}^{\lceil\frac{K}{H}\rceil}R_i\bigg]\notag\\
    &=\EE\bigg[\sum_{i=1}^{\lceil\frac{K}{H}\rceil}R_i|\mathcal{E}\bigg]P(\mathcal{E})+\EE\bigg[\sum_{i=1}^{\lceil\frac{K}{H}\rceil}R_i|\overline{\mathcal{E}}\bigg]P(\mathcal{\overline {\mathcal{E}}})\notag\\
    &\leq\Tilde{O}\bigg(L_{\text{max}}\sqrt{\frac{K}{H}\left|\mathcal{P}\right|}\bigg)\cdot(1-\frac{2}{K})+\Tilde{O}(K)\cdot\frac{2}{K}\notag\\
    &=\Tilde{O}(\sqrt{H\left|\mathcal{P}\right|K})\notag\\
    &=\Tilde{O}(d^{\frac{1}{5}}K^{\frac{7}{10}})\,,
\end{align}
where $L_{\text{max}}:=\max_{i\in[\lceil\frac{K}{H}\rceil]} L_i$, the first inequality holds due to the standard regret upper bound result for Exp3 \cite{auer2002nonstochastic}, the third equality holds due to Lemma \ref{lemma:bob}, the last equality holds since $H=\lceil d^{\frac{2}{5}}K^{\frac{2}{5}}\rceil$, and $\left|\mathcal{P}\right|=O(\log K)$.

Finally, combining the above results for term (1) and term (2), we have
    \begin{align}
        \text{Regret}(K) = \tilde O(d^{4/5}V_K^{2/5}B_K^{1/5}K^{2/5} + d^{2/3}B_K^{1/3}K^{2/3}+d^{\frac{1}{5}}K^{\frac{7}{10}}).
    \end{align}

\subsection{Technical Lemmas}
\begin{theorem}[Theorem 4.3, \cite{zhou2022computationally}]\label{lemma:concentration_variance} 
Let $\{\cG_k\}_{k=1}^\infty$ be a filtration, and $\{\xb_k,\eta_k\}_{k\ge 1}$ be a stochastic process such that
$\xb_k \in \RR^d$ is $\cG_k$-measurable and $\eta_k \in \RR$ is $\cG_{k+1}$-measurable.
Let $L,\sigma,\lambda, \epsilon>0$, $\bmu^*\in \RR^d$. 
For $k\ge 1$, 
let $y_k = \la \bmu^*, \xb_k\ra + \eta_k$ and
suppose that $\eta_k, \xb_k$ also satisfy 
\begin{align}
    \EE[\eta_k|\cG_k] = 0,\ \EE [\eta_k^2|\cG_k] \leq \sigma^2,\  |\eta_k| \leq R,\,\|\xb_k\|_2 \leq L.
\end{align}
For $k\ge 1$, let $\Zb_k = \lambda\Ib + \sum_{i=1}^{k} \xb_i\xb_i^\top$, $\bbb_k = \sum_{i=1}^{k}y_i\xb_i$, $\bmu_k = \Zb_k^{-1}\bbb_k$, and
\begin{align}
    \beta_k &= 12\sqrt{\sigma^2d\log(1+kL^2/(d\lambda))\log(32(\log(R/\epsilon)+1)k^2/\delta)} \notag \\
    &\quad + 24\log(32(\log(R/\epsilon)+1)k^2/\delta)\max_{1 \leq i \leq k} \{|\eta_i|\min\{1, \|\xb_i\|_{\Zb_{i-1}^{-1}}\}\} \notag\\
    &\quad+ 6\log(32(\log(R/\epsilon)+1)k^2/\delta)\epsilon.\notag
\end{align}
Then, for any $0 <\delta<1$, we have with probability at least $1-\delta$ that, 
\begin{align}
    \forall k\geq 1,\ \big\|\textstyle{\sum}_{i=1}^{k} \xb_i \eta_i\big\|_{\Zb_k^{-1}} \leq \beta_k,\ \|\bmu_k - \bmu^*\|_{\Zb_k} \leq \beta_k + \sqrt{\lambda}\|\bmu^*\|_2.\notag
\end{align}
\end{theorem}

\begin{lemma}[Lemma 11,  \cite{abbasi2011improved}]\label{Lemma:abba}
    For any $\lambda>0$ and sequence $\{\xb_k\}_{k=1}^K \subset \RR^d$
for $k\in [K]$, define $\Zb_k = \lambda \Ib+ \sum_{i=1}^{k-1}\xb_i\xb_i^\top$.
Then, provided that $\|\xb_k\|_2 \leq L$ holds for all $k\in [K]$,
we have
\begin{align}
    \sum_{k=1}^K \min\big\{1, \|\xb_k\|_{\Zb_{k}^{-1}}^2\big\} \leq 2d\log\big(1+KL^2/(d\lambda)\big).\notag
\end{align}
\end{lemma}

% \begin{lemma}[Lemma 5, \cite{cheung2019learning}]
%     \label{lemma:sw2}
%     For any $i\leq \lceil T/w\rceil-1,$ 
%     \begin{align*}
%         \sum_{t=i\cdot w+1}^{(i+1) w}1\wedge\left\|X_t\right\|^2_{V^{-1}_{t-1}}\leq\sum_{t=i\cdot w+1}^{(i+1) w}1\wedge\left\|X_t\right\|^2_{\overline{V}_{t-1}^{-1}},
%     \end{align*}
%     where 
%     \begin{align}
%         \overline{V}_{t-1}=\sum_{s=i\cdot w+1}^{t-1}X_sX_s^{\top}+\lambda I, V_{t-1}=\lambda I+\sum_{\min\{1,t-w+1\}}^{t-1}X_sX_s^{\top}.
%     \end{align}
% \end{lemma}
\begin{theorem}[Theorem 2.1, \cite{zhao2023variance}]\label{thm:bernstein1}
    Let $\{\cG_k\}_{k=1}^\infty$ be a filtration, and $\{\xb_k,\eta_k\}_{k\ge 1}$ be a stochastic process such that
    $\xb_k \in \RR^d$ is $\cG_k$-measurable and $\eta_k \in \RR$ is $\cG_{k+1}$-measurable.
    Let $L,\sigma,\lambda, \epsilon>0$, $\bmu^*\in \RR^d$. 
    For $k\ge 1$, 
    let $y_k = \la \bmu^*, \xb_k\ra + \eta_k$, where $\eta_k, \xb_k$ satisfy 
    \begin{align}
        \EE[\eta_k|\cG_k] = 0, \ |\eta_k| \leq R, \ \sum_{i = 1}^k \EE[\eta_i^2|\cG_i] \le v_k,  \ \ \text{for}\ \forall \ k\geq 1 \notag
    \end{align}
    For $k\ge 1$, let $\Zb_k = \lambda\Ib + \sum_{i=1}^{k} \xb_i\xb_i^\top$, $\bbb_k = \sum_{i=1}^{k}y_i\xb_i$, $\bmu_k = \Zb_k^{-1}\bbb_k$, and
    %\begin{small}
    \begin{align}
        \beta_k &= 16\rho \sqrt{v_k \log(4w^2 / \delta)} + 6 \rho R  \log(4w^2 / \delta), \notag
    \end{align}
    %\end{small} 
    where $\rho \ge \sup_{k \ge 1}\|\xb_k\|_{\Zb_{k - 1}^{-1}}$. 
Then, for any $0 <\delta<1$, we have with probability at least $1-\delta$ that, 
    \begin{align}
        \forall k\geq 1,\ \big\|\textstyle{\sum}_{i=1}^{k} \xb_i \eta_i\big\|_{\Zb_k^{-1}} \leq \beta_k,  \|\bmu_k - \bmu^*\|_{\Zb_k} \leq \beta_k + \sqrt{\lambda}\|\bmu^*\|_2. \notag
    \end{align}
\end{theorem}

   \begin{lemma} [Adopted from Lemma B.4, \cite{zhao2023variance}]\label{lemma:var}
        Let weight $w_i$ be defined in Algorithm \ref{alg:1}. 
        With probability at least $1 - 2\delta$, for all $k \ge 1$, $\ell \in [L]$, the following two inequalities hold simultaneously: \begin{align*}
            \sum_{i \in \hat\Psi_{k + 1, \ell}} w_i^2 \sigma_i^2 \le 2 \sum_{i \in \hat\Psi_{k + 1, \ell}} w_i^2 \epsilon_i^2 + \frac{14}{3} R^2 \log(4w^2 L / \delta), 
            \\
            \sum_{i \in \hat\Psi_{k + 1, \ell}} w_i^2 \epsilon_i^2 \le \frac{3}{2} \sum_{i \in \hat\Psi_{k + 1, \ell}} w_i^2 \sigma_i^2 + \frac{7}{3} R^2 \log(4w^2 L / \delta). 
        \end{align*}
        %We denote the corresponding event by $\cE_{\var}$. 
    \end{lemma}
        For simplicity, we denote $\cE_{\var}$ as the event such that the two inequalities in Lemma \ref{lemma:var} holds. 
    \begin{lemma} [Adopted from Lemma B.5, \cite{zhao2023variance}]\label{lemma:varest}
Suppose that $\|\btheta^*\|_2 \le B$. Let weight $w_i$ be defined in Algorithm \ref{alg:1}. 
        On the event $\cE_{\conf}$ and $\cE_{\var}$ (defined in Eq.(\ref{event unknown 1}), Lemma \ref{lemma:var}), for all $k \ge 1$, $\ell \in [L]$ such that $2^\ell \ge 64 \sqrt{\log\left(4(w + 1)^2 L /\delta\right)}$, we have the following inequalities: 
        \begin{align*} 
            \sum_{i \in \Psi_{k + 1, \ell}} w_i^2 \sigma_i^2 \le 8 \sum_{i \in \Psi_{k + 1, \ell}} w_i^2 \left(r_i - \la \hat\btheta_{k + 1, \ell}, \ab_i \ra \right)^2 + 6R^2 \log(4(w + 1)^2 L / \delta) + 2^{-2\ell + 2}B^2, \\
            \sum_{i \in \Psi_{k + 1, \ell}} w_i^2 \left(r_i - \la \hat\btheta_{k + 1, \ell}, \ab_i \ra \right)^2 \le \frac{3}{2} \sum_{i \in \Psi_{k + 1, \ell}} w_i^2 \sigma_i^2 + \frac{7}{3} R^2 \log(4w^2 L / \delta) + 2^{-2\ell}B^2. 
        \end{align*}
    \end{lemma}
%     \begin{lemma}[Adopted from Lemma 7 in \cite{zhao2020simple}]
%   \label{lemma:bob}
%   Let $N= \lceil \frac{K}{H}\rceil$. Denote by $L_i$ the absolute value of cumulative rewards for episode $i$, i.e., $L_i =\sum_{k = (i-1)H + 1}^{iH} r_k$, then 
%   \begin{equation}
%   \label{eq:concentration}
%     \mathbb{P}\left[\forall i\in [N], L_i\leq H+2R\sqrt{H\ln\frac{K}{\sqrt{H}}}\right] \geq 1-\frac{2}{K}.
%   \end{equation}
% \end{lemma}
    \begin{lemma}[\cite{freedman1975tail}]  \label{lemma:freedman}
        Let $M, v > 0$ be fixed constants. Let $\{x_i\}_{i = 1}^n$ be a stochastic process, $\{\cG_i\}_i$ be a filtration so that for all $i \in [n]$, $x_i$ is $\cG_{i}$-measurable, while almost surely \begin{align*}
            \EE\left[x_i | \cG_{i - 1}\right] = 0, \quad |x_i| \le M, \quad \sum_{i = 1}^n \EE[x_i^2|\cG_{i-1}] \le v. 
        \end{align*}
        Then for any $\delta > 0$, with probability at least $1 - \delta$, we have \begin{align*} 
            \sum_{i = 1}^n x_i \le \sqrt{2v \log(1 / \delta)} + 2 / 3 \cdot M \log(1 / \delta). 
        \end{align*}
    \end{lemma}
    \begin{lemma}
  \label{lemma:bob}
  Let $N= \lceil \frac{K}{H}\rceil$. Denote by $L_i$ the absolute value of cumulative rewards for episode $i$, i.e., $L_i =\sum_{k = (i-1)H + 1}^{iH} r_k$, then 
  \begin{equation}
  \label{eq:concentration}
    \mathbb{P}\left[\forall i\in [N], L_i\leq H+R\sqrt{\frac{H}{2} \log\big(K(\frac{K}{H}+1)\big)} + \frac{2}{3} \cdot R \log\big(K(\frac{K}{H}+1)\big)\right] \geq 1-\frac{1}{K}.
  \end{equation}
\end{lemma}
\begin{proof}
    By Lemma \ref{lemma:freedman}, we have that with probability at least $1-1/K$
    \begin{align}
        \sum_{k = (i-1)\cdot H+1}^{i\cdot H} \epsilon_i &\le \sqrt{2\sum_{k = (i-1)\cdot H+1}^{i\cdot H}\sigma_k^2 \log(NK)} + 2 / 3 \cdot R \log(NK)\notag\\
        &\leq \sqrt{2H\frac{R^2}{4} \log(NK)} + 2 / 3 \cdot R \log(NK)\notag\\
        &\leq R\sqrt{\frac{H}{2} \log\big(K\cdot(\frac{K}{H}+1)\big)} + \frac{2}{3} \cdot R \log\big(K\cdot(\frac{K}{H}+1)\big)\,,
    \end{align}
    where we use union bound, and in the second inequality we use the fact that since $\left|\epsilon_k\right|\leq R$, we have $\sigma_k^2\leq \frac{R^2}{4}$. Finally, together with the assumption that $r_k\leq 1$ for all $k\in[K]$, we complete the proof. 
\end{proof}
\onecolumn
\allowdisplaybreaks

\section{Clustering Of Neural Dueling Bandits (CONDB) Algorithm}
\label{app:sec:condb:algo}
Here we provide the complete statement of our CONDB algorithm.

\begin{algorithm*}[h] 
\caption{Clustering Of Neural Dueling Bandits (CONDB)}
\label{algo:neural:dueling:bandits}
	\begin{algorithmic}[1]
    \STATE {\bf Input:} $f(T_{i,t}) \triangleq \frac{\beta_T + B \sqrt{\frac{\lambda}{\kappa_\mu}} + 1}{\sqrt{2\tilde{\lambda}_x T_{i,t}}}$, regularization parameter $\lambda>0$, confidence parameter $\beta_T \triangleq \frac{1}{\kappa_\mu} \sqrt{ \widetilde{d} + 2\log(u/\delta)}$.
    $\phi(\bx) = \frac{1}{\sqrt{m_{\text{NN}}}}g(\bx;\btheta_0)$ where $\btheta_0$ denotes the NN parameters at initialization.
    \STATE {\bf Initialization:} 
$\bV_0=\bV_{i,0} = \frac{\lambda}{\kappa_\mu} \mathbf{I}$ , $\hat\btheta_{i,0}=\bzero$, $\forall{i \in \mathcal{U}}$, a complete Graph $G_0 = (\mathcal{U},E_0)$ over $\mathcal{U}$.\alglinelabel{algo line: init}
		\FOR{$t= 1, \ldots, T$}
            \STATE Receive user $i_t\in\mathcal{U}$, and feasible arm set $\cX_t$;\label{user serve neural}
            \STATE Find the connected component $\overline C_t$ for user $i_t$ in the current graph $G_{t-1}$ as the current cluster; \alglinelabel{detection:neural}

            \STATE Train the neural network using $\{(\bx_{s,1}, \bx_{s,2}, y_s)\}_{s\in[t-1], i_s\in \overline C_t}$ by minimizing the following loss function:\begin{small}
                            \begin{align}
                \overline{\btheta}_t&=\arg\min_{\btheta} 
                % \mathcal{L}_t(\btheta)= 
                - \frac{1}{m} \sum_{s\in[t-1]\atop i_s\in \overline C_t}\left( y_s\log\mu\left( h(\bx_{s,1};\btheta) - h(\bx_{s,2};\btheta) \right) + (1-y_s)\log\mu\left(h(\bx_{s,2};\btheta) - h(\bx_{s,1};\btheta)\right) \right) \notag\\
                &+ \frac{\lambda}{2}\norm{\btheta - \btheta_0}_2^2;
            \end{align}\alglinelabel{algo line: common theta:neural}
            \end{small}

            % \STATE Estimate the common preference vector $\overline{\btheta}_t$ for the current cluster $\overline C_t$:
            % \begin{equation}
            %     \overline{\btheta}_t=\arg\min_{\btheta} - \sum_{s\in[t-1]\atop i_s\in \overline C_t}\left( y_s\log\mu\left({\btheta}^{\top}\left[\phi(\bx_{s,1}) - \phi(\bx_{s,2})\right]\right) + (1-y_s)\log\mu\left({\btheta}^{\top}\left[\phi(\bx_{s,2}) - \phi(\bx_{s,1})\right]\right) \right) + \frac{1}{2}\lambda\norm{\btheta}_2^2;
            % \end{equation}\alglinelabel{algo line: common theta}
            
            \STATE Calculate the aggregated information matrix for cluster $\overline C_t$: $\bV_{t-1}=\bV_0+\sum_{s\in[t-1]\atop i_s\in \overline C_t} (\phi(\bx_{s,1}) - \phi(\bx_{s,2}))(\phi(\bx_{s,1}) - \phi(\bx_{s,2}))^\top$. \alglinelabel{algo line: common matrix:neural}
            \STATE Choose the first arm  $\bx_{t,1} = \arg\max_{\bx\in\mathcal{X}_t} h(\bx;\overline{\btheta}_t)$; \alglinelabel{algo line: choose x1:neural}
            \STATE Choose the second arm $\bx_{t,2} = \arg\max_{\bx\in\mathcal{X}_t} h(\bx;\overline{\btheta}_t) + \left( \beta_T + B\sqrt{\frac{\lambda}{\kappa_\mu}} + 1 \right) \norm{\left(\phi(\bx) - \phi(\bx_{t,1})\right)}_{\bV_{t-1}^{-1}}$; \alglinelabel{algo line: choose x2:neural}
            % $x_{t,2}$ by maximizing the upper confidence bound in \eqref{eq:choose:arm:2}
		% \STATE Select $(x_{t,1}, x_{t,2})$ and observe human feedback $y_t$
		\STATE Observe the preference feedback: $y_t = \mathbbm{1}(\bx_{t,1}\succ \bx_{t,2})$, and update history: $\mathcal{D}_t=\{i_s, \bx_{s,1}, \bx_{s,2}, y_s\}_{s=1,\ldots,t}$;\alglinelabel{algo line: feedback:neural}
        \STATE Train the neural network using all data for user $i_t$: $\{(\bx_{s,1}, \bx_{s,2}, y_s)\}_{s\in[t], i_s = i_t}$ by minimizing the following loss function:\alglinelabel{algo line: update it:neural}
        \begin{small}
                        \begin{align}
                \hat{\btheta}_{i_t,t}&=\text{argmin}_{\btheta} 
                % \mathcal{L}_t(\btheta)= 
                - \frac{1}{m_{\text{NN}}} \sum_{s\in[t-1]\atop i_s = i_t}\Big( y_s\log\mu\left( h(\bx_{s,1};\btheta) - h(\bx_{s,2};\btheta) \right) \notag\\
                &+ (1-y_s)\log\mu\left(h(\bx_{s,2};\btheta) - h(\bx_{s,1};\btheta)\right) \Big)+ \frac{\lambda}{2}\norm{\btheta - \btheta_0}_2^2;
                \label{eq:loss:func:individial}
            \end{align}
        \end{small}

        % \STATE Update the estimation for the current served user $i_t$: \alglinelabel{algo line: update it}
        % \begin{equation}
        % \hat{\btheta}_{i_t,t}=\arg\min_{\btheta} - \sum_{s\in[t-1]\atop i_s=i_t}\left( y_s\log\mu\left({\btheta}^{\top}\left[\phi(\bx_{s,1}) - \phi(\bx_{s,2})\right]\right) + (1-y_s)\log\mu\left({\btheta}^{\top}\left[\phi(\bx_{s,2}) - \phi(\bx_{s,1})\right]\right) \right) + \frac{1}{2}\lambda\norm{\btheta}_2^2, 
        %     \end{equation}
            keep the estimations of other users unchanged;
            \STATE Delete the edge $(i_t,\ell)\in E_{t-1}$ if
            \begin{equation}
                \sqrt{m_{\text{NN}}}\norm{\hat\btheta_{i_t,t}-\hat\btheta_{\ell,t}}_2>f(T_{i_t,t})+f(T_{\ell,t})
            \end{equation} \alglinelabel{algo line: delete:neural}
            % \[
            %     \sqrt{m_{\text{NN}}} \norm{\btheta_{f,i} - \btheta_{f,l}} \geq \gamma'.
            % \]
		\ENDFOR
	\end{algorithmic}
\end{algorithm*}

\section{Proof of Theorem \ref{thm: linear regret bound}}
\label{app: proof linear}
First, we prove the following lemma.
\begin{lemma}\label{lemma:concentration:theta}
With probability at least $1-\delta$ for some $\delta\in(0,1)$, at any $t\in[T]$:
\begin{equation}
    \norm{\hat{\btheta}_{i,t}-\btheta^{j(i)}}_2\leq\frac{\sqrt{\lambda \kappa_\mu}+\sqrt{2\log(u/\delta)+d\log(1+T_{i,t}\kappa_\mu/d\lambda)}}{\kappa_\mu\sqrt{\lambda_{\text{min}}(\bV_{i,t-1})}}, \forall{i\in\mathcal{U}}\label{l2 norm difference bound}\,,
\end{equation}
where $\bV_{i,t-1}=\frac{\lambda}{\kappa_\mu} \mathbf{I}+\sum_{s\in[t-1]\atop i_s=i}(\phi(\bx_{s,1}) - \phi(\bx_{s,2}))(\phi(\bx_{s,1}) - \phi(\bx_{s,2}))^\top$, and $T_{i,t}$ denotes the number of rounds of seeing user $i$ in the first $t$ rounds.
\end{lemma}
\begin{proof}
    First, we prove the following result. 
    
    For a fixed user $i$, with probability at least $1-\delta$ for some $\delta\in(0,1)$, at any $t\in[T]$:
\begin{equation}
    \norm{\hat{\btheta}_{i,t}-\btheta^{j(i)}}_{\bV_{i,t-1}}\leq\frac{\sqrt{\lambda \kappa_\mu}+\sqrt{2\log(1/\delta)+d\log(1+4T_{i,t}\kappa_\mu/d\lambda)}}{\kappa_\mu}\label{V norm difference bound}\,,
\end{equation}
Recall that $f_i(\bx)=\btheta_i^\top\phi(\bx)$. In iteration $s$, define $\widetilde{\phi}_s = \phi(\bx_{s,1}) - \phi(\bx_{s,2})$.
And we define $\widetilde{f}_{i,s} = f_i(\bx_{s,1}) - f_i(\bx_{s,2}) =\btheta_i^{\top} \widetilde{\phi}_s$.

For any $\btheta_{f^\prime} \in\mathbb{R}^{d}$, define 
\[
G_{i,t}(\btheta_{f^\prime}) = \sum_{s\in[t-1]:\atop i_s=i}\left(\mu(\btheta_{f^\prime}^\top\widetilde{\phi}_s ) - \mu(\btheta_i^\top\widetilde{\phi}_s) \right) \widetilde{\phi}_s  + \lambda \btheta_{f'}.
\]
For $\lambda'\in(0, 1)$, setting $\btheta_{\bar{f}} = \lambda' \btheta_{f^\prime_1} + (1 - \lambda')\btheta_{f^\prime_2}$.
and using the mean-value theorem, we get:
\begin{align}
    \label{eqn:glb}
    G_{i,t}(\btheta_{f^\prime_1}) - G_{i,t}(\btheta_{f^\prime_2}) &= \left[\sum_{s\in[t-1]:\atop i_s=i} \nabla\mu(\btheta_{\bar{f}}^\top\widetilde{\phi}_s)\widetilde{\phi}_s \widetilde{\phi}_s^\top + \lambda \mathbf{I} \right](\btheta_{f^\prime_1} - \btheta_{f^\prime_2}) & \left( \btheta_i\text{ is constant} \right)  \nonumber \\
\end{align}
Define $\bM_{i,t-1} = \left[\sum_{s\in[t-1]:\atop i_s=i}\nabla\mu(\btheta_{\bar{f}}^\top\widetilde{\phi}_s)\widetilde{\phi}_s \widetilde{\phi}_s^\top + \lambda \mathbf{I} \right]$, and recall that $\bV_{i,t-1} = \sum_{s\in[t-1]:\atop i_s=i} \widetilde{\phi}_s \widetilde{\phi}_s^\top + \frac{\lambda}{\kappa_\mu} \mathbf{I}$.
Then we have that $\bM_{i,t-1} \succeq \kappa_\mu \bV_{i,t-1}$ and that $\bV^{-1}_{i,t-1} \succeq \kappa_\mu \bM^{-1}_{i,t-1}$, where we use the notation $\bM\succeq \bV$ to denote that $\bM-\bV$ is a positive semi-definite matrix. Then we have
% It is easy to verify that $M_t V^{-1/2} \geq M'_t V^{-1/2}$.

\begin{align*}
    &\norm{G_{i,t}(\hat\btheta_{i,t})-\lambda\btheta_i}_{\bV_{i,t-1}^{-1}}^2 \notag\\&=  \norm{G_{i,t}(\btheta_i) - G_t(\hat\btheta_{i,t})}_{\bV_{i,t-1}^{-1}}^2 \notag\\&= \norm{\bM_{i,t-1} (\btheta_i - \hat\btheta_{i,t})}_{\bV_{i,t-1}^{-1}}^2 & \left( G_{i,t}(\btheta_i) = \lambda\btheta_i\text{ by definition} \right)\\
    & = (\btheta_i - \hat\btheta_{i,t})^{\top} \bM_{i,t-1} \bV_{i,t-1}^{-1} \bM_{i,t-1} (\btheta_i - \hat\btheta_{i,t})\\
    &\geq (\btheta_i - \hat\btheta_{i,t})^{\top} \bM_{i,t-1} \kappa_\mu \bM_{i,t-1}^{-1} \bM_{i,t-1} (\btheta_i - \hat\btheta_{i,t})\\
    & = \kappa_\mu(\btheta_i - \hat\btheta_{i,t})^{\top} \bM_{i,t-1} (\btheta_i - \hat\btheta_{i,t})\\
    & \geq \kappa_\mu(\btheta_i - \hat\btheta_{i,t})^{\top} \kappa_\mu \bV_{i,t-1} (\btheta_i - \hat\btheta_{i,t})\\
    & = \kappa_\mu^2 (\btheta_i - \hat\btheta_{i,t})^{\top} \bV_{i,t-1} (\btheta_i - \hat\btheta_{i,t})\\
    & = \kappa_\mu^2 \norm{\btheta_i - \hat\btheta_{i,t}}^2_{\bV_{i,t-1}}  & \left(\text{as } ||\bx||_{\bA}^2 = \bx^\top \bA \bx \right)
\end{align*}
The first inequality is because $\bV^{-1}_{i,t-1} \succeq \kappa_\mu \bM^{-1}_{i,t-1}$, and the second inequality follows from $\bM_{i,t-1} \succeq \kappa_\mu \bV_{i,t-1}$.

Note that $\frac{\kappa_\mu}{\lambda} \mathbf{I}\succeq\bV_{i,t-1}$, which allows us to show that
\begin{equation}
\begin{split}
\norm{\lambda \btheta_i}_{\bV_{i,t-1}^{-1}} = \lambda \sqrt{ \btheta_i^{\top} \bV_{i,t-1}^{-1} \btheta_i} \leq \lambda \sqrt{ \btheta_i^{\top} \frac{\kappa_\mu}{\lambda} \btheta_i} \leq \sqrt{\lambda\kappa_\mu} \norm{\btheta_i}_2 \leq \sqrt{\lambda\kappa_\mu}.
\end{split}
\end{equation}
Using the two equations above, we have that
\begin{equation}
\begin{split}
\norm{\btheta_i - \hat{\btheta}_{i,t}}_{\bV_{i,t-1}} \le \frac{1}{\kappa_\mu} \norm{G_{i,t}(\hat{\btheta}_{i,t}) - \lambda \btheta_i}_{\bV_{i,t-1}^{-1}} &\leq \frac{1}{\kappa_\mu} \norm{G_{i,t}(\hat{\btheta}_{i,t})}_{\bV_{i,t-1}^{-1}} + \frac{1}{\kappa_\mu}\norm{\lambda \btheta_i}_{\bV_{i,t-1}^{-1}} \\
&\leq \frac{1}{\kappa_\mu} \norm{G_{i,t}(\hat{\btheta}_{i,t})}_{\bV_{i,t-1}^{-1}} + \sqrt{\frac{\lambda}{\kappa_\mu}}
\end{split}
\end{equation}
% \begin{equation}
% \begin{split}
% \norm{\btheta_i - \overline{\btheta}_t}_{\bV_{t-1}} \le \frac{1}{\kappa_\mu} \norm{G_t(\overline{\btheta}_t) - \lambda \btheta_i}_{\bV_{t-1}^{-1}} &\leq \frac{1}{\kappa_\mu} \norm{G_t(\overline{\btheta}_t)}_{\bV_{t-1}^{-1}} + \frac{1}{\kappa_\mu}\norm{\lambda \btheta_i}_{\bV_{t-1}^{-1}} \\
% &\leq \frac{1}{\kappa_\mu} \norm{G_t(\overline{\btheta}_t)}_{\bV_{t-1}^{-1}} + \sqrt{\frac{\lambda}{\kappa_\mu}}
% \end{split}
% \end{equation}

Then, let $f^i_{t,s}=\hat\btheta_{i,t}^\top\tilde\phi_s$, we have:
\begin{align*}
\frac{1}{\kappa_\mu^2} \norm{G_{i,t}(\hat\btheta_{i,t})}_{\bV_{i,t-1}^{-1}}^2 &= \frac{1}{\kappa_\mu^2} \norm{\sum_{s\in[t-1]:\atop i_s=i} (\mu(\hat\btheta_{i,t}^\top \widetilde{\phi}_s ) - \mu(\btheta_i^\top \widetilde{\phi}_s) ) \widetilde{\phi}_s + \lambda \hat\btheta_{i,t}}_{\bV_{i,t-1}^{-1}}^2 ) \\
    &= \frac{1}{\kappa_\mu^2} \norm{\sum_{s\in[t-1]:\atop i_s=i} (\mu(f^i_{t,s}) - \mu(\widetilde{f}_{i,s}) ) \widetilde{\phi}_s + \lambda \hat\btheta_{i,t}}_{\bV_{i,t-1}^{-1}}^2  \\
    &= \frac{1}{\kappa_\mu^2} \norm{\sum_{s\in[t-1]:\atop i_s=i} (\mu(f^i_{t,s}) - (y_s - \epsilon_s) ) \widetilde{\phi}_s + \lambda \hat\btheta_{i,t}}_{\bV_{i,t-1}^{-1}}^2 \\
    &= \frac{1}{\kappa_\mu^2} \norm{\sum_{s\in[t-1]:\atop i_s=i} \left(\mu(f^i_{t,s}) - y_s\right) \widetilde{\phi}_s + \sum_{s\in[t-1]:\atop i_s=i}\epsilon_s \widetilde{\phi}_s  + \lambda \hat\btheta_{i,t}}_{\bV_{i,t-1}^{-1}}^2 \\
    &\leq \frac{1}{\kappa_\mu^2} \norm{\sum_{s\in[t-1]:\atop i_s=i}\epsilon_s \widetilde{\phi}_s}_{\bV_{i,t-1}^{-1}}^2.
\end{align*}
The last step holds due to the following reasoning. Recall that $\hat\btheta_{i,t}$ is computed using MLE by solving the following equation:
    \begin{align}
         &\hat{\btheta}_{i_t,t} = \arg\min_{\btheta} \Bigg[ - \sum_{\substack{s \in [t-1] \\ i_s = i_t}} \bigg( y_s \log \mu\big(\btheta^\top [\phi(\bx_{s,1}) - \phi(\bx_{s,2})]\big)
        \notag\\&+ (1 - y_s) \log \mu\big(\btheta^\top [\phi(\bx_{s,2}) - \phi(\bx_{s,1})]\big) \bigg) + \frac{\lambda}{2} \|\btheta\|_2^2 \Bigg].
    \end{align}
Setting its gradient to $0$, the following is satisfied:
\begin{equation}
    \sum_{s\in[t-1]:\atop i_s=i} \left(\mu\left( \hat\btheta_{i,t}^{\top} \widetilde{\phi}_s \right) - y_s\right) \widetilde{\phi}_s + \lambda \hat\btheta_{i,t} = 0,
\end{equation}
which is used in the last step.

Now we have
\begin{equation}
    \label{eqn:parUB}
    \frac{1}{\kappa_\mu^2} \norm{G_{i,t}(\hat\btheta_{i,t})}_{\bV_{i,t-1}^{-1}}^2 \le \frac{1}{\kappa_\mu^2} \norm{\sum_{s\in[t-1]:\atop i_s=i}\epsilon_s \widetilde{\phi}_s}_{\bV_{i,t-1}^{-1}}^2.
\end{equation}

Denote $\bV \triangleq \frac{\lambda}{\kappa_\mu} \mathbf{I}$.
Note that the sequence of observation noises $\{\epsilon_s\}$ is $1$-sub-Gaussian.

Next, we can apply Theorem 1 from \cite{abbasi2011improved}, to obtain
\begin{equation}
\norm{\sum_{s\in[t-1]:\atop i_s=i}\epsilon_s \widetilde{\phi}_s}_{\bV_{i,t-1}^{-1}}^2 \leq 2\log\left( \frac{\det(\bV_{i,t-1})^{1/2}}{\delta\det(\bV)^{1/2}} \right),
\end{equation}
which holds with probability of at least $1-\delta$.

Next, based on our assumption that $\norm{\widetilde{\phi}_s}_2 \leq 2$, according to Lemma 10 from \cite{abbasi2011improved}, we have that
\begin{equation}
\det(\bV_{i,t-1}) \leq \left( \lambda/\kappa_\mu + 4T_{i,t} / d \right)^d\,,
\end{equation}
where $T_{i,t}$ denotes the number of rounds of serving user $i$ in the first $t$ rounds.
Therefore, 
\begin{equation}
\sqrt{\frac{\det{\bV_{i,t-1}}}{\det(V)}} \leq \sqrt{\frac{\left( \lambda/\kappa_\mu + 4T_{i,t} / d \right)^d}{(\lambda/\kappa_\mu)^d}} = \left( 1 + 4T_{i,t}\kappa_{\mu}/(d\lambda) \right)^{\frac{d}{2}}
\label{eq:upper:bound:det:Vt:V}
\end{equation}
This gives us
\begin{equation}
\norm{\sum_{s\in[t-1]:\atop i_s=i}\epsilon_s \widetilde{\phi}_s}_{\bV_{i,t-1}^{-1}}^2 \leq 2\log\left( \frac{\det(\bV_{i,t-1})^{1/2}}{\delta\det(V)^{1/2}} \right) \leq 2\log(1/\delta) + d\log\left( 1 + 4T_{i,t}\kappa_{\mu}/(d\lambda) \right)
\label{eq:upper:bound:proof:theta:interm}
\end{equation}

Then, with the above reasoning, we have that with probability at least $1-\delta$ for some $\delta\in(0,1)$, at any $t\in[T]$:
\begin{equation}
    \norm{\hat{\btheta}_{i,t}-\btheta^{j(i)}}_{\bV_{i,t-1}}\leq\frac{\sqrt{\lambda \kappa_\mu}+\sqrt{2\log(1/\delta)+d\log(1+4T_{i,t}\kappa_\mu/d\lambda)}}{\kappa_\mu}\,,
\end{equation}
% \begin{equation}
%     \norm{\hat{\btheta}_{i,t}-\btheta^{j(i)}}_{\bV_{i,t-1}}\leq\frac{\sqrt{\lambda/\kappa_\mu}+\sqrt{2\log(1/\delta)+d\log(1+4T_{i,t}\kappa_\mu/d\lambda)}}{\kappa_\mu}\,,
% \end{equation}

Taking a union bound over $u$ users, we have that with probability at least $1-\delta$ for some $\delta\in(0,1)$, at any $t\in[T]$:
\begin{equation}
    \norm{\hat{\btheta}_{i,t}-\btheta^{j(i)}}_{\bV_{i,t-1}}\leq\frac{\sqrt{\lambda \kappa_\mu}+\sqrt{2\log(u/\delta)+d\log(1+4T_{i,t}\kappa_\mu/d\lambda)}}{\kappa_\mu}\label{V norm difference bound union}\,, \forall i\in \mathcal{U}.
\end{equation}
Then we have that with probability at least $1-\delta$ for all $t\in[T]$ and all $i\in\mathcal{U}$
\begin{align}
    \norm{\hat{\btheta}_{i,t}-\btheta^{j(i)}}&\leq \frac{\norm{\hat{\btheta}_{i,t}-\btheta^{j(i)}}_{\bV_{i,t-1}}}{\sqrt{\lambda_{\text{min}(\bV_{i,t-1})}}}\notag\\
    &\leq \frac{\sqrt{\lambda \kappa_\mu}+\sqrt{2\log(u/\delta)+d\log(1+4T_{i,t}\kappa_\mu/d\lambda)}}{\kappa_\mu{\sqrt{\lambda_{\text{min}(\bV_{i,t-1})}}}}.
\end{align}
\end{proof}

Then, we prove the following lemma, which gives a sufficient time $T_0$ for the COLDB algorithm to cluster all the users correctly with high probability.

\begin{lemma}\label{T0 lemma}
    With the carefully designed edge deletion rule, after 
\begin{equation*}
    \begin{aligned}
        T_0&\triangleq 16u\log(\frac{u}{\delta})+4u\max\{
        \frac{128d}{\kappa_\mu^2\tilde{\lambda}_x\gamma^2}\log(\frac{u}{\delta}),\frac{16}{\tilde{\lambda}_x^2}\log(\frac{8ud}{\tilde{\lambda}_x^2\delta})\}\\
        &=O\bigg(u\left( \frac{d}{\kappa_\mu^2\tilde{\lambda}_x\gamma^2}+\frac{1}{\tilde{\lambda}_x^2}\right)\log \frac{1}{\delta}\bigg)
    \end{aligned}
\end{equation*}
rounds, with probability at least $1-3\delta$ for some $\delta\in(0,\frac{1}{3})$, COLDB can cluster all the users correctly.
\end{lemma}
\begin{proof}
    Then, with the item regularity assumption stated in Assumption \ref{assumption3}, Lemma J.1 in \cite{wang2024onlinea}, together with Lemma 7 in \cite{li2018online}, and applying a union bound, with probability at least $1-\delta$, for all $i\in\mathcal{U}$, at any $t$ such that $T_{i,t}\geq\frac{16}{\tilde{\lambda}_x^2}\log(\frac{8ud}{\tilde{\lambda}_x^2\delta})$, we have:
\begin{equation}
    \lambda_{\text{min}}(\bV_{i,t})\geq2\tilde{\lambda}_x T_{i,t}\,.
    \label{min eigen}
\end{equation}
Then, together with Lemma \ref{lemma:concentration:theta}, we have: if $T_{i,t}\geq\frac{16}{\tilde{\lambda}_x^2}\log(\frac{8ud}{\tilde{\lambda}_x^2\delta})$, then with probability $\geq 1-2\delta$, we have:
\begin{align}
    \norm{\hat{\btheta}_{i,t}-\btheta^{j(i)}}
    &\leq \frac{\sqrt{\lambda \kappa_\mu}+\sqrt{2\log(u/\delta)+d\log(1+4T_{i,t}\kappa_\mu/d\lambda)}}{\kappa_\mu{\sqrt{\lambda_{\text{min}(\bV_{i,t-1})}}}}\notag\\
    &\leq \frac{\sqrt{\lambda \kappa_\mu}+\sqrt{2\log(u/\delta)+d\log(1+4T_{i,t}\kappa_\mu/d\lambda)}}{\kappa_\mu{\sqrt{2\tilde{\lambda}_x T_{i,t}}}}\notag\,.
\end{align}

Now, let
\begin{equation}
    \frac{\sqrt{\lambda \kappa_\mu}+\sqrt{2\log(u/\delta)+d\log(1+4T_{i,t}\kappa_\mu/d\lambda)}}{\kappa_\mu{\sqrt{2\tilde{\lambda}_x T_{i,t}}}}<\frac{\gamma}{4}\,,
\end{equation}
Let $\lambda \kappa_\mu\leq2\log(u/\delta)+d\log(1+4T_{i,t}\kappa_\mu/d\lambda)$, which typically holds ($\kappa_\mu$ is typically very small), we can get
\begin{equation}
        \frac{2\log(u/\delta)+d\log(1+4T_{i,t}\kappa_\mu/d\lambda)}{2\kappa_\mu^2{\tilde{\lambda}_x T_{i,t}}}<\frac{\gamma^2}{64}\,,
\end{equation}
and a sufficient condition for it to hold is
\begin{align}
     \frac{2\log(u/\delta)}{2\kappa_\mu^2{\tilde{\lambda}_x T_{i,t}}}<\frac{\gamma^2}{128}\label{condition1}
\end{align}
and 
\begin{equation}
    \frac{d\log(1+4T_{i,t}\kappa_\mu/d\lambda)}{2\kappa_\mu^2{\tilde{\lambda}_x T_{i,t}}}<\frac{\gamma^2}{128}\,.\label{condition2}
\end{equation}
Solving Eq.(\ref{condition1}), we can get
\begin{equation}
    T_{i,t}\geq \frac{128\log(u/\delta)}{\kappa_\mu^2\tilde{\lambda}_x\gamma^2}\,.
\end{equation}
Following Lemma 9 in \cite{li2018online}, we can get the following sufficient condition for Eq.(\ref{condition2}):
\begin{equation}
    T_{i,t}\geq \frac{128d}{\kappa_\mu^2\tilde\lambda_x\gamma^2}\log(\frac{512}{\lambda\kappa_\mu\tilde{\lambda}_x\gamma^2})\,.
\end{equation}
Let $u/\delta\geq 512/\lambda\kappa_\mu\tilde{\lambda}_x\gamma^2$, which is typically held. Then, combining all together, we have that if
\begin{equation}
    T_{i,t}\geq\max\{\frac{128d}{\kappa_\mu^2\tilde{\lambda}_x\gamma^2}\log(\frac{u}{\delta}),\frac{16}{\tilde{\lambda}_x^2}\log(\frac{8ud}{\tilde{\lambda}_x^2\delta})\}, \forall i\in \mathcal{U}\,, \label{condition final}
\end{equation}
then with probability at least $1-2\delta$, we have
\begin{equation}
    \norm{\hat{\btheta}_{i,t}-\btheta^{j(i)}}<\gamma/4, \forall i\in\mathcal{U}\,.
\end{equation}
By Lemma 8 in \cite{li2018online}, and Assumption \ref{assumption2} of user arrival uniformness, we have that for all
\begin{equation*}
    \begin{aligned}
        T_0&\triangleq 16u\log(\frac{u}{\delta})+4u\max\{
        \frac{128d}{\kappa_\mu^2\tilde{\lambda}_x\gamma^2}\log(\frac{u}{\delta}),\frac{16}{\tilde{\lambda}_x^2}\log(\frac{8ud}{\tilde{\lambda}_x^2\delta})\}\\
        &=O\bigg(u\left( \frac{d}{\kappa_\mu^2\tilde{\lambda}_x\gamma^2}+\frac{1}{\tilde{\lambda}_x^2}\right)\log \frac{1}{\delta}\bigg)\,,
    \end{aligned}
\end{equation*}
the condition in Eq.(\ref{condition final}) is satisfied with probability at least $1-\delta$.

Therefore we have that for all $t\geq T_0$, with probability $\geq 1-3\delta$:
\begin{equation}
    \norm{\hat{\btheta}_{i,t}-\btheta^{j(i)}}_2<\frac{\gamma}{4}\,,\forall{i\in\mathcal{U}}\,.
    \label{final condition}
\end{equation}
Finally, we only need to show that with $\norm{\hat{\btheta}_{i,t}-\btheta^{j(i)}}_2<\frac{\gamma}{4}\,,\forall{i\in\mathcal{U}}$, the algorithm can cluster all the users correctly. First, when the edge $(i,l)$ is deleted, user $i$ and user $j$ must belong to different \gtclusters{}, i.e., $\norm{\btheta_i-\btheta_l}_2>0$. This is because by the deletion rule of the algorithm, the concentration bound, and triangle inequality
\begin{align}
   &\norm{\btheta_i-\btheta_l}_2=\norm{\btheta^{j(i)}-\btheta^{j(l)}}_2\notag\\
   &\geq\norm{\hat{\btheta}_{i,t}-\hat{\btheta}_{l,t}}_2-\norm{\btheta^{j(l)}-\hat{\btheta}_{l,t}}_2-\norm{\btheta^{j(i)}-\hat{\btheta}_{i,t}}_2\notag\\
   &\geq\norm{\hat{\btheta}_{i,t}-\hat{\btheta}_{l,t}}_2-f(T_{i,t})-f(T_{l,t})>0 \,.
\end{align}
Second, we can show that if $\norm{\btheta_i-\btheta_l}>\gamma$, meaning that user $i$ and user $l$ are not in the same \gtcluster, COLDB will delete the edge $(i,l)$ after $T_0$. This is because
\begin{align}
    \norm{\hat\btheta_{i,t}-\hat{\btheta}_{l,t}}&\geq \norm{\btheta_i-\btheta_l}-\norm{\hat{\btheta}_{i,t}-\btheta^{j(i)}}_2-\norm{\hat{\btheta}_{l,t}-\btheta^{j(l)}}_2\notag\\
    &>\gamma-\frac{\gamma}{4}-\frac{\gamma}{4}\notag\\
    &=\frac{\gamma}{2}>f(T_{i,t})+f(T_{l,t})\,,
\end{align}
which will trigger the edge deletion rule to delete edge $(i,l)$. Combining all the reasoning above, we can finish the proof.
\end{proof}

Then, we prove the following lemmas for the cluster-based statistics.
\begin{lemma}\label{lemma:concentration:theta cluster}
With probability at least $1-4\delta$ for some $\delta\in(0,1/4)$, at any $t\geq T_0$:
\begin{equation}
    \norm{\overline \btheta_t-\btheta_{i_t}}_{\bV_{t-1}}\leq\frac{\sqrt{\lambda \kappa_\mu}+\sqrt{2\log(u/\delta)+d\log(1+4T\kappa_\mu/d\lambda)}}{\kappa_\mu}\label{l2 norm difference bound for cluster}\,.
\end{equation}
\end{lemma}
\begin{proof}
First, by Lemma \ref{T0 lemma}, we have that with probability at least $1-3\delta$, all the users are clustered correctly, i.e., $\overline{C}_t=C_{j(i_t)}, \forall t\geq T_0$.  
Recall that $f_i(\bx)=\btheta_i^\top\phi(\bx)$. In iteration $s$, define $\widetilde{\phi}_s = \phi(\bx_{s,1}) - \phi(\bx_{s,2})$.
And we define $\widetilde{f}_{i,s} = f_i(\bx_{s,1}) - f_i(\bx_{s,2}) =\btheta_i^{\top} \widetilde{\phi}_s$.

For any $\btheta_{f^\prime} \in\mathbb{R}^{d}$, define 
\[
G_t(\btheta_{f^\prime}) = \sum_{s\in[t-1]:\atop i_s\in\overline{C}_t}\left(\mu(\btheta_{f^\prime}^\top\widetilde{\phi}_s ) - \mu(\btheta_{i_t}^\top\widetilde{\phi}_s) \right) \widetilde{\phi}_s  + \lambda \btheta_{f'}.
\]
For $\lambda'\in(0, 1)$, setting $\btheta_{\bar{f}} = \lambda' \btheta_{f^\prime_1} + (1 - \lambda')\btheta_{f^\prime_2}$.
and using the mean-value theorem, we get:
\begin{align}
    \label{eqn:glb}
    G_t(\btheta_{f^\prime_1}) - G_t(\btheta_{f^\prime_2}) &= \left[\sum_{s\in[t-1]:\atop i_s\in\overline{C}_t} \nabla\mu(\btheta_{\bar{f}}^\top\widetilde{\phi}_s)\widetilde{\phi}_s \widetilde{\phi}_s^\top + \lambda \mathbf{I} \right](\btheta_{f^\prime_1} - \btheta_{f^\prime_2})\nonumber \\
\end{align}
Define $\bM_{t-1} = \left[\sum_{s\in[t-1]:\atop i_s\in\overline{C}_t}\nabla\mu(\btheta_{\bar{f}}^\top\widetilde{\phi}_s)\widetilde{\phi}_s \widetilde{\phi}_s^\top + \lambda \mathbf{I} \right]$, and recall that $\bV_{t-1} = \sum_{s\in[t-1]:\atop i_s\in\overline{C}_t} \widetilde{\phi}_s \widetilde{\phi}_s^\top + \frac{\lambda}{\kappa_\mu} \mathbf{I}$.
Then we have that $\bM_{t-1} \succeq \kappa_\mu \bV_{t-1}$ and that $\bV^{-1}_{t-1} \succeq \kappa_\mu \bM^{-1}_{t-1}$. Then we have
% It is easy to verify that $M_t V^{-1/2} \geq M'_t V^{-1/2}$.

\begin{align*}
    &\norm{G_t(\overline{\btheta}_t) - \lambda \btheta_{i_t}}_{\bV_{t-1}^{-1}}^2 \notag\\&=  \norm{G_t(\btheta_{i_t}) - G_t(\overline{\btheta}_t)}_{\bV_{t-1}^{-1}}^2 \notag\\&= \norm{\bM_{t-1} (\btheta_{i_t} - \overline{\btheta}_t)}_{\bV_{t-1}^{-1}}^2 & \left( G_t(\btheta_{i_t}) = \lambda \btheta_{i_t} \text{ by definition} \right)\\
    & = (\btheta_{i_t} - \overline{\btheta}_t)^{\top} \bM_{t-1} \bV_{t-1}^{-1} \bM_{t-1} (\btheta_{i_t} - \overline{\btheta}_t)\\
    &\geq (\btheta_{i_t} - \overline{\btheta}_t)^{\top} \bM_{t-1} \kappa_\mu \bM_{t-1}^{-1} \bM_{t-1} (\btheta_{i_t} - \overline{\btheta}_t)\\
    & = \kappa_\mu(\btheta_{i_t} - \overline{\btheta}_t)^{\top} \bM_{t-1} (\btheta_{i_t} - \overline{\btheta}_t)\\
    & \geq \kappa_\mu(\btheta_{i_t} - \overline{\btheta}_t)^{\top} \kappa_\mu \bV_{t-1} (\btheta_{i_t} - \overline{\btheta}_t)\\
    & = \kappa_\mu^2 (\btheta_{i_t} - \overline{\btheta}_t)^{\top} \bV_{t-1} (\btheta_{i_t} - \overline{\btheta}_t)\\
    & = \kappa_\mu^2 \norm{\btheta_{i_t} - \overline{\btheta}_t}^2_{\bV_{t-1}}  & \left(\text{as } ||\bx||_{\bA}^2 = \bx^\top \bA \bx \right)
\end{align*}
The first inequality is because $\bV^{-1}_{t-1} \succeq \kappa_\mu \bM^{-1}_{t-1}$, and the second inequality follows from $\bM_{t-1} \succeq \kappa_\mu \bV_{t-1}$.

Note that $\frac{\kappa_\mu}{\lambda} \mathbf{I}\succeq\bV_{t-1}$, which allows us to show that
\begin{equation}
\begin{split}
\norm{\lambda \btheta_{i_t}}_{\bV_{t-1}^{-1}} = \lambda \sqrt{ \btheta_{i_t}^{\top} \bV_{t-1}^{-1} \btheta_{i_t}} \leq \lambda \sqrt{ \btheta_{i_t}^{\top} \frac{\kappa_\mu}{\lambda} \btheta_{i_t}} \leq \sqrt{\lambda\kappa_\mu} \norm{\btheta_{i_t}}_2 \leq \sqrt{\lambda\kappa_\mu}.
\end{split}
\end{equation}
Using the two equations above, we have that
\begin{equation}
\begin{split}
\norm{\btheta_{i_t} - \overline{\btheta}_t}_{\bV_{t-1}} \le \frac{1}{\kappa_\mu} \norm{G_t(\overline{\btheta}_t) - \lambda \btheta_{i_t}}_{\bV_{t-1}^{-1}} &\leq \frac{1}{\kappa_\mu} \norm{G_t(\overline{\btheta}_t)}_{\bV_{t-1}^{-1}} + \frac{1}{\kappa_\mu}\norm{\lambda \btheta_{i_t}}_{\bV_{t-1}^{-1}} \\
&\leq \frac{1}{\kappa_\mu} \norm{G_t(\overline{\btheta}_t)}_{\bV_{t-1}^{-1}} + \sqrt{\frac{\lambda}{\kappa_\mu}}
\end{split}
\end{equation}

Then, let $\overline f_{t,s}=\overline{\btheta}_t^\top\tilde\phi_s$, we have:

    \begin{align*}
% \norm{\btheta_{i_t} - \overline{\btheta}_t}^2_{\bV_{t-1}} &\le 
\frac{1}{\kappa_\mu^2} \norm{G_t(\overline{\btheta}_t)}_{\bV_{t-1}^{-1}}^2
    &\leq \frac{1}{\kappa_\mu^2} \norm{\sum_{s\in[t-1]:\atop i_s\in\overline{C}_t} (\mu(\overline{\btheta}_t^\top \widetilde{\phi}_s ) - \mu(\btheta_{i_t}^\top \widetilde{\phi}_s) ) \widetilde{\phi}_s + \lambda \overline{\btheta}_t}_{\bV_{t-1}^{-1}}^2  \\
    &= \frac{1}{\kappa_\mu^2} \norm{\sum_{s\in[t-1]:\atop i_s\in\overline{C}_t} (\mu(\overline f_{t,s}) - \mu(\widetilde{f}_{i_t,s}) ) \widetilde{\phi}_s + \lambda \overline{\btheta}_t}_{\bV_{t-1}^{-1}}^2 
     \\
    &= \frac{1}{\kappa_\mu^2} \norm{\sum_{s\in[t-1]:\atop i_s\in\overline{C}_t} (\mu(\overline f_{t,s}) - (y_s - \epsilon_s) ) \widetilde{\phi}_s + \lambda \overline{\btheta}_t}_{\bV_{t-1}^{-1}}^2  \notag\\
    & \left(y_s = \mu(\widetilde{f}_{i_t,s}) + \epsilon_s \text{if } i_s=i_t, \text{and} i_s=i_t, \forall i_s\in \overline{C}_t, \forall t\geq T_0 \right) \\
    &= \frac{1}{\kappa_\mu^2} \norm{\sum_{s\in[t-1]:\atop i_s\in\overline{C}_t} \left(\mu(\overline f_{t,s}) - y_s\right) \widetilde{\phi}_s + \sum_{s\in[t-1]:\atop i_s\in\overline{C}_t}\epsilon_s \widetilde{\phi}_s  + \lambda \overline{\btheta}_t}_{\bV_{t-1}^{-1}}^2 \\
    &\leq \frac{1}{\kappa_\mu^2} \norm{\sum_{s\in[t-1]:\atop i_s\in\overline{C}_t}\epsilon_s \widetilde{\phi}_s}_{\bV_{t-1}^{-1}}^2.
\end{align*}

The last step holds due to the following reasoning. Recall that $\overline{\btheta}_t$ is computed using MLE by solving the following equation:
    \begin{align}
        \overline{\btheta}_t&=\arg\min_{\btheta} - \sum_{s\in[t-1]\atop i_s\in \overline C_t}\left( y_s\log\mu\Big({\btheta}^{\top}\left[\phi(\bx_{s,1}) - \phi(\bx_{s,2})\right]\right) \notag\\&+ (1-y_s)\log\mu\left({\btheta}^{\top}\left[\phi(\bx_{s,2}) - \phi(\bx_{s,1})\right]\right) \Big) + \frac{1}{2}\lambda\norm{\btheta}_2^2.
    \end{align}
Setting its gradient to $0$, the following is satisfied:
\begin{equation}
    \sum_{s\in[t-1]:\atop i_s\in\overline{C}_t} \left(\mu\left( \overline{\btheta}_t^{\top} \widetilde{\phi}_s \right) - y_s\right) \widetilde{\phi}_s + \lambda \overline{\btheta}_t = 0,
\end{equation}
which is used in the last step.

Now we have
\begin{equation}
    \label{eqn:parUB}
    \norm{\btheta_{i_t} - \overline{\btheta}_t}_{\bV_{t-1}} \le \frac{1}{\kappa_\mu} \norm{\sum_{s\in[t-1]:\atop i_s\in\overline{C}_t}\epsilon_s \widetilde{\phi}_s}_{\bV_{t-1}^{-1}}  + \sqrt{\frac{\lambda}{\kappa_\mu}}.
\end{equation}
% \begin{equation}
%     \label{eqn:parUB}
%     \norm{\btheta_{i_t} - \overline{\btheta}_t}_{\bV_{t-1}}^2 \le \frac{1}{\kappa_\mu^2} \norm{\sum_{s\in[t-1]:\atop i_s\in\overline{C}_t}\epsilon_s \widetilde{\phi}_s}_{\bV_{t-1}^{-1}}^2.
% \end{equation}

Denote $\bV \triangleq \frac{\lambda}{\kappa_\mu} \mathbf{I}$.
Note that the sequence of observation noises $\{\epsilon_s\}$ is $1$-sub-Gaussian.

Next, we can apply Theorem 1 from \cite{abbasi2011improved}, to obtain
\begin{equation}
\norm{\sum_{s\in[t-1]:\atop i_s\in\overline{C}_t}\epsilon_s \widetilde{\phi}_s}_{\bV_{t-1}^{-1}}^2 \leq 2\log\left( \frac{\det(\bV_{t-1})^{1/2}}{\delta\det(\bV)^{1/2}} \right),
\end{equation}
which holds with probability of at least $1-\delta$.

Next, based on our assumption that $\norm{\widetilde{\phi}_s}_2 \leq 2$, according to Lemma 10 from \cite{abbasi2011improved}, we have that
\begin{equation}
\det(\bV_{t-1}) \leq \left( \lambda/\kappa_\mu + 4T / d \right)^d\,.
\end{equation}

Therefore, 
\begin{equation}
\sqrt{\frac{\det{\bV_{t-1}}}{\det(V)}} \leq \sqrt{\frac{\left( \lambda/\kappa_\mu + 4T / d \right)^d}{(\lambda/\kappa_\mu)^d}} = \left( 1 + 4T\kappa_{\mu}/(d\lambda) \right)^{\frac{d}{2}}
\label{eq:upper:bound:det:Vt:V}
\end{equation}
This gives us
\begin{equation}
\norm{\sum_{s\in[t-1]:\atop i_s\in\overline{C}_t}\epsilon_s \widetilde{\phi}_s}_{\bV_{t-1}^{-1}}^2 \leq 2\log\left( \frac{\det(\bV_{t-1})^{1/2}}{\delta\det(V)^{1/2}} \right) \leq 2\log(1/\delta) + d\log\left( 1 + 4T\kappa_{\mu}/(d\lambda) \right)
\label{eq:upper:bound:proof:theta:interm}
\end{equation}

Combining all together, we have with probability at least $1-4\delta$ for some $\delta\in(0,1/4)$, at any $t\geq T_0$:
\begin{equation}
    \norm{\overline \btheta_t-\btheta_{i_t}}_{\bV_{t-1}}\leq\frac{\sqrt{\lambda \kappa_\mu}+\sqrt{2\log(u/\delta)+d\log(1+4T\kappa_\mu/d\lambda)}}{\kappa_\mu}\,.
\end{equation}\end{proof}
Then, we prove the following lemma with the help of Lemma \ref{lemma:concentration:theta cluster}.
\begin{lemma}
    \label{lemma:ucb:diff}
For any iteration $t\geq T_0$, for all $\bx,\bx'\in\mathcal{X}_t$, with probability of at least $1-4\delta$, we have 
\[
|\left(f_{i_t}(\bx) - f_{i_t}(\bx')\right) - \overline\btheta_t^{\top}\left( \phi(\bx) - \phi(\bx') \right)| \leq \frac{\beta_T}{\kappa_\mu}\norm{\phi(\bx) - \phi(\bx')}_{\bV_{t-1}^{-1}}\,,
\]
where $\beta_T=\sqrt{\lambda \kappa_\mu}+\sqrt{2\log(u/\delta)+d\log(1+4T\kappa_\mu/d\lambda)}$.
\end{lemma}
\begin{proof}
\begin{equation}
\begin{split}
&|\left(f_{i_t}(\bx) - f_{i_t}(\bx')\right) - \overline\btheta_t^{\top}\left( \phi(\bx) - \phi(\bx') \right)| \notag\\&= |\btheta_{i_t}^{\top} \left[(\phi(\bx) - \phi(\bx')\right] - \overline\btheta_t^{\top}\left[ \phi(\bx) - \phi(\bx') \right]|\\
&= | \left(\btheta_{i_t} - \overline\btheta_t\right)^{\top} \left[\phi(\bx) - \phi(\bx') \right]|\\
&\leq \norm{\btheta_{i_t} - \overline\btheta_t}_{\bV_{t-1}} \norm{\phi(\bx) - \phi(\bx')}_{\bV_{t-1}^{-1}}\\
&\leq \frac{\beta_T}{\kappa_\mu} \norm{\phi(\bx) - \phi(\bx')}_{\bV_{t-1}^{-1}},
\end{split}
\end{equation}
in which the last inequality follows from Lemma \ref{lemma:concentration:theta cluster}.
\end{proof}

We also prove the following lemma to upper bound the summation of squared norms which will be used in proving the final regret bound.
\begin{lemma}
With probability at least $1-4\delta$, we have
\label{lemma:concentration:square:std}
\[
\sum^T_{t=T_0}\mathbb{I}\{i_t\in C_j\}\norm{\phi(\bx_{t,1}) - \phi(\bx_{t,2})}_{\bV_{t-1}^{-1}}^2 \leq 2 d\log \left( 1 + 4T \kappa_{\mu}/(d\lambda) \right)\,, \forall j\in[m]\,,
\]
\end{lemma}
where $\mathbb{I}$ denotes the indicator function.
\begin{proof}
We denote $\widetilde{\phi}_t = \phi(\bx_{t,1}) - \phi(\bx_{t,2})$.
Recall that we have assumed that $\norm{\phi(\bx_{t,1}) - \phi(\bx_{t,2})}_2 \leq 2$.
It is easy to verify that $\bV_{t-1} \succeq \frac{\lambda}{\kappa_\mu} I$ and hence $\bV_{t-1}^{-1} \preceq \frac{\kappa_\mu}{\lambda}I$.
Therefore, we have that $\norm{\widetilde{\phi}_t}_{\bV_{t-1}^{-1}}^2 \leq \frac{\kappa_\mu}{\lambda} \norm{\widetilde{\phi}_t}_{2}^2 \leq \frac{4\kappa_\mu}{\lambda}$. We choose $\lambda$ such that $\frac{4\kappa_\mu}{\lambda} \leq 1$, which ensures that $\norm{\widetilde{\phi}_t}_{\bV_{t-1}^{-1}}^2 \leq 1$.
Our proof here mostly follows from Lemma 11 of \cite{abbasi2011improved} and Lemma J.2 of \cite{wang2024onlinea}. To begin with, note that $x\leq 2\log(1+x)$ for $x\in[0,1]$. Denote $\bV_{t,j}=\sum_{s\in[t-1]:\atop i_s\in C_j} \widetilde{\phi}_s \widetilde{\phi}_s^\top + \frac{\lambda}{\kappa_\mu} \mathbf{I}$. Then we have that 
\begin{equation}
\begin{split}
\sum^T_{t=T_0}\mathbb{I}\{i_t\in C_j\}\norm{\widetilde{\phi}_t}_{\bV_{t-1}^{-1}}^2 &\leq \sum^T_{t=T_0} 2\log\left(1 + \mathbb{I}\{i_t\in C_j\}\norm{\widetilde{\phi}_t}_{\bV_{t-1}^{-1}}^2\right)\\
&= 2 \left( \log\det V_{T,j} - \log\det V \right)\\
&= 2 \log \frac{\det V_{T,j}}{\det V}\\
&\leq 2\log \left( \left( 1 + 4T \kappa_{\mu}/(d\lambda) \right)^{d} \right)\\
&= 2 d\log \left( 1 + 4T \kappa_{\mu}/(d\lambda) \right).
% &\leq 2 \left(d\log( (\text{trace}(V) + tL^2) / d) - \log\det V \right)\\
% &= 2 \left(d\log( ( \frac{\lambda}{\kappa_\mu} d + tL^2) / d) - \log\det V \right)\\
\end{split}
\end{equation}
The second inequality follows the same reasoning as \eqref{eq:upper:bound:det:Vt:V}.
This completes the proof.
\end{proof}

Now we are ready to prove Theorem \ref{thm: linear regret bound}.
First, we have
\begin{equation}
    R_T=\sum_{t=1}^T r_t\leq T_0 +\sum_{t=T_0}^T r_t\,,
\end{equation}
where we use that the reward at each round is bounded by 1.

Then, we only need to upper bound the regret after $T_0$. By Lemma \ref{T0 lemma}, we know that with probability at least $1-4\delta$, the algorithm can cluster all the users correctly, $\overline{C}_t=C_{j(i_t)}$, and the statements of all the above lemmas hold. We have that for any $t\geq T_0$:
\begin{equation}
\begin{split}
r_t &= f_{i_t}(\bx^*_t) - f_{i_t}(\bx_{t,1}) + f_{i_t}(\bx^*_t) - f_{i_t}(x_{t,2})\\
&\stackrel{(a)}{\leq} \overline\btheta_t^\top \left( \phi(\bx^*_t) - \phi(\bx_{t,1}) \right) + \frac{\beta_T}{\kappa_\mu}\norm{\phi(\bx^*_t) - \phi(\bx_{t,1})}_{\bV_{t-1}^{-1}} +  \overline\btheta_t^\top \left( \phi(\bx^*_t) - \phi(\bx_{t,2}) \right) \notag\\&+ \frac{\beta_T}{\kappa_\mu}\norm{\phi(\bx^*_t) - \phi(\bx_{t,2})}_{\bV_{t-1}^{-1}}\\
&= \overline\btheta_t^\top \left( \phi(\bx^*_t) - \phi(\bx_{t,1}) \right)+ \frac{\beta_T}{\kappa_\mu}\norm{\phi(\bx^*_t) - \phi(\bx_{t,1})}_{\bV_{t-1}^{-1}} + \\
&\qquad \overline\btheta_t^\top \left( \phi(\bx^*_t) - \phi(\bx_{t,1}) \right) + \overline\btheta_t^\top \left( \phi(\bx_{t,1}) - \phi(\bx_{t,2}) \right) + \notag\\&\frac{\beta_T}{\kappa_\mu}\norm{\phi(\bx^*_t) - \phi(\bx_{t,1}) + \phi(\bx_{t,1}) - \phi(\bx_{t,2})}_{\bV_{t-1}^{-1}}\\
&\stackrel{(b)}{\leq} 2 \overline\btheta_t^\top \left( \phi(x^*) - \phi(\bx_{t,1}) \right) + 2 \frac{\beta_T}{\kappa_\mu}\norm{\phi(x^*) - \phi(\bx_{t,1})}_{\bV_{t-1}^{-1}} + \\
&\qquad \overline\btheta_t^\top \left( \phi(\bx_{t,1}) - \phi(\bx_{t,2}) \right) + \frac{\beta_T}{\kappa_\mu}\norm{\phi(\bx_{t,1}) - \phi(\bx_{t,2})}_{\bV_{t-1}^{-1}}\\
&\stackrel{(c)}{\leq} 2 \overline\btheta_t^\top \left( \phi(\bx_{t,2}) - \phi(\bx_{t,1}) \right) + 2 \frac{\beta_T}{\kappa_\mu}\norm{\phi(\bx_{t,2}) - \phi(\bx_{t,1})}_{\bV_{t-1}^{-1}} + \\
&\qquad \overline\btheta_t^\top \left( \phi(\bx_{t,1}) - \phi(\bx_{t,2}) \right) + \frac{\beta_T}{\kappa_\mu}\norm{\phi(\bx_{t,1}) - \phi(\bx_{t,2})}_{\bV_{t-1}^{-1}}\\
&\leq \overline\btheta_t^\top \left( \phi(\bx_{t,2}) - \phi(\bx_{t,1}) \right) + 3 \frac{\beta_T}{\kappa_\mu}\norm{\phi(\bx_{t,2}) - \phi(\bx_{t,1})}_{\bV_{t-1}^{-1}} \\
&\stackrel{(d)}{\leq} 3 \frac{\beta_T}{\kappa_\mu}\norm{\phi(\bx_{t,1}) - \phi(\bx_{t,2})}_{\bV_{t-1}^{-1}} \\
\end{split}
\label{eq:upper:bound:inst:regret}
\end{equation}
Step $(a)$ follows from Lemma \ref{lemma:ucb:diff}. Step $(b)$ makes use of the triangle inequality.
Step $(c)$ follows from the way in which we choose the second arm $\bx_{t,2}$: $\bx_{t,2} = \arg\max_{x\in\mathcal{X}_t} \overline\btheta_t^\top \left( \phi(x) - \phi(\bx_{t,1}) \right) + \frac{\beta_T}{\kappa_\mu}\norm{\phi(x) - \phi(\bx_{t,1})}_{\bV_{t-1}^{-1}}$.
Step $(d)$ results from the way in which we select the first arm: $\bx_{t,1} = \arg\max_{x\in\mathcal{X}_t}\overline\btheta_t^\top \phi(x)$.

Then we have
\begin{align}
    \sum_{t=T_0}^T r_t &\leq 3 \frac{\beta_T}{\kappa_\mu}\sum_{t=T_0}^T\norm{\phi(\bx_{t,1}) - \phi(\bx_{t,2})}_{\bV_{t-1}^{-1}}\notag\\
    &=3 \frac{\beta_T}{\kappa_\mu}\sum_{t=T_0}^T\sum_{j\in[m]}\mathbb{I}\{i_t\in C_j\}\norm{\phi(\bx_{t,1}) - \phi(\bx_{t,2})}_{\bV_{t-1}^{-1}}\notag\\
    &\leq 3 \frac{\beta_T}{\kappa_\mu}\sqrt{\sum_{t=T_0}^T\sum_{j\in[m]}\mathbb{I}\{i_t\in C_j\}\sum_{t=T_0}^T\sum_{j\in[m]}\mathbb{I}\{i_t\in C_j\}\norm{\phi(\bx_{t,1}) - \phi(\bx_{t,2})}_{\bV_{t-1}^{-1}}^2}\notag\\
    &\leq 3 \frac{\beta_T}{\kappa_\mu}\sqrt{T\cdot m\cdot 2 d\log \left( 1 + 4T \kappa_{\mu}/(d\lambda) \right)}\,,
\end{align}
where in the second inequality we use the Cauchy-Swarchz inequality, and in the last step we use $\sum_{t=T_0}^T\sum_{j\in[m]}\mathbb{I}\{i_t\in C_j\}\leq T$ and Lemma \ref{lemma:concentration:square:std}.

Therefore, finally, we have with probability at least $1-4\delta$
\begin{align}
    R_T & \leq T_0+ 3 \frac{\beta_T}{\kappa_\mu}\sqrt{T\cdot m\cdot 2 d\log \left( 1 + 4T \kappa_{\mu}/(d\lambda) \right)}\notag\\
    &\leq O(u(\frac{d}{\kappa_\mu^2\tilde\lambda_x \gamma^2}+\frac{1}{\tilde\lambda_x^2})\log T+\frac{1}{\kappa_\mu}d\sqrt{mT})\notag\\
        &=O(\frac{1}{\kappa_\mu}d\sqrt{mT})\,,
\end{align}

\section{Proof of Theorem \ref{thm: neural regret bound}}
\label{app: proof neural}

\subsection{Auxiliary Definitions and Explanations}
\label{app:subsec:aux:defs}
% {\color{blue}
% Definition of the NTK matrix $\mathbf{H}$ [to be added later]
% }

% {\color{blue}
% To begin with, we provide the details regarding the NTK matrice $\mathbf{H}_j$ for different clusters $j$'s.
\paragraph{Denifition of the NTK matrix $\mathbf{H}_j$ for cluster $j$.}
Recall that we use $T_j$ to denote the total number of iterations in which the users in cluster $j$ are served.
For cluster $j$, let $\{x_{(i)}\}_{i=1}^{T_j K}$ be a set of all $T_j \times K$ possible arm feature vectors: $\{x_{t,a}\}_{1\le t \le T_j, 1\le a \le K}$, where $i = K(t-1) + a$. 
Firstly, we define $\mathbf{h}_t = [f^j(x_{(i)})]_{i=1,\ldots,T_j K}^{\top}$, i.e., $\mathbf{h}_t$ is the $T_j K$-dimensional vector containing the reward function values of the arms corresponding to cluster $j$.
Next, define 
$$
\widetilde{\mathbf{H}}_{p,q}^{(1)} = \mathbf{\Sigma}_{p,q}^{(1)} = \langle x_{(p)}, x_{(q)}  \rangle, \\
\mathbf{A}_{p,q}^{(l)} =\begin{pmatrix}
	\mathbf{\Sigma}_{p,q}^{(l)} & \mathbf{\Sigma}_{p,q}^{(l)} &\\
	\mathbf{\Sigma}_{p,q}^{(l)} &\mathbf{\Sigma}_{q,q}^{(l)} &
\end{pmatrix},
$$
$$
\mathbf{\Sigma}_{p,q}^{(l+1)} = 2\mathbb{E}_{(u,v)\sim\mathcal{N}(0,\mathbf{A}_{p,q}^{(l)} )}[\max\{u,0\}\max\{v,0\}],
$$
$$
\widetilde{\mathbf{H}}_{p,q}^{(l+1)} = 2\widetilde{\mathbf{H}}_{p,q}^{(l)}\mathbb{E}_{(u,v)\sim\mathcal{N}(0,\mathbf{A}_{p,q}^{(l)} )}[\mathbbm{1}(u \ge 0)\mathbbm{1}(v \ge 0)] + \mathbf{\Sigma}_{p,q}^{(l+1)}.
$$
With these definitions, the NTK matrix for cluster $j$ is then defined as $\mathbf{H}_j = (\widetilde{\mathbf{H}}^{(L)} + \mathbf{\Sigma}^{(L)})/2$.
% is called the neural tangent kernel (NTK) matrix on the set of context-arm feature vectors $\{x_{(n)}\}_{n=1}^{TK}$. 
% }

\paragraph{The Initial Parameters $\btheta_0$.}
Next, we discuss how the initial parameters $\btheta_0$ are obtained.
We adopt the same initialization method from \cite{zhang2020neural,zhou2020neural}.
% \paragraph{Details of the Initialization Scheme $\theta_0\sim\text{init}(\cdot)$.}
Specifically, for each $l=1,\ldots,L-1$, let 
%$\mathbf{W}_l=(\mathbf{W},\mathbf{0};\mathbf{0},\mathbf{W})$ 
$\mathbf{W}_l=\left(
\begin{array}{cc} 
  \mathbf{W} & \mathbf{0} \\ 
  \mathbf{0} & \mathbf{W} 
\end{array} 
\right)$
in which every entry of $\mathbf{W}$ is independently and randomly sampled from $\mathcal{N}(0, 4/m_{\text{NN}})$, and choose $\mathbf{W}_L=(\mathbf{w}^{\top},-\mathbf{w}^{\top})$ in which every entry of $\mathbf{w}$ is independently and randomly sampled from $\mathcal{N}(0,2/m_{\text{NN}})$.
% This initialization scheme is the same as that used by the works of \cite{zhang2020neural,zhou2020neural}.

\paragraph{Justifications for Assumption \ref{assumption:main:neural}.}
The last assumption in Assumption \ref{assumption:main:neural}, together with the way we initialize $\theta_0$ as discussed above, ensures that the initial output of the NN is $0$: $h(x;\theta_0)=0,\forall x\in\mathcal{X}$.
The assumption of $x_{j}=x_{j+d/2}$ from Assumption \ref{assumption:main:neural} is a mild assumption which is commonly adopted by previous works on neural bandits \cite{zhou2020neural,zhang2020neural}. 
% This assumption is only for convenience of the analysis: 
To ensure that this assumption holds, for any arm $x$, we can always firstly normalize it such that $||x|| = 1$, and then construct a new context $x' = (x^\top,x^\top)^\top/\sqrt{2}$ to satisfy this assumption \cite{zhou2020neural}.

\subsection{Proof}
\label{app:subsec:proof:neural:real:proof}

% {\color{blue}
% Specific conditions on the width $m$ of the NN [to be added later]
To begin with, we first list the specific conditions we need for the width $m_{\text{NN}}$ of the NN:
\begin{equation}
	\begin{split}
	&m_{\text{NN}} \geq C T^4K^4 L^6\log(T^2K^2 L/\delta) / \lambda_0^4,\\
	&m_{\text{NN}}(\log m)^{-3} \geq C \kappa_\mu^{-3} T^{8} L^{21} \lambda^{-5} ,\\
	&m_{\text{NN}}(\log m_{\text{NN}})^{-3} \geq C \kappa_\mu^{-3} T^{14} L^{21} \lambda^{-11} L_\mu^6,\\
	&m_{\text{NN}}(\log m_{\text{NN}})^{-3} \geq C T^{14} L^{18} \lambda^{-8},
	\end{split}
	\label{eq:conditions:on:m}
\end{equation}
for some absolute constant $C>0$.
To ease exposition, we express these conditions above as 
$m_{\text{NN}} \geq \text{poly}(T, L, K, 1/\kappa_\mu, L_\mu, 1/\lambda_0, 1/\lambda, \log(1/\delta))$.
% }

In our proof here, we use the gradient of the NN  at $\btheta_0$ to derive the feature mapping for the arms, i.e., we let $\phi(\bx) = g(\bx;\btheta_0) / \sqrt{m_{\text{NN}}}$.
We use $\hat{\btheta}_{i,t}$ to denote the paramters of the NN after training in iteration $t$ (see Algorithm \ref{algo:neural:dueling:bandits}).

% To begin with, w
We use the following lemma to show that for every cluster $j\in\mathcal{C}$, its reward function $f^j$ can be expressed as a linear function with respect to the initial gradient $g(\bx;\btheta_0)$.
\begin{lemma}[Lemma B.3 of \cite{zhang2020neural}]
\label{lemma:linear:utility:function}
As long as the width $m$ of the NN is large enough:
\[
	m_{\text{NN}} \geq C_0 T^4K^4 L^6\log(T^2K^2 L/\delta) / \lambda_0^4,
\]
then for all clusters $j\in[m]$, with probability of at least $1-\delta$, there exits a $\btheta^j_{f}$ such that 
\[ 
	f^j(\bx) = \langle g(\bx;\btheta_0), \btheta^j_{f} - \btheta_0 \rangle, \qquad \sqrt{m_{\text{NN}}} \norm{\btheta^j_{f} - \btheta_0}_2 \leq \sqrt{2\mathbf{h}_j^{\top} \mathbf{H}_j^{-1} \mathbf{h}_j} \leq B.
\]
for all $\bx\in\mathcal{X}_{t}$, $t\in[T]$ with $i_t\in C_{j}$.
\end{lemma}
% \begin{lemma}[Lemma B.3 of \cite{zhang2020neural}]
% \label{lemma:linear:utility:function}
% As long as the width $m$ of the NN is large enough:
% \[
% 	m_{\text{NN}} \geq C_0 T^4K^4u^4L^6\log(T^2K^2u^2L/\delta) / \lambda_0^4,
% \]
% then for all users $i\in\mathcal{U}$, with probability of at least $1-\delta$, there exits a $\btheta_{f,i}$ such that 
% \[ 
% 	f_i(\bx) = \langle g(\bx;\btheta_0), \btheta_{f,i} - \btheta_0 \rangle, \qquad \sqrt{m_{\text{NN}}} \norm{\btheta_{f,i} - \btheta_0}_2 \leq \sqrt{2\mathbf{h}_{i}^{\top} \mathbf{H}_{i}^{-1} \mathbf{h}_{i}} \leq B.
% \]
% for all $\bx\in\mathcal{X}_{t}$, for all $t\in[T]$ such that $i_t=i$.
% \end{lemma}
% This lemma differs from Lemma B.3 of \cite{zhang2020neural} only by an extra union bound over all $u$ users.

Lemma \ref{lemma:linear:utility:function} is the formal statement of Lemma \ref{lemma:linear:utility:function:informal} from Sec.~\ref{subsec:problem:setting:neural}.
% In Lemma \ref{lemma:linear:utility:function}, although each user $i$ only collects $T_{i,T} \leq T$ observations, we have considered the worst-case scenario for each user, i.e., we have taken a union bound over a total of $T\times K \times u$ contexts.
% % Recall that we use $T_{i,t}$ to denote the number of iterations user $i$ is received till iteration $t$.
% Here $\mathbf{H}_{i}$ is the $T_{i,T} \times K$-dimensinoal NTK matrix for all the contexts received by user $i$.
Note that the constant $B$ is applicable to all $m$ clusters.

The following lemma converts our assumption about cluster separation (Assumption \ref{assumption:gap:neural:bandits}) into the difference between the linearized parameters for different clusters.
\begin{lemma}
\label{lemma:neural:gap:theta}
If users $i$ and $l$ belong to different clusters, then we have that
\[
\sqrt{m_{\text{NN}}} \norm{\btheta_{f,i} - \btheta_{f,l}} \geq \gamma'.
\]
\end{lemma}
\begin{proof}
To begin with, Lemma \ref{lemma:linear:utility:function} tells us that
\begin{equation}
\begin{split}
|f_i(\bx) - f_l(\bx)| = | \langle g(\bx;\btheta_0),  \btheta_{f,i} - \btheta_{f,l}\rangle | \leq \norm{g(\bx;\btheta_0)} \norm{\btheta_{f,i} - \btheta_{f,l}}.
\end{split}
\end{equation}
This leads to
\begin{equation}
\begin{split}
\norm{\btheta_{f,i} - \btheta_{f,l}} \geq \frac{|f_i(\bx) - f_l(\bx)|}{\norm{g(\bx;\btheta_0)}} \geq \frac{\gamma'}{\sqrt{m_{\text{NN}}}} ,
\end{split}
\end{equation}
in which we have made use of Assumption \ref{assumption:gap:neural:bandits} and our assumption that $\frac{1}{m_{\text{NN}}}\langle g(\bx;\btheta_0), g(\bx;\btheta_0) \rangle \leq 1$ in the last inequality.
This completes the proof.
\end{proof}

% \begin{lemma}
% % \label{lemma:concentration:theta}
% With probability at least $1-\delta$ for some $\delta\in(0,1)$, at any $t\in[T]$:
% \begin{equation}
%     \norm{\hat{\btheta}_{i,t}-\btheta^{j(i)}}_2\leq\frac{\sqrt{\lambda \kappa_\mu}+\sqrt{2\log(u/\delta)+d\log(1+T_{i,t}\kappa_\mu/d\lambda)}}{\kappa_\mu\sqrt{\lambda_{\text{min}}(\bV_{i,t-1})}}, \forall{i\in\mathcal{U}}\label{l2 norm difference bound}\,,
% \end{equation}
% where $\bV_{i,t-1}=\frac{\lambda}{\kappa_\mu} \mathbf{I}+\sum_{s\in[t-1]\atop i_s=i}(\phi(\bx_{s,1}) - \phi(\bx_{s,2}))(\phi(\bx_{s,1}) - \phi(\bx_{s,2}))^\top$, and $T_{i,t}$ denotes the number of rounds of seeing user $i$ in the first $t$ rounds.
% \end{lemma}
The following lemma shows that for every user, the output of the NN trained using its own local data can be approximated by a linear function.
\begin{lemma}
\label{lemma:bound:approx:error:linear:nn:duel:individual}
    Let $\varepsilon'_{m_{\text{NN}},t} \triangleq C_2 m_{\text{NN}}^{-1/6}\sqrt{\log m_{\text{NN}}} L^3 \left(\frac{t}{\lambda}\right)^{4/3}$ where $C_2>0$ is an absolute constant.
    % For all $t\in[T]$, we have for all $x,x'\in\mathcal{X}_t$ that
    % We have that
    Then
    \[
   		|\langle g(\bx;\btheta_0), \hat{\btheta}_{i,t} - \btheta_0 \rangle - h(\bx;\hat{\btheta}_{i,t}) | \leq \varepsilon'_{m_{\text{NN}},t}, \,\,\, \forall t\in[T], \bx,\bx'\in\mathcal{X}_t.
    \]
\end{lemma}
% \end{restatable}
\begin{proof}
This lemma can be proved following a similar line of proof as Lemma 1 from \cite{verma2024neural}.
Here the $t$ in $\varepsilon'_{m_{\text{NN}},t}$ can in fact be replaced by $T_{i,t} \leq t$, however, we have simply used its upper bound $t$ for simplicity.
\end{proof}

\begin{lemma}
\label{lemma:conf:ellip:neural}
Let $\beta_T \triangleq \frac{1}{\kappa_\mu} \sqrt{ \widetilde{d} + 2\log(u/\delta)}$.
Assuming that the conditions on $m_{\text{NN}}$ from \cref{eq:conditions:on:m} are satisfied.
With probability of at least $1-\delta$, we have that
\[
	\sqrt{m_{\text{NN}}} \norm{\btheta_{f,i} - \hat{\btheta}_{i,t}}_{2} \leq  \frac{\beta_T + B \sqrt{\frac{\lambda}{\kappa_\mu}} + 1}{\sqrt{\lambda_{\min}(\bV_{i,t-1})}}, \qquad \forall t\in[T].
\]
where $\bV_{i,t-1}=\frac{\lambda}{\kappa_\mu} \mathbf{I}+\sum_{s\in[t-1]\atop i_s=i}(\phi(\bx_{s,1}) - \phi(\bx_{s,2}))(\phi(\bx_{s,1}) - \phi(\bx_{s,2}))^\top$, $\phi(\bx) = \frac{1}{\sqrt{m_{\text{NN}}}} g(\bx;\btheta_0)$, and $T_{i,t}$ denotes the number of rounds of seeing user $i$ in the first $t$ rounds.
\end{lemma}
\begin{proof}
In iteration $t$, for any user $i\in\mathcal{U}$, the user leverages its current history of observations $\{(\bx_{s,1}, \bx_{s,2}, y_s)\}_{s\in[t-1], i_s = i}$ to train the NN by minimizing the loss function (\eqref{eq:loss:func:individial}), to obtain the NN parameters $\hat{\btheta}_{i,t}$.
Note that the NN has been trained when the most recent observation in $\{(\bx_{s,1}, \bx_{s,2}, y_s)\}_{s\in[t-1], i_s = i}$ was collected, i.e., the last time when user $i$ was encountered.
Of note, according to Lemma \ref{lemma:linear:utility:function}, the latent reward function of user $i$ can be expressed as $f_i(\bx) = \langle g(\bx;\btheta_0), \btheta_{f,i} - \btheta_0 \rangle$.
Therefore, from the perspective of each individual user $i$, the user is faced with a \emph{neural dueling bandit} problem instance.
As a result, we can modifying the proof of Lemma 6 from \cite{verma2024neural} to show that
% \begin{lemma}
% \label{lemma:conf:ellip:neural}
% Let $\beta_T \triangleq \frac{1}{\kappa_\mu} \sqrt{ \widetilde{d} + 2\log(1/\delta)}$.
% Assuming that the conditions on $m$ from \cref{eq:conditions:on:m} are satisfied.
with probability of at least $1-\delta$,
% {\color{blue}[proof to be added later]}, 
\[
	\sqrt{m_{\text{NN}}} \norm{\btheta_{f,i} - \hat{\btheta}_{i,t}}_{\bV_{i,t-1}} \leq  \beta_T + B \sqrt{\frac{\lambda}{\kappa_\mu}} + 1, \qquad \forall t\in[T], i\in\mathcal{U}.
\]
Here in our definition of $\beta_T \triangleq \frac{1}{\kappa_\mu} \sqrt{ \widetilde{d} + 2\log(u/\delta)}$, we have replaced the error probability $\delta$ (from \cite{verma2024neural}) by $\delta/u$ to account for the use of an extra union bound over all $u$ users.
% \end{lemma}
% Here $\bV_{i,t-1}=\frac{\lambda}{\kappa_\mu} \mathbf{I}+\sum_{s\in[t-1]\atop i_s=i}(\phi(\bx_{s,1}) - \phi(\bx_{s,2}))(\phi(\bx_{s,1}) - \phi(\bx_{s,2}))^\top$, in which $\phi(\bx) = g(\bx;\btheta_0) / \sqrt{m}$.

This allows us to show that
\begin{equation}
\begin{split}
\sqrt{m_{\text{NN}}} \norm{\btheta_{f,i} - \hat{\btheta}_{i,t}}_2 &\leq \frac{\sqrt{m_{\text{NN}}} \norm{\btheta_{f,i} - \hat{\btheta}_{i,t}}_{\bV_{i,t-1}}}{\sqrt{\lambda_{\min}(\bV_{i,t-1})}}\\
&\leq \frac{\beta_T + B \sqrt{\frac{\lambda}{\kappa_\mu}} + 1}{\sqrt{\lambda_{\min}(\bV_{i,t-1})}}
\end{split}
\end{equation}
This completes the proof.
\end{proof}

\begin{lemma}\label{T0 lemma neural}
    With the carefully designed edge deletion rule in Algorithm \ref{algo:neural:dueling:bandits}, after 
\begin{equation*}
    \begin{aligned}
        T_0&\triangleq 16u\log(\frac{u}{\delta})+4u \max\left\{\frac{32 \left( \widetilde{d} + 2\log(u/\delta)\right)}{\tilde{\lambda}_x \gamma^2 \kappa_\mu^2},  \frac{16}{\tilde{\lambda}_x^2}\log(\frac{24u d m^2(L-1)}{\tilde{\lambda}_x^2\delta}) \right\}\\
        &= O\left(u \left( \frac{ \widetilde{d}}{\kappa_\mu^2\tilde{\lambda}_x \gamma^2} + \frac{1}{\tilde{\lambda}_x^2} \right)\log(\frac{1}{\delta}) \right),
        % O\bigg(u\left( \frac{d}{\kappa_\mu^2\tilde{\lambda}_x\gamma^2}+\frac{1}{\tilde{\lambda}_x^2}\right)\log \frac{1}{\delta}\bigg)\,,
    \end{aligned}
\end{equation*}
rounds, with probability at least $1-3\delta$ for some $\delta\in(0,\frac{1}{3})$, CONDB can cluster all the users correctly.
\end{lemma}
\begin{proof}
Recall that we use $p = dm_{\text{NN}} + m_{\text{NN}}^2(L-1) + m_{\text{NN}}$ to denote the total number of parameters of the NN.
Similar to the proof of Lemma \ref{T0 lemma}, with the item regularity assumption stated in Assumption \ref{assumption3}, Lemma J.1 in \cite{wang2024onlinea}, together with Lemma 7 in \cite{li2018online} (note that when using these technical results, we use $g(\bx;\btheta)/\sqrt{m_{\text{NN}}}$ as the feature vector to replace the original feature vector of $\bx$), and applying a union bound, with probability at least $1-\delta$, for all $i\in\mathcal{U}$, at any $t$ such that $T_{i,t}\geq\frac{16}{\tilde{\lambda}_x^2}\log(\frac{8up}{\tilde{\lambda}_x^2\delta})$, we have:
\begin{equation}
    \lambda_{\text{min}}(\bV_{i,t})\geq2\tilde{\lambda}_x T_{i,t}\,.
    \label{min eigen}
\end{equation}
Note that compared with the proof of \ref{T0 lemma}, in the lower bound on $T_{i,t}$ here, we have replaced the dimension $d$ by $p$. This has led to a logarithmic dependence 
% of the lower bound 
on the width $m_{\text{NN}}$ of the NN.
To simplify the exposition, using the fact that $p \geq 3dm_{\text{NN}}^2(L-1)$, we replace this condition on $T_{i,t}$ by a slightly stricter condition: $T_{i,t}\geq\frac{16}{\tilde{\lambda}_x^2}\log(\frac{8u \times 3dm_{\text{NN}}^2(L-1)}{\tilde{\lambda}_x^2\delta}) = \frac{16}{\tilde{\lambda}_x^2}\log(\frac{24u d m_{\text{NN}}^2(L-1)}{\tilde{\lambda}_x^2\delta})$.

Then, together with Lemma \ref{lemma:conf:ellip:neural}, we have: if 
% $T_{i,t}\geq\frac{16}{\tilde{\lambda}_x^2}\log(\frac{8ud}{\tilde{\lambda}_x^2\delta})$, 
$T_{i,t}\geq\frac{16}{\tilde{\lambda}_x^2}\log(\frac{8u \times 3dm_{\text{NN}}^2(L-1)}{\tilde{\lambda}_x^2\delta})$, 
then with probability $\geq 1-2\delta$, we have:
\begin{align}
    \sqrt{m_{\text{NN}}} \norm{\hat{\btheta}_{i,t}-\btheta^{j(i)}}
    &\leq \frac{\beta_T + B \sqrt{\frac{\lambda}{\kappa_\mu}} + 1}{\sqrt{\lambda_{\min}(\bV_{i,t-1})}} \leq \frac{\beta_T + B \sqrt{\frac{\lambda}{\kappa_\mu}} + 1}{\sqrt{2\tilde{\lambda}_x T_{i,t}}}\notag\,.
\end{align}

Now, let
\begin{equation}
    \frac{\beta_T + B \sqrt{\frac{\lambda}{\kappa_\mu}} + 1}{\sqrt{2\tilde{\lambda}_x T_{i,t}}}<\frac{\gamma}{4}\,,
\end{equation}

Note that in Algorithm \ref{algo:neural:dueling:bandits}, we have defined the funciton $f$ as 
\begin{equation}
f(T_{i,t}) \triangleq \frac{\beta_T + B \sqrt{\frac{\lambda}{\kappa_\mu}} + 1}{\sqrt{2\tilde{\lambda}_x T_{i,t}}}
\end{equation}
This immediately leads to
\begin{equation}
\sqrt{m_{\text{NN}}} \norm{\hat{\btheta}_{i,t}-\btheta^{j(i)}} \leq f(T_{i,t}) < \frac{\gamma}{4}.
\end{equation}

For simplicity, now let $B \sqrt{\frac{\lambda}{\kappa_\mu}} + 1 \leq \beta_T$ which is typically satisfied. This allows us to show that
\begin{equation}
    T_{i,t} > \frac{32\beta_T^2}{\tilde{\lambda}_x \gamma^2} = \frac{32 \left(\frac{1}{\kappa_\mu} \sqrt{ \widetilde{d} + 2\log(u/\delta)}\right)^2}{\tilde{\lambda}_x \gamma^2} = \frac{32 \left( \widetilde{d} + 2\log(u/\delta)\right)}{\tilde{\lambda}_x \gamma^2 \kappa_\mu^2}.
\label{condition final neural}
\end{equation}

Combining both conditions on $T_{i,t}$ together, we have that
\begin{equation}
T_{i,t}\geq \max\left\{\frac{32 \left( \widetilde{d} + 2\log(u/\delta)\right)}{\tilde{\lambda}_x \gamma^2 \kappa_\mu^2},  \frac{16}{\tilde{\lambda}_x^2}\log(\frac{24u d m_{\text{NN}}^2(L-1)}{\tilde{\lambda}_x^2\delta}) \right\}
\end{equation}

By Lemma 8 in \cite{li2018online} and Assumption \ref{assumption2} of user arrival uniformness, we have that for all
\begin{equation*}
    \begin{aligned}
        T_0&\triangleq 16u\log(\frac{u}{\delta})+4u \max\left\{\frac{32 \left( \widetilde{d} + 2\log(u/\delta)\right)}{\tilde{\lambda}_x \gamma^2 \kappa_\mu^2},  \frac{16}{\tilde{\lambda}_x^2}\log(\frac{24u d m_{\text{NN}}^2(L-1)}{\tilde{\lambda}_x^2\delta}) \right\}\\
        &= O\left(u \left( \frac{ \widetilde{d}}{\kappa_\mu^2\tilde{\lambda}_x \gamma^2} + \frac{1}{\tilde{\lambda}_x^2} \right)\log(\frac{1}{\delta}) \right),
        % O\bigg(u\left( \frac{d}{\kappa_\mu^2\tilde{\lambda}_x\gamma^2}+\frac{1}{\tilde{\lambda}_x^2}\right)\log \frac{1}{\delta}\bigg)\,,
    \end{aligned}
\end{equation*}
the condition in Eq.(\ref{condition final neural}) is satisfied with probability at least $1-\delta$.

Therefore we have that for all $t\geq T_0$, with probability $\geq 1-3\delta$:
\begin{equation}
    \sqrt{m_{\text{NN}}}\norm{\hat{\btheta}_{i,t}-\btheta^{j(i)}}_2<\frac{\gamma}{4}\,,\forall{i\in\mathcal{U}}\,.
    % \label{final condition neural}
\end{equation}
Finally, we show that as long as the condition $\sqrt{m_{\text{NN}}}\norm{\hat{\btheta}_{i,t}-\btheta^{j(i)}}_2<\frac{\gamma}{4}\,,\forall{i\in\mathcal{U}}$, our algorithm can cluster all the users correctly.

First, we show that when the edge $(i,l)$ is deleted, user $i$ and user $j$ must belong to different \gtclusters{}, i.e., $\norm{\btheta_{f,i}-\btheta_{f,l}}_2>0$. 
This is because by the deletion rule of the algorithm, the concentration bound, and triangle inequality
\begin{align}
   &\sqrt{m_{\text{NN}}}\norm{\btheta_{f,i}-\btheta_{f,l}}_2=\sqrt{m_{\text{NN}}}\norm{\btheta^{j(i)}-\btheta^{j(l)}}_2\notag\\
   &\geq \sqrt{m_{\text{NN}}}\norm{\hat{\btheta}_{i,t}-\hat{\btheta}_{l,t}}_2 - \sqrt{m_{\text{NN}}}\norm{\btheta^{j(l)}-\hat{\btheta}_{l,t}}_2 - \sqrt{m_{\text{NN}}}\norm{\btheta^{j(i)}-\hat{\btheta}_{i,t}}_2\notag\\
   &\geq \sqrt{m_{\text{NN}}}\norm{\hat{\btheta}_{i,t}-\hat{\btheta}_{l,t}}_2-f(T_{i,t})-f(T_{l,t})>0 \,.
\end{align}
Second, we can show that if 
% $\norm{\btheta_i-\btheta_l}>\gamma$, 
$|f_i(\bx) - f_l(\bx)| \geq \gamma',\forall \bx\in\mathcal{X}$,
meaning that user $i$ and user $l$ are not in the same \gtcluster, CONDB will delete the edge $(i,l)$ after $T_0$.
Note that when user $i$ and user $l$ are not in the same \gtcluster, Lemma \ref{lemma:neural:gap:theta} tells us that $\sqrt{m_{\text{NN}}} \norm{\btheta_{f,i} - \btheta_{f,l}} \geq \gamma'$.
Then we have that
\begin{align}
    &\sqrt{m_{\text{NN}}}\norm{\hat\btheta_{i,t}-\hat{\btheta}_{l,t}}
    \notag\\&\geq \sqrt{m_{\text{NN}}}\norm{\btheta_{f,i}-\btheta_{f,l}}- \sqrt{m_{\text{NN}}}\norm{\hat{\btheta}_{i,t}-\btheta^{j(i)}}_2-\sqrt{m_{\text{NN}}}\norm{\hat{\btheta}_{l,t}-\btheta^{j(l)}}_2\notag\\
    &>\gamma-\frac{\gamma}{4}-\frac{\gamma}{4}\notag\\
    &=\frac{\gamma}{2}>f(T_{i,t})+f(T_{l,t})\,,
\end{align}
which will trigger the edge deletion rule to delete edge $(i,l)$. 
% Combining all the reasoning above, we can finish the proof.
This completes the proof.
\end{proof}

Then, we prove the following lemmas for the cluster-based statistics.
\begin{lemma}\label{lemma:concentration:theta cluster:neural}
Assuming that the conditions on $m$ from \cref{eq:conditions:on:m} are satisfied.
With probability at least $1-4\delta$ for some $\delta\in(0,1/4)$, at any $t\geq T_0$:
% \begin{equation}
%     \norm{\overline \btheta_t-\btheta_{i_t}}_{\bV_{t-1}}\leq\frac{\sqrt{\lambda \kappa_\mu}+\sqrt{2\log(u/\delta)+d\log(1+4T\kappa_\mu/d\lambda)}}{\kappa_\mu}\label{l2 norm difference bound for cluster}\,.
% \end{equation}
\[
	\sqrt{m_{\text{NN}}} \norm{\btheta_{f,i_{t}} - \overline{\btheta}_{t}}_{\bV_{t-1}} \leq  \beta_T + B \sqrt{\frac{\lambda}{\kappa_\mu}} + 1, \qquad \forall t\in[T].
\]
\end{lemma}
\begin{proof}
To begin with, note that by Lemma \ref{T0 lemma neural}, we have that with probability of at least $1-3\delta$, all users are clustered correctly, i.e., $\overline{C}_t=C_{j(i_t)}, \forall t\geq T_0$.
Note that according to our Algorithm \ref{algo:neural:dueling:bandits}, in iteration $t$, we select the pair of arms using all the data collected by all users in cluster $\overline{C}_t$.
That is, $\overline{\btheta}_{t}$ represents the NN parameters trained using the data from all users in the cluster $\overline{C}_t$ (i.e., $\{(\bx_{s,1}, \bx_{s,2}, y_s)\}_{s\in[t-1], i_s\in \overline C_t}$), and $\bV_t$ also contains the data from all users in this cluster $\overline{C}_t$.
Therefore, in iteration $t$, we are effectively following a neural dueling bandit algorithm using $\{(\bx_{s,1}, \bx_{s,2}, y_s)\}_{s\in[t-1], i_s\in \overline C_t}$ as the current observation history.
This allows us to leverage the proof of Lemma 6 from \cite{verma2024neural} to complete the proof.
% {\color{blue}[proof to be added later]}
% Lemma \ref{lemma:conf:ellip:neural}, 
% following the proof here can also be achieved by modifying the proof of Lemma 6 from \cite{verma2024neural}.
\end{proof}

% \begin{restatable}{lem}{linearNN}
\begin{lemma}
\label{lemma:bound:approx:error:linear:nn:duel}
    Let $\varepsilon'_{m_{\text{NN}},t} \triangleq C_2 m_{\text{NN}}^{-1/6}\sqrt{\log m_{\text{NN}}} L^3 \left(\frac{t}{\lambda}\right)^{4/3}$ where $C_2>0$ is an absolute constant.
    % For all $t\in[T]$, we have for all $x,x'\in\mathcal{X}_t$ that
    % We have that
    Then
    \[
   		|\langle g(\bx;\btheta_0) -g(\bx';\btheta_0), \overline{\btheta}_t - \btheta_0 \rangle - (h(\bx;\overline{\btheta}_t) - h(\bx';\overline{\btheta}_t)) | \leq  2\varepsilon'_{m_{\text{NN}},t}, \,\,\, \forall t\in[T], \bx,\bx'\in\mathcal{X}_t.
    \]
\end{lemma}
% \end{restatable}
\begin{proof}
This lemma can be proved following a similar line of proof as Lemma 1 from \cite{verma2024neural}.
\end{proof}

% \begin{restatable}{thm}{confBound}
\begin{lemma}
	\label{thm:confBound:neural}  
    Let $\delta\in(0,1)$, $\varepsilon'_{m_{\text{NN}},t} \doteq C_2 m_{\text{NN}}^{-1/6}\sqrt{\log m_{\text{NN}}} L^3 \left(\frac{t}{\lambda}\right)^{4/3}$ for some $C_2>0$.
    As long as $m_{\text{NN}} \geq \text{poly}(T, L, K, u, 1/\kappa_\mu, L_\mu, 1/\lambda_0, 1/\lambda, \log(1/\delta))$, then with probability of at least $1-\delta$, at any $t\geq T_0$,
    \[
        |\left[f_{i_t}(\bx) - f_{i_t}(\bx')\right] - \left[h(\bx;\overline{\btheta}_t) - h(\bx';\overline{\btheta}_t)\right]| \leq \nu_T \sigma_{t-1}(\bx, \bx') + 2\varepsilon'_{m_{\text{NN}},t},
    \]
    for all $\bx,\bx'\in\mathcal{X}_t, t\in[T]$. 
\end{lemma}
\begin{proof}
	Denote $\phi(\bx) = \frac{1}{\sqrt{m_{\text{NN}}}} g(\bx;\btheta_0)$.
	Recall that \cref{lemma:linear:utility:function} tells us that $f_{i_t}(\bx) = \langle g(\bx;\btheta_0), \btheta_{f,i_t} - \btheta_0 \rangle=\langle \phi(\bx), \btheta_{f,i_t} - \btheta_0 \rangle$ for all $\bx\in\mathcal{X}_t,t\in[T]$.
	To begin with, for all $\bx,\bx'\in\mathcal{X}_t,t\in[T]$ we have that
		\begin{equation}
		\begin{split}
			|&f_{i_t}(\bx) - f_{i_t}(\bx') - \langle g(\bx;\btheta_0) - g(\bx';\btheta_0), \overline{\btheta}_t - \btheta_0 \rangle| \\
			&= |\langle g(\bx;\btheta_0) - g(\bx';\btheta_0), \btheta_{f,i_t} - \theta_0 \rangle - \langle g(\bx;\btheta_0) - g(\bx';\btheta_0), \overline{\btheta}_t - \btheta_0 \rangle|\\
			&= |\langle g(\bx;\btheta_0) - g(\bx';\btheta_0), \btheta_{f,i_t} - \overline{\btheta}_t \rangle  \rangle|\\
			&= |\langle  \phi(\bx)-\phi(\bx'), \sqrt{m_{\text{NN}}}\left( \btheta_{f,i_t} - \overline{\btheta}_t\right) \rangle  |\\
			&\leq \norm{\left(\phi(\bx)-\phi(\bx')\right)}_{\bV_{t-1}^{-1}} \sqrt{m_{\text{NN}}}\norm{\btheta_{f,i_t} - \overline{\btheta}_t}_{\bV_{t-1}}\\
			&\leq \norm{\left(\phi(\bx)-\phi(\bx')\right)}_{\bV_{t-1}^{-1}} \left( \beta_T + B \sqrt{\frac{\lambda}{\kappa_\mu}} + 1 \right),
		\end{split}
		\label{eq:diff:between:func:and:linear:approx:dueling}
		\end{equation}
	in which we have used Lemma \ref{lemma:concentration:theta cluster:neural} in the last inequality.
	Now making use of the equation above and \cref{lemma:bound:approx:error:linear:nn:duel}, we have that 
	\begin{equation}
    \begin{split}
		&|f_{i_t}(\bx) - f_{i_t}(\bx') - (h(\bx;\btheta_t) - h(\bx';\btheta_t))| \\
		&= | f_{i_t}(\bx) - f_{i_t}(\bx') - \langle g(\bx;\btheta_0) - g(\bx';\btheta_0), \overline{\btheta}_t - \btheta_0 \rangle \\
		&\qquad\qquad\qquad + \langle g(\bx;\btheta_0) - g(\bx';\btheta_0), \overline{\btheta}_t - \btheta_0 \rangle - (h(\bx;\overline{\btheta}_t) - h(\bx';\overline{\btheta}_t)) |\\
        &\leq | f_{i_t}(\bx) - f_{i_t}(\bx') - \langle g(\bx;\btheta_0) - g(\bx';\btheta_0), \overline{\btheta}_t - \btheta_0 \rangle | \\
		&\qquad\qquad\qquad + |\langle g(\bx;\btheta_0) - g(\bx';\btheta_0), \overline{\btheta}_t - \btheta_0 \rangle - (h(\bx;\overline{\btheta}_t) - h(\bx';\overline{\btheta}_t)) |\\
		% &\leq |(f_{i_t}(x) - f_{i_t}(x')) - \langle \phi(x)-\phi(x'), \theta_t - \theta_0 \rangle  | \\
		% &\qquad\qquad\qquad + |\langle \phi(x)-\phi(x'), \theta_t - \theta_0 \rangle - (h(x;\theta_t) - h(x';\theta_t)) |\\
		&\leq \norm{\frac{1}{\sqrt{m_{\text{NN}}}}\left(\phi(\bx)-\phi(\bx')\right)}_{\bV_{t-1}^{-1}} \left( \beta_T + B \sqrt{\frac{\lambda}{\kappa_\mu}} + 1 \right) + 2\varepsilon'_{m_{\text{NN}},t}.\\
    \end{split}
    \end{equation}
	
	This completes the proof.
    % of \cref{thm:confBound}.
\end{proof}

We also prove the following lemma to upper bound the summation of squared norms which will be used in proving the final regret bound.
\begin{lemma}
With probability at least $1-4\delta$, we have
\label{lemma:concentration:square:std:neural}
\[
\sum^T_{t=T_0}\mathbb{I}\{i_t\in C_j\} \norm{\phi(\bx_{t,1}) - \phi(\bx_{t,2})}_{\bV_{t-1}^{-1}}^2 \leq 16 \widetilde{d}\,, \forall j\in[m]\,,
\]
\end{lemma}
where $\mathbb{I}$ denotes the indicator function.
\begin{proof}
We denote $\widetilde{\phi}_t = \phi(\bx_{t,1}) - \phi(\bx_{t,2})$.
Note that we have defined $\phi(\bx) = \frac{1}{\sqrt{m_{\text{NN}}}}g(\bx;\btheta_0)$.
Here we assume that $\norm{\phi(\bx_{t,1}) - \phi(\bx_{t,2})}_2 = \frac{1}{\sqrt{m_{\text{NN}}}}\norm{g(\bx_{t,1};\btheta_0) - g(\bx_{t,2};\btheta_0)}_{2} \leq 2$.
Replacing $2$ by an absolute constant $c_0$ would only change the final regret bound by a constant factor, so we omit it for simplicity.

% Recall that we have assumed that $\norm{\phi(\bx_{t,1}) - \phi(\bx_{t,2})}_2 \leq 2$.
It is easy to verify that $\bV_{t-1} \succeq \frac{\lambda}{\kappa_\mu} I$ and hence $\bV_{t-1}^{-1} \preceq \frac{\kappa_\mu}{\lambda}I$.
Therefore, we have that $\norm{\widetilde{\phi}_t}_{\bV_{t-1}^{-1}}^2 \leq \frac{\kappa_\mu}{\lambda} \norm{\widetilde{\phi}_t}_{2}^2 \leq \frac{4\kappa_\mu}{\lambda}$. We choose $\lambda$ such that $\frac{4\kappa_\mu}{\lambda} \leq 1$, which ensures that $\norm{\widetilde{\phi}_t}_{\bV_{t-1}^{-1}}^2 \leq 1$.
% and hence $\frac{\kappa_\mu}{\lambda} \frac{1}{m} \norm{\widetilde{\phi}_t}_{\bV_{t-1}^{-1}}^2 \leq 1$.
Our proof here mostly follows from Lemma 11 of \cite{abbasi2011improved} and Lemma J.2 of \cite{wang2024onlinea}. To begin with, note that $x\leq 2\log(1+x)$ for $x\in[0,1]$. Denote $\bV_{t,j}=\sum_{s\in[t-1]:\atop i_s\in C_j} \widetilde{\phi}_s \widetilde{\phi}_s^\top + \frac{\lambda}{\kappa_\mu} \mathbf{I}$. Then we have that 
\begin{equation}
\begin{split}
\sum^T_{t=T_0}\mathbb{I}\{i_t\in C_j\} \norm{\widetilde{\phi}_t}_{\bV_{t-1}^{-1}}^2 &\leq \sum^T_{t=T_0} 2\log\left(1 + \mathbb{I}\{i_t\in C_j\} \norm{\widetilde{\phi}_t}_{\bV_{t-1}^{-1}}^2\right)\\
&\leq 16 \log\det\left(\frac{\kappa_\mu}{\lambda}\mathbf{H}' + \mathbf{I}\right) \\
&\triangleq 16 \widetilde{d}.
% &= 2 \left( \log\det V_{T,j} - \log\det V \right)\\
% &= 2 \log \frac{\det V_{T,j}}{\det V}\\
% &\leq 2\log \left( \left( 1 + 4T \kappa_{\mu}/(d\lambda) \right)^{d} \right)\\
% &= 2 d\log \left( 1 + 4T \kappa_{\mu}/(d\lambda) \right).
% &\leq 2 \left(d\log( (\text{trace}(V) + tL^2) / d) - \log\det V \right)\\
% &= 2 \left(d\log( ( \frac{\lambda}{\kappa_\mu} d + tL^2) / d) - \log\det V \right)\\
\end{split}
\end{equation}
% The second inequality follows the same reasoning as \eqref{eq:upper:bound:det:Vt:V}.
The second inequality follows from the proof in Section A.3 from \cite{verma2024neural}.
This completes the proof.
\end{proof}

Now we are ready to prove Theorem \ref{thm: neural regret bound}.
To begin with, we have that
% \begin{equation}
$
    R_T=\sum_{t=1}^T r_t\leq T_0 +\sum_{t=T_0}^T r_t$.
% \end{equation}
% where we use that the reward at each round is bounded by 1.

Then, we only need to upper-bound the regret after $T_0$. By Lemma \ref{T0 lemma neural}, we know that with probability at least $1-4\delta$, the algorithm can cluster all the users correctly, $\overline{C}_t=C_{j(i_t)}$, and the statements of all the above lemmas hold. We have that for any $t\geq T_0$:
    
To simplify exposion here, we denote $\beta_T' \triangleq \beta_T + B \sqrt{\lambda / \kappa_\mu} + 1$.
\begin{equation}
\begin{split}
r_t &= f_{i_t}(\bx^*_t) - f_{i_t}(\bx_{t,1}) + f_{i_t}(\bx^*_t) - f_{i_t}(x_{t,2})\\
&\stackrel{(a)}{\leq} \langle g(\bx^*_t;\btheta_0) - g(\bx_{t,1};\btheta_0), \overline\btheta_t-\btheta_0 \rangle + \beta_T' \norm{\phi(\bx^*_t) - \phi(\bx_{t,1})}_{\bV_{t-1}^{-1}} + \\
&\qquad \langle g(\bx^*_t;\btheta_0) - g(\bx_{t,2};\btheta_0), \overline\btheta_t-\btheta_0 \rangle + \beta_T'\norm{\phi(\bx^*_t) - \phi(\bx_{t,2})}_{\bV_{t-1}^{-1}}\\
&= \langle g(\bx^*_t;\btheta_0) - g(\bx_{t,1};\btheta_0), \overline\btheta_t-\btheta_0 \rangle + \beta_T' \norm{\phi(\bx^*_t) - \phi(\bx_{t,1})}_{\bV_{t-1}^{-1}} + \\
&\qquad \langle g(\bx^*_t;\btheta_0) - g(\bx_{t,1};\btheta_0), \overline\btheta_t-\btheta_0 \rangle + \langle g(\bx_{t,1};\btheta_0) - g(\bx_{t,2};\btheta_0), \overline\btheta_t-\btheta_0 \rangle + \\
&\qquad \beta_T'\norm{\phi(\bx^*_t) - \phi(\bx_{t,1}) + \phi(\bx_{t,1}) - \phi(\bx_{t,2})}_{\bV_{t-1}^{-1}}\\
% &= \overline\btheta_t^\top \left( \phi(\bx^*_t) - \phi(\bx_{t,1}) \right) + \frac{\beta_T}{\kappa_\mu}\norm{\phi(\bx^*_t) - \phi(\bx_{t,1})}_{\bV_{t-1}^{-1}} + \\
% &\qquad \overline\btheta_t^\top \left( \phi(\bx^*_t) - \phi(\bx_{t,1}) \right) + \overline\btheta_t^\top \left( \phi(\bx_{t,1}) - \phi(\bx_{t,2}) \right) + \frac{\beta_T}{\kappa_\mu}\norm{\phi(\bx^*_t) - \phi(\bx_{t,1}) + \phi(\bx_{t,1}) - \phi(\bx_{t,2})}_{\bV_{t-1}^{-1}}\\
&\stackrel{(b)}{\leq} 2 \langle g(\bx^*_t;\btheta_0) - g(\bx_{t,1};\btheta_0), \overline\btheta_t-\btheta_0 \rangle + 2 \beta_T' \norm{\phi(\bx^*_t) - \phi(\bx_{t,1})}_{\bV_{t-1}^{-1}} + \\
&\qquad \langle g(\bx_{t,1};\btheta_0) - g(\bx_{t,2};\btheta_0), \overline\btheta_t-\btheta_0 \rangle + \beta_T'\norm{\phi(\bx_{t,1}) - \phi(\bx_{t,2})}_{\bV_{t-1}^{-1}}\\
&\stackrel{(c)}{\leq} 2 h(\bx^*_t;\overline{\btheta}_t) - 2 h(\bx_{t,1};\overline{\btheta}_t) + 4\varepsilon'_{m_{\text{NN}},t} + 2 \beta_T' \norm{\phi(\bx^*_t) - \phi(\bx_{t,1})}_{\bV_{t-1}^{-1}} + \\
&\qquad h(\bx_{t,1};\overline{\btheta}_t) - h(\bx_{t,2};\overline{\btheta}_t) + 2\varepsilon'_{m_{\text{NN}},t} + \beta_T'\norm{\phi(\bx_{t,1}) - \phi(\bx_{t,2})}_{\bV_{t-1}^{-1}}\\
&\stackrel{(d)}{\leq} 2 h(\bx_{t,2};\overline{\btheta}_t) - 2 h(\bx_{t,1};\overline{\btheta}_t) + 2 \beta_T' \norm{\phi(\bx_{t,2}) - \phi(\bx_{t,1})}_{\bV_{t-1}^{-1}} + \\
&\qquad h(\bx_{t,1};\overline{\btheta}_t) - h(\bx_{t,2};\overline{\btheta}_t) + 6\varepsilon'_{m_{\text{NN}},t} + \beta_T'\norm{\phi(\bx_{t,1}) - \phi(\bx_{t,2})}_{\bV_{t-1}^{-1}}\\
&= h(\bx_{t,2};\overline{\btheta}_t) - h(\bx_{t,1};\overline{\btheta}_t) + 3 \beta_T' \norm{\phi(\bx_{t,1}) - \phi(\bx_{t,2})}_{\bV_{t-1}^{-1}} + 6\varepsilon'_{m_{\text{NN}},t}\\
&\stackrel{(e)}{\leq} 3 \beta_T' \norm{\phi(\bx_{t,1}) - \phi(\bx_{t,2})}_{\bV_{t-1}^{-1}} + 6\varepsilon'_{m_{\text{NN}},t}\\
% &\stackrel{(b)}{\leq} 2 \overline\btheta_t^\top \left( \phi(x^*) - \phi(\bx_{t,1}) \right) + 2 \frac{\beta_T}{\kappa_\mu}\norm{\phi(x^*) - \phi(\bx_{t,1})}_{\bV_{t-1}^{-1}} + \\
% &\qquad \overline\btheta_t^\top \left( \phi(\bx_{t,1}) - \phi(\bx_{t,2}) \right) + \frac{\beta_T}{\kappa_\mu}\norm{\phi(\bx_{t,1}) - \phi(\bx_{t,2})}_{\bV_{t-1}^{-1}}\\
% &\stackrel{(c)}{\leq} 2 \overline\btheta_t^\top \left( \phi(\bx_{t,2}) - \phi(\bx_{t,1}) \right) + 2 \frac{\beta_T}{\kappa_\mu}\norm{\phi(\bx_{t,2}) - \phi(\bx_{t,1})}_{\bV_{t-1}^{-1}} + \\
% &\qquad \overline\btheta_t^\top \left( \phi(\bx_{t,1}) - \phi(\bx_{t,2}) \right) + \frac{\beta_T}{\kappa_\mu}\norm{\phi(\bx_{t,1}) - \phi(\bx_{t,2})}_{\bV_{t-1}^{-1}}\\
% &\leq \overline\btheta_t^\top \left( \phi(\bx_{t,2}) - \phi(\bx_{t,1}) \right) + 3 \frac{\beta_T}{\kappa_\mu}\norm{\phi(\bx_{t,2}) - \phi(\bx_{t,1})}_{\bV_{t-1}^{-1}} \\
% &\stackrel{(d)}{\leq} 3 \frac{\beta_T}{\kappa_\mu}\norm{\phi(\bx_{t,1}) - \phi(\bx_{t,2})}_{\bV_{t-1}^{-1}} \\
\end{split}
\label{eq:upper:bound:inst:regret:neural}
\end{equation}

Step $(a)$ follows from Equation \ref{eq:diff:between:func:and:linear:approx:dueling}, step $(b)$ results from the triangle inequality, step $(c)$ has made use of Lemma \ref{lemma:bound:approx:error:linear:nn:duel}.
% Step $(a)$ follows from Lemma \ref{lemma:ucb:diff}. Step $(b)$ makes use of the triangle inequality.
Step $(d)$ follows from the way in which we choose the second arm $\bx_{t,2}$: $\bx_{t,2} = \arg\max_{\bx\in\mathcal{X}_t} h(\bx;\overline{\btheta}_t) + \left( \beta_T + B\sqrt{\frac{\lambda}{\kappa_\mu}} + 1 \right) \norm{\left(\phi(\bx) - \phi(\bx_{t,1})\right)}_{\bV_{t-1}^{-1}}$.
% : $\bx_{t,2} = \arg\max_{x\in\mathcal{X}_t} \overline\btheta_t^\top \left( \phi(x) - \phi(\bx_{t,1}) \right) + \frac{\beta_T}{\kappa_\mu}\norm{\phi(x) - \phi(\bx_{t,1})}_{\bV_{t-1}^{-1}}$.
Step $(e)$ results from the way in which we select the first arm: $\bx_{t,1} = \arg\max_{\bx\in\mathcal{X}_t} h(\bx;\overline{\btheta}_t)$.

Then we have
\begin{align}
    \sum_{t=T_0}^T r_t &\leq 3 \beta_T' \sum_{t=T_0}^T\norm{\phi(\bx_{t,1}) - \phi(\bx_{t,2})}_{\bV_{t-1}^{-1}} + 6T\varepsilon'_{m_{\text{NN}},T} \notag \\
    &=3 \beta_T'\sum_{t=T_0}^T\sum_{j\in[m]}\mathbb{I}\{i_t\in C_j\}\norm{\phi(\bx_{t,1}) - \phi(\bx_{t,2})}_{\bV_{t-1}^{-1}}  + 6T\varepsilon'_{m_{\text{NN}},T}\notag\\
    &\leq 3 \beta_T'\sqrt{\sum_{t=T_0}^T\sum_{j\in[m]}\mathbb{I}\{i_t\in C_j\}\sum_{t=T_0}^T\sum_{j\in[m]}\mathbb{I}\{i_t\in C_j\}\norm{\phi(\bx_{t,1}) - \phi(\bx_{t,2})}_{\bV_{t-1}^{-1}}^2} \notag\\
                &+ 6T\varepsilon'_{m_{\text{NN}},T}\notag\\
    &\leq 3 \beta_T' \sqrt{T\cdot m\cdot 16 \widetilde{d}} + 6T\varepsilon'_{m_{\text{NN}},T}\\
    &\leq 12 \beta_T' \sqrt{T\cdot m\cdot \widetilde{d}} + 6T\varepsilon'_{m_{\text{NN}},T}\,,
\end{align}
where in the second inequality we use the Cauchy-Swarchz inequality, and in the last step we use $\sum_{t=T_0}^T\sum_{j\in[m]}\mathbb{I}\{i_t\in C_j\}\leq T$ and Lemma \ref{lemma:concentration:square:std:neural}.
It can be easily verified that as long as the conditions on $m$ specified in \cref{eq:conditions:on:m} are satisfied (i.e., as long as the NN is wide enough), we have that 
$6T\varepsilon'_{m_{\text{NN}},T} \leq 1$.

Recall that $\beta_T' \triangleq \beta_T + B \sqrt{\lambda / \kappa_\mu} + 1$ and $\beta_T \triangleq \frac{1}{\kappa_\mu} \sqrt{ \widetilde{d} + 2\log(u/\delta)}$.
Therefore, finally, we have with probability at least $1-4\delta$
\begin{align}
    R_T & \leq T_0+ 12 (\beta_T + B \sqrt{\lambda / \kappa_\mu} + 1) \sqrt{T\cdot m\cdot \widetilde{d}} + 1\notag\\
    &\leq O\left(u(\frac{\widetilde{d}}{\kappa_\mu^2\tilde\lambda_x \gamma^2}+\frac{1}{\tilde\lambda_x^2})\log T+\left(\frac{\sqrt{\widetilde{d}}}{\kappa_\mu} + B\sqrt{\frac{\lambda}{\kappa_\mu}}\right)\sqrt{\widetilde{d}mT} \right)\notag\\
        &=O\left(\left(\frac{\sqrt{\widetilde{d}}}{\kappa_\mu} + B\sqrt{\frac{\lambda}{\kappa_\mu}}\right)\sqrt{\widetilde{d}mT} \right)\,.
\end{align}

% \include{./tex/ZSG/appendix}
% \include{./tex/mis/Appendix}
% \input{./tex/detection/nips_appendix}
% \input{./tex/AAAI/appendix}
% \input{./tex/AISTATS/appendix}
% \include{./tex/ICML2025dueling/Appendix}
% \include{./tex/hucb/appendix}
% \include{./tex/sybilsan/appendix}
% =======================================================

% bibliography
\newpage
\bibliography{merged_references} % ⚠️ 不要加 .bib 后缀！

% \printbibliography[heading=bibintoc]

\end{document}